\documentclass[10pt,a4paper,oneside,onecolumn]{article}
\usepackage{etex}

\pdfoutput=1

\usepackage{amsmath}
\usepackage[all]{xy}
\usepackage{cases}
\usepackage{times}
\usepackage{graphicx}
\usepackage{latexsym,footnote}
\usepackage[font=small,labelfont=bf]{caption}
\usepackage{algorithm}
\usepackage[noend]{algpseudocode}
\algrenewcommand\algorithmicrequire{\textbf{Input:}}
\algrenewcommand\algorithmicensure{\textbf{Output:}}
\usepackage{mathnotation}
\usepackage{array}
\usepackage{dashrule}
\usepackage[T1]{fontenc} 
\usepackage[draft]{fixme}
\usepackage{subfigure}
\usepackage{multirow}
\usepackage{tabularx}
\usepackage{cuted}
\usepackage{amssymb}
\usepackage{changepage}
\usepackage{enumerate}
\usepackage{amsthm}
\usepackage{color}															 
\usepackage[table]{xcolor}
\usepackage{booktabs}
\usepackage{arydshln}
\usepackage{cancel}
\usepackage{rotating}
\usepackage{framed}
\usepackage{float}
\usepackage{ marvosym }
\usepackage[a4paper]{geometry}
\usepackage{changepage}
\usepackage{url}
\usepackage{amsfonts}
\usepackage{wasysym}
\usepackage{environ}
\usepackage{tikz}

	\usepackage{pifont}
\usepackage{relsize}
\usepackage{enumitem}
\usepackage{hyperref}

\definecolor{light-gray1}{gray}{0.95}

\newcounter{examplecounter}
\newenvironment{example}{
    \refstepcounter{examplecounter}%
  
	\vspace{7pt}
	\noindent\textbf{Example \arabic{examplecounter}}%
  \quad
}{

\vspace{7pt}
}

\newcounter{remarkcounter}
\newenvironment{remark}{
    \refstepcounter{remarkcounter}%
  
	\vspace{7pt}
	\noindent\textbf{Remark \arabic{remarkcounter}}%
  \quad
}{

\vspace{7pt}
}

\newcommand{\tax}{\mathit{ax}}

\newcommand{\mo}{\mathcal{K}}

\newcommand{\mb}{\mathcal{B}}
\newcommand{\md}{\mathcal{D}}

\newcommand{\mc}{\mathcal{C}}

\newcommand{\Tp}{\mathit{P}}
\newcommand{\Tn}{\mathit{N}}
\newcommand{\tp}{\mathit{p}}
\newcommand{\tn}{\mathit{n}}

\newcommand{\ot}{\mathcal{K}^{*}}

\newcommand{\mD}{{\bf{D}}}
\newcommand{\allC}{{\bf{aC}}}
\newcommand{\minC}{{\bf{mC}}}
\newcommand{\minD}{{\bf{mD}}}
\newcommand{\allD}{{\bf{aD}}}
\newcommand{\dx}[1]{{\bf D}_{#1}^+}
\newcommand{\dnx}[1]{{\bf D}_{#1}^{-}}
\newcommand{\dz}[1]{{\bf D}_{#1}^0}
 
\newcommand{\qc}{\mathit{c}_q}
\newcommand{\uc}{c}

\newcommand{\ER}{{\mathit{ER}}}
\newcommand{\scHS}{{\textsc{HS}}}

\newcommand{\RQ}{{\mathit{R}}}

\newcommand{\SO}{{\mathbf{Sol}}}

\newcommand{\Queue}{{\mathsf{Queue}}}
\newcommand{\scQX}{{\textsc{QX}}}
\newcommand{\tr}{{\mathit{t}}}

\newcommand{\mQ}{{\bf{Q}}}
\newcommand{\all}{{\mathit{all}}}

\newcommand{\NHR}{{\mathit{NHR}}}
\newcommand{\LC}{{\mathit{LC}}}
\newcommand{\base}{{\mathit{B\! A\! S\! E}}}

\newcommand{\FP}{\mathbf{FP}}

\newcommand{\Pt}{\mathfrak{P}}
\newcommand{\Just}{\mathsf{Just}}
\newcommand{\DiscAx}{\mathsf{DiscAx}}
\newcommand{\adj}{\mathit{A\! F}}

\newtheorem{definition}{Definition}{}
\newtheorem{prob_def}{Problem Definition}{}
\newtheorem{proposition}{Proposition}{}
\newtheorem{lemma}{Lemma}{}
\newtheorem{corollary}{Corollary}{}
\newtheorem{property}{Property}

\newcommand{\subscript}[2]{$#1 _ #2$}

\begin{document}
\title{Towards Better Response Times and Higher-Quality Queries in Interactive Knowledge Base Debugging}

\author{Patrick Rodler \\
  \multicolumn{1}{p{.7\textwidth}}{\centering\emph{
  Alpen-Adria Universit\"at, Klagenfurt, 9020 Austria}} \\ 
	patrick.rodler@aau.at}







\maketitle

\begin{abstract}
Many artificial intelligence applications rely on knowledge about a relevant real-world domain that is encoded in a knowledge base (KB) by means of some logical knowledge representation language. The most essential benefit of such logical KBs is the opportunity to perform automatic reasoning to derive implicit knowledge or to answer complex queries about the modeled domain. The feasibility of meaningful reasoning requires a KB to meet some minimal quality criteria such as consistency or coherency. Without adequate tool assistance, the task of resolving such violated quality criteria in a KB can be extremely hard, especially when the problematic KB is complex, large in size, developed by multiple people or generated by means of some automatic systems. 

To this end, interactive KB debuggers have been introduced which solve soundness, completeness and scalability problems of non-interactive debugging systems. User interaction takes place in the form of queries asked to a person, e.g.\ a domain expert. A query is a number of (logical) statements and the user is asked whether these statements must or must not hold in the intended domain that should be modeled by the KB. To construct a query, a minimal set of two solution candidates, i.e.\ possible KB repairs, must be available. After the answer to a query is known, the search space for solutions is pruned. Iteration of this process until there is only a single solution candidate left yields a repaired KB which features exactly the semantics desired and expected by the user.

Existing interactive debuggers rely on 
strategies for query computation which 
often lead to the generation of many unnecessary query candidates and thus to a high number of expensive calls to logical reasoning services. Such an overhead can have a severe impact on the response time of the interactive debugger, i.e.\ the computation time between two consecutive queries. The actual problems of these approaches are (1)~the \emph{quantitative} nature of the query quality functions used to assess the goodness of queries (such functions provide more or less only a black-box to use in a trial-and-error search for acceptable queries), (2)~the necessity to actually \emph{compute} a query candidate in order to asses its goodness, (3)~the very expensive verification whether a candidate is a query, (4)~their inability to recognize candidates that are no queries before verification, (5)~the possibility of the generation of query duplicates and (6)~the strong dependence on the output of used reasoning services regarding the completeness and quality of the set of queries generated. 

To tackle these issues, we conduct an in-depth mathematical analysis of diverse quantitative active learning query selection measures published in literature in order to determine \emph{qualitative} criteria that make a query favorable from the viewpoint of a given measure. These qualitative criteria are our key to devise an efficient heuristic query computation process that solves all the mentioned problems of existing approaches. This proposed process involves a three-staged optimization of a query.

For the first stage, we introduce a new, theoretically well-founded and sound method for query generation that works completely \emph{without the use of logical reasoners}. This method is based on the notion of \emph{canonical queries}. Together with the developed heuristics, it enables to compute an (arbitrarily near to) optimal canonical query w.r.t.\ a given quality measure, e.g.\ information gain. For one canonical query, in general, multiple alternative queries with the same quality w.r.t.\ the given measure exist. 

To this end, we show that a hitting set tree search (second stage) can be employed to extract the best query among these w.r.t.\ additional criteria such as minimum cardinality or best understandability for the user. This search \emph{does not rely on logical reasoners} either. With existing methods, the extraction of such queries is not (reasonably) possible. They can just calculate \emph{any} subset-minimal query.

\emph{Consequently, this work for the first time proposes algorithms that enable a completely reasoner-free query generation for interactive KB debugging while at the same time guaranteeing optimality conditions of the generated query that existing methods cannot realize.}

In the third query optimization stage, which is optional, the one already computed query which optimizes a given quality measure and some of the additional criteria, can be enriched by further logical statements of simple and easily conceivable form and afterwards be minimized again. The reason of these optional steps, involving altogether only a polynomial number of reasoner calls, can be the simplification of the statements comprised in the query. The new approach we propose for accomplishing this improves the existing algorithm for query minimization insofar as it guarantees the finding of the query that is easiest to answer for the interacting user under plausible assumptions.

Furthermore, we study different relations between diverse active learning measures, e.g.\ superiority and equivalence relations. The obtained picture gives a hint about which measures are more favorable in which situation or which measures always lead to the same outcomes, based on given types of queries. 
\end{abstract}

\newpage
\tableofcontents
\clearpage
\listoffigures
\listoftables
\listofalgorithms
\addtocontents{loa}{\def\string\figurename{Algorithm}}
\newpage
\section{Introduction}
\label{sec:Introduction}

As nowadays artificial intelligence applications become more and more ubiquitous and undertake ever more critical tasks, e.g.\ in the medical domain or in autonomous vehicles, 
the quality of the underlying logical knowledge bases (KBs) becomes absolutely crucial. 
A concrete example of a repository containing KBs that are partly extensively used in practical applications is the Bioportal,\footnote{http://bioportal.bioontology.org} which comprises vast biomedical ontologies with tens or even hundreds of thousands of terms each. One of these huge ontologies is SNOMED-CT, the most comprehensive, multilingual clinical healthcare terminology in the world with over 316.000 terms which is currently used as a basis for diverse eHealth applications in more than fifty countries.\footnote{http://www.ihtsdo.org/snomed-ct} Such KBs pose a significant challenge for people as well as tools involved in their evolution, maintenance, application and quality assurance. 

Similar to programming languages which allow people to use syntax (instructions) to express a desired semantics, i.e.\ what the program should do, a KB consists of syntactical expressions (logical formulas) that should have an intended semantics, i.e.\ describe a domain of interest in a correct way. The quality of a KB is then usually measured in terms of how well the KB models the application domain. In more concrete terms, the quality criteria are constituted by a set of test cases (assertions that must and assertions that must not hold in the intended domain) and a set of general requirements (e.g.\ logical consistency). This is again an analogon to familiar practices known from the area of software engineering. There, test cases are exploited to verify the semantics of the program code, i.e.\ to check whether the program exhibits the expected behavior. A general quality requirement imposed on software might be that the code must be compilable, otherwise it is completely useless, and so is an inconsistent KB, which prohibits meaningful automatic reasoning to derive implicit knowledge or to answer queries about the modeled domain.

To cope with such faulty KBs (violating test cases and/or requirements), diverse systems for KB debugging \cite{Schlobach2007,Kalyanpur.Just.ISWC07,friedrich2005gdm,Horridge2008} have been developed. However, due to the complexity of the subproblems that must be solved by a KB debugger, e.g.\ the Hitting Set Problem \cite{karp1972} for localization of possible KB repairs or the SAT Problem \cite{karp1972} for reasoning and consistency checking (in case of a Propositional Logic KB, worse for more expressive logics, e.g.\ the Web Ontology Language OWL~2, for which reasoning is 2-$\textsc{NExpTime}$-complete~\cite{Grau2008a,Kazakov2008}), 
the usage of such non-interactive systems is often problematic. This is substantiated by \cite{Stuckenschmidt2008} where several (non-interactive) debugging systems were put to the test using faulty real-world KBs. The result was that most of the investigated systems had serious \emph{performance problems}, ran \emph{out of memory}, were not able to locate all the existing faults in the KB (\emph{incompleteness}), reported parts of a KB as faulty which actually were not faulty (\emph{unsoundness}), produced \emph{only trivial solutions} or suggested non-minimal faults (\emph{non-minimality}). 

Apart from these problems, a general problem of abductive inference (of which KB debugging is an instance \cite[p.~8]{Rodler2015phd}) is that there might be a potentially huge number of competing solutions or, respectively, explanations for the faultiness of a KB in KB debugging terms. These explanations are called \emph{diagnoses}. For example, in~\cite{ksgf2010} a sample study of real-world KBs revealed that the number of different (set-)minimal diagnoses might exceed thousand by far (1782 minimal diagnoses for a KB with only 1300 formulas). Without the option for a human (e.g.\ a domain expert) to take part actively in the debugging process to provide the debugging system with some guidance towards the correct diagnosis, the system cannot do any better than outputting the best (subset of all) possible solution(s) w.r.t.\ some criterion. An example of such a criterion is the maximal probability regarding some given, e.g.\ statistical, meta information about faults in the KB. 

However, any two solution KBs resulting from the KB repair using two different minimal diagnoses have different semantics in terms of entailed and non-entailed formulas \cite[Sec.~7.6]{Rodler2015phd}. Additionally, the available fault meta information might be misleading \cite{Rodler2013,Shchekotykhin2012} resulting in completely unreasonable solutions proposed by the system. 
Selecting a wrong diagnosis can lead to unexpected entailments or non-entailments, lost desired entailments and surprising future faults when the KB is further developed. Manual inspection of a large set of (minimal) diagnoses is time-consuming, error-prone, and often simply practically infeasible or computationally infeasible due to the complexity of diagnosis computation.

As a remedy for this dilemma, interactive KB debugging systems \cite{ksgf2010,Shchekotykhin2012,Rodler2013,Rodler2015phd,Shchekotykhin2014} were proposed. These aim at the gradual reduction of compliant minimal diagnoses by means of user interaction, thereby seeking to prevent the search tree for minimal diagnoses from exploding in size by performing regular pruning operations. Throughout an interactive debugging session, the user, usually a (group of) KB author(s) and/or domain expert(s), is asked a set of automatically chosen queries about the intended domain that should be modeled by a given faulty KB. A query can be created by the system after a set $\mD$ comprising a minimum of two minimal diagnoses has been precomputed (we call $\mD$ the \emph{leading diagnoses}). 
Each query is a set (i.e.\ a conjunction) of logical formulas that are entailed by some non-faulty subset of the KB. 
In particular, a set of logical formulas is a query iff any answer to it enables a \emph{search space restriction} (i.e.\ some diagnoses are ruled out) and guarantees the \emph{preservation of some solution} (i.e.\ not all diagnoses must be ruled out).
With regard to one particular query $Q$, any set of minimal diagnoses for the KB, in particular the set $\mD$ which has been utilized to generate $Q$, can be partitioned into three sets, the first one ($\dx{}$) including all diagnoses in $\mD$ compliant only with a positive answer to $Q$, the second ($\dnx{}$) including all diagnoses in $\mD$ compliant only with a negative answer to $Q$, and the third ($\dz{}$) including all diagnoses in $\mD$ compliant with both answers. A positive answer to $Q$ signalizes that the conjunction of formulas in $Q$ must be entailed by the correct KB why $Q$ is added to the set of positive test cases. Likewise, if the user negates $Q$, this is an indication that at least one formula in $Q$ must not be entailed by the correct KB. As a consequence, $Q$ is added to the set of negative test cases.

Assignment of a query $Q$ to either set of test cases results in a new debugging scenario. In this new scenario, all elements of $\dnx{}$ (and usually further diagnoses that have not yet been computed) are no longer minimal diagnoses given that $Q$ has been classified as a positive test case. Otherwise, all diagnoses in $\dx{}$ (and usually further diagnoses that have not yet been computed) are invalidated. In this vein, the successive reply to queries generated by the system will lead the user to the single minimal solution diagnosis that perfectly reflects their intended semantics. In other words, after deletion of all formulas in the solution diagnosis from the KB and the addition of the conjunction of all formulas in the specified positive test cases to the KB, the resulting KB meets all requirements and positive as well as negative test cases. In that, the added formulas contained in the positive test cases serve to replace the desired entailments that are broken due to the deletion of the solution diagnosis from the KB. 

Thence, in the interactive KB debugging scenario the user is not required to cope with the understanding of which faults (e.g.\ sources of inconsistency or implications of negative test cases) occur in the faulty initial KB, why these are faults (i.e.\ why particular entailments are given and others not) and how to repair them. All these tasks are undertaken by the interactive debugging system. The user is just required to answer queries whether certain assertions should or should not hold \emph{in the intended domain}.

The schema of an interactive debugging system is pictured by Figure~\ref{fig:interactive_debugging_workflow}. The system receives as input a \emph{diagnosis problem instance (DPI)} consisting of 
\begin{itemize}
	\item some KB $\mo$ formulated using some (monotonic) logical language $\mathcal{L}$ (every formula in $\mo$ might be correct or faulty),
	\item (optionally) some KB $\mb$ (over $\mathcal{L}$) formalizing some background knowledge relevant for the domain modeled by $\mo$ (such that $\mb$ and $\mo$ do not share any formulas; all formulas in $\mb$ are considered correct)
	\item a set of requirements $\RQ$ to the correct KB,
	\item sets of positive ($\Tp$) and negative ($\Tn$) test cases (over $\mathcal{L}$) asserting desired semantic properties of the correct KB
\end{itemize}
and (optionally) some fault information, e.g.\ in terms of fault probabilities of logical formulas in $\mo$.
Further on, a range of additional parameters $\mathsf{Tuning\;Params}$ might be provided to the system. These serve as a means to fine-tune the system's behavior in various aspects. For an explanation of these parameters the reader is referred to \cite{Rodler2015phd}. The system outputs a repaired solution KB which is built from a minimal diagnosis which has a sufficiently high probability where some predefined parameter in $\mathsf{Tuning\;Params}$ determines what ``sufficiently'' means. Note that we address debugging systems in this work which use the reasoning services provided by a logical reasoner in a black-box manner. That is, it is assumed that no influence can be taken on the way reasoning is performed which means in particular that no operations instrumental to the debugging task as such must be conducted by the reasoner (cf.\ \cite[p.~6]{Rodler2015phd} \cite{kalyanpur2005}).

The workflow implemented by existing interactive KB debugging systems illustrated by Figure~\ref{fig:interactive_debugging_workflow} is the following:
\begin{enumerate}
	\item A set of leading diagnoses $\mD$ is computed by the diagnosis engine (by means of the fault information, if available) using the logical reasoner and passed to the query generation module.
	\item The query generation module computes queries exploiting the set of leading diagnoses and delivers them to the query selection module.
	\item The query selection module is responsible for filtering out the best or a sufficiently good query according to some query quality measure (QQM) $m$ (often by means of the fault information, if available). This ``best query'' is shown to the interacting user.
	\item The user submits an answer to the query. (Alternatively, the user might reject the query and request surrogate queries as long as there are alternative queries left.)
	\item The query along with the given answer is used to formulate a new test case.
	\item This new test case is transferred back to the diagnosis engine and taken into account in prospective iterations. If some predefined (stop) criterion (e.g.\ sufficiently high probability of one leading diagnosis) is not met, another iteration starts at step 1. Otherwise, the solution KB $\ot$ constructed from the currently best minimal diagnosis w.r.t.\ the given criterion is output.
\end{enumerate}

\begin{figure*}[t]
	\centering
		\includegraphics[width=0.85\textwidth]{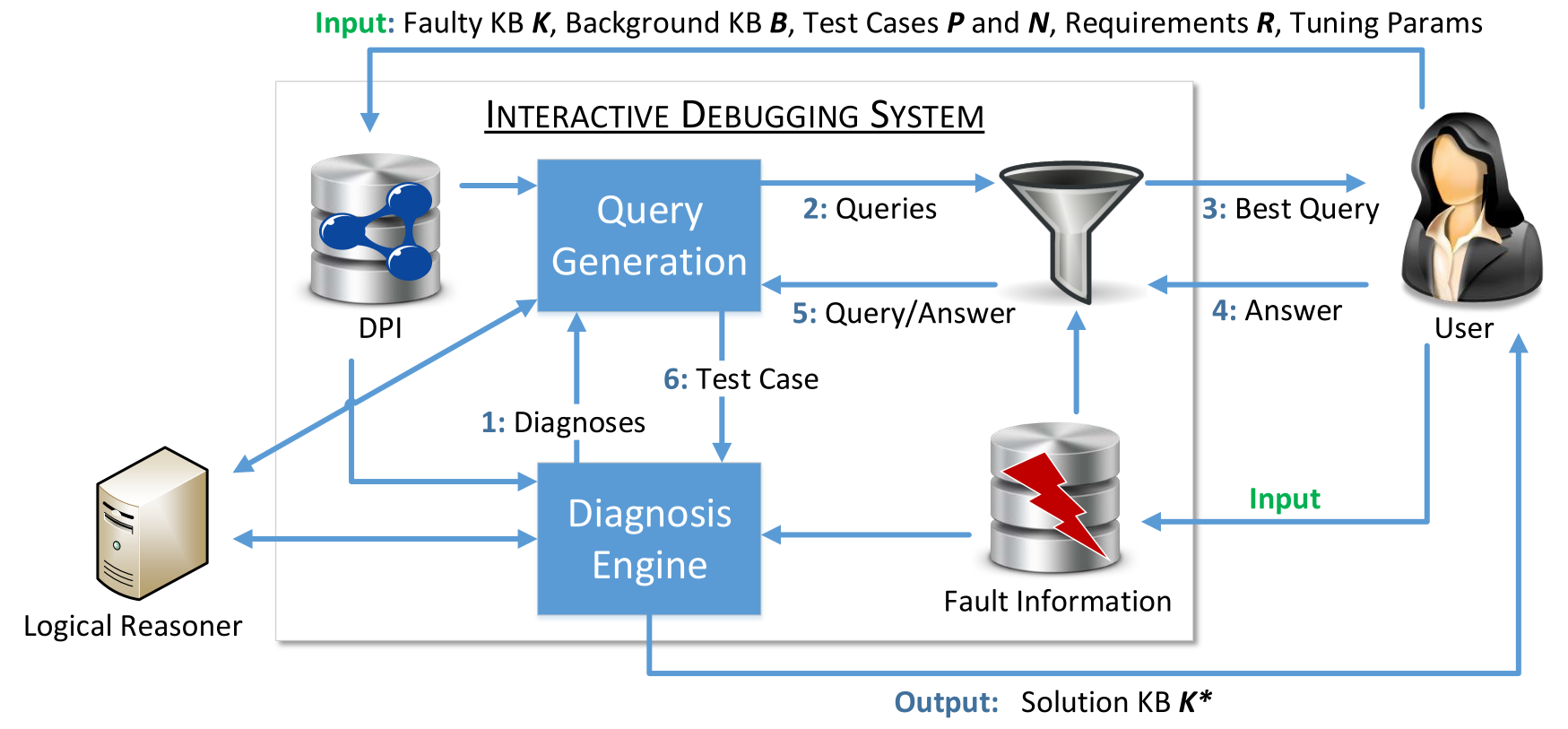}
	\caption[The Principle of Interactive KB Debugging]{The principle of interactive KB debugging \cite{Rodler2015phd}.}
	\label{fig:interactive_debugging_workflow}
\end{figure*}

In this work, we want to lay the focus on steps 2 and 3 in this process. These are usually \cite{Rodler2013,Rodler2015phd,ksgf2010} implemented as a \emph{pool-based} approach, involving an initial precomputation of a pool of possible queries and a subsequent selection of the best query from this pool. In \cite[Sec.~8.6]{Rodler2015phd}, a range of shortcomings of this query computation strategy has been identified and discussed. Another \emph{search-based} implementation \cite{Shchekotykhin2012} employs an exhaustive search coupled with some heuristics to filter out (sufficiently) good queries faster. This idea goes in the same direction as the methods we will present in this work, but our approach will optimize the approach of \cite{Shchekotykhin2012} along various dimensions. Further, one might use a \emph{stream-based} implementation where queries are generated in a one-by-one manner and directly tested for sufficient quality w.r.t.\ a given QQM.

All of these strategies suffer from the following shortcomings: 
\begin{enumerate}[label=(\alph*)]
	\item \label{enum:drawback_of_std_approach_1} These approaches usually employ QQMs for query quality estimation \emph{directly}. These QQMs are usually of \emph{quantitative} nature, i.e.\ constitute real-valued functions over the set of possible queries which need to be optimized, i.e.\ minimized or maximized. However, without further analysis such functions often provide only a black-box to use in a trial-and-error search for acceptable queries.
	\item In order to assess the goodness of a query, these approaches require the actual \emph{computation} of the respective query.
	\item The verification whether a candidate, i.e.\ a set of logical formulas, is indeed a query is computationally very expensive with these approaches.
	\item These approaches are not able to recognize candidates that are in fact no queries before running the verification.
	\item The generation of large numbers of query duplicates is possible with these approaches, i.e.\ a significant amount of unnecessary computations might be performed.
	\item \label{enum:drawback_of_std_approach_6} These approaches strongly depend on the output of the used reasoning services with regard to the completeness and quality of the set of queries generated. For instance, they might fail to take into account favorable queries (with good discrimination properties) at the cost of computing ones that are proven worse.
\end{enumerate}
Independent of the actually used strategy, e.g.\ pool-, search- or stream-based, we will henceforth call all query computation strategies (i.e.\ implementations of steps 2 and 3 in Figure~\ref{fig:interactive_debugging_workflow}) exhibiting some or all of these drawbacks \ref{enum:drawback_of_std_approach_1} -- \ref{enum:drawback_of_std_approach_6} ``standard strategy'', \emph{Std} for short.

The rest of this work provides a well-founded theoretical framework and algorithms built upon it that serve to resolve all mentioned problems of Std. 
In particular, we aim on the one hand at \emph{improving the performance (time and space)} of the query computation component. Since debugging systems like the one considered in this work use the logical reasoner as a black-box and since reasoning costs during a debugging session amount to a very high ratio (up to $93.4\%$ reported in \cite{Shchekotykhin2014}) of overall debugging costs, it is material to \emph{reduce the number of reasoner calls} in order to achieve a performance gain in KB debugging. For that reason we will present methods that enable to generate an optimized query without any calls to a reasoner.
Another goal is the \emph{minimization of the overall effort for the user} both in terms of the number as well as in terms of the size and difficulty of queries that a user needs to answer in order to filter out the correct diagnosis. To this end, we will provide algorithms that minimize the probability of getting the wrong answer by the user under plausible assumptions and we will describe various query selection measures that aim at minimizing the remaining number of queries that need to be answered by the user in order to get a sufficiently optimal repaired KB.
%
A brief outline of our basic ideas how to resolve the problems of Std is given next.
First, due to the generally exponential number of possible queries w.r.t.\ to a given set of leading diagnoses $\mD$, Std exhibits a worst case 
exponential time and (potentially) space complexity regarding the parameter $|\mD|$. If the decision is to search through only a small subset of queries to save resources, no guarantees about the quality (regarding the QQM $m$) of the returned query can be given. For instance, in the worst case one might happen to capture a subset of cardinality $x$ including the $x$ least favorable queries w.r.t.\ $m$ using this approach which will of course severely affect query quality.

As a key to achieving (arbitrarily) high query quality and an exploration of a (usually) small number of queries without performing an \emph{uninformed} preliminary restriction of the search space, we suggest a heuristic-search-based strategy which enables a \emph{direct} search for a query with optimal $m$ based on some heuristics. In order to derive suitable heuristics from a quantitative QQM $m$, we transform $m$ into a set of \emph{qualitative} criteria that must be met in order to optimize a query w.r.t.\ $m$. From these qualitative criteria, most of the heuristics can be extracted in a straightforward way. 
Whereas in literature only three QQM measures, namely split-in-half, entropy and RIO \cite{Shchekotykhin2012, Rodler2013, Rodler2015phd, Shchekotykhin2014, ksgf2010}, have been analyzed and discussed, we investigate numerous further (active learning \cite{settles2012}) measures that might be favorably employed in interactive debuggers and deduce heuristics from them. We conduct in-depth mathematical analyses regarding all these measures, compare their properties and derive superiority relationships as well as equivalence relations between them. The obtained results provide guidance which measure(s) to use under which circumstances. Due to their generality, these results apply to all active learning scenarios where binary queries are used, e.g.\ when the learner tries to find a binary classifier.

Second, a crucial problem of Std is the extensive use of reasoning services ($O(|\mD|)$ calls for each query) which must be considered computationally very expensive in the light of the fact that reasoning for (the relatively inexpressive) Propositional Logic is already $\mathrm{NP}$-complete. Reasoners are (1)~applied by Std to calculate a query candidate as a set of entailments of a non-faulty part of the faulty KB, and (2)~to compute a so called \emph{q-partition}, a (unique) partition of the leading diagnoses $\mD$, for each query candidate. One the one hand, the q-partition serves to test whether a query candidate is indeed a query, and on the other hand it is a prerequisite for the determination of the goodness of a query w.r.t.\ a QQM. 

Our idea to avoid reasoner dependence is to compute some well-defined \emph{canonical query} which facilitates the construction of the associated (canonical) q-partition without reasoning aid by using just set operations and comparisons.
In this vein, guided by the qualitative criteria mentioned before, a search for the best q-partition, and hence the best query w.r.t.\ a QQM, can be accomplished \emph{without reasoning at all}. Side effects of the reliance upon canonical queries is the natural and proven guarantee that 
no computation time can be wasted (1)~for the generation of query candidates which turn out to be no queries, (2)~for queries that have unfavorable discrimination properties w.r.t.\ the leading diagnoses, i.e.\ for which some leading diagnoses remain valid independently of the query answer, and (3)~for computing duplicate queries or q-partitions. Such potential overheads (1) and (3) cannot be avoided in Std approaches at all, and (2) only by means of postulations about the used reasoning engine (cf.\ \cite[Prop.~8.3, p.~107]{Rodler2015phd}).

Having found a sufficiently good canonical q-partition, as a next step the associated canonical query can usually be reduced to a query of smaller size, including fewer logical formulas. For there are generally exponentially many different queries w.r.t.\ a fixed q-partition (even if we neglect syntactic variants of queries with equal semantics). Whereas Std methods require $O(k \log \frac{n}{k})$ calls to a reasoning engine for the minimization of a single query (of size $n$) and cannot guarantee that the minimized query (of size $k$) satisfies certain desired properties such as 
minimal cardinality,
our canonical-query-based approach is able to extract minimal queries in a best-first manner according to the given desired properties \emph{without reasoning at all}.
For instance, given some information about faults the interacting user is prone to, this could be exploited to estimate how well this user might be able to understand and answer a query. Assume the case that the user frequently has problems to apply $\exists$ in a correct manner to express what they intend to express, but has never made any mistakes in formulating implications $\rightarrow$. Then the First-order Logic query $Q_1 = \setof{\forall X\,p(X) \rightarrow q(X)}$ will probably be better comprehended than $Q_2 = \setof{\forall X \exists Y s(X,Y)}$. This principle can be taken advantage of to let our algorithm find the query with minimal estimated answer fault probability.
Hence, among all (possibly exponentially many) queries with the same q-partition -- and thus optimal properties as per some QQM -- the presented algorithms can find the best query according to some secondary criteria in a very efficient way. Neither the compliance with any secondary criteria nor the achieved efficiency can be granted by Std approaches.

Optionally, if requested by the user, the best minimized query can be enriched with additional logical formulas with a predefined simple syntax, e.g.\ atoms or definite clauses in case of a Propositional Logic KB. The aim of this step might be the better comprehension of query formulas for the user. For example, users might have a strong bias towards considering formulas authored by themselves as correct (because one usually acts to the best of their knowledge). Moreover, KB formulas occurring in the best minimized query might be of quite complex form, e.g.\ including many quantifiers. Therefore, it might be less error-prone as well as easier and more convenient for the user to be presented with simple-to-understand formulas instead of more complex ones. We demonstrate how such a query enrichment can be realized in a way the optimality (q-partition) of the query w.r.t.\ the given QQM is preserved.

Finally, we show how a well-known divide-and-conquer technique \cite{junker04} can be adapted to compute a subset-minimal reduction of the enriched query (keeping the q-partition and optimality w.r.t.\ the QQM constant) which is easiest to comprehend for the user (assessed based on given fault information) among all subset-minimal reductions under plausible assumptions. The output of the entire query computation and optimization process is then a subset-minimal query which is both optimal w.r.t.\ the given QQM (aiming at the minimization of the overall number of queries until the debugging problem is solved) and w.r.t.\ user comprehension or effort, and which is simplified w.r.t.\ the formula complexity.

All algorithms introduced in this work to accomplish the described query computation and optimization are described in much detail, formally proven correct and illustrated by examples.

The rest of this work is organized as follows. Section~\ref{sec:Preliminaries} represents an introductory section which provides the reader with our main assumptions (Section~\ref{sec:assumptions}), e.g.\ monotonicity of the used logic, and with extensively exemplified basics on KB Debugging (Section~\ref{sec:KnowledgeBaseDebuggingBasics}) and Interactive KB Debugging (Section~\ref{sec:InteractiveKnowledgeBaseDebuggingBasics}). Subsequently, Section~\ref{sec:QueryComputation} delves into the details of the new query computation strategies. After giving some preliminaries and definitions required for our formal analysis 
(Section~\ref{sec:QPartitionSelectionMeasures}), various existing (Section~\ref{sec:ExistingActiveLearningMeasuresForKBDebugging}) and new (Section~\ref{sec:NewActiveLearningMeasuresForKBDebugging}) active learning QQM measures as well as a reinforcement learning query selection strategy (Section~\ref{sec:rio}) are analyzed, discussed in detail, compared with each other and transformed into qualitative criteria. 
These enable the derivation of heuristics for each QQM to be employed in the q-partition search (Section~\ref{sec:QPartitionRequirementsSelection}). Section~\ref{sec:FindingOptimalQPartitions} addresses the finding of an optimal q-partition by means of the notion of canonical queries. Initially, some useful theoretical results are derived that provide a basis for our formally defined heuristic search technique. Next, we discuss the completeness of the proposed search and specify optimality conditions, conditions for best successor nodes and pruning conditions for all discussed QQM measures based on the qualitative criteria established earlier. Further on, we deal with the identification of the best query for the optimal q-partition w.r.t.\ a secondary criterion (Section~\ref{sec:FindingOptimalQueriesGivenAnOptimalQPartition}). Then we explain how a q-partition-preserving query enrichment can be executed (Section~\ref{sec:EnrichmentOfAQuery}) and, finally, we demonstrate a technique to extract an optimized query from the enriched query (Section~\ref{sec:QueryOptimization}). Section~\ref{sec:Conclusion} summarizes our results and concludes.

\clearpage
\section{Preliminaries}
\label{sec:Preliminaries}
In this section, we guide the reader through the basics of (interactive) KB debugging, list the assumptions made and illustrate the concepts by several examples.
\subsection{Assumptions}
\label{sec:assumptions}
The techniques described in this work are applicable to any logical knowledge representation formalism $\mathcal{L}$ for which 
the entailment relation is
\begin{enumerate}[label=(\subscript{L}{{\arabic*}})]
	\item \label{logic:cond1} \emph{monotonic:} is given when adding a new logical formula to a KB $\mo_{\mathcal{L}}$ cannot invalidate any entailments of the KB, i.e. $\mo_{\mathcal{L}} \models \alpha_{\mathcal{L}}$ implies that $\mo_{\mathcal{L}} \cup \setof{\beta_{\mathcal{L}}} \models \alpha_{\mathcal{L}}$,
	\item \label{logic:cond2} \emph{idempotent:} is given when adding implicit knowledge explicitly to a KB $\mo_{\mathcal{L}}$ does not yield new entailments of the KB, i.e. $\mo_{\mathcal{L}} \models \alpha_{\mathcal{L}}$ and $\mo_{\mathcal{L}} \cup \setof{\alpha_{\mathcal{L}}} \models \beta_{\mathcal{L}}$ implies $\mo_{\mathcal{L}} \models \beta_{\mathcal{L}}$ and
	\item \label{logic:cond3} \emph{extensive:} is given when each logical formula entails itself, i.e. $\{\alpha_{\mathcal{L}}\} \models \alpha_{\mathcal{L}}$ for all $\alpha_{\mathcal{L}}$,
\end{enumerate}   
and for which 
\begin{enumerate}[label=(\subscript{L}{{\arabic*}})]
\setcounter{enumi}{3}
	\item \label{logic:cond4} reasoning procedures for \emph{deciding consistency} and \emph{calculating logical entailments} of a KB are available, 
\end{enumerate}
where $\alpha_{\mathcal{L}}, \beta_{\mathcal{L}}$ are logical sentences and $\mo_{\mathcal{L}}$ is a set $\setof{\tax_\mathcal{L}^{(1)},\dots,\tax_\mathcal{L}^{(n)}}$ of logical formulas formulated over the language $\mathcal{L}$. $\mo_{\mathcal{L}}$ is to be understood as the conjunction $\bigwedge_{i=1}^{n} \tax_\mathcal{L}^{(i)}$.\footnote{Please note that we will omit the index $\mathcal{L}$ for brevity when referring to formulas or KBs over $\mathcal{L}$ throughout the rest of this work. However, whenever $\mathcal{L}$ appears in the text (e.g.\ in proofs), we bear in mind that any KB we speak of is formulated over $\mathcal{L}$ and that $\mathcal{L}$ meets all the conditions given above.}  
%
Examples of logics that comply with these requirements include, but are not restricted to Propositional Logic, Datalog~\cite{Ceri1989a}, (decidable fragments of) First-order Predicate Logic, The Web Ontology Language (OWL~\cite{patel2004owl}, OWL~2~\cite{Grau2008a,Motik2009a}), sublanguages thereof such as the OWL~2 EL Profile (with polynomial time reasoning complexity \cite{kazakov2014}) or various Description Logics~\cite{Baader2007}.

\renewcommand{\arraystretch}{1.4} 
\begin{table}[t]
	\footnotesize
	\centering
		\rowcolors[]{2}{gray!8}{gray!16} 
		\begin{tabular}{ c c c c } 
			\rowcolor{gray!40}
			\toprule\addlinespace[0pt]
			$i$ & $\tax_i$ & $\mo$ & $\mb$  \\ \addlinespace[0pt]\midrule\addlinespace[0pt]
			1 & $\lnot H \lor \lnot G$ & $\bullet$ & 	\\
			2 & $X \lor F \to H$ & $\bullet$ &  	\\
			3 & $E \to \lnot M \land X$ & $\bullet$ &  	\\
			4 & $A \to \lnot F$ & $\bullet$ &  	\\
			5 & $K \to E$ & $\bullet$ &  	\\
			6 & $C \to B$ & $\bullet$ &  	\\
			7 & $M \to C \land Z$ & $\bullet$ &  	\\
			8 & $H \to A$ &  &  $\bullet$	\\
			9 & $\lnot B \lor K$ &  &  $\bullet$	\\
			\addlinespace[0pt]\bottomrule 
			\rowcolor{gray!40}
			$i$ & \multicolumn{3}{c}{$\tp_i\in\Tp$} \\ \addlinespace[0pt]\midrule\addlinespace[0pt]
			$1$ & \multicolumn{3}{c}{$\lnot X \to \lnot Z$} 	\\ \addlinespace[0pt]\toprule\addlinespace[0pt]
			\rowcolor{gray!40}
			$i$ & \multicolumn{3}{c}{$\tn_i\in\Tn$} \\ \addlinespace[0pt]\midrule\addlinespace[0pt]
			$1$ & \multicolumn{3}{c}{$M \to A$} 	\\ \addlinespace[0pt]
			$2$ & \multicolumn{3}{c}{$E \to \lnot G$} 	\\ \addlinespace[0pt]
			$3$ & \multicolumn{3}{c}{$F \to L$} 	\\ \addlinespace[0pt]
\toprule\addlinespace[0pt]
			\rowcolor{gray!40}
			$i$ & \multicolumn{3}{c}{$r_i\in\RQ$} \\ \addlinespace[0pt]\midrule\addlinespace[0pt]
			$1$ & \multicolumn{3}{c}{consistency} \\ \addlinespace[0pt]\bottomrule
			\end{tabular}
	\caption{Propositional Logic Example DPI}
	\label{tab:example_dpi_0}
\end{table}


\subsection{Knowledge Base Debugging: Basics}
\label{sec:KnowledgeBaseDebuggingBasics}

KB debugging can be seen as a test-driven procedure comparable to test-driven software development and debugging. The process involves the specification of test cases to restrict the possible faults until the user detects the actual fault manually or there is only one (highly probable) fault remaining which is in line with the specified test cases. 
 
\paragraph{The KB Debugging Problem.} The inputs to a KB debugging problem can be characterized as follows:
Given is a KB $\mo$ and a KB $\mb$ (background knowledge), both formulated over some logic $\mathcal{L}$ complying with the conditions~\ref{logic:cond1} -- \ref{logic:cond4} given above.  
All formulas in $\mb$ are considered to be correct and all formulas in $\mo$ are considered potentially faulty. $\mo \cup \mb$ does not meet postulated requirements 
$\RQ$ where $\setof{\text{consistency}} \subseteq \RQ \subseteq \setof{\text{coherency, consistency}}$
or does not feature desired semantic properties, called test cases.\footnote{We assume consistency a minimal requirement to a solution KB provided by a debugging system, as inconsistency makes a KB completely useless from the semantic point of view.} Positive test cases (aggregated in the set $\Tp$) correspond to desired entailments and negative test cases (aggregated in the set $\Tn$) represent undesired entailments of the correct (repaired) KB (together with the background KB $\mb$). \label{etc:test_cases_are_sets_or_conjuntions_of_formulas} 
Each test case $\tp \in \Tp$ and $\tn \in \Tn$ is \emph{a set of} logical formulas over $\mathcal{L}$. The meaning of a positive test case $\tp \in \Tp$ is that the union of the correct KB and $\mb$ must entail each formula (or the conjunction of formulas) in $\tp$, whereas a negative test case $\tn \in \Tn$ signalizes that some formula (or the conjunction of formulas) in $\tn$ must not be entailed by this union.
\begin{remark}\label{rem:entailments_as_sets_of_formulas}
In the sequel, we will write $\mo \models X$ for some set of formulas $X$ to denote that $\mo \models \tax$ for all $\tax \in X$ and $\mo \not\models X$ to state that $\mo \not\models \tax$ for some $\tax \in X$. When talking about a singleton $X = \setof{x_1}$, we will usually omit the brackets and write $\mo \models x_1$ or $\mo \not\models x_1$, respectively.\qed
\end{remark} 
The described inputs to the KB debugging problem are captured by the notion of a diagnosis problem instance:
\begin{definition}[Diagnosis Problem Instance]\label{def:dpi}
Let 
\begin{itemize}
	\item $\mo$ be a KB over $\mathcal{L}$,
	\item $\Tp, \Tn$ sets including sets of formulas over $\mathcal{L}$,
	\item $\setof{\text{consistency}}\subseteq \RQ \subseteq \setof{\text{coherency, consistency}}$,\footnote{Coherency was originally defined for Description Logic (DL) KBs~\cite{Schlobach2007,Parsia2005} and postulates that there is no concept $C$ in a DL KB $\mo_{DL}$ such that $\mo_{DL} \models C \sqsubseteq \bot$. Translated to a First-order Predicate Logic KB  $\mo_{FOL}$ and generalized to not only unary predicate symbols, this means that there must not be any $k$-ary predicate symbol for $k \geq 1$ in $\mo_{FOL}$ such that $\mo_{FOL} \models \forall X_1,\dots,X_k \, \lnot p(X_1,\dots,X_k)$.} 
	\item $\mb$ be a KB over $\mathcal{L}$ such that $\mo \cap \mb = \emptyset$ and $\mb$ satisfies all requirements $r \in \RQ$, and
	\item the cardinality of all sets $\mo$, $\mb$, $\Tp$, $\Tn$ be finite.
\end{itemize}
Then we call the tuple $\langle\mo,\mb,\Tp,\Tn\rangle_\RQ$ a \emph{diagnosis problem instance (DPI) over $\mathcal{L}$}.\footnote{
In the following we will often call a DPI over $\mathcal{L}$ simply a DPI for brevity and since the concrete logic will not be relevant to our theoretical analyses as long as it is compliant with the conditions \ref{logic:cond1} -- \ref{logic:cond4} given at the beginning of this section. Nevertheless we will mean exactly the logic over which a particular DPI is defined when we use the designator $\mathcal{L}$.
}
\end{definition}

\begin{example}\label{ex:dpi}
An example of a Propositional Logic DPI is depicted by Table~\ref{tab:example_dpi_0}. This DPI will serve as a running example throughout this paper. It includes seven KB axioms $\tax_1,\dots,\tax_7$ in $\mo$, two axioms $\tax_8,\tax_9$ belonging to the background KB, one singleton positive test case $\tp_1$ and three singleton negative test cases $\tn_1,\dots,\tn_3$. There is one requirement imposed on the correct (repaired) KB, namely $r_1$, which corresponds to \emph{consistency}. It is easy to verify that the standalone KB $\mb = \setof{\tax_8,\tax_9}$ is consistent, i.e.\ satisfies all $r \in \RQ$ and that $\mo \cap \mb = \emptyset$. Note that we omitted the set brackets for the test cases, i.e.\ lines for $\tp_i$ or $\tn_i$ might comprise more than one formula in case the respective test case is not a singleton.\qed
\end{example}

A solution (KB) for a DPI is characterized as follows:
\begin{definition}[Solution KB]\label{def:solution_KB} Let $\langle\mo,\mb,\Tp,\Tn\rangle_\RQ$ be a DPI. Then a KB $\ot$ is called \emph{solution KB w.r.t. $\langle\mo,\mb,\Tp,\Tn\rangle_\RQ$}, written as $\ot \in \SO_{\langle\mo,\mb,\Tp,\Tn\rangle_\RQ}$, iff all the following conditions hold:
\begin{eqnarray}
		 \forall \, r  \in \RQ&:& \;\ot \cup \mb \,\text{ fulfills }\, r  \label{e:1} \\ 
		 \forall \,\tp \in \Tp&:& \;\ot \cup \mb \,\models\, \tp					\label{e:2} \\ 
		 \forall \,\tn \in \Tn&:& \;\ot \cup \mb \,\not\models\, \tn .		\label{e:3}  
\end{eqnarray}
A solution KB $\ot$ w.r.t. a DPI is called \emph{maximal}, written as $\ot \in \SO^{\max}_{\langle\mo,\mb,\Tp,\Tn\rangle_\RQ}$, iff there is no solution KB $\mo'$ such that $\mo' \cap \mo \supset \ot\cap\mo$. 
\end{definition}
\begin{example}\label{ex:solution_KB}
For the DPI given by Table~\ref{tab:example_dpi_0}, $\mo = \setof{\tax_1,\dots,\tax_7}$ is not a solution KB w.r.t.\ $\langle\mo,\mb,\Tp,\Tn\rangle_\RQ$ since, e.g.\, $\mo \cup \mb = \setof{\tax_1,\dots,\tax_9} \not\models \tp_1$ which is a positive test case and therefore has to be entailed. This is easy to verify since $\tp_1$ is equivalent to $Z \to X$, but $Z$ does not occur on any left-hand side of a rule in $\mo \cup \mb$. Another reason why $\mo = \setof{\tax_1,\dots,\tax_7}$ is not a solution KB w.r.t.\ the given DPI is that $\mo \cup \mb \supset \setof{\tax_1,\dots,\tax_3} \models \tn_2$, which is a negative test case and hence must not be an entailment. This is straightforward since $\setof{\tax_1,\dots,\tax_3}$ imply $E \to X$, $X \to H$ and $H \to \lnot G$ and thus clearly $\tn_2 = \setof{E \to \lnot G}$.

On the other hand, $\mo_a^* := \setof{} \cup \setof{Z \to X}$ is clearly a solution KB w.r.t.\ the given DPI as $\setof{Z \to X} \cup \mb$ is obviously consistent (satisfies all $r\in\RQ$), does entail $\tp_1 \in \Tp$ and does not entail any $\tn_i \in \Tn, (i \in \setof{1,\dots,3})$. However, $\mo_a^*$ is not a maximal solution KB since, e.g.\, $\tax_5 = K \to E \in \mo$ can be added to $\mo_a^*$ without resulting in the violation of any of the conditions (\ref{e:1})-(\ref{e:3}) stated by Definition~\ref{def:solution_KB}. Please note that also, e.g.\ $\setof{\lnot X \to \lnot Z, A_1 \to A_2, A_2 \to A_3, \dots, A_{k-1}\to A_k}$ for arbitrary finite $k \geq 0$ is a solution KB, although it has no axioms in common with $\mo$. This is the reason why Parsimonious KB Debugging (see below) searches for some maximal solution KB w.r.t.\ to a DPI which has a set-maximal intersection with the faulty KB among all solution KBs.

Maximal solution KBs w.r.t.\ the given DPI are, e.g.\, $\mo_b^* := \setof{\tax_1,\tax_4,\tax_5,\tax_6,\tax_7,\tp_1}$ (resulting from deletion of $\setof{\tax_2,\tax_3}$ from $\mo$ and addition of $\tp_1$) or $\mo_c^* := \setof{\tax_1,\tax_2,\tax_5,\tax_6,\tp_1}$ (resulting from deletion of $\setof{\tax_1,\tax_4,\tax_7}$ from $\mo$ and addition of $\tp_1$). That these KBs constitute solution KBs can be verified by checking the three conditions named by Definition~\ref{def:solution_KB}. Indeed, adding an additional axiom in $\mo$ to any of the two KBs leads to the entailment of a negative test case $\tn\in\Tn$. That is, no solution KB can contain a proper superset of the axioms from $\mo$ that are contained in any of the two solution KBs. Hence, both are maximal.\qed
\end{example}
The problem to be solved by a KB debugger is given as follows:
\begin{prob_def}[Parsimonious KB Debugging] \label{prob_def:pars_KB_debugging}\cite{Rodler2015phd}
Given a DPI $\langle\mo,\mb,\Tp,\Tn\rangle_\RQ$, the task is to find a maximal solution KB w.r.t.\ $\langle\mo,\mb,\Tp,\Tn\rangle_\RQ$.
\end{prob_def}
Whereas the definition of a solution KB refers to the desired properties of the \emph{output} of a KB debugging system, the following definition can be seen as a helpful characterization of KBs provided as an \emph{input} to a KB debugger. If a KB is valid (w.r.t.\ the background knowledge, the requirements and the test cases), then finding a solution KB w.r.t.\ the DPI is trivial. Otherwise, obtaining a solution KB from it involves modification of the input KB and subsequent addition of suitable formulas. Usually, the KB $\mo$ part of the DPI given as an input to a debugger is initially assumed to be invalid w.r.t.\ this DPI.
\begin{definition}[Valid KB]\label{def:valid_onto}
Let $\langle\mo,\mb,\Tp,\Tn\rangle_\RQ$ be a DPI. Then, we say that a KB $\mo'$ is \emph{valid w.r.t.\ $\langle\cdot, \mb,$ $\Tp, \Tn \rangle_{\RQ}$} iff $\mo' \cup \mb \cup U_\Tp$ does not violate any $r\in\RQ$ and does not entail any $\tn \in \Tn$. A KB is said to be \emph{invalid (or faulty) w.r.t.\ $\langle\cdot,\mb,\Tp,\Tn\rangle_\RQ$} iff it is not valid w.r.t.\ $\langle\cdot,\mb,\Tp,\Tn\rangle_\RQ$.
\end{definition} 
\begin{example}\label{ex:valid_KB}
Recalling our running example DPI $\tuple{\mo,\mb,\Tp,\Tn}_\RQ$ in Table~\ref{tab:example_dpi_0} and $\mo_b^*$ as well as $\mo_c^*$ defined in Example~\ref{ex:solution_KB}, we observe that deletion of $\tp_1$ from both these KBs yields a valid KB w.r.t.\ $\tuple{\cdot,\mb,\Tp,\Tn}_\RQ$, i.e.\ one to which the known correct formulas, i.e.\ the union of the positive test cases as well as the background KB, can be safely added in the sense that, by this addition, no faults are introduced in terms of requirement violations or entailed negative test cases. On the other hand, (any superset of) $\setof{\tax_2,\tax_4}$ cannot be a valid KB w.r.t.\ $\tuple{\cdot,\mb,\Tp,\Tn}_\RQ$ due to the monotonicity of Propositional Logic (cf.\ assumption \ref{logic:cond1} in Section~\ref{sec:assumptions}) and the fact that $\tax_2 \equiv \setof{X \to H, F \to H}$ together with $\tax_8 (\in \mb) = H \to A$ and $\tax_4 = A \to \lnot F$ clearly yields $F \to \lnot F \equiv \lnot F$ which, in particular, implies $n_3 = \setof{F \to L} \equiv \setof{\lnot F \lor H}$. Another example of a faulty KB w.r.t.\ $\tuple{\cdot,\mb,\Tp,\Tn}_\RQ$ is $\setof{A,A\to B,\lnot B}$. Reason for its invalidity is its inconsistency, i.e.\ the violation of $r_1 \in \RQ$. \qed 
\end{example}
\paragraph{Diagnoses.} The key to solving the problem of Problem Definition~\ref{prob_def:pars_KB_debugging} is the notion of a (minimal) diagnosis given next. That is (cf.\ \cite[Prop.~3.6]{Rodler2015phd}): 
\begin{quote}
The task of finding maximal solution KBs w.r.t.\ a DPI can be reduced to computing minimal diagnoses w.r.t.\ this DPI.
\end{quote}
In other words, all maximal solution KBs w.r.t.\ a DPI can be constructed from all minimal diagnoses w.r.t.\ this DPI. Hence, one usually tackles the Parsimonious KB Debugging Problem by focusing on the computation of minimal diagnoses and constructing just \emph{one} (canonical) maximal solution KB from a minimal diagnosis $\md$ (see \cite[pages~32,33]{Rodler2015phd}). 
\begin{remark}\label{rem:infinitely_many_max_solution_KBs_wrt_min_diagnosis}
Please note that, generally, there are infinitely many possible maximal solution KBs constructible from a minimal diagnosis $\md$. This stems from the fact that there are infinitely many (semantically equivalent) syntactical variants of any set of suitable formulas that might be added to the KB resulting from the deletion of $\md$ from the faulty KB.  One reason for this is that there are infinitely many tautologies that might be included in these formulas, another reason is that formulas can be equivalently rewritten, e.g.\ $A \to B \equiv A \to B \lor \lnot A \equiv A \to B \lor \lnot A \lor \lnot A \equiv \dots$.\qed
\end{remark}

Such a canonical maximal solution KB can be defined by means of a minimal diagnosis $\md$, the original faulty KB $\mo$ and the set of positive test cases $\Tp$, namely as $(\mo\setminus\md)\cup U_\Tp$. Please notice that due to the restriction on canonical solution KBs we use a definition of a diagnosis in this work that is less general than the one given in \cite{Rodler2015phd} (but still does not affect the set of diagnoses w.r.t.\ a DPI, cf.\ \cite[page~33 and Corollary~3.3]{Rodler2015phd}).
%
\begin{definition}[Diagnosis]\label{def:diagnosis}
Let $\langle\mo,\mb,\Tp,\Tn\rangle_\RQ$ be a DPI. A set of formulas $\md \subseteq \mo$ is called a \emph{diagnosis w.r.t.\ $\langle\mo,\mb,\Tp,\Tn\rangle_\RQ$}, written as $\md \in \allD_{\langle\mo,\mb,\Tp,\Tn\rangle_\RQ}$, iff $(\mo\setminus\md)\cup U_\Tp$ is a solution KB w.r.t.\ $\langle\mo,\mb,\Tp,\Tn\rangle_\RQ$.

A diagnosis $\md$ w.r.t.\ $\langle\mo,\mb,\Tp,\Tn\rangle_\RQ$ is \emph{minimal}, written as $\md \in \minD_{\langle\mo,\mb,\Tp,\Tn\rangle_\RQ}$, iff there is no $\md' \subset \md$ such that 
$\md'$ is a diagnosis w.r.t.\ $\langle\mo,\mb,\Tp,\Tn\rangle_\RQ$. 
A diagnosis $\md$ w.r.t.\ $\langle\mo,\mb,\Tp,\Tn\rangle_\RQ$ is a \emph{minimum cardinality diagnosis w.r.t.\ $\langle\mo,\mb,\Tp,\Tn\rangle_\RQ$} iff there is no diagnosis $\md'$ w.r.t.\ $\langle\mo,\mb,\Tp,\Tn\rangle_\RQ$ such that $|\md'| < |\md|$.
\end{definition}
\begin{example}\label{ex_min_diagnoses}
A list of all minimal diagnoses w.r.t.\ our running example DPI (Table~\ref{tab:example_dpi_0}) is shown in Table~\ref{tab:min_diagnoses_example_dpi_0}. An explanation why these are diagnoses will be given later in Example~\ref{ex:diag_is_hitting_set_of_conflict}.\qed
\end{example}
\begin{table}[t]
	\footnotesize
	\centering
		\rowcolors[]{2}{gray!8}{gray!16} 
		\begin{tabular}{ c c } 
			\rowcolor{gray!40}
			\toprule\addlinespace[0pt]
			min diagnosis $X$ & $\setof{i \, |\, \tax_i \in X} $ \\ \addlinespace[0pt]\midrule\addlinespace[0pt]
			$\md_1$ & $\setof{2,3}$ \\
			$\md_2$ & $\setof{2,5}$ \\
			$\md_3$ & $\setof{2,6}$ \\
			$\md_4$ & $\setof{2,7}$ \\
			$\md_5$ & $\setof{1,4,7}$ \\
			$\md_6$ & $\setof{3,4,7}$ \\
		\end{tabular}
	\caption[Example: Minimal Diagnoses]{All minimal diagnoses w.r.t.\ the DPI given by Table~\ref{tab:example_dpi_0}.}
	\label{tab:min_diagnoses_example_dpi_0}
\end{table}
\begin{table}[t]
	\footnotesize
	\centering
		\rowcolors[]{2}{gray!8}{gray!16} 
		\begin{tabular}{ c c c} 
			\rowcolor{gray!40}
			\toprule\addlinespace[0pt]
			min conflict $X$ & $\setof{i \, |\, \tax_i \in X} $ & explanation \\ \addlinespace[0pt]\midrule\addlinespace[0pt]
			$\mc_1$ & $\setof{1,2,3}$ & $\models n_2$ \\
			$\mc_2$ & $\setof{2,4}$ & $\cup \setof{8} \models \lnot F \;(\models n_3)$ \\
			$\mc_3$ & $\setof{2,7}$ & $\cup \setof{p_1,8} \models n_1$\\
			$\mc_4$ & $\setof{3,5,6,7}$ & $\cup \setof{9} \models \lnot M \; (\models n_1)$ \\
		\end{tabular}
	\caption[Example: Minimal Conflict Sets]{All minimal conflict sets w.r.t.\ the DPI given by Table~\ref{tab:example_dpi_0}.}
	\label{tab:min_conflicts_example_dpi_0}
\end{table}
\paragraph{Conflict Sets.} \cite[Sec.~4.1]{Rodler2015phd} The search space for minimal diagnoses w.r.t.\ $\langle\mo,\mb,\Tp,\Tn\rangle_\RQ$, the size of which is in general $O(2^{|\mo|})$ (if all subsets of the KB $\mo$ are investigated), can be reduced to a great extent by exploiting the notion of a conflict set~\cite{Reiter87,dekleer1987,Shchekotykhin2012}.
\begin{definition}[Conflict Set]\label{def:cs} Let $\langle\mo,\mb,\Tp,\Tn\rangle_\RQ$ be a DPI. A set of formulas $\mc \subseteq \mo$ is called a \emph{conflict set w.r.t.\ $\langle\mo,\mb,\Tp,\Tn\rangle_\RQ$}, written as $\mc \in \allC_{\langle\mo,\mb,\Tp,\Tn\rangle_\RQ}$, iff $\mc$ is invalid w.r.t.\ $\langle\cdot,\mb,\Tp,\Tn\rangle_\RQ$. A conflict set $\mc$ is minimal, written as $\mc \in \minC_{\langle\mo,\mb,\Tp,\Tn\rangle_\RQ}$, iff there is no $\mc' \subset \mc$ such that $\mc'$ is a conflict set.
\end{definition}
Simply put, a (minimal) conflict set is a (minimal) faulty KB that is a subset of $\mo$. That is, a conflict set is one source causing the faultiness of $\mo$ in the context of $\mb \cup U_\Tp$. A minimal conflict set has the property that deletion of any formula in it yields a set of formulas which is not faulty in the context of $\mb$, $\Tp$, $\Tn$ and $\RQ$.
\begin{example}\label{ex:min_conflict_sets}
A list of all minimal conflict sets w.r.t.\ our running example DPI $\langle\mo,\mb,\Tp,\Tn\rangle_\RQ$ is enumerated in Table~\ref{tab:min_conflicts_example_dpi_0}. Let us briefly reflect why these are conflict sets w.r.t.\ $\langle\mo,\mb,\Tp,\Tn\rangle_\RQ$. An explanation of why $\mc_1$ is a conflict set was given in the first paragraph in Example~\ref{ex:solution_KB}. That $\mc_1$ is minimal can also be recognized at a glance since the elimination of any axiom $\tax_i (i\in \setof{1,2,3})$ from $\mc_1$ breaks the entailment of the negative test case $\tn_2$ and yields a consistent KB (in accordance with $r_1 \in \RQ$). The reasons for the set $\mc_2$ being a conflict set were explicated in Example~\ref{ex:valid_KB}.

$\mc_3$ is a set-minimal subset of the KB $\mo$ which, along with $\mb$ and $U_{\Tp}$ (in particular with $\tax_8 \in \mb$ and $\tp_1 \in \Tp$), implies that $\tn_1 \in \Tn$ must be true. This must hold as $\tax_7 \models M \to Z$, $\tp_1 = Z \to X$, $\tax_2 \models X \to H$ and $\tax_8 = H \to A$, from which $\tn_1 = \setof{M \to A}$ follows in a straightforward way.

Finally, $\mc_4$ is a conflict since $\tax_7 \models M\to C$, $\tax_6 = C \to B$, $\tax_9 \equiv B \to K$, $\tax_5 = K \to E$ and $\tax_3 \models E \to \lnot M$. Again, it is now obvious that this chain yields the entailment $\lnot M$ which in turn entails $\setof{\lnot M \lor A} \equiv \setof{M \to A} = \tn_1$. Clearly, the removal of any axiom from this chain breaks the entailment $\lnot M$. As this chain is neither inconsistent nor implies any negative test cases other than $\tn_1$, the conflict set $\mc_4$ is also minimal. 

It is also not very hard to verify that there are no other minimal conflict sets w.r.t.\ the example DPI apart from $\mc_1,\dots,\mc_4$.\qed
\end{example}
The relationship between the notions \emph{diagnosis}, \emph{solution KB}, \emph{valid KB} and \emph{conflict set} is as follows (cf.\ \cite[Corollary~3.3]{Rodler2015phd}):
\begin{proposition}\label{prop:notions_equiv} 
Let $\md \subseteq \mo$. Then the following statements are equivalent:
\begin{enumerate}
\item $\md$ is a diagnosis w.r.t.\ $\langle\mo,\mb,\Tp,\Tn\rangle_\RQ$.
\item $(\mo \setminus \md) \cup U_\Tp$ is a solution KB w.r.t.\ $\langle\mo,\mb,\Tp,\Tn\rangle_\RQ$.
\item $(\mo \setminus \md)$ is valid w.r.t.\ $\langle\cdot,\mb,\Tp,\Tn\rangle_\RQ$.
\item $(\mo \setminus \md)$ is not a conflict set w.r.t.\ $\langle\mo,\mb,\Tp,\Tn\rangle_\RQ$.
\end{enumerate}
\end{proposition}
\begin{example}\label{ex:notions_equiv}
Since, e.g.\, $\setof{\tax_1,\tax_2}$ is a valid KB w.r.t.\ to our example DPI (Table~\ref{tab:example_dpi_0}), we obtain that $\mo \setminus \setof{\tax_1,\tax_2} = \setof{\tax_3,\dots,\tax_7} =: \md$ is a diagnosis w.r.t.\ this DPI. However, this diagnosis $\md$ is not minimal since $\tax_5$ as well as $\tax_6$ can be deleted from it while preserving its diagnosis property ($\setof{\tax_3,\tax_4,\tax_7}$ is a minimal diagnosis). Further on, we can derive by Proposition~\ref{prop:notions_equiv} that $\setof{\tax_1,\tax_2,\tp_1}$ is a solution KB and that no set $S \subseteq \setof{\tax_1,\tax_2}$ can be a conflict set w.r.t.\ the example DPI.\qed
\end{example}
\paragraph{Relation between Conflict Sets and Diagnoses.} By deletion of at least one formula from \emph{each} minimal conflict set w.r.t.\ $\langle\mo,\mb,\Tp,\Tn\rangle_\RQ$, a valid KB can be obtained from $\mo$. Thus, a (maximal) solution KB $(\mo\setminus\md) \cup U_\Tp$ can be obtained by the calculation of a (minimal) hitting set $\md$ of all minimal conflict sets in $\minC_{\langle\mo,\mb,\Tp,\Tn\rangle_\RQ}$: 
\begin{definition}[Hitting Set]\label{def:hs}
Let $S=\setof{S_1,\dots,S_n}$ be a set of sets. Then, $H$ is called a \emph{hitting set of $S$} iff $H \subseteq U_S$ and $H \cap S_i \neq \emptyset$ for all $i=1,\dots,n$. 
A hitting set $H$ of $S$ is \emph{minimal} iff there is no hitting set $H'$ of $S$ such that $H' \subset H$.
\end{definition}
\begin{proposition}\cite{friedrich2005gdm}\label{prop:mindiag_mincs}
A (minimal) diagnosis w.r.t.\ the DPI $\langle\mo,\mb,\Tp,\Tn\rangle_\RQ$ is a (minimal) hitting set of all minimal conflict sets w.r.t.\ $\langle\mo,\mb,\Tp,\Tn\rangle_\RQ$.
\end{proposition}
Hence, what we might call the standard method for computing (minimal) diagnoses w.r.t.\ some given DPI is via the computation of minimal conflict sets and the calculation of minimal hitting sets thereof. As described comprehensively in \cite{Rodler2015phd}, this might be accomplished by the usage of a hitting set tree (originally due to Reiter~\cite{Reiter87}) along with a method such as QuickXPlain (originally due to Junker~\cite{junker04}). For details we want to refer the reader to \cite{Rodler2015phd}. 
\begin{example}\label{ex:diag_is_hitting_set_of_conflict}
The reader is invited to verify that all minimal diagnoses shown in Table~\ref{tab:min_diagnoses_example_dpi_0} are indeed minimal hitting sets of all minimal conflict sets shown in Table~\ref{tab:min_conflicts_example_dpi_0}. For instance, using the simplified notation for axioms as in the tables, $\md_1 = \setof{2,3}$ ``hits'' the element $2$ of $\mc_i (i \in \setof{1,2,3})$ and the element $3$ of $\mc_4$. Note also that it hits two elements of $\mc_1$ which, however, is not necessarily an indication of the non-minimality of the hitting set. Indeed, if $2$ is deleted from $\md_1$, it has an empty intersection with $\mc_2$ and $\mc_3$ and, otherwise, if $3$ is deleted from it, it becomes disjoint with $\mc_4$. Hence $\md_1$ is actually a minimal hitting set of $\minC_{\langle\mo,\mb,\Tp,\Tn\rangle_\RQ}$. Clearly, this implies that $\md_1$ is also a hitting set of $\allC_{\langle\mo,\mb,\Tp,\Tn\rangle_\RQ}$.

Please realize that, by now, we have also completed the explanation why the diagnoses listed in Table~\ref{tab:min_diagnoses_example_dpi_0} actually correspond to the set of all minimal diagnoses w.r.t.\ the example DPI $\langle\mo,\mb,\Tp,\Tn\rangle_\RQ$ shown by Table~\ref{tab:example_dpi_0}. In particular, this follows from Proposition~\ref{prop:mindiag_mincs} and the illustration (given in Example~\ref{ex:min_conflict_sets}) of why $\mc_1,\dots,\mc_4$ constitute a complete set of minimal conflict sets w.r.t.\ $\langle\mo,\mb,\Tp,\Tn\rangle_\RQ$.\qed
\end{example}

\subsection{Interactive Knowledge Base Debugging: Basics}
\label{sec:InteractiveKnowledgeBaseDebuggingBasics}
In real-world scenarios, debugging tools often have to cope with large numbers of minimal diagnoses where the trivial application, i.e.\ deletion, of any minimal diagnosis leads to a valid KB with different semantics in terms of entailed and non-entailed formulas. For example, a sample study of real-world KBs \cite{ksgf2010} detected 1782 different minimal diagnoses for a KB with only 1300 formulas. In such situations a simple visualization of all these alternative modifications of the KB is clearly ineffective. Selecting a wrong diagnosis (in terms of its semantics, \emph{not} in terms of fulfillment of test cases and requirements) can lead to unexpected, surprising or fatal entailments or non-entailments. Assume for example that a patient is administered a wrong therapy (e.g.\ a chemotherapy) due to an incorrect semantics of the KB which is part of the expert system a doctor is consulting. 

Moreover, as mentioned in Section~\ref{sec:Introduction}, non-interactive debugging systems often have serious \emph{performance problems}, run \emph{out of memory}, are not able to find all diagnoses w.r.t.\ the DPI (\emph{incompleteness}), report diagnoses which are actually none (\emph{unsoundness}), produce \emph{only trivial solutions} or suggest non-minimal diagnoses (\emph{non-minimality}). Often, performance problems and incompleteness of non-interactive debugging methods can be traced back to an explosion of the search tree for minimal diagnoses.

Interactive KB debugging addresses and tackles all the mentioned problems of non-interactive KB debuggers in that it uses a sound, complete and best-first iterative minimal diagnosis computation process (see \cite[Part~2]{Rodler2015phd}) and between each two iterations consults a user for additional information that serves to specify constraints and in further consequence to gradually reduce the search space for minimal diagnoses. 

\paragraph{Assumptions about the Interacting User.} We suppose the user of an interactive KB debugger to be either a single person or multiple persons, usually experts in the particular domain the faulty KB is dealing with or authors of the faulty KB, or some (automatic) oracle, e.g.\ some knowledge extraction system. Anyway, the consultation of the user must be seen as an expensive operation. That is, the number of calls to the user should be minimized.

About a user $u$ of an (interactive) debugging system, we make the following plausible assumptions (for a discussion and plausibility check see \cite[p.~9 -- 11]{Rodler2015phd}):
\begin{enumerate}[label=(\subscript{U}{{\arabic*}})]
	\item $u$ is not 
	able to explicitly enumerate a set of logical sentences that express the intended domain that should be modeled in a satisfactory way, i.e.\ without unwanted entailments or non-fulfilled requirements, 
	\item $u$ is able to answer concrete queries about the intended domain that should be modeled, i.e.\ $u$ can classify a given logical formula (or a conjunction of logical formulas) as a wanted or unwanted proposition in the intended domain (i.e.\ an entailment or non-entailment of the correct domain model). 
\end{enumerate}
User consultation means asking the user to answer an (automatically generated and potentially optimized) query.

\paragraph{Queries.} In interactive KB debugging, a set of logical formulas $Q$ is presented to the user who should decide whether to assign $Q$ to the set of positive ($\Tp$) or negative ($\Tn$) test cases w.r.t.\ a given DPI $\langle\mo,\mb,\Tp,\Tn\rangle_\RQ$. In other words, the system asks the user ``should the KB you intend to model entail all formulas in $Q$?''. In that, $Q$ is generated by the debugging algorithm in a way that \emph{any} decision of the user about $Q$
\begin{enumerate}
	\item invalidates at least one minimal diagnosis (\emph{search space restriction}) and
	\item preserves validity of at least one minimal diagnosis (\emph{solution preservation}).
\end{enumerate}
We call a set of logical formulas $Q$ with these properties a \emph{query}. Successive classification of queries as entailments (all formulas in $Q$ must be entailed) or non-entailments (at least one formula in $Q$ must not be entailed) of the correct KB enables gradual restriction of the search space for (minimal) diagnoses. Further on, classification of sufficiently many queries guarantees the detection of a single correct solution diagnosis which can be used to determine a solution KB with the correct semantics w.r.t.\ a given DPI.
%
\begin{definition}[Query]\label{def:query}
Let $\langle\mo,\mb,\Tp,\Tn\rangle_\RQ$ over $\mathcal{L}$ and $\mD \subseteq \minD_{\langle\mo,\mb,\Tp,\Tn\rangle_\RQ}$. 
Then a set of logical formulas $Q\neq\emptyset$ over $\mathcal{L}$ is called a \emph{query w.r.t.\ $\mD$ and $\langle\mo,\mb,\Tp,\Tn\rangle_\RQ$} iff there are diagnoses $\md, \md' \in \mD$ such that $\md \notin \minD_{\langle\mo,\mb,\Tp \cup \setof{Q},\Tn\rangle_\RQ}$ and $\md'\notin\minD_{\langle\mo,\mb,\Tp,\Tn \cup \setof{Q}\rangle_\RQ}$. The set of all queries w.r.t.\ $\mD$ and $\langle\mo,\mb,\Tp,\Tn\rangle_\RQ$ is denoted by $\mQ_{\mD,\langle\mo,\mb,\Tp,\Tn\rangle_\RQ}$.
\end{definition}
%

So, w.r.t.\ a set of minimal diagnoses $\mD \subseteq \minD_{\langle\mo,\mb,\Tp,\Tn\rangle_\RQ}$, a query $Q$ is a set of logical formulas that rules out at least one diagnosis in $\mD$ (and therefore in $\minD_{\langle\mo,\mb,\Tp,\Tn\rangle_\RQ}$) as a candidate to formulate a solution KB, regardless of whether $Q$ is classified as a positive or negative test case. Moreover, as it turns out as a consequence of Definition~\ref{def:query} (cf.\ Proposition~\ref{prop:properties_of_q-partitions}), a user can never end up with a new DPI for which no solutions exist as long as the new DPI results from a given one by the addition of an answered query to one of the test case sets $\Tp$ or $\Tn$. 

\paragraph{Leading Diagnoses.} Query generation requires a precalculated set of minimal diagnoses $\mD$, some subset of $\minD_{\langle\mo,\mb,\Tp,\Tn\rangle_\RQ}$, that serves as a representative for all minimal diagnoses $\minD_{\langle\mo,\mb,\Tp,\Tn\rangle_\RQ}$. As already mentioned, computation of the entire set $\minD_{\langle\mo,\mb,\Tp,\Tn\rangle_\RQ}$ is generally not tractable within reasonable time. Usually, $\mD$ is defined as a set of most probable (if some probabilistic meta information is given) or minimum cardinality diagnoses. Therefore, $\mD$ is called the set of \emph{leading diagnoses w.r.t. $\langle\mo,\mb,\Tp,\Tn\rangle_\RQ$} \cite{Shchekotykhin2012}. 

The leading diagnoses $\mD$ are then exploited to determine a query $Q$, the answering of which enables a discrimination between the diagnoses in $\minD_{\langle\mo,\mb,\Tp,\Tn\rangle_\RQ}$. That is, a subset of $\minD_{\langle\mo,\mb,\Tp,\Tn\rangle_\RQ}$ which is not ``compatible'' with the new information obtained by adding the test case $Q$ to $\Tp$ or $\Tn$ is ruled out. For the computation of the subsequent query only a leading diagnoses set $\mD_{new}$ w.r.t.\ the minimal diagnoses still compliant with the new sets of test cases $\Tp'$ and $\Tn'$ is taken into consideration, i.e.\ $\mD_{new} \subseteq \minD_{\langle\mo,\mb,\Tp',\Tn'\rangle_\RQ}$.

The number of precomputed leading diagnoses $\mD$ affects the quality of the obtained query. The higher $|\mD|$, the more representative is $\mD$ w.r.t. $\minD_{\langle\mo,\mb,\Tp,\Tn\rangle_\RQ}$, the more options there are to specify a query in a way that a user can easily comprehend and answer it, and the higher is the chance that a query that eliminates a high rate of diagnoses w.r.t. $\mD$ will also eliminate a high rate of all minimal diagnoses $\minD_{\langle\mo,\mb,\Tp,\Tn\rangle_\RQ}$. The selection of a lower $|\mD|$ on the other hand means better timeliness regarding the interaction with a user, first because fewer leading diagnoses might be computed much faster and second because the search space for an ``optimal'' query is smaller.
%

\begin{example}\label{ex:queries_wrt_leading_diags}
Let us consider our running example DPI $\langle\mo,\mb,\Tp,\Tn\rangle_\RQ$ (Table~\ref{tab:example_dpi_0}) again and let us further suppose that some diagnosis computation algorithm provides the set of leading diagnoses $\mD = \setof{\md_1,\md_2,\md_3,\md_4}$ (see Table~\ref{tab:min_diagnoses_example_dpi_0}). Then $Q_1 := \setof{M \to B}$ is a query w.r.t.\ $\mD$ and $\langle\mo,\mb,\Tp,\Tn\rangle_\RQ$. 

To verify this, we have to argue why at least one diagnosis in $\mD$ is eliminated in either case, i.e.\ for positive answer yielding $\Tp' \gets \Tp \cup \setof{Q_1}$ as well as for negative answer yielding $\Tn' \gets \Tn \cup \setof{Q_1}$. Formally, we must show that $\mD \cap \minD_{\langle\mo,\mb,\Tp',\Tn\rangle_\RQ} \subset \mD$ (reduction of the leading diagnoses given the positive query answer) and $\mD \cap \minD_{\langle\mo,\mb,\Tp,\Tn'\rangle_\RQ} \subset \mD$ (reduction of the leading diagnoses given the negative query answer). 

We first assume a positive answer, i.e.\ we take as a basis the DPI $DPI^+ := \langle\mo,\mb,\Tp',\Tn\rangle_\RQ$. Let us test the diagnosis property w.r.t.\ $DPI^+$ for $\md_4 \in \mD$. By Proposition~\ref{prop:notions_equiv} we can check whether $\mo\setminus\md_4$ is valid w.r.t.\ $DPI^+$ in order to check whether $\md_4 = \setof{\tax_2,\tax_7}$ is still a diagnosis w.r.t.\ this DPI. To this end, (using the simplified axiom notation of Tables~\ref{tab:min_diagnoses_example_dpi_0} and \ref{tab:min_conflicts_example_dpi_0} again) we have a look at $\mo^+_4 := (\mo \setminus \md_4) \cup \mb \cup U_{\Tp'} = \setof{1,3,4,5,6,8,9,\tp_1,M\to B}$. As $\tax_9 \equiv B \to K$, $\tax_5 = K \to E$ and $\tax_3 \models E \to \lnot M$, it is clear that $\setof{M \to A} = \tn_1 \in \Tn$ is an entailment of $\mo^+_4$. Hence $\mo\setminus\md_4$ is faulty w.r.t.\ $DPI^+$ which is why $\md_4 \notin \minD_{\langle\mo,\mb,\Tp',\Tn\rangle_\RQ}$. 

For the case of receiving a negative answer to $Q_1$, let us consider diagnosis $\md_2 = \setof{\tax_2,\tax_5}$ w.r.t.\ $DPI^- := \langle\mo,\mb,\Tp,\Tn'\rangle_\RQ$. Now, $\mo^-_2 := (\mo \setminus \md_2) \cup \mb \cup U_{\Tp} = \setof{1,3,4,6,7,8,9,\tp_1} \models \setof{M \to B} = Q_1 \in \Tn'$ due to $\tax_7 \models M\to C$ and $\tax_6 = C \to B$. Therefore, $\mo \setminus \md_2$ is faulty w.r.t.\ $DPI^-$ which is why $\md_2 \notin \minD_{\langle\mo,\mb,\Tp,\Tn'\rangle_\RQ}$.

All in all we have substantiated that $Q_1$ is a query w.r.t.\ $\mD$ and $\langle\mo,\mb,\Tp,\Tn\rangle_\RQ$. Note in addition that $\md_2$ is a diagnosis w.r.t.\ $DPI^+$ since $\mo^+_2$ is consistent and does not entail any negative test case in $\Tn$. Furthermore, $\md_4$ is a diagnosis w.r.t.\ $DPI^-$ since $\mo^-_4$ is consistent and does not entail any negative test case in $\Tn'$. That is, while being deleted for one answer, diagnoses remain valid when facing the complementary answer. This can also be shown in general, see Proposition~\ref{prop:properties_of_q-partitions}.\qed
\end{example}
\paragraph{Q-Partitions.} A \emph{q-partition} is a partition of the leading diagnoses set $\mD$ induced by a query w.r.t.~$\mD$. A q-partition will be a helpful instrument in deciding whether a set of logical formulas is a query or not. It will facilitate an estimation of the impact a query answer has in terms of invalidation of minimal diagnoses. And, given fault probabilities, it will enable us to gauge the probability of getting a positive or negative answer to a query. 

From now on, given a DPI $\langle\mo,\mb,\Tp,\Tn\rangle_\RQ$ and some minimal diagnosis $\md_i$ w.r.t.\ $\langle\mo,\mb,\Tp,\Tn\rangle_\RQ$, we will use the following abbreviation for the solution KB obtained by deletion of $\md_i$ along with the given background knowledge $\mb$:
\begin{align} 
\mo^{*}_i \; := \; (\mo \setminus \md_i) \cup \mb \cup U_\Tp \label{eq:sol_ont_candidate} 
\end{align}
\begin{definition}[q-Partition\footnote{In existing literature, e.g.\ \cite{Shchekotykhin2012,Rodler2013,ksgf2010}, a q-partition is often simply referred to as partition. We call it q-partition to emphasize that not each partition of $\mD$ into three sets is necessarily a q-partition.}]\label{def:q-partition}
Let $\langle\mo,\mb,\Tp,\Tn\rangle_\RQ$ be a DPI over $\mathcal{L}$, $\mD\subseteq \minD_{\langle\mo,\mb,\Tp,\Tn\rangle_\RQ}$. 
Further, let $Q$ be a set of logical sentences over $\mathcal{L}$ and
\begin{itemize}
\item $\dx{}(Q):=\setof{\md_i \in \mD\,|\,\mo^{*}_i \models Q}$, 
\item $\dnx{}(Q):=\setof{\md_i \in \mD\,|\,\exists x\in\RQ\cup\Tn: \mo^{*}_i \cup Q \text{ violates } x}$,
\item $\dz{}(Q) := \mD \setminus (\dx{j} \cup \dnx{j})$. 
\end{itemize}
Then $\Pt(Q) := \langle \dx{}(Q), \dnx{}(Q), \dz{}(Q) \rangle$ is called a \emph{q-partition} iff $Q$ is a query w.r.t.\ $\mD$ and $\langle\mo,\mb,\Tp,\Tn\rangle_\RQ$. Note, given some q-partition $\Pt = \tuple{\dx{},\dnx{},\dz{}}$, we sometimes refer to $\dx{}$, $\dnx{}$ and $\dz{}$ by $\dx{}(\Pt)$, $\dnx{}(\Pt)$ and $\dz{}(\Pt)$, respectively.
\end{definition}
\begin{example}\label{ex:q-partition}
Given the leading diagnoses set $\mD = \setof{\md_1,\dots,\md_4}$ w.r.t.\ our example DPI (Table~\ref{tab:example_dpi_0}) and the query $Q_1 = \setof{M \to B}$ characterized in Example~\ref{ex:queries_wrt_leading_diags}. Then, as we have already established in Example~\ref{ex:queries_wrt_leading_diags}, $\md_4$ is invalidated given the positive answer to $Q_1$ since $\mo^+_4 = \mo_4^* \cup Q_1 \models \tn_1 \in \Tn$. Hence, by Definition~\ref{def:q-partition}, $\md_4$ must be an element of $\dnx{}(Q_1)$.

On the contrary, $\md_2 \in \dx{}(Q_1)$ because $\mo^-_2 = \mo_2^* = \setof{1,3,4,6,7,8,9,\tp_1} \models \setof{M \to B} = Q_1$, as demonstrated in Example~\ref{ex:queries_wrt_leading_diags}.

To complete the q-partition for the query $Q_1$, we need to assign the remaining diagnoses in $\mD$, i.e.\ $\md_1 = \setof{\tax_2,\tax_3}, \md_3 = \setof{\tax_2, \tax_6}$ to the respective set of the q-partition. Since $\mo_1^*$ includes $\tax_6$ and $\tax_7$ we can conclude in an analogous way as in Example~\ref{ex:queries_wrt_leading_diags} that $\mo_1^* \models Q_1$ which is why $\md_1 \in \dx{}(Q_1)$. In the case of $\md_3$, we observe that $\mo_3^* \cup Q_1$ includes $\tax_3$, $\tax_5$, $M\to B$ as well as $\tax_9$. As explicated in Example~\ref{ex:queries_wrt_leading_diags}, these axioms imply the entailment of $\tn_1 \in \Tn$. Consequently, $\md_3 \in \dnx{}(Q_1)$ must hold and thus $\dz{}(Q_1)$ is the empty set. \qed 
\end{example}

\begin{algorithm}[t]
\small
\caption{Interactive KB Debugging}\label{algo:inter_onto_debug}
\begin{algorithmic}[1]
\Require a DPI $\langle\mo,\mb,\Tp,\Tn\rangle_\RQ$, 
a probability measure $p$ (used to compute formula and diagnoses fault probabilities),
a query quality measure $m$,
a solution fault tolerance $\sigma$ 
\Ensure A (canonical) maximal solution KB $(\mo \setminus \md) \cup U_{\Tp}$ w.r.t.\ $\langle\mo,\mb,\Tp,\Tn\rangle_\RQ$ constructed from a minimal diagnosis $\md$ w.r.t.\ $\langle\mo,\mb,\Tp,\Tn\rangle_\RQ$ that has at least $1-\sigma$ probability of being the true diagnosis
\While{\true}  \label{algoline:inter_onto_debug:while}
	\State $\mD \gets \Call{computeLeadingDiagnoses}{\langle\mo,\mb,\Tp,\Tn\rangle_\RQ}$
	\State $\md \gets \Call{getBestDiagnosis}{\mD,p}$
	\If{$p(\md) \geq 1-\sigma$}
		\State \Return $(\mo \setminus \md) \cup U_{\Tp}$
	\EndIf
	\State $Q \gets \Call{calcQuery}{\mD,\langle\mo,\mb,\Tp,\Tn\rangle_\RQ, p, m}$     \Comment{see Algorithm~\ref{algo:query_comp}}
	\State $answer \gets u(Q)$				\Comment{user interaction}
	\State $\langle\mo,\mb,\Tp,\Tn\rangle_\RQ \gets \Call{updateDPI}{\langle\mo,\mb,\Tp,\Tn\rangle_\RQ, Q, answer}$
\EndWhile
\end{algorithmic}
\normalsize
\end{algorithm}

Some important properties of q-partitions and their relationship to queries are summarized by the next proposition:
\begin{proposition}\label{prop:properties_of_q-partitions}\cite[Sec.\ 7.3 -- 7.6]{Rodler2015phd} Let $\langle\mo,\mb,\Tp,\Tn\rangle_\RQ$ be a DPI over $\mathcal{L}$ and $\mD \subseteq \minD_{\langle\mo,\mb,\Tp,\Tn\rangle_\RQ}$. Further, let $Q$ be any query w.r.t.\ $\mD$ and $\langle\mo,\mb,\Tp,\Tn\rangle_\RQ$. Then:
~\begin{enumerate} 
	\item \label{prop:properties_of_q-partitions:enum:q-partition_is_partition} $\langle \dx{}(Q)$, $\dnx{}(Q)$, $\dz{}(Q) \rangle$ is a partition of $\mD$.
	\item \label{prop:properties_of_q-partitions:enum:dx_dnx_dz_contain_exactly_those_diags_that_are...} $\dnx{}(Q)$ contains exactly those diagnoses $\md_i\in\mD$ where $\mo \setminus \md_i$ is invalid w.r.t.\ $\tuple{\cdot,\mb,\Tp\cup\setof{Q},\Tn}_\RQ$. $\dx{}(Q)$ contains exactly those diagnoses $\md_i\in\mD$ where $\mo \setminus \md_i$ is invalid w.r.t.\ $\tuple{\cdot,\mb,\Tp,\Tn\cup\setof{Q}}_\RQ$. $\dz{}(Q)$ contains exactly those diagnoses $\md_i\in\mD$ where $\mo \setminus \md_i$ is valid w.r.t.\ $\tuple{\cdot,\mb,\Tp\cup\setof{Q},\Tn}_\RQ$ and valid w.r.t.\ $\tuple{\cdot,\mb,\Tp,\Tn\cup\setof{Q}}_\RQ$. 
	\item \label{prop:properties_of_q-partitions:enum:q-partition_is_unique_for_query} For $Q$ there is one and only one q-partition $\langle\dx{}(Q),$ $\dnx{}(Q), \dz{}(Q)\rangle$.
	\item \label{prop:properties_of_q-partitions:enum:query_is_set_of_common_ent} $Q$ is a set of common entailments of all KBs in $\setof{\mo^{*}_i\,|\,\md_i\in\dx{}(Q)}$.
	\item \label{prop:properties_of_q-partitions:enum:set_of_logical_formulas_is_query_iff...} A set of logical formulas $X\neq\emptyset$ over $\mathcal{L}$ is a query w.r.t. $\mD$ iff $\dx{}(X) \neq \emptyset$ and $\dnx{}(X)~\neq~\emptyset$.
	\item \label{prop:properties_of_q-partitions:enum:for_each_q-partition_dx_is_empty_and_dnx_is_empty} For each q-partition $\Pt(Q) = \langle \dx{}(Q), \dnx{}(Q), \dz{}(Q)\rangle$ w.r.t.\ $\mD$ it holds that $\dx{}(Q) \neq \emptyset$ and $\dnx{}(Q)~\neq~\emptyset$.
	\item \label{prop:properties_of_q-partitions:enum:D+=d_i_is_q-partition_and_lower_bound_of_queries} If $|\mD|\geq 2$, then 
	\begin{enumerate}
		\item $Q:=U_\mD\setminus\md_i$ is a query w.r.t.\ $\mD$ and $\langle\mo,\mb,\Tp,\Tn\rangle_\RQ$ for all $\md_i\in\mD$,
		\item $\langle \setof{\md_i}, \mD \setminus \setof{\md_i}, \emptyset\rangle$ is the q-partition associated with $Q$, and
		\item a lower bound for the number of queries w.r.t.\ $\mD$ and $\langle\mo,\mb,\Tp,\Tn\rangle_\RQ$ is $|\mD|$.
	\end{enumerate}
\end{enumerate}
\end{proposition}

Algorithm~\ref{algo:inter_onto_debug} describes the overall interactive debugging process that has already been sketched and discussed in Section~\ref{sec:Introduction}. It explicates in formal terms what is depicted by Figure~\ref{fig:interactive_debugging_workflow}. It first computes a set of leading diagnoses (\textsc{computeLeadingDiagnoses}), then uses some probability measure $p$ to extract the most probable leading diagnosis (\textsc{getBestDiagnosis}). If there are no probabilities given, then $p$ is specified so as to assign the maximal probability to diagnoses of minimal cardinality. If the probability of the best diagnosis $\md$ is already above the threshold $1-\sigma$ (a common setting for $\sigma$ is $0.05$), i.e.\ the stop criterion is met, then the maximal (canonical) solution KB according to $\md$ is returned and the debugging session terminates. Otherwise, a sufficiently good query, where goodness is measured in terms of some query quality measure $m$, is computed (\textsc{calcQuery}) and shown to the user. The latter is modeled as a function $Q \mapsto u(Q) \in \setof{t,f}$ that maps any query $Q$ to a positive ($t$) or negative ($f$) answer. After incorporating the answered query as a new test case into the DPI (\textsc{updateDPI}), the next iteration starts and uses the new DPI instead of the original one. Note that the \textsc{updateDPI} function subsumes further operations which we do not consider in detail such as the (Bayesian) update \cite{dekleer1987} of the probability measure $p$ based on the new information given by the answered query (see \cite[Sec.~9.2.4]{Rodler2015phd}). Further information on this algorithm and a much more detailed description including proofs of correctness and optimality can be found in \cite[Chap.~9]{Rodler2015phd}.

Finally, the major advantage of interactive debugging techniques can be summarized as follows:\vspace{3pt}

\noindent\fcolorbox{black}{light-gray1}{\parbox[c][8.8em][c]{0.975\linewidth}{\vspace{-1pt}
A user can debug their faulty KB \textbf{without} needing to analyze 
\vspace{-1pt}
\begin{itemize}[itemsep=0pt]
	\item \emph{which} entailments do or do not hold \emph{in the faulty KB}
	\item \emph{why} certain entailments do or do not hold \emph{in the faulty KB} or 
	\item \emph{why} exactly \emph{the faulty KB} does not meet the requirements $\RQ$ and test cases $\Tp, \Tn$
\end{itemize}
\vspace{-1pt}
\textbf{by just} answering queries
\emph{whether} certain entailments should or should not hold \emph{in the intended domain}, \textbf{under the guarantee} that the resulting solution KB will feature exactly the desired semantics.
\vspace{-9pt}
}}

\clearpage
\section{Query Computation}
\label{sec:QueryComputation}
In this section we will present the implementation of the $\textsc{calcQuery}$ function in Algorithm~\ref{algo:inter_onto_debug}. 
In particular, we will
\begin{enumerate} 
\item discuss and study existing quantitative measures $m$ already used in works on KB debugging \cite{ksgf2010, Shchekotykhin2012, Rodler2013} that can be employed in the \textsc{calcQuery} function in order to establish the goodness of queries (Section~\ref{sec:ExistingActiveLearningMeasuresForKBDebugging}), 
\item adapt various general active learning measures \cite{settles2012} for use within the function \textsc{calcQuery} in the KB debugging framework and provide a comprehensive mathematical analysis of them (Section~\ref{sec:NewActiveLearningMeasuresForKBDebugging}), 
\item study in depth a recently proposed measure $\mathsf{RIO}$ (presented in \cite{Rodler2013}) that is embedded in a reinforcement learning scheme and particularly tackles the problem that the performance (i.e.\ the number of queries a user must answer) of existing measures depends strongly on the quality of the initially given fault information $p$ (Section~\ref{sec:rio}),
\item study all discussed measures (i.e.\ functions) with regard to which input yields their \emph{optimal value}, define a plausible \emph{preference order on queries}, called the discrimination-preference relation, and analyze for all discussed measures how far they preserve this order, characterize a \emph{superiority relation on query quality measures} based on the (degree of) their compliance with the discrimination-preference relation and figure out superiority relationships between the discussed measures, formalize the notions of \emph{equivalence between two measures} and examine situations or describe subsets of all queries, respectively, where different measures are equivalent to one another (Sections~\ref{sec:QPartitionSelectionMeasures} -- \ref{sec:rio}), 
\item derive improved versions from some query quality measures -- including those used in \cite{Shchekotykhin2012,Rodler2013,ksgf2010} -- to overcome shortcomings unveiled in the conducted in-depth theoretical evaluations (Sections~\ref{sec:ExistingActiveLearningMeasuresForKBDebugging} -- \ref{sec:rio}),
%
\item exploit the carried out optimality analyses to derive from all the discussed (quantitative) query quality measures \emph{qualitative} goodness criteria, i.e.\ properties, of queries which allow us to derive \emph{heuristics} to design efficient search strategies to find an optimal query (Sections \ref{sec:ExistingActiveLearningMeasuresForKBDebugging} -- \ref{sec:QPartitionRequirementsSelection}), 
\item detail a new, generic and efficient way of generating q-partitions \emph{without relying on (expensive) reasoning services at all} in a best-first order which facilitates the finding of q-partitions that are 
arbitrarily close to the optimal value w.r.t.\ a given query quality measure
  (Section~\ref{sec:FindingOptimalQPartitions}). Further on, we will
\item demonstrate how set-minimal queries can be found in a best-first order (e.g.\ minimum-size or minimal-answer-fault-probability queries first), which is not possible with existing methods (Section~\ref{sec:FindingOptimalQueriesGivenAnOptimalQPartition}), we will
\item show that that minimization of a query is equivalent to solving a (small and restricted) hitting set problem (Section~\ref{sec:FindingOptimalQueriesGivenAnOptimalQPartition}), and we will 
\item explain how a query can be enriched with additional logical formulas that have a simple predefined (syntactical) structure (e.g.\ formulas of the form $A \to B$ for propositional symbols $A,B$) in a way that the optimality of the query w.r.t.\ the goodness measure $m$ is preserved (Section~\ref{sec:EnrichmentOfAQuery}). This should grant a larger pool of formulas to choose from when reducing the query to a set-minimal set of logical formulas and increase the query's likeliness of being well-understood and answered correctly by the interacting user. Finally, we will
\item explicate how such an enriched query can be optimized, i.e.\ minimized in size while giving some nice guarantees such as the inclusion of only formulas with a simple syntactical structure in the query, if such a query exists (Section~\ref{sec:QueryOptimization}).
\end{enumerate}

The overall procedure $\textsc{calcQuery}$ to compute a query is depicted by Algorithm~\ref{algo:query_comp}. The algorithm takes as an input amongst others some query quality measure $m$ that serves as a preference criterion for the selection of a query. More concretely, $m$ maps a query $Q$ to some real number $m(Q)$ that quantifies the goodness of $Q$ in terms of $m$.
\begin{definition}\label{def:query_quality_measure}
Let $\langle\mo,\mb,\Tp,\Tn\rangle_\RQ$ be a DPI and $\mD \in \minD_{\langle\mo,\mb,\Tp,\Tn\rangle_\RQ}$ be a set of leading diagnoses. Then we call a function $m: \mQ_{\mD,\langle\mo,\mb,\Tp,\Tn\rangle_\RQ} \to \mathbb{R}$ that maps a query $Q \in \mQ_{\mD,\langle\mo,\mb,\Tp,\Tn\rangle_\RQ}$ to some real number $m(Q)$ a \emph{query quality measure}. Depending on the choice of $m$, optimality of $Q$ is signalized by $m(Q) \rightarrow \max$ or $m(Q) \rightarrow \min$.
\end{definition}
So, the problem to be solved by the function \textsc{calcQuery} is:
\begin{prob_def}\label{prob_def:calc_query}
Given some measure $m$ and a set of leading diagnoses $\mD$ w.r.t.\ some DPI $\langle\mo,\mb,\Tp,\Tn\rangle_\RQ$, the task is to find a query $Q$ with optimal (or sufficiently good) value $m(Q)$ among all queries $Q\in\mQ_{\mD,\langle\mo,\mb,\Tp,\Tn\rangle_\RQ}$.
\end{prob_def}

\begin{algorithm}[t]
\small
\caption[Query Computation]{Query Computation (\textsc{calcQuery})}\label{algo:query_comp}
\begin{algorithmic}[1]
\Require DPI $\langle\mo,\mb,\Tp,\Tn\rangle_\RQ$, leading diagnoses $\mD \subseteq \minD_{\langle\mo,\mb,\Tp,\Tn\rangle_\RQ}$, probability measure $p$ (used to compute formula and diagnoses fault probabilities), 
query quality measure $m$, threshold $t_m$ specifying the region of optimality around the theoretically optimal value $m_{opt}$ of $m$ (queries $Q$ with $|m(Q) - m_{opt}| \leq t_m$ are considered optimal),
parameters $qSelParams = \tuple{strat,time,n_{\min},n_{\max}}$ to steer query selection w.r.t.\ a given q-partition 
(explicated in Algorithm~\ref{algo:select_query_for_q-partition} and Section~\ref{sec:FindingOptimalQueriesGivenAnOptimalQPartition}),
sound and complete reasoner $Rsnr$ (w.r.t.\ the used logic $\mathcal{L}$), set $ET$ of entailment types (w.r.t.\ the used logic $\mathcal{L}$)
\Ensure an optimal query $Q$ w.r.t.\ $\mD$ and $\langle\mo,\mb,\Tp,\Tn\rangle_\RQ$ as per $m$ and $t_m$, if existent; the best found query, otherwise
\State $r_m \gets \Call{qPartitionReqSelection}{m}$ 
\State $\Pt \gets \Call{findQPartition}{\mD,p,t_m,r_m}$
\State $Q \gets \Call{selectQueryForQPartition}{\Pt,p,qSelParams}$
\State $Q' \gets \Call{enrichQuery}{\langle\mo,\mb,\Tp,\Tn\rangle_\RQ, Q, \Pt, Rsnr, ET}$    \Comment{optional}
\State $Q^* \gets \Call{optimizeQuery}{Q',Q,\Pt,\langle\mo,\mb,\Tp,\Tn\rangle_\RQ,p}$												 		\Comment{optional}	
\State \Return $Q^*$
\end{algorithmic}
\normalsize
\end{algorithm}

\label{etc:discussion_quantitative_vs_qualitative_requirements} However, such a quantitative query quality measure $m$ alone helps only to a limited extent when it comes to devising an efficient method of detecting an optimal query w.r.t.\ $m$. More helpful is some set of \emph{qualitative} criteria in terms of ``properties'' of an optimal query that provides a means of guiding a systematic search. 

To realize this, imagine a game where the player must place $k$ tokens into $b$ buckets and afterwards gets a number of points for it. Assume that the player does not know anything about the game except that the goal is to maximize their points. What is the best the player can do?
This game can be compared to a function that must be optimized, about whose optimal input values one does not know anything. This function can be compared to a black-box. Provided with an input value, the function returns an output value. No more, no less. If the given output value is not satisfactory, all one can do is try another input value and check the result. Because one does not know in which way to alter the input, or which properties of it must be changed, in order to improve the output. In a word, one finds themself doing a brute force search. But unfortunately even worse than that. Even if one happens to come across the optimum, they would not be aware of this and stop the search. Hence, it is a brute force search with an unknown stop condition. 
Of course, one could try to learn, i.e.\ remember which token positions yielded the highest reward and in this way try to figure out the optimal token positions. However, given a huge number of possible token configurations, in the case of our game $b^k$ many,
this will anyway take substantial time given non-trivial values of $b$ and $k$. Similar is the situation concerning queries (or better:\ the associated q-partitions), where we have available $O(b^k)$ different possible input values given $b = 3$ and $k$ leading diagnoses (an explanation of this will be given in Section~\ref{sec:RelatedWork}), and must filter out one which yields a (nearly) optimal output. Consequently, one must study and analyze the functions representing the query quality measures in order to extract properties of optimal queries and be able to do better than a brute force search or, equivalently, overcome shortcoming~\ref{enum:drawback_of_std_approach_1} of the standard approaches Std (cf.\ Section~\ref{sec:Introduction}).

As we know from Section~\ref{sec:InteractiveKnowledgeBaseDebuggingBasics}, the ``properties'' of a query $Q$ are determined by the associated q-partition $\Pt(Q)$. On the one hand, $\Pt(Q)$ gives a prior estimate which and how many diagnoses might be invalidated after classification of $Q$ as a positive or negative test case (\emph{elimination rate}) and, on the other hand, together with the leading diagnoses probability distribution $p$, it provides a way to compute the likeliness of positive and negative query answers and thereby gauge the uncertainty (\emph{information gain}) of $Q$. Therefore, in Section~\ref{sec:QPartitionRequirementsSelection}, we will 
focus on a translation of measures $m$ into explicit requirements $r_m$ (function \textsc{qPartitionReqSelection}) that a q-partition $\Pt$ must meet in order for a query $Q$ with associated q-partition $\Pt(Q) = \Pt$ to optimize $m(Q)$.
So, the idea is to first search for a q-partition with desired properties $r_m$, and then compute a query for this q-partition.

In Section~\ref{sec:FindingOptimalQPartitions} we will show how the usage of requirements $r_m$ instead of $m$ can be exploited to steer the search for an optimal q-partition and present strategies facilitating to detect 
arbitrarily close-to-optimum q-partitions
completely without reasoner support (\textsc{findQPartition}). After a suitable partition $\Pt$ has been identified, a query $Q$ with q-partition $\Pt(Q) = \Pt$ is calculated such that $Q$ is optimal as to some criterion such as minimum cardinality or maximum likeliness 
of being answered correctly (\textsc{selectQueryForQPartition}). Then, in order to come up with a query that is as simple and easy to answer as possible for the respective user $U$, this query $Q$ can optionally be enriched by additional logical sentences (\textsc{enrichQuery}) by invoking a reasoner for entailments calculation. This causes a larger pool of sentences to select from in the query optimization step (\textsc{optimizeQuery}). The latter constructs a set-minimal query where most complex sentences in terms of given fault estimates w.r.t.\ logical symbols (operators) and non-logical symbols (elements of the signature, e.g.\ predicate and constant names) are eliminated from $Q$ and the most simple ones retained under preservation of $\Pt(Q) = \Pt$.\footnote{We assume for simplicity that the given fault estimates be subsumed by the probability measure $p$. For an in-depth treatment of different types of fault estimates and methods how to compute these estimates the reader is referred to \cite[Sec.~4.6]{Rodler2015phd}.} 

Note that in Algorithm~\ref{algo:query_comp} only the functions \textsc{enrichQuery} and \textsc{optimizeQuery} rely on a reasoner, and none of the functions involves user interaction. Particularly, this means that the algorithm can optionally output a query $Q$ with 
an
optimal q-partition without the need to use a reasoning service.
This is because reasoning is only required for query optimization, which is optional. If query optimization is activated, it is only employed for this \emph{single} query $Q$ with optimal q-partition. Moreover, query optimization (functions \textsc{enrichQuery} and \textsc{optimizeQuery}) overall requires only a polynomial number of reasoner calls.

\subsection{Related Work}
\label{sec:RelatedWork}

Existing works~\cite{ksgf2010,Rodler2013}, on the other hand, do not exploit qualitative requirements $r_m$ but rely on the measure $m$ using more or less a brute-force method for query selection. In the worst case, this method requires the precalculation of one (enriched) query (by use of a reasoner) for $O(2^{|\mD|})$ (q-)partitions w.r.t.\ $\mD$ for each of which the $m$ measure must be evaluated in order to identify one of sufficient quality. What could make the time complexity issues even more serious in this case is that not each partition of $\mD$ is necessarily a q-partition, and deciding this for a single partition requires $O(|\mD|)$ potentially expensive calls to a reasoner. More concretely, the method applied by these works is the following: 

\begin{enumerate}
	\item Start from a set $\dx{}(Q) \subset \mD$, use a reasoner to compute an (enriched) query candidate w.r.t.\ $\dx{}(Q)$, i.e.\ a set of common entailments $Q$ of all $\mo^*_i := (\mo \setminus \md_i) \cup \mb \cup U_\Tp$ where $\md_i\in\dx{}(Q)$, and, if $Q\neq \emptyset$ (first criterion of Definition~\ref{def:query}), use a reasoner again for the completion of the (q-)partition, i.e.\ the allocation of all diagnoses $\md \in \mD\setminus\dx{}(Q)$ to the respective set among $\dx{}(Q)$, $\dnx{}(Q)$ and $\dz{}(Q)$. The latter is necessary for the verification of whether $Q$ leads to a q-partition indeed (equivalent to the second criterion of Definition~\ref{def:query}), i.e.\ whether $\dnx{}(Q) \neq \emptyset$ (cf.\ Proposition~\ref{prop:properties_of_q-partitions},(\ref{prop:properties_of_q-partitions:enum:for_each_q-partition_dx_is_empty_and_dnx_is_empty}.)). Another reason why the completion of the q-partition is required is that query quality measures use exactly the q-partition for the assessment of the query.
	\item Continue this procedure for different sets $\dx{}(Q) \subset \mD$ until $Q$ evaluates to a (nearly) optimal $m$ measure.
	\item Then minimize $Q$ under preservation of its q-partition $\Pt(Q)$, again by means of an expensive reasoning service. The query minimization costs $O\left(|Q_{\min}| \log_2\left(\frac{|Q|}{|Q_{\min}|}\right)\right)$ calls to a reasoner and yields just \emph{any} subset-minimal query $Q_{\min}$ with q-partition $\Pt(Q)$, i.e.\ no guarantees whatsoever can be made about the output of the minimization.
\end{enumerate}

In \cite{Shchekotykhin2012}, qualitative requirements $r_m$ were already exploited together with the CKK algorithm \cite{Korf1998} for the Two-Way Number Partitioning Problem\footnote{Given a set of numbers, the \emph{Two-Way Number Partitioning Problem} is to divide them into two subsets, so that the sum of the numbers in each subset are as nearly equal as possible \cite{Korf1998}.} to accelerate the search for a (nearly) optimal query w.r.t.\ $m\in\setof{\text{entropy},\text{split-in-half}}$. However, the potential overhead of computing partitions that do not turn out to be q-partitions, and still an extensive use of the reasoner, is not solved by this approach. We tackle these problems in Section~\ref{sec:FindingOptimalQPartitions}. Apart from that, we derive qualitative requirements for numerous other measures which are not addressed in \cite{Shchekotykhin2012}. 

Furthermore, a maximum of $2^{|\mD|}$ (the number of subsets $\dx{}(Q)$ of $\mD$) of all theoretically possible $3^{|\mD|} - 2^{|\mD|+1} + 1$ q-partitions w.r.t.\ $\mD$ are considered in the existing approaches.\footnote{The number of theoretically possible q-partitions w.r.t.\ $\mD$ results from taking all partitions of a $|\mD|$-element set into 3 parts (labeled as $\dx{},\dnx{}$ and $\dz{}$) neglecting all those assigning the empty set to $\dx{}$ and/or to $\dnx{}$.} This incompleteness is basically no problem as $O(2^{|\mD|})$ is still a large number for practicable cardinalities of leading diagnoses sets, e.g.\ $|\mD|\approx 10$, that usually enables the detection of (nearly) optimal q-partitions. More problematic is the fact that no guarantee is given that no optimal q-partitions are missed while suboptimal ones are captured. Which q-partition is found for some $\dx{}(Q)$, i.e.\ to which of the sets $\dx{}(Q), \dnx{}(Q), \dz{}(Q)$ each diagnosis in $\mD \setminus \dx{}(Q)$ is assigned, depends on the (types of) common entailments calculated for the diagnoses in $\dx{}(Q)$ (cf.\ \cite[Sec.~8.2]{Rodler2015phd}). 
For example, for $|\mD| = 10$ and the set $\dx{}(Q) := \setof{\md_1,\dots,\md_5}$ there may exist the q-partitions $\Pt_1:= \langle\setof{\md_1,\dots,\md_5},\setof{\md_6,\md_7},\setof{\md_8,\dots,\md_{10}}\rangle$ and $\Pt_2 :=\langle\setof{\md_1,\dots,\md_5},\setof{\md_6,\dots,\md_{10}},\emptyset\rangle$. With existing approaches, $\Pt_1$ might be found and $\Pt_2$ might not, depending on the query $Q$ (i.e.\ the set of common entailments) calculated from $\dx{}(Q)$ that is utilized to complete the (q-)partition. However, $\Pt_2$ is strictly better even in theory than $\Pt_1$ since a query $Q_2$ associated with $\Pt_2$ is more restrictive in terms of diagnosis elimination than a query $Q_1$ for $\Pt_1$ since $\dz{}(Q_2) =\emptyset$ (cf.\ Proposition~\ref{prop:properties_of_q-partitions},(\ref{prop:properties_of_q-partitions:enum:dx_dnx_dz_contain_exactly_those_diags_that_are...}.)). This can be seen by investigating the set of invalidated leading diagnoses for positive and negative answers. That is, for positive answer, $Q_2$ eliminates $\setof{\md_6,\dots,\md_{10}}$ whereas $Q_1$ eliminates only $\setof{\md_6,\md_7}$, and for negative answer, both eliminate $\setof{\md_1,\dots,\md_5}$. 

Addressing i.a.\ this problem, \cite[Chap.~8]{Rodler2015phd} proposes and analyzes an improved variant of the basic query computation methods employed by \cite{ksgf2010,Rodler2013}. Some results shown in this work are the following: 
\begin{enumerate}
	\item \emph{Empty $\dz{}$:} Postulating the computation of (at least) a particular well-defined set of entailments when computing a set of common entailments of $\dx{}(Q)$, it can be guaranteed that only queries with empty $\dz{}$ will be computed (cf.\ \cite[Prop.~8.3]{Rodler2015phd}).
	\item \emph{No Duplicates:} The improved variant guarantees that any query (and q-partition) can be computed only once (cf.\ \cite[Prop.~8.2]{Rodler2015phd}).
	\item \emph{Completeness w.r.t.\ $\dx{}$:} The improved variant is complete w.r.t.\ the set $\dx{}$, i.e.\ if there is a query $Q$ with $\dx{}(Q) = Y$, then the improved variant will return a query $Q'$ with $\dx{}(Q') = Y$ (cf.\ \cite[Prop.~8.4]{Rodler2015phd}). 
	\item \emph{Optimal Partition Extraction w.r.t.\ Entailment Function:} If a query $Q$ with $\dx{}(Q) = Y$ is returned by the improved variant, then $Q$ is a query with minimal $|\dz{}(Q)|$ among all queries $Q$ with $\dx{}(Q) = Y$ computable by the used entailment-computation function (cf.\ \cite[Prop.~8.6]{Rodler2015phd}). 
\end{enumerate}
Note that from these properties one can derive that suboptimal q-partitions (with non-empty $\dz{}$) will (automatically) be ignored while a strictly better q-partition for each suboptimal one will be included in the query pool.

These are clearly nice properties of this improved variant. However, some issues persist (cf.\ \cite[Sec.~8.6]{Rodler2015phd}), e.g.\ the still extensive reliance upon reasoning services. Another aspect subject to improvement is the query minimization which does not enable to extract a set-minimal query with desired properties, as discussed above. However, from the point of view of how well a query might be understood by an interacting user, of course not all minimized queries can be assumed equally good in general. 

Building on the ideas of \cite{Rodler2015phd}, e.g.\ bullet (1) above, we will provide a definition of a canonical query and a canonical q-partition that will be central to our proposed heuristic query search. Roughly, a canonical query constitutes a well-defined set of common entailments that needs to be computed for some $\dx{}$ in order to guarantee the fulfillment of all properties (1)--(4) above and to guarantee that no reasoning aid is necessary during the entire query computation process. Further on, we will devise a method that guarantees that the best minimized query w.r.t.\ some desirable and plausible criterion is found. Importantly, our definitions of a canonical query and q-partition serve to overcome the non-determinism given by the dependence upon the \emph{particular} entailments or entailment types used for query computation that existing approaches suffer from.

\subsection{Active Learning in Interactive Debugging}
\label{sec:ActiveLearningInInteractiveOntologyDebugging}
In this section we discuss various (active learning) criteria that can act as a query quality measure which is given as an input $m$ to Algorithm~\ref{algo:query_comp}.

A supervised learning scenario where the learning algorithm is allowed to choose the training data from which it learns is referred to as \emph{Active Learning} \cite{settles2012}.
%
In active learning, an oracle can be queried to label any (unlabeled) instance 
belonging to a predefined input space.
The set of labeled instances is then used to build a set of most probable hypotheses. Refinement of these hypotheses is accomplished by sequential queries to the oracle and by some update operator that takes a set of hypotheses and a new labeled instance as input and returns a new set of most probable hypotheses. 
%

A first necessary step towards using active learning in our interactive debugging scenario is to think of what ``input space'', ``(un)labeled instance'', ``oracle'', ``hypothesis'' and ``update operator'' mean in this case. In terms of a query as defined in Section~\ref{sec:InteractiveKnowledgeBaseDebuggingBasics}, the input space (w.r.t.\ a set of leading diagnoses $\mD$) is the set $\mQ_{\mD,\tuple{\mo,\mb,\Tp,\Tn}_{\RQ}} := \setof{Q\,|\,Q \text{ is a query w.r.t. }\mD \text{ and }\tuple{\mo,\mb,\Tp,\Tn}_{\RQ}}$. An unlabeled instance can be seen as an (unanswered) query, denoted by $(Q,?)$, and a labeled instance as a query together with its answer, written as $(Q,u(Q))$.
The oracle in this case corresponds to the user interacting with the debugging system. Hypothesis refers to a model learned by the algorithm that is consistent with all instances labeled so far; in the debugging case a hypothesis corresponds to a diagnosis and the most likely hypotheses correspond to the set of leading diagnoses $\mD$. In addition to the most likely hypotheses there is often an associated probability distribution. In our debugging scenario this is the probability distribution $p$ assigning a probability to each leading diagnosis in $\mD$. The Bayesian update of $p$ together with the computation of new leading diagnoses 
(cf.\ Algorithm~\ref{algo:inter_onto_debug})
act as update operator.

An active learning strategy can be characterized by two parameters, the \emph{active learning mode} and the \emph{selection criterion}~\cite{settles2012}. The mode refers to the way how the learning algorithm can access instances in the input space. The selection criterion, usually given as some quantitative measure, specifies whether an accessed instance should be shown to the oracle or not. So, the selection criterion corresponds to the measure $m$ described above that is used for partition assessment in Algorithm~\ref{algo:query_comp}.

In~\cite{settles2012}, different active learning modes are distinguished. 
These are \emph{membership query synthesis~(QS)}, \emph{stream-based selective sampling~(SS)} and \emph{pool-based sampling~(PS)}. QS assumes that the learner at each step generates any unlabeled instance within the input space as a query without taking into account the distribution of the input space. In this vein, lots of meaningless queries may be created, which is a major downside of this approach.
In SS mode, the learner is provided with single instances sampled from the input space as per the underlying distribution, one-by-one, and for each instance decides whether to use it as a query or not. 
PS, in contrast, assumes availability of a pool of unlabeled instances, from which the best instance according to some measure is chosen as a query. 
Contrary to the works~\cite{ksgf2010, Rodler2013, Rodler2015phd} which rely on a PS scenario, thus requiring the precomputation of a pool of queries (or q-partitions), we implement an approach that unites characteristics of QS and PS. Namely, we generate q-partitions one-by-one in a best-first manner, 
quasi on demand, and proceed until a (nearly) optimal q-partition is found. This is similar to what is done for the entropy ($\mathsf{ENT}$) measure (see below) in~\cite[Sec.~4]{Shchekotykhin2012}, but we will optimize and generalize this approach in various aspects.

In the subsequent subsections we will describe and analyze diverse active learning query quality measures and adapt them to our heuristic-based approach. 

\begin{remark}\label{rem:query_quality_measures_called_q-partition_quality_measures}
Although these measures serve the purpose of the determination of the best \emph{query} w.r.t.\ the diagnosis elimination rate (cf.\ \cite{Rodler2013} and Section~\ref{sec:rio}), the information gain (cf.\ \cite{dekleer1987,Shchekotykhin2012}) or other criteria (cf.\ \cite{settles2012}), a more appropriate name would be \emph{q-partition quality measures}.
The first reason for this is that we will, in a two step optimization approach, first optimize the q-partition w.r.t.\ such a measure and second compute an optimal query (amongst exponentially many, in general) for this fixed q-partition w.r.t.\ different criteria such as the answer fault probability of the query. 
The second reason is that the (adapted) active learning measures will only be relevant to the primary optimization step w.r.t.\ the q-partition in the course of this approach. This means, all properties of an optimal query w.r.t.\ any of the measures discussed in the following subsections \ref{sec:QPartitionSelectionMeasures} and \ref{sec:rio} will solely depend on the q-partition associated with this query and not on the query itself. Nevertheless, since we defined these measures to map a \emph{query} to some value, we will stick to the denotation \emph{query quality measure} instead of \emph{q-partition quality measure} throughout the rest of this work.\qed
\end{remark} 

\subsubsection{Query Quality Measures: Preliminaries and Definitions}
\label{sec:QPartitionSelectionMeasures}
In the following we discuss various measures for query selection. In doing so, we draw on general active learning strategies, demonstrate relationships between diverse strategies and adapt them for employment in our debugging scenario, in particular within the function \textsc{findQPartition} in Algorithm~\ref{algo:query_comp}. Note that these active learning measures 
\emph{aim at minimizing the overall number of queries} until the correct solution KB is found. The minimization of a query in terms of number of logical formulas asked to the user is done in the \textsc{selectQueryForQPartition} and \textsc{optimizeQuery} functions in Algorithm~\ref{algo:query_comp}.

Concerning the notation, we will, for brevity, use $\mQ$ to refer to $\mQ_{\mD,\tuple{\mo,\mb,\Tp,\Tn}_{\RQ}}$ throughout the rest of Section~\ref{sec:ActiveLearningInInteractiveOntologyDebugging} unless stated otherwise. Moreover, we define according to \cite{dekleer1987,Shchekotykhin2012,Rodler2013,Rodler2015phd} 
\begin{align}
p(\mathbf{X}) := \sum_{\md \in \mathbf{X}} p(\md)      \label{eq:sum_of_probs_of_diagnoses_set}
\end{align}
for some set of diagnoses $\mathbf{X}$, with assuming that for $\mathbf{X} := \mD$ this sum is equal to 1 \cite{Rodler2015phd} (this can always be achieved by normalizing each diagnosis probability $p(\md)$ for $\md\in\mD$, i.e.\ by multiplying it with $\frac{1}{p(\mD)}$). Further, it must hold for each $\md \in \mD$ that $p(\md) > 0$,\label{etc:prob_of_each_diag_must_be_greater_zero} i.e.\ each diagnosis w.r.t.\ a given DPI is possible, howsoever improbable it might be (cf.\ \cite{Rodler2015phd}). For some $Q\in\mQ$ we define
\begin{align}
p(Q = t) := p(\dx{}(Q))+\frac{1}{2}p(\dz{}(Q))  \label{eq:p(Q=t)}
\end{align}
which is why 
\begin{align}
p(Q = f) = 1-p(Q=t) = p(\dnx{}(Q))+\frac{1}{2}p(\dz{}(Q))  \label{eq:p(Q=f)}
\end{align}
Further on, the posterior probability of a diagnosis can be computed by the Bayesian Theorem as
\begin{align*} 
p(\md\,|\,Q=a) = \frac{p(Q=a\,|\,\md)\;\,p(\md)}{p(Q=a)} 
\end{align*}
where 
\begin{equation}\label{eq:cond_query_prob}
p(Q=t\,|\,\md) := 
\begin{cases}
1, 						& \mbox{if } \md \in \dx{}(Q) \\   
0, 						& \mbox{if } \md \in \dnx{}(Q) \\
\frac{1}{2}, 	& \mbox{if } \md \in \dz{}(Q) 
\end{cases} 
\end{equation}
For further information on these definitions have a look at \cite[Chap.~4 and 9]{Rodler2015phd} and \cite{dekleer1987}. Finally, we note that we will treat the expression $0 \log_2 0$\label{convention:0log0_is_0} and as being equal to zero in the rest of this section.\footnote{This assumption is plausible since it is easy to show by means of the rule of de l'Hospital that $x \log_2 x$ converges to zero for $x \rightarrow 0^+$.}

In the following, we define some new notions about query quality measures and queries which shall prove helpful to formalize and state precisely what we want to mean when calling 
\begin{enumerate}
	\item a query $Q$ \emph{preferred} to another query $Q'$ by some measure $m$,
	\item a query $Q$ \emph{discrimination-preferred} to another query $Q'$,
	\item a query $Q$ \emph{theoretically optimal} w.r.t.\ some measure $m$,
	\item two measures \emph{equivalent} to each other,
	\item a measure $m$ \emph{consistent with the discrimination-preference relation} 
	\item a measure $m$ \emph{discrimination-preference relation preserving}, and
	\item one measure $m_1$ \emph{superior} to another measure $m_2$.
\end{enumerate}
Before giving formal definitions of the notions, we state roughly for all notions in turn what they mean:
\begin{enumerate}
	\item Means that the measure $m$ would select $Q$ as the best query among $\setof{Q,Q'}$.
	\item Means that in the worst case at least as many leading diagnoses (or: known hypotheses) are eliminated using $Q$ as a query than by using $Q'$ as a query, and in the best case more leading diagnoses (or: known hypotheses) are eliminated using $Q$ as a query than by using $Q'$ as a query.
	\item Means that no modification of the properties of $Q$ (or more precisely: the q-partition of $Q$, cf.\ Remark~\ref{rem:query_quality_measures_called_q-partition_quality_measures}) could yield a query that would be preferred to $Q$ by $m$.
	\item Means that both measures impose exactly the same preference order on any set of queries.
	\item Means that if $m$ prefers $Q$ to $Q'$, then $Q'$ cannot be discrimination-preferred to $Q$.
	\item Means that if $Q$ is discrimination-preferred to $Q'$, then $m$ also prefers $Q$ to $Q'$.
	\item Means that the set of query pairs for which $m_1$ satisfies the discrimination-preference relation is a proper superset of the set of query pairs for which $m_2$ satisfies the discrimination-preference relation.
\end{enumerate}

 Every query quality measure imposes a (preference) order on a given set of queries:
\begin{definition}\label{def:precedence_order_defined_by_q-partition_quality_measure} 
Let $m$ be a query quality measure, i.e.\ $m$ assigns a real number $m(Q)$ to every query $Q$. Then we want to say that \emph{$Q$ is preferred to $Q'$ by $m$}, formally $Q \prec_m Q'$, iff 
\begin{itemize}
	\item $m(Q) < m(Q')$ in case lower $m(Q)$ means higher preference of $Q$ by $m$,
	\item $m(Q) > m(Q')$ in case higher $m(Q)$ means higher preference of $Q$ by $m$.
\end{itemize}
\end{definition}
It can be easily verified that the preference order imposed on a set of queries by a measure $m$ is a strict order:
\begin{proposition}\label{prop:precedence_order_is_strict_order}
Let $m$ be a query quality measure. Then $\prec_m$ is a strict order, i.e.\ it is an irreflexive, asymmetric and transitive relation over queries.
\end{proposition}
The following definition captures what we want to understand when we call two measures equivalent, a query theoretically optimal w.r.t.\ a measure and a measure superior to another measure.
\begin{definition}\label{def:measures_equivalent_theoretically-optimal_superior}
Let $m,m_1,m_2$ be query quality measures and let $\mD$ be a set of leading diagnoses w.r.t.\ a DPI $\tuple{\mo,\mb,\Tp,\Tn}_\RQ$. Then:
\begin{itemize}
	\item We call \emph{$m_1$ equivalent to $m_2$}, formally $m_1\equiv m_2$, iff the following holds for all queries $Q,Q'$: $Q \prec_{m_1} Q'$ iff $Q \prec_{m_2} Q'$.
	\item We call \emph{$m_1$ $\mathbf{X}$-equivalent to $m_2$}, formally $m_1\equiv_{\mathbf{X}} m_2$, iff the following holds for all queries $Q,Q'\in \mathbf{X}$: $Q \prec_{m_1} Q'$ iff $Q \prec_{m_2} Q'$.
	\item 
	We call a query $Q$ \emph{theoretically optimal w.r.t.\ $m$ and $\mD$ (over the domain $Dom$)} iff $m(Q)$ is a global optimum of $m$ (over $Dom$). 
	
	If no domain $Dom$ is stated, then we implicitly refer to the domain of all possible partitions $\tuple{\dx{},\dnx{},\dz{}}$ (and, if applicable and not stated otherwise, all possible (diagnosis) probability measures $p$) w.r.t.\ a set of leading diagnoses $\mD$. If a domain $Dom$ is stated, we only refer to queries in this domain of queries $Dom$ over $\mD$ (and, if applicable and not stated otherwise, all possible (diagnosis) probability measures $p$).
	(The name theoretically optimal stems from the fact that there must not necessarily be such a query for a given set of leading diagnoses. In other words, the optimal query among a set of candidate queries might not be theoretically optimal.) 
	\item We call a query $Q$ \emph{discrimination-preferred to $Q'$} w.r.t.\ a DPI $\tuple{\mo,\mb,\Tp,\Tn}_\RQ$ and a set of minimal diagnoses $\mD$ w.r.t.\ this DPI iff there is an injective function $f$ that maps each answer $a'$ to $Q'$ to an answer $a = f(a')$ to $Q$ such that
	\begin{enumerate}
		\item $Q$ answered by $f(a')$ eliminates a superset of the diagnoses in $\mD$ eliminated by $Q'$ answered by $a'$ (i.e.\ the worst case answer to $Q$ eliminates at least all leading diagnoses eliminated by the worst case answer to $Q'$), and
		\item there is an answer $a^*$ to $Q'$ such that $Q$ answered by $f(a^*)$ eliminates a proper superset of the diagnoses in $\mD$ eliminated by $Q'$ answered by $a^*$ (i.e.\ the best case answer to $Q$ eliminates all leading diagnoses eliminated by the best case answer to $Q'$ and at least one more).
	\end{enumerate}
	\item We call the relation that includes all tuples $(Q,Q')$ where $Q$ is discrimination-preferred to $Q'$ \emph{the discrimination-preference relation} and denote it by $DPR$.
	\item We say that $m$ 
	\begin{itemize}
		\item \emph{preserves (or: satisfies) the discrimination-preference relation (over $\mathbf{X}$)} iff whenever $(Q,Q') \in DPR$ (and $Q,Q' \in \mathbf{X}$), it holds that $Q \prec_{m} Q'$ (i.e.\ the preference relation imposed on a set of queries by $m$ is a superset of the discrimination-preference relation).
		\item \emph{is consistent with the discrimination-preference relation (over $\mathbf{X}$)} iff whenever $(Q,Q') \in DPR$ (and $Q,Q' \in \mathbf{X}$), it does not hold that $Q' \prec_{m} Q$ (i.e.\ the preference relation imposed on a set of queries by $m$ has an empty intersection with the inverse discrimination-preference relation).
	\end{itemize}
	\item We call $m_2$ \emph{superior to }$m_1$ (or: $m_1$ \emph{inferior to }$m_2$), formally $m_2 \prec m_1$, iff 
	\begin{enumerate}
		\item there is some pair of queries $Q,Q'$ such that $Q$ is discrimination-preferred to $Q'$ and not $Q \prec_{m_1} Q'$ (i.e.\ $m_1$ does not satisfy discrimination-preference order),
		\item for some pair of queries $Q,Q'$ where $Q$ is discrimination-preferred to $Q'$ and not $Q \prec_{m_1} Q'$ it holds that $Q \prec_{m_2} Q'$ (i.e.\ in some cases $m_2$ is better than $m_1$),\footnote{Note that the second condition (2.)\ already implies the first condition (1.). Nevertheless, we state the first condition explicitly for better understandability.} and
		\item for no pair of queries $Q,Q'$ where $Q$ is discrimination-preferred to $Q'$ and not $Q \prec_{m_2} Q'$ it holds that $Q \prec_{m_1} Q'$ (i.e.\ $m_1$ is never better than $m_2$).
	\end{enumerate}
	\item Analogously, we call $m_2$ \emph{$\mathbf{X}$-superior to }$m_1$ (or: $m_1$ \emph{$\mathbf{X}$-inferior to }$m_2$), formally $m_2 \prec_{\mathbf{X}} m_1$, iff 
	\begin{enumerate}
		\item there is some pair of queries $Q,Q'\in\mathbf{X}$ such that $Q$ is discrimination-preferred to $Q'$ and not $Q \prec_{m_1} Q'$,
		\item for some pair of queries $Q,Q'\in\mathbf{X}$ where $Q$ is discrimination-preferred to $Q'$ and not $Q \prec_{m_1} Q'$ it holds that $Q \prec_{m_2} Q'$, and
		\item for no pair of queries $Q,Q'\in\mathbf{X}$ where $Q$ is discrimination-preferred to $Q'$ and not $Q \prec_{m_2} Q'$ it holds that $Q \prec_{m_1} Q'$.
	\end{enumerate}
\end{itemize}
\end{definition}
The next propositions are simple consequences of this definition (those stated without a proof are very easy to verify).
\begin{proposition}\label{prop:equivalent_measures_suggest_identical_preference_orders}
Any two ($\mathbf{X}$-)equivalent measures impose identical preference orders on any given fixed set of queries (in $\mathbf{X}$).
\end{proposition}
\begin{proposition}\label{prop:equivalent_measures_suggest_equivalent_optimal_queries}
Let $Q_{m_i} (\in \mQ)$ denote the optimal query w.r.t.\ the measure $m_i, i\in \setof{1,2}$ (in a set of queries $\mQ$). Then $Q_{m_1} = Q_{m_2}$ if $m_1 \equiv_{(\mQ)} m_2$.
\end{proposition}
\begin{proposition}\label{prop:equivalence_between_measures_is_equivalence_relation}
The relations $\equiv$ and $\equiv_{\mathbf{X}}$ for any $\mathbf{X}$ are equivalence relations.
\end{proposition}
\begin{proposition}\label{prop:superiority_is_strict_order}
The superiority ($\prec$) and $\mathbf{X}$-superiority ($\prec_{\mathbf{X}}$) relations are strict orders.
\end{proposition}
\begin{proof}
We have to show the irreflexivity, asymmetry and transitivity of $\equiv$ and $\equiv_{\mathbf{X}}$. We give a proof for $\equiv$. The proof for $\equiv_{\mathbf{X}}$ is completely analogous. Irreflexivity is clearly given due to the second condition of the definition of superiority (Definition~\ref{def:measures_equivalent_theoretically-optimal_superior}) because $\lnot(Q \prec_{m_1} Q')$ cannot hold at the same time as $Q \prec_{m_1} Q'$. 

To verify asymmetry, let us restate the third condition of the definition of superiority in an equivalent, but different manner: (3.')~For all pairs of queries $Q,Q'$ where $Q$ is discrimination-preferred to $Q'$ and $Q \prec_{m_1} Q'$ it holds that $Q \prec_{m_2} Q'$. Now, for the assumption that $m_2$ is superior to $m_1$, we have that $\lnot(Q \prec_{m_1} Q') \land Q \prec_{m_2} Q'$ must hold for some pair of queries $Q,Q'$ where $Q$ is discrimination-preferred to $Q'$ (by condition (2.)\ of the definition of superiority). If $m_1$ were at the same time superior to $m_2$, then by (3.'), we would obtain $Q \prec_{m_2} Q' \Rightarrow Q \prec_{m_1} Q'$ for all pairs of queries $Q,Q'$ where $Q$ is discrimination-preferred to $Q'$. This is clearly a contradiction. Hence, $m_1 \prec m_2$ implies $\lnot(m_2 \prec m_1)$, i.e.\ asymmetry is given. 

To check transitivity, let us suppose that (i)~$m_1 \prec m_2$ as well as (ii)~$m_2 \prec m_3$. We now demonstrate that $m_1 \prec m_3$. By (i) and condition (1.)\ of the definition of superiority, we know that there is a pair of queries $Q,Q'$ where $Q$ is discrimination-preferred to $Q'$ and not $Q \prec_{m_1} Q'$. This is at the same time the first condition that must hold for $m_1 \prec m_3$ to be met. By (i) and condition (2.)\ of the definition of superiority, we know that for some pair of queries $Q,Q'$ where $Q$ is discrimination-preferred to $Q'$ it holds that $\lnot(Q \prec_{m_1} Q') \land Q \prec_{m_2} Q'$. Due to $Q \prec_{m_2} Q'$ and (3.')\ we have that in this case $Q \prec_{m_3} Q'$ is also true. This implies $\lnot(Q \prec_{m_1} Q') \land Q \prec_{m_3} Q'$ which is why the second condition that must hold for $m_1 \prec m_3$ to be met is satisfied. The third condition can be directly deduced by applying (3.')\ to (i) and (ii). Thence, also the transitivity of $\prec$ is given.
\end{proof}
Due to the next proposition, we call the discrimination-preference relation also \emph{the discrimination-preference order}.
\begin{proposition}\label{prop:equivalence_between_measures_is_equivalence_relation}
The discrimination-preference relation $DPR$ is a strict order, i.e.\ it is irreflexive, asymmetric and transitive.
\end{proposition}
\begin{proposition}\label{prop:if_m1_satisfies_DPR_and_m2_does_not_then_m1_superior_to_m2}
Given two measures $m_1,m_2$ where $m_1$ does and $m_2$ does not satisfy the discrimination-preference relation. Then, $m_1$ is superior to $m_2$.
\end{proposition}
\begin{proposition}\label{prop:if_Q_discrimination-preferred_over_Q'_then_dz(Q')_supset_dz(Q)}
Let $Q$ be discrimination-preferred to $Q'$ w.r.t.\ a DPI $\tuple{\mo,\mb,\Tp,\Tn}_\RQ$ and a set of minimal diagnoses $\mD$ w.r.t.\ this DPI. Then $\dz{}(Q') \supset \dz{}(Q)$. In particular, it holds that $\dz{}(Q') \neq \emptyset$.
\end{proposition}
\begin{proposition}\label{prop:construction_of_Q'_from_Q_if_Q_discrimination-preferred_over_Q'}
Let $Q \in \mQ_{\mD,\tuple{\mo,\mb,\Tp,\Tn}_\RQ}$ be a query. Then any query $Q' \in \mQ_{\mD,\tuple{\mo,\mb,\Tp,\Tn}_\RQ}$ for which it holds that $Q$ is discrimination-preferred to $Q'$ can be obtained from $Q$ by transferring a non-empty set $\mathbf{X} \subset \dx{}(Q) \cup \dnx{}(Q)$ to $\dz{}(Q)$ and by possibly interchanging the positions of the resulting sets $\dx{}(Q) \setminus \mathbf{X}$ and $\dnx{}(Q) \setminus \mathbf{X}$ in the q-partition, i.e.\ 
\[\Pt(Q') = \tuple{\dx{}(Q'),\dnx{}(Q'),\dz{}(Q')} = \langle\dx{}(Q) \setminus \mathbf{X}, \dnx{}(Q) \setminus \mathbf{X},\dz{}(Q) \cup \mathbf{X}\rangle\]
or 
\[\Pt(Q') = \tuple{\dx{}(Q'),\dnx{}(Q'),\dz{}(Q')} = \langle\dnx{}(Q) \setminus \mathbf{X}, \dx{}(Q) \setminus \mathbf{X},\dz{}(Q) \cup \mathbf{X}\rangle\]
\end{proposition}
\begin{proof}
By Proposition~\ref{prop:properties_of_q-partitions},(\ref{prop:properties_of_q-partitions:enum:dx_dnx_dz_contain_exactly_those_diags_that_are...}.), of all diagnoses in $\mD$ exactly the diagnoses $\dx{}(Q)$ are eliminated for the negative answer and exactly the diagnoses $\dnx{}(Q)$ are eliminated for the positive answer to $Q$. 
By the definition of discrimination-preference between two queries (Definition~\ref{def:measures_equivalent_theoretically-optimal_superior}), $\dx{}(Q')$ must be a subset of one element in $\setof{\dx{}(Q),\dnx{}(Q)}$ and $\dnx{}(Q')$ must be a subset of the other element in $\setof{\dx{}(Q),\dnx{}(Q)}$ where one of these subset relationships must be proper. 
This clearly means that either (a)~$\dx{}(Q') = \dx{}(Q) \setminus \mathbf{X}$ and $\dnx{}(Q') = \dnx{}(Q) \setminus \mathbf{X}$ or (b)~$\dx{}(Q') = \dnx{}(Q) \setminus \mathbf{X}$ and $\dnx{}(Q') = \dx{}(Q) \setminus \mathbf{X}$ for some $\mathbf{X}\subset\dx{}(Q) \cup \dnx{}(Q)$. The fact that one subset relationship mentioned above must be proper enforces additionally that $\mathbf{X} \neq \emptyset$. Due to the assumption that $Q'$ is a query and hence $\dx{}(Q')\cup\dnx{}(Q')\cup\dz{}(Q') = \mD$ must hold (Proposition~\ref{prop:properties_of_q-partitions},(\ref{prop:properties_of_q-partitions:enum:q-partition_is_partition}.)), the set of diagnoses $\mathbf{X}$ eliminated from $\dx{}(Q) \cup \dnx{}(Q)$ must be added to $\dz{}(Q)$ to obtain $Q'$.

Moreover, the possibility of transferring some diagnoses from $\dx{}(Q)$ to $\dnx{}(Q)$ or from $\dnx{}(Q)$ to $\dx{}(Q)$ in order to obtain $Q'$ is ruled out. For instance, assume that $\mathbf{S} \subset \dx{}(Q)$ is transferred to $\dnx{}(Q)$, i.e.\ $\dnx{}(Q') = \dnx{}(Q) \cup \mathbf{S}$, then $\dnx{}(Q')$ cannot be a subset of $\dnx{}(Q)$ since $\dnx{}(Q')\supset\dnx{}(Q)$ and it cannot be a subset of $\dx{}(Q)$. The latter must hold since $Q$ is a query by assumption which implies that $\dnx{}(Q) \neq \emptyset$ (Proposition~\ref{prop:properties_of_q-partitions},(\ref{prop:properties_of_q-partitions:enum:for_each_q-partition_dx_is_empty_and_dnx_is_empty}.)) and $\dx{}(Q) \cap \dnx{}(Q) = \emptyset$ (Proposition~\ref{prop:properties_of_q-partitions},(\ref{prop:properties_of_q-partitions:enum:q-partition_is_partition}.)) which is why there must be a diagnosis $\md \in \dnx{}(Q')$ which is not in $\dx{}(Q)$.

Overall, we have derived that all queries $Q'$ where $Q$ is discrimination-preferred to $Q'$ can be obtained from $Q$ only in either of the two ways stated in the proposition. This completes the proof.
%
\end{proof}
%
According to \cite[Def.~6.1]{Rodler2015phd} we define the true diagnosis as follows: 
\begin{definition}\label{def:true_diagnosis}
Given a DPI $\tuple{\mo,\mb,\Tp,\Tn}_\RQ$, we call $\md_t \subseteq \mo$ the \emph{true diagnosis} iff $\md_t$ is the diagnosis that remains as a single possible solution at the end of a diagnostic session, i.e.\ after successively answering queries until a single diagnosis is left.  
\end{definition}
\begin{property}\label{property:if_correct_diag_in_then_user_answer_is}
Let $\md_t \subseteq \mo$ denote the true diagnosis w.r.t.\ a DPI $\tuple{\mo,\mb,\Tp,\Tn}_\RQ$ and $Q$ be an arbitrary query w.r.t.\ some $\mD \subseteq \minD_{\tuple{\mo,\mb,\Tp,\Tn}_\RQ}$. Then, $u(Q) = t$ given that $\md_t \in \dx{}(Q)$ and $u(Q) = f$ given that $\md_t \in \dnx{}(Q)$. In all other cases, no statement about the answer of the oracle (user) $u$ as to $Q$ can be made.
\end{property}
If the true diagnosis $\md_t$ is an element of the leading diagnoses $\mD$ and predicts an answer for both queries $Q_1, Q_2$ w.r.t.\ $\mD$ where $Q_1$ is discrimination-preferred to $Q_2$, then the answer $u(Q_1)$ to $Q_1$ will definitely eliminate as many or strictly more leading diagnoses than the answer $u(Q_2)$ to $Q_2$:
\begin{proposition}\label{prop:Q1_DPR_Q2=>Q1_eliminates_more_than_Q1_anyway}
Let $Q_1,Q_2$ be queries w.r.t.\ some set of leading diagnoses $\mD$ for a DPI $\tuple{\mo,\mb,\Tp,\Tn}_\RQ$ where $(Q_1,Q_2) \in DPR$ and let $\md_t$ be the true diagnosis. Further, let $\md_t \in \dx{}(Q_1) \cup \dnx{}(Q_1)$ and $\md_t \in \dx{}(Q_2) \cup \dnx{}(Q_2)$. Finally, let $\tuple{\mo,\mb,\Tp^{(1)},\Tn^{(1)}}_\RQ$ and $\tuple{\mo,\mb,\Tp^{(2)},\Tn^{(2)}}_\RQ$ denote the DPIs resulting from $\tuple{\mo,\mb,\Tp,\Tn}_\RQ$ after incorporating the answers $u(Q_1)$ and $u(Q_2)$ (i.e.\ after adding $Q_i$ to $\Tp$ in case of $u(Q_i) = t$ and to $\Tn$ otherwise), respectively. Then $\mD\cap\minD_{\tuple{\mo,\mb,\Tp^{(1)},\Tn^{(1)}}_\RQ} \subseteq \mD\cap\minD_{\tuple{\mo,\mb,\Tp^{(2)},\Tn^{(2)}}_\RQ}$, i.e.\ the set of remaining leading diagnoses after knowing $Q_1$'s answer is a subset of the set of remaining leading diagnoses after knowing $Q_2$'s answer. 
\end{proposition} 
\begin{proof}
The proposition is a direct consequence of the definition of the discrimination-preference relation, Property~\ref{property:if_correct_diag_in_then_user_answer_is} as well as Proposition~\ref{prop:properties_of_q-partitions},(\ref{prop:properties_of_q-partitions:enum:dx_dnx_dz_contain_exactly_those_diags_that_are...}.). 
\end{proof}
\begin{proposition}\label{prop:linear_function_of_measure_is_equivalent_to_measure_itself}
Let $m$ be a query quality measure and let $f(m)$ be the measure resulting from the application of a linear function $f(x) = ax+b$ with $a \in \mathbb{R}^+ \setminus \setof{0}$ and $b \in \mathbb{R}$ to $m$. Then $f(m) \equiv m$. 
\end{proposition}
Example~\ref{ex:measure_notions_illustrated} on page~\pageref{ex:measure_notions_illustrated} will illustrate the notions given in this section. However, before, in Section~\ref{sec:ExistingActiveLearningMeasuresForKBDebugging}, we present and discuss two query quality measures that have already been extensively analyzed in (debugging) literature \cite{ksgf2010, Shchekotykhin2012, Rodler2013, Rodler2015phd}, which we will use in the example.
%
%
\subsubsection{Existing Active Learning Measures for KB Debugging}
\label{sec:ExistingActiveLearningMeasuresForKBDebugging}
Active learning strategies
that use \emph{information theoretic measures} as selection criterion 
have been shown to reduce the number of queries as opposed to standard ``non-active'' supervised learning methods in many diverse application scenarios such as text, video and image classification, information extraction, speech recognition and cancer diagnosis~\cite{settles2012}. 
Also in the field of interactive ontology debugging, two active learning measures have been adopted~\cite{ksgf2010,Shchekotykhin2012,Rodler2013} and shown to clearly outperform a (non-active) random strategy of selecting the next query~\cite{Shchekotykhin2012}. These are \emph{entropy-based} ($\mathsf{ENT}$) and \emph{split-in-half} ($\mathsf{SPL}$) query selection. 

\vspace{1em}

\noindent\emph{Entropy-Based Query Selection ($\mathsf{ENT}$).} The best query $Q_{\mathsf{ENT}}$ according to $\mathsf{ENT}$ has the property to maximize the information gain or equivalently to minimize the expected entropy in the set of leading diagnoses after $Q_{\mathsf{ENT}}$ is classified by the user and added as a test case to the DPI. 
Concretely,
\begin{align}\label{eq:best_query_ENT}
Q_{\mathsf{ENT}} = \argmin_{Q \in \mQ} \left(\mathsf{ENT}(Q)\right) 
\end{align}
where
\begin{align}
\mathsf{ENT}(Q) := \sum_{a \in \setof{t,f}} p(Q = a) \cdot \left[-\sum_{\md \in \mD} p(\md\,|\,Q=a)\log_2 p(\md\,|\,Q=a)\right]
\label{eq:ENT}
\end{align} 
Theoretically optimal w.r.t.\ $\mathsf{ENT}$ is a query whose positive and negative answers are equally likely and none of the diagnoses $\md \in \mD$ is consistent with both answers. 
\begin{proposition}\label{prop:ent}
Let $\mD$ be a set of leading diagnoses w.r.t.\ a DPI $\tuple{\mo,\mb,\Tp,\Tn}_\RQ$. Then:
\begin{enumerate}
	\item A query $Q$ is theoretically optimal w.r.t.\ $\mathsf{ENT}$ and $\mD$ iff $p(\dx{}(Q))=p(\dnx{}(Q))$ as well as ${p(\dz{}(Q)) = 0}$.
	\item Let $q$ be a fixed number in $[0,1]$ and let $\mQ$ be a set of queries where each query $Q\in\mQ$ satisfies $p(\dz{}(Q)) = q$. Then the theoretically optimal query w.r.t.\ $\mathsf{ENT}$ over $\mQ$ satisfies $p(Q=t)=p(Q=f)$ and $p(\dx{}(Q))=p(\dnx{}(Q))$.
\end{enumerate}
\end{proposition}
\begin{proof}
In~\cite{dekleer1987} it was shown that $\mathsf{ENT}(Q)$ can be equivalently transformed into 
\begin{align}
\left[\sum_{a\in\setof{t,f}} p(Q=a) \log_2 p(Q=a)\right] + p(\dz{}(Q)) + 1  \label{eq:scoring_funtion_dekleer}
\end{align}
First, observe that, by definition of $p(Q=a)$ and since $Q$ is a query (for which $p(\dx{}(Q)) > 0$ and $p(\dnx{}(Q))>0$), $p(\dz{}(Q)) < 2 \min_{a\in\setof{t,f}}(p(Q=a))$. If $p(\dz{}(Q)) = 0$, this inequality is certainly fulfilled for \emph{any values} $p(Q=0)$ and $p(Q=1)$. Since Eq.~\eqref{eq:scoring_funtion_dekleer} must be minimized and $p(\dz{}(Q))$ is a positive summand in the formula, $p(\dz{}(Q)) = 0$ minimizes this summand.

Now, the sum in squared brackets can be minimized separately. 
We can write the sum $\sum_{a\in\setof{t,f}} p(Q=a) \log_2 p(Q=a)$ as $x\log_2 x + (1-x)\log_2 (1-x)$ for $0<x:=p(Q=t)<1$. Differentiation by $x$ and setting the derivative equal to zero yields $\log_2 x = \log_2 (1-x)$. By the fact that $\log_2 x$ is monotonically increasing, this implies that $x = p(Q=t)=p(Q=f)=\frac{1}{2}$. Due to Eq.~\eqref{eq:p(Q=t)} and Eq.~\eqref{eq:p(Q=f)}, this implies that $p(\dx{}(Q))=p(\dnx{}(Q))$. Building the second derivative yields $\frac{1}{x-x^2}$ which is positive for $x \in (0,1)$ since $x^2 < x$ in this interval. Hence $p(Q=t)=\frac{1}{2}$ is a local minimum. Since however the sum $\sum_{a\in\setof{t,f}} p(Q=a) \log_2 p(Q=a)$ is a strictly convex function and $(0,1)$ a convex set, $p(Q=t)=\frac{1}{2}$ must be the (unique) global minimum as well. Hence, the global minimum is attained if $p(Q=t)=p(Q=f)$. By Eq.~\eqref{eq:p(Q=t)} and Eq.~\eqref{eq:p(Q=f)} this is equivalent to $p(\dx{}(Q))=p(\dnx{}(Q))$. This proves statement~(2.) of the proposition.
To optimize the entire formula for $Q_{\mathsf{ENT}}$, $p(\dz{}(Q))$ must additionally be minimized as $p(\dz{}(Q))$ is a positive addend in the formula. This proves statement~(1.) of the proposition. 
\end{proof}
%
The following proposition analyzes for which classes of queries the $\mathsf{ENT}$ measure does and does not preserve the discrimination-preference order (according to Definition~\ref{def:measures_equivalent_theoretically-optimal_superior}). Before stating the proposition, however, we define a measure $\mathsf{ENT}_z$ that is used in the proposition and constitutes a generalization of the $\mathsf{ENT}$ measure. Namely, $\mathsf{ENT}_z$ selects the query  
\begin{align}\label{eq:best_query_ENTz}
Q_{\mathsf{ENT}_z} := \argmin_{Q \in \mQ} \left(\mathsf{ENT}_z(Q)\right)
\end{align}
where
\begin{align}\label{eq:ENTz}
\mathsf{ENT}_z(Q) := \left[\sum_{a\in\setof{t,f}} p(Q=a) \log_2 p(Q=a)\right] + z\,p(\dz{}(Q)) + 1
\end{align}
%
Note that $\mathsf{ENT}_z(Q) = \mathsf{ENT}(Q) + (z-1)p(\dz{}(Q))$ and thus that $\mathsf{ENT}_1$ is equal to $\mathsf{ENT}$ and $\mathsf{ENT}_0 - 1$ is equal to $\mathsf{H}$, $\mathsf{LC}$ and $\mathsf{M}$ (cf.\ Proposition~\ref{prop:uncertainty_sampling} and the proof of Proposition~\ref{prop:Q_H_eq_Q_ENT_if_dz=0}).
\begin{proposition}\label{prop:ENT_preserves_discrimination-pref-order}
Let $Q$ be a query w.r.t.\ a set of minimal diagnoses $\mD$ w.r.t.\ a DPI $\tuple{\mo,\mb,\Tp,\Tn}_\RQ$. Further, let $p:=\min_{a\in\setof{t,f}} (p(Q=a))$ and $Q'$ be a query such that $Q$ is discrimination-preferred to $Q'$. In addition, let $x:= p(\dz{}(Q')) - p(\dz{}(Q))$. 
Then:
\begin{enumerate}
	\item $x > 0$.
	\item If $x \geq 1-2p$, then $Q \prec_{\mathsf{ENT}} Q'$.
	\item If $x < 1-2p$, then: 
	\begin{enumerate}
		\item In general it does not hold that $Q \prec_{\mathsf{ENT}} Q'$.
		\item If not $Q \prec_{\mathsf{ENT}} Q'$, then $p \in (0,t]$ where $t := \frac{1}{5}$.
		\item If not $Q \prec_{\mathsf{ENT}_z} Q'$, then $p \in (0,t(z)]$ where $t(z) := (2^{2z} + 1)^{-1}$ and $\lim_{z\rightarrow\infty} t(z) \rightarrow 0$, i.e.\ the upper interval bound can be arbitrarily minimized. In other words, the range of queries that might be misordered by $\mathsf{ENT}_z$ w.r.t.\ discrimination-preference order can be made arbitrarily small.
		The statements (2.)\ and (3.a) of this proposition still hold if $\mathsf{ENT}$ is replaced by $\mathsf{ENT}_z$ for all real numbers $z\geq 0$.
	\end{enumerate}
\end{enumerate}
\end{proposition}
\begin{proof} We prove the statements in turn:
\begin{enumerate}
	\item This is a direct consequence of Proposition~\ref{prop:if_Q_discrimination-preferred_over_Q'_then_dz(Q')_supset_dz(Q)} and the fact that $p(\md) >0$ for all $\md\in\mD$ (cf.\ page~\pageref{etc:prob_of_each_diag_must_be_greater_zero}).
	\item The idea to show this statement is (1)~to describe the best possible query $Q^*$ in terms of the $\mathsf{ENT}$ measure for some given $x$ such that $Q$ is discrimination-preferred to $Q^*$ and $x = p(\dz{}(Q^*)) - p(\dz{}(Q))$ and (2)~to demonstrate that $Q \prec_{\mathsf{ENT}} Q^*$ for this best possible query $Q^*$. The conclusion is then that $Q \prec_{\mathsf{ENT}} Q'$ must hold for all queries $Q'$ where $Q$ is discrimination-preferred to $Q'$ and $x = p(\dz{}(Q')) - p(\dz{}(Q))$.

Regarding (1), we make the following observations: Let w.l.o.g.\ $p(\dx{}(Q)) \leq p(\dnx{}(Q))$, i.e.\ $p=p(Q=t)$. Through Proposition~\ref{prop:construction_of_Q'_from_Q_if_Q_discrimination-preferred_over_Q'} we know that we can imagine $Q'$ ``resulting'' from $Q$ by transferring diagnoses with overall probability mass $x$ from $\dx{}(Q) \cup \dnx{}(Q)$ to $\dz{}(Q)$ and by possibly replacing all diagnoses in the resulting set $\dx{}(Q)$ by all diagnoses in the resulting set $\dnx{}(Q)$ and vice versa. Due to Eq.~\eqref{eq:p(Q=t)} and Eq.~\eqref{eq:p(Q=f)} and $p\leq\frac{1}{2}$ we have that $|p(\dx{}(Q))-p(\dnx{}(Q))| = (1-p)-p = 1-2p$. Now, by Proposition~\ref{prop:ent} we know that the best query among all queries $Q$ with fixed $p(\dz{}(Q)) = q$ satisfies $p(\dx{}(Q)) = p(\dnx{}(Q))$. Here, $q = p(\dz{}(Q))+x$. That is, the best query $Q^*$ must satisfy $p(\dx{}(Q^*)) = p(\dnx{}(Q^*))$. This can be achieved for $Q$ since $x \geq 1-2p$, i.e.\ diagnoses with a probability mass of $1-2p$ can be transferred from $\dnx{}(Q)$ to $\dx{}(Q)$. 	
	If $x$ were larger than $1-2p$, then further diagnoses with a probability mass of $\frac{x-1-2p}{2}>0$ would have to be transferred from each of the sets $\dx{}(Q)$ and $\dnx{}(Q)$ to $\dz{}(Q)$ resulting in a query $Q''$ with unchanged $p(\dx{}(Q'')) = p(\dx{}(Q^*))$ and $p(\dnx{}(Q'')) = p(\dnx{}(Q^*))$, but with a worse $\mathsf{ENT}$ measure as $p(\dz{}(Q^*)) = p(\dz{}(Q))+1-2p < p(\dz{}(Q))+1-2p+\frac{x-1-2p}{2} = p(\dz{}(Q''))$. Hence, $Q^*$ results from transferring \emph{exactly} a probability mass of $1-2p$ from the more probable set $\dnx{}(Q)$ (see assumption above) in $\setof{\dx{}(Q), \dnx{}(Q)}$ to $\dz{}(Q)$.
	
		Concerning (2), we now demonstrate that $Q \prec_{\mathsf{ENT}} Q^*$: We know that $p \leq \frac{1}{2}$. Assume that $p = \frac{1}{2}$. In this case, the assumption of (2.) is $x \geq 1-2p = 1-2\cdot\frac{1}{2} = 0$. However, due to (1.)\ which states that $x > 0$, we have that $x > 1-2p$ in this case. Hence $Q^*$, whose construction requires $x = 1-2p$, cannot be constructed. Since all other queries (which can be constructed for $x > 1-2p$) have a worse $\mathsf{ENT}$ measure than $Q^*$, as shown above, statement (2.)\ of the proposition holds for $p = \frac{1}{2}$.
	
	If, otherwise, $p < \frac{1}{2}$, let $z:=p(\dz{}(Q))$. Next, recall Eq.~\eqref{eq:scoring_funtion_dekleer} and observe that the sum in squared brackets is independent from the two addends right from it and minimized by setting $p(Q=t) = p(Q=f) = \frac{1}{2}$ (cf.\ the proof of Proposition~\ref{prop:ent}). The minimum is then $-\log_2 2 = -1$. That is, the $\mathsf{ENT}$ measure of $Q^*$ is given by $\mathsf{ENT}(Q^*):=-1 + (z+1-2p) + 1$. On the other hand, the $\mathsf{ENT}$ measure of $Q$ amounts to $\mathsf{ENT}(Q):=p\log_2 p + (1-p)\log_2 (1-p) + z + 1$. Since a smaller $\mathsf{ENT}$ measure signalizes a better query, we must show that $\mathsf{ENT}(Q^*) > \mathsf{ENT}(Q)$. That is, we need to argue that $2p < -p\log_2 p - (1-p)\log_2 (1-p)$. Now, because $p \leq 1-p$ and $\log_2 x \leq 0$ for $x \leq 1$ we have that the righthand side of the inequality $-p\log_2 p - (1-p)\log_2 (1-p) \geq -p\log_2 p - p\log_2 (1-p) = p\cdot(-\log_2 p - \log_2 (1-p))$. It remains to be shown that $-\log_2 p - \log_2 (1-p) > 2$. To prove this, we observe that $p \in (0,\frac{1}{2})$ must hold since $\dx{}(Q)$ and $\dnx{}(Q)$ must not be equal to the empty set (cf.\ Proposition~\ref{prop:properties_of_q-partitions}) and due to $p(\md) > 0$ for all $\md \in \mD$. It is well known that a differentiable function of one variable (in our case $p$) is convex on an interval (in our case $(0,\frac{1}{2}]$) iff its derivative is monotonically increasing on that interval which in turn is the case iff its second derivative is larger than or equal to $0$ on that interval. Further, it is well known that a local minimum of a convex function is a global minimum of that function. The first derivative of $-\log_2 p - \log_2 (1-p)$ is $-\frac{1}{p}+\frac{1}{1-p}$, the second is $\frac{1}{p^2}+\frac{1}{(1-p)^2}$ which is clearly greater than $0$ for $p\in (0,\frac{1}{2}]$. Consequently, $-\log_2 p - \log_2 (1-p)$ is convex for $p\in (0,\frac{1}{2}]$. Setting the first derivative equal to $0$, we obtain $1=2p$, i.e.\ $p = \frac{1}{2}$ as the (single) point where the global minimum is attained. The global minimum of $-\log_2 p - \log_2 (1-p)$ for $p\in (0,\frac{1}{2}]$ is then $2 \log_2 2 = 2$. Now, the deletion of the point $\frac{1}{2}$ from the interval means that $-\log_2 p - \log_2 (1-p) > 2$ for $p\in(0,\frac{1}{2})$. This completes the proof.
	\item We again prove all statements in turn:
	\begin{enumerate}
		\item Here a simple counterexample suffices. One such is given by assuming that $p=0.1$ and $p(\dz{}(Q)) = 0$ and that diagnoses with a probability mass of $x=0.05$ are deleted from $\dnx{}(Q)$ and transferred to $\dz{}(Q)$ to obtain a query $Q'$. Clearly, $Q$ is discrimination-preferred to $Q'$. However, the $\mathsf{ENT}$ measure for $Q$ amounts to $0.531$ whereas for $Q'$ it amounts to $0.506$ which is why $Q' \prec_{\mathsf{ENT}} Q$ holds (recall that better queries w.r.t.\ $\mathsf{ENT}$ have a lower $\mathsf{ENT}$ measure). By the asymmetry of the strict order $\prec_{\mathsf{ENT}}$, we obtain that $Q \prec_{\mathsf{ENT}} Q'$ cannot be satisfied.
		\item Let us assume that $x < 1-2p$, that $Q$ is discrimination-preferred to $Q'$ and $Q' \prec_{\mathsf{ENT}} Q$. We recall from the proof of (2.)\ that we can imagine that $Q'$ ``results'' from $Q$ by transferring diagnoses with overall probability mass $x$ from $\dx{}(Q) \cup \dnx{}(Q)$ to $\dz{}(Q)$ and by possibly replacing all diagnoses in the resulting set $\dx{}(Q)$ by all diagnoses in the resulting set $\dnx{}(Q)$ and vice versa (holds due to Proposition~\ref{prop:construction_of_Q'_from_Q_if_Q_discrimination-preferred_over_Q'}). Now, by assumption, $x$ is smaller than the difference $1-2p$ between the larger and the smaller value in $\setof{p(\dx{}(Q)),p(\dnx{}(Q))}$. Further, for fixed $p(\dz{}(Q))$, the function $p \log_2 p + (1-p) \log_2 (1-p) + p(\dz{}(Q)) + 1$ characterizing $\mathsf{ENT}(Q)$ is convex and thus has only one local (and global) minimum for $p\in (0,1)$ since its second derivative $\frac{1}{p}+\frac{1}{1-p} > 0$ on this interval. The minimum of $\mathsf{ENT}(Q)$ is given by $p=\frac{1}{2}$ (calculated by setting the first derivative $\log_2 p - \log_2 (1-p)$ equal to $0$). Hence, the best $Q'$ (with minimal $\mathsf{ENT}$ value) that can be ``constructed'' from $Q$ is given by transferring diagnoses with a probability mass of $x$ from $\argmax_{X\in\setof{\dx{}(Q),\dnx{}(Q)}}\setof{p(X)}$ to $p(\dz{}(Q))$.
		
For these reasons and due to Equations \eqref{eq:p(Q=t)} and \eqref{eq:p(Q=f)}, we can write $\mathsf{ENT}(Q')$ as $(p+\frac{x}{2})\log_2 (p+\frac{x}{2}) + (1-p-\frac{x}{2}) \log_2 (1-p-\frac{x}{2}) + (p(\dz{}(Q)) + x) + 1$. We now want to figure out the interval $i \subseteq (0,\frac{1}{2}]$ of $p$ where $\mathsf{ENT}(Q')$ might be smaller than or equal to $\mathsf{ENT}(Q)$. In other words, we ask the question for which values of $p$ it is possible that (*): $p \log_2 p + (1-p) \log_2 (1-p) \geq (p+\frac{x}{2})\log_2 (p+\frac{x}{2}) + (1-p-\frac{x}{2}) \log_2 (1-p-\frac{x}{2}) + x$ (note that $p(\dz{}(Q)) + 1$ was eliminated on both sides of the inequality). To this end, let us denote the left side of this inequality by LS and consider once again the derivative $\log_2 p - \log_2 (1-p)$ of LS. Clearly, for values of $p$ where this derivative is smaller than or equal to 
$-2$, there is a (possibly infinitesimally small) value of $x$ such that the increase of $p$ by $\frac{x}{2}$ 
yields a decrease of LS by at least $2\frac{x}{2} = x$. That is, $(p+\frac{x}{2})\log_2 (p+\frac{x}{2}) + (1-p-\frac{x}{2}) \log_2 (1-p-\frac{x}{2}) \leq p \log_2 p + (1-p) \log_2 (1-p) - x$ holds which is equivalent to (*).

Plugging in $\frac{1}{5}$ into the derivative of LS yields $\log_2 \frac{1}{5} - \log_2 \frac{4}{5} = -\log_2 5 - (\log_2 4 - \log_2 5) = -\log_2 4 = -2$. Since the derivative of LS is monotonically increasing (second derivative is greater than zero for $p \in (0,1)$, see above), we obtain that the sought interval $i$ of $p$ must be given by $(0,\frac{1}{5}]$.
	\item Using the same notation as above, $\mathsf{ENT}_z(Q)$ reads as $p \log_2 p + (1-p) \log_2 (1-p) + {z \cdot p(\dz{}(Q))} + 1$ (cf.\ Eq.~\eqref{eq:ENTz}). Rerunning the proof of statement (3.b) above for $\mathsf{ENT}_z$ implies deriving the interval for $p$ where (**): $p \log_2 p + (1-p) \log_2 (1-p) \geq (p+\frac{x}{2})\log_2 (p+\frac{x}{2}) + (1-p-\frac{x}{2}) \log_2 (1-p-\frac{x}{2}) + z\cdot x$ holds instead of (*) above. Analogously to the argumentation above, the sought interval is given by $(0,t(z)]$ where $t(z)$ is the value of $p$ for which the derivative $\log_2 p - \log_2 (1-p)$ of $p \log_2 p + (1-p) \log_2 (1-p)$ is equal to $-2z$ (which yields $-z\cdot x$ after multiplication with $\frac{x}{2}$, see above). 
	
	To calculate this value, we exploit the inverse function $f^{-1}(q)$ of the (injective) derivative $\log_2 p - \log_2 (1-p) = f(p) = q$. The derivation of $f^{-1}(q)$ goes as follows:
	\begin{align*}
	q &= \log_2 p - \log_2 (1-p) = \log_2(\frac{p}{1-p}) \\
	2^q &= \frac{p}{1-p} \\
	2^{-q} &= \frac{1-p}{p} = \frac{1}{p}-1 \\
	(2^{-q}+1)^{-1} &= p = f^{-1}(q)
	\end{align*}
	Plugging in $q=-2z$ into $f^{-1}(q)$ yields $t(z) = (2^{2z}+1)^{-1}$.
	For instance, for $z:=1$ (standard $\mathsf{ENT}$ measure), we obtain $t=\frac{1}{5}$ by plugging in $-2$ into the inverse function. For e.g.\ $z:=5$ we get $t=\frac{1}{1025}$ (by plugging in $-10$). 
	
	Clearly, $\lim_{z\rightarrow\infty} t(z) \rightarrow 0$. It is moreover straightforward to verify that $\mathsf{ENT}_z$ for $z\geq 0$ satisfies the statements (2.)\ and (3.a) of this proposition. 
	\end{enumerate}
\end{enumerate}
\end{proof}
The next corollary is an immediate consequence of Proposition~\ref{prop:ENT_preserves_discrimination-pref-order}. It establishes that the discrimination-preference order is preserved by $\mathsf{ENT}$ for the most interesting class of queries w.r.t.\ $\mathsf{ENT}$, namely those for which the probabilities of both answers are similar (cf.\ Proposition~\ref{prop:ent}). In other words, $\mathsf{ENT}$ might only (i)~not prefer a discrimination-preferred query to a suboptimal one or (ii)~prefer a suboptimal query to a discrimination-preferred one in case the discrimination-preferred one is not desirable itself. That is, $\mathsf{ENT}$ orders the most desirable queries correctly w.r.t.\ the discrimination-preference order.
\begin{corollary}\label{cor:ENT_satisfies_discrimination-pref_order_for_queries_with_p>=1/5}
Let $Q$ be a query and $p_Q:=\min_{a\in\setof{t,f}} (p(Q=a))$. Let further $\mQ$ be a set of queries where each query $Q$ in this set satisfies $p_Q > \frac{1}{5}$. Then, $\mathsf{ENT}$ 
satisfies the discrimination-preference relation (over $\mQ$) (Definition~\ref{def:measures_equivalent_theoretically-optimal_superior}).
\end{corollary}
\begin{proof}
This statement is a consequence of applying the law of contraposition to the implication stated by Proposition~\ref{prop:ENT_preserves_discrimination-pref-order},(3.b).
\end{proof}
Roughly, the next corollary states that the relation $\prec_{\mathsf{ENT}_z}$ 
with a higher value of $z$ includes at least as many correct query tuples w.r.t.\ the discrimination-preference order as the relation $\prec_{\mathsf{ENT}_z}$ for a lower value of $z$:
\begin{corollary}\label{cor:if_r<s_then_it_holds_that_if_not_Q_precENTs_Q'_then_not_Q_precENTr_Q'}
Let $Q,Q'$ be queries such that $Q$ is discrimination-preferred to $Q'$ and $0\leq r < s$. Then: If not $Q \prec_{\mathsf{ENT}_s} Q'$, then not $Q \prec_{\mathsf{ENT}_r} Q'$ either.
\end{corollary}
\begin{proof}
Due to Proposition~\ref{prop:construction_of_Q'_from_Q_if_Q_discrimination-preferred_over_Q'} any $Q'$ for which $Q$ is discrimination-preferred to $Q'$ ``results'' from $Q$ by transferring diagnoses with overall probability mass $x$ from $\dx{}(Q) \cup \dnx{}(Q)$ to $\dz{}(Q)$ and by possibly replacing all diagnoses in the resulting set $\dx{}(Q)$ by all diagnoses in the resulting set $\dnx{}(Q)$ and vice versa. Now, if not $Q \prec_{\mathsf{ENT}_s} Q'$, then $p \log_2 p + (1-p) \log_2 (1-p) \geq (p+\frac{x}{2})\log_2 (p+\frac{x}{2}) + (1-p-\frac{x}{2}) \log_2 (1-p-\frac{x}{2}) + s\cdot x$ (using the same notation as in and referring to the proofs of bullets (3.b) and (3.c) of Proposition~\ref{prop:ENT_preserves_discrimination-pref-order}). However, due to $r < s$ and $x > 0$ (as per bullet (1.)\ of Proposition~\ref{prop:ENT_preserves_discrimination-pref-order}), we also obtain that $p \log_2 p + (1-p) \log_2 (1-p) \geq (p+\frac{x}{2})\log_2 (p+\frac{x}{2}) + (1-p-\frac{x}{2}) \log_2 (1-p-\frac{x}{2}) + r \cdot x$ which implies that $Q \prec_{\mathsf{ENT}_r} Q'$ does not hold.
\end{proof}

In addition to that, Proposition~\ref{prop:ENT_preserves_discrimination-pref-order} enables us to deduce a measure $\mathsf{ENT}_z$ from $\mathsf{ENT}$ which minimizes the proportion of queries that might potentially be misordered by $\mathsf{ENT}_z$ w.r.t.\ the discrimination-preference order to arbitrarily small size:
\begin{corollary}\label{cor:ENTz_satisfies_discrimination-pref_order_for_queries_with_p>=arbitrary_small_number}
Let $p_Q:=\min_{a\in\setof{t,f}} (p(Q=a))$ for some query $Q$. Let further $\mQ$ be a set of queries where each query $Q\in\mQ$ satisfies $p_Q > t$ for an arbitrary real number $t > 0$. 
Then, for each $z \geq \max \setof{-\frac{1}{2}(\log_2 t - \log_2 (1-t)),1}$ the measure $\mathsf{ENT}_z$ 
satisfies the discrimination-preference relation over $\mQ$ (Definition~\ref{def:measures_equivalent_theoretically-optimal_superior}).
\end{corollary}
\begin{proof}
Since $p_Q \leq \frac{1}{2}$ due to the definition of $p_Q$, we have that for $t \geq \frac{1}{2}$ the set $\mQ$ must be empty which is why the statement of the corollary is true. 

Otherwise, if $t \in [\frac{1}{5},\frac{1}{2})$, we first consider the case $t = \frac{1}{5}$ and compute $-\frac{1}{2}(\log_2 t - \log_2 (1-t)) = 1$. 
Since $\log_2 t - \log_2 (1-t)$ is strictly monotonically increasing (its derivative $\frac{1}{t} + \frac{1}{1-t}$ is greater than $0$) for $t\in (0,\frac{1}{2})$, i.e.\ particularly for $t\in [\frac{1}{5},\frac{1}{2})$, we have that $-\frac{1}{2}(\log_2 t - \log_2 (1-t))$ is strictly monotonically decreasing for $t\in (0,\frac{1}{2})$, i.e.\ particularly for $t\in [\frac{1}{5},\frac{1}{2})$. Therefore, for all $t\in (\frac{1}{5},\frac{1}{2})$, the result of plugging $t$ into $-\frac{1}{2}(\log_2 t - \log_2 (1-t))$ is smaller than the result of plugging in $t = \frac{1}{5}$. Consequently, $\max \setof{-\frac{1}{2}(\log_2 t - \log_2 (1-t)),1} = 1$ for $t\in [\frac{1}{5},\frac{1}{2})$.

So, for $t\in [\frac{1}{5},\frac{1}{2})$ we need to verify that $\mathsf{ENT}_z$ for $z \geq 1$ satisfies the discrimination-preference relation over $\mQ$. 
Indeed, for $z = 1$ and $t = \frac{1}{5}$ the measure $\mathsf{ENT}_z$ satisfies the discrimination-preference relation over $\mQ$ due to Corollary~\ref{cor:ENT_satisfies_discrimination-pref_order_for_queries_with_p>=1/5} and the fact that $\mathsf{ENT}$ is equal to $\mathsf{ENT}_z$ for $z=1$. For $z > 1$ and $t = \frac{1}{5}$, $\mathsf{ENT}_z$ must also satisfy the discrimination-preference relation over $\mQ$ as a consequence of Corollary~\ref{cor:if_r<s_then_it_holds_that_if_not_Q_precENTs_Q'_then_not_Q_precENTr_Q'}. In other words, we have shown so far that for any set of queries where each query $Q$ in this set satisfies $p_Q > \frac{1}{5}$ the measure $\mathsf{ENT}_z$ for all $z \geq 1$ satisfies the discrimination-preference relation.

However, as $p_Q > t$ for $t > \frac{1}{5}$ in particular means that $p_Q > \frac{1}{5}$, it is an implication of the proof so far that $\mathsf{ENT}_z$ for $z \geq 1$ satisfies the discrimination-preference relation over $\mQ$ for all $t \in [\frac{1}{5},\frac{1}{2})$.
%
%

Now, assume an arbitrary $t \in (0,\frac{1}{5})$. By the analysis above we know that for all $t\in (0,\frac{1}{5})$ the value of $-\frac{1}{2}(\log_2 t - \log_2 (1-t))$ must be greater than $1$. So, $\max \setof{-\frac{1}{2}(\log_2 t - \log_2 (1-t)),1} = -\frac{1}{2}(\log_2 t - \log_2 (1-t))$ for $t\in (0,\frac{1}{5})$.

Let us recall the proof of bullet (3.b) of Proposition~\ref{prop:ENT_preserves_discrimination-pref-order}. According to the argumentation there, we compute the slope of $p \log_2 p + (1-p) \log_2 (1-p)$ at $p = t$ by plugging $p=t$ into the derivative $\log_2 p - \log_2 (1-p)$ of this function yielding $\log_2 t - \log_2 (1-t)$. Further, as it became evident in the proof of bullet (3.c) of Proposition~\ref{prop:ENT_preserves_discrimination-pref-order}, the suitable value of $z$ can be obtained by dividing this number by $-2$. Now, because each query $Q\in\mQ$ satisfies $p_Q > t$, setting $z := -\frac{1}{2}(\log_2 t - \log_2 (1-t))$ results in a measure $\mathsf{ENT}_z$ 
that satisfies the discrimination-preference relation. That this does hold also for all $z > -\frac{1}{2}(\log_2 t - \log_2 (1-t))$ is a direct consequence of Corollary~\ref{cor:if_r<s_then_it_holds_that_if_not_Q_precENTs_Q'_then_not_Q_precENTr_Q'}.
\end{proof}
\begin{example}\label{ex:ENT_z}
Let $t := \frac{1}{10^9}$. Then $\mathsf{ENT}_z$ for $z := 15$ preserves the discrimination-preference order for all queries for which the more unlikely answer has a probability of at least $10^{-9}$. 
For $t := \frac{1}{10^{20}}$, a value of $z := 34$ guarantees that $\mathsf{ENT}_z$ preserves the discrimination-preference order for all queries for which the more unlikely answer has a probability of at least $10^{-20}$. This illustrates that the $\mathsf{ENT}_z$ measure can be configured in a way it imposes an order consistent with the discrimination-preference order on any set of queries that might practically occur. \qed
\end{example}
%
%
%
\begin{corollary}\label{cor:if_0<=r<s_then_ENTr_inferior_to_ENTs}
Let $0 \leq r < s$. Then $\mathsf{ENT}_r$ is inferior to $\mathsf{ENT}_s$.
\end{corollary}
\begin{proof}
To show that $\mathsf{ENT}_r$ is inferior to $\mathsf{ENT}_s$ we must verify all conditions of the definition of inferiority between two measures (cf.\ Definition~\ref{def:measures_equivalent_theoretically-optimal_superior}). Since $0 \leq r < s$, condition (3.)\ is met because of Corollary~\ref{cor:if_r<s_then_it_holds_that_if_not_Q_precENTs_Q'_then_not_Q_precENTr_Q'}. An example confirming the validity of conditions (1.)\ and (2.)\ can be constructed in an analogous way as in the proof of bullet (3.a) of Proposition~\ref{prop:ENT_preserves_discrimination-pref-order} by using the train of thoughts of the proof of bullet (3.b) of Proposition~\ref{prop:ENT_preserves_discrimination-pref-order}.
\end{proof}

\vspace{1em}

\noindent\emph{Split-In-Half Query Selection ($\mathsf{SPL}$).} The selection criterion $\mathsf{SPL}$, on the other hand, votes for the query 
\begin{align*}
Q_{\mathsf{SPL}} = \argmin_{Q\in\mQ} \left(\mathsf{SPL}(Q)\right) 
\quad \text{ where } \quad 
\mathsf{SPL}(Q):=\left|\, |\dx{}(Q)| - |\dnx{}(Q)| \,\right| + |\dz{}(Q)|
\end{align*}
and hence is optimized by queries for which the number of leading diagnoses predicting the positive answer is equal to the number of leading diagnoses predicting the negative answer and for which each leading diagnosis does predict an answer. 
\begin{proposition}\label{prop:spl}
Let $\mD$ be a set of leading diagnoses w.r.t.\ a DPI $\tuple{\mo,\mb,\Tp,\Tn}_\RQ$. Then:
\begin{enumerate}
	\item A query $Q$ is theoretically optimal w.r.t.\ $\mathsf{SPL}$ and $\mD$ iff $|\dx{}(Q)|=|\dnx{}(Q)|$ as well as $|\dz{}(Q)| = 0$.
	\item Let $q \in \setof{0,\dots,|\mD|}$ be a fixed number. Further, let $\mQ$ be a set of queries w.r.t.\ $\mD$ and $\tuple{\mo,\mb,\Tp,\Tn}_\RQ$ where each query $Q\in\mQ$ satisfies $|\dz{}(Q)| = q$. Then the theoretically optimal query w.r.t.\ $\mathsf{SPL}$ and $\mD$ over $\mQ$ satisfies $|\dx{}(Q)|=|\dnx{}(Q)|$.
\end{enumerate}
\end{proposition}
%
{\begin{proposition}\label{prop:spl_consistent_with_DPR_but_does_not_satisfy_DPR}\leavevmode 
\samepage
\begin{enumerate}
	\item $\mathsf{SPL}$ is consistent with the discrimination-preference relation $DPR$, but does not satisfy $DPR$.
	\item Let $d(Z):=\left|\, |\dx{}(Z)| - |\dnx{}(Z)| \,\right|$ for some query $Z$ and let $Q,Q'$ be two queries such that $Q$ is discrimination-preferred to $Q'$. Then not $Q \prec_{\mathsf{SPL}} Q'$ iff $\mathsf{SPL}(Q) = \mathsf{SPL}(Q')$ and $d(Q) - d(Q') = (|\dz{}(Q')| - |\dz{}(Q)|)$.
\end{enumerate}
\end{proposition}}
\begin{proof}
Let $Q,Q'$ be two queries such that $Q$ is discrimination-preferred to $Q'$. By Proposition~\ref{prop:construction_of_Q'_from_Q_if_Q_discrimination-preferred_over_Q'} $Q'$ differs from $Q$ in that $|\mathbf{X}| > 0$ diagnoses are deleted from $\dx{}(Q)\cup\dnx{}(Q)$ and added to $\dz{}(Q)$ to obtain $Q'$.
Thence, $|\dz{}(Q')| = |\dz{}(Q)| + |\mathbf{X}|$. 
It holds that $|d(Q) - d(Q')| \leq |\mathbf{X}|$ since $d(Q)$ can apparently not change by more than $|\mathbf{X}|$ through the deletion of $|\mathbf{X}|$ diagnoses from $\dx{}(Q)\cup\dnx{}(Q)$. Hence, $\mathsf{SPL}(Q) \leq \mathsf{SPL}(Q')$ which is why either $Q \prec_{\mathsf{SPL}} Q'$ or $\mathsf{SPL}(Q) = \mathsf{SPL}(Q')$. In the latter case, $Q$ and $Q'$ do not stand in a $\prec_{\mathsf{SPL}}$ relationship with one another (cf.\ Definition~\ref{def:precedence_order_defined_by_q-partition_quality_measure}). In the former case, due to Proposition~\ref{prop:precedence_order_is_strict_order} (asymmetry of $\prec_{\mathsf{SPL}}$) $Q' \prec_{\mathsf{SPL}} Q$ cannot hold. Altogether, we have shown that $Q' \prec_{\mathsf{SPL}} Q$ cannot be valid for any two queries $Q,Q'$ where $Q$ is discrimination-preferred to $Q'$. Therefore, $\mathsf{SPL}$ is consistent with the discrimination-preference relation $DPR$.

To see why $\mathsf{SPL}$ does not satisfy $DPR$, let us construct an example of two queries $Q,Q'$ where $Q$ is discrimination-preferred to $Q'$, but not $Q \prec_{\mathsf{SPL}} Q'$. To this end, let 
\begin{align*}
\tuple{\dx{}(Q),\dnx{}(Q),\dz{}(Q)} &= \tuple{\setof{\md_1,\md_2},\setof{\md_3},\emptyset} \\
\tuple{\dx{}(Q'),\dnx{}(Q'),\dz{}(Q')} &= \tuple{\setof{\md_1},\setof{\md_3},\setof{\md_2}} 
\end{align*}
Apparently, $Q$ is discrimination-preferred to $Q'$ (cf.\ Proposition~\ref{prop:construction_of_Q'_from_Q_if_Q_discrimination-preferred_over_Q'} with the set $\mathbf{X} = \setof{\md_2}$). Furthermore, $\mathsf{SPL}(Q) = \mathsf{SPL}(Q') = 1$ which implies by Definition~\ref{def:precedence_order_defined_by_q-partition_quality_measure} that $Q \prec_{\mathsf{SPL}} Q'$ does not hold. This finishes the proof of bullet (1.). 

The truth of the statement of bullet (2.) becomes evident by reconsidering the argumentation used to prove (1.). Clearly, not $Q \prec_{\mathsf{SPL}} Q'$ holds iff $\mathsf{SPL}(Q) = \mathsf{SPL}(Q')$. The latter, however, is true iff $|\mathbf{X}|$, by which $|\dz{}(Q')|$ is larger than $|\dz{}(Q)|$, is exactly the difference between $d(Q) - d(Q')$.
\end{proof}

Let us now, in a similar way as done with $\mathsf{ENT}$, characterize a measure $\mathsf{SPL}_z$ that constitutes a generalization of the $\mathsf{SPL}$ measure. Namely, $\mathsf{SPL}_z$ selects the query 
	\begin{align*}
	 Q_{\mathsf{SPL}_z} := \argmin_{Q\in\mQ} \left(\mathsf{SPL}_z(Q)\right)
	\end{align*}
	where
	\begin{align}\label{eq:SPLz}
	\mathsf{SPL}_z(Q):= \left|\, |\dx{}(Q)| - |\dnx{}(Q)| \,\right| + z\,|\dz{}(Q)|
	\end{align} 
Note that $\mathsf{SPL}_z(Q) = \mathsf{SPL}(Q) + (z-1) |\dz{}(Q)|$ and thus that $\mathsf{SPL}_1$ is equal to $\mathsf{SPL}$. The analysis of this generalized split-in-half measure yields the following results:
\begin{proposition}\label{prop:SPL2_satisfies_DPR}
$\mathsf{SPL}_z$ satisfies the debug preference relation $DPR$ for all real numbers $z > 1$ and is inconsistent with the $DPR$ for all real numbers $z < 1$.
\end{proposition}
\begin{proof}
Let $z>1$ and let us recall the proof of Proposition~\ref{prop:spl_consistent_with_DPR_but_does_not_satisfy_DPR} and reuse the assumptions and the notation of this proof. We observe that $z |\dz{}(Q')| = z(|\dz{}(Q)| + |\mathbf{X}|)$ and thus $z |\dz{}(Q')| - z|\dz{}(Q)| = z |\mathbf{X}| > |\mathbf{X}|$. But, still $|d(Q) - d(Q')| \leq |\mathbf{X}|$ holds.
Hence, is must be the case that $\mathsf{SPL}_z(Q) < \mathsf{SPL}_z(Q')$, i.e.\ $Q \prec_{\mathsf{SPL}_z} Q'$ (cf.\ Definition~\ref{def:precedence_order_defined_by_q-partition_quality_measure}), for any two queries $Q,Q'$ where $Q$ is discrimination-preferred to $Q'$ and $z > 1$. 

The queries $Q,Q'$ given in the proof of Proposition~\ref{prop:spl_consistent_with_DPR_but_does_not_satisfy_DPR} constitute a counterexample witnessing that $\mathsf{SPL}_z$ is not consistent with the $DPR$ for $z < 1$. 
\end{proof}
%
%
\begin{corollary}\label{cor:SPL2_is_superior_to_SPL}
$\mathsf{SPL}_z$ is superior to $\mathsf{SPL}$ for all real numbers $z > 1$.
\end{corollary}
\begin{proof}
We have to verify the three bullets of the definition of superiority (cf.\ Definition~\ref{def:measures_equivalent_theoretically-optimal_superior}). The first bullet is satisfied due to the example given in the proof of Proposition~\ref{prop:spl_consistent_with_DPR_but_does_not_satisfy_DPR}. The condition of the second bullet is met due to the example given in the proof of Proposition~\ref{prop:spl_consistent_with_DPR_but_does_not_satisfy_DPR} and Proposition~\ref{prop:SPL2_satisfies_DPR}. Finally, the postulation of the third bullet is fulfilled due to Proposition~\ref{prop:SPL2_satisfies_DPR}.
\end{proof}
We want to emphasize at this point that Corollary~\ref{cor:SPL2_is_superior_to_SPL} does not imply that $\mathsf{SPL}_r \equiv \mathsf{SPL}_s$ for all real numbers $r,s > 1$. To recognize this, consider the following example:
\begin{example}\label{ex:SPLr_not_equiv_SPLs_for_r,s>1}
Assume two queries $Q_1,Q_2$ in $\mQ$ where 
\begin{align*}
\langle |\dx{}(Q_1)|,|\dnx{}(Q_1)|,|\dz{}(Q_1)| \rangle &= \langle 7,3,8 \rangle \\
\langle |\dx{}(Q_2)|,|\dnx{}(Q_2)|,|\dz{}(Q_2)| \rangle &= \langle 4,4,10 \rangle
\end{align*}
and let $r := 1.1$ and $s := 10$. Then, 
\begin{align*}
\mathsf{SPL}_r(Q_1) &= 4 + 1.1 \cdot 8 = 12.8 \\
\mathsf{SPL}_r(Q_2) &= 0 + 1.1 \cdot 10 = 11 \\
\mathsf{SPL}_s(Q_1) &= 4 + 10 \cdot 8 = 84 \\
\mathsf{SPL}_s(Q_2) &= 0 + 10 \cdot 10 = 100 
\end{align*}
That is, $Q_2 \prec_{\mathsf{SPL}_r} Q_1$, but $Q_1 \prec_{\mathsf{SPL}_s} Q_2$. Consequently, $\mathsf{SPL}_r \not\equiv \mathsf{SPL}_s$ (cf.\ Definition~\ref{def:measures_equivalent_theoretically-optimal_superior}).\qed
\end{example}

The following example (cf.\ \cite[Ex.~9.1]{Rodler2015phd}) illustrates the notions explicated in Section~\ref{sec:QPartitionSelectionMeasures}: 
%
\begin{example}\label{ex:measure_notions_illustrated}
Let $\tuple{\mo,\mb,\Tp,\Tn}_\RQ$ be a DPI and $\mD := \setof{\md_1,\dots,\md_4}$ a set of leading diagnoses w.r.t.\ $\tuple{\mo,\mb,\Tp,\Tn}_\RQ$. Now, consider the example q-partitions for queries named $Q_1,\dots,Q_{10}$ w.r.t.\ this DPI listed in the lefthand columns of Table~\ref{tab:example_q-partitions}. The righthand columns of Table~\ref{tab:example_q-partitions} show the preference order of queries (lower values mean higher preference) induced by different measures discussed in this example. Now the following holds:
\begin{itemize}
	\item $Q_6$ is theoretically optimal w.r.t.\ $\mathsf{ENT}$ since $p(\dx{}(Q_6)) = 0.5$, $p(\dnx{}(Q_6)) = 0.5$ and $\dz{}(Q_6) = \emptyset$ (cf.\ Proposition~\ref{prop:ent}).
	\item $Q_7,Q_8$ and $Q_9$ are theoretically optimal w.r.t.\ $\mathsf{SPL}$ since $|\dx{}(Q_i)| = |\dnx{}(Q_i)|$ and $|\dz{}(Q_i)| = 0$ for $i\in\setof{7,8,9}$ (cf.\ Proposition~\ref{prop:spl}).
	\item Let $\mathsf{SPL}'$ be a measure defined in a way that it selects the query
	\begin{align*}
	Q_{\mathsf{SPL}'} := \argmin_{Q\in\mQ} \left(\mathsf{SPL}'(Q)\right) 
	\end{align*} 
	where
	\begin{align*}
	\mathsf{SPL}'(Q):=x \left(\left|\, |\dx{}(Q)| - |\dnx{}(Q)| \,\right| + |\dz{}(Q)|\right) + y
	\end{align*}
	for some fixed $x \in \mathbb{R}^+ \setminus \setof{0}$ and some fixed $y \in \mathbb{R}$. Then $\mathsf{SPL}' \equiv \mathsf{SPL}$ (cf.\ Proposition~\ref{prop:linear_function_of_measure_is_equivalent_to_measure_itself}).
	
	\item 
	For instance, it holds that $\mathsf{SPL}_2 \not\equiv \mathsf{SPL}$ (cf.\ Eq.~\eqref{eq:SPLz}). To see this, consider the righthand side of Table~\ref{tab:example_q-partitions} and compare the preference orders induced by $\mathsf{SPL}$ and $\mathsf{SPL}_2$. We can clearly identify that, e.g.\, $Q_3 \prec_{\mathsf{SPL}_2} Q_2$, but not $Q_3 \prec_{\mathsf{SPL}} Q_2$.
		
	\item It holds that $Q_9$ is discrimination-preferred to $Q_5$ as well as to $Q_2$ because $Q_i$ for $i\in\setof{2,5}$ shares one set in $\{\dx{}(Q_i), \dnx{}(Q_i)\}$  with $\setof{\dx{}(Q_9),\dnx{}(Q_9)}$, but exhibits a set $\dz{}(Q_i) \supset \dz{}(Q_9)$. 
	For instance, let us verify the definition of discrimination-preference (Definition~\ref{def:measures_equivalent_theoretically-optimal_superior}) for $Q_9$ versus $Q_5$: The positive answer to $Q_9$ eliminates $\setof{\md_1,\md_2}$ whereas the negative answer to $Q_5$ eliminates only a proper subset of this set, namely $\setof{\md_2}$. The negative answer to $Q_9$ eliminates $\setof{\md_3,\md_4}$ which is equal -- and thus a (non-proper) subset -- of the set of leading diagnoses eliminated by the positive answer to $Q_5$. Hence, we have found an injective function compliant with Definition~\ref{def:measures_equivalent_theoretically-optimal_superior} (which maps the negative answer of $Q_5$ to the positive answer of $Q_9$ and the positive answer to $Q_5$ to the negative answer to $Q_9$).
	
	On the other hand, $Q_9$ is for example not discrimination-preferred to $Q_{10}$. Whereas for the negative answer to $Q_{10}$ (eliminates one diagnosis $\md_1$ in $\mD$) there is an answer to $Q_9$, namely the positive one, which eliminates $\setof{\md_1,\md_2}$, a superset of $\setof{\md_1}$, there is no answer to $Q_9$ which eliminates a superset of the set $\setof{\md_2,\md_3,\md_4}$ eliminated by the positive answer to $Q_{10}$. Hence, the queries $Q_9$ and $Q_{10}$ are in no discrimination-preference relation with one another.
	
	Please note that $Q_4$ is not discrimination-preferred to $Q_5$ either, although we have the situation where for each answer to $Q_5$ there is an answer to $Q_4$, namely the negative one, that eliminates a superset of the leading diagnoses eliminated by the respective answer to $Q_5$. However, the reason why this does not result in a discrimination-preference relation between these two queries is that the function $f$ is not injective in this case because $f$ maps both answers to $Q_5$ to the \emph{same} answer to $Q_4$. Intuitively, we do not consider $Q_5$ worse than $Q_4$ since using $Q_4$ we might end up in a scenario where none of the diagnoses eliminated by some answer to $Q_5$ are ruled out. This scenario would be present when getting a positive answer to $Q_4$ implying the elimination of $\md_1$ and thus neither $\setof{\md_2}$ nor $\setof{\md_3,\md_4}$.
	\item It holds that $Q_3 \prec_{\mathsf{ENT}} Q_{10} \prec_{\mathsf{ENT}} Q_1$, but $Q_3, Q_{10}$ and $Q_1$ do not stand in any $\prec_{\mathsf{SPL}}$ relationship with one another. The former holds since all three queries feature an empty set $\dz{}$, but the difference between $p(\dx{})$ and $p(\dnx{})$ is largest for $Q_1$ ($p(\dx{}(Q_1)) = 0.95$), second largest for $Q_{10}$ ($p(\dnx{}(Q_{10})) = 0.85$) and smallest for $Q_3$ ($p(\dx{}(Q_3)) = 0.7$). The latter is a consequence of the fact that $\mathsf{SPL}$ considers all three queries $Q_3, Q_{10}$ and $Q_1$ as equally preferable (all have an $\mathsf{SPL}$ value of $2$).
	\item Let $\mathsf{MAX}$ be a measure defined in a way that
	\begin{align*}
	Q_{\mathsf{MAX}} := \argmax_{Q\in\mQ} \left(\mathsf{MAX}(Q)\right) \quad \text{ where } \quad \mathsf{MAX}(Q):=\max\setof{|\dx{}(Q)|,|\dnx{}(Q)|} 
	\end{align*}
	Then $\mathsf{MAX}$ is inferior to $\mathsf{SPL}_2$ since not $Q_9 \prec_{\mathsf{MAX}} Q_5$ although $Q_9$ is discrimination-preferred to $Q_5$ (bullet 1 of the definition of superiority in Definition~\ref{def:measures_equivalent_theoretically-optimal_superior} is met). As mentioned above, $Q_9 \prec_{\mathsf{SPL}_2} Q_5$ holds (bullet 2 is fulfilled). Moreover, there is no pair of queries $Q,Q'$ where $Q$ is discrimination-preferred to $Q'$ and not $Q \prec_{\mathsf{SPL}_2} Q'$, as per Proposition~\ref{prop:SPL2_satisfies_DPR} (bullet 3 is satisfied).\qed
\end{itemize}
\end{example}
\begin{table*}[tb]
\footnotesize
\centering
\begin{tabular}{llll|cccc|cccc}
		&								&								&								&  \multicolumn{4}{c|}{assigned value $m(Q_i)$} & \multicolumn{4}{c}{preference order $\prec_m$} \\
$i$ & $\dx{}(Q_i)$  & $\dnx{}(Q_i)$ &  $\dz{}(Q_i)$ & $\mathsf{SPL}$ & $\mathsf{ENT}$ & $\mathsf{SPL}_2$ & $\mathsf{MAX}$ & $\mathsf{SPL}$ & $\mathsf{ENT}$ & $\mathsf{SPL}_2$ & $\mathsf{MAX}$\\ \hline
$1$ & $\{\md_1,\md_2,\md_4\}$ & $\{\md_3\}$ & $\emptyset$ 			& 2 & 0.71 & 2 		& 3 & 2 & 8 & 2 & 2	\\
$2$ & $\{\md_3, \md_4\}$ & $\{\md_2\}$ & $\{\md_1\}$ 						& 2 & 0.15 & 3	 	& 2 & 2 & 5 & 3 & 1	\\
$3$ & $\{\md_1,\md_3,\md_4\}$ & $\{\md_2\}$ & $\emptyset$ 			& 2 & 0.12 & 2 		& 3 & 2 & 4 & 2 & 2	\\
$4$ & $\{\md_2,\md_3,\md_4\}$ & $\{\md_1\}$ & $\emptyset$ 			& 2 & 0.39 & 2 		& 3 & 2 & 7 & 2 & 2 \\
$5$ & $\{\md_2\}$ & $\{\md_3, \md_4\}$ & $\{\md_1\}$ 						& 2 & 0.15 & 3	 	& 2 & 2 & 5 & 3 & 1	\\ 
$6$ & $\{\md_4\}$ & $\{\md_1,\md_2, \md_3\}$ & $\emptyset$ 			& 2 & 0 	 & 2 		& 3 & 2 & 1 & 2 & 2	\\
$7$ & $\{\md_1, \md_4\}$ & $\{\md_2,\md_3\}$ & $\emptyset$ 			& 0 & 0.07 & 0 		& 2 & 1 & 3 & 1 & 1	\\
$8$ & $\{\md_2, \md_4\}$ & $\{\md_1,\md_3\}$ & $\emptyset$ 			& 0 & 0.28 & 0 		& 2 & 1 & 6 & 1 & 1	\\
$9$ & $\{\md_3, \md_4\}$ & $\{\md_1,\md_2\}$ & $\emptyset$ 			& 0 & 0.01 & 0 		& 2 & 1 & 2 & 1 & 1	\\ 
$10$& $\{\md_1\}$ & $\{\md_2,\md_3, \md_4\}$ & $\emptyset$ 			& 2 & 0.39 & 2 		& 3 & 2 & 7 & 2 & 2	\\  
\hline
\end{tabular}
\caption[Example: Preference Order on Queries Induced by Measures]{The lefthand sector of the table lists some q-partitions associated with queries $Q_1,\dots,Q_{10}$ for the DPI given in Example~\ref{ex:measure_notions_illustrated}. The middle sector shows the values that are assigned to the given queries $Q_i$ by the various measures $m \in \setof{\mathsf{SPL},\mathsf{ENT},\mathsf{SPL}_2, \mathsf{MAX}}$. The righthand sector gives the preference orders over the given queries induced by the measures (lower values signify higher preference).}
\label{tab:example_q-partitions}
\end{table*}
%
\begin{table}[tb]
	\centering
		\begin{tabular}{lcccc}
			$\md \in \mD$ & $\md_1$ & $\md_2$ & $\md_3$ & $\md_4$ \\\hline
			$p(\md)$             & 0.15  & 0.3  & 0.05  & 0.5  
		\end{tabular}
\caption[Example: Diagnosis Probabilities]{Diagnosis probabilities for $\mD$ in Example~\ref{ex:measure_notions_illustrated}.}
\label{tab:example_diag-probs}
\end{table}
%
Query selection by means of the $\mathsf{ENT}$ criterion relies strongly on the diagnosis probability distribution $p$ and thus on the initial fault information that is provided as an input to the debugging system. 
%
In this vein, when applying $\mathsf{ENT}$ as measure $m$ for query computation, a user can profit from a good prior fault estimation w.r.t.\ the number of queries that need to be answered to identify the true diagnosis, but may at the same time have to put up with a serious overhead in answering effort in case of poor estimates that assign a low probability to the true diagnosis.
$\mathsf{SPL}$, in contrast, refrains from using any probabilities and aims at maximizing the elimination rate by selecting a query that guarantees the invalidation of the half set of leading diagnoses $\mD$. As a consequence, $\mathsf{SPL}$ is generally inferior to $\mathsf{ENT}$ for good estimates and superior to $\mathsf{ENT}$ for misleading probabilities, as experiments conducted in~\cite{Shchekotykhin2012,Rodler2013} indicate. 
So, an unfavorable combination of selection measure $m$ and quality of initial fault estimates 
can lead to a significant overhead of necessary queries which means extra time and work for the user in our interactive scenario. Before we revisit a recently proposed measure RIO \cite{Rodler2013} for query selection in Section~\ref{sec:rio} as solution to this dilemma, we look at further information theoretic measures that might act as beneficial selection criteria in KB debugging. 

\subsubsection{New Active Learning Measures for KB Debugging}
\label{sec:NewActiveLearningMeasuresForKBDebugging}

In this section, we analyze various general active learning measures presented in \cite{settles2012} with regard to their application in the KB debugging scenario. In \cite{settles2012}, these measures are classified into different query strategy frameworks. We also stick to this categorization. The four frameworks we consider in sequence next are \emph{Uncertainty Sampling}, \emph{Query-by-Committee}, \emph{Expected Model Change} and \emph{Expected Error Reduction}. 

\paragraph{Uncertainty Sampling (US).} A learning algorithm that relies on uncertainty sampling targets queries whose labeling is most uncertain under the current state of knowledge. The measures belonging to this framework are:

\vspace{1em}

\noindent\emph{Least Confidence ($\mathsf{LC}$)}: \label{etc:measure_desc_LC}
Selects the query $Q_{\mathsf{LC}}$ whose prediction is least confident, i.e.\ whose most likely answer $a_{Q,\max}$ has least probability. Formally: 
\begin{align*}
Q_{\mathsf{LC}} := \argmin_{Q \in \mQ} \left(\mathsf{LC}(Q)\right)  \quad \text{ where } \quad \mathsf{LC}(Q):=p(Q = a_{Q,\max})
\end{align*}

\vspace{1em}

\noindent\emph{Margin Sampling ($\mathsf{M}$)}: 
Selects the query $Q_{\mathsf{M}}$ for which the probabilities between most and second most likely label $a_{Q,1}$ and $a_{Q,2}$ are most similar. Formally: 
\begin{align*}
Q_{\mathsf{M}} := \argmin_{Q \in \mQ} \left(\mathsf{M}(Q)\right) \quad \text{ where } \quad \mathsf{M}(Q):=p(Q=a_{Q,1}) - p(Q=a_{Q,2})
\end{align*}
	
	\vspace{1em}
	
\noindent\emph{Entropy ($\mathsf{H}$)}:
Selects the query $Q_{\mathsf{H}}$ whose outcome is most uncertain in terms of information entropy. Formally: 
\begin{align}
Q_{\mathsf{H}} := \argmax_{Q \in \mQ} \left(\mathsf{H}(Q)\right) \quad \text{ where } \quad \mathsf{H}(Q):=-\sum_{a \in\setof{t,f}} p(Q=a) \log_2 p(Q=a) \label{eq:Q_H}
\end{align}
\begin{proposition}\label{prop:uncertainty_sampling}
Let $\mD \subseteq \minD_{\langle\mo,\mb,\Tp,\Tn\rangle_\RQ}$ a set of leading diagnoses and $p$ a probability measure. 
Then $\mathsf{LC}\equiv\mathsf{M}\equiv\mathsf{H}$.
A query $Q$ is theoretically optimal w.r.t.\ $\mathsf{LC}$, $\mathsf{M}$ and $\mathsf{H}$ iff $p(\dx{}(Q))=p(\dnx{}(Q))$.
\end{proposition}
\begin{proof}
The measure $\mathsf{M}$ evaluates a query $Q$ the better, the lower $p(Q=a_{Q,1}) - p(Q=a_{Q,2})$ is. By the existence of only two query outcomes, $p(Q=a_{Q,1}) - p(Q=a_{Q,2}) = p(Q=a_{Q,\max}) - (1 - p(Q=a_{Q,\max})) = 2\, p(Q=a_{Q,\max}) - 1$. This is also the result of the application of a linear function $f(x) = 2x-1$ to $p(Q=a_{Q,\max})$. The latter is used as a criterion to be minimized by the best query w.r.t.\ the measure $\mathsf{LC}$. By Proposition~\ref{prop:linear_function_of_measure_is_equivalent_to_measure_itself}, $\mathsf{M} \equiv \mathsf{LC}$ holds.

Let $p:=p(Q=a_{Q,\min})$. Due to $p(\md)>0$ for all $\md\in\mD$ and Proposition~\ref{prop:properties_of_q-partitions},(\ref{prop:properties_of_q-partitions:enum:for_each_q-partition_dx_is_empty_and_dnx_is_empty}.), it holds that $p\in(0,1)$. It is well known that a local maximum of a concave function in one variable is at the same time a global maximum. Moreover, a function in one variable is concave iff its first derivative is monotonically decreasing. We observe that $\mathsf{H}(Q)$ is concave for $p \in (0,1)$ as the first derivative $\log_2 p - \log_2 (1-p)$ of $\mathsf{H}(Q)$ is monotonically decreasing due to the fact that the second derivative $\frac{1}{p}+\frac{1}{1-p}$ is clearly greater than $0$ for $p\in (0,1)$. Setting the first derivative equal to $0$ yields $p=\frac{1}{2}$, i.e.\ $p(Q'=a_{Q',\max}) = p(Q'=a_{Q',\min}) = \frac{1}{2}$ as the property of the query $Q'$ for which the global maximum of $\mathsf{H}(Q)$ is attained. By Eq.~\eqref{eq:p(Q=t)} and Eq.~\eqref{eq:p(Q=f)} this means that $p(\dx{}(Q'))+\frac{p(\dz{}(Q'))}{2}=p(\dnx{}(Q'))+\frac{p(\dz{}(Q'))}{2}$, i.e.\ $p(\dx{}(Q'))=p(\dnx{}(Q'))$. Consequently, the property of a theoretically optimal query $Q$ w.r.t.\ $\mathsf{H}$ is $p(\dx{}(Q))=p(\dnx{}(Q))$.

Since $\mathsf{H}$ is monotonically increasing for $p\in(0,\frac{1}{2}]$, this means that a query $Q$ is the better w.r.t.\ $\mathsf{H}$ the higher the lower probability $p(Q=a_{Q,\min})$ (and thence the lower the higher probability $p(Q=a_{Q,\max})$) which implies that $\mathsf{LC} \equiv \mathsf{H}$. Because $\equiv$ is an equivalence relation by Proposition~\ref{prop:equivalence_between_measures_is_equivalence_relation} and therefore is transitive, we have that $\mathsf{LC}\equiv\mathsf{M}\equiv\mathsf{H}$.

Finally, due to the equivalence between $\mathsf{LC}$, $\mathsf{M}$ and $\mathsf{H}$ and Proposition~\ref{prop:equivalent_measures_suggest_equivalent_optimal_queries}, the property of a theoretically optimal query $Q$ w.r.t.\ $\mathsf{LC}$, $\mathsf{M}$ and $\mathsf{H}$ is $p(\dx{}(Q))=p(\dnx{}(Q))$.
\end{proof}
In fact, the $\mathsf{H}$ measure is only a slight modification of the $\mathsf{ENT}$ measure~\cite{Shchekotykhin2012} and both measures coincide for queries $Q$ satisfying $\dz{}(Q) = \emptyset$. 
\begin{proposition}\label{prop:Q_H_eq_Q_ENT_if_dz=0}
Let $\mQ$ comprise only queries $Q$ with $\dz{}(Q) = \emptyset$. Then $\mathsf{H} \equiv_{\mQ} \mathsf{ENT}$.
\end{proposition}
\begin{proof}
We have seen that $\mathsf{ENT}(Q)$ can be represented as in Eq.~\eqref{eq:scoring_funtion_dekleer} and that the query selected by $\mathsf{ENT}$ is given by $\argmin_{Q \in \mQ} (\mathsf{ENT}(Q))$ (cf.\ Eq.~\eqref{eq:best_query_ENT}). Further, we observe that the query selected by $\mathsf{H}$ is given by $Q_{\mathsf{H}} = \argmin_{Q \in \mQ} \left(\sum_{a \in\setof{t,f}} p(Q=a) \log_2 p(Q=a)\right)$ by writing it equivalently as in Eq.~\eqref{eq:Q_H} with $\argmin$ instead of $\argmax$, but without the minus sign. Now, considering Eq.~\eqref{eq:scoring_funtion_dekleer}, since $\dz{}(Q) = \emptyset$ which means $p(\dz{}(Q)) = 0$ and since an addend of 1 has no bearing on the result of the $\argmin$, the proposition follows immediately.
\end{proof}
We point out that the difference between $\mathsf{H}$ and $\mathsf{ENT}$ is that $p(\dz{}(Q))$ is taken into account as a penalty in $\mathsf{ENT}$ which comes in handy for our debugging scenario where we always want to have zero diagnoses that can definitely not be invalidated by the query $Q$, i.e.\ $\dz{}(Q) = \emptyset$.  

However, if $\dz{}(Q) = \emptyset$ is not true for some query, one must be careful when using $\mathsf{H}$ since it might suggest counterintuitive and less favorable queries than $\mathsf{ENT}$. The following example illustrates this fact: 
\begin{example}\label{ex:Q_H_vs_Q_ENT}
Suppose two queries $Q_1,Q_2$ in $\mQ$ where 
\begin{align*}
\langle p(\dx{}(Q_1)),p(\dnx{}(Q_1)),p(\dz{}(Q_1)) \rangle &= \langle 0.49,0.49,0.02 \rangle \\
\langle p(\dx{}(Q_2)),p(\dnx{}(Q_2)),p(\dz{}(Q_2)) \rangle &= \langle 0.35,0.35,0.3 \rangle
\end{align*}
Of course, we would prefer $Q_1$ since it is a fifty-fifty decision (high information gain, no tendency towards any answer, cf.\ Eq.~\eqref{eq:p(Q=t)}) and only some very improbable diagnosis or diagnoses cannot be ruled out in any case. $Q_2$ is also fifty-fifty w.r.t.\ the expected query answer, but with a clearly worse setting of $\dz{}$. In other words, $Q_1$ will eliminate a set of hypotheses that amounts to a probability mass of $0.49$ (either all diagnoses in $\dx{}(Q_1)$ or all diagnoses in $\dnx{}(Q_1)$, cf.\ Section~\ref{sec:InteractiveKnowledgeBaseDebuggingBasics}). For $Q_2$, on the other hand, diagnoses taking only $35\%$ of the probability mass are going to be ruled out.
 Evaluating these queries by means of $\mathsf{H}$ in fact does not lead to the preference of either query. Hence, a selection algorithm based on $\mathsf{H}$ might suggest any of the two (e.g.\ $Q_2$) as the better query. 
\qed 
\end{example}
\begin{corollary}\label{cor:ENT_superior_to_H_M_LC}\leavevmode
\begin{enumerate}
	\item For any real number $q > 0$, $\mathsf{ENT}_q$ is superior to $\mathsf{H}$, $\mathsf{LC}$ and $\mathsf{M}$.
	\item $\mathsf{ENT}$ is superior to $\mathsf{H}$, $\mathsf{LC}$ and $\mathsf{M}$.
\end{enumerate}
\end{corollary}
\begin{proof}
We prove both statements in turn:
\begin{enumerate}
	\item As we noted above, $\mathsf{ENT}_0 - 1$ is equal to $\mathsf{H}$, $\mathsf{LC}$ and $\mathsf{M}$ and thus equivalent to $\mathsf{H}$, $\mathsf{LC}$ and $\mathsf{M}$. Now, by Proposition~\ref{prop:linear_function_of_measure_is_equivalent_to_measure_itself}, $\mathsf{ENT}_0 - 1 \equiv \mathsf{ENT}_0$. Hence, $\mathsf{ENT}_0 \equiv \mathsf{H} \equiv \mathsf{LC} \equiv \mathsf{M}$. The statement follows immediately from this by application of Corollary~\ref{cor:if_0<=r<s_then_ENTr_inferior_to_ENTs}.
	\item This statement is a special case of statement (1.)\ since $\mathsf{ENT}(Q) = \mathsf{ENT}_1(Q)$.
\end{enumerate}
\end{proof}
Consequently, we can regard $\mathsf{ENT}$ as a more suitable measure for our debugging scenario than $\mathsf{LC}$, $\mathsf{M}$ and $\mathsf{H}$.
\begin{example}\label{ex:LC,M,H_inferior_to_ENT}
Substantiating the third bullet of Corollary~\ref{cor:ENT_superior_to_H_M_LC}, we illustrate in this example that there is a query $Q_1$ that is discrimination-preferred to another query $Q_2$ such that $Q_1 \prec_{\mathsf{H}} Q_2$ does not hold whereas $Q_1 \prec_{\mathsf{ENT}} Q_2$ does. To this end, reconsider the queries $Q_1,Q_2$ from Example~\ref{ex:Q_H_vs_Q_ENT}. Assume a concrete set of leading diagnoses $\mD$ and let 
\begin{align*}
\tuple{\dx{}(Q_1),\dnx{}(Q_1),\dz{}(Q_1)} &= \tuple{\setof{\md_1,\md_2},\setof{\md_3,\md_4},\setof{\md_5}} \\
\tuple{\dx{}(Q_2),\dnx{}(Q_2),\dz{}(Q_2)} &= \tuple{\setof{\md_1},\setof{\md_3},\setof{\md_5,\md_2,\md_4}} 
\end{align*}
with probabilities $p_i := p(\md_i)$ as follows: $\tuple{p_1,\dots,p_5} = \tuple{0.35,0.14,0.35,0.14,0.02}$. 
It is obvious that $Q_1$ is discrimination-preferred to $Q_2$ (cf.\ Definition~\ref{def:measures_equivalent_theoretically-optimal_superior}). That $Q_1 \prec_{\mathsf{H}} Q_2$ does not hold and $Q_1 \prec_{\mathsf{ENT}} Q_2$ does was discussed in Example~\ref{ex:Q_H_vs_Q_ENT}. 
\qed
\end{example}
\begin{proposition}\label{prop:M,LC,H_not_consistent_with_DPR}
The measures $\mathsf{H}$, $\mathsf{LC}$ and $\mathsf{M}$ are not consistent with the discrimination-preference relation $DPR$ (and thus do not satisfy $DPR$ either).
\end{proposition}
\begin{proof}
Reconsider Example~\ref{ex:LC,M,H_inferior_to_ENT} and assume the following diagnoses probabilities $p_i := p(\md_i)$: \[\tuple{p_1,\dots,p_5} = \tuple{0.3,0.13,0.3,0.01,0.26}\] Then, $\mathsf{LC}(Q_2) = \frac{1}{2} < \mathsf{LC}(Q_1) = 0.57$. Hence, $Q_2 \prec_{\mathsf{LC}} Q_1$. As $\mathsf{H} \equiv \mathsf{LC} \equiv \mathsf{M}$ by Proposition~\ref{prop:uncertainty_sampling}, we conclude by Definition~\ref{def:measures_equivalent_theoretically-optimal_superior} that $Q_2 \prec_{\mathsf{M}} Q_1$ and $Q_2 \prec_{\mathsf{H}} Q_1$ must hold as well. Since $Q_1$ is discrimination-preferred to $Q_2$ (cf.\ Example~\ref{ex:LC,M,H_inferior_to_ENT}), the measures $\mathsf{H}$, $\mathsf{LC}$ and $\mathsf{M}$ are not consistent with the discrimination-preference relation $DPR$.
\end{proof}

\vspace{1em}

\paragraph{Query-By-Committee (QBC).} QBC criteria involve maintaining a committee $C$ of competing hypotheses which are all in accordance with the set of queries answered so far. Each committee member has a vote on the classification of a new query candidate. The candidate that yields the highest disagreement among all committee members is considered to be most informative. The measures belonging to this framework are the \emph{vote entropy} and the \emph{Kullback-Leibler-Divergence} which we discuss in turn.

\vspace{1em}

\noindent\emph{Vote Entropy ($\mathsf{VE}$)}: Selects the query 
\begin{align}
Q_{\mathsf{VE}} := \argmax_{Q \in \mQ_\mD} \left(\mathsf{VE}(Q)\right) \quad \text{ where } \quad \mathsf{VE}(Q) := - \sum_{a\in\setof{t,f}} \frac{|C_{Q=a}|}{|C|} \log_2 \left(\frac{|C_{Q=a}|}{|C|}\right)   \label{eq:Q_VE}
\end{align}
about whose classification there is the largest discrepancy in prediction between the committee members in $C$ where $C_{Q=a} \subseteq C$ terms the set of committee members that predict $Q=a$. At this, the most likely hypotheses $\mD$ act as the committee $C$. A diagnosis $\md\in\mD$ votes for $Q=t$ iff $\md\in\dx{}(Q)$ and for $Q=f$ iff $\md\in\dnx{}(Q)$ (cf.\ Eq.~\eqref{eq:cond_query_prob}). Diagnoses in $\dz{}(Q)$ do not provide a vote for any answer (cf.\ Eq.~\eqref{eq:cond_query_prob} and Proposition~\ref{prop:properties_of_q-partitions},(\ref{prop:properties_of_q-partitions:enum:dx_dnx_dz_contain_exactly_those_diags_that_are...}.)). Therefore, more precisely, the (actually voting) committee $C := \mD\setminus \dz{}(Q) = \dx{}(Q) \cup \dnx{}(Q)$. 
If we apply this to the formula for $Q_{\mathsf{VE}}$, it becomes evident that:
\begin{proposition}\label{prop:theoretically_opt_query_wrt_VE}
Then the theoretically optimal query $Q$ w.r.t.\ $\mathsf{VE}$ and $\mD$ satisfies $|\dx{}(Q)| = |\dnx{}(Q)|$.
\end{proposition}
\begin{proof}
Let $Q$ by a query and $c := |\mD|-y = |\dx{}(Q)| + |\dnx{}(Q)|$ for some $y \in \setof{0,\dots,|\mD|-2}$ be the size of the voting committee $C$ w.r.t.\ $Q$. That is, $|C| = c$. Let further $c_t = |C_{Q=t}| = |\dx{}(Q)|$. We have to find the maximum of the function $f(c_t):= - \frac{c_t}{c} \log_2 \left(\frac{c_t}{c}\right) - \frac{c-c_t}{c} \log_2 \left(\frac{c-c_t}{c}\right)$.
The first derivative of this function is given by $\frac{1}{c}[\log_2(\frac{c-c_t}{c})-\log_2(\frac{c_t}{c})]$. Setting this derivative equal to $0$ yields $c_t = \frac{c}{2}$. Since $-f(c_t)$ is convex, i.e.\ $f(c_t)$ concave, $\frac{c}{2}$ is clearly the point where the global maximum of $f(c_t)$ is attained (see the proof of Proposition~\ref{prop:ENT_preserves_discrimination-pref-order} for a more lengthy argumentation regarding similar facts). 
Consequently, the theoretical optimum w.r.t.\ $\mathsf{VE}$ is attained if $|C_{Q=t}| = |C_{Q=f}| = \frac{c}{2}$, i.e.\ if $|\dx{}(Q)| = |\dnx{}(Q)|$.
\end{proof}
%
%
%
Thus, 
$\mathsf{VE}$ is similar to $\mathsf{SPL}$ and both measures coincide for queries $Q$ satisfying $\dz{}(Q) = \emptyset$.
\begin{proposition}\label{prop:Q_VE_eq_Q_SPL_if_dz=0}\leavevmode
\begin{enumerate}
	\item Let $\mQ$ comprise only queries $Q$ with $\dz{}(Q) = \emptyset$. Then $\mathsf{VE} \equiv_{\mQ} \mathsf{SPL}$.
	\item $\mathsf{VE} \equiv \mathsf{SPL}_0$.
\end{enumerate}
\end{proposition}
\begin{proof}
Let $Q,Q'$ be arbitrary queries in $\mQ$. Clearly, $Q$ being preferred to $Q'$ by $\mathsf{SPL}$ is equivalent to $\left||\dx{}(Q)| - |\dnx{}(Q)|\right| < \left||\dx{}(Q')| - |\dnx{}(Q')|\right|$. Due to the argumentation given in the proof of Proposition~\ref{prop:theoretically_opt_query_wrt_VE} (concavity of $f(c_t)$), this is equivalent to $Q$ being preferred to $Q'$ by $\mathsf{VE}$. This proves (1.). Bullet (2.) is shown analogously.
\end{proof}
\begin{corollary}\label{cor:SPL_z_for_all_z_equivalent_mQ_to_VE_for_mQ_includes_only_dz=0_queries}
Let $\mQ$ be a set of queries w.r.t.\ a set of minimal diagnoses w.r.t.\ a DPI $\tuple{\mo,\mb,\Tp,\Tn}_\RQ$ where each query $Q \in \mQ$ satisfies $\dz{}(Q) = \emptyset$. Then $\mathsf{SPL}_z \equiv_{\mQ} \mathsf{VE}$ for all $z \in \mathbb{R}$.
\end{corollary}
\begin{proof}
Immediate from Proposition~\ref{prop:Q_VE_eq_Q_SPL_if_dz=0} and the fact that $\mathsf{SPL}_z$ is obviously equal to $\mathsf{SPL}$ for all $Q\in\mQ$ (due to $\dz{}(Q) = \emptyset$) for all $z \in \mathbb{R}$.
\end{proof}
We note that $\mathsf{SPL}$, contrary to $\mathsf{VE}$, includes a penalty corresponding to the size of $\dz{}(Q)$. However, in spite of the high similarity between $\mathsf{SPL}$ and $\mathsf{VE}$, one still must be careful when using $\mathsf{VE}$ as it might suggest counterintuitive queries. This is analogical to the comparison we made between $\mathsf{H}$ and $\mathsf{ENT}$. Let us again exemplify this:
%
%
\begin{example}\label{ex:VE_vs_SPL}
Suppose two queries $Q_1,Q_2$ in $\mQ$ where 
\begin{align*}
\langle |\dx{}(Q_1)|,|\dnx{}(Q_1)|,|\dz{}(Q_1)| \rangle &= \langle 4,4,2 \rangle \\
\langle |\dx{}(Q_2)|,|\dnx{}(Q_2)|,|\dz{}(Q_2)| \rangle &= \langle 1,1,8 \rangle
\end{align*}
The queries $Q_1$ and $Q_2$ both feature the largest possible discrepancy in prediction between committee members (four diagnoses predict the positive and four the negative answer in case of $Q_1$, comparing with a $1$-to-$1$ situation in case of $Q_2$). Clearly, $\mathsf{VE}$ evaluates both queries as equally good. However, the setting of $Q_2$ w.r.t.\ $\dz{}$ is clearly worse, since eight leading diagnoses (versus only two for $Q_2$) can definitely not be ruled out after having gotten an answer to $Q_2$. In other words, $Q_1$ will rule out only a guaranteed number of four leading diagnoses whereas $Q_1$ guarantees only the elimination of a single leading diagnosis. This aspect is taken into account by $\mathsf{SPL}$ which is why $\mathsf{SPL}$ favors $Q_1$ outright ($\mathsf{SPL}(Q_1) = 2$, $\mathsf{SPL}(Q_2) = 8$). To sum up, $\mathsf{SPL}$ goes for the query that we clearly prefer in the interactive debugging scenario, whereas $\mathsf{VE}$ could possibly suggest $Q_2$ instead of $Q_1$. Note also that $Q_1, Q_2$ can be specified in a way that $Q_1$ is discrimination-preferred to $Q_2$, e.g.\ if 
\begin{align*}
\langle \dx{}(Q_1),\dnx{}(Q_1),\dz{}(Q_1) \rangle &= \langle \setof{\md_1,\dots,\md_4},\setof{\md_5,\dots,\md_8},\setof{\md_9,\md_{10}} \rangle \\
\langle \dx{}(Q_2),\dnx{}(Q_2),\dz{}(Q_2) \rangle &= \langle \setof{\md_1},\setof{\md_5},\setof{\md_2,\dots,\md_4,\md_6,\dots,\md_8,\md_9,\md_{10}} \rangle
\end{align*}
As a consequence of this, we can declare that $\mathsf{VE}$ does not satisfy the discrimination-preference relation $DPR$ (the fact that $Q_1$ is discrimination-preferred to $Q_2$ can be easily verified by consulting Proposition~\ref{prop:construction_of_Q'_from_Q_if_Q_discrimination-preferred_over_Q'}). When we consider another example, i.e.\ 
\begin{align*}
\langle \dx{}(Q_3),\dnx{}(Q_3),\dz{}(Q_3) \rangle &= \langle \setof{\md_1,\dots,\md_3},\setof{\md_5,\dots,\md_8},\setof{\md_4,\md_9,\md_{10}} \rangle \\
\langle \dx{}(Q_4),\dnx{}(Q_4),\dz{}(Q_4) \rangle &= \langle \setof{\md_1},\setof{\md_5},\setof{\md_2,\dots,\md_4,\md_6,\dots,\md_8,\md_9,\md_{10}} \rangle
\end{align*}
we directly see that $Q_4 \prec_{\mathsf{VE}} Q_3$ (q-partition set cardinalities $\langle 3,4,3 \rangle$ versus $\langle 1,1,8 \rangle$) which is why $Q_3 \prec_{\mathsf{VE}} Q_4$ cannot hold ($\prec_{\mathsf{VE}}$ is strict order, cf.\ Proposition~\ref{prop:precedence_order_is_strict_order}) which in turn is why $\mathsf{VE}$ is not consistent with the discrimination-preference relation $DPR$ (the fact that $Q_3$ is discrimination-preferred to $Q_4$ can be easily verified by consulting Proposition~\ref{prop:construction_of_Q'_from_Q_if_Q_discrimination-preferred_over_Q'}). Furthermore, we point out that $\mathsf{SPL}(Q_3) = 1+3 = 4 < 8 = 0+8 = \mathsf{SPL}(Q_4)$. Therefore, $Q_3 \prec_{\mathsf{SPL}} Q_4$. \qed
\end{example}

\begin{corollary}\label{cor:VE_not_consistent_with_DPR}
$\mathsf{VE}$ is not consistent with the discrimination-preference relation $DPR$ (and thus does not satisfy $DPR$ either).
\end{corollary}
\begin{proof}
Reconsider queries $Q_3$ and $Q_4$ from Example~\ref{ex:VE_vs_SPL}. There, we showed that $Q_4 \prec_{\mathsf{VE}} Q_3$ although $Q_3$ is discrimination-preferred to $Q_4$. This completes the proof.
\end{proof}
\begin{proposition}\label{prop:SPL_superior_to_VE} 
$\mathsf{SPL}$ is superior to $\mathsf{VE}$.
\end{proposition}
\begin{proof}
Let $d(Z):=\left|\, |\dx{}(Z)| - |\dnx{}(Z)| \,\right|$ for some query $Z$ and let $Q,Q'$ be queries such that $Q$ is discrimination-preferred to $Q'$. Further, assume that not $Q \prec_{\mathsf{SPL}} Q'$, i.e.\ $d(Q) - d(Q') = (|\dz{}(Q')| - |\dz{}(Q)|)$
(cf.\ Proposition~\ref{prop:spl_consistent_with_DPR_but_does_not_satisfy_DPR}). Due to Proposition~\ref{prop:if_Q_discrimination-preferred_over_Q'_then_dz(Q')_supset_dz(Q)}, $|\dz{}(Q')| > |\dz{}(Q)|$ must hold. Therefore, it must be true that $d(Q') < d(Q)$. By the argumentation used in the proof of Proposition~\ref{prop:spl} (concavity), $\mathsf{VE}$ evaluates a query $Z$ the better, the lower $d(Z)$ is. Hence, $Q' \prec_{\mathsf{VE}} Q$ which is why $Q \prec_{\mathsf{VE}} Q'$ cannot hold by Proposition~\ref{prop:precedence_order_is_strict_order} (asymmetry of the strict order $\prec_{\mathsf{VE}}$). Therefore, the third bullet in the definition of superiority (Definition~\ref{def:measures_equivalent_theoretically-optimal_superior}) is met. The first and second bullets are satisfied as is illustrated by queries $Q_3,Q_4$ in Example~\ref{ex:VE_vs_SPL}. 
\end{proof}
\begin{corollary}\label{cor:SPLz_superior_to_VE_for_z>1}
$\mathsf{SPL}_z$ is superior to $\mathsf{VE}$ for all real numbers $z\geq 1$.
\end{corollary}
\begin{proof}
Due to Corollary~\ref{cor:SPL2_is_superior_to_SPL}, $\mathsf{SPL}_z$ is superior to $\mathsf{SPL}$ for all real numbers $z>1$. Since Proposition~\ref{prop:SPL_superior_to_VE} witnesses that $\mathsf{SPL}$ (which is equal to $\mathsf{SPL}_1$) is superior to $\mathsf{VE}$, we obtain by application of Proposition~\ref{prop:superiority_is_strict_order} (transitivity of the superiority relation) that $\mathsf{SPL}_z$ is superior to $\mathsf{VE}$ for all real numbers $z \geq 1$. 
\end{proof}
%

\vspace{1em}

\noindent\emph{Kullback-Leibler-Divergence ($\mathsf{KL}$)}: Selects the query 
\begin{align}
Q_{\mathsf{KL}} := \argmax_{Q\in\mQ_\mD} \left(\mathsf{KL}(Q)\right) \quad \text{ where } \quad \mathsf{KL}(Q) := \frac{1}{|C|} \sum_{c\in C} D_{\mathsf{KL}}(p_c(Q)||p_\all(Q))\label{eq:Q_KL}
\end{align}
that manifests the largest average disagreement $D_{\mathsf{KL}}(p_c(Q)||p_\all(Q))$ between the label distributions of any (voting) committee member $c$ and the consensus $\all$ of the entire (voting) committee $C$. Adapted to the debugging scenario, $\all$ corresponds to the ``opinion'' of $C = \dx{}(Q) \cup \dnx{}(Q)$ and $c$ corresponds to the ``opinion'' of some $\md\in\dx{}(Q) \cup \dnx{}(Q)$ (cf.\ Proposition~\ref{prop:properties_of_q-partitions},(\ref{prop:properties_of_q-partitions:enum:dx_dnx_dz_contain_exactly_those_diags_that_are...}.) which says that diagnoses in $\dz{}(Q)$ do not predict, or vote for, any answer to $Q$). Note that the committee $C$ is in general different for different queries. Let for the remaining discussion of the $\mathsf{KL}$ measure the probabilities of the diagnoses in $\dx{}(Q) \cup \dnx{}(Q)$ be normalized in a way that they sum up to $1$, i.e.\ $p(\dx{}(Q)) + p(\dnx{}(Q)) = 1$. This is accomplished by $p(\md_i) \gets p(\md_i) / \sum_{\md\in\dx{}(Q)\cup\dnx{}(Q)} p(\md)$ for all $\md_i \in \dx{}(Q)\cup\dnx{}(Q)$ and $p(\md_i) \gets 0$ for all $\md_i \in \dz{}(Q)$. Further, we use $p_x(Q=a)$ as a shorthand for $p(Q=a\,|\,x)$ for $x\in\setof{c,\all}$.
According to Eq.~\eqref{eq:cond_query_prob}, under the hypothesis $c=\md$ it holds that 
\begin{align} 
p_c(Q=t)=1  \quad&\text{ for } \quad \md\in\dx{}(Q) \label{eq:p_c(Q=t)} \\ 
p_c(Q=f)=1  \quad&\text{ for } \quad \md\in\dnx{}(Q) \label{eq:p_c(Q=f)}   
\end{align}
The probability 
\begin{align*}
p_\all(Q=a) := \frac{1}{|C|}\sum_{c\in C} p_c(Q=a)
\end{align*}
for $a\in\setof{t,f}$ expresses the ``average prediction tendency'' of the committee $C$ as a whole under the assumption of equal likeliness or weight $\frac{1}{|C|}$ of each committee member $c \in C$. 

However, as probabilities of leading diagnoses are maintained in our debugging scenario, these can be taken into account by assigning to $c = \md$ the weight $p(c) := p(\md)$ instead of the factor $\frac{1}{|C|}$, resulting in another variant of $p_\all(Q=a)$ which reads as
\begin{align}
p_\all(Q=a) := \sum_{c\in C} p(c)\, p_c(Q=a)		\label{eq:KL_p_all(Q=a)}
\end{align}
and is equal to $p(\dx{}(Q))$ for ${a = t}$ and equal to $p(\dnx{}(Q))$ for ${a = f}$ as per Eq.~\eqref{eq:sum_of_probs_of_diagnoses_set}. After normalization, we obtain $p_\all(Q=t) = \frac{p(\dx{}(Q))}{p(\dx{}(Q))+p(\dnx{}(Q))}$ and $p_\all(Q=f) = \frac{p(\dnx{}(Q))}{p(\dx{}(Q))+p(\dnx{}(Q))}$. Note that the normalization is necessary since (1)~there are only two possible answers $t$ and $f$ for any query and (2)~$p(\dx{}(Q))+p(\dnx{}(Q)) < 1$ given that $p(\dz{}(Q)) > 0$.
Using Eq.~\eqref{eq:KL_p_all(Q=a)}, the disagreement measure given in \cite[p.~29, bottom line]{settles2012} can be defined for the debugging scenario as 
\begin{align*}
D_{\mathsf{KL}}\left(p_c(Q)||p_\all(Q)\right) := \sum_{a\in\setof{t,f}} p_c({Q=a}) \log_2 \left(\frac{p_c(Q=a)}{p_\all(Q=a)}\right)
\end{align*}
\begin{proposition}\label{prop:kl_derived}
Let $Q$ be a query and $\mD_Q := \dx{}(Q)\cup\dnx{}(Q)$. 
Then $\mathsf{KL}(Q)$ can be equivalently written as
\begin{align}
\left(-\sum_{\mD^*\in\setof{\dx{}(Q),\dnx{}(Q)}}\frac{|\mD^*|}{|\mD_Q|} \log_2\left(\frac{p(\mD^*)}{p(\mD_Q)}\right)\right)   \label{eq:Q_KL_derived}
\end{align}
\end{proposition}
\begin{proof}
As argued before, we have that $p_\all(Q=t)=\frac{p(\dx{}(Q))}{p(\mD_Q)}$ and $p_\all(Q=f)=\frac{p(\dnx{}(Q))}{p(\mD_Q)}$. By Equations~\eqref{eq:p_c(Q=t)} and \eqref{eq:p_c(Q=f)}, $p_c(Q=a)=0$ if $c=\md\in\dx{}(Q)$ and $a=f$, or if $c=\md\in\dnx{}(Q)$ and $a=t$, and $p_c(Q=a)=1$ otherwise. 
In case $p_c(Q=a)=0$ it holds that $p_c(Q=a) \log_2 \left(\frac{p_c(Q=a)}{p_\all(Q=a)}\right) = 0 \,log_2 0$ which is equal to zero by convention (see page~\pageref{convention:0log0_is_0}). 
Consequently, if $c=\md\in\dx{}(Q)$, then 
\begin{align*}
D_{\mathsf{KL}}(p_c(Q)||p_\all(Q)) := \log_2 \left(\frac{1}{p_\all(Q=t)}\right) = - \log_2 (p_\all(Q=t))
\end{align*}
Analogously, we obtain 
\begin{align*}
D_{\mathsf{KL}}(p_c(Q)||p_\all(Q)) := - \log_2 (p_\all(Q=f))
\end{align*}
for $c=\md\in\dnx{}(Q)$. Since $\mathsf{KL}(Q) := \frac{1}{|C|} \sum_{c\in C} D_{\mathsf{KL}}(p_c(Q)||p_\all(Q))$ and by the derivation of $D_{\mathsf{KL}}(p_c(Q)||p_\all(Q))$ above, we observe that we add $|\dx{}(Q)|$ times the term $- \frac{1}{|C|}\,\log_2 (p_\all(Q=t))$ and $|\dnx{}(Q)|$ times the term $- \frac{1}{|C|}\,\log_2 (p_\all(Q=f))$, which leads to the statement of the proposition by the fact that $|C|=|\mD_Q|$.
\end{proof}
The next proposition is based on Proposition~\ref{prop:kl_derived} and constitutes the key for the design of suitable heuristics that can be used by a q-partition search in order to locate the q-partition with best $\mathsf{KL}$ measure (among a set of queries) efficiently. It testifies that there is no theoretically optimal query w.r.t.\ $\mathsf{KL}$ since, in fact, the $\mathsf{KL}$ measure can be arbitrarily improved by decreasing the probability of one of the sets $\dx{}(Q)$ or $\dnx{}(Q)$. So, there is no global maximum of $\mathsf{KL}$. However, we are still able to derive a subset of q-partitions that must include the best q-partition w.r.t.\ $\mathsf{KL}$ among a given set of queries (with empty $\dz{}$-set). Roughly, a best q-partition w.r.t.\ $\mathsf{KL}$ must feature either a $\dx{}$-set with maximal probability in relation to its cardinality or a $\dnx{}$-set with maximal probability in relation to its cardinality. 
This insight can be exploited to devise a search that attempts to examine q-partitions in a way such q-partitions are preferred. A complete restriction to only such q-partitions is though not easy to achieve as not all subsets of the leadings diagnoses constitute $\dx{}$- or $\dnx{}$-sets of a q-partition and since therefore a decision between $\dx{}$- or $\dnx{}$-sets of different cardinality must in general be made at a certain point during the search. The class of these preferred q-partitions tends to be small the higher the variability in the probabilities of the leading diagnoses is because facing different $\dx{}$- or $\dnx{}$-sets with one and the same cardinality and probability is unlikely in such a situation. 
\begin{proposition}\label{prop:kl_opt}
Let $\mD$ be a set of minimal diagnoses w.r.t.\ a DPI $\tuple{\mo,\mb,\Tp,\Tn}_\RQ$. Then, for $\mathsf{KL}$ the following holds:
\begin{enumerate}
	\item \label{prop:kl_opt:no_theoretically_opt_query} There is no theoretically optimal query w.r.t.\ $\mathsf{KL}$ and $\mD$.
	\item \label{prop:kl_opt:for_any_given_query_better_query_can_be_found} If the (diagnosis) probability measure $p$ is not assumed fixed, then for any given query $Q$ w.r.t.\ $\mD$ and $\tuple{\mo,\mb,\Tp,\Tn}_\RQ$ a query $Q'$ w.r.t.\ $\mD$ and $\tuple{\mo,\mb,\Tp,\Tn}_\RQ$ satisfying $Q' \prec_{\mathsf{KL}} Q$ can be found.
	\item \label{prop:kl_opt:properties_of_best_query} Let $\mQ$ be any set of queries w.r.t.\ a DPI $\tuple{\mo,\mb,\Tp,\Tn}_\RQ$ and a set of leading diagnoses $\mD$ w.r.t.\ this DPI such that each query $Q\in\mQ$ satisfies $\dz{}(Q) = \emptyset$. If the (diagnosis) probability measure $p$ is assumed fixed, then the query $Q \in \mQ$ that is considered best in $\mQ$ by $\mathsf{KL}$ (i.e.\ there is no query $Q' \in \mQ$ w.r.t.\ $\mD$ and $\tuple{\mo,\mb,\Tp,\Tn}_\RQ$ such that 
$Q' \prec_{\mathsf{KL}} Q$) satisfies $\dx{}(Q) \in \mathbf{MaxP}_k^+$ where \[\mathbf{MaxP}_k^+ := \setof{\mathbf{S}\,|\, \mathbf{S}\in \mathcal{S}_k \land \forall \mathbf{S}'\in \mathcal{S}^+_k: p(\mathbf{S}) \geq p(\mathbf{S}')}\] for some $k \in \setof{1,\dots,|\mD|-1}$ where 
\[\mathcal{S}^+_k = \setof{\mathbf{X}\,|\,\emptyset\subset\mathbf{X}\subset\mD,|\mathbf{X}|=k,\tuple{\mathbf{X},\mathbf{Y},\mathbf{Z}}\text{ is a q-partition w.r.t.\ } \mD, \tuple{\mo,\mb,\Tp,\Tn}_\RQ}\] 
or $\dnx{}(Q) \in \mathbf{MaxP}_m^-$ 
where \[\mathbf{MaxP}_m^- := \setof{\mathbf{S}\,|\, \mathbf{S}\in \mathcal{S}^-_m \land \forall \mathbf{S}'\in \mathcal{S}^-_m: p(\mathbf{S}) \geq p(\mathbf{S}')}\] for some $m \in \setof{1,\dots,|\mD|-1}$ where 
\[\mathcal{S}^-_m = \setof{\mathbf{Y}\,|\,\emptyset\subset\mathbf{Y}\subset\mD,|\mathbf{Y}|=m,\tuple{\mathbf{X},\mathbf{Y},\mathbf{Z}}\text{ is a q-partition w.r.t.\ } \mD, \tuple{\mo,\mb,\Tp,\Tn}_\RQ}\]
\end{enumerate}
\end{proposition}
\begin{proof}
We prove all statements (1.)--(3.) in turn.\footnote{In case the reader is unfamiliar with any terms or facts regarding calculus and optimization used in this proof, we recommend to consult \cite[p.~295 ff.]{nash1996}.}

Ad (1.): By the definition of theoretical optimality of a query w.r.t.\ a measure (cf.\ Definition~\ref{def:measures_equivalent_theoretically-optimal_superior}), we can assume that the probability measure $p$ is variable, i.e.\ not fixed. Let for each query $Q$ w.r.t.\ $\mD$ and $\tuple{\mo,\mb,\Tp,\Tn}_\RQ$ the set (representing the voting committee) $\mD_Q$ denote $\dx{}(Q)\cup\dnx{}(Q)$. By Proposition~\ref{prop:kl_derived}, we have to analyze the function $f(b,x) := -\frac{b}{d} \log_2(x) - \frac{d-b}{d} \log_2(1-x)$ with regard to its global optimum for $b\in\setof{1,\dots,d-1}$ (cf.\ Proposition~\ref{prop:properties_of_q-partitions},(\ref{prop:properties_of_q-partitions:enum:for_each_q-partition_dx_is_empty_and_dnx_is_empty}.)) and $x \in (0,1)$ (cf.\ the assumption that no diagnosis has zero probability on page~\pageref{etc:prob_of_each_diag_must_be_greater_zero}) where $b := |\dx{}(Q)|$, $x := \frac{p(\dx{}(Q))}{p(\mD_Q)}$ and $d := |\mD_Q|$. Now, let us apply the relaxation to the domain of $f(b,x)$ that $b\in [1,d-1]$ and that $d$ is fixed. This simplifies the analysis and is legal since we will demonstrate that there does not exist a global maximum. If this does not exist for the continuous interval $b\in [1,d-1]$, it clearly does not exist for the discrete values $1,\dots,d-1$ of $b$ either. Further, if the global maximum does not exist for a fixed (but arbitrary) $d$, it does not exist for any $d$.

A first step towards finding a global maximum is to analyze the function for local maxima. To this end, we start with the localization of stationary points. Let $f_b, f_x$ denote the partial derivatives of $f$ w.r.t.\ $b$ and $x$, respectively. Now, setting both $f_b = \frac{1}{d\cdot\ln 2} (\ln(\frac{1}{x}) - \ln(\frac{1}{1-x}))$ and $f_x = \frac{1}{d\cdot\ln 2} (\frac{d-b}{1-x} - \frac{b}{x})$ equal to zero, we get a system of two equations involving two variables. From the first equation we obtain $x = \frac{1}{2}$, the second lets us derive $\frac{b}{d-b} = \frac{x}{1-x}$. Using $x = \frac{1}{2}$, we finally get $b = \frac{d}{2}$. That is, there is a single stationary point, given by the $b$-$x$-coordinates $(\frac{d}{2},\frac{1}{2})$. In order to find out whether it is a local optimum or a saddle point, we build the Hessian 
\[H = 
\begin{pmatrix}
f_{bb} & f_{bx} \\
f_{xb} & f_{xx} \\
\end{pmatrix}
= 
\begin{pmatrix}
0 & \frac{1}{d\cdot\ln 2}(-\frac{1}{1 - x} - \frac{1}{x}) \\
\frac{1}{d\cdot\ln 2}(-\frac{1}{1 - x} - \frac{1}{x}) & \frac{1}{d\cdot\ln 2}(\frac{d-b}{(1 - x)^2} + \frac{b}{x^2} ) \\
\end{pmatrix}\]
where e.g.\ $f_{xb}$ denotes the result of computing the partial derivative of $f_x$ w.r.t.\ $b$ (analogous for the other matrix coefficients). It is well known that $f$ attains a local maximum at a point $x$ if the gradient of $f$ is the zero vector at this point and the Hessian at this point is negative definite. Since the gradient is the zero vector at $(\frac{d}{2},\frac{1}{2})$ (we obtained this stationary point exactly by setting the gradient equal to zero), we next check the Hessian for negative definiteness. This can be accomplished by computing the eigenvalues of $H$ at $(\frac{d}{2},\frac{1}{2})$, which is given by  
\[H\left(\frac{d}{2},\frac{1}{2}\right) = 
\begin{pmatrix}
0 & -\frac{4}{d\cdot\ln 2} \\
-\frac{4}{d\cdot\ln 2} & 0 \\
\end{pmatrix}
\]
The characteristic polynomial of $H(\frac{d}{2},\frac{1}{2})$ is then $\det(H(\frac{d}{2},\frac{1}{2})-\lambda\mathbb{I}) = \lambda^2 - (\frac{4}{d\cdot \ln 2})^2$ yielding the eigenvalues (roots of the characteristic polynomial) $\pm \frac{4}{d\cdot\ln 2}$. Since there is one positive and one negative eigenvalue, the Hessian at the single existing stationary point $(\frac{d}{2},\frac{1}{2})$ is indefinite which means that $(\frac{d}{2},\frac{1}{2})$ is a saddle point. Hence, no local or global maxima (or minima) exist for $f(b,x)$ over the given set of points. Since therefore there is no global maximum of $\mathsf{KL}$, there is no theoretically optimal query w.r.t.\ $\mathsf{KL}$ and $\mD$ (cf.\ Definition~\ref{def:measures_equivalent_theoretically-optimal_superior}).

Ad (2.): Let us consider again the function $f(b,x)$ defined in the proof of (1.). Let us further assume an arbitrary query $Q$ w.r.t.\ $\mD$ and $\tuple{\mo,\mb,\Tp,\Tn}_\RQ$. Let for this query $d$ be $d_Q$ and let this query have the value $\mathsf{KL}(Q) = f(b_Q,x_Q) = -\frac{b_Q}{d_Q} \log_2(x_Q) - \frac{d_Q-b_Q}{d_Q} \log_2(1-x_Q)$ where w.l.o.g.\ $b_Q \geq \frac{d_Q}{2}$. We observe that for some probability $q \in (0,1]$, $- \log_2(q) = \log_2(\frac{1}{q}) \geq 0$ increases for decreasing $q$. More precisely, $-\log_2(q) \rightarrow \infty$ for $q \rightarrow 0^+$. Now
we can increase $f(b_Q,x_Q)$ by subtracting some $\epsilon> \max\{0,2x_Q - 1\}$ from $x_Q$, i.e.\ $f(b_Q,x_Q-\epsilon) > f(b_Q,x_Q)$.

This must hold, first, because the logarithm $-\log_2(x_Q)$ with the higher or equal weight $\frac{b_Q}{d_Q}$ (holds due to $b_Q \geq \frac{d_Q}{2}$) in $f(b_Q,x_Q)$ is increased while the other logarithm $-\log_2 (1-x_Q)$ is decreased due to $\epsilon > 0$. Second, the increase of the former is higher than the decrease of the latter, i.e.\ $(*): -\log_2 (x_Q - \epsilon) - (- \log_2 (x_Q)) > -\log_2 (1 - x_Q) - (-\log_2 (1-x_Q + \epsilon))$, due to $\epsilon> 2x_Q - 1$. To understand that $(*)$ holds, let us transform it as follows:
\begin{align*}
\log_2 \left(\frac{1}{x_Q - \epsilon}\right) - \log_2 \left(\frac{1}{x_Q}\right) &> \log_2 \left(\frac{1}{1 - x_Q}\right) - \log_2 \left(\frac{1}{1-x_Q + \epsilon}\right) \\
\log_2 \left(\frac{x_Q}{x_Q - \epsilon}\right) &> \log_2 \left(\frac{1-x_Q + \epsilon}{1 - x_Q}\right) \\
\frac{x_Q}{x_Q - \epsilon} &> \frac{1-x_Q + \epsilon}{1 - x_Q} \\
x_Q(1-x_Q) &> (1-x_Q + \epsilon)(x_Q-\epsilon) \\
x_Q - x_{Q}^2 &> x_Q - \epsilon - x_Q^2 + 2x_Q\epsilon-\epsilon^2 \\
0 &> \epsilon (2x_Q - 1 - \epsilon) \\
\epsilon &> 2x_Q - 1 
\end{align*}
Hence, we see that $\epsilon > 2x_Q - 1$ iff $(*)$.

Note, since the interval for $x_Q$ is open, i.e.\ $(0,1)$, we can always find a (sufficiently small) $\epsilon$ such that $x_Q - \epsilon > 0$. To verify this, assume first that $\max\{0,2x_Q - 1\} = 0$. In this case, since $x_Q > 0$, $\epsilon$ can be chosen so small that $x_Q - \epsilon > 0$. On the other hand, if $\max\{0,2x_Q - 1\} = 2x_Q - 1$, we observe that $x_Q - \epsilon < x_Q - 2x_Q + 1 = 1 - x_Q$ must hold. Since $x_Q < 1$ and thus $1 - x_Q > 0$, we can chose $\epsilon$ sufficiently small such that $x_Q - \epsilon > 0$. 
Consequently, we can construct some query $Q'$ with $d_Q' = d_Q$, $x_Q' = x_Q - \epsilon$ and $b_Q' = b_Q$ which satisfies $Q' \prec_{\mathsf{KL}} Q$. 

Ad (3.): 
Let $d := |\mD|$ be arbitrary, but fixed, and let us reuse again the function $f(b,x) := -\frac{b}{d} \log_2(x) - \frac{d-b}{d} \log_2(1-x)$ defined in the proof of (1.) above and let w.l.o.g.\ $b$ and $x$, respectively, denote $|\dx{}(Q)|$ and $p(\dx{}(Q))$ for $Q \in \mQ$. Notice that -- due to $\dz{}(Q) = \emptyset$ for all $Q \in \mQ$ -- we have that (1)~$p(\dz{})=0$ and thus $p(\dx{})+p(\dnx{}) = 1$ which implies $\frac{p(\dx{})}{p(\dx{})+p(\dnx{})} = p(\dx{})$ and that (2)~$d$ (denoting the size of the committee) can be fixed for all queries w.r.t.\ $\mD$ and $\tuple{\mo,\mb,\Tp,\Tn}_\RQ$. 

First of all, we demonstrate that the function $f_b(x)$ resulting from $f(b,x)$ by fixing $b$ (where the value of $b$ is assumed arbitrary in $\setof{1,\dots,d-1}$) 
is strictly convex and thus attains exactly one (global) minimum for $x \in (0,1)$ (cf.\ \cite[p.~22]{nash1996}). Figuratively speaking, this means that $f_b(x)$ is strictly monotonically increasing along both directions starting from the argument of the minimum. In order to derive the strict convexity of $f_b(x)$, we build the first and second derivatives 
$f'_b(x) = -\frac{b}{d\cdot \ln 2} \frac{1}{x} + \frac{d-b}{d \cdot \ln 2} \frac{1}{1-x}$ and $f''_b(x) = \frac{b}{d \cdot \ln 2} \frac{1}{x^2} + \frac{d-b}{d \cdot \ln 2} \frac{1}{(1-x)^2}$. It is easy to see that $f''(x)>0$ since $d \geq 2$, $b \geq 1$, $d-b \geq 1$ and $x^2$ as well as $(1-x)^2$ are positive due to $x \in (0,1)$. The $x$-value $x_{\min,b}$ at which the minimum is attained can be computed by setting $f'_b(x) = 0$. From this we get $x_{\min,b} = \frac{b}{d}$.

Let us now assume some query $Q^* \in \mQ$ w.r.t.\ $\mD$ and $\tuple{\mo,\mb,\Tp,\Tn}_\RQ$ where $Q^*$ is considered best in $\mQ$ by $\mathsf{KL}$ and for all $k\in\setof{1,\dots,|\mD|-1}$ it holds that $\dx{}(Q^*) \notin \mathbf{MaxP}_k^+$ and for all $m\in\setof{1,\dots,|\mD|-1}$ it holds that $\dnx{}(Q^*) \notin \mathbf{MaxP}_m^-$. We show that there is some query $Q$ w.r.t.\ $\mD$ and $\tuple{\mo,\mb,\Tp,\Tn}_\RQ$ such that $Q \prec_{\mathsf{KL}} Q^*$. 

To this end, assume $|\dx{}(Q^*)| = b$. Then it must hold that $\mathbf{MaxP}_{b}^+ \neq \emptyset$ and $\mathbf{MaxP}_{d-b}^- \neq \emptyset$. This follows directly from the definition of $\mathbf{MaxP}_{k}^+$ and $\mathbf{MaxP}_{m}^-$.
%
Let $x_{Q^*} := p(\dx{}(Q^*))$. We now distinguish between two possible cases, i.e.\ (i)~$x_{Q^*} > \frac{b}{d}$ and (ii)~$x_{Q^*} \leq \frac{b}{d}$. In case~(i), $x_{Q^*}$ is already larger than the (unique) argument $x_{\min,b} = \frac{b}{d}$ of the minimum of the function $f_{b}(x)$. Since $\dx{}(Q^*) \notin \mathbf{MaxP}_{b}^+$, there must be some query $Q$ with $\dx{}(Q) \in \mathbf{MaxP}_b^+$  such that $p(\dx{}(Q)) > p(\dx{}(Q^*))$. This strict inequality must hold since $\dx{}(Q^*) \notin \mathbf{MaxP}_b^+$ and $\dx{}(Q) \in \mathbf{MaxP}_b^+$. Hence, 
$x_Q := p(\dx{}(Q)) > p(\dx{}(Q^*)) = x_{Q^*} > x_{\min,b}$ which is why $\mathsf{KL}(Q) = f_b(x_Q) > f_b(x_{Q^*}) = \mathsf{KL}(Q^*)$ by the fact that $f_b(x)$ is strictly monotonically increasing along both directions starting from $x_{\min,b}$. Consequently, $Q \prec_{\mathsf{KL}} Q^*$ holds.

In case~(ii), $x_{Q^*}$ is smaller than or equal to the (unique) argument $x_{\min,b} = \frac{b}{d}$ of the minimum of the function $f_b(x)$. In case of equality, i.e.\ $x_{Q^*} = \frac{b}{d}$, the exact same argumentation as in case (i) can be used. Assume now that $x_{Q^*} < \frac{b}{d}$. Since the cardinality of $\dx{}(Q^*)$ is $b$ and $\dz{}(Q^*) = \emptyset$, the cardinality of $\dnx{}(Q^*)$ must be $d-b$. As $\dnx{}(Q^*) \notin \mathbf{MaxP}_{d-b}^-$ there must be some query $Q$ with $\dnx{}(Q) \in \mathbf{MaxP}_{d-b}^-$ such that $p(\dnx{}(Q)) > p(\dnx{}(Q^*))$. This strict inequality must hold since $\dnx{}(Q^*) \notin \mathbf{MaxP}_{d-b}^-$ and $\dnx{}(Q) \in \mathbf{MaxP}_{d-b}^-$. From $p(\dnx{}(Q)) > p(\dnx{}(Q^*))$, however, we directly obtain $p(\dx{}(Q)) < p(\dx{}(Q^*))$ by the fact that $\dz{}(Q^*) = \dz{}(Q) = \emptyset$ (and hence $p(\dz{}(Q^*)) = p(\dz{}(Q)) = 0$). Therefore, we have that $x_Q = p(\dx{}(Q)) < p(\dx{}(Q^*)) = x_{Q^*} < x_{\min,b}$. So, $\mathsf{KL}(Q) = f_b(x_Q) > f_b(x_{Q^*}) = \mathsf{KL}(Q^*)$ by the fact that $f_b(x)$ is strictly monotonically increasing along both directions starting from $x_{\min,b}$. Consequently, $Q \prec_{\mathsf{KL}} Q^*$ holds.
\end{proof}
\begin{proposition}\label{prop:kl_does_not_satisfy_DPR}
$\mathsf{KL}$ is not consistent with the discrimination-preference relation $DPR$ (and consequently does not satisfy $DPR$ either).
\end{proposition}
\begin{proof}
It suffices to provide an example of two queries $Q,Q'$ where $Q$ is discrimination-preferred to $Q'$ and $Q' \prec_{\mathsf{KL}} Q$. To this end, let $Q$ and $Q'$ be characterized by the following q-partitions
\begin{align*}
\langle \dx{}(Q),\dnx{}(Q),\dz{}(Q) \rangle &= \langle \setof{\md_1,\md_2},\setof{\md_3,\md_4,\md_{5}},\emptyset \rangle \\
\langle \dx{}(Q'),\dnx{}(Q'),\dz{}(Q') \rangle &= \langle \setof{\md_1,\md_2},\setof{\md_3,\md_4},\setof{\md_5} \rangle
\end{align*}
and let the diagnosis probabilities $p_i := p(\md_i)$ be as follows: \[\tuple{p_1,\dots,p_5} = \tuple{0.35,0.05,0.15,0.25,0.2}\]
It is straightforward that $Q$ is discrimination-preferred to $Q'$ (cf.\ Proposition~\ref{prop:construction_of_Q'_from_Q_if_Q_discrimination-preferred_over_Q'}).
Plugging in the values \[\tuple{|\dx{}(Q)|,|\mD_Q|,\frac{p(\dx{}(Q))}{p(\mD_Q)}} = \tuple{2,5,0.4}\] into Eq.~\eqref{eq:Q_KL_derived}, we get $\mathsf{KL}(Q) \approx 0.97$ as a result, whereas plugging in the values \[\tuple{|\dx{}(Q')|,|\mD_{Q'}|,\frac{p(\dx{}(Q'))}{p(\mD_{Q'})}} = \tuple{2,4,0.5}\] yields $\mathsf{KL}(Q') = 1$.
Thence, $Q' \prec_{\mathsf{KL}} Q$. Note that $\frac{p(\dx{}(Q'))}{p(\mD_{Q'})}$ is computed as $\frac{0.35 + 0.05}{0.35+0.05+0.15+0.25} = \frac{0.4}{0.8}$ where the denominator represents the probability of the voting committee $\mD_{Q'}$ which is given by $\setof{\md_1,\dots,\md_4}$ in case of $Q'$, and not by $\setof{\md_1,\dots,\md_5}$ as in case of $Q$.  
\end{proof}
\begin{corollary}\label{}
All measures satisfying the discrimination-preference relation are superior to $\mathsf{KL}$.
\end{corollary}
\begin{proof}
Immediate from Propositions~\ref{prop:if_m1_satisfies_DPR_and_m2_does_not_then_m1_superior_to_m2} and \ref{prop:kl_does_not_satisfy_DPR}.
\end{proof}

\begin{remark}\label{rem:kullback_leibler_theoretical_opt}
Given a diagnosis probability measure $p$, we can compute the $\mathsf{KL}$ measure for all partitions (not necessarily q-partitions) with different maximal probable $\dx{}$-sets (regarding the cardinality $|\dx{}|$) and store the maximal among all $\mathsf{KL}$ measures obtained in this manner as $opt_{\mathsf{KL},p,\mD}$. By Proposition~\ref{prop:kl_opt},(\ref{prop:kl_opt:properties_of_best_query}.), the parameter $opt_{\mathsf{KL},p,\mD}$ is then exactly the value of the sought best q-partition (with empty $\dz{}$-set) w.r.t.\ $\mathsf{KL}$ for $\mD$ and $p$ in case the partition from which it was computed is indeed a q-partition. Otherwise, again by Proposition~\ref{prop:kl_opt},(\ref{prop:kl_opt:properties_of_best_query}.), it is an upper bound of the $\mathsf{KL}$ measure of the best q-partition (with empty $\dz{}$-set). The maximal probable $\dx{}$-sets can be easily computed by starting from an empty set and adding step-by-step the diagnosis with the highest probability among those diagnoses not yet added. At each of the $|\mD|-1$ steps, we have present one maximal probable $\dx{}$-set. Notice that for maximal probable $\dnx{}$-sets the same result would be achieved (due to the ``symmetry'' of Eq.~\eqref{eq:Q_KL_derived} w.r.t.\ $\dx{}$ and $\dnx{}$) which is why this process must only be performed once.
%
\qed
\end{remark}

\paragraph{Expected Model Change (EMC).} The principle of the EMC framework is to query the instance that would impart the greatest change to the current model if its label was known. Translated to the debugging scenario, the ``model'' can be identified with the set of leading diagnoses according to which ``maximum expected model change'' can be interpreted in a way that 
\begin{enumerate}[label=(\alph*)]
\item the expected \emph{probability mass} of invalidated leading diagnoses is maximized or
\item the expected \emph{number} of invalidated leading diagnoses is maximized.
\end{enumerate}

\vspace{1em}

\noindent\emph{Expected Model Change - Variant (a) ($\mathsf{EMCa}$)}: Formally, variant~(a) selects the query 
\begin{align*}
Q_{\mathsf{EMCa}} := \argmax_{Q\in\mQ_\mD} \left(\mathsf{EMCa}(Q)\right) \quad \text{ where } \quad \mathsf{EMCa}(Q) := p(Q=t) p(\dnx{}(Q)) + p(Q=f) p(\dx{}(Q))   
\end{align*}
since the set of leading diagnoses invalidated for positive and negative answer to $Q$ is $\dnx{}(Q)$ and $\dx{}(Q)$, respectively (cf.\ Proposition~\ref{prop:properties_of_q-partitions},(\ref{prop:properties_of_q-partitions:enum:dx_dnx_dz_contain_exactly_those_diags_that_are...}.)).
\begin{proposition}\label{prop:emca}
For $\mathsf{EMCa}$, the following holds:
\begin{enumerate}
	\item \label{prop:emca:EMCa_can_be_equiv_represented_as} $\mathsf{EMCa}(Q)$ can be equivalently represented as 
	\begin{align}
2\,\left[p(Q=t) - [p(Q=t)]^2 \right] - \frac{p(\dz{}(Q))}{2}  \label{eq:EMCa_derived}
\end{align}
	\item \label{prop:emca:theoretically_optimal_query_wrt_EMCa} The theoretically optimal query w.r.t.\ $\mathsf{EMCa}$ satisfies $p(Q=t) = p(Q=f)$ and $\dz{}(Q) = \emptyset$.
	\item \label{prop:emca:for_mQ_incl_only_queries_with_empty_dz_EMCa_equiv_ENT} Let $\mQ$ be a set of queries w.r.t.\ a set of minimal diagnoses w.r.t.\ a DPI $\tuple{\mo,\mb,\Tp,\Tn}_\RQ$ where each query $Q \in \mQ$ satisfies $\dz{}(Q) = \emptyset$. Then $\mathsf{EMCa} \equiv_{\mQ} \mathsf{ENT}$.
\end{enumerate}
\end{proposition}
\begin{proof}
Let for brevity $x:=p(Q=t)$, i.e.\ $1-x = p(Q=f)$, and $x_0 := p(\dz{}(Q))$. Then $\mathsf{EMCa}(Q) = x \cdot p(\dnx{}(Q)) + (1-x)\cdot p(\dx{}(Q))$. According to Eq.~\eqref{eq:p(Q=t)} and Eq.~\eqref{eq:p(Q=f)}, we can write $p(\dnx{}(Q))$ as $p(Q=f) - \frac{p(\dz{}(Q))}{2} = (1-x) - \frac{x_0}{2}$ and $p(\dx{}(Q))$ as $p(Q=t) - \frac{p(\dz{}(Q))}{2} = x - \frac{x_0}{2}$. Consequently, $\mathsf{EMCa}(Q) = x \cdot [(1-x) - \frac{x_0}{2}] + (1-x)\cdot [x - \frac{x_0}{2}]$. After some simple algebra, $\mathsf{EMCa}(Q)$ looks as follows: $[2(x-x^2)]-\frac{x_0}{2}$. This completes the proof of (1.). 

The best query $Q$ w.r.t.\ $\mathsf{EMCa}$ is the one which maximizes $\mathsf{EMCa}(Q)$. Hence, the best query w.r.t.\ $\mathsf{EMCa}$ is the one which minimizes $-\mathsf{EMCa}(Q) = [2(x^2-x)]+\frac{1}{2} x_0$. We now observe that $x_0 = p(\dz{}(Q))$ is independent of the sum in squared brackets. Thus, both terms can be minimized separately. Therefore, we immediately see that $x_0 = p(\dz{}(Q)) = 0$ must be true for the theoretically optimal query w.r.t.\ $\mathsf{EMCa}$. This implies $\dz{}(Q) = \emptyset$ by $p(\md)>0$ for all $\md \in \mD$ (cf.\ page~\pageref{etc:prob_of_each_diag_must_be_greater_zero}). Further on, if $x_0 = p(\dz{}(Q)) = \emptyset$, $-\mathsf{EMCa}(Q) = [2(x^2-x)]$.

Next, we analyze the term $f(x) := 2(x^2-x)$ in squared brackets for extreme points. To this end, we build the first and second derivatives $f'(x) = 2(2x-1)$ and $f''(x) = 4$. Clearly, $f(x)$ is a strictly convex function as $f''(x) > 0$ for all $x \in (0,1)$. Hence, there is exactly a unique (global) minimum of it. By setting $f'(x) = 0$, we obtain $x = \frac{1}{2}$ as the argument of this minimum. As $x = p(\dx{}(Q))$ and $p(\dx{}(Q)) = \frac{1}{2}$ implies that $p(\dnx{}(Q)) = 1-p(\dx{}(Q)) = \frac{1}{2}$, (2.) is thereby proven.

Now, recall that $\mathsf{ENT}(Q)$ can be represented as in Eq.~\eqref{eq:scoring_funtion_dekleer} and bring back to mind the proof of Proposition~\ref{prop:ent} where we showed that the term $x \log_2 x + (1-x) \log_2 (1-x)$ in squared brackets in Eq.~\eqref{eq:scoring_funtion_dekleer} is strictly convex as well, attaining its minimum exactly at $x = \frac{1}{2}$.
Moreover, given $x_0 = p(\dz{}(Q)) = \emptyset$, we have that $\mathsf{ENT}(Q) = x \log_2 x + (1-x) \log_2 (1-x)$. So, for queries with empty $\dz{}$, both $-\mathsf{EMCa}(Q)$ and $\mathsf{ENT}(Q)$ are strictly convex and feature the same point where the global minimum is attained. That is, both $-\mathsf{EMCa}(Q)$ and $\mathsf{ENT}(Q)$ are strictly monotonically increasing along both directions outgoing from the argument of the minimum $x =\frac{1}{2}$. As a consequence, for queries $Q_i, Q_j \in \mQ$ we have that $Q_i \prec_{\mathsf{ENT}} Q_j$ iff $|p(Q_i = t)-\frac{1}{2}| < |p(Q_j = t)-\frac{1}{2}|$ iff $Q_i \prec_{\mathsf{EMCa}} Q_j$. This completes the proof of (3.).
\end{proof}
Similarly as in the case of $\mathsf{ENT}$, we define a measure $\mathsf{EMCa}_z$ which constitutes a generalization of the $\mathsf{EMCa}$ measure. Namely, $\mathsf{EMCa}_z$ selects the query  
\begin{align*}
Q_{\mathsf{EMCa}_z} := \argmax_{Q \in \mQ} \left(\mathsf{EMCa}_z(Q)\right)
\end{align*}
where
\begin{align}\label{eq:EMCa_z}
\mathsf{EMCa}_z(Q) := 2\,\left[p(Q=t) - [p(Q=t)]^2 \right] - z\, \frac{p(\dz{}(Q))}{2} 
\end{align}
%
Note that $\mathsf{EMCa}_z(Q) = \mathsf{EMCa}(Q) + \frac{z-1}{2}p(\dz{}(Q))$ and thus that $\mathsf{EMCa}_1$ is equal to $\mathsf{EMCa}$ (cf.\ Proposition~\ref{prop:EMCa_preserves_discrimination-pref-order} and Eq.~\eqref{eq:EMCa_derived}, respectively). Moreover, $\mathsf{EMCa}_0$ is equal to the \emph{Gini Index} \cite{rokach2005}, a frequently adopted (information) gain measure in decision tree learning. Using our terminology, the Gini Index is defined as $1 - [p(Q=t)]^2 - [p(Q=f)]^2$ which can be easily brought into the equivalent form $2\,\left[p(Q=t) - [p(Q=t)]^2 \right]$ by using the fact that $p(Q=t) = 1 - p(Q=f)$. 
\begin{corollary}\label{cor:EMCa_z_for_all_z_equivalent_mQ_to_ENT_for_mQ_includes_only_dz=0_queries}
Let $\mQ$ be a set of queries w.r.t.\ a set of minimal diagnoses w.r.t.\ a DPI $\tuple{\mo,\mb,\Tp,\Tn}_\RQ$ where each query $Q \in \mQ$ satisfies $\dz{}(Q) = \emptyset$. Then $\mathsf{EMCa}_z \equiv_{\mQ} \mathsf{ENT}$ for all $z \in \mathbb{R}$.
\end{corollary}
\begin{proof}
Immediate from Proposition~\ref{prop:emca},(\ref{prop:emca:for_mQ_incl_only_queries_with_empty_dz_EMCa_equiv_ENT}.) and the fact that $\mathsf{EMCa}_z$ is obviously equal to $\mathsf{EMCa}$ for all $Q\in\mQ$ (due to $\dz{}(Q) = \emptyset$) for all $z \in \mathbb{R}$.
\end{proof}
\begin{proposition}\label{prop:EMCa_preserves_discrimination-pref-order}
Let $Q$ be a query w.r.t.\ a set of minimal diagnoses $\mD$ w.r.t.\ a DPI $\tuple{\mo,\mb,\Tp,\Tn}_\RQ$. Further, let $p:=\min_{a\in\setof{t,f}} (p(Q=a))$ and $Q'$ be a query such that $Q$ is discrimination-preferred to $Q'$. In addition, let $x:= p(\dz{}(Q')) - p(\dz{}(Q))$. 
Then:
\begin{enumerate}
	\item $x > 0$.
	\item If $x \geq 1-2p$, then $Q \prec_{\mathsf{EMCa}} Q'$.
	\item If $x < 1-2p$, then: 
	\begin{enumerate}
		\item In general it does not hold that $Q \prec_{\mathsf{EMCa}} Q'$.
		\item If not $Q \prec_{\mathsf{EMCa}} Q'$, then $p \in (0,t]$ where $t := \frac{1}{4}$.
		\item If not $Q \prec_{\mathsf{EMCa}_z} Q'$, then $p \in (0,t(z)]$ where $t(z) := \frac{-z+2}{4}$, i.e.\ $\mathsf{EMCa}_z$ for all real numbers $z \geq 2$ satisfies the discrimination-preference relation $DPR$.
	\end{enumerate}
\end{enumerate}
\end{proposition}
\begin{proof}
The proof follows exactly the same line of argumentation as used in the proof of Proposition~\ref{prop:ENT_preserves_discrimination-pref-order}, just that Equations~\eqref{eq:EMCa_derived} and \eqref{eq:EMCa_z} are analyzed instead of Equations~\eqref{eq:scoring_funtion_dekleer} and \eqref{eq:ENTz}, respectively.
\end{proof}
\begin{corollary}\label{cor:EMCa_satisfies_discrimination-pref_order_for_queries_with_p>=1/4}
Let $Q$ be a query and $p_Q:=\min_{a\in\setof{t,f}} (p(Q=a))$. Let further $\mQ$ be a set of queries where each query $Q$ in this set satisfies $p_Q > \frac{1}{4}$. Then, $\mathsf{EMCa}$ 
satisfies the discrimination-preference relation (over $\mQ$) (Definition~\ref{def:measures_equivalent_theoretically-optimal_superior}).
\end{corollary}
\begin{proof}
This statement is a consequence of applying the law of contraposition to the implication stated by Proposition~\ref{prop:EMCa_preserves_discrimination-pref-order},(3.b).
\end{proof}
The next corollary is quite interesting, as it testifies that $\mathsf{EMCa}_r$ for any selection of a \emph{finite} $r \geq 2$ satisfies the discrimination-preference relation $DPR$. Recall from Proposition~\ref{prop:ENT_preserves_discrimination-pref-order} that $\mathsf{ENT}_z$ for no finite selection of a non-negative $z$ value (and thus neither $\mathsf{ENT}$) satisfies the $DPR$ (at least theoretically). For that reason, $\mathsf{EMCa}_2$, for instance, is already superior to $\mathsf{ENT}_z$ \emph{for all} finite non-negative values of $z$.
\begin{corollary}\label{cor:EMCa_z_for_z>=2_is_better_than_ENTr_for_all_r}
For all $r \geq 2$ and $z \geq 0$, $\mathsf{EMCa}_{r}$ satisfies the discrimination-preference relation $DPR$ and is superior to $\mathsf{ENT}_z$. 
\end{corollary}
\begin{proof}
By (the argumentation given in the proof of) Proposition~\ref{prop:ENT_preserves_discrimination-pref-order}, one can construct two queries $Q,Q'$ where $Q$ is discrimination-preferred to $Q'$ and not $Q \prec_{\mathsf{ENT}_z} Q'$ for all $z \geq 0$. This holds because $t(z) > 0$ and thus the interval $(0,t(z)]$ is non-empty for all $z \geq 0$ (cf.\ Proposition~\ref{prop:ENT_preserves_discrimination-pref-order},(c)). Since, by Proposition~\ref{prop:EMCa_preserves_discrimination-pref-order},(c), $\mathsf{EMCa}_r$ preserves the discrimination-preference order for all $r \geq 2$, we can conclude that $Q \prec_{\mathsf{EMCa}_r} Q'$ must be given and that there cannot be any pair of queries $Q_i,Q_j$ where $(Q_i,Q_j)$ is in the debug preference relation $DPR$ and not $Q_i \prec_{\mathsf{EMCa}_r} Q_j$. Hence all three bullets of the definition of superiority (Definition~\ref{def:measures_equivalent_theoretically-optimal_superior}) are met.
\end{proof}
\begin{corollary}\label{cor:if_0<=r<s<=2_then_EMCa_r_inferior_to_EMCa_s}
Let $0 \leq r < s \leq 2$. Then $\mathsf{EMCa}_r$ is inferior to $\mathsf{EMCa}_s$.
\end{corollary}
\begin{proof}
The proof follows exactly the same line of argumentation as was given in the proofs of Corollaries~\ref{cor:if_r<s_then_it_holds_that_if_not_Q_precENTs_Q'_then_not_Q_precENTr_Q'} and \ref{cor:if_0<=r<s_then_ENTr_inferior_to_ENTs}. The upper bound $2$ for $s$ must hold since $\mathsf{EMCa}_s$ for $s \geq 2$ satisfies the discrimination-preference relation $DPR$ (as per Proposition~\ref{prop:EMCa_preserves_discrimination-pref-order},(3c.)) and thus no other measure can be superior to $\mathsf{EMCa}_s$ in this case since the first and second bullets of the definition of superiority (cf.\ Definition~\ref{def:measures_equivalent_theoretically-optimal_superior}) can never be satisfied. 
\end{proof}

\vspace{1em}

\noindent\emph{Expected Model Change - Variant (b) ($\mathsf{EMCb}$)}: Formally, variant~(b) selects the query 
\begin{align*}
Q_{\mathsf{EMCb}} := \argmax_{Q\in\mQ_\mD}      
\end{align*}
where
\begin{align}\label{eq:EMCb}
\mathsf{EMCb}(Q) := p(Q=t) |\dnx{}(Q)| + p(Q=f) |\dx{}(Q)|
\end{align}
since the set of leading diagnoses invalidated for positive and negative answer to $Q$ is $\dnx{}(Q)$ and $\dx{}(Q)$, respectively (cf.\ Proposition~\ref{prop:properties_of_q-partitions},(\ref{prop:properties_of_q-partitions:enum:dx_dnx_dz_contain_exactly_those_diags_that_are...}.)). For $\mathsf{EMCb}$ there is no theoretically optimal query since there is no global optimum for the probabilities for positive and negative answers to queries assuming values in $(0,1)$ and the cardinalities of the $\dx{}$- and $\dnx{}$-sets being in the range $\setof{1,\dots,|\mD|-1}$. The reason for this is the \emph{open} interval $(0,1)$ for the probabilities which results from Equations~\eqref{eq:p(Q=t)} and \eqref{eq:p(Q=f)} as well as the facts that all diagnoses $\md\in\mD$ have a positive probability (cf.\ page~\pageref{etc:prob_of_each_diag_must_be_greater_zero}) and that neither $\dx{}(Q)$ nor $\dnx{}(Q)$ must be the empty set for any query. This open interval enables to find for each query, no matter how good it is w.r.t.\ $\mathsf{EMCb}$, another query that is even better w.r.t.\ $\mathsf{EMCb}$. In all cases but the one where $|\dx{}(Q)| = |\dnx{}(Q)|$ it suffices to simply increase the probability of one of the answers accordingly which is always possible due the open interval. Moreover, the set of candidate queries that must include the one that is regarded best by the $\mathsf{EMCb}$ measure among a given set of queries with empty $\dz{}$ and given a fixed probability measure $p$ can be characterized equally as in the case of the $\mathsf{KL}$ measure. Hence, interestingly, although belonging to different active learning frameworks, $\mathsf{EMCb}$ and $\mathsf{KL}$ prove to bear strong resemblance to one another as far as the discussed debugging scenario is concerned. The next proposition summarizes this similarity (cf.\ Proposition~\ref{prop:kl_opt}).
\begin{proposition}\label{prop:EMCb_opt}
Let $\mD$ be a set of minimal diagnoses w.r.t.\ a DPI $\tuple{\mo,\mb,\Tp,\Tn}_\RQ$. Then, for $\mathsf{EMCb}$ the following holds:
\begin{enumerate}
	\item \label{prop:EMCb_opt:no_theoretically_opt_query} There is no theoretically optimal query w.r.t.\ $\mathsf{EMCb}$ and $\mD$.
	\item \label{prop:EMCb_opt:for_any_given_query_better_query_can_be_found} If the (diagnosis) probability measure $p$ is not assumed fixed and $|\mD| \geq 3$, then for any given query $Q$ w.r.t.\ $\mD$ and $\tuple{\mo,\mb,\Tp,\Tn}_\RQ$ a query $Q'$ w.r.t.\ $\mD$ and $\tuple{\mo,\mb,\Tp,\Tn}_\RQ$ satisfying $Q' \prec_{\mathsf{EMCb}} Q$ can be found.
	\item \label{prop:EMCb_opt:properties_of_best_query} Let $\mQ$ be any set of queries w.r.t.\ a DPI $\tuple{\mo,\mb,\Tp,\Tn}_\RQ$ and a set of leading diagnoses $\mD$ w.r.t.\ this DPI such that each query $Q\in\mQ$ satisfies $\dz{}(Q) = \emptyset$. If the (diagnosis) probability measure $p$ is assumed fixed, then the query $Q \in \mQ$ that is considered best in $\mQ$ by $\mathsf{EMCb}$ (i.e.\ there is no query $Q' \in \mQ$ w.r.t.\ $\mD$ and $\tuple{\mo,\mb,\Tp,\Tn}_\RQ$ such that 
$Q' \prec_{\mathsf{EMCb}} Q$) satisfies $|\dx{}(Q)| = |\dnx{}(Q)| = \frac{|\mD|}{2}$ or $\dx{}(Q) \in \mathbf{MaxP}_k^+$ where \[\mathbf{MaxP}_k^+ := \setof{\mathbf{S}\,|\, \mathbf{S}\in \mathcal{S}^+_k \land \forall \mathbf{S}'\in \mathcal{S}^+_k: p(\mathbf{S}) \geq p(\mathbf{S}')}\] for some $k \in \setof{1,\dots,|\mD|-1}$ where 
\[\mathcal{S}^+_k = \setof{\mathbf{X}\,|\,\emptyset\subset\mathbf{X}\subset\mD,|\mathbf{X}|=k,\tuple{\mathbf{X},\mathbf{Y},\mathbf{Z}}\text{ is a q-partition w.r.t.\ } \mD, \tuple{\mo,\mb,\Tp,\Tn}_\RQ}\] 
or $\dnx{}(Q) \in \mathbf{MaxP}_m^-$ 
where \[\mathbf{MaxP}_m^- := \setof{\mathbf{S}\,|\, \mathbf{S}\in \mathcal{S}^-_m \land \forall \mathbf{S}'\in \mathcal{S}^-_m: p(\mathbf{S}) \geq p(\mathbf{S}')}\] for some $m \in \setof{1,\dots,|\mD|-1}$ where 
\[\mathcal{S}^-_m = \setof{\mathbf{Y}\,|\,\emptyset\subset\mathbf{Y}\subset\mD,|\mathbf{Y}|=m,\tuple{\mathbf{X},\mathbf{Y},\mathbf{Z}}\text{ is a q-partition w.r.t.\ } \mD, \tuple{\mo,\mb,\Tp,\Tn}_\RQ}\]
\end{enumerate}
\end{proposition}
\begin{proof}
The proof follows exactly the same line of argumentation that was used in the proof of Proposition~\ref{prop:kl_opt}. To nevertheless provide a sketch of the proof, observe regarding (1.) that, for each $d$, there is a single stationary point $(\frac{d}{2},\frac{1}{2})$ of the function $g(b,x) = x (d-b) + (1-x) b$ corresponding to $\mathsf{EMCb}$ using the same denotations as in the proof of Proposition~\ref{prop:kl_opt}. This stationary point turns out to be a saddle point. Hence there is no (local or global) maximum of $g(b,x)$ for $b \in \setof{1,\dots,d-1}, x\in(0,1)$.

Concerning (2.), we bring $g(b,x)$ into the form $x(d-2b)+b$ which exposes that, for fixed and arbitrary $b$, $g_b(x)$ corresponds to a straight line with a slope of $d-2b$. Therefore, if $b < \frac{d}{2}$, the slope is positive and we obtain a better query w.r.t.\ $\mathsf{EMCb}$ by simply adding some (arbitrarily small) $\epsilon > 0$ to $x$. Such $\epsilon$ always exists since $x$ has some value in the \emph{open} interval $(0,1)$. Otherwise, in case $b > \frac{d}{2}$, the slope is negative which is why we can subtract some small $\epsilon' >0$ from $x$ to construct a better query w.r.t.\ $\mathsf{EMCb}$. Finally, if the slope is zero, i.e.\ $b = \frac{d}{2}$ (which can by the way only occur if $d$ is even), we must show that there is another query with a better $\mathsf{EMCb}$ measure. If $x > 1-x$, such a query is obviously given by setting $b \gets b - 1$; if $x < 1-x$, then by setting $b \gets b + 1$; if $x = 1-x$, then by setting e.g.\ $x \gets x + \epsilon$ and $b \gets b - 1$ for some $\epsilon > 0$. Note that the different settings for $b$ are possible (without implying that any of the $\dx{}$- or $\dnx{}$-sets becomes empty) since $|\mD| = d \geq 3$ which entails that $b = \frac{d}{2}$ only if $d \geq 4$.

With regard to (3.), we again use $g_b(x)$, as specified above, and consider a specific query $Q$ that does not satisfy the conditions given in (3.). If the slope $d-2b$ for this query is positive, we know that there must be a query $Q'$ with $\dx{}(Q') \in \mathbf{MaxP}_b^+$ that is better w.r.t.\ $\mathsf{EMCb}$ than $Q$ due to the fact that $\dx{}(Q) \notin \mathbf{MaxP}_b^+$. Given a negative slope, we likewise exploit the fact that there must be a query $Q'$ with $\dnx{}(Q') \in \mathbf{MaxP}_{d-b}^-$ that is better w.r.t.\ $\mathsf{EMCb}$ than $Q$ due to the fact that $\dx{}(Q) \notin \mathbf{MaxP}_{d-b}^-$. Finally, we point out that $Q$, as it does not satisfy the condition of (3.), cannot have the property of a slope of zero because this would imply $|\dx{}(Q)| = b = \frac{|\mD|}{2}$ and due to $\dz{}(Q) = \emptyset$ also $|\dnx{}(Q)| = \frac{|\mD|}{2}$, i.e.\ that $Q$ satisfies the condition of (3.). 
\end{proof}
Please note that, probably a bit surprisingly, in spite of the similarity between $\mathsf{KL}$ and $\mathsf{EMCb}$ witnessed by Propositions~\ref{prop:kl_opt} and \ref{prop:EMCb_opt}, these two measures are not equivalent, i.e.\ $\mathsf{KL} \not\equiv \mathsf{EMCb}$, as we elaborate by the following example:
\begin{example}\label{ex:KL_not_equiv_EMCb}
The see the non-equivalence, we state two queries $Q$ and $Q'$ where $Q \prec_{\mathsf{KL}} Q'$ and $Q' \prec_{\mathsf{EMCb}} Q$. Let the set of leading diagnoses be $\mD$ where $|\mD|=10$. The queries $Q,Q'$ are characterized by the following properties:
\begin{align*}
\langle |\dx{}(Q)|,p(\dx{}(Q)) \rangle &= \langle 3, 0.05 \rangle \\
\langle |\dx{}(Q')|,p(\dx{}(Q')) \rangle &= \langle 5, 0.25 \rangle 
\end{align*}
Calculating $\mathsf{KL}(Q) \approx 1.35 $, $\mathsf{KL}(Q') \approx 1.21 $, $\mathsf{EMCb}(Q) = 3.2$ and $\mathsf{EMCb}(Q') = 5$ from these values (see Equations~\eqref{eq:Q_KL_derived} and \eqref{eq:EMCb}), we clearly see that $\mathsf{KL}(Q) > \mathsf{KL}(Q')$ whereas $\mathsf{EMCb}(Q) < \mathsf{EMCb}(Q')$. Hence, $Q \prec_{\mathsf{KL}} Q'$ and $Q' \prec_{\mathsf{EMCb}} Q$.\qed 
\end{example}

\begin{remark}\label{rem:EMCb_theoretical_opt}
In a way completely analogous to the process described in Remark~\ref{rem:kullback_leibler_theoretical_opt}, we can calculate $opt_{\mathsf{EMCb},p,\mD}$. By Proposition~\ref{prop:EMCb_opt},(\ref{prop:EMCb_opt:properties_of_best_query}.) and because all queries satisfying $|\dx{}(Q)| = |\dnx{}(Q)| = \frac{|\mD|}{2}$ have an equal $\mathsf{EMCb}$ measure due to the fact that Eq.~\eqref{eq:EMCb} in this case reduces to $\frac{|\mD|}{2}$, the parameter $opt_{\mathsf{EMCb},p,\mD}$ is then exactly the value of the sought best q-partition (with empty $\dz{}$-set) w.r.t.\ $\mathsf{EMCb}$ for $\mD$ and $p$ in case the partition from which it was computed is indeed a q-partition. Otherwise, again by Proposition~\ref{prop:EMCb_opt},(\ref{prop:EMCb_opt:properties_of_best_query}.), it is an upper bound of the $\mathsf{EMCb}$ measure of the best q-partition (with empty $\dz{}$-set).\qed
\end{remark}

\begin{proposition}\label{prop:EMCb_does_not_satisfy_DPR}
$\mathsf{EMCb}$ is not consistent with the discrimination-preference relation $DPR$ (and consequently does not satisfy $DPR$ either).
\end{proposition}
\begin{proof}
It suffices to provide an example of two queries $Q,Q'$ where $Q$ is discrimination-preferred to $Q'$ and $Q' \prec_{\mathsf{KL}} Q$. To this end, let $Q$ and $Q'$ be characterized by the following q-partitions
\begin{align*}
\langle \dx{}(Q),\dnx{}(Q),\dz{}(Q) \rangle &= \langle \setof{\md_1,\md_2,\md_3},\setof{\md_4},\emptyset \rangle \\
\langle \dx{}(Q'),\dnx{}(Q'),\dz{}(Q') \rangle &= \langle \setof{\md_1,\md_2},\setof{\md_4},\setof{\md_3} \rangle
\end{align*}
and let the diagnosis probabilities $p_i := p(\md_i)$ be as follows: \[\tuple{p_1,\dots,p_4} = \tuple{0.03,0.07,0.8,0.1}\]
It is straightforward that $Q$ is discrimination-preferred to $Q'$ (cf.\ Proposition~\ref{prop:construction_of_Q'_from_Q_if_Q_discrimination-preferred_over_Q'}).
We observe that $p(Q=t) = p(\dx{}(Q))+\frac{1}{2}p(\dz{}(Q)) = 0.03+0.07+0.8+\frac{1}{2}\cdot 0 = 0.9$, i.e.\ $p(Q=f) = 0.1$. Likewise, $p(Q'=t) = p(\dx{}(Q'))+\frac{1}{2}p(\dz{}(Q')) = 0.03+0.07+\frac{1}{2}\cdot 0.8 = 0.5$, i.e.\ $p(Q'=f) = 0.5$.

Now, plugging in the values 
\[\tuple{p(Q=t),|\dx{}(Q)|,p(Q=f),|\dnx{}(Q)|} = \tuple{0.9,3,0.1,1}\] 
into Eq.~\eqref{eq:EMCb}, we get $\mathsf{EMCb}(Q) = 0.9 \cdot 1 + 0.1 \cdot 3 = 1.2$ as a result, whereas plugging in the values 
\[\tuple{p(Q'=t),|\dx{}(Q')|,p(Q'=f),|\dnx{}(Q')|} = \tuple{0.5,2,0.5,1}\] 
yields $\mathsf{EMCb}(Q') = 0.5 \cdot 1 + 0.5 \cdot 2 = 1.5$.
Since a higher $\mathsf{EMCb}$ value means that a query is preferred by $\mathsf{EMCb}$, it holds that $Q' \prec_{\mathsf{EMCb}} Q$.   
\end{proof}
\begin{corollary}\label{cor:all_measures_satisfying_DPR_superior_to_EMCb}
All measures satisfying the discrimination-preference relation are superior to $\mathsf{EMCb}$.
\end{corollary}
\begin{proof}
Immediate from Propositions~\ref{prop:if_m1_satisfies_DPR_and_m2_does_not_then_m1_superior_to_m2} and \ref{prop:EMCb_does_not_satisfy_DPR}.
\end{proof}

Next we introduce two other measures which seem reasonable and fit into the EMC framework, \emph{most probable singleton} and \emph{biased maximal elimination}:

\vspace{1em}

\noindent\emph{Most Probable Singleton ($\mathsf{MPS}$)}: The intention of $\mathsf{MPS}$ is the elimination of all but one diagnoses in $\mD$, i.e.\ leaving valid only a singleton diagnoses set. In order to achieve this with highest likeliness, it selects among all queries that either eliminate all but one or one diagnosis in $\mD$ the one query $Q$ whose singleton set in $\setof{\dx{}(Q),\dnx{}(Q)}$ has highest probability. 
Formally, $\mathsf{MPS}$ chooses the query
\begin{align}
Q_{\mathsf{MPS}} := \argmax_{Q\in\mQ_\mD} \left(\mathsf{MPS}(Q)\right)  
\label{eq:Q_MPS}
\end{align}
where $\mathsf{MPS}(Q) := 0$ for all queries $Q$ satisfying 
\begin{align}
\tuple{|\dx{}(Q)|,|\dnx{}(Q)|,|\dz{}(Q)|} \notin \setof{\tuple{|\mD|-1,1,0},\tuple{1,|\mD|-1,0}} \label{eq:condition_MPS}
\end{align}
and $\mathsf{MPS}(Q) := p(\mD_{Q,\min})$ where $\mD_{Q,\min} := \argmin_{\mathbf{X}\in\setof{\dx{}(Q),\dnx{}(Q)}}(|\mathbf{X}|)$ otherwise.

Note that the function $\mathsf{MPS}(Q)$ is not injective. Hence, there might be more than one maximum of the function, e.g.\ in case all diagnoses in $\mD$ have the same probability. In case of non-uniqueness of the maximum, the $\argmax$ statement in Eq.~\eqref{eq:Q_MPS} is meant to select just any of the arguments where the maximum is attained. 

The $\mathsf{MPS}$ measure fits into the EMC framework because it maximizes the probability of the maximal theoretically possible change on the current model (leading diagnoses).
\begin{proposition}\label{prop:MPS_consistent_with_DPR_but_not_satisfies_DPR}
$\mathsf{MPS}$ is consistent with the discrimination-preference relation $DPR$, but does not satisfy $DPR$.
\end{proposition}
\begin{proof}
Assume two queries $Q,Q'$ such that $(Q,Q') \in DPR$. 
Then, by Proposition~\ref{prop:if_Q_discrimination-preferred_over_Q'_then_dz(Q')_supset_dz(Q)}, $Q'$ must feature a non-empty set $\dz{}(Q')$.
In case $\dz{}(Q)=\emptyset$ and one of $|\dx{}(Q)| = 1$ or $|\dnx{}(Q)| = 1$ holds, we obtain $p(\mD_{Q,\min}) = \mathsf{MPS}(Q) > \mathsf{MPS}(Q') = 0$ (as $|\mD_{Q,\min}| = 1$ and each diagnosis in $\mD$ has non-zero probability, cf.\ page~\pageref{etc:prob_of_each_diag_must_be_greater_zero}). Otherwise, we have that $\mathsf{MPS}(Q) = \mathsf{MPS}(Q') = 0$. In both cases we can conclude that $Q' \prec_{\mathsf{MPS}} Q$ does not hold. 

To see why $\mathsf{MPS}$ does not satisfy $DPR$, consider the two queries $Q$ and $Q'$ characterized by the following q-partitions
\begin{align*}
\langle \dx{}(Q),\dnx{}(Q),\dz{}(Q) \rangle &= \langle \setof{\md_1,\md_2},\setof{\md_3},\setof{\md_4} \rangle \\
\langle \dx{}(Q'),\dnx{}(Q'),\dz{}(Q') \rangle &= \langle \setof{\md_1},\setof{\md_3},\setof{\md_2,\md_4} \rangle
\end{align*}
Obviously, $Q$ is discrimination-preferred to $Q'$, i.e.\ $(Q,Q')\in DPR$ (cf.\ Proposition~\ref{prop:construction_of_Q'_from_Q_if_Q_discrimination-preferred_over_Q'}). But, as $\dz{}(Q) \neq \emptyset$ and $\dz{}(Q') \neq \emptyset$, we point out that both queries are assigned a zero $\mathsf{MPS}$ value which is why $Q \prec_{\mathsf{MPS}} Q'$ does not hold.
\end{proof}
In fact, we can slightly modify the $\mathsf{MPS}$ measure such that the resulting measure $\mathsf{MPS}'$ satisfies the discrimination-preference relation and the set of query tuples satisfying the relation $\prec_{\mathsf{MPS}'}$ is a superset of the set of query tuples satisfying the relation $\prec_{\mathsf{MPS}}$. The latter means in other words that given $\mathsf{MPS}$ prefers a query to another one, $\mathsf{MPS}'$ will do so as well. So, the order imposed by $\mathsf{MPS}$ on any set of queries is maintained by $\mathsf{MPS}'$, just that $\mathsf{MPS}'$ includes some more preference tuples in order to comply with the discrimination-preference order.
The $\mathsf{MPS}'$ is defined so as to select the query
\begin{align}
Q_{\mathsf{MPS}'} := \argmax_{Q\in\mQ_\mD} \left(\mathsf{MPS}'(Q)\right)  
\label{eq:Q_MPS'}
\end{align}
where $\mathsf{MPS}'(Q) := -|\dz{}(Q)|$ for all queries $Q$ satisfying the condition given by Eq.~\eqref{eq:condition_MPS} and $\mathsf{MPS}'(Q) := \mathsf{MPS}(Q)$ otherwise.
\begin{proposition}\label{prop:MPS'_satisfies_DPR}
$\mathsf{MPS}'$ satisfies the discrimination-preference relation $DPR$. Further, $Q \prec_{\mathsf{MPS}'} Q'$ whenever $Q \prec_{\mathsf{MPS}} Q'$ for all queries $Q,Q' \in \mQ_{\mD,\tuple{\mo,\mb,\Tp,\Tn}_\RQ}$ for any DPI $\tuple{\mo,\mb,\Tp,\Tn}_\RQ$ and any $\mD \subseteq \minD_{\tuple{\mo,\mb,\Tp,\Tn}_\RQ}$.
\end{proposition}
\begin{proof}
Let $Q,Q'$ be two queries such that $(Q,Q') \in DPR$. Then, $\dz{}(Q') \supset \dz{}(Q)$ by Proposition~\ref{prop:if_Q_discrimination-preferred_over_Q'_then_dz(Q')_supset_dz(Q)}, i.e.\ $|\dz{}(Q')| > |\dz{}(Q)| \geq 0$. Hence, in case $Q$ satisfies $\tuple{|\dx{}(Q)|,|\dnx{}(Q)|,|\dz{}(Q)|} \notin \setof{\tuple{|\mD|-1,1,0},\tuple{1,|\mD|-1,0}}$, we conclude that $\mathsf{MPS}'(Q') = -|\dz{}(Q')| < -|\dz{}(Q)| = \mathsf{MPS}'(Q)$. That is, $Q \prec{\mathsf{MPS}'} Q'$. Otherwise, $\mathsf{MPS}'(Q) = \mathsf{MPS}(Q) > 0$. This holds since $\mathsf{MPS}(Q) := p(\mD_{Q,\min})$ where $|\mD_{Q,\min}| = 1$, i.e.\ $\mD_{Q,\min}$ includes one diagnosis, and each diagnosis has a probability greater than zero (cf.\ page~\pageref{etc:prob_of_each_diag_must_be_greater_zero}). In that, $|\mD_{Q,\min}| = 1$ holds since $\mD_{Q,\min} := \argmin_{\mathbf{X}\in\setof{\dx{}(Q),\dnx{}(Q)}}(|\mathbf{X}|)$ and $\tuple{|\dx{}(Q)|,|\dnx{}(Q)|,|\dz{}(Q)|} \in \setof{\tuple{|\mD|-1,1,0},\tuple{1,|\mD|-1,0}}$. All in all, we have derived that $\mathsf{MPS}'(Q') = -|\dz{}(Q')| < 0$ and $\mathsf{MPS}'(Q) > 0$ which implies $\mathsf{MPS}'(Q) >\mathsf{MPS}'(Q')$, i.e.\ $Q \prec{\mathsf{MPS}'} Q'$. Thence, $\mathsf{MPS}'$ satisfies $DPR$.

Second, assume two queries $Q,Q' \in \mQ_{\mD,\tuple{\mo,\mb,\Tp,\Tn}_\RQ}$ and let $Q \prec_{\mathsf{MPS}} Q'$, i.e.\ $\mathsf{MPS}(Q) > \mathsf{MPS}(Q')$. Since $\mathsf{MPS}$ is either positive (a non-zero probability) in case Eq.~\eqref{eq:condition_MPS} is met or zero otherwise, we have that $\mathsf{MPS}(Q_i) \geq 0$ for all $Q_i \in \mQ_{\mD,\tuple{\mo,\mb,\Tp,\Tn}_\RQ}$. Therefore, $Q$ cannot satisfy Eq.~\eqref{eq:condition_MPS} as this would imply that $\mathsf{MPS}(Q') < 0$. Due to this fact, we obtain that $0 < \mathsf{MPS}(Q) = \mathsf{MPS}'(Q)$ by the definition of $\mathsf{MPS}'(Q)$ and the fact that the probability of each diagnosis in greater than zero. As a consequence, there are two cases to distinguish: (a)~$Q'$ does and (b)~does not satisfy Eq.~\eqref{eq:condition_MPS}. Given case (a), it clearly holds that $\mathsf{MPS}'(Q') \leq 0$ by the definition of $\mathsf{MPS}'(Q')$. So, combining the results, we get $\mathsf{MPS}'(Q') \leq 0 < \mathsf{MPS}'(Q)$ which is equivalent to $Q \prec_{\mathsf{MPS}'} Q'$. If case (b) arises, $\mathsf{MPS}(Q') = \mathsf{MPS}'(Q')$ by the definition of $\mathsf{MPS}'(Q)$. Hence, both $\mathsf{MPS}(Q) = \mathsf{MPS}'(Q)$ and $\mathsf{MPS}(Q') = \mathsf{MPS}'(Q')$. However, by assumption $\mathsf{MPS}(Q) > \mathsf{MPS}(Q')$. Thus, $\mathsf{MPS}'(Q) > \mathsf{MPS}'(Q')$ which is why $Q \prec_{\mathsf{MPS}'} Q'$.  
\end{proof}
\begin{corollary}\label{cor:MPS_equiv_MPS'_for_mQ_incl_only_dz=0_queries}
Let $\mQ$ be any set of queries w.r.t.\ a DPI $\tuple{\mo,\mb,\Tp,\Tn}_\RQ$ and a set of leading diagnoses $\mD$ w.r.t.\ this DPI such that each query $Q\in\mQ$ satisfies $\dz{}(Q) = \emptyset$. Then $\mathsf{MPS} \equiv_{\mQ} \mathsf{MPS}'$.
\end{corollary}
\begin{proof}
Is a direct consequence of the definitions of $\mathsf{MPS}$ and $\mathsf{MPS}'$.
\end{proof}
\begin{corollary}\label{cor:MPS'_superior_to_MPS}
$\mathsf{MPS}'$ is superior to $\mathsf{MPS}$.
\end{corollary}
\begin{proof}
An example of two queries $Q,Q'$ where $(Q,Q') \in DPR$ and $Q \prec_{\mathsf{MPS}} Q'$ does not hold was stated in the proof of Proposition~\ref{prop:MPS_consistent_with_DPR_but_not_satisfies_DPR}. Moreover, Proposition~\ref{prop:MPS'_satisfies_DPR} testifies that $\mathsf{MPS}'$ satisfies the $DPR$. Hence, all three bullets of the definition of superiority (Definition~\ref{def:measures_equivalent_theoretically-optimal_superior}) are met.
\end{proof}

The charm of the $\mathsf{MPS}$ and $\mathsf{MPS}'$ measures lies in the fact that we can, in theory, derive the \emph{exact} shape of the q-partition of the best query (which might be one of potentially multiple best queries) w.r.t.\ $\mathsf{MPS}$ or $\mathsf{MPS}'$, respectively -- and not only properties it must meet. What is more, it is guaranteed that such a best query must exist for any concrete DPI and set of leading diagnoses. This is confirmed by the next proposition. Note that this means that no (heuristic) search is necessary to determine the optimal query w.r.t.\ $\mathsf{MPS}$ or $\mathsf{MPS}'$, respectively.
\begin{proposition}\label{prop:best_query_wrt_MPS_has_q-partition}
Let $\langle\mo,\mb,\Tp,\Tn\rangle_\RQ$ be an arbitrary DPI and $\mD \subseteq \minD_{\langle\mo,\mb,\Tp,\Tn\rangle_\RQ}$. Then
one query $Q \in \mQ_{\mD,\langle\mo,\mb,\Tp,\Tn\rangle_\RQ}$ that is considered best in $\mQ_{\mD,\langle\mo,\mb,\Tp,\Tn\rangle_\RQ}$ by $\mathsf{MPS}$ (i.e.\ there is no query $Q' \in \mQ_{\mD,\langle\mo,\mb,\Tp,\Tn\rangle_\RQ}$ such that $Q' \prec_{\mathsf{MPS}} Q$) has the q-partition $\tuple{\setof{\md},\mD\setminus\setof{\md},\emptyset}$ where $p(\md) \geq p(\md')$ for all $\md' \in \mD$. The same holds if $\mathsf{MPS}$ is replaced by $\mathsf{MPS}'$.
\end{proposition}
\begin{proof}
By \cite[Prop.~7.5]{Rodler2015phd} each partition $\tuple{\setof{\md},\mD\setminus\setof{\md},\emptyset}$ is a q-partition w.r.t.\ any set of leading diagnoses $\mD$ where $|\mD| \geq 2$ (note that the latter is a necessary requirement for queries in $\mQ_\mD$ to exist, cf.\ Proposition~\ref{prop:properties_of_q-partitions},(\ref{prop:properties_of_q-partitions:enum:for_each_q-partition_dx_is_empty_and_dnx_is_empty}.)). Each such q-partition has a positive $\mathsf{MPS}$ value due to the fact that its $\dz{}$-set is empty. Since all queries $Q$ where neither $\dx{}(Q)$ nor $\dnx{}(Q)$ is a singleton feature $\mathsf{MPS}(Q) = 0$, the statement of the Proposition follows from the definition of $\mD_{Q,\min}$ and $Q_{\mathsf{MPS}}$.
%

We show that the same must hold for $\mathsf{MPS}'$ by contradiction. Suppose $\mathsf{MPS}'$ does not consider the query with the given q-partition as a best one among $\mQ_{\mD,\langle\mo,\mb,\Tp,\Tn\rangle_\RQ}$. Then, since $\mathsf{MPS}'$ equals $\mathsf{MPS}$, by definition, for queries not satisfying Eq.~\eqref{eq:condition_MPS} and all queries with associated q-partitions of the form $\tuple{\setof{\md},\mD\setminus\setof{\md},\emptyset}$ do not satisfy Eq.~\eqref{eq:condition_MPS}, the best query w.r.t.\ $\mathsf{MPS}'$ must not be of the form $\tuple{\setof{\md},\mD\setminus\setof{\md},\emptyset}$. There are two possible types of q-partitions remaining which could be regarded best by $\mathsf{MPS}'$, either $\tuple{\mD\setminus\setof{\md},\setof{\md},\emptyset}$ or some q-partition that does meet Eq.~\eqref{eq:condition_MPS}. In the former case, the value returned for any such q-partition can be at most the value achieved for the best q-partition of the form $\tuple{\setof{\md},\mD\setminus\setof{\md},\emptyset}$. This would imply that $\mathsf{MPS}'$ considers the same q-partition as a best one as $\mathsf{MPS}$, contradiction. The latter case entails a contradiction in the same style due to the fact that, for each q-partition satisfying Eq.~\eqref{eq:condition_MPS}, the $\mathsf{MPS}'$ measure is less than or equal to zero and thus smaller than for all q-partitions of the form $\tuple{\setof{\md},\mD\setminus\setof{\md},\emptyset}$, in particular smaller than for the one q-partition considered best by $\mathsf{MPS}$ and hence, as argued above, smaller than for the one q-partition considered best by $\mathsf{MPS}'$.
\end{proof}

\vspace{1em}

\noindent\emph{Biased Maximal Elimination ($\mathsf{BME}$)}:
The idea underlying $\mathsf{BME}$ is to prefer a situation where a maximal number of the leading diagnoses $\mD$, i.e.\ at least half of them, can be eliminated with a probability of more than $0.5$. Formally, this can be expressed as preferring the query 
\begin{align*}
Q_{\mathsf{BME}} := \argmax_{Q\in\mQ_\mD} \left(\mathsf{BME}(Q)\right) \quad \text{ where } \quad \mathsf{BME}(Q) := |\mD_{Q,p,\min}|
\end{align*}
where 
\[
    \mD_{Q,p,\min} := \left\{\begin{array}{lr}
        \dnx{}(Q), & \quad\text{if } p(\dnx{}(Q)) < p(\dx{}(Q))\\
        \dx{}(Q), & \quad\text{if } p(\dnx{}(Q)) > p(\dx{}(Q))\\
        0, & \text{otherwise } 
        \end{array}
				\right.
\]

At this, $\mD_{Q,p,\min}$ is exactly the subset of $\mD$ that is invalidated if the more probable answer to $Q$ is actually given. Because, if say $t$ is the more likely answer, i.e.\ $p(Q = t) = p(\dx{}(Q)) + \frac{1}{2}p(\dz{}(Q)) > p(Q=f) = p(\dnx{}(Q)) + \frac{1}{2}p(\dz{}(Q))$, then $p(\dx{}(Q)) > p(\dnx{}(Q))$, i.e.\ $\mD_{Q,p,\min} = \dnx{}(Q)$, which is the set invalidated as a result of a positive answer to $Q$ (cf.\ Proposition~\ref{prop:properties_of_q-partitions},(\ref{prop:properties_of_q-partitions:enum:dx_dnx_dz_contain_exactly_those_diags_that_are...}.)). $\mD_{Q,p,\min}$ is set to zero in case none of the answers is more likely. That is because $\mathsf{BME}$, as the word ``biased'' in its name suggests, strives for queries with a bias towards one of the answers and hence rates unbiased queries very low. Note that each unbiased query is indeed worse w.r.t.\ $\mathsf{BME}$ than each biased one since $|\dx{}(Q)|$ as well as $|\dnx{}(Q)|$ are always positive for any query (cf.\ Proposition~\ref{prop:properties_of_q-partitions},(\ref{prop:properties_of_q-partitions:enum:for_each_q-partition_dx_is_empty_and_dnx_is_empty}.)).

The $\mathsf{BME}$ measure fits into the EMC framework because it maximizes the most probable change on the current model (leading diagnoses) resulting from the query's answer.
\begin{proposition}\label{prop:BME_is_not_consistent_with_DPR}
$\mathsf{BME}$ is not consistent with the discrimination-preference relation $DPR$ (and hence does not satisfy $DPR$).
\end{proposition}
\begin{proof}
Let the queries $Q$ and $Q'$ be characterized by the following q-partitions:
\begin{align*}
\langle \dx{}(Q),\dnx{}(Q),\dz{}(Q) \rangle &= \langle \setof{\md_1,\md_2,\md_3},\setof{\md_4},\emptyset \rangle \\
\langle \dx{}(Q'),\dnx{}(Q'),\dz{}(Q') \rangle &= \langle \setof{\md_1,\md_2},\setof{\md_4},\setof{\md_3} \rangle
\end{align*}
Clearly, by Proposition~\ref{prop:construction_of_Q'_from_Q_if_Q_discrimination-preferred_over_Q'}, $Q$ is discrimination-preferred to $Q'$. Let further the diagnosis probabilities $p_i := p(\md_i)$ be as follows: \[\tuple{p_1,\dots,p_4} = \tuple{0.1,0.1,0.5,0.3}\]
Now, $\mD_{Q,p,\min} = \dnx{}(Q) = \setof{\md_4}$, i.e.\ $\mathsf{BME}(Q) = |\setof{\md_4}| = 1$, whereas $\mD_{Q',p,\min} = \dx{}(Q') = \setof{\md_1,\md_2}$, i.e.\ $\mathsf{BME}(Q') = |\setof{\md_1,\md_2}| = 2$. 
In other words, $Q$ is more likely (with a probability of $0.7$) to eliminate one diagnosis in $\mD$ (than to eliminate three, i.e.\ $\setof{\md_1,\md_2,\md_3}$). On the other hand, $Q'$ is more likely (probability $0.3 + \frac{1}{2} 0.5 = 0.55$) to invalidate two diagnoses in $\mD$ (than to invalidate one, i.e.\ $\setof{\md_4}$). 
Consequently, $\mathsf{BME}(Q') > \mathsf{BME}(Q)$. That is, $Q' \prec_{\mathsf{BME}} Q$.
\end{proof}
\begin{corollary}\label{cor:all_measures_satisfying_DPR_superior_to_BME}
All measures satisfying the discrimination-preference relation are superior to $\mathsf{BME}$.
\end{corollary}
\begin{proof}
Immediate from Propositions~\ref{prop:if_m1_satisfies_DPR_and_m2_does_not_then_m1_superior_to_m2} and \ref{prop:BME_is_not_consistent_with_DPR}.
\end{proof}
%

\vspace{1em}

\paragraph{Expected Error Reduction (EER).}
The EER framework judges the one unlabeled instance $(Q,?)$ most favorably, the labeling of which implies the greatest reduction of the expected generalization error. There are (at least) two methods of estimating the expected generalization error \cite[Sec.~4.1]{settles2012}. The first is the \emph{expected classification error} over some validation set, i.e.\ the expected probability of misclassifying some new query using the new model $M_{\mathsf{new}}$ resulting from the current model $M$ through addition of the labeled instance $(Q,u(Q))$. The validation set is usually a (sufficiently large) subset of all unlabeled instances. The best query w.r.t.\ expected classification error (cf.\ \cite[Eq.~4.1]{settles2012}) is given by 
\begin{align}
Q_{\mathsf{CE}} := \argmin_{Q \in \mQ_\mD} \sum_{a \in\setof{t,f}} p(Q=a) \left[ \frac{1}{|\mQ_{\mD'(Q=a)}|}\sum_{Q_i \in \mQ_{\mD'(Q=a)}} 1 - p_{\mD'(Q=a)}(Q_i = a_{Qi,\max}) \right]
\label{eq:Q_CE}
\end{align}
where $p$, as usual, is the current probability distribution of the leading diagnoses $\mD$, $\mD'(Q=a)$ is the new set of leading diagnoses w.r.t.\ the new DPI resulting from the addition of $Q$ to the positive test cases $\Tp$ if $a = t$ or to the negative test cases $\Tn$ otherwise, and $p_{\mD'(Q=a)}$ refers to the (updated) probability distribution of the new leading diagnoses $\mD'(Q=a)$. The term in squared brackets in Eq.~\eqref{eq:Q_CE} corresponds to the expected probability that the new model (given by the new set of leading diagnoses $\mD'(Q=a)$) wrongly classifies a query among all queries w.r.t.\ the new leading diagnoses set $\mD'(Q=a)$. We remark that, unlike in \cite[Eq.~4.1]{settles2012} (where the (size of the) pool of unlabeled instances is assumed static and independent of the current model), the factor $\frac{1}{|\mQ_{\mD'(Q=a)}|}$, i.e.\ the division by the number of all queries, cannot be omitted in our case as there might be different numbers of queries w.r.t.\ different leading diagnoses sets.
Intuitively, a query is given a good score by $\mathsf{CE}$ if the average new model resulting from getting an answer to it predicts the answer to a large proportion of new queries with high certainty.

The second way of assessing the expected generalization error is by relying upon the \emph{expected entropy}. This means favoring the query characterized as follows:
\begin{align}
Q_{\mathsf{EE}} := \argmin_{Q \in \mQ_\mD} \sum_{a \in\setof{t,f}} p(Q=a) \left[ \frac{1}{|\mQ_{\mD'(Q=a)}|} \sum_{Q_i \in \mQ_{\mD'(Q=a)}} H_{\mD'(Q=a)}(Q_i)  \right]
\label{eq:Q_EE}
\end{align}
where the information entropy $H_{\mD'(Q=a)}(Q_i)$ of a given query $Q_i$ w.r.t.\ the new leading diagnoses $\mD'(Q=a)$ and the probability measure $p_{\mD'(Q=a)}$ is given by
\begin{align*}
H_{\mD'(Q=a)}(Q_i) := -\sum_{a_{i}\in\setof{t,f}} p_{\mD'(Q=a)}(Q_i = a_{i})\log_2 \left(p_{\mD'(Q=a)}(Q_i = a_{i})\right)
\end{align*}
The expected entropy represented within the squared brackets in Eq.~\eqref{eq:Q_EE} is the lower, the more tendency there is in $p_{\mD'(Q=a)}$ towards a single answer w.r.t.\ queries in $\mQ_{\mD'(Q=a)}$. Hence, similarly as in the case of $\mathsf{CE}$, a query is given a high score by $\mathsf{EE}$ if the average new model resulting from getting an answer to it predicts the answer to a large proportion of new queries with high certainty.

In order to compute $Q_{\mathsf{CE}}$ or $Q_{\mathsf{EE}}$, there are several requirements: 
\begin{enumerate}
	\item \label{enum:req_CE_EE_1} A pool of queries $\mQ_\mD$ must be generated w.r.t.\ the current set of leading diagnoses $\mD$ (over this pool, the optimal query is sought; we are not able to derive any properties of the best q-partition w.r.t.\ $\mathsf{CE}$ (or $\mathsf{EE}$) that would enable a direct search),
	\item \label{enum:req_CE_EE_2} for both ficticiuous answers $a \in \setof{t,f}$ to each query $Q\in\mQ_\mD$, a new DPI resulting from the current one by a corresponding addition of the test case $Q$ must be considered and a set of leading diagnoses $\mD'(Q=a)$ must be computed w.r.t.\ it,
	\item \label{enum:req_CE_EE_3} the probability measure $p$ must be updated (Bayesian update, cf.\ \cite[p.~130]{Rodler2015phd}) yielding $p_{\mD'(Q=a)}$, a probability measure w.r.t.\ $\mD'(Q=a)$, 
	\item \label{enum:req_CE_EE_4} another pool of queries $\mQ_{\mD'(Q=a)}$ must be generated w.r.t.\ the new set of leading diagnoses $\mD'(Q=a)$, and 
	\item \label{enum:req_CE_EE_5} the probability w.r.t.\ the measure $p_{\mD'(Q=a)}$ of answers to each query in $\mQ_{\mD'(Q=a)}$ must be computed.
\end{enumerate}
   
Scrutinizing these necessary computations, we first observe that $\mathsf{CE}$ and $\mathsf{EE}$ require as an active learning mode the pool-based sampling (PS) approach (see Section~\ref{sec:ActiveLearningInInteractiveOntologyDebugging}). Hence, contrary to the measures we analyzed before, $\mathsf{CE}$ and $\mathsf{EE}$ are not amenable to a heuristic q-partition search method. Second, by the efficient q-partition generation technique we will introduce in Section~\ref{sec:FindingOptimalQPartitions}, q-partitions can be generated inexpensively, i.e.\ without costly calls to any oracles, e.g.\ logical reasoning services. Hence, requirements \ref{enum:req_CE_EE_1} and \ref{enum:req_CE_EE_4} do not pose significant challenges. We want to emphasize, however, that, with current methods of q-partition or query generation, these two requirements would already be a knock-out criterion for $\mathsf{CE}$ as well as $\mathsf{EE}$ (see Section~\ref{sec:RelatedWork}). An issue that is present in spite of the efficient computation of q-partitions is that if the sets of queries $\mQ_\mD$ and $\mQ_{\mD'(Q=a)}$ are not restricted in whatsoever way, they will generally have an exponential cardinality. That is, the determination of the best query $Q_{\mathsf{CE}}$ ($Q_{\mathsf{EE}}$) requires an exponential number of evaluations (requirement \ref{enum:req_CE_EE_5}) an exponential number of times (i.e.\ for all queries in $\mQ_\mD$).
The computational complexity of requirement \ref{enum:req_CE_EE_3} associated with the computation of the new probability measure is linear in the number of new test cases added throughout the debugging session, as was shown in \cite[Prop.~9.2]{Rodler2015phd} and hence is not a great problem. 
The actual issue concerning the computation of $Q_{\mathsf{CE}}$ ($Q_{\mathsf{EE}}$) is requirement \ref{enum:req_CE_EE_2}, the computational costs of which are prohibitive and thus hinder an efficient deployment of $\mathsf{CE}$ ($\mathsf{EE}$) in the interactive debugging scenario where timeliness has highest priority. The reason for this is the hardness of the diagnosis computation which is proven to be at least as hard as $\Sigma_2^{\mathrm{P}} = \mathrm{NP}^{\mathrm{NP}}$ even for some DPI over propositional logic \cite[Cor.~9.1]{Rodler2015phd}. Such a diagnosis computation must be accomplished once for each answer to each query in $\mQ_\mD$, i.e.\ $2 \cdot |\mQ_\mD|$ often where $|\mQ_\mD|$ is (generally) of exponential size. 
Consequently, the EER framework cannot be reasonably applied to our ontology debugging scenario.
\subsubsection{Risk-Optimized Active Learning Measure}
\label{sec:rio} 
In this subsection we focus on a measure for \mbox{q-partition} selection called \textsf{RI}sk \textsf{O}ptimization measure, $\mathsf{RIO}$ for short. Based on our lessons learned from Sections~\ref{sec:ExistingActiveLearningMeasuresForKBDebugging} and \ref{sec:NewActiveLearningMeasuresForKBDebugging}, we can consider the $\mathsf{ENT}$ and $\mathsf{SPL}$ measures as representatives for basically different active learning paradigms, uncertainty sampling (US) and expected model change (EMC) on the one hand ($\mathsf{ENT}$ is strongly related to $\mathsf{H}$, $\mathsf{LC}$, $\mathsf{M}$ from the US framework, cf.\ Proposition~\ref{prop:Q_H_eq_Q_ENT_if_dz=0} and Corollary~\ref{cor:ENT_superior_to_H_M_LC}, and to $\mathsf{EMCa}$ from the EMC framework, cf.\ Propositions~\ref{prop:ENT_preserves_discrimination-pref-order} and \ref{prop:EMCa_preserves_discrimination-pref-order} as well as Corollary~\ref{cor:EMCa_z_for_z>=2_is_better_than_ENTr_for_all_r}), and query-by-committee (QBC) on the other hand ($\mathsf{ENT}$ is strongly related to $\mathsf{VE}$ from the QBC framework).
 
The fundamental difference between $\mathsf{ENT}$ and $\mathsf{SPL}$ concerning the number of queries to a user required during a debugging session is witnessed by experiments conducted in~\cite{Shchekotykhin2012, ksgf2010, Rodler2013}. Results in these works point out the importance of a careful choice of query selection measure
in the light of the used fault estimates, i.e.\ the probability measure $p$. 
Concretely, for good estimates $\mathsf{ENT}$ is appropriate. $\mathsf{SPL}$, in contrast, works particularly well in situations where estimates are doubted. The essential problem, however, is that measuring the quality of the estimates is not possible before additional information in terms of answered queries from a user is collected. In fact, a significant overhead in the number of queries necessary until identification of the true diagnosis was reported in case an unfavorable combination of the quality of the estimates on the one hand, and the q-partition selection measure on the other hand, is used. Compared to the usage of the more appropriate measure among $\setof{\mathsf{ENT},\mathsf{SPL}}$ in a particular debugging scenario, the reliance upon the worse measure among $\setof{\mathsf{ENT},\mathsf{SPL}}$ amounted to several hundred percent of time or effort overhead for the interacting user on average and even reached numbers higher than 2000\%~\cite{Rodler2013}. Hence, there is a certain risk of weak performance when opting for any of the measures $\mathsf{ENT}$ and $\mathsf{SPL}$. This probably grounds the name of $\mathsf{RIO}$.

The idea of $\mathsf{RIO}$ is to unite advantages of both $\mathsf{ENT}$ and $\mathsf{SPL}$ measures in a dynamic reinforcement learning measure that is constantly adapted based on the observed performance. The performance is measured in terms of the diagnosis elimination rate w.r.t.\ the set of leading diagnoses $\mD$. In that, $\mathsf{RIO}$ relies on the given fault estimates $p$ and recommends similar q-partitions as $\mathsf{ENT}$ when facing good performance, whereas aggravation of performance leads to a gradual neglect of $p$ and a behavior akin to $\mathsf{SPL}$. 
%

The $\mathsf{RIO}$ measure aims at selecting a query that is primarily not too ``risky'' and secondarily features a high information gain. Intuitively, a query which might invalidate only a small number of leading diagnoses if answered unfavorably, is more risky (or less cautious) than a query that eliminates a higher number of leading diagnoses for both answers. So, the risk of a query $Q$ w.r.t.\ the leading diagnoses $\mD$ can be quantified in terms of the worst case diagnosis elimination rate of $Q$ which is determined by its q-partition $\Pt(Q) = \langle\dx{}(Q),\dnx{}(Q),\dz{}(Q)\rangle$. Namely, the worst case occurs when the given answer $u(Q)$ to $Q$ is unfavorable, i.e.\ the smaller set of diagnoses in $\setof{\dx{}(Q), \dnx{}(Q)}$ gets invalid.
The notion of cautiousness or risk aversion of a query is formally captured as follows:
\begin{definition}[Cautiousness of a Query and q-Partition]\label{def:query_cautiousness}
Let $\mD \subseteq \minD_{\langle\mo,\mb,\Tp,\Tn\rangle_\RQ}$, 
$\Pt = \langle\dx{},\dnx{},\dz{}\rangle$ a q-partition w.r.t.\ $\mD$, 
and $|\mD|_{\min} := \min(|\dx{}| ,|\dnx{}|)$. Then we call 
$\qc(\Pt) := \frac{ |\mD|_{\min} }{|\mD|}$
\emph{the cautiousness of $\Pt$}.

Let $Q$ be a query in $\mQ_\mD$ with associated q-partition $\Pt(Q)$. Then, we define \emph{the cautiousness $\qc(Q)$ of $Q$} as $\qc(Q):=\qc(\Pt(Q))$.

Let $\Pt, \Pt'$ be two q-partitions w.r.t.\ $\mD$ and $Q$ and $Q'$ be queries associated with $\Pt$ and $\Pt'$, respectively. Then, we call the q-partition $\Pt$ (the query $Q$) \emph{more} respectively \emph{less cautious} than the q-partition $\Pt'$ (the query $Q'$) iff $\qc(\Pt) > \qc(\Pt')$ respectively $\qc(\Pt) < \qc(\Pt')$.

\end{definition}
The following proposition is a simple consequence of Definition~\ref{def:query_cautiousness}:
\begin{proposition}
The cautiousness $\qc(\Pt)$ of a q-partition w.r.t.\ $\mD$ can attain values in the interval 
\[ [\underline{\qc},\overline{\qc}] :=  \left[\frac{1}{|\mD|},\frac{\left\lfloor \frac{|\mD|}{2}\right\rfloor}{|\mD|}\right]\]
\end{proposition}
$\mathsf{RIO}$ is guided by the reinforcement learning parameter $\uc$ that appoints how cautious the next selected query should minimally be. The user may specify desired lower and upper bounds $\underline{\uc}$ and $\overline{\uc}$ for $\uc$ such that 
$[\underline{\uc},\overline{\uc}] \subseteq [\underline{\qc},\overline{\qc}]$ is preserved. In this vein, the user can take influence on the dynamic behavior of the $\mathsf{RIO}$ measure. Trust (disbelief) in the fault estimates and thus in the probability measure $p$ can be expressed by specification of lower (higher) values for $\underline{\uc}$ and/or $\overline{\uc}$. Simply said, a lower (higher) value of $\underline{\uc}$ means higher (lower) maximum possible risk, whereas a lower (higher) value of $\overline{\uc}$ means higher (lower) minimal risk of a query $\mathsf{RIO}$ might suggest during a debugging session.

The parameter $\uc$ can be used to subdivide the set of q-partitions w.r.t.\ $\mD$ into high-risk and non-high-risk q-partitions. Formally:
\begin{definition}[High-Risk Query and q-Partition]\label{def:high-risk_query}
Let 
$\Pt = \langle\dx{},\dnx{},\dz{}\rangle$ be a q-partition w.r.t.\ $\mD$, $Q$ a query associated with $\Pt$ and $\uc \in [\underline{\qc},\overline{\qc}]$. Then we call $\Pt$ a \emph{high-risk q-partition w.r.t.\ $\uc$} ($Q$ a \emph{high-risk query w.r.t.\ $\uc$}) iff $\qc(\Pt) < \uc$, and a \emph{non-high-risk q-partition w.r.t.\ $\uc$} (\emph{non-high-risk query w.r.t.\ $\uc$}) otherwise. 
The \emph{set of non-high-risk q-partitions w.r.t.\ $\uc$ and $\mD$} is denoted by $\NHR_{\uc,\mD}$. 
\end{definition}
Let $\LC(\mathcal{X})$ denote the subset of least cautious q-partitions of a set of q-partitions $\mathcal{X}$ and be defined as 
\[\LC(\mathcal{X}) := \{\Pt \in \mathcal{X} \;|\; \forall \Pt' \in \mathcal{X}:\, \qc(\Pt) \leq \qc(\Pt')\}\] 
Then, the query $Q_{\mathsf{RIO}}$ in $\mQ_\mD$ selected by the $\mathsf{RIO}$ measure is formally defined as follows:
\begin{align}
Q_{\mathsf{RIO}} := 
\begin{cases}
  Q_{\mathsf{ENT}} \quad &\text{if }\;\Pt(Q_{\mathsf{ENT}})\in\NHR_{\uc,\mD}  \\ 
  \underset{\setof{Q\,|\, \Pt(Q) \in \LC(\NHR_{\uc,\mD})}}{\argmin}(\mathsf{ENT}(Q)) \quad  &\text{otherwise } 
\end{cases}
\label{eq:Q_RIO}
\end{align}
where $Q_{\mathsf{ENT}}$ and $\mathsf{ENT}(Q)$ are specified in Equations~\eqref{eq:best_query_ENT} and \eqref{eq:ENT}, respectively.
So, (1)~$\mathsf{RIO}$ prefers the query $Q_{\mathsf{ENT}}$ favored by $\mathsf{ENT}$ in case this query is a non-high-risk query with regard to the parameter $\uc$ and the leading diagnoses $\mD$. Otherwise, (2)~$\mathsf{RIO}$ chooses the one query among the least cautious non-high-risk queries which has minimal $\mathsf{ENT}$ value. Note that this way of selecting queries is not equivalent to just using (2) and omitting (1). The reason is that in (1) we only postulate that $Q_{\mathsf{ENT}}$ must be a non-high-risk query, i.e.\ it must not violate what is prescribed by $\uc$. However, $Q_{\mathsf{ENT}}$ must not necessarily be least cautious among the non-high-risk queries.  

The restriction to non-high-risk queries, i.e.\ q-partitions in $\NHR_{\uc,\mD}$, should guarantee a sufficiently good $\mathsf{SPL}$ measure, depending on the current value of the dynamic parameter $\uc$. The restriction to least cautious queries among the non-high-risk queries in (2) 
in turn has the effect that a query with maximum admissible risk according to $\uc$ is preferred. Among those queries the one with optimal information gain ($\mathsf{ENT}$) is selected. Note that the situation $Q_{\mathsf{RIO}}=Q_{\mathsf{ENT}}$ gets less likely, the higher $\uc$ is. So, the worse the measured performance so far when relying on $\mathsf{ENT}$ and thus on the probability measure $p$, the lower the likeliness of using $p$ for the selection of the next query. Instead, a query with better $\mathsf{SPL}$ measure is favored and $p$ becomes a secondary criterion. 

The reinforcement learning of parameter $\uc$ is based on the measured diagnosis elimination rate after the answer $u(Q_{\mathsf{RIO}})$ to query $Q_{\mathsf{RIO}}$ has been obtained.  
\begin{definition}[Elimination Rate, (Un)favorable Answer]\label{def:elimination_rate}
Let $Q\in\mQ_\mD$ be a query and 
$u(Q) \in \setof{t,f}$ a user's answer to $Q$. Further, let $\mD^* := \dnx{}(Q)$ for $u(Q)=t$ and $\mD^* := \dx{}(Q)$ for $u(Q)=f$.
Then we call $\ER(Q,u(Q))=\frac{|\mD^*|}{|\mD|}$ the \emph{elimination rate of $Q$}.

Moreover, $u(Q)$ is called \emph{(un)favorable answer to $Q$} iff $u(Q)$ (minimizes) maximizes $\ER(Q,u(Q))$. In the special case of even $|\mD|$ and $\ER(Q,u(Q))=\frac{1}{2}$, we call $u(Q)$ a \emph{favorable answer to $Q$}.
\end{definition}
As update function for $\uc$ we use $\uc \leftarrow \uc + \uc_{\adj}$ where the cautiousness adjustment factor 
\[\uc_{\adj} :=\; 2\,(\overline{\uc} - \underline{\uc})\mathit{\adj}(Q,u(Q))\] 
where
\begin{align*}
\mathit{\adj}(Q,u(Q)) :=  \frac{\left\lfloor \frac{|\mD|-1}{2}\right\rfloor + \frac{1}{2}}{|\mD|} - \ER(Q,u(Q))
\end{align*}
The (always non-negative) scaling factor $2 \, (\overline{\uc} - \underline{\uc})$ regulates the extent of the adjustment depending on the interval length $\overline{\uc} - \underline{\uc}$. More crucial is the factor $\mathit{\adj}(Q,u(Q))$ that indicates the sign and magnitude of the cautiousness adjustment. This adjustment means a decrease of cautiousness $\uc$ if $u(Q)$ is favorable, i.e.\ a bonus to take higher risk in the next iteration, and an increase of $\uc$ if $u(Q)$ is unfavorable, i.e.\ a penalty to be more risk-averse when choosing the next query.
\begin{proposition}
Let $Q\in\mQ_\mD$ and $\uc^{(i)}$ the current value of the cautiousness parameter. Then: 
\begin{itemize}
	\item If $\overline{\uc} = \underline{\uc}$, then $\uc^{(i+1)} = \uc^{(i)}$.
	\item If $\overline{\uc} > \underline{\uc}$, then for the updated value $\uc^{(i+1)} \leftarrow \uc^{(i)} + c_\adj$ it holds that $\uc^{(i+1)} < \uc^{(i)}$ for favorable and $\uc^{(i+1)} > \uc^{(i)}$ for unfavorable answer $u(Q)$.
\end{itemize}
\end{proposition}
\begin{proof}
The first bullet is clear since $\overline{\uc} = \underline{\uc}$ implies that $c_{AF} = 0$ due to $\overline{\uc} - \underline{\uc} = 0$.
Regarding the second bullet, we argue as follows:
Due to the fact that $\uc_{\adj}$ has the same sign as $\adj(Q,u(Q))$ by $2\,(\overline{\uc} - \underline{\uc}) > 0$, it is sufficient to show that $\adj(Q,u(Q)) < 0$ for favorable answer $u(Q)$ and $\adj(Q,u(Q)) > 0$ for unfavorable answer $u(Q)$ for all queries $Q\in\mQ_\mD$.
The first observation is that the maximal value the minimal elimination rate of a query $Q\in\mQ_\mD$ can attain is $\max\min_{\ER}:=\frac{1}{|\mD|}\lfloor \frac{|\mD|}{2}\rfloor$ and the minimal value the maximal elimination rate of a query $Q\in\mQ_\mD$ can attain is $\min\max_{\ER} := \frac{1}{|\mD|}\lceil \frac{|\mD|}{2}\rceil$. In other words, for odd $|\mD|$, by Definition~\ref{def:elimination_rate}, the minimal (maximal) elimination rate that must be achieved by an answer to a query w.r.t.\ the leading diagnoses $\mD$ in order to be called favorable (unfavorable) is given by $\min\max_{\ER}$ ($\max\min_{\ER}$). For even $|\mD|$, $\min\max_{\ER} = \max\min_{\ER}$ holds and, by Definition~\ref{def:elimination_rate}, an answer to a query yielding an elimination rate of $\max\min_{\ER}$ is called favorable. 
Let the term 
\[\frac{\left\lfloor \frac{|\mD|-1}{2}\right\rfloor + \frac{1}{2}}{|\mD|}\]
be called $\base_{\ER}$.

Let us first consider even $|\mD|$: It holds that $\lceil \frac{|\mD|}{2}\rceil = \lfloor \frac{|\mD|}{2}\rfloor = \frac{|\mD|}{2}$. Further, $\left\lfloor \frac{|\mD|-1}{2}\right\rfloor = \frac{|\mD|-2}{2}$ since $|\mD|-1$ is odd. Hence, the numerator of $\base_{\ER}$ reduces to $\frac{|\mD|-2+1}{2} = \frac{|\mD|-1}{2}$. 
Let the answer be favorable. That is, the elimination rate must be at least $\min\max_{\ER} = \frac{1}{|\mD|}\frac{|\mD|}{2}$. 
So, $\adj(Q,u(Q)) \leq \base_{\ER} - \min\max_{\ER} = \frac{1}{|\mD|} (\frac{|\mD|-1}{2} - \frac{|\mD|}{2}) = -\frac{1}{2|\mD|} < 0$. 
%
Now, let the answer be unfavorable. Then, 
an elimination rate smaller than or equal to $\min\max_{\ER} - \frac{1}{|\mD|} = \frac{1}{|\mD|}(\frac{|\mD|}{2}-1)$ must be achieved by this query answer. Consequently, using the calculation of $\adj(Q,u(Q))$ above (by using an elimination rate of maximally $\min\max_{\ER} - \frac{1}{|\mD|}$ instead of minimally $\min\max_{\ER}$), we have that $\adj(Q,u(Q)) = -\frac{1}{2|\mD|} + \frac{1}{|\mD|} = -\frac{1}{2|\mD|} + \frac{2}{2|\mD|} = \frac{1}{2|\mD|} > 0$.

For odd $|\mD|$, it holds that $\lceil \frac{|\mD|}{2}\rceil = \frac{|\mD|+1}{2}$ and $\lfloor \frac{|\mD|}{2}\rfloor = \frac{|\mD|-1}{2}$. The maximal elimination rate that still means that an answer $u(Q)$ is unfavorable is $\max\min_{\ER} = \frac{1}{|\mD|}\frac{|\mD|-1}{2}$. So, 
$\adj(Q,u(Q)) \geq \base_{\ER} - \max\min_{\ER} = \frac{1}{|\mD|} (\frac{|\mD|}{2} - \frac{|\mD|-1}{2}) = \frac{1}{2|\mD|} >0$. 
On the other hand, the least elimination rate for a favorable answer $u(Q)$ is $\min\max_{\ER} = \frac{1}{|\mD|}\frac{|\mD|+1}{2} $ which yields $\adj(Q,u(Q)) \leq \base_{\ER} - \min\max_{\ER} = \frac{1}{|\mD|} (\frac{|\mD|}{2} - \frac{|\mD|+1}{2}) = -\frac{1}{2|\mD|} <0$. 
\end{proof}

If the updated cautiousness $\uc + \uc_{\adj}$ is outside the user-defined cautiousness interval $[\underline{\uc},\overline{\uc}]$, it is set to $\underline{\uc}$ if $\uc + \uc_{\adj} < \underline{\uc}$ and to $\overline{\uc}$ if $\uc + \uc_{\adj} > \overline{\uc}$. Note that the update of $\uc$ cannot take place directly after a query $Q$ has been answered, but only before the next query selection, after the answer to $Q$ is known (and the elimination rate can be computed). 

\begin{proposition}\label{prop:RIO_not_consistent_with_DPR}
$\mathsf{RIO}$ is not consistent with the discrimination-preference relation $DPR$ (and hence does not satisfy the $DPR$).
\end{proposition}
\begin{proof}
The statement of this proposition holds since both cases in Eq.~\eqref{eq:Q_RIO} involve a query selection by $\mathsf{ENT}$ among a particular set of queries and since $\mathsf{ENT}$ is not consistent with the discrimination-preference relation by Proposition~\ref{prop:ENT_preserves_discrimination-pref-order}. 
\end{proof}
Let us now, in a similar fashion as done with $\mathsf{ENT}_z$, define $\mathsf{RIO}_z$ that selects the following query:
\begin{numcases}{Q_{\mathsf{RIO}_z} :=}
  Q_{\mathsf{ENT}_z} 						&\mbox{if} \; $\Pt(Q_{\mathsf{ENT}_z})\in\NHR_{\uc,\mD}$ 													\label{eq:Q_RIOz:case1} \\ 
  \argmin_{\setof{Q\,|\, \Pt(Q) \in \LC(\NHR_{\uc,\mD})}}(\mathsf{ENT}_z(Q))   			&\text{otherwise } 		\label{eq:Q_RIOz:case2}
\end{numcases}
\begin{proposition}\label{prop:RIO_with_ENTz_satisfies_DPR_for_restricted_query_set}
$\mathsf{RIO}_z$ 
with sufficiently large $z$ satisfies the discrimination-preference relation $DPR$ 
\begin{itemize}
	\item over $\mQ_\mD$ in the first case (Eq.~\eqref{eq:Q_RIOz:case1}).
	\item over $\setof{Q\,|\, \Pt(Q) \in \LC(\NHR_{\uc,\mD})}$ in the second case (Eq.~\eqref{eq:Q_RIOz:case2}).
\end{itemize}
\end{proposition}
\begin{proof}
Immediate from Corollary~\ref{cor:ENTz_satisfies_discrimination-pref_order_for_queries_with_p>=arbitrary_small_number} and the fact that $\mathsf{RIO}_z$ chooses the best query w.r.t.\ $\mathsf{ENT}_z$ over $\mQ_\mD$ in the first case (Eq.~\eqref{eq:Q_RIOz:case1}) and the best query w.r.t.\ $\mathsf{ENT}_z$ over $\setof{Q\,|\, \Pt(Q) \in \LC(\NHR_{\uc,\mD})}$ in the second case (Eq.~\eqref{eq:Q_RIOz:case2}).
\end{proof}
\begin{remark}\label{RIOz_not_satisfies_DPR}
Proposition~\ref{prop:RIO_with_ENTz_satisfies_DPR_for_restricted_query_set} does not imply that $\mathsf{RIO}_z$ generally satisfies the discrimination-preference relation over $\mQ_\mD$. To see this, realize that, given two queries $Q,Q'$ where $(Q,Q')\in DPR$, the cautiousness $\qc(\Pt(Q')) \leq \qc(\Pt(Q))$. This holds due to Definition~\ref{def:query_cautiousness} and Proposition~\ref{prop:construction_of_Q'_from_Q_if_Q_discrimination-preferred_over_Q'} which implies that $\min\setof{|\dx{}(Q')|,|\dnx{}(Q')|} \leq \min\setof{|\dx{}(Q)|,|\dnx{}(Q)|}$. Hence, $Q'$ might be among the least cautious non-high-risk queries, i.e.\ in $\setof{Q\,|\, \Pt(Q) \in \LC(\NHR_{\uc,\mD})}$, and $Q$ might not. Given that additionally, e.g.\ $p(\dx{}(Q)) = p(\dnx{}(Q)) = \frac{1}{2}$, $Q'$ might be be selected by $\mathsf{RIO}_z$ and thus preferred to $Q$ (because $Q$ is simply not least cautious and therefore, as desired, not taken into consideration).\qed 
\end{remark}
\begin{remark}\label{RIOz_used_with_EMCa_2}
Thinkable variants of $\mathsf{RIO}$ other than $\mathsf{RIO}_z$ are to use $\mathsf{EMCa}_2$ or $\mathsf{MPS}'$, respectively, instead of $\mathsf{ENT}$ in Eq.~\eqref{eq:Q_RIO} (i.e.\ either replace all occurrences of $\mathsf{ENT}$ in $\mathsf{RIO}$ by $\mathsf{EMCa}_2$ or all occurrences by $\mathsf{MPS}'$). Both of these measures satisfy the discrimination-preference relation, unlike $\mathsf{ENT}$. Also $\mathsf{ENT}_z$ only does so for a large enough setting of $z$ which depends on the given set of queries from which an optimal one should be suggested. $\mathsf{EMCa}_2$ or $\mathsf{MPS}'$, on the other hand, satisfy the discrimination-preference relation anyway, without the need to adjust any parameters. 

It is however worth noting that whether $\mathsf{EMCa}_2$ or $\mathsf{MPS}'$ is used within the $\mathsf{RIO}$ measure is quite different. Whereas $\mathsf{EMCa}_2$ implies a very similar query preference order as $\mathsf{ENT}$ does (cf.\ Section~\ref{sec:QPartitionRequirementsSelection} and Table~\ref{tab:requirements_for_all_measures}), just putting more emphasize on the minimization of $|\dz{}(Q)|$, the employment of $\mathsf{MPS}'$ would, by definition, favor only queries of least possible cautiousness, or equivalently of highest possible risk, if such queries are available. In this case, the situation of $\mathsf{RIO}$ directly accepting the best query w.r.t.\ $\mathsf{MPS}'$ (first line in Eq.~\eqref{eq:Q_RIO}) becomes rather unlikely as this happens only if the current cautiousness parameter $\uc$ is not higher than $\frac{1}{|\mD|}$. Consequently, besides modifying the property of basically preferring queries with an optimal information gain, the usage of $\mathsf{MPS}'$ might also involve a higher average query computation time, as the (additional) second criterion (second line in Eq.~\eqref{eq:Q_RIO}) focusing only on non-high-risk queries tends to come into effect more often.

Independently of these thoughts, Proposition~\ref{prop:RIO_with_ENTz_satisfies_DPR_for_restricted_query_set} also holds for $\mathsf{RIO}$ using either of the two measures $\mathsf{EMCa}_2$ or $\mathsf{MPS}'$, see below.
\qed 
\end{remark}
\begin{proposition}\label{prop:RIO_with_EMCa_2_or_MPS'_satisfies_DPR_for_restricted_query_set}
$\mathsf{RIO}$ with every occurrence of $\mathsf{ENT}$ in Eq.~\eqref{eq:Q_RIO} replaced by exactly one of $\setof{\mathsf{EMCa}_2,\mathsf{MPS}'}$ satisfies the discrimination-preference relation $DPR$ 
\begin{itemize}
	\item over $\mQ_\mD$ in the first case (Eq.~\eqref{eq:Q_RIOz:case1}).
	\item over $\setof{Q\,|\, \Pt(Q) \in \LC(\NHR_{\uc,\mD})}$ in the second case (Eq.~\eqref{eq:Q_RIOz:case2}).
\end{itemize}
\end{proposition}
The theoretically optimal query w.r.t.\ to $\mathsf{RIO}$ and $\mathsf{RIO}_z$, respectively, is one that is optimal w.r.t.\ $\mathsf{ENT}$ and $\mathsf{ENT}_z$, respectively, and that is sufficiently cautious as prescribed by the current parameter $\uc$. 
\begin{proposition}\label{prop:theoretical_opt_wrt_RIO_RIOz}\leavevmode
Let the current parameter of $\mathsf{RIO}$ ($\mathsf{RIO}_z$) be $\uc$ and the set of leading diagnoses be $\mD$. Then, the theoretically optimal query w.r.t.\ $\mathsf{RIO}$ ($\mathsf{RIO}_z$) satisfies $|p(\dx{}(Q)) - p(\dnx{}(Q))| = 0$, $\dz{}(Q) = \emptyset$ and $\qc(Q) \geq \uc$.
\end{proposition}
Note that the third condition given in this proposition involves an inequality. This implies that there are $|\mD|-\lceil \uc \rceil+1$ different global optima for $\mathsf{RIO}$ and $\mathsf{RIO}_z$, respectively, i.e.\ one for each query cautiousness $\qc(Q) \in \setof{\lceil \uc \rceil,\dots,|\mD|}$.

\subsubsection{Summary}
\label{sec:Summary}
In the previous sections we have thoroughly analyzed the properties of various existing (Section~\ref{sec:ExistingActiveLearningMeasuresForKBDebugging}) and new (Section~\ref{sec:NewActiveLearningMeasuresForKBDebugging}) active learning measures with regard to their application as query selection measures in the interactive KB debugging. In this section, we summarize the obtained results. 

\paragraph{Discrimination-Preference Relation and Theoretical Optimum.} Table~\ref{tab:measures_satisfy_DPR_theoretical_opt_exists} reports on the extent the discussed measures, grouped by active learning frameworks they belong to, comply with the discrimination-preference relation (satisfaction of $DPR$: $\checkmark$; non-satisfaction of, but consistency with $DPR$: $(\checkmark)$; non-consistency with $DPR$: $\times$) and indicates whether there is a theoretical optimum ($\checkmark$) of a measure or not ($\times$). Note that $\mathsf{ENT}$ as well as $\mathsf{ENT}_z$ were not assigned to any of the active learning frameworks discussed throughout Section~\ref{sec:NewActiveLearningMeasuresForKBDebugging} since they do not seem to fit any of these sufficiently well. The measures $\mathsf{SPL}$ and $\mathsf{SPL}_z$, on the other hand, match the QBC framework pretty well with the additional assumption that a penalty of $z$ is awarded for each non-voting committee member. 

What becomes evident from Table~\ref{tab:measures_satisfy_DPR_theoretical_opt_exists} is that the only measures that satisfy the discrimination-preference relation $DPR$ without the need of appropriate additional conditions being fulfilled are $\mathsf{EMCa}_z$ for $z>2$, $\mathsf{SPL}_z$ for $z>1$ as well as $\mathsf{MPS}'$. That is, the order imposed by any of these measures over an arbitrary set of queries includes all preferences between query pairs $DPR$ includes. For $\mathsf{ENT}_z$ and fixed selection of $z$, all queries in the set of queries $\mQ$ of interest must feature a large enough discrepancy -- which depends on $z$ -- between the likeliness of positive and negative answers in order for $\mathsf{ENT}_z$ to satisfy the $DPR$ over $\mQ$. This fact is made precise by Corollary~\ref{cor:ENTz_satisfies_discrimination-pref_order_for_queries_with_p>=arbitrary_small_number}. Further, we have examined two measures, $\mathsf{SPL}$ and $\mathsf{MPS}$, which do not violate the $DPR$, i.e.\ do not imply any preferences between query pairs which are reciprocally stated by $DPR$, but are proven not to imply all preferences between query pairs stated by $DPR$. All other investigated measures are not consistent with the $DPR$ in general.

Concerning the theoretical optimum, i.e.\ the global optimum of the real-valued functions characterizing the measures, there are only two measures, $\mathsf{KL}$ and $\mathsf{EMCb}$, for which such a one does not exist. Responsible for this fact to hold is, roughly, that given any fixed query, we can construct another query with a better value w.r.t.\ these measures, theoretically. On the practical side, however, we were still able to derive criteria -- which are, remarkably, almost identical for both measures, albeit these belong to different active learning frameworks -- the best query among a number of queries must comply with. Nevertheless, a drawback of these criteria is that they are non-deterministic in that they still leave open a number of different shapes the best query might have and that must be searched through in order to unveil this sought query. The reason why we studied the theoretical optima of the measures is that they provide us with a valuable insight into the properties that render queries good or bad w.r.t.\ the respective measure. And, as we have seen, these properties are often not extractable in a very easy way just from the given quantitative function. We use these deduced properties in the next section to pin down qualitative requirements to optimal queries w.r.t.\ all the discussed measures. These qualitative criteria will be an essential prerequisite to the successful implementation of heuristics to guide a systematic search for optimal queries. We will delve into this latter topic in Section~\ref{sec:FindingOptimalQPartitions}.

\paragraph{Superiority Relation.} We were also studying the relationship between different measures in terms of superiority. The superiority relation $\prec$ (cf.\ Definition~\ref{def:measures_equivalent_theoretically-optimal_superior}) is defined on the basis of the discrimination-preference relation $DPR$ and holds between two measures $m_1$ and $m_2$ (i.e.\ $m_1$ is superior to $m_2$) iff the intersection of $DPR$ with the (preference) relation imposed on queries by $m_1$ is a proper superset of the intersection of $DPR$ with the (preference) relation imposed on queries by $m_2$. 
Figure~\ref{fig:superiority_relation_graph} summarizes all superiority relationships between measures we discovered. Note this figure raises no claim to completeness as there might be further valid relationships not depicted which we did not detect. The figure is designed in a way that higher compliance with the $DPR$ means a higher position of the node representing the respective measure in the graph. We can read from the figure that all the measures that satisfy the $DPR$ (framed nodes in the graph) are superior to all others. In fact, this must be the case due to Proposition~\ref{prop:if_m1_satisfies_DPR_and_m2_does_not_then_m1_superior_to_m2}. Please observe that the superiority relation is transitive (witnessed by Proposition~\ref{prop:superiority_is_strict_order}). This implicates that, given a directed path between to measures, the source measure of this path is inferior to the sink measure of the path. That is, for instance, the $\mathit{GiniIndex}$ is inferior to $\mathsf{MPS}'$ although there is no direct edge between both measures. Interestingly, none of the measures $\mathsf{SPL}$ (only consistent with $DPR$) and $\mathsf{ENT}$ (inconsistent with $DPR$ in general) so far used in KB debugging satisfy the $DPR$. Hence, for each of these, there are other measures superior to them. That said, for each of $\mathsf{SPL}$ and $\mathsf{ENT}$, there is a superior measure which favors queries with the very same properties, just that it additionally obeys the discrimination-preference order. For $\mathsf{SPL}$, one of such measures is e.g.\ $\mathsf{SPL}_2$. In fact, all measures $\mathsf{SPL}_z$ with a parametrization $z>1$ are candidates to improve $\mathsf{SPL}$ regarding the fulfillment of the $DPR$, and hence are superior to $\mathsf{SPL}$. As to $\mathsf{ENT}$, one measure involving the optimization of the same query features, just with a higher penalization of features that can lead to a violation of the $DPR$, is e.g.\ $\mathsf{EMCa}_2$. Actually, all measures $\mathsf{EMCa}_z$ with an arbitrary setting of $z\geq 2$ are superior to $\mathsf{SPL}$. This adherence to the same strategy of query selection of $(\mathsf{SPL},\mathsf{SPL}_{z\,(z>1)})$ and $(\mathsf{ENT},\mathsf{EMC}_{z\,(z\geq 2)})$ with more appropriate prioritization of properties achieved by the superior measure finds expression also in the derived requirements to optimal queries which we will discuss in Section~\ref{sec:QPartitionRequirementsSelection} (cf.\ Table~\ref{tab:requirements_for_all_measures}).

\paragraph{Equivalence Relations.} Finally, Table~\ref{tab:measure_equiv_classes} gives information about equivalence relations between measures, in particular about the equivalence classes w.r.t.\ the relations $=$ (two measures optimize exactly the same function) , $\equiv$ (two measures impose exactly the same preference order on any set of queries) and $\equiv_{\mathsf{\mQ}}$ (two measures impose exactly the same preference order on the query set $\mQ$) where $\mQ$ is perceived to be any set of queries comprising only queries with an empty $\dz{}$-set. It is important to recognize that the derived equivalences between measures do hold as far as the KB debugging setting is concerned, but must not necessarily hold in a different setting, e.g.\ an active learning setting where the learner is supposed to learn a classifier with a discrete class variable with domain of cardinality $3$ or higher. 
Further on, note that the higher the number of the row in the table, the more coarse grained is the partitioning of the set of all measures given by the equivalence classes imposed by the relations. This results from the fact that if two measures are equal, they must be also equivalent, and if they are equivalent, they must be equivalent under any assumption of underlying set of queries.
This table tells us that, up to equality, there are $18$ different measures we discussed, each with a different quantitative function (first vertical section of the table). If the preference order imposed by the measures is taken as the criterion, then there remain $15$ different equivalence classes of measures (second section of the table). We observe that, compared with the relation $=$, the classes $\setof{\mathsf{LC}}$, $\setof{\mathsf{M}}$, $\setof{\mathsf{H}}$ and $\setof{\mathsf{EMCa}_0,\mathit{GiniIndex}}$ as well as the classes $\setof{\mathsf{VE}}$ and $\setof{\mathsf{SPL}_0}$ (in the first section only implicitly listed in terms of $\setof{\mathsf{SPL}_{z\,(z\neq 1)}}$) join up when considering $\equiv$. That is, these four and two classes of different functions lead to only two different query selection measures in the context of our KB debugging scenario. Comparing the third section of the table with the second, we see various coalescences of classes, indicating that, over queries with empty $\dz{}$-set, essentially different measures suddenly constitute equivalent measures. In particular, measures motivated by the maximization of \emph{information gain}, the maximization of uncertainty about query answers (\emph{uncertainty sampling}), the maximization of the expected probability mass of invalidated (leading) diagnoses (\emph{expected model change, variant (a)}), and the $\mathsf{GiniIndex}$ conflate into one single equivalence class. On the other hand, \emph{vote entropy} and \emph{split-in-half} (with arbitrary $z$-parametrization) merge. Third, $\mathsf{MPS}'$ is no longer better (in terms of the superiority relation) than $\mathsf{MPS}$, because both are now equivalent. In fact, the order on the set of measures imposed by the superiority relation $\prec$ given by Figure~\ref{fig:superiority_relation_graph} collapses completely after restricting the set of arbitrary queries to only queries in $\mQ$, i.e.\ all edges in the graph depicted by Figure~\ref{fig:superiority_relation_graph} vanish. This is substantiated by Proposition~\ref{prop:if_Q_discrimination-preferred_over_Q'_then_dz(Q')_supset_dz(Q)} which implies that the discrimination-preference relation $DPR$ upon $\mQ$ must be the empty relation. So, overall, we count seven remaining classes of measures w.r.t.\ $\equiv_{\mQ}$.

\renewcommand{\arraystretch}{1.4}
\begin{table}[t!]
\small
	\centering
		\begin{tabular}{lccc}
\toprule
 framework 															& measure $m$  			& $DPR$							&   	theoretical optimum exists	 								\\
\midrule
\multirow{3}{*}{US} 										& 	$\mathsf{LC}$ 	&	 $\times$	 				&  			$\checkmark$	\\
																				& 	$\mathsf{M}$ 		&	 $\times$	 				&  			$\checkmark$			\\
																				& 	$\mathsf{H}$ 		&	 $\times$	 				&  				$\checkmark$		\\
\hline
\multirow{2}{*}{IG}											& 	$\mathsf{ENT}$ 	&$\times$/$\checkmark_{1)}$&	$\checkmark$		\\
																				& 	$\mathsf{ENT}_z$&$\times$/$\checkmark_{2)}$&	$\checkmark$		\\
\hline
\multirow{4}{*}{QBC}		 								& 	$\mathsf{SPL}$ 	&	 $(\checkmark)$		&	 $\checkmark$		\\
																				& 	$\mathsf{SPL}_z$&	 $\times_{(z<1)}$/$(\checkmark)_{(z=1)}$/$\checkmark_{(z>1)}$&	  	$\checkmark$			\\
																				& 	$\mathsf{VE}$ 	&	 $\times$					&  	$\checkmark$		\\
																				& 	$\mathsf{KL}$ 	&	 $\times$					&  	$\times$	\\
\hline
\multirow{6}{*}{EMC}										& 	$\mathsf{EMCa}$	&	 $\times$/$\checkmark_{3)}$	&  	$\checkmark$		\\
																				& 	$\mathsf{EMCa}_z$&	$\times_{(z< 2)}$/$\checkmark_{(z\geq 2)}$ 	&  	$\checkmark$			\\
																				& 	$\mathsf{EMCb}$	&	 $\times$					&  	$\times$		\\
																				& 	$\mathsf{MPS}$ 	&	 $(\checkmark)$		&  	$\checkmark$		\\
																				& 	$\mathsf{MPS}'$ 	&	 $\checkmark$		&  	$\checkmark$		\\
																				& 	$\mathsf{BME}$ 	&	 $\times$					&  	$\checkmark$		\\
\hline
\multirow{2}{*}{RL} 										& 	$\mathsf{RIO}$ 	&	 $\times$			 		&  	$\checkmark$		\\
																				& 	$\mathsf{RIO}_z$&	 $\times_{4)}$		&  	$\checkmark$		\\
\bottomrule
\end{tabular}
\renewcommand{\arraystretch}{1}

\vspace{5pt}

\begin{tabular}{lp{10cm}}
1):& In general, $\checkmark$ holds only over a set of queries $\mQ$ satisfying Corollary~\ref{cor:ENT_satisfies_discrimination-pref_order_for_queries_with_p>=1/5}. \\ 
2):& In general, $\checkmark$ holds only over a set of queries $\mQ$ and a parameter $z$ that 
 satisfy Corollary~\ref{cor:ENTz_satisfies_discrimination-pref_order_for_queries_with_p>=arbitrary_small_number}. \\ 
3):& In general, $\checkmark$ holds only over a set of queries $\mQ$ satisfying Corollary~\ref{cor:EMCa_satisfies_discrimination-pref_order_for_queries_with_p>=1/4}. \\  
4):& The parameter $z$ can be specified in a way such that the $DPR$ is met, but only over a restricted set of queries (as per Proposition~\ref{prop:RIO_with_ENTz_satisfies_DPR_for_restricted_query_set}).
		\end{tabular}
\caption[discrimination-Preference Relation and Theoretical Optima]{The table lists all the discussed query quality measures $m$ (second column), grouped by the respective active learning frameworks (first column) uncertainty sampling (US), information gain (IG), query-by-committee (QBC), expected model change (EMC) or reinforcement learning (RL) they belong to. 
The column $DPR$ signalizes whether the measure satisfies the $DPR$ (denoted by a $\checkmark$), is consistent with it and does not satisfy it (denoted by a bracketed $\checkmark$), or is not consistent with it (denoted by a $\times$). Subscript numbers such as $_{1)}$ or bracketed statements such as $_{(z>2)}$ give additional conditions (for the numbers: listed at the bottom of the table) that must be satisfied for the respective property to hold. The rightmost column reports whether ($\checkmark$) or not ($\times$) a theoretical optimum exists for the measure.} 
	\label{tab:measures_satisfy_DPR_theoretical_opt_exists}
\end{table}

\begin{figure}[tb]
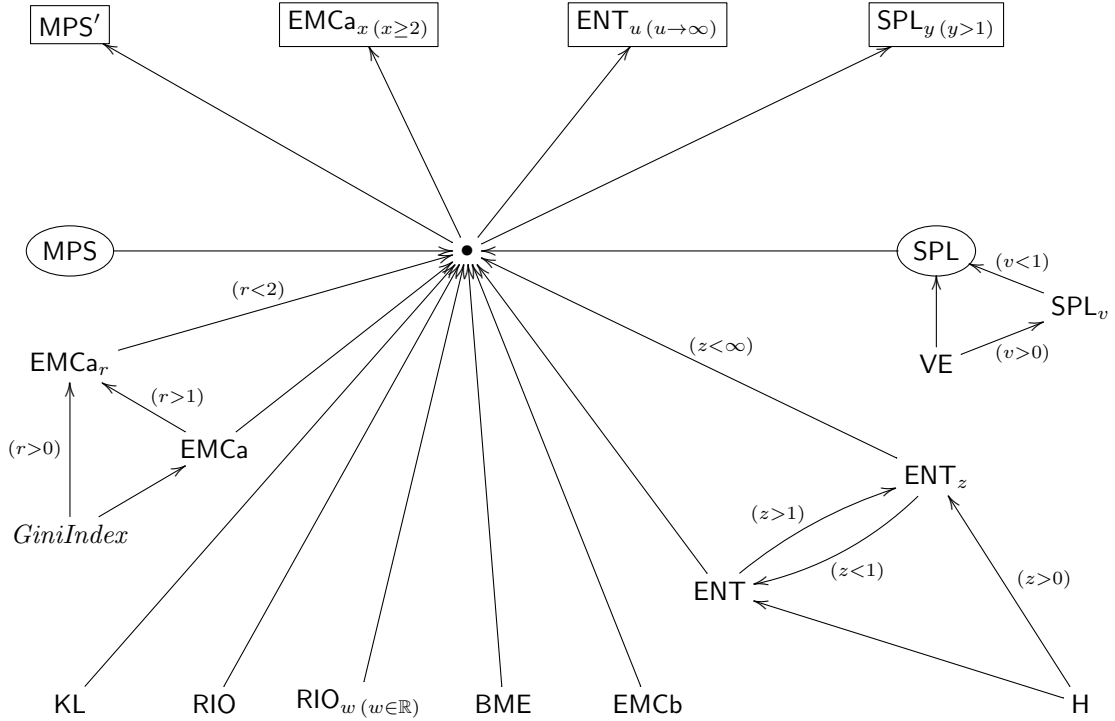

\xygraph{
!{<0cm,0cm>;<1.9cm,0cm>:<0cm,1.5cm>::}
!{(0,5)}*+[F]{\mathsf{MPS}'}="MPS'"
!{(2,5)}*+[F]{\mathsf{EMCa}_{x\,(x\geq 2)}}="EMCa_x"
!{(4,5)}*+[F]{\mathsf{ENT}_{u\,(u\to\infty)}}="ENT_u"
!{(6,5)}*+[F]{\mathsf{SPL}_{y\,(y > 1)}}="SPL_y"
!{(0,3)}*++[o][F]{\mathsf{MPS}}="MPS"
!{(6,3)}*++[o][F]{\mathsf{SPL}}="SPL"
!{(0,2)}*+{\mathsf{EMCa}_r}="EMCa_r"
!{(7,2.5)}*+{\mathsf{SPL}_v}="SPL_v"
!{(2.75,3)}*+{\bullet}="bullet"
!{(6,2)}*+{\mathsf{VE}}="VE"
!{(1,1.25)}*+{\mathsf{EMCa}}="EMCa"
!{(4.5,0)}*+{\mathsf{ENT}}="ENT"
!{(6,1)}*+{\mathsf{ENT}_z}="ENT_z"
!{(0,0.5)}*+{\mathit{GiniIndex}}="Gini"
!{(0,-1)}*+{\mathsf{KL}}="KL"
!{(1,-1)}*+{\mathsf{RIO}}="RIO"
!{(2,-1)}*+{\mathsf{RIO}_{w\,(w\in\mathbb{R})}}="RIO_w"
!{(3,-1)}*+{\mathsf{BME}}="BME"
!{(4,-1)}*+{\mathsf{EMCb}}="EMCb"
!{(7,-1)}*+{\mathsf{H}}="H"
"Gini":"EMCa"
"EMCa_r":"bullet"^{(r<2)}
"Gini":"EMCa_r"^{(r>0)}
"EMCa":"EMCa_r"_(0.4){(r>1)}
"EMCa":"bullet"
"MPS":"bullet"
"ENT":"bullet"
"bullet":"EMCa_x"
"bullet":"MPS'"
"bullet":"SPL_y"
"bullet":"ENT_u"
"VE":"SPL_v"_{(v>0)}
"VE":"SPL"
"SPL_v":"SPL"_{(v<1)}
"H":"ENT"
"H":"ENT_z"_{(z>0)}
"KL":"bullet"
"RIO":"bullet"
"RIO_w":"bullet"
"BME":"bullet"
"EMCb":"bullet"
"SPL":"bullet"
"ENT_z":@/^0.3cm/"ENT"^{(z<1)}
"ENT":@/^/"ENT_z"^(0.4){(z>1)}
"ENT_z":"bullet"_{(z<\infty)}
%
}
\caption[The Superiority Relationship Graph]{The graph in this figure depicts all the derived superiority relationships between the discussed measures. $m_1 \to m_2$ denotes that $m_2$ is superior to $m_1$, i.e.\ $m_2 \prec m_1$ (cf.\ Definition~\ref{def:measures_equivalent_theoretically-optimal_superior}). Arrows with a label correspond to conditional superiority relations that hold only if the condition given by the label is satisfied. Framed (circled) nodes indicate measures that satisfy (are consistent with) the discrimination-preference relation $DPR$. Nodes without frame or circle are (in general) not consistent with the $DPR$. Note that whenever possible we have included only one element of each equivalence class (see Table~\ref{tab:measure_equiv_classes}, row ``$\equiv$'') for simplicity, e.g.\ all relationships that are pictured for $\mathsf{H}$ are valid for $\mathsf{LC}$ or $\mathsf{M}$ as well. The bullet in the middle of the graph has no semantics other than allowing for a clearer representation of the graph by bundling all arrows that point to the measures that satisfy the $DPR$ (due to Proposition~\ref{prop:if_m1_satisfies_DPR_and_m2_does_not_then_m1_superior_to_m2}).}
\label{fig:superiority_relation_graph}
\end{figure}

\renewcommand{\arraystretch}{1.65}
\begin{table}
\footnotesize
	\centering
		\begin{tabular}{ll}
\toprule
relation					& 		equivalence classes			\\ 
\midrule
\multirow{3}{*}{$=$}								& $\setof{\mathsf{ENT}_1,\mathsf{ENT}}, \setof{\mathsf{ENT}_{z\,(z \neq 1)}}, \setof{\mathsf{SPL}_1, \mathsf{SPL}},\setof{\mathsf{SPL}_{z\,(z \neq 1)}}, \setof{\mathsf{RIO}_1, \mathsf{RIO}},\setof{\mathsf{RIO}_{z\,(z \neq 1)}},    $ \\
									& $\setof{\mathsf{EMCa}_1,\mathsf{EMCa}}, \setof{\mathsf{EMCa}_{z\,(z \notin \setof{0,1})}},\setof{\mathsf{EMCa}_0,\mathit{GiniIndex}},\setof{\mathsf{LC}},\setof{\mathsf{M}},\setof{\mathsf{H}},	$\\
									& $\setof{\mathsf{VE}},\setof{\mathsf{KL}},\setof{\mathsf{EMCb}},\setof{\mathsf{MPS}},\setof{\mathsf{MPS}'}, \setof{\mathsf{BME}}$ \\
\hline
\multirow{3}{*}{$\equiv$}					& $\setof{\mathsf{ENT}_1,\mathsf{ENT}}, \setof{\mathsf{ENT}_{z\,(z \notin\setof{0,1})}}, \setof{\mathsf{SPL}_1, \mathsf{SPL}},\setof{\mathsf{SPL}_{z\,(z \notin \setof{0,1})}}, \setof{\mathsf{RIO}_1, \mathsf{RIO}},\setof{\mathsf{RIO}_{z\,(z \neq 1)}},    $ \\
									& $\setof{\mathsf{EMCa}_1,\mathsf{EMCa}}, \setof{\mathsf{EMCa}_{z\,(z \notin \setof{0,1})}},\setof{\mathsf{EMCa}_0,\mathit{GiniIndex},\mathsf{LC},\mathsf{M},\mathsf{H}, \mathsf{ENT}_0},	$\\
									& $\setof{\mathsf{VE},\mathsf{SPL}_0},\setof{\mathsf{KL}},\setof{\mathsf{EMCb}},\setof{\mathsf{MPS}},\setof{\mathsf{MPS}'}, \setof{\mathsf{BME}}$ \\
\hline
\multirow{2}{*}{$\equiv_{\mQ}$}		&	 $ \setof{\mathsf{EMCa},\mathsf{EMCa}_{z\,(z\in\mathbb{R})},\mathit{GiniIndex},\mathsf{LC},\mathsf{M},\mathsf{H},\mathsf{ENT},\mathsf{ENT}_{z\,(z\in\mathbb{R})}}^{{\larger\textcircled{\smaller[2]1}}},\setof{\mathsf{SPL}, \mathsf{SPL}_{z\,(z\in\mathbb{R})},\mathsf{VE}}^{{\larger\textcircled{\smaller[2]2}}},$ \\
									&	$\setof{\mathsf{RIO},\mathsf{RIO}_{z\,(z\in\mathbb{R})}}^{{\larger\textcircled{\smaller[2]3}}},\setof{\mathsf{KL}}^{{\larger\textcircled{\smaller[2]4}}},\setof{\mathsf{EMCb}}^{{\larger\textcircled{\smaller[2]5}}}, \setof{\mathsf{MPS},\mathsf{MPS}'}^{{\larger\textcircled{\smaller[2]6}}}, \setof{\mathsf{BME}}^{{\larger\textcircled{\smaller[2]7}}}$														\\
\bottomrule
		\end{tabular}
\caption[Equivalence Classes of Measures]{Equivalence classes of measures w.r.t.\ the relations $\equiv$ and $\equiv_{\mQ}$ (cf.\ Definition~\ref{def:measures_equivalent_theoretically-optimal_superior}) and $=$ where $m_1=m_2$ iff $m_1$ and $m_2$ are represented by the same function. $\mQ \subseteq \mQ_{\mD,\tuple{\mo,\mb,\Tp,\Tn}_\RQ}$ denotes any set of queries where each query $Q \in \mQ$ satisfies $\dz{}(Q) = \emptyset$. 
The circled numbers in the third row have no semantics in this table and just provide an (arbitrary) numbering of the equivalence classes which enables easy reference to these classes. Any measure that is mentioned with a parameter $z$ without without assignment of $z$ to a fixed number, is to be understood as the measure resulting from a specification of $z$ to an arbitrary, but fixed number. That is, e.g.\ the class $\setof{\mathsf{ENT}_{z\,(z\neq 1)}}$ stands for infinitely many \emph{different} equivalence classes, one for each setting of $z$ over $\mathbb{R}\setminus\setof{1}$. Note that in case a measure with undefined $z$ parameter occurs in a class along with some other non-parameterized measure, this implies that for any concrete value of $z$, the measure must be in one and the same equivalence class. E.g.\ $\setof{\mathsf{SPL}, \mathsf{SPL}_{z\,(z\in\mathbb{R})},\mathsf{VE}}$ yields one and the same class for any setting of $z$ over $\mathbb{R}$.}
	\label{tab:measure_equiv_classes}
\end{table}

\subsection{Q-Partition Requirements Selection}
\label{sec:QPartitionRequirementsSelection}
In this section we enumerate and discuss the qualitiative requirements $r_m$ that can be derived for all the discussed quantitative query quality measures $m$. These requirements provide the basis for the construction of efficient search strategies for the optimal q-partition $\Pt$ that do not depend on a precalculation of a pool of q-partitions (cf.\ the discussion on page~\pageref{etc:discussion_quantitative_vs_qualitative_requirements}) and are used by the \textsc{qPartitionReqSelection} function in Algorithm~\ref{algo:query_comp} on page~\pageref{algo:query_comp}.

The function \textsc{qPartitionReqSelection} in Algorithm~\ref{algo:query_comp} is described by Table~\ref{tab:requirements_for_all_measures}. It outputs a set of requirements $r_m$ that must hold for a q-partition $\Pt$ to be optimal w.r.t.\ some quantitative query quality measure measure $m$ given as input. That is, the function can be imagined as a simple lookup in this table given a particular measure $m$ as input.
The table summarizes the results obtained in Sections~\ref{sec:ExistingActiveLearningMeasuresForKBDebugging} and \ref{sec:NewActiveLearningMeasuresForKBDebugging}. The rightmost column of the table gives the numbers of the respective Proposition (P) or Corollary (C), from which the requirements $r_m$ for some measure $m$ can be deduced in a straightforward way. In case the requirements are directly recognizable from the definition of a measure, this is indicated by a ``D'' in the last column of the table.

The table groups measures by their associated active learning frameworks (see Section~\ref{sec:Summary} as to why $\mathsf{ENT}$ and $\mathsf{ENT}_z$ are not allocated to any of the discussed frameworks). 
Drawing our attention to the ``requirements'' column of the table, roman numbers in brackets signalize the priority of requirements, e.g.\ $(\mathrm{I})$ denotes higher priority than $(\mathrm{II})$. In rows where no prioritization of requirements is given by roman numbers, either all requirements have equal weight (which holds e.g.\ for $\mathsf{SPL}$) or the prioritization depends on additional conditions. The latter is denoted by a (*) on the right of the respective row in the table. Such additional conditions might be e.g.\ the probability of query answers. For example, in the case of $\mathsf{ENT}$, for queries with a high difference ($\geq \frac{3}{5}$, cf.\ Proposition~\ref{prop:ENT_preserves_discrimination-pref-order}) between the probabilities of positive and negative answer, more weight is put on the minimization of $\left|p(\dx{}(Q)) - p(\dnx{}(Q))\right|$, whereas for queries with a lower difference ($<\frac{3}{5}$) the prioritization order of the requirements is tilted implying a higher weight on the minimization of $p(\dz{}(Q))$. 

Note further the disjunctive (``either...or'') and indeterministic character (``some'') of the requirements for $\mathsf{EMCb}$ and $\mathsf{KL}$. This comes due to the fact that there is no theoretical optimum for these two equivalence classes (cf.\ Table~\ref{tab:measures_satisfy_DPR_theoretical_opt_exists}). In the last two rows, $r_{\mathsf{ENT}}$ as well as $r_{\mathsf{ENT}_z}$ is a shorthand for the requirements $r_m$ for $m$ equal to $\mathsf{ENT}$ and $\mathsf{ENT}_z$, respectively, given in the IG section of the table. Other formalisms given in the last two rows are explained in detail in Section~\ref{sec:rio}.

If we compare the requirements for measures not satisfying the discrimination-preference relation $DPR$ with those for measures satisfying the $DPR$, we realize that the former (a)~do not consider or (b)~do not put enough weight on the minimization of the set $\dz{}(Q)$. As an example for (a), the requirements for both $\mathsf{H}$ and $\mathsf{VE}$ do not mention $\dz{}(Q)$ at all, i.e.\ disregard minimizing this set. The inevitable consequence of this is that these measures are not consistent with the $DPR$ (cf.\ Table~\ref{tab:measures_satisfy_DPR_theoretical_opt_exists}). Case (b) occurs e.g.\ with $\mathsf{SPL}_z$ for $z<1$. Here, more emphasize is put on the minimization of $\left| |\dx{}(Q)| - |\dnx{}(Q))| \right|$ than on the minimization of $|\dz{}(Q)|$, again resulting in an inconsistency w.r.t.\ the $DPR$ (cf.\ Figure~\ref{fig:superiority_relation_graph}).

On the other hand, all measures preserving the $DPR$, i.e.\ $\mathsf{SPL}_z$ for $z>1$, $\mathsf{MPS}'$, $\mathsf{ENT}_z$ for $z\to\infty$ as well as $\mathsf{EMCa}_z$ for $z\geq 2$, have a single highest priority requirement to the best query, namely to minimize $|\dz{}(Q)|$. Note that the postulation of minimizing $p(\dz{}(Q))$ is equivalent to the postulation of minimizing $|\dz{}(Q)|$ since all diagnoses have a positive probability (cf.\ page~\pageref{etc:prob_of_each_diag_must_be_greater_zero}).

Another aspect that catches one's eye is the fact that $\mathsf{KL}$ and $\mathsf{EMCb}$, albeit not equivalent as illustrated by Example~\ref{ex:KL_not_equiv_EMCb}, lead to the same set of requirements to the best query. However, we have to be aware and not conclude from this fact that both measures will always cause the selection of the same query. Responsible for that is the indeterministic character of the requirements set of these two measures, as dicussed above. It is more appropriate to see these same requirements as a sign that both measures need only search through the same subset of all given queries in order to find their individual -- and not necessarily equal -- optimum.

\renewcommand{\arraystretch}{1.4}
\begin{table}[tb]
\small
	\centering
		\begin{tabular}{lcl}
\toprule
 measure(s) $m$ & EC of $m$ w.r.t.\ $\equiv_{\mQ}$		&   requirements $r_m$  							\\
\midrule
 $\mathsf{ENT}$ & {\larger\textcircled{\smaller[2]1}} & $\left|p(\dx{}(Q)) - p(\dnx{}(Q))\right| \to \min$ \\ 	
 $\mathsf{SPL}$ & {\larger\textcircled{\smaller[2]2}} &	$\left| |\dx{}(Q)| - |\dnx{}(Q))| \right| \to \min$ \\
$\mathsf{RIO}$ & {\larger\textcircled{\smaller[2]3}} &	$(\mathrm{I})~\Pt(Q) \in \NHR_{\uc,\mD},\quad (\mathrm{II})~\left|\qc(Q) - \uc\right| \rightarrow \min,\quad (\mathrm{II	I})~r_{\mathsf{ENT}}$ \\
\multirow{2}{*}{$\mathsf{KL},\mathsf{EMCb}$} & \multirow{2}{*}{{\larger\textcircled{\smaller[2]4}},{\larger\textcircled{\smaller[2]5}}} &	either $p(\dx{}(Q)) \to \max$ for some fixed $|\dx{}(Q)| \in \setof{1,\dots,|\mD|-1}$ \\ 
															 &  																										&	or $p(\dnx{}(Q)) \to \max$ for some fixed $|\dnx{}(Q)| \in \setof{1,\dots,|\mD|-1}$ \\
$\mathsf{MPS}$ & {\larger\textcircled{\smaller[2]6}} &	$(\mathrm{I})~|\mD^*|=1$ where $\mD^*\in\setof{\dx{}(Q),\dnx{}(Q)}$, $(\mathrm{II})~p(\mD^*) \rightarrow \max$ \\
$\mathsf{BME}$ & {\larger\textcircled{\smaller[2]7}} &	$(\mathrm{I})~p(\mD^*)<0.5$ where $\mD^*\in\setof{\dx{}(Q),\dnx{}(Q)}$, $(\mathrm{II})~|\mD^*| \rightarrow \max$  \\																		
\bottomrule
		\end{tabular}
\caption[Qualitative Requirements for Equivalence Classes of Measures]{Requirements derived for the different equivalence classes (EC) of query quality measures w.r.t.\ the relation $\equiv_{\mQ}$ where $\mQ \subseteq \mQ_{\mD,\tuple{\mo,\mb,\Tp,\Tn}_\RQ}$ is a set of queries where $\dz{}(Q) = \emptyset$ for all $Q \in \mQ$. The leftmost column gives one representative of each equivalence class. The circled number in the middle column refers to the respective equivalence class given in row ``$\equiv_{\mQ}$'' in Table~\ref{tab:measure_equiv_classes}. Roman numbers in brackets signalize the priority of requirements, e.g.\ $(\mathrm{I})$ denotes higher priority than $(\mathrm{II})$. 
$r_{\mathsf{ENT}}$ in the third row is a shorthand for the requirements $r_m$ for $m = \mathsf{ENT}$ given in the first row. Other formalisms given in the third row are explained in Section~\ref{sec:rio}.}
	\label{tab:requirements_for_equiv_classes_of_measures_wrt_equiv_mQ}
\end{table}

Next, we analyze in which way requirements $r_m$ might be used to realize a search for an optimal q-partition w.r.t.\ the $\mathsf{RIO}$ and $\mathsf{RIO}_z$ measures. This will also give reasons for the prioritization of the requirements w.r.t.\ these measures used in Table~\ref{tab:requirements_for_all_measures}: 

Basically, we can think of two different methods for computing an optimal query w.r.t.\ $\mathsf{RIO}$ ($\mathsf{RIO}_z$). Before we describe them, note that we define ``optimal q-partition w.r.t.\ the $\mathsf{ENT}$ ($\mathsf{ENT}_z$) measure'' as some q-partition for which the requirements $r_{\mathsf{ENT}}$ ($r_{\mathsf{ENT}_z}$) deviate by not more than some threshold $t_{\mathsf{ent}}$ (cf.\ Section~\ref{sec:FindingOptimalQPartitions} and Algorithm~\ref{algo:dx+_best_successor}) from the theoretical optimum, which is why there may be multiple ``optimal'' q-partitions.

The first method works exactly according to the definition of $\mathsf{RIO}$ ($\mathsf{RIO}_z$), i.e.\ run a (heuristic) search $S_1$ using $r_{\mathsf{ENT}}$ ($r_{\mathsf{ENT}_z}$), check whether the returned q-partition $\Pt_1$ is a non-high-risk q-partition. If so, return $\Pt_1$ and stop. Otherwise execute another search $S_2$ that takes into consideration only least cautious non-high-risk q-partitions, and find among them the best q-partition $\Pt_2$ w.r.t.\ $r_{\mathsf{ENT}}$ ($r_{\mathsf{ENT}_z}$). Return $\Pt_2$. While the first search $S_1$ uses the requirements $r_m$ specified for $\mathsf{ENT}$ ($\mathsf{ENT}_z$) and has no specifics of $\mathsf{RIO}$ ($\mathsf{RIO}_z$) to take into account, the second one, $S_2$, 
strictly postulates the fulfillment of requirements $(\mathrm{I})~\Pt(Q) \in \NHR_{\uc,\mD}$ and $(\mathrm{II})~|\qc(Q) - \uc| \rightarrow \min$ (for explanations of these terms, see Section~\ref{sec:rio}) while trying to optimize $(\mathrm{III})~r_{\mathsf{ENT}} (r_{\mathsf{ENT}_z})$.
This approach involves carrying out a (heuristic) search $S_1$ and, if needed, another (heuristic) search $S_2$. Thence, this first method might be more time-consuming than the second one.

The second method, on the other hand, implements $\mathsf{RIO}$ in a way that only one search suffices. To this end, the first observation is that $\mathsf{RIO}$ will never accept a high-risk q-partition, cf.\ both lines of Eq.~\eqref{eq:Q_RIO} (as to $\mathsf{RIO}$) as well as  Eq.~\eqref{eq:Q_RIOz:case1} and \eqref{eq:Q_RIOz:case2} (regarding $\mathsf{RIO}_z$) where a query with $\Pt(Q) \in \NHR_{\uc,\mD}$ is postulated. That is, the primary (and necessary) requirement is given by $(\mathrm{I})~\Pt(Q) \in \NHR_{\uc,\mD}$. 
So, the search $S_1$ can be understood as a search for a q-partition with optimal $\mathsf{ENT}$ measure among all non-high-risk q-partitions (because any output high-risk q-partition will be discarded).
%
%
As explained above, $S_2$ searches for a q-partition with optimal $\mathsf{ENT}$ measure among all \emph{least-cautious} non-high-risk q-partitions.
Therefore, the least-cautious non-high-risk q-partitions can be simply considered first in the search, before other non-high-risk q-partitions are taken into account. In this vein, $(\mathrm{II})~|\qc(Q) - \uc| \rightarrow \min$ can be seen as the secondary criterion postulating to focus on least-cautious q-partitions first. This criterion will only be ``activated'' during $S_2$ and will be ``switched off'' during $S_1$. 

Virtually, this can be conceived as executing the search $S_1$ in a way its beginning simulates the search $S_2$. If this search, which is executed exactly as $S_2$, then locates a solution q-partition complying with the given threshold $t$, then this q-partition is at the same time a valid solution of $S_1$ and $S_2$ (since ``least cautious non-high-risk'' implies ``non-high-risk'') and the overall search can stop.
%
Otherwise, the best found (non-optimal) q-partition $\Pt_{lc}$ among the least-cautious non-high-risk q-partitions is stored and the search moves on towards more cautious q-partitions, i.e.\ search $S_1$ continues to iterate.
To this end, however, q-partitions representing nodes in the search tree at which the tree was pruned during $S_2$ must be stored in memory in order to enable a resumption of the search $S_1$ at these nodes at some later point.
If $S_1$ detects a solution q-partition complying with the given threshold $t$, this q-partition is returned. Otherwise, $\Pt_{lc}$ is output.

To summarize the second method, the search  
\begin{enumerate}
	\item starts seeking a q-partition meeting $(\mathrm{III})$ 
	among all q-partitions meeting $(\mathrm{I})$ and $(\mathrm{II})$, and, if there is a (sufficiently optimal, as per threshold $t$) result, returns this q-partition and stops. Otherwise, it
	\item stores the best found q-partition $\Pt_{lc}$ w.r.t.\ the $\mathsf{ENT}$ measure among all q-partitions meeting $(\mathrm{I})$ and $(\mathrm{II})$, and
	\item continues the search for a q-partition satisfying $(\mathrm{III})$ 
	among all q-partitions satisfying only $(\mathrm{I})$. If some (sufficiently optimal, as per threshold $t$) q-partition is found this way, it is returned and the search stops. Otherwise, the returned search result is $\Pt_{lc}$.
\end{enumerate}

\begin{remark}\label{rem:closed_RIO_function}
Taking as a basis the second method of computing a $\mathsf{RIO}$ query, we can express the $\mathsf{RIO}$ measure as a closed function. We call the measure characterized by this closed function $\mathsf{RIO}'$ to stress the non-equivalence to $\mathsf{RIO}$. Equivalence, however, holds under the assumption that all q-partitions that deviate by not more than some (arbitrarily small) $t_{\mathsf{ent}}>0$ from the theoretical optimum of $\mathsf{ENT}$ are considered $\mathsf{ENT}$-optimal. The closed-form $\mathsf{RIO}'$ is defined as 
\begin{align}
\mathsf{RIO}'(Q) &:= \frac{\mathsf{ENT}(Q)}{2}+\mD^*_{dev}    \label{eq:RIO'}
\end{align}
where
\begin{align*} 
\mD^*_{dev} &:= \begin{cases} 
\min\{|\dx{}(Q)|,|\dnx{}(Q)|\} - n 	& \mbox{if } \min\{|\dx{}(Q)|,|\dnx{}(Q)|\} \geq n \\ 
|\mD| 															& \mbox{otherwise.}
\end{cases}
\end{align*}
and $n := \lceil \uc |\mD| \rceil$ is the number of leading diagnoses that must be elimininated (in the worst case) by $Q$ as prescribed by the learned parameter $\uc$ of $\mathsf{RIO}$. To be precise, note that if $n > \frac{|\mD|}{2}$, then $n$ is set to $\frac{|\mD|}{2}-1$ (because no more than $\frac{|\mD|}{2}$ leading diagnoses can be eliminated in the worst case).

$\mathsf{RIO}'$ as given by Equation~\ref{eq:RIO'} in fact forces primarily the optimization of requirements $(\mathrm{I})~\Pt(Q) \in \NHR_{\uc,\mD}$ and $(\mathrm{II})~|\qc(Q) - \uc| \rightarrow \min$ before $(\mathrm{III})~r_{\mathsf{ENT}}$. 
Because, first, $\min\{|\dx{}(Q)|,|\dnx{}(Q)|\} \leq |\mD|-1$ and $n \geq 1$ we have that the value of $\mD^*_{dev}$ is always better given $\min\{|\dx{}(Q)|,|\dnx{}(Q)|\} \geq n$ than otherwise. That is, non-high-risk queries are always favorized as per requirement $(\mathrm{I})$. Second, $\mD^*_{dev}$ is lower the smaller $|\qc(Q) - \uc|$ is, i.e.\ $\mD^*_{dev}$ is better the more requirement $(\mathrm{II})$ is optimized. Third, $\mathsf{ENT}(Q) < 2$ for all queries $Q$ and hence $\frac{\mathsf{ENT}(Q)}{2} < 1$. Therefore, minimizing $\mD^*_{dev}$ by one improves $\mathsf{RIO}'$ by more than an improvement of a query's entropy measure $\mathsf{ENT}$ could ever do. This forces the optimization of requirement $(\mathrm{II})$ before requirement $(\mathrm{III})$.

To see why $\mathsf{ENT}(Q)<2$ must hold, assume an arbitrary query $Q$ and let $p_0, p_1, q \in (0,1)$ such that $p_0+p_1=1$, $q := 1- p(\dz{}(Q))$, $q p_1 := p(\dx{}(Q))$ and $q p_0 := p(\dnx{}(Q))$. Note, as required, $p(\dx{}(Q)) + p(\dnx{}(Q)) + p(\dz{}(Q)) = 1$. Further, $p(\dz{}(Q)) = 1-q$ and $p_1 = 1-p_0$. Then, using the representation of $\mathsf{ENT}$ as per Eq.~\ref{eq:scoring_funtion_dekleer}: 
\begin{align*}
\mathsf{ENT}(Q) &= \\
q p_1 \log_2 (q p_1) + q (1-p_1) \log_2 (q (1-p_1)) + (1-q) + 1&= \\
q \left[ p_1 (\log_2 q + \log_2 p_1) + (1-p_1) (\log_2 q + \log_2 (1-p_1))\right] + (1-q) + 1 &= \\
q \left[ (\log_2 q) + \left(p_1 \log_2 p_1 + (1-p_1) \log_2 (1-p_1)\right)\right] + (1-q) + 1 &= \\
q \left[ (c_1) + (c_2)\right] + (1-q) + 1 
\end{align*}
where $c_1 := \log_2 q < 0$ due to $q \in (0,1)$ and $c_2$ is the negative Shannon entropy of the (random) variable $p_1$ which attains its maximum at zero given some $p_1 \in \setof{0,1}$. As $p_1 \in (0,1)$, we have that $c_2 < 0$. Therefore, the squared brackets expression is smaller than zero -- also, after being multiplied by $q$. So, we have that $\mathsf{ENT}(Q)$ is equal to some expression smaller than zero plus $(1-q) + 1 = 2-q$ which is why $\mathsf{ENT}(Q) < 2$ must hold (due to $q > 0$).

Analogously, to overcome the $\mathsf{ENT}$-measure's non-satisfaction of the $DPR$ we can define a parameterized $\mathsf{RIO}'_z$ measure as 
\begin{align*}
\mathsf{RIO}'_z(Q) &:= \frac{\mathsf{ENT}_z(Q)}{2}+\mD^*_{dev}    
\end{align*}
where $\mD^*_{dev}$ is defined exactly as above.
\qed
\end{remark}

\begin{remark}\label{rem:search_method_for_RIO} 
As we shall see in Section~\ref{sec:FindingOptimalQPartitions}, both discussed approaches to tackle the finding of an optimal q-partition w.r.t.\ $\mathsf{RIO}$ will perfectly go along with the search method we will define, since the proposed search will naturally start with least-cautious q-partitions. 

However, note that the Algorithms~\ref{algo:dx+_update_best}, \ref{algo:dx+_opt}, \ref{algo:dx+_prune} and \ref{algo:dx+_best_successor} we will give in Section~\ref{sec:FindingOptimalQPartitions}, which characterize the behavior of our proposed heuristic q-partition search for the particular used query quality measure, will specify for $\mathsf{RIO}$ only the search $S_2$ (see above). That is, only an optimal q-partition meeting $(\mathrm{III})$ 
	among all q-partitions meeting $(\mathrm{I})$ and $(\mathrm{II})$ will be sought. The underlying assumption is that the first search method for $\mathsf{RIO}$ described above is adopted. This involves a prior run of the heuristic search algorithm using $\mathsf{ENT}$ as query quality measure and a test whether the output q-partition is a non-high-risk q-partition. Only if not so, then another run of the heuristic search using $\mathsf{RIO}$ as query quality measure is necessary.
	
	We want to emphasize that the implementation of the second search method for $\mathsf{RIO}$ described above would involve only slight changes to the given algorithms. Roughly, the following three modifications are necessary:
	\begin{enumerate}
		\item Specification of a second set of pruning conditions, optimality conditions, heuristics and update conditions regarding the currently best q-partition, which fit search $S_1$.
		\item Storage of all q-partitions at which search $S_2$ pruned the search tree.
		\item Storage of two currently best q-partitions, the first (i.e.\ $\Pt_{\mathsf{best}}$ in Algorithm~\ref{algo:dx+_part}) to memorize the best q-partition found using $S_2$ (i.e.\ the heursitic search using $\mathsf{RIO}$ we describe in Section~\ref{sec:FindingOptimalQPartitions}), and another one (call it $\Pt_{\mathsf{best},S_1}$) to memorize the best q-partition found afterwards using $S_1$ (i.e.\ the heursitic search using $\mathsf{RIO}$ and the second set of conditions and heuristics mentioned in (1.)). During $S_2$, \emph{only} $\Pt_{\mathsf{best}}$ is subject to change, whereas during $S_1$ \emph{only} $\Pt_{\mathsf{best},S_1}$ might be modified. 
	\end{enumerate}
	If $\Pt_{\mathsf{best}}$ is sufficiently good (w.r.t.\ the threshold $t$) after $S_2$, $\Pt_{\mathsf{best}}$ is returned and no resumption of the search ($S_1$) is necessary. Otherwise, if $\Pt_{\mathsf{best},S_1}$ is sufficiently good (w.r.t.\ the threshold $t$) after $S_2$ and $S_1$, then $\Pt_{\mathsf{best},S_1}$ is returned. Otherwise, $\Pt_{\mathsf{best}}$ is returned.
\qed
\end{remark}

Another material property of the search we will introduce in the next section will be that it automatically neglects q-partitions with a non-empty set $\dz{}$. Hence, any two measures in one and the same equivalence class in the section ``$\equiv_{\mQ}$'' of Table~\ref{tab:requirements_for_equiv_classes_of_measures_wrt_equiv_mQ} will yield exactly the same behavior of the q-partition finding algorithm. Moreover, the derived requirements (Table~\ref{tab:requirements_for_all_measures}) for some of the discussed measures can be simplified in this scenario. On this account, we provide with Table~\ref{tab:requirements_for_equiv_classes_of_measures_wrt_equiv_mQ} a summary of the resulting requirements $r_m$, for each equivalence class w.r.t.\ $\equiv_{\mQ}$. Exactly these requirements in Table~\ref{tab:requirements_for_equiv_classes_of_measures_wrt_equiv_mQ} will be the basis for the specification of \emph{pruning conditions} (Algorithm~\ref{algo:dx+_prune}), \emph{optimality conditions} (Algorithm~\ref{algo:dx+_opt}) and \emph{heuristics} to identify most promising successor q-partitions (Algorithm~\ref{algo:dx+_best_successor}) and of how to \emph{update the currently best q-partition} (Algorithm~\ref{algo:dx+_update_best}).

\renewcommand{\arraystretch}{1.4}
\begin{table}[H]
\small
	\centering
		\begin{tabular}{lclrl}
\toprule
 FW 															& measure $m$  			& requirements $r_m$ to optimal query $Q$ w.r.t.\ $m$  &	 			&	 						\\
\midrule
\multirow{3}{*}{US} 										& 	$\mathsf{LC}$ 	&	 $\left|p(\dx{}(Q)) - p(\dnx{}(Q))\right| \to \min$ &	 &	P\ref{prop:uncertainty_sampling} 		\\
																				& 	$\mathsf{M}$ 		&	 $\left|p(\dx{}(Q)) - p(\dnx{}(Q))\right| \to \min$ &	 &	P\ref{prop:uncertainty_sampling} 	\\
																				& 	$\mathsf{H}$ 		&	 $\left|p(\dx{}(Q)) - p(\dnx{}(Q))\right| \to \min$ &	 &	P\ref{prop:uncertainty_sampling} 	\\
\hline
\multirow{3}{*}{IG}											& 	$\mathsf{ENT}$ 	&  $\left|p(\dx{}(Q)) - p(\dnx{}(Q))\right| \to \min$, $p(\dz{}(Q))\to \min$ &	  (*) &	P\ref{prop:ent} 	\\
																				& 	\multirow{2}{*}{$\mathsf{ENT}_z$} & $z<\infty$: $\left|p(\dx{}(Q)) - p(\dnx{}(Q))\right| \to \min$, $p(\dz{}(Q))\to \min$ &	  (*) &	\multirow{2}{*}{P\ref{prop:ENT_preserves_discrimination-pref-order}} 	\\
																				&																			& $z\to\infty$: $\,(\mathrm{I})~p(\dz{}(Q))\to \min$, $(\mathrm{II})~\left|p(\dx{}(Q)) - p(\dnx{}(Q))\right| \to \min$ &	&	 	 \\
\hline
\multirow{6}{*}{QBC}		 								& 	$\mathsf{SPL}$ 	&	$\left| |\dx{}(Q)| - |\dnx{}(Q))| \right| \to \min$, $|\dz{}(Q)|\to \min$ &	 &  P\ref{prop:spl}\\
																				& 	\multirow{2}{*}{$\mathsf{SPL}_z$} & $z<1$: $\,(\mathrm{I})~\left| |\dx{}(Q)| - |\dnx{}(Q))| \right| \to \min$, $(\mathrm{II})~|\dz{}(Q)|\to \min$ &	 &	\multirow{2}{*}{P\ref{prop:SPL2_satisfies_DPR}} 	 \\
																				&																			& $z>1$: $\,(\mathrm{I})~|\dz{}(Q)|\to \min$, $(\mathrm{II})~\left| |\dx{}(Q)| - |\dnx{}(Q))| \right| \to \min$ &	&	 	 \\
																				& 	$\mathsf{VE}$ 	&		$\left| |\dx{}(Q)| - |\dnx{}(Q))| \right| \to \min$ &	 &	P\ref{prop:theoretically_opt_query_wrt_VE} 	\\
																				& 	\multirow{2}{*}{$\mathsf{KL}$} 	&	either $p(\dx{}(Q)) \to \max$ for some fixed $|\dx{}(Q)| \in \setof{1,\dots,|\mD|-1}$ &	&	\multirow{2}{*}{P\ref{prop:kl_opt}} 	\\ 
																				&																		& or $p(\dnx{}(Q)) \to \max$ for some fixed $|\dnx{}(Q)| \in \setof{1,\dots,|\mD|-1}$ &	 &	 	\\
\hline
\multirow{10}{*}{EMC}										& 	$\mathsf{EMCa}$	&	 $\left|p(\dx{}(Q)) - p(\dnx{}(Q))\right| \to \min$, $p(\dz{}(Q))\to \min$ &	  (*) &	 P\ref{prop:emca}	\\
																				& 	\multirow{2}{*}{$\mathsf{EMCa}_z$} & $z<2$: $\left|p(\dx{}(Q)) - p(\dnx{}(Q))\right| \to \min$, $p(\dz{}(Q))\to \min$ &	 (*) &	\multirow{2}{*}{C\ref{cor:EMCa_z_for_z>=2_is_better_than_ENTr_for_all_r}} 	\\
																				&																			& $z \geq 2$: $\,(\mathrm{I})~p(\dz{}(Q))\to \min$, $(\mathrm{II})~\left|p(\dx{}(Q)) - p(\dnx{}(Q))\right| \to \min$ &	&	 	  \\
																				& 	\multirow{2}{*}{$\mathsf{EMCb}$} 	&	either $p(\dx{}(Q)) \to \max$ for some fixed $|\dx{}(Q)| \in \setof{1,\dots,|\mD|-1}$ &	 &	\multirow{2}{*}{P\ref{prop:EMCb_opt}} 	\\ 
																				&																		& or $p(\dnx{}(Q)) \to \max$ for some fixed $|\dnx{}(Q)| \in \setof{1,\dots,|\mD|-1}$ &	 &	 	\\
																				& 	\multirow{2}{*}{$\mathsf{MPS}$} 	&	$(\mathrm{I})~|\mD^*|=1, |\dz{}(Q)|=0$ where $\mD^*\in\setof{\dx{}(Q),\dnx{}(Q)}$ &  &	\multirow{2}{*}{D} 	\\
																				& 																		& $(\mathrm{II})~p(\mD^*) \rightarrow \max$ &	 & \\ 
																				& 	\multirow{2}{*}{$\mathsf{MPS}'$} 	&	$(\mathrm{I})~|\dz{}(Q)|=0$, $(\mathrm{II})~|\mD^*|=1$ where $\mD^*\in\setof{\dx{}(Q),\dnx{}(Q)}$ &  &  \multirow{2}{*}{P\ref{prop:MPS'_satisfies_DPR}}  \\ 
																				&																			& $(\mathrm{III})~p(\mD^*) \rightarrow \max$ &	&	 	 \\ 
																				& 	\multirow{2}{*}{$\mathsf{BME}$} 	&	$(\mathrm{I})~p(\mD^*)<p(\mD^{**})$ where $\mD^*,\mD^{**}\in\setof{\dx{}(Q),\dnx{}(Q)}$ & & \multirow{2}{*}{D} \\ 
																				&										& $(\mathrm{II})~|\mD^*| \rightarrow \max$ 																								&	& \\
\hline
\multirow{2}{*}{RL} 										& 	$\mathsf{RIO}$ 	&	 $(\mathrm{I})~\Pt(Q) \in \NHR_{\uc,\mD},\quad (\mathrm{II})~\left|\qc(Q) - \uc\right| \rightarrow \min,\quad (\mathrm{II	I})~r_{\mathsf{ENT}}$ &	 &	\multirow{2}{*}{P\ref{prop:theoretical_opt_wrt_RIO_RIOz}} 	 \\
																				& 	$\mathsf{RIO}_z$&	 $(\mathrm{I})~\Pt(Q) \in \NHR_{\uc,\mD},\quad (\mathrm{II})~\left|\qc(Q) - \uc\right| \rightarrow \min,\quad (\mathrm{II	I})~r_{\mathsf{ENT}_z}$ &	 &	 	 \\
\bottomrule
		\end{tabular}
\caption[Qualitative Requirements for All Measures]{Requirements derived from the analyses conducted in Sections~\ref{sec:ExistingActiveLearningMeasuresForKBDebugging} and \ref{sec:NewActiveLearningMeasuresForKBDebugging} for the different query quality measures, grouped by active learning frameworks FW (see description of Table~\ref{tab:measures_satisfy_DPR_theoretical_opt_exists}). Roman numbers in brackets signalize the priority of requirements, e.g.\ $(\mathrm{I})$ denotes higher priority than $(\mathrm{II})$. In rows where no prioritization of requirements is given, 
either all requirements have equal weight or the prioritization depends on additional conditions, denoted by (*). 
$r_{\mathsf{ENT}}$ as well as $r_{\mathsf{ENT}_z}$ in the last two rows is a shorthand for the requirements $r_m$ for $m$ equal to $\mathsf{ENT}$ and $\mathsf{ENT}_z$, respectively, given in the IG section of the table. Other formalisms given in the last two rows are explained in Section~\ref{sec:rio}. The abbreviations P$i$, C$j$ and D at the far right mean that the requirements in the respective row can be derived from Proposition~$i$, Corollary~$j$ or directly from the definition of the measure itself, respectively.} 
	\label{tab:requirements_for_all_measures}
\end{table}

\subsection{Finding Optimal Q-Partitions}
\label{sec:FindingOptimalQPartitions}
In this section we want to discuss the implementation of the function \textsc{findQPartition} of Algorithm~\ref{algo:query_comp} which searches for an optimal q-partition 
w.r.t.\ some query quality measure $m$ given the optimality requirements $r_m$, a probability measure $p$, a set of leading diagnoses $\mD$ and a threshold $t_m$ as input. Note, all we can expect from \textsc{findQPartition} is to compute (an approximation of) the best \emph{existing} q-partition w.r.t.\ the leading diagnoses $\mD$. In fact, depending on the used $m$, $p$ and the number of diagnoses in $\mD$, this best existing q-partition might exhibit a significantly worse value regarding $m$ than the theoretically optimal q-partition. Nevertheless, we will determine the behavior or goal of the search by specifying $t_m$ in relation to the \emph{theoretical} optimum of $m$, i.e.\ the requirements $r_m$, because before the search is executed we usually will not know anything about the best existing q-partition w.r.t.\ $m$. Hence, given a too small threshold $t_m$ we might require the algorithm to find a q-partition that in fact does not exist. In this case the search will simply explore the entire search space and finally return the best found, and therefore the best existing, q-partition.

To realize \textsc{findQPartition}, we will derive systematic search methods. According to~\cite{russellnorvig2010}, a search problem is defined by an \emph{initial state}, a \emph{successor function} (that returns all direct successor states of a state), \emph{path costs} and a \emph{goal test}. In our case, a state corresponds to a q-partition $\Pt$ and a successor state to a q-partition $\Pt'$ resulting from $\Pt$ by minimal changes. Depending on the successor function $s$ used, the initial state is defined as one distinguished partition (not necessarily q-partition) of $\mD$ that enables the exploration of the entire search space for q-partitions by means of $s$. Concerning the finding of a goal q-partition, the path costs are immaterial as we are interested only in the shape of the resulting goal q-partition.
As hinted above, we will consider a q-partition as a goal state of the search if it meets the requirements $r_m$ to a sufficient degree. More specifically, we will 
regard a state as optimal if its ``distance'' to perfect fulfillment of $r_m$ is smaller than the optimality threshold $t_m$ (cf.~\cite{Shchekotykhin2012}). In the case of the $\mathsf{LC}$ measure (cf.\ page \pageref{etc:measure_desc_LC}), for example, a q-partition $\Pt = \langle \dx{}, \dnx{},\dz{}\rangle$ is viewed as optimal if $|p(\dx{}) - p(\dnx{})| \leq t_m$ (cf.\ Table~\ref{tab:requirements_for_all_measures}).

An informed (or heuristic) search problem additionally assumes some heuristic function being given which estimates the distance of a state to the next goal state. Such a heuristic function can be used to determine which successor state of a given state should be best visited next, thereby guiding the search faster towards a goal state. 
In order to define appropriate heuristic functions for $m$, we will exploit the results about the qualitative requirements $r_m$ for $m$ obtained in Section~\ref{sec:QPartitionRequirementsSelection}.


\subsubsection{Canonical Queries and Q-Partitions}\label{sec:can_query} 
In the context of Algorithm~\ref{algo:query_comp}, the \textsc{findQPartition} function aims at identifying a q-partition with postulated properties, such that a query associated with this q-partition can be extracted \emph{in the next step} by function \textsc{selectQueryForQPartition}. Consequently, realizing the fact that there can be (exponentially) many queries for one and the same q-partition (see Corollary~\ref{cor:upper_lower_bound_for_canonical_q-partitions} later, on page~\pageref{cor:upper_lower_bound_for_canonical_q-partitions}), 
and the fact that the construction respectively verification of a q-partition requires a query (cf. Definition~\ref{def:q-partition}), it becomes evident that the specification of some well-defined ``canonical'' query is desirable for this purpose.

\paragraph{Search-Space Minimization.} 
Furthermore, in order to devise a time and space saving search method, the potential size of the search tree needs to be minimized. To this end, in case of q-partitions, a key idea is to omit those q-partitions in the search that are proven suboptimal.
One such class of suboptimal q-partitions comprises those that have weak \emph{discriminative power} in the sense that they do not discriminate between all leading diagnoses. In other words, there is a non-empty set of (leading) diagnoses that queries associated with these q-partitions do not allow to eliminate \emph{in any case}. That is, these suboptimal q-partitions are exactly the q-partitions with the property $\dz{} \neq \emptyset$ as no answer to a query $Q$ 
can invalidate any diagnosis in $\dz{}(Q)$ (cf.\ Proposition~\ref{prop:properties_of_q-partitions},(\ref{prop:properties_of_q-partitions:enum:dx_dnx_dz_contain_exactly_those_diags_that_are...}.)). However, a query should be as constraining as possible such that a maximum number of (leading) diagnoses can be eliminated by it. So, it is desirable that the number of diagnoses that are elements of $\dz{}(Q)$, i.e. consistent with both query outcomes, is minimal, zero at the best.
As we shall see later, the search method we present will automatically neglect all q-partitions in this suboptimal class.

Recalling that all discrimination-dispreferred queries, i.e.\ those for which there is another query that is discrimination-preferred over them (cf.\ Definition~\ref{def:measures_equivalent_theoretically-optimal_superior}), must also have a non-empty $\dz{}$-set by Proposition~\ref{prop:if_Q_discrimination-preferred_over_Q'_then_dz(Q')_supset_dz(Q)}, the fact that the proposed search method ignores the abovementioned suboptimal class of q-partitions means that it \emph{automatically turns any query quality measure into one that satisfies the discrimination-preference relation} $DPR$. 
This is a simple consequence of Definition~\ref{def:measures_equivalent_theoretically-optimal_superior} and Proposition~\ref{prop:if_Q_discrimination-preferred_over_Q'_then_dz(Q')_supset_dz(Q)}.

\paragraph{Explicit-Entailments Queries.} Let an entailment $\alpha$ of a set of axioms $X$ be called \emph{explicit} iff $\alpha \in X$, \emph{implicit} otherwise.\label{etc:def_explicit_implicit_entailment} Then, a natural way of specifying a canonical query and at the same time achieving a neglect of suboptimal q-partitions of the mentioned type is to define a query as an adequate subset of all common explicit entailments implied by all KBs $(\mo \setminus \md_i)$ for which $\md_i$ is an element of some fixed set $\dx{}(Q) \subset \mD$. We call a query $Q \subseteq \mo$ an \emph{explicit-entailments query}. The following proposition witnesses that no explicit-entailments query can have a suboptimal q-partition in terms of $\dz{}(Q) \neq \emptyset$.
\begin{proposition}\label{prop:d0}
Let $\langle\mo,\mb,\Tp,\Tn\rangle_\RQ$ be a DPI, $\mD \subseteq \minD_{\langle\mo,\mb,\Tp,\Tn\rangle_\RQ}$ and $Q$ a query in $\mQ_\mD$. If $Q \subseteq \mo$, then $\dz{}(Q) = \emptyset$ holds.
\end{proposition}
\begin{proof}
We have to show that for an arbitrary diagnosis $\md_i \in \mD$ either $\md_i\in\dx{}(Q)$ or $\md_i\in\dnx{}(Q)$. Therefore two cases must be considered: (a)~$\mo\setminus\md_i \supseteq Q$ and (b)~$\mo\setminus\md_i \not\supseteq Q$. In case (a), by the fact that the entailment relation is extensive for $\mathcal{L}$, $\mo\setminus\md_i \models Q$ and thus, by monotonicity of $\mathcal{L}$, $\mo^{*}_i = (\mo\setminus\md_i) \cup \mb \cup U_P \models Q$. So $\md_i \in \dx{}(Q)$. In case (b) there exists some axiom $\tax \in Q \subseteq \mo$ such that $\tax\notin \mo\setminus \md_i$, which means that $(\mo \setminus \md_i) \cup Q \supset (\mo\setminus\md_i)$. From this we can derive that $\mo^{*}_i \cup Q$ must violate some $r\in\RQ$ or $\tn\in\Tn$ by the subset-minimality property of each diagnosis in $\mD$, in particular of $\md_i$. Hence, $\md_i \in \dnx{}(Q)$.
\end{proof}
The proof of Proposition~\ref{prop:d0} exhibits a decisive advantage of using explicit-entailments queries for the construction of q-partitions during the search. This is the opportunity to use set comparison instead of reasoning, which means that it must solely be determined whether $\mo\setminus\md_i \supseteq Q$ or not for each $\md_i\in\mD$ in order to build the q-partition $\Pt(Q)$ associated with some query $Q$. More specifically, given an explicit-entailments query $Q$, it holds that $\md \in \dx{}(Q)$ if $\mo\setminus\md \supseteq Q$ and $\md \in \dnx{}(Q)$ if $\mo\setminus\md \not\supseteq Q$. Using this result, contrary to existing works, queries and q-partitions can be computed without any reasoner calls. In fact, our approach permits to employ a reasoner optionally (in the \textsc{enrichQuery} function of Algorithm~\ref{algo:query_comp}) to enrich \emph{only one} already selected explicit-entailments query by additional (non-explicit) entailments (see Section~\ref{sec:EnrichmentOfAQuery}).  

\begin{example}\label{ex:explicit_ents_query}
Recall the propositional example DPI $\tuple{\mo,\mb,\Tp,\Tn}_\RQ$ given by Table~\ref{tab:example_dpi_0}. Let the leading diagnoses 
\begin{align} \label{eq:ex_explicit_ents_query:leading_diags}
\mD &= \setof{\md_1,\md_2,\md_3} = \{\{\tax_2,\tax_3\},\{\tax_2,\tax_5\},\{\tax_2,\tax_6\}\} \quad \mbox{(cf.\ Table~\ref{tab:min_diagnoses_example_dpi_0})}
\end{align}
Then a query w.r.t.\ $\mD$ and $\tuple{\mo,\mb,\Tp,\Tn}_\RQ$ is given, for instance, by $Q_1 = \setof{\tax_1,\tax_4,\tax_6,\tax_7} = (\mo\setminus\md_1) \cap (\mo\setminus \md_2)$, i.e.\ $Q_1$ is the set of common explicit entailments of $\mo\setminus\md_i$ for $i \in \setof{1,2}$. As $Q_1 \subset \mo$ holds, $Q_1$ is an explicit-entailments query. To verify that it indeed satisfies the sufficient and necessary criteria $\dx{}(Q_1) \neq \emptyset$ and $\dnx{}(Q_1) \neq \emptyset$ for a query as per Proposition~\ref{prop:properties_of_q-partitions},(\ref{prop:properties_of_q-partitions:enum:for_each_q-partition_dx_is_empty_and_dnx_is_empty}.), it suffices to test whether $\mo \setminus \md_3 \not\supseteq Q_1$. Plugging in the sets we indeed obtain $\setof{\tax_1,\tax_3,\tax_4,\tax_5,\tax_7} \not\supseteq \setof{\tax_1,\tax_4,\tax_6,\tax_7}$ due to $\tax_6$. In this vein, we could determine the q-partition $\Pt(Q_1) = \tuple{\setof{\md_1,\md_2},\setof{\md_3},\emptyset}$ of $Q_1$ by relying just on set comparisons and operations and without the need to use any reasoning service. Note that $Q'_1 := \setof{\tax_6}$ is a query with the same discrimination properties w.r.t.\ $\mD$ as $Q_1$. Hence, one would prefer to ask the user $Q'_1$ instead of $Q_1$ since answering $Q'_1$ (one sentence to assess) involves less effort than answering $Q_1$ (four sentences to check). A general method of minimizing explicit-entailments queries while at the same time (1) leaving the q-partition (i.e.\ the query's discrimination properties) unchanged and (2) guaranteeing other query properties such as best comprehensibility for the user will be discussed in Section~\ref{sec:FindingOptimalQueriesGivenAnOptimalQPartition}.

On the other hand, $Q_2 = \setof{\tax_2}$ as well as $Q_3 = \setof{\tax_1,\tax_4,\tax_7}$ are no queries w.r.t.\ $\mD$ and $\tuple{\mo,\mb,\Tp,\Tn}_\RQ$. The reason is that the former is not a subset of any $\mo\setminus\md_i$ for $i \in \setof{1,2,3}$, i.e.\ $\Pt(Q_2) = \tuple{\emptyset,\mD,\emptyset}$ is not a q-partition (cf.\ Proposition~\ref{prop:properties_of_q-partitions},(\ref{prop:properties_of_q-partitions:enum:for_each_q-partition_dx_is_empty_and_dnx_is_empty}.)), and that the latter is a subset of all $\mo\setminus\md_i$ for $i \in \setof{1,2,3}$, i.e.\ $\Pt(Q_3) = \tuple{\mD,\emptyset,\emptyset}$ is not a q-partition.
\qed
\end{example}

What we did not address yet is which subset of all explicit entailments of some set of KBs $\mo_i^*$ to select as canonical query. For computation purposes, i.e.\ to make set comparisons more efficient, the selected subset should have as small a cardinality as possible, but the minimization process should not involve any significant computational overhead. Recall from our explanations of Algorithm~\ref{algo:query_comp} at the beginning of Section~\ref{sec:QueryComputation} that the actual computation of a minimal (optimal) query associated with a selected q-partition should take place only in the course of the function \textsc{selectQueryForQPartition}. A helpful tool in our analysis will be the notion of a justification, as defined next:

\begin{definition}[Justification]\label{def:justification}\cite{Kalyanpur.Just.ISWC07} 
Let $\mo$ be a KB and $\alpha$ an axiom, both over $\mathcal{L}$.
Then $J\subseteq\mo$ is called a \emph{justification for $\alpha$ w.r.t.\ $\mo$}, written as $J \in \mathsf{Just}(\alpha, \mo )$, iff $J\models\alpha$ and for all $J' \subset J$ it holds that $J'\not\models\alpha$. 
\end{definition}

\begin{example}\label{ex:justification}
Consider the KB $\mo \cup \mb \cup U_\Tp = \setof{\tax_1,\dots,\tax_7} \cup \setof{\tax_8,\tax_9} \cup \setof{\tp_1}$, as given by our running example DPI (Table~\ref{tab:example_dpi_0}). We have that $\mo \cup \mb \cup U_\Tp \models \lnot M \lor A =: \alpha$. The set $\mathsf{Just}(\alpha,\mo \cup \mb \cup U_\Tp)$ is then given by $\setof{\setof{\tax_2,\tax_7,\tax_8,\tp_1},\setof{\tax_3,\tax_5,\tax_6,\tax_7,\tax_9}}$. For an explanation why this holds, consult Example~\ref{ex:min_conflict_sets}. In fact, the minimal conflict sets w.r.t.\ a DPI are strongly related to justifications of negative test cases or inconsistencies and incoherencies, respectively. An exact formalization of this relationship is provided in \cite[Sec.~4.2]{Rodler2015phd}.  

If we investigate the justifications for $\alpha := C \to E$, we obtain $\mathsf{Just}(\alpha',\mo \cup \mb \cup U_\Tp)$ as the singleton $\setof{\setof{\tax_5,\tax_6,\tax_9}}$, i.e.\ there is only one explanation for $\alpha'$ in $\mo \cup \mb \cup U_\Tp$. The set $\setof{\tax_5,\tax_6,\tax_9}$ is a justification for $\alpha'$ since $\tax_6 = C \to B$, $\tax_9 \equiv B \to K$ and $\tax_5 = K \to E$, and eliminating any of these sentences from the set clearly implies that the reduced set does not entail $\alpha'$. 

Note also that $\mathsf{Just}(\alpha,\mo) = \mathsf{Just}(\alpha',\mo) = \emptyset$. That is, without the background KB and the union of all positive test cases, the example KB $\mo$ does not entail $\alpha$ and $\alpha'$ at all. Since $\alpha$ is equivalent to the negative test case $\tn_1$, i.e.\ a sentence that must not hold in the intended domain, this shows the importance and power of putting a KB into a context with already approved knowledge (about the domain). So, the standalone KB $\mo$ already includes some faults which however come to light only after using the KB within a certain context.\qed
\end{example}

\paragraph{Properties of (General) Queries.} Before analyzing the properties of the special class of explicit-entailments queries, we give some general results about queries (that possibly also include implicit entailments) by the subsequently stated propositions. Before formulating the propositions, we give a simple lemma necessary for the proof of some of these.
\begin{lemma}\label{lem:U_D_notin_B_and_notin_U_P}
For any DPI $\langle\mo,\mb,\Tp,\Tn\rangle_\RQ$ and $\mD \subseteq \minD_{\langle\mo,\mb,\Tp,\Tn\rangle_\RQ}$ it holds that $(\mo \cup \mb \cup U_\Tp) \setminus U_\mD = (\mo\setminus U_\mD) \cup \mb \cup U_\Tp$.
\end{lemma}
\begin{proof}
We have to show that (i)~$U_\mD \cap \mb = \emptyset$ and (ii)~$U_\mD \cap U_\Tp = \emptyset$. 
Fact~(i) holds by $U_\mD \subseteq \mo$ and $\mo \cap \mb = \emptyset$ (cf.\ Definition~\ref{def:dpi}). Fact~(ii) is true since each sentence in $U_\mD$ occurs in at least one diagnosis in $\mD$. Let us call this diagnosis $\md_k$. So, the assumption of $U_\mD \cap U_\Tp \neq \emptyset$ implies that $\mo\setminus\md_k$ is invalid by the subset-minimality of $\md_k \in \mD$. This is a contradiction to $\md_k \in \minD_{\langle\mo,\mb,\Tp,\Tn\rangle_\RQ}$ (cf.\ Proposition~\ref{prop:notions_equiv}).
\end{proof}
The subsequently formulated proposition states necessary conditions a query $Q$ must satisfy. In particular, at least one sentence in $Q$ must be derivable from the extended KB $\mo \cup \mb \cup U_\Tp$ only in the presence of at least one axiom in $U_\mD$. Second, no sentence in $Q$ must be derivable from the extended KB only in the presence of the axioms in $I_\mD$. Third, $Q$ is consistent. 
\begin{proposition}\label{prop:discax}
Let $\mD \subseteq \minD_{\langle\mo,\mb,\Tp,\Tn\rangle_\RQ}$ with $|\mD|\geq 2$. Further, let $\dx{}(Q)$ be an arbitrary subset of $\mD$ and $E$ be the set of common entailments (of a predefined type) of all KBs in $\{\mo_{i}^*\,|\,\md_i \in \dx{}(Q)\}$.
Then, a set of axioms $Q\subseteq E$ is a query w.r.t. $\mD$ only if (1.), (2.)\ and (3.)\ hold: 
\begin{enumerate}
	\item There is at least one sentence $\tax \in Q$ for which $U_\mD \cap J \neq \emptyset$ for each justification $J\in\Just(\tax,\mo\cup\mb\cup U_\Tp)$.
	\item For all sentences $\tax \in Q$ there is a justification $J\in\Just(\tax,\mo\cup\mb\cup U_\Tp)$ such that $I_\mD \cap J = \emptyset$.
	\item $Q$ is (logically) consistent.
\end{enumerate}
\end{proposition}
\begin{proof}
Ad 1.: 
Let us assume that $Q$ is a query and for each sentence $\tax \in Q$ there is a justification $J\in\Just(\tax,\mo\cup\mb\cup U_\Tp)$ with $J \cap U_{\mD} = \emptyset$. So, using Lemma~\ref{lem:U_D_notin_B_and_notin_U_P} we have that $J \subseteq (\mo\setminus U_\mD) \cup \mb \cup U_\Tp$. Then $(\mo \setminus U_{\mD}) \cup \mb \cup U_\Tp\models \tax$ and since $\mo \setminus \md_i \supseteq \mo \setminus U_{\mD}$ for each $\md_i\in\mD$, we can deduce that (*)~$\mo_i^* = (\mo \setminus \md_i) \cup \mb \cup U_P \models \tax$ by monotonicity of $\mathcal{L}$. Thence, $\mo_i^*\models Q$ for all $\md_i\in\mD$ which results in $\dnx{}(Q) = \emptyset$ which in turn contradicts the assumption that $Q$ is a query. 

Ad 2.: 
Let us assume that $Q$ is a query and there is some sentence $\tax \in Q$ for which for all justifications $J\in\Just(\tax,\mo\cup\mb\cup U_\Tp)$ it holds that $J \cap I_{\mD} \neq \emptyset$. Moreover, it is true that $\mo_i^* \cap I_{\mD} = [(\mo \setminus \md_i) \cup \mb \cup U_\Tp] \cap I_{\mD} = \emptyset$ for each $\md_i \in \mD$ due to the following reasons: 
\begin{itemize}
	\item $\mb \cap I_{\mD} = \emptyset$ since $I_{\mD} \subseteq \mo$ and $\mo \cap \mb = \emptyset$ (see Definition~\ref{def:dpi});
	\item $U_\Tp \cap I_{\mD} =: X \neq \emptyset$ cannot hold since $I_{\mD} \cap \md_i \neq \emptyset$ for each $\md_i \in \mD \subseteq \minD_{\langle\mo,\mb,\Tp,\Tn\rangle_\RQ}$ and thus, for arbitrary $i$, $\md'_i:=\md_i \setminus X$  must already be a diagnosis w.r.t.\ $\langle\mo,\mb,\Tp,\Tn\rangle_\RQ$ which is a contradiction to the subset-minimality of $\md_i$, i.e.\ to the fact that $\md_i \in \minD_{\langle\mo,\mb,\Tp,\Tn\rangle_\RQ}$ (cf.\ Definition~\ref{def:diagnosis});
	\item $(\mo \setminus \md_i) \cap I_{\mD} = \emptyset$ by $I_{\mD}\subset \md_i$ for all $\md_i\in\mD$.
\end{itemize} 
Thus, for all $\md_i\in\mD$ there can be no justification $J\in\Just(\tax,\mo\cup\mb\cup U_\Tp)$ such that $J\subseteq\mo_i^*$, wherefore $\tax$ cannot be entailed by any $\mo_i^*$. By Remark~\ref{rem:entailments_as_sets_of_formulas}, this means that no $\mo_i^*$ entails $Q$ which lets us conclude that $\dx{}(Q)=\emptyset$ which is a contradiction to the assumption that $Q$ is a query.

Ad 3.: By Definition~\ref{def:dpi} which states that $\setof{\text{consistency}}\subseteq \RQ$ and by Definitions~\ref{def:diagnosis} and \ref{def:solution_KB}, every $\mo^{*}_i$ constructed from some $\md_i\in\mD$ satisfies all $r\in\RQ$. As $E$ is a set of common entailments of some subset of $\{\mo_{i}^*\,|\,\md_i \in \mD\}$ and $Q \subseteq E$, the proposition follows.
\end{proof}
\begin{remark}\label{rem:systematic_query_gen_necessary} 
Proposition~\ref{prop:discax} emphasizes the importance of a \emph{systematic} tool-assisted construction of queries, i.e.\ as a set of common entailments of all KBs in $\{\mo_{i}^*\,|\,\md_i \in \dx{}\}$ (see Eq.~\eqref{eq:sol_ont_candidate}), rather than asking the user some axioms or entailments intuitively. Such a guessing of queries could provoke asking a useless ``query'' that does not yield the invalidation of any diagnoses, thus increasing user effort unnecessarily.

For instance, given the running example DPI $\langle\mo,\mb,\Tp,\Tn\rangle_\RQ$ (see Table~\ref{tab:example_dpi_0}),
there is no point in asking the user the ``query'' $Q:=\setof{B \to K \lor Z}$ since $Q$ is entailed by all KBs $\mo_{i}^*$ for $\md_i \in \minD_{\langle\mo,\mb,\Tp,\Tn\rangle_\RQ}$ (cf.\ Table~\ref{tab:min_diagnoses_example_dpi_0}) and therefore cannot be used to discriminate between the diagnoses w.r.t.\ the DPI. To see this, observe that $Q \equiv \setof{\lnot B \lor K \lor Z}$ and $\setof{\lnot B \lor K} = \setof{\tax_9}$ which in turn is a subset of $\mb$. Hence $\mb \models \setof{\tax_9}$ and $\setof{\tax_9} \models Q$ which is why $\mb \models Q$. Since, by definition, $\mb \subseteq \mo_{i}^*$ for all $\md_i \in \mD$, $Q$ is a common entailment of all KBs $\mo_{i}^*$ for $\md_i \in \minD_{\langle\mo,\mb,\Tp,\Tn\rangle_\RQ}$.

However, note that in case one wants to track down faults in a KB for which there is no symptom or evidence given yet, then one might ask ``queries'' about suspicious parts of the KB which are not queries in the sense of this work. For instance, one might use general heuristics or common fault patterns as discussed in \cite{Roussey2009} to ask the user e.g.\ (a)~entailments computed from such a pattern along with (a part of) the given KB or (b)~sentences that yield an inconsistency together with this pattern. In this way one can gather test cases with the goal to obtain new conflict sets and thus diagnoses which enable to locate now-evident faults in the KB. Specifically, if some query constructed as per (a) is answered negatively (and added to the negative test cases) or some query constructed as per (b) is answered positively (and added to the positive test cases), then the probing was successful and the fault can be located by means of a debugging session using exactly the queries discussed in this work. 
\qed
\end{remark}
The next proposition demonstrates a way of reducing a given query in size under preservation of the associated q-partition which means that all information-theoretic properties of the query -- in particular, the value that each measure discussed in Section~\ref{sec:ActiveLearningInInteractiveOntologyDebugging} assigns to it -- will remain unaffected by the modification. The smaller the cardinality of a query, the lower the effort for the answering user usually is. Concretely, and complementary to the first statement (1.)\ of Proposition~\ref{prop:discax} (which says that the entailment of at least one element of the query must depend on axioms in $U_\mD$), Proposition~\ref{prop:reduct_to_discax} asserts that a query does not need to include any elements whose entailment does not depend on $U_\mD$.
\begin{proposition}\label{prop:reduct_to_discax}
Let $\mD \subseteq \minD_{\langle\mo,\mb,\Tp,\Tn\rangle_\RQ}$ with $|\mD|\geq 2$ and $Q \in \mQ_\mD$ be a query with q-partition $\Pt(Q) = \langle\dx{}(Q),\dnx{}(Q), \emptyset\rangle$. 
Then 
\[ Q':=\setof{\tax\,|\,\tax \in Q, \forall J \in \Just(\tax,\mo\cup\mb\cup U_\Tp): J \cap U_\mD \neq \emptyset}\subseteq Q 
\] 
is a query in $\mQ_\mD$ with q-partition $\Pt(Q') = \Pt(Q)$.
\end{proposition}
\begin{proof}
If $Q' = Q$ then the validity of the statement is trivial. Otherwise, $Q' \subset Q$ and for each sentence $\tax$ such that $\tax \in Q$, $\tax \notin Q'$ there is at least one justification $J' \in \Just(\tax,\mo\cup\mb\cup U_\Tp)$ with $J' \cap U_\mD = \emptyset$. So, by Lemma~\ref{lem:U_D_notin_B_and_notin_U_P} it is true that $J' \subseteq (\mo\setminus U_\mD)\cup\mb\cup U_\Tp$.  Hence, $(\mo\setminus U_\mD)\cup\mb\cup U_\Tp \models \tax$ must hold. As $\mo_i^* \supseteq (\mo\setminus U_\mD)\cup\mb\cup U_\Tp$ for all $\md_i \in \mD$, it holds by monotonicity of $\mathcal{L}$ that $\mo_i^* \models \tax$. 

In order for $\Pt(Q') \neq \Pt(Q)$ to hold, either $\dx{}(Q') \subset \dx{}(Q)$ or $\dnx{}(Q') \subset \dnx{}(Q)$ must be true. However, the elimination of $\tax$ from $Q$ does not involve the elimination of any diagnosis from $\dx{}(Q)$ since, for all $\md_i \in \mD$, if $\mo_i^* \models Q$ then, in particular, $\mo_i^* \models Q'$ since $Q' \subseteq Q$.

Also, no diagnosis can be eliminated from $\dnx{}(Q)$. To see this, suppose that some $\md_k$
is an element of $\dnx{}(Q)$, but not an element of $\dnx{}(Q')$. From the latter fact we can conclude that $\mo_k^* \cup Q'$ satisfies all $r\in\RQ$ as well as all negative test cases $\tn\in\Tn$. Now, due to the fact that the entailment relation is idempotent for $\mathcal{L}$ and because $\mo_k^* \models \tax$ for all $\tax\in Q\setminus Q'$ as shown above, we can deduce that $\mo_k^* \cup Q = \mo_k^* \cup Q' \cup (Q\setminus Q')$ is logically equivalent to $\mo_k^* \cup Q'$. This is a contradiction to the assumption that $\md_k$ is an element of $\dnx{}(Q)$ since in this case $\mo_k^* \cup Q$ must violate some $r\in\RQ$ or $\tn\in\Tn$. This completes the proof.
\end{proof}
\begin{example}\label{ex:illustration_of_prop_reduct_to_discax}
To illustrate Proposition~\ref{prop:reduct_to_discax}, let us reconsider the query $Q_1 = \setof{\tax_1,\tax_4, \tax_6,\tax_7}$ w.r.t.\ the set of leading diagnoses $\mD$ given by Eq.~\eqref{eq:ex_explicit_ents_query:leading_diags}.
First of all, we give all justification sets $\mathsf{Just}(\tax,\mo\cup\mb\cup U_\Tp)$ for $\tax \in Q_1$:
\begin{align*}
\mathsf{Just}(\tax_1,\mo\cup\mb\cup U_\Tp) &= \setof{\setof{\tax_1}}  \\
\mathsf{Just}(\tax_4,\mo\cup\mb\cup U_\Tp) &= \setof{\setof{\tax_4}} \\
\mathsf{Just}(\tax_6,\mo\cup\mb\cup U_\Tp) &= \setof{\setof{\tax_6}} \\
\mathsf{Just}(\tax_7,\mo\cup\mb\cup U_\Tp) &= \setof{\setof{\tax_7}}
\end{align*} 
The set $U_\mD$ is given by $\setof{\tax_2,\tax_3,\tax_5,\tax_6}$. Now, only for $\tax_6 \in Q_1$ all its justifications (in this case there is only a single one) have a non-empty intersection with $U_\mD$. For the other sentences in $Q_1 \setminus \setof{\tax_6}$ there is one justification with an empty intersection with $U_\mD$. Hence, these three sentences can be omitted without affecting the q-partition $\Pt(Q_1) = \tuple{\setof{\md_1,\md_2},\setof{\md_3},\emptyset}$ of the query $Q_1$. The resulting query is then $Q'_1 = \setof{\tax_6}$. According to the discussion in Example~\ref{ex:explicit_ents_query}, this can be easily verified by checking that $\mo \setminus \md_1 \supseteq \setof{\tax_6}$, $\mo \setminus \md_2 \supseteq \setof{\tax_6}$ as well as $\mo \setminus \md_3 \not\supseteq \setof{\tax_6}$ which indeed holds since $\md_1,\md_2$ both do not contain $\tax_6$ whereas $\md_3$ does so.\qed
\end{example}

\paragraph{The Discrimination Axioms.} As the previous two propositions suggest, the key axioms in the faulty KB $\mo$ of a DPI in terms of query generation are given by $U_\mD$ as well as $I_\mD$. These will also have a principal role in the specification of canonical queries. We examine this role next.
%
%
%
%
\begin{lemma}\label{lem:no_K*i_includes_just_of_any_ax_in_I_D}
For any DPI $\langle\mo,\mb,\Tp,\Tn\rangle_\RQ$ and $\mD \subseteq \minD_{\langle\mo,\mb,\Tp,\Tn\rangle_\RQ}$ it holds that no KB in $\{\mo_{i}^*\,|\,\md_i \in \mD\}$ includes a justification for any axiom in $I_\mD$. That is, $\bigcup_{\md_i \in \mD} \Just(\tax,\mo_{i}^*) = \emptyset$ for all $\tax\in I_\mD$.
\end{lemma}
\begin{proof}
Assume that there is an axiom $\tax\in I_\mD$ and some $J\in \Just(\tax,\mo_{k}^*)$ for some $\md_k \in \mD$. Then, $\mo_{k}^* \models \tax$. Since $\tax \in I_\mD = \bigcap_{\md_i\in\mD} \md_i$ we can conclude that $\tax \in \md_k$. Since $\mo_{k}^* = (\mo\setminus\md_k)\cup\mb\cup U_\Tp$ meets all requirements $r \in \RQ$ and all negative test cases $\tn\in\Tn$, this must also hold for $(\mo\setminus(\md_k\setminus\tax))\cup\mb\cup U_\Tp$. But this is a contradiction to the subset-minimality of $\md_k \in \minD_{\langle\mo,\mb,\Tp,\Tn\rangle_\RQ}$.
\end{proof}
The latter lemma means that each minimal diagnosis among the leading diagnoses $\mD$ must hit every justification of every axiom in $I_\mD$. That is, no matter which $\dx{}(Q) \subseteq \mD$ we select for the computation of a query $Q$ (as a subset of the common entailments of KBs $\mo^{*}_i$ for $\md_i\in\dx{}(Q)$), $Q$ can never comprise an axiom of $I_\mD$. Of course, this holds in particular for explicit-entailments queries and is captured by the following proposition:
\begin{proposition}\label{prop:no_ax_in_I_D_can_be_element_of_any_query_over_mD}
Let $\langle\mo,\mb,\Tp,\Tn\rangle_\RQ$ be any DPI, $\mD \subseteq \minD_{\langle\mo,\mb,\Tp,\Tn\rangle_\RQ}$ and $Q$ be any query in $\mQ_\mD$. Then $Q \cap I_\mD = \emptyset$.
\end{proposition}
\begin{proof}
The proposition follows immediately from Lemma~\ref{lem:no_K*i_includes_just_of_any_ax_in_I_D} and the fact that $Q \subseteq E$ where $E$ is a set of common entailments of some subset of KBs in $\{\mo_{i}^*\,|\,\md_i \in \mD\}$.
\end{proof}
The next finding is that any explicit-entailments query must comprise some axiom(s) from $U_\mD$.
\begin{proposition}\label{prop:each_expl-ents_query_must_incl_ax_in_U_D}
Let $\langle\mo,\mb,\Tp,\Tn\rangle_\RQ$ be any DPI, $\mD \subseteq \minD_{\langle\mo,\mb,\Tp,\Tn\rangle_\RQ}$ and $Q$ be any explicit-entailments query in $\mQ_\mD$. Then $Q \cap U_\mD \neq \emptyset$.
\end{proposition}
\begin{proof}
Suppose that $Q$ is an explicit-entailments query in $\mQ_\mD$ and $Q \cap U_\mD = \emptyset$. Then $Q \subseteq \mo$. Further, since $\mo^{*}_i \supseteq \mo\setminus\md_i \supseteq \mo\setminus U_\mD \supseteq Q$ and by the monotonicity of $\mathcal{L}$, we can deduce that $\mo^{*}_i \models Q$ for all $\md_i \in \mD$. Therefore, $\dx{}(Q) = \mD$ and $\dnx{}(Q) = \emptyset$ which contradicts the assumption that $Q$ is a query. 
\end{proof}

Now, we can summarize these results for explicit-entailments queries as follows:
\begin{corollary}\label{cor:expl_ent_query_must_neednot_mustnot_include_ax}
Let $\langle\mo,\mb,\Tp,\Tn\rangle_\RQ$ be any DPI, $\mD \subseteq \minD_{\langle\mo,\mb,\Tp,\Tn\rangle_\RQ}$ and $Q$ be any explicit-entailments query in $\mQ_\mD$. Then $Q$
\begin{itemize}
	\item must include some axiom(s) in $U_\mD$,
	\item need not include any axioms in $\mo\setminus U_\mD$, and
	\item must not include any axioms in $I_\mD$.
\end{itemize}
Further on, elimination of axioms in $\mo\setminus U_\mD$ from $Q$ does not affect the q-partition $\Pt(Q)$ of $Q$.
\end{corollary}
Due to this result which implies that $Q' := Q \cap (U_\mD \setminus I_\mD)$ is a query equivalent to $Q$ in terms of the q-partition for any explicit-entailments query $Q$, we now give the set $U_\mD \setminus I_\mD$ the specific denotation $\DiscAx_\mD$. We term this set \emph{the discrimination axioms w.r.t. the leading diagnoses $\mD$}. These are precisely the essential axioms facilitating discrimination between the leading diagnoses.

\begin{example}\label{ex:disc_ax}
Let us consider again the set of leading diagnoses 
w.r.t.\ the running example DPI (Table~\ref{tab:example_dpi_0}) shown by Eq.~\eqref{eq:ex_explicit_ents_query:leading_diags}. Then, $U_\mD = \setof{\tax_2,\tax_3,\tax_5,\tax_6}$ and $I_\mD = \setof{\tax_2}$. Now, all $\subseteq$-minimal explicit-entailments query candidates we might build according to Corollary~\ref{cor:expl_ent_query_must_neednot_mustnot_include_ax} (which provides necessary criteria to explicit-entailments queries) are 
\[\setof{\setof{\tax_3,\tax_5,\tax_6},\setof{\tax_3,\tax_5},\setof{\tax_3,\tax_6},\setof{\tax_5,\tax_6},\setof{\tax_3},\setof{\tax_5},\setof{\tax_6}}\]
That is, all these candidates include at least one element out of $U_\mD$ and no elements out of $\mo \setminus U_\mD$ or $I_\mD$. 
Clearly, there are exactly six different q-partitions with empty $\dz{}$ w.r.t.\ three diagnoses (i.e.\ three possibilities to select one, and three possibilities to select two diagnoses to constitute the set $\dx{}$). Since all these candidates are \emph{explicit-entailments} query candidates, as per Proposition~\ref{prop:d0} the q-partition of each of those must feature an empty $\dz{}$. Consequently, by the pigeonhole principle, either at least two candidates have the same q-partition or at least one candidate is not a query at all. As we require that there must be exactly one canonical query per q-partition and require a method that computes only queries (and no candidates that turn out to be no queries), we see that Corollary~\ref{cor:expl_ent_query_must_neednot_mustnot_include_ax} does not yet refine the set of candidates enough in order to be also a sufficient criterion for canonical queries.

So let us now find out where the black sheep is among the candidates above. The key to finding it is the fact that each query is a common entailment of all $\mo_i^* := (\mo \setminus \md_i) \cup \mb \cup U_\Tp$ (cf.\ Eq.~\eqref{eq:sol_ont_candidate}) for all $\md_i$ in the $\dx{}$-set of the q-partition of it (cf.\ Proposition~\ref{prop:properties_of_q-partitions},(\ref{prop:properties_of_q-partitions:enum:query_is_set_of_common_ent}.)). 
Since the candidates for canonical queries above are all explicit-entailments queries (which can be computed without the need to use a reasoning service), they, by definition, comprise only explicit entailments $\alpha \in \mo$. So, we immediately see that we must postulate that each canonical query is a set of common elements of $\mo \setminus \md_i$ for all $\md_i$ in the $\dx{}$-set of the q-partition of it. Starting to verify this for the first candidate $\setof{\tax_3,\tax_5,\tax_6}$ above, we quickly find out that there is no possible $\dx{}$-set of a q-partition such that $Q \subseteq \bigcap_{\md_i \in \dx{}} \mo \setminus \md_i$ because none of these intersected sets includes all elements out of $\setof{\tax_3,\tax_5,\tax_6}$. Hence, it became evident that the first candidate is no query at all. 

Performing an analogue verification for the other candidates, we recognize that all of them are indeed queries and no two of them exhibit the same q-partition. Concretely, the q-partitions associated with the queries in the set above (minus the first set $\setof{\tax_3,\tax_5,\tax_6}$) are as follows: 
\begin{align}
\begin{split} \label{eq:canQ+canQP_for_diags1,2,3}
\Pt(\setof{\tax_3,\tax_5}) &= \tuple{\setof{\md_3},\setof{\md_1,\md_2},\emptyset} \\
\Pt(\setof{\tax_3,\tax_6}) &= \tuple{\setof{\md_2},\setof{\md_1,\md_3},\emptyset} \\
\Pt(\setof{\tax_5,\tax_6}) &= \tuple{\setof{\md_1},\setof{\md_2,\md_3},\emptyset} \\
\Pt(\setof{\tax_3}) &= \tuple{\setof{\md_2,\md_3},\setof{\md_1},\emptyset} \\
\Pt(\setof{\tax_5}) &= \tuple{\setof{\md_1,\md_3},\setof{\md_2},\emptyset} \\
\Pt(\setof{\tax_6}) &= \tuple{\setof{\md_1,\md_2},\setof{\md_3},\emptyset} 
\end{split}
\end{align}
which can be easily seen from the sets $\mo \setminus \md_i$:
\begin{align*}
\mo \setminus \md_1 &= \setof{\tax_1,\tax_4,\tax_5,\tax_6,\tax_7} \\
\mo \setminus \md_2 &= \setof{\tax_1,\tax_3,\tax_4,\tax_6,\tax_7} \\
\mo \setminus \md_3 &= \setof{\tax_1,\tax_3,\tax_4,\tax_5,\tax_7} 
\end{align*}
As it will turn out, these six queries are exactly all canonical queries w.r.t.\ the leading diagnoses $\mD$.\qed
\end{example}

We are now in the state to formally define a canonical query. 
To this end, let $E_{\mathsf{exp}}(\mathbf{X})$ denote all common explicit entailments of all elements of the set $\{\mo \setminus \md_i\,|\,\md_i \in \mathbf{X}\}$ for some set $\mathbf{X}$ of minimal diagnoses w.r.t.\ $\tuple{\mo,\mb,\Tp,\Tn}_\RQ$. Then:
\begin{proposition}\label{prop:E_exp} For $E_{\mathsf{exp}}(\dx{}(Q))$, the following statements are true:
\begin{enumerate}
	\item $E_{\mathsf{exp}}(\dx{}(Q)) = \mo \setminus U_{\dx{}(Q)}$.
	\item $E_{\mathsf{exp}}(\dx{}(Q)) \cap I_\mD = \emptyset$.
\end{enumerate}
\end{proposition}
\begin{proof}
Ad 1.: "$\subseteq$":~Clearly, each common explicit entailment of $\{\mo \setminus \md_i\,|\,\md_i \in \dx{}(Q)\}$ must be in $\mo$. Furthermore, Lemma~\ref{lem:Ki^*_does_not_entail_any_ax_in_Di} (see below) applied to $\mo_i^*$ for $\md_i\in\dx{}(Q)$ yields $\mo_i^* \not\models \tax$ for each $\tax \in \md_i$. By monotonicity of $\mathcal{L}$, this is also valid for $\mo \setminus \md_i$. Since $U_{\dx{}(Q)}$ is exactly the set of axioms that occur in at least one diagnosis in $\dx{}(Q)$, no axiom in this set can be a \emph{common} (explicit) entailment of all KBs in $\setof{\mo\setminus\md_i \,|\,\md_i\in\dx{}(Q)}$. Consequently, each common explicit entailment of $\{\mo \setminus \md_i\,|\,\md_i \in \dx{}(Q)\}$ must be in $\mo \setminus U_{\dx{}(Q)}$.

"$\supseteq$":~~All axioms in $\mo \setminus U_{\dx{}(Q)}$ occur in each $\mo\setminus \md_i$ for $\md_i\in\dx{}(Q)$ and are therefore entailed by each $\mo\setminus\md_i$ due to that fact that the entailment relation in $\mathcal{L}$ is extensive.

Ad 2.: Since clearly $I_\mD \subseteq U_\mD$ and, by (1.), $E_{\mathsf{exp}}(\dx{}(Q)) = \mo \setminus U_{\dx{}(Q)}$, the proposition follows by the fact that $I_\mD \subseteq U_{\dx{}(Q)}$. The latter holds because all axioms that occur in all diagnoses in $\mD$ (i.e.\ the axioms in $I_\mD$) must also occur in all diagnoses in $\dx{}(Q)$ (i.e.\ in $U_{\dx{}(Q)}$) since $\dx{}(Q) \subseteq \mD$.
\end{proof}
\begin{lemma}\label{lem:Ki^*_does_not_entail_any_ax_in_Di}
Let $\langle\mo,\mb,\Tp,\Tn\rangle_\RQ$ be a DPI and $\mD \subseteq \minD_{\langle\mo,\mb,\Tp,\Tn\rangle_\RQ}$. Then, for each $\md_i\in\mD$ and for each $\tax\in\md_i$ it holds that $\mo_i^{*} \not\models \tax$.
\end{lemma}
\begin{proof}
Let us assume that $\mo_i^{*} \models \tax$ for some $\tax \in \md_i$ for some $\md_i \in \mD$. Then, by the fact that the entailment relation in $\mathcal{L}$ is idempotent, $\mo_i^{*} \cup \tax \equiv \mo_i^{*}$. But $\mo_i^{*} \cup \tax \equiv (\mo \setminus (\md_i \setminus \setof{\tax})) \cup \mb \cup U_\Tp$. Since $\mo_i^{*}$ satisfies all negative test cases $\tn\in\Tn$ as well as all requirements $r\in\RQ$, $(\mo \setminus (\md_i \setminus \setof{\tax})) \cup \mb \cup U_\Tp$ must do so as well, wherefore $\md_i \setminus \setof{\tax}$ is a diagnosis w.r.t.\ $\langle\mo,\mb,\Tp,\Tn\rangle_\RQ$. This however is a contradiction to the assumption that $\md_i \in \minD_{\langle\mo,\mb,\Tp,\Tn\rangle_\RQ}$, i.e.\ that $\md_i$ is a minimal diagnosis w.r.t.\ $\langle\mo,\mb,\Tp,\Tn\rangle_\RQ$.
\end{proof}
Roughly, a canonical query is an explicit-entailments query that has been reduced as per Corollary~\ref{cor:expl_ent_query_must_neednot_mustnot_include_ax}.
\begin{definition}[Canonical Query]\label{def:canonical_query}
Let $\langle\mo,\mb,\Tp,\Tn\rangle_\RQ$ be a DPI, $\mD \subseteq \minD_{\langle\mo,\mb,\Tp,\Tn\rangle_\RQ}$ with $|\mD| \geq 2$. Let further $\emptyset\subset\mathbf{S}\subset\mD$ be a seed set of minimal diagnoses. Then we call $Q_{\mathsf{can}}(\mathbf{S}) := E_{\mathsf{exp}}(\mathbf{S}) \cap \DiscAx_\mD$ \emph{the canonical query w.r.t.\ $\mathbf{S}$} if $Q_{\mathsf{can}}(\mathbf{S}) \neq \emptyset$. Otherwise, $Q_{\mathsf{can}}(\mathbf{S})$ is undefined.
\end{definition}
It is trivial to see that:
\begin{proposition}\label{prop:canonical_query_unique_for_seed}
If existent, the canonical query w.r.t.\ $\mathbf{S}$ is unique. 
\end{proposition}
We now show that a canonical query is a query (as per Definition~\ref{def:query}).
\begin{proposition}\label{prop:canonical_query_is_query}
If $Q$ is a canonical query, then $Q$ is a query.
\end{proposition} 
\begin{proof}
Let $Q$ be a canonical query w.r.t.\ some seed $\mathbf{S}$ satisfying $\emptyset \subset \mathbf{S} \subset \mD$. We demonstrate that $\dx{}(Q) \neq \emptyset$ as well as $\dnx{}(Q) \neq \emptyset$ which is sufficient to show the proposition due to Proposition~\ref{prop:properties_of_q-partitions},(\ref{prop:properties_of_q-partitions:enum:set_of_logical_formulas_is_query_iff...}.). 

First, since for all $\md_i \in \mathbf{S}$ it holds that $\md_i \subseteq U_{\mathbf{S}}$, we have that $\mo \setminus \md_i \supseteq \mo \setminus U_{\mathbf{S}}$. Due to the fact that the entailment relation in the used logic $\mathcal{L}$ is extensive, we can derive that $\mo \setminus \md_i \models \mo \setminus U_{\mathbf{S}}$. Hence, by the monotonicity of $\mathcal{L}$, also $\mo_i^* := (\mo \setminus \md_i) \cup \mb \cup U_\Tp \models \mo \setminus U_{\mathbf{S}}$ for all $\md_i \in \mathbf{S}$. But, by Proposition~\ref{prop:E_exp}, $\mo \setminus U_{\mathbf{S}}$ is equal to $E_{\mathsf{exp}}(\mathbf{S})$. As $Q \subseteq E_{\mathsf{exp}}(\mathbf{S})$ by Definition~\ref{def:canonical_query}, it becomes evident that $\mo_i^* \models Q$ for all $\md_i \in \mathbf{S}$. Therefore $\dx{}(Q) \supseteq \mathbf{S} \supset \emptyset$ by the definition of $\dx{}(Q)$ (Definition~\ref{def:q-partition}). 

Second, (*): If $U_{\mathbf{S}} \subset U_\mD$, then $\dnx{}(Q) \neq \emptyset$. 
To see why (*) holds, we point out that the former implies the existence of a diagnosis $\md_i \in \mD$ which is not in $\mathbf{S}$. 
Also, there must be some sentence $\tax \in \md_i$ where $\tax \notin U_{\mathbf{S}}$. 
This implies that $\tax \in U_\mD$ because it is in some diagnosis in $\mD$ and that $\tax \notin I_\mD$ because clearly $I_\mD \subseteq U_{\mathbf{S}}$ (due to $\mathbf{S} \supset \emptyset$). 
Thus, $\tax \in (U_{\mD} \setminus I_{\mD}) =: \DiscAx_\mD$. 
Since $Q = (\mo \setminus U_{\mathbf{S}}) \cap \DiscAx_\mD$ by Definition~\ref{def:canonical_query} and Proposition~\ref{prop:E_exp}, we find that $\tax \in Q$ must hold. 
So, by $Q \subseteq \mo$, we have $\mo_i^* \cup Q := (\mo \setminus \md_i) \cup \mb \cup U_\Tp \cup Q = (\mo \setminus \md'_i) \cup \mb \cup U_\Tp$ for some $\md'_i \subseteq \md_i \setminus \setof{\tax} \subset \md_i$. Now, the $\subseteq$-minimality of all diagnoses in $\mD$, in particular $\md_i$, lets us derive that $\mo_i^* \cup Q$ must violate some $x \in \Tn \cup \RQ$. By the definition of $\dnx{}(Q)$ (Definition~\ref{def:q-partition}), this means that $\md_i \in \dnx{}(Q)$. This proves (*).

Finally, let us assume that $\dnx{}(Q) = \emptyset$. By application of the law of contraposition to (*), this implies that $U_{\mathbf{S}} = U_\mD$. This in turn yields by Proposition~\ref{prop:E_exp} that $E_{\mathsf{exp}}(\mathbf{S}) = \mo \setminus U_\mD$. Hence, by Definition~\ref{def:canonical_query}, $Q = (\mo \setminus U_\mD) \cap (U_\mD \setminus I_\mD) \subseteq (\mo \setminus U_\mD) \cap (U_\mD) = \emptyset$, i.e.\ $Q = \emptyset$. From this we get by Definition~\ref{def:canonical_query} that $Q$ is undefined and hence no canonical query. Thence, $\dnx{}(Q) \neq \emptyset$ must hold.
 
This completes the proof that $Q$ is a query.
\end{proof}
A canonical q-partition is a q-partition for which there is a canonical query with exactly this q-partition:
\begin{definition}[Canonical Q-Partition]\label{def:canonical_q-partition}
Let $\langle\mo,\mb,\Tp,\Tn\rangle_\RQ$ be a DPI, $\mD \subseteq \minD_{\langle\mo,\mb,\Tp,\Tn\rangle_\RQ}$ with $|\mD| \geq 2$. Let further $\Pt' = \langle\dx{},\dnx{}, \emptyset\rangle$ be a partition of $\mD$. Then we call $\Pt'$ \emph{a canonical q-partition} iff $\Pt' = \Pt(Q_{\mathsf{can}}(\dx{}))$, i.e.\ $\langle\dx{},\dnx{}, \emptyset\rangle = \langle\dx{}(Q_{\mathsf{can}}(\dx{})),\dnx{}(Q_{\mathsf{can}}(\dx{})), \emptyset\rangle$. In other words, given the partition $\langle\dx{},\dnx{}, \emptyset\rangle$, the canonical query w.r.t.\ the seed $\dx{}$ must have exactly the q-partition $\langle\dx{},\dnx{}, \emptyset\rangle$.
\end{definition}
As a consequence of this definition and Proposition~\ref{prop:canonical_query_is_query}, we obtain that each canonical q-partition is a q-partition (as per Definition~\ref{def:q-partition}):
\begin{corollary}\label{cor:canonical_q-partition_is_q-partition}
Each canonical q-partition is a q-partition.
\end{corollary}
\begin{proof}
By Definition~\ref{def:canonical_q-partition}, for each canonical q-partition $\Pt'$ there is a canonical query $Q$ such that $\Pt(Q) = \Pt'$. Since however $Q$ is a query as per Proposition~\ref{prop:canonical_query_is_query}, $\Pt'$ is a q-partition by Definition~\ref{def:q-partition}.
\end{proof}
It is very easy to verify that there is a one-to-one relationship between canonical q-partitions and canonical queries:
\begin{proposition}\label{prop:canonical_query_unique_for_seed}
Given a canonical q-partition, there is exactly one canonical query associated with it and vice versa.
\end{proposition}
So, the idea is to use for a q-partition $\langle\dx{}(Q),\dnx{}(Q), \emptyset\rangle$ the canonical query $E_{\mathsf{exp}}(\dx{}(Q)) \cap \DiscAx_\mD$ as a standardized representation of an associated query during the search for an optimal q-partition. Stated differently, the search strategy proposed in this work will consider solely canonical q-partitions and will use a successor function based on canonical queries that computes in each node expansion step in the search tree (all and only) canonical q-partitions that result from a given canonical q-partition by minimal changes.

Moreover, it is easy to see that the following must be valid:
\begin{proposition}\label{prop:canonical_q-partition_has_empty_dz}
Any canonical q-partition $\langle\dx{},\dnx{}, \dz{}\rangle$ satisfies $\dz{} = \emptyset$.
\end{proposition}
\begin{proof}
By Definition~\ref{def:canonical_q-partition}, for each canonical q-partition $\Pt$ there is a canonical query $Q$ for which $\Pt$ is the q-partition associated with $Q$. Further on, by Definition~\ref{def:canonical_query}, $Q := E_{\mathsf{exp}}(\mathbf{S}) \cap \DiscAx_\mD \subseteq \mo$ for some $\emptyset\subset\mathbf{S}\subset\mD$. Since $Q \subseteq \mo$, Proposition~\ref{prop:d0} yields that $\dz{}(Q) = \emptyset$ must hold.
%
\end{proof}

Let us at this point discuss a small example that should illustrate the introduced notions:
\begin{table}[t]
	\footnotesize
	\centering
		\rowcolors[]{2}{gray!8}{gray!16} 
		\begin{tabular}{ c c c } 
			\rowcolor{gray!40}
			\toprule\addlinespace[0pt]
			Seed $\mathbf{S}$ & $\setof{i\,|\, \tax_i \in Q_{\mathsf{can}}(\mathbf{S})}$ & canonical q-partition \\ 
			\addlinespace[0pt]\midrule\addlinespace[0pt]
			$\setof{\md_5,\md_6}$ & $\setof{2,5,6} \cap \setof{1,2,3,4,7} = \setof{2}$ & $\tuple{\setof{\md_5,\md_6},\setof{\md_1},\emptyset}$ \\
			$\setof{\md_1,\md_6}$ & $\setof{1,5,6} \cap \setof{1,2,3,4,7} = \setof{1}$ & $\tuple{\setof{\md_1,\md_6},\setof{\md_5},\emptyset}$ \\
			$\setof{\md_1,\md_5}$ & $\setof{5,6} \cap \setof{1,2,3,4,7} = \emptyset$ & $\times$ \\
			$\setof{\md_1}$ & $\setof{1,4,7} \cap \setof{1,2,3,4,7} = \setof{1,4,7}$ & $\tuple{\setof{\md_1},\setof{\md_5,\md_6},\emptyset}$ \\
			$\setof{\md_5}$ & $\setof{2,3} \cap \setof{1,2,3,4,7} = \setof{2,3}$ & $\tuple{\setof{\md_5},\setof{\md_1,\md_6},\emptyset}$ \\
			$\setof{\md_6}$ & $\setof{1,2} \cap \setof{1,2,3,4,7} = \setof{1,2}$ & $\tuple{\setof{\md_6},\setof{\md_1,\md_5},\emptyset}$ \\
			\addlinespace[0pt]\bottomrule 
			\end{tabular}
	\caption[Example: Canonical Queries and Q-Partitions]{All canonical queries and associated canonical q-partitions w.r.t.\ $\mD = \setof{\md_1,\md_5,\md_6}$ (cf.\ Table~\ref{tab:min_diagnoses_example_dpi_0}) and the example DPI given by Table~\ref{tab:example_dpi_0}.}
	\label{tab:Qcan+canQPart_for_example_dpi_0}
\end{table}
\begin{example}\label{ex:canonical_queries_q-partitions}
Consider the DPI $\langle\mo,\mb,\Tp,\Tn\rangle_\RQ$ defined by Table~\ref{tab:example_dpi_0} and the set of leading diagnoses 
\begin{align} \label{eq:ex_canonical_queries_q-partitions:leading_diags}
\mD = \setof{\md_1,\md_5,\md_6} = \setof{\setof{\tax_2,\tax_3},\setof{\tax_1,\tax_4,\tax_7},\setof{\tax_3,\tax_4,\tax_7}} \quad \mbox{(cf.\ Table~\ref{tab:min_diagnoses_example_dpi_0})}
\end{align}
w.r.t.\ this DPI.
The potential solution KBs given this set of leading diagnoses $\mD$ are $\setof{\mo_1^{*},\mo_2^{*},\mo_3^{*}}$ (cf.\ Eq.~\eqref{eq:sol_ont_candidate}) where 
\begin{align*}
\mo_1^{*} &= \setof{\tax_1,\tax_4,\tax_5,\tax_6,\tax_7,\tax_8,\tax_9,\tp_1} \\
\mo_2^{*} &= \setof{\tax_2,\tax_3,\tax_5,\tax_6,\tax_8,\tax_9,\tp_1} \\
\mo_3^{*} &= \setof{\tax_1,\tax_2,\tax_5,\tax_6,\tax_8,\tax_9,\tp_1} 
\end{align*}
The discrimination axioms $\DiscAx_\mD$ are $U_\mD \setminus I_\mD = \setof{\tax_1,\tax_2,\tax_3,\tax_4,\tax_7}\setminus\emptyset$. Table~\ref{tab:Qcan+canQPart_for_example_dpi_0} lists all possible seeds $\mathbf{S}$ (i.e.\ proper non-empty subsets of $\mD$) and, if existent, the respective (unique) canonical query $Q_{\mathsf{can}}(\mathbf{S})$ as well as the associated (unique) canonical q-partition. Note that the canonical query for the seed $\mathbf{S} = \setof{\md_1,\md_5}$ is undefined which is why there is no canonical q-partition with a \mbox{$\dx{}$-set} $\setof{\md_1,\md_5}$. 
This holds since (by Proposition~\ref{prop:E_exp}) $E_{\mathsf{exp}}(\mathbf{S}) = (\mo \setminus \md_1) \cap (\mo \setminus \md_5) = \mo \setminus U_{\mathbf{S}} = \setof{\tax_1,\dots,\tax_7}\setminus (\setof{\tax_2,\tax_3} \cup \setof{\tax_1,\tax_4,\tax_7}) = \setof{\tax_5,\tax_6}$ has an empty intersection with $\DiscAx_\mD$. So, by Definition~\ref{def:canonical_query}, $Q_{\mathsf{can}}(\mathbf{S}) = \emptyset$.

Additionally, we point out that there is no query (and hence no q-partition) -- and therefore not just no \emph{canonical} query -- for which $\dx{}$ corresponds to $\setof{\md_1,\md_5}$ because for this to hold $\md_6$ must be in $\dnx{}$. To this end, there must be a set of common entailments $Ents$ of $\mo_1^{*}$ and $\mo_5^{*}$ which, along with $\mo_6^{*}$, violates some requirement $r\in\RQ$ (i.e.\ is inconsistent, cf.\ $\RQ$ in Table~\ref{tab:example_dpi_0}) or entails some negative test case $\tn\in\Tn$ (cf.\ Proposition~\ref{prop:properties_of_q-partitions},(\ref{prop:properties_of_q-partitions:enum:query_is_set_of_common_ent}) and (\ref{prop:properties_of_q-partitions:enum:dx_dnx_dz_contain_exactly_those_diags_that_are...})). As $Ents$ does not include any sentences that are not entailed by $\mo \setminus (\md_1 \cup \md_5) \cup \mb \cup U_\Tp$ as well (insofar as tautologies are excluded from the entailment types that are computed, which is always our assumption), and due to the observation that $\mo \setminus (\md_1 \cup \md_5) \subset \mo \setminus \md_6$, by monotonicity of propositional logic, every common entailment of $\mo_1^{*}$ and $\mo_5^{*}$ is also an entailment of $\mo_6^{*}$. Due to the definition of a solution KB (cf.\ Definition~\ref{def:solution_KB}), which implies that $\mo_6^*$ does not violate any requirements or test cases, this means that there cannot be a query w.r.t.\ the $\dx{}$-set $\setof{\md_1,\md_5}$. So, obviously, in this example every q-partition (with empty $\dz{}$) is also a canonical q-partition.\qed
\end{example}
\subsubsection{Search Completeness Using Only Canonical Q-Partitions}
\label{sec:SearchCompletenessUsingOnlyCanonicalQPartitions}
Now, in the light of the previous example, the question arises whether a q-partition can exist which is not canonical. If yes, which properties must this q-partition have? Answering these questions means at the same time answering the question whether the search for q-partitions taking into account only canonical q-partitions is complete.

Unfortunately, we are not (yet) in the state to give a definite answer to these questions. Actually, we did not yet manage to come up with a counterexample witnessing the incompleteness of the search for q-partitions that takes into consideration only canonical q-partitions. But, we are in the state to give numerous precise properties such non-canonical q-partitions must meet, if they do exist. This is what we present next.

To this end, we assume that $\Pt(Q) = \tuple{\dx{}(Q), \dnx{}(Q), \emptyset}$ is a q-partition w.r.t.\ some set of leading diagnoses $\mD \subseteq \minD_{\langle\mo,\mb,\Tp,\Tn\rangle_\RQ}$ ($|\mD|\geq 2$) w.r.t.\ some DPI $\langle\mo,\mb,\Tp,\Tn\rangle_\RQ$, but not a canonical one. From this we can directly deduce that $Q$ must be a query w.r.t.\ $\mD$ and $\langle\mo,\mb,\Tp,\Tn\rangle_\RQ$, but there cannot be a canonical query with associated q-partition $\Pt(Q)$. Hence, by Definition~\ref{def:canonical_q-partition}, either 
\begin{enumerate}[label=(\alph*)]
	\item \label{case:canon_query_not_exists_a} $Q_{\mathsf{can}}(\dx{}(Q))$ is not defined, i.e.\ $E_{\mathsf{exp}}(\dx{}(Q)) \cap \DiscAx_\mD = \emptyset$ (see Definition~\ref{def:canonical_query}), or
	\item \label{case:canon_query_not_exists_b} $Q_{\mathsf{can}}(\dx{}(Q))$ is defined, but $\tuple{\dx{}(Q_{\mathsf{can}}(\mathbf{S})), \dnx{}(Q_{\mathsf{can}}(\mathbf{S})), \emptyset} \neq \tuple{\dx{}(Q), \dnx{}(Q), \emptyset}$ for all $\mathbf{S}$ such that $\emptyset\subset\mathbf{S}\subset\mD$ for $\mD \subseteq \minD_{\langle\mo,\mb,\Tp,\Tn\rangle_\RQ}$
\end{enumerate}
holds.

Concerning Case~\ref{case:canon_query_not_exists_a}, we want to emphasize that (the explicit-entailments query) $Q_{\mathsf{can}}(\dx{}(Q))$ might be the empty set in spite of Proposition~\ref{prop:discax},(1.)\ which states that there must necessarily be one sentence in a query, each justification of which includes an axiom in $U_\mD$. The reason why $Q_{\mathsf{can}}(\dx{}(Q)) = \emptyset$ is still possible is that all these justifications do not necessarily comprise at least one \emph{identical} axiom in $U_\mD$. That is to say that $Q$ might be a query in spite of $Q_{\mathsf{can}}(\dx{}(Q))$ being the empty set. 

The occurrence of Case~\ref{case:canon_query_not_exists_b} is possible since $Q$ might include entailments which are not entailed by $E_{\mathsf{exp}}(\dx{}(Q)) = \mo \setminus U_{\dx{}(Q)}$ and thence not by $Q_{\mathsf{can}}(\dx{}(Q)) = E_{\mathsf{exp}}(\dx{}(Q)) \cap \DiscAx_\mD$ either. In this case, $\dnx{}(Q_{\mathsf{can}}(\dx{}(Q))) \subset \dnx{}(Q)$ and $\dx{}(Q_{\mathsf{can}}(\dx{}(Q))) \supset \dx{}(Q)$ would need to hold.
%

For these two cases to arise, however, plenty of sophisticated conditions must apply, as shown by the next Proposition. Therefore, the occurrence of this situation can be rated rather unlikely.
\begin{proposition}\label{prop:necess_cond_when_q-partition_is_not_canonical_CASE_a}
Case~\ref{case:canon_query_not_exists_a}, i.e.\ that $Q_{\mathsf{can}}(\dx{}(Q))$ is not defined, can only arise if:
\begin{enumerate}
	\item $U_{\dx{}(Q)} = U_\mD$ (which holds only if $|\dx{}(Q)| \geq 2$), and
	\item $Q$ includes some implicit entailment, and
	\item For each diagnosis $\md_j\in\dnx{}(Q)$ there is a set of sentences $S_j$ such that
	\begin{itemize}
		\item $S_j \subseteq Q$ and
		\item $\mo_i^* \models S_j$ for all $\md_i\in\dx{}(Q)$ and
		\item $\mo_j^* \cup S_j$ violates some requirement $r\in\RQ$ or some test case $\tn\in\Tn$.
	\end{itemize}
\end{enumerate}
For a set $S_j$ as in (3.)\ to exist, it must be fulfilled that
\begin{enumerate}[label=3.\arabic*]
	\item for all $\md_i\in\dx{}(Q)$ there is a justification $J \in \Just(S_j,\mo_i^*)$, and
	\item for all consistent justifications $J \in \Just(S_j,\mo\cup\mb\cup U_\Tp)$ it is true that $J \cap \md_j \neq \emptyset$, and
	\item for each axiom $\tax \in \md_j$ there is some $k \in \setof{i\,|\,\md_i\in\dx{}(Q)}$ such that for all $J \in \Just(S_j,\mo_k^{*})$ the property $\tax \notin J$ holds.
\end{enumerate}
\end{proposition}
\begin{proof}
Ad 1.: Assume the opposite, i.e.\ $U_{\dx{}(Q)} \neq U_\mD$, which implies $U_{\dx{}(Q)} \subset U_\mD$ since $\dx{}(Q) \subset \mD$. Additionally, suppose Case~\ref{case:canon_query_not_exists_a} occurs. Now, by Proposition~\ref{prop:E_exp}, $E_{\mathsf{exp}}(\dx{}(Q)) = \mo \setminus U_{\dx{}(Q)}$ includes an axiom $\tax\in\DiscAx_\mD = U_\mD \setminus I_\mD$. This is true since axioms in $U_\mD \setminus U_{\dx{}(Q)} \subseteq \mo \setminus U_{\dx{}(Q)}$ are obviously in $U_\mD$, but not in $I_\mD$ as otherwise they would need to appear in each $\md\in \dx{}(Q)$ and thus in $U_{\dx{}(Q)}$, contradicting the fact that these axioms are in $U_\mD \setminus U_{\dx{}(Q)}$. Hence, $Q_{\mathsf{can}}(\dx{}(Q)) = E_{\mathsf{exp}}(\dx{}(Q)) \cap \DiscAx_\mD \supseteq \setof{\tax} \supset \emptyset$.

We now show that $U_{\dx{}(Q)} = U_\mD$ can only hold if $|\dx{}(Q)| \geq 2$. This must be true, as the assumption of the opposite, i.e.\ $|\dx{}(Q)| = 1$ or equivalently $\dx{}(Q) = \setof{\md}$ leads to $\md = U_\mD$. Since $\tuple{\dx{}(Q), \dnx{}(Q), \emptyset}$ is a q-partition, $\dnx{}(Q) \neq \emptyset$. Thence,
$\md \supseteq \md'$ must hold for all $\md' \in \dnx{}(Q)$. Since $\md' \neq \md$ must hold (due to $\dx{}(Q) \cap \dnx{}(Q) = \emptyset$ by Proposition~\ref{prop:properties_of_q-partitions}.(\ref{prop:properties_of_q-partitions:enum:q-partition_is_partition}.)), this is a contradiction to the subset-minimality of all diagnoses in $\mD$.

Ad 2.: Let Case~\ref{case:canon_query_not_exists_a} occur and let $Q$ include only explicit entailments. Then $Q \subseteq E_{\mathsf{exp}}(\dx{}(Q))$. However, as, by occurrence of Case~\ref{case:canon_query_not_exists_a}, $Q_{\mathsf{can}}(\dx{}(Q)) := E_{\mathsf{exp}}(\dx{}(Q)) \cap \DiscAx_\mD = E_{\mathsf{exp}}(\dx{}(Q)) \cap (U_\mD \setminus I_\mD) = \emptyset$ holds, and, by Proposition~\ref{prop:E_exp},(2.), we know that $E_{\mathsf{exp}}(\dx{}(Q)) \cap I_\mD = \emptyset$, we have that $Q \cap U_\mD = \emptyset$ and thence $Q \subseteq (\mo \setminus U_\mD) \cup \mb \cup U_\Tp$. So, by the fact that the entailment relation in $\mathcal{L}$ is extensive, $(\mo \setminus U_\mD) \cup \mb \cup U_\Tp \models Q$. This, however, constitutes a contradiction to $Q$ being a query by Proposition~\ref{prop:discax},(1.)\ because all sentences in $Q$ are entailed by $(\mo \setminus U_\mD) \cup \mb \cup U_\Tp$, wherefore there is a justification $J \in \Just(\tax,\mo\cup\mb\cup U_\Tp)$ for each sentence in $Q$ such that $J \cap U_\mD = \emptyset$.

Ad 3.: The proposition is a direct consequence of Definition~\ref{def:q-partition}. Therefore, what we have to show is that (3.1) -- (3.3) are necessary consequences of (3.). To this end, assume that Case~\ref{case:canon_query_not_exists_a} is given.

Ad 3.1: This is obviously true by Definition~\ref{def:justification} and since $\mo_i^* \models S_j$ for all $\md_i\in\dx{}(Q)$.

Ad 3.2: Assume that there is some consistent justification $J \in \Just(S_j,\mo\cup\mb\cup U_\Tp)$ such that $J \cap \md_j = \emptyset$. Then $J \subseteq \mo_j^{*}$ wherefore $\mo_j^{*} \models S_j$. Due to the property of the entailment relation in $\mathcal{L}$ to be idempotent, we have that $\mo_j^{*} \cup S_j \equiv \mo_j^{*}$. As $\mo_j^{*}$ does not violate any requirement $r\in\RQ$ or test case $\tn \in \Tn$ by the fact that $\md_j$ is a diagnosis, it holds that $\mo_j^{*} \cup S_j$ does not do so either. But this is a contradiction to (3.).

Ad 3.3: Assume that there is some axiom $\tax \in \md_j$ such that for all $k \in \setof{i\,|\,\md_i\in\dx{}(Q)}$ there is some $J \in \Just(S_j,\mo_k^{*})$ with the property $\tax \in J$. This means that $\tax \in \mo_k^*$ for all $k \in \setof{i\,|\,\md_i\in\dx{}(Q)}$. Hence, $\tax \notin \md_k$ for all $k \in \setof{i\,|\,\md_i\in\dx{}(Q)}$ which implies that $\tax \notin U_{\dx{}(Q)}$. But, $\tax \in \md_j$ wherefore $\tax\in U_\mD$. This is a contradiction to (1.)\ which asserts that $U_\mD = U_{\dx{}(Q)}$.
\end{proof}
To summarize this, Condition~(2.)\ postulates \emph{some non-explicit entailment(s) in $Q$}. Condition~(3.1) means that either one and the same justification for $S_j$ must be a \emph{subset of -- by Condition~(1.), multiple -- different KBs} or there must be -- by Condition~(1.)\ -- \emph{multiple justifications for $S_j$}. Condition~(3.2), on the other hand, postulates the \emph{presence of some axiom(s) of a specific set $\md_j$ in every justification} for $S_j$, and Condition~(3.3) means that not all justifications for $S_j$ are allowed to comprise one and the same axiom from the specific set $\md_j$. 
And all these conditions must apply for $|\dnx{}(Q)|$ different sets $S_j$. 
\begin{proposition}\label{prop:necess_cond_when_q-partition_is_not_canonical_CASE_b}
Case~\ref{case:canon_query_not_exists_b}, i.e.\ that $\tuple{\dx{}(Q_{\mathsf{can}}(\mathbf{S})), \dnx{}(Q_{\mathsf{can}}(\mathbf{S})), \emptyset} \neq \tuple{\dx{}(Q), \dnx{}(Q), \emptyset}$ for all seeds $\emptyset\subset\mathbf{S}\subset\mD$ for $\mD \subseteq \minD_{\langle\mo,\mb,\Tp,\Tn\rangle_\RQ}$, can only arise if:
\begin{enumerate}
	\item $Q$ includes some implicit entailment, and
	\item $Q$ includes some (implicit) entailment which is not an entailment of $\bigcap_{\setof{i\,|\,\md_i\in\dx{}(Q)}} \mo_i^{*}$ (which is true only if $|\dx{}(Q)| \geq 2$) and
	\item $\dx{}(Q_{\mathsf{can}}(\dx{}(Q))) \supset \dx{}(Q)$ and $\dnx{}(Q_{\mathsf{can}}(\dx{}(Q))) \subset \dnx{}(Q)$, and
	\item $J$ is not a conflict w.r.t.\ $\langle\mo,\mb,\Tp,\Tn\rangle_\RQ$ for all $J \in  \bigcup_{\setof{i\,|\,\md_i\in\dx{}(Q)}} \Just(Q,\mo_i^{*})$, and
	\item there is some $\md_m \in \dnx{}(Q)$ such that:
	\begin{enumerate}
		\item $\md_m \subset U_{\dx{}(Q)}$, and
		\item $(U_{\dx{}(Q)}\setminus\md_m) \cap \mc \neq \emptyset$ for all (minimal) conflicts $\mc$ w.r.t.\ $\langle\mo\setminus\md_m,\mb,\Tp\cup\setof{Q},\Tn\rangle_\RQ$, and
		\item $J \cap \md_m \neq \emptyset$ for all $J \in  \bigcup_{\setof{i\,|\,\md_i\in\dx{}(Q)}} \Just(Q,\mo_i^{*})$.
	\end{enumerate}
\end{enumerate}
\end{proposition}
\begin{proof}
Ad 1.: 
Let us denote by $\tuple{\dx{1}, \dnx{1}, \emptyset}$ an arbitrary q-partition for some query $Q$ w.r.t.\ $\mD$ and $\langle\mo,\mb,\Tp,\Tn\rangle_\RQ$. Assume that $Q$ consists of only explicit entailments. Then, $Q \subseteq \mo \cup \mb \cup U_\Tp$. Since $Q$ is a set of common entailments of all KBs in $\setof{\mo_i^{*}\,|\, \md_i \in \dx{1}}$ (cf.\ Proposition~\ref{prop:properties_of_q-partitions},(\ref{prop:properties_of_q-partitions:enum:query_is_set_of_common_ent}.)), we have that $Q \subseteq (\mo \setminus U_{\dx{1}}) \cup \mb \cup U_\Tp$. Moreover, by Proposition~\ref{prop:reduct_to_discax} and Corollary~\ref{cor:expl_ent_query_must_neednot_mustnot_include_ax}, we can conclude that the deletion of all axioms in $\mb \cup U_\Tp$ as well as all axioms not in $\DiscAx_\mD$ from $Q$ does not alter the q-partition associated with $Q$. Now, we have to distinguish two cases: 
\begin{enumerate}[label=(\roman*)]
	\item $Q = \mo \setminus U_{\dx{1}} \cap \DiscAx_\mD = E_\mathsf{exp}(\dx{1}) \cap \DiscAx_\mD$ and
	\item $Q \subset (\mo \setminus U_{\dx{1}}) \cap \DiscAx_\mD = E_\mathsf{exp}(\dx{1}) \cap \DiscAx_\mD$.
\end{enumerate}

In the first case (i), since, by Definition~\ref{def:canonical_query}, $E_\mathsf{exp}(\dx{1}) \cap \DiscAx_\mD = Q_{\mathsf{can}}(\dx{1})$, we observe that $Q$ is equal to the canonical query w.r.t.\ the non-empty seed $\mathbf{S} := \dx{1} \subset \mD$ which lets us immediately conclude that $Q_{\mathsf{can}}(\mathbf{S})$ has the same q-partition as $Q$ -- contradiction. 

In the second case (ii), 
we can exploit Proposition~\ref{prop:explicit-ents_query_lower+upper_bound} (which we state and prove later) which gives a lower and upper bound (in terms of subset relationships) for explicit-entailments queries $Q' \subseteq \DiscAx_\mD$. Concretely, given a query $Q'' \subseteq \DiscAx_\mD$ with associated q-partition $\tuple{\dx{}, \dnx{}, \emptyset}$ (which partitions $\mD$) it states that $Q' \subseteq \DiscAx_\mD$ is a query associated with this q-partition iff $Q'$ is a superset or equal to a minimal hitting set of all elements in $\dnx{}$ and a subset or equal to $E_{\mathsf{exp}}(\dx{}) \cap \DiscAx_\mD =: Q_{\mathsf{can}}(\dx{})$. 
Now, what we know by the assumption of case~(ii) is that $Q \subset E_\mathsf{exp}(\dx{}(Q)) \cap \DiscAx_\mD \subseteq \DiscAx_\mD$ and that $\tuple{\dx{1}, \dnx{1}, \emptyset}$ is the q-partition associated with $Q$. By Proposition~\ref{prop:explicit-ents_query_lower+upper_bound}, we can now deduce, as long as there is a minimal hitting set of $\dnx{1}$ which is a subset of $Q$, that $\tuple{\dx{1}, \dnx{1}, \emptyset}$ is also the q-partition for $Q_{\mathsf{can}}(\dx{1})$. Therefore, we have found a non-empty seed $\mathbf{S} := \dx{1} \subset \mD$ such that $Q_{\mathsf{can}}(\mathbf{S})$ has the same q-partition as $Q$ -- contradiction. 

What remains open is the case when there is no minimal hitting set of $\dnx{1}$ which is a subset of $Q$. This is equivalent to $Q$ being no hitting set of $\dnx{1}$. Let w.l.o.g.\ $\md_1,\dots,\md_k$ ($k\geq 1$) be the diagnoses in $\dnx{}$ that have an empty set intersection with $Q$. Then, since $Q \subset E_\mathsf{exp}(\dx{1}) \cap \DiscAx_\mD \subseteq \mo$ and $\md_1,\dots,\md_k \subseteq \mo$, we have that $Q \subseteq \mo\setminus \md_r$ for all $r\in\setof{1,\dots,k}$. Therefore, $\mo_r^{*} \models Q$ for all $r\in\setof{1,\dots,k}$ and thus $\setof{\md_1,\dots,\md_k} \subseteq \dx{1}$ (cf.\ Definition~\ref{def:q-partition}). However, this is a contradiction to the assumption that $\tuple{\dx{1}, \dnx{1}, \emptyset}$ is the \emph{q-partition} associated with $Q$. The reason is that every q-partition must be a partition of $\mD$ (Proposition~\ref{prop:properties_of_q-partitions},(\ref{prop:properties_of_q-partitions:enum:q-partition_is_partition}.)) and $\dx{1}$ has no empty set intersection with $\dnx{1}$ because of $\emptyset \subset \setof{\md_1,\dots,\md_k} \subseteq \dx{1} \cap \dnx{1}$.

Ad 2.: Suppose that $Q$ includes only entailments of $\bigcap_{\setof{i\,|\,\md_i\in\dx{}(Q)}} \mo_i^{*}$. It is easy to derive that $\bigcap_{\setof{i\,|\,\md_i\in\dx{}(Q)}} \mo_i^{*} = (\mo \setminus U_{\dx{}(Q)}) \cup \mb \cup U_\Tp$. By Proposition~\ref{prop:reduct_to_discax} and Corollary~\ref{cor:expl_ent_query_must_neednot_mustnot_include_ax}, the deletion of all axioms in $\mb \cup U_\Tp$ as well as all axioms not in $\DiscAx_\mD$ from $Q$ does not alter the q-partition associated with $Q$. Hence, the query equal to $\mo \setminus U_{\dx{}(Q)} \cap \DiscAx_\mD = E_\mathsf{exp}(\dx{}(Q)) \cap \DiscAx_\mD = Q_{\mathsf{can}}(\dx{}(Q))$ has the same q-partition as $Q$, wherefore we have found a non-empty seed $\mathbf{S} := \dx{}(Q) \subset \mD$ for which $\tuple{\dx{}(Q_{\mathsf{can}}(\mathbf{S})), \dnx{}(Q_{\mathsf{can}}(\mathbf{S})), \emptyset} = \tuple{\dx{}(Q), \dnx{}(Q), \emptyset}$, which is a contradiction.

Now we show that (2.)\ can only be true if $|\dx{}(Q)| \geq 2$. For, the assumption of the opposite, i.e.\ $|\dx{}(Q)| = 1$ or equivalently $\dx{}(Q) = \setof{\md_i}$, means that $Q$ must not include an entailment of $\mo_i^{*}$. However, as $Q \neq \emptyset$ (Definition~\ref{def:query}) and $Q$ is a subset of the entailments of $\mo_i^{*}$ (Proposition~\ref{prop:properties_of_q-partitions},(\ref{prop:properties_of_q-partitions:enum:query_is_set_of_common_ent}.)). This is a contradiction.

Ad 3.: Let us assume that $\dx{}(Q_{\mathsf{can}}(\dx{}(Q))) \subseteq \dx{}(Q)$. Then, either $\dx{}(Q_{\mathsf{can}}(\dx{}(Q))) = \dx{}(Q)$ or $\dx{}(Q_{\mathsf{can}}(\dx{}(Q))) \subset \dx{}(Q)$. In the former case, we directly obtain a contradiction since clearly $\tuple{\dx{}(Q_{\mathsf{can}}(\mathbf{S})), \dnx{}(Q_{\mathsf{can}}(\mathbf{S})), \emptyset} = \tuple{\dx{}(Q), \dnx{}(Q), \emptyset}$ for the non-empty seed $\mathbf{S} := \dx{}(Q) \subset \mD$. In the latter case, there is some $\md_i \in \dx{}(Q) \setminus \dx{}(Q_{\mathsf{can}}(\dx{}(Q)))$. Since $\dz{}(Q)$ is assumed to be the empty set and since $Q_{\mathsf{can}}(\dx{}(Q))$ is an explicit-entailments query which must satisfy $\dz{}(Q_{\mathsf{can}}(\dx{}(Q))) =\emptyset$ by Proposition~\ref{prop:d0}, $\md_i \in \dnx{}(Q_{\mathsf{can}}(\dx{}(Q)))$ must hold. This implies that $\mo_i^{*} \cup Q_{\mathsf{can}}(\dx{}(Q))$ violates some $x \in \RQ \cup \Tn$. But, as $\md_i \in \dx{}(Q)$ and as the entailment relation in $\mathcal{L}$ is extensive, we must also have that $\mo_i^{*} = (\mo \setminus \md_i) \cup \mb \cup U_\Tp \supseteq (\mo \setminus U_{\dx{}(Q)}) \cup \mb \cup U_\Tp \supseteq (\mo \setminus U_{\dx{}(Q)}) \supseteq (\mo \setminus U_{\dx{}(Q)}) \cap \DiscAx_\mD = E_\mathsf{exp}(\dx{}(Q)) \cap \DiscAx_\mD = Q_{\mathsf{can}}(\dx{}(Q)) \models Q_{\mathsf{can}}(\dx{}(Q))$. By idempotence of the entailment relation in $\mathcal{L}$, we obtain that $\mo_i^{*} \cup Q_{\mathsf{can}}(\dx{}(Q)) \equiv \mo_i^{*}$. Altogether, we have derived that $\mo_i^{*}$ violates some $x \in \RQ \cup \Tn$ which is a contradiction to $\md_i$ being a diagnosis w.r.t.\ $\langle\mo,\mb,\Tp,\Tn\rangle_\RQ$.

Ad 4.: Assume that $J^*$ is a conflict w.r.t.\ $\langle\mo,\mb,\Tp,\Tn\rangle_\RQ$ for some $J^* \in  \bigcup_{\setof{i\,|\,\md_i\in\dx{}(Q)}} \Just(Q,\mo_i^{*})$. However, for $J^*$, by Definition~\ref{def:justification}, $J^* \subseteq \mo_{i^*}^{*}$ must hold for some $i^* \in \setof{i\,|\,\md_i\in\dx{}(Q)}$. Due to the fact that $\md_{i^*}$ is a diagnosis w.r.t.\ $\langle\mo,\mb,\Tp,\Tn\rangle_\RQ$, we see (by Propositions~\ref{prop:notions_equiv} and \ref{prop:mindiag_mincs}) that $\mo_{i^*}^{*}$ cannot comprise any conflict sets w.r.t.\ $\langle\mo,\mb,\Tp,\Tn\rangle_\RQ$. Hence, $J^*$ cannot be a conflict set w.r.t.\ $\langle\mo,\mb,\Tp,\Tn\rangle_\RQ$.

Ad 5.(a): Assume that $\md_m \not\subset U_{\dx{}(Q)}$ for all $\md_m \in \dnx{}(Q)$. Then, for all $\md_m \in \dnx{}(Q)$, there is one axiom $\tax_m \in \md_m$ such that $\tax_m \notin U_{\dx{}(Q)}$. Additionally, $\tax_m \notin \md$ for at least one $\md \in \dx{}(Q)$ (as otherwise $\tax_m \in U_{\dx{}(Q)}$ would be true), i.e.\ $\tax_m \notin I_\mD$, and $\tax_m \in U_\mD$ since $\tax_m \in \md_m$ with $\md_m \in \mD$. So, $\tax_m \in \DiscAx_\mD$ must be valid. Because of $Q_{\mathsf{can}}(\dx{}(Q)) = (\mo \setminus U_{\dx{}(Q)}) \cap \DiscAx_\mD$, we obtain that $\tax_m \in Q_{\mathsf{can}}(\dx{}(Q))$.  Therefore, for all $\md_m \in \dnx{}(Q)$, it must be true that $\mo_m^{*} \cup Q_{\mathsf{can}}(\dx{}(Q))$ violates some $x \in \RQ \cup \Tn$ by the subset-minimality of the diagnoses $\md_m \in \dnx{}(Q)$. Consequently, $\md_m \in \dnx{}(Q_{\mathsf{can}}(\dx{}(Q)))$ for all $\md_m \in \dnx{}(Q)$ which implies that $\dnx{}(Q_{\mathsf{can}}(\dx{}(Q))) = \dnx{}(Q)$. Hence, (cf.\ argumentation in the proof of (3.)\ before) also $\dx{}(Q_{\mathsf{can}}(\dx{}(Q))) = \dx{}(Q)$. This, however, means that we have found a non-empty seed $\mathbf{S} := \dx{}(Q) \subset \mD$ for which $\tuple{\dx{}(Q_{\mathsf{can}}(\mathbf{S})), \dnx{}(Q_{\mathsf{can}}(\mathbf{S})), \emptyset} = \tuple{\dx{}(Q), \dnx{}(Q), \emptyset}$, which is a contradiction.

Ad 5.(b): By (5.(a)), we know that there is some $\md_m \in \dnx{}(Q)$ such that $\md_m \subset U_{\dx{}(Q)}$. Since $\md_m \in \dnx{}(Q)$, it holds that $\mo_m^{*} \cup Q$ violates some $x \in \RQ \cup \Tn$. Further on, as $\mo_i^{*} \cup Q$ does not violate any $x \in \RQ \cup \Tn$ for all $\md_i \in \dx{}(Q)$, we can derive by the monotonicity of $\mathcal{L}$ that $(\mo \setminus U_{\dx{}(Q)}) \cup \mb \cup U_\Tp$ cannot violate any $x \in \RQ \cup \Tn$ either. A conflict set $\mc$ w.r.t.\ $\langle\mo\setminus\md_m,\mb,\Tp\cup\setof{Q},\Tn\rangle_\RQ$ is a subset of $\mo\setminus\md_m$ such that $\mc_m \cup \mb \cup U_\Tp \cup Q$ violates some $x \in \RQ\cup\Tn$. Because of the fact that $\md_m \subset U_{\dx{}(Q)}$, we have that $\mo_m^{*} \cup Q \supset (\mo \setminus U_{\dx{}(Q)}) \cup \mb \cup U_\Tp$. The set difference of these two sets is then exactly $U_{\dx{}(Q)} \setminus \md_m$. As the smaller set does not contain any conflicts whereas the larger set does so, we can conclude that each conflict set for the larger set, i.e.\ each conflict set $\mc$ w.r.t.\ $\langle\mo\setminus\md_m,\mb,\Tp\cup\setof{Q},\Tn\rangle_\RQ$, must contain at least one axiom in $U_{\dx{}(Q)} \setminus \md_m$.

Ad 5.(c): Assume that $J \cap \md_m = \emptyset$ for some $J \in  \bigcup_{\setof{i\,|\,\md_i\in\dx{}(Q)}} \Just(Q,\mo_i^{*})$. That is, there is some $i^* \in \setof{i\,|\,\md_i\in\dx{}(Q)}$ such that $J \in \Just(Q,\mo_{i^*}^{*})$. Since however $J \cap \md_m = \emptyset$, we can deduce (cf.\ Lemma~\ref{lem:U_D_notin_B_and_notin_U_P}) that $J \subseteq (\mo_{i^*}^{*} \setminus \md_m) = [(\mo \setminus \md_{i^*}) \cup \mb \cup U_\Tp]\setminus \md_m = (\mo \setminus (\md_{i^*}\cup\md_m)) \cup \mb \cup U_\Tp \subseteq (\mo \setminus \md_m) \cup \mb \cup U_\Tp = \mo_m^{*}$. So, by monotonicity of $\mathcal{L}$, we have that $\mo_m^{*} \models Q$. On the other hand, $\md_m \in \dnx{}(Q)$ implies that $\mo_m^{*} \cup Q$ violates some $x \in \RQ \cup \Tn$. Since the entailment relation in $\mathcal{L}$ is idempotent, it is thus true that $\mo_m^{*} \cup Q \equiv \mo_m^{*}$. This yields that $\mo_m^{*}$ violates some $x \in \RQ \cup \Tn$ which is a contradiction to the fact that $\md_m$ is a diagnosis w.r.t.\ $\langle\mo,\mb,\Tp,\Tn\rangle_\RQ$.
\end{proof}

We want to point out that the non-fulfillment of both some criterion given by Proposition~\ref{prop:necess_cond_when_q-partition_is_not_canonical_CASE_a} and some criterion given by Proposition~\ref{prop:necess_cond_when_q-partition_is_not_canonical_CASE_b} leads to the impossibility of any q-partitions that are not canonical. 

For instance, a comprehensive study~\cite{Horridge2012b} on entailments and justifications dealing with a large corpus of real-world KBs (ontologies) from the Bioportal Repository\footnote{http://bioportal.bioontology.org} reveals that the probability for a q-partition to 
be non-canonical
can be rated pretty low in the light of the sophisticated requirements enumerated by Propositions~\ref{prop:necess_cond_when_q-partition_is_not_canonical_CASE_a} and \ref{prop:necess_cond_when_q-partition_is_not_canonical_CASE_b}. In particular, the study showed that 
\begin{enumerate}[label=(i)]
	\item only one third, precisely 72 of 214, of the KBs featured so-called non-trivial entailments.
\end{enumerate}
A logical sentence $\alpha$ with $\mo \models \alpha$ is a \emph{non-trivial entailment} of the KB $\mo$ iff 
$\mo \setminus \setof{\alpha} \models \alpha$, i.e.\ $\alpha$ is either not explicitly stated as KB axiom or is still entailed by the KB after having been deleted from it. Otherwise, $\alpha$ is termed \emph{trivial entailment}. In other words, an entailment $\alpha$ is non-trivial iff there is at least one justification for $\alpha$ that does not include $\alpha$. 

\begin{example}\label{ex:non-trivial_entailments}
For instance, all explicit entailments of the KB $\mo$ in our example DPI given by Table~\ref{tab:example_dpi_0}, i.e.\ all $\alpha \in \mo$, are trivial entailments. So, the deletion of, e.g., $\tax_5 = K \to E$ cannot be ``compensated'' by the rest of the KB, i.e.\ $\mo \setminus \setof{\tax_5} \not\models K \to E$.  

An example of a non-trivial entailment of $\mo$ is $\alpha := F \to \lnot G$ which is not stated explicitly in $\mo$, but a logical consequence of $\tax_2 = X \vee F \to H$ and $\tax_1 = \lnot H \vee \lnot G$, i.e.\ $\setof{\tax_1,\tax_2}$ is a justification of $\alpha$ that does not include $\alpha$.\qed
\end{example}
Note that an implicit entailment as per our definition (see page~\pageref{etc:def_explicit_implicit_entailment}) is always a non-trivial entailment since it is not an axiom in the KB. 
Consequently, at least two thirds of the Bioportal KBs did not exhibit any implicit entailments.
By Condition~(2.)\ of Proposition~\ref{prop:necess_cond_when_q-partition_is_not_canonical_CASE_a} as well as Condition~(1.)\ of Proposition~\ref{prop:necess_cond_when_q-partition_is_not_canonical_CASE_b}, there cannot be any non-canonical q-partitions in such a KB.

Further on, the study unveiled that 
\begin{enumerate}[label=(ii), resume*]
	\item in more than a third (25) of the remaining 72 KBs the average number of justifications per non-trivial entailment was lower than 2.
\end{enumerate}
Such entailments, however, 
do not satisfy 
Condition (3.1) of Proposition~\ref{prop:necess_cond_when_q-partition_is_not_canonical_CASE_a}, 
unless there is a single justification that is a subset of all KBs $\mo_i^*$ for $\md_i\in\dx{}(Q)$. Furthermore, such entailments do not comply with Condition~(2.)\ of Proposition~\ref{prop:necess_cond_when_q-partition_is_not_canonical_CASE_b}. 

Additionally, it was reported that
\begin{enumerate}[label=(iii), resume*]
	\item the average size of a justification for a non-trivial entailment was measured to be lower than 4 for 85\% (61) and below 2 axioms for a good half (43) of the 72 KBs.
\end{enumerate}
On average, a KB among those 72 had 10 645 axioms. Now, let us assume in such a KB a diagnosis $\md_j$ with $|\md_j|=10$ that -- by Condition~(3.2) of Proposition~\ref{prop:necess_cond_when_q-partition_is_not_canonical_CASE_a} -- must be hit by every justification $J$ for $S_j$. Then the probability for $J$ with $|J|=4$ and $|J|=2$ to contain at least one axiom of $\md_j$ is roundly $0.004$ and $0.002$.\footnote{The latter two values can be easily computed by means of the Hypergeometric Distribution as $1 - p(S = 0)$ where ${S \sim Hyp_{n,N,M}}$ measures the number of successes (number of ``good ones'') when drawing $n$ elements from a set of $N$ elements with $M$ ``good ones'' among the $N$. In that, 
$p(S = s) = \binom{M}{s}\binom{N-M}{n-s}/\binom{N}{n}$. 
For the calculation of the two values, the parameters $\tuple{n,N,M}=\tuple{4,10 645,10}$ and $\tuple{n,N,M}=\tuple{2,10 645,10}$ are used.} Fulfillment of Condition~(5.(c)) of Proposition~\ref{prop:necess_cond_when_q-partition_is_not_canonical_CASE_b} can be analyzed in a very similar way.
Thus, the necessary conditions for non-canonical q-partitions will be hardly satisfied on average in this dataset of KBs.

Consequently, albeit we might (in case non-canonical q-partitions do exist) give up perfect completeness of q-partition search by the restriction to only canonical q-partitions, we have argued based on one comprehensive real-world dataset of KBs that in practice there might be high probability that we might miss none or only very few q-partitions. We can cope well with that since canonical queries and q-partitions bring along nice computational properties. Moreover, given practical numbers of leading diagnoses per iteration, e.g.\ $\approx 10$, cf.\ \cite{Shchekotykhin2012,Rodler2013}, the number of canonical q-partitions considered this way will prove to be still large enough to identify (nearly) optimal q-partitions (and queries) for all discussed measures, as our preliminary experiments (still unpublished) suggest. Theoretical support for this is given e.g.\ by Corollary~\ref{cor:upper_lower_bound_for_canonical_q-partitions}.

\subsubsection{The Search for Q-Partitions}\label{sec:search_for_q-partitions}

\paragraph{Overall Algorithm.}
We now turn to the specification of the search procedure for q-partitions. According to \cite{russellnorvig2010}, a search problem can be described by giving the \emph{initial state}, a \emph{successor function}\footnote{For problems in which only the possible actions in a state, but not the results (i.e.\ neighbor states) of these actions are known, one might use a more general way of specifying the successor function. One such way is called \emph{transition model}, see \cite[p.~67]{russellnorvig2010}. For our purposes in this work, the notion of a successor function is sufficiently general.} enumerating all direct neighbor states of a given state, the \emph{step/path costs} from a given state to a successor state, some \emph{heuristics} which should estimate the remaining effort (or: additional steps) towards a goal state from some given state, and the \emph{goal test} which determines whether a given state is a goal state or not. In the case of our search for an optimal q-partition according to some query quality measure $m$ (see Section~\ref{sec:ActiveLearningInInteractiveOntologyDebugging}), these parameters can be instantiated as follows:
\begin{itemize}
	\item \emph{Initial State.}\label{etc:initial_state} 
	There are two natural candidates for initial states in our q-partition search, either $\langle\emptyset,\mD,\emptyset\rangle$ or $\langle\mD,\emptyset,\emptyset\rangle$. Note that both of these initial states are partitions of $\mD$, but no q-partitions -- however, all other states in the search will be (canonical) q-partitions. These initial states induce two different searches, one starting from empty $\dx{}$ which is successively filled up by transferring diagnoses from $\dnx{}$ to $\dx{}$, and the other conducting the same procedure with $\dx{}$ and $\dnx{}$ interchanged. We will call the former $\dx{}$-partitioning and the latter $\dnx{}$-partitioning search.
	\item \emph{Successor Function.} To specify a suitable successor function, we rely on the notion of a minimal transformation. In that, $\dx{}$-partitioning and $\dnx{}$-partitioning search will require slightly different characterizations of a minimal transformation: 
	\begin{definition}\label{def:minimal_transformation}
	Let $\mD \subseteq \minD_{\langle\mo,\mb,\Tp,\Tn\rangle_\RQ}$, $\Pt_i := \langle \dx{i},\dnx{i},\emptyset\rangle$ be a partition (not necessarily a q-partition) of $\mD$ and $\Pt_j := \langle \dx{j},\dnx{j},\emptyset\rangle$ be a canonical q-partition w.r.t.\ $\mD$ such that $\langle \dx{j},\dnx{j},\emptyset\rangle \neq \langle \dx{i},\dnx{i},\emptyset\rangle$. Then, we call 
	\begin{itemize}
		\item $\Pt_i \mapsto \Pt_j$ a \emph{minimal $\dx{}$-transformation} from $\Pt_i$ to $\Pt_j$ iff $\dx{i} \subset \dx{j}$ and there is no canonical q-partition $\langle \dx{k},\dnx{k},\emptyset\rangle$ such that $\dx{i} \subset \dx{k} \subset \dx{j}$.
		\item $\Pt_i \mapsto \Pt_j$ a \emph{minimal $\dnx{}$-transformation} from $\Pt_i$ to $\Pt_j$ iff $\dnx{i} \subset \dnx{j}$ and there is no canonical q-partition $\langle \dx{k},\dnx{k},\emptyset\rangle$ such that $\dnx{i} \subset \dnx{k} \subset \dnx{j}$.
	\end{itemize}
	\end{definition}
	
	The successor function then maps a given partition $\Pt$ w.r.t.\ $\mD$ to the set of all its possible successors, i.e.\ to the set including all canonical q-partitions that result from $\Pt$ by a minimal transformation.
	
	The reliance upon a minimal transformation guarantees that the search is complete w.r.t.\ canonical q-partitions 
	because we cannot skip over any canonical q-partitions when transforming a state into a direct successor state. In the light of the initial states for $\dx{}$- as well as $\dnx{}$-partitioning being no q-partitions, the definition of the successor function involves specifying 
	\begin{itemize}
		\item a function $S_{\mathsf{init}}$ that maps the initial state to the set of all canonical q-partitions that can be reached by it by a single minimal transformation, and
		\item a function $S_{\mathsf{next}}$ that maps any canonical q-partition to the set of all canonical q-partitions that can be reached by it by a single minimal transformation
	\end{itemize}
As $\dx{}$- and $\dnx{}$-partitioning require different successor functions, they must be treated separately. In this work, we restrict ourselves to $\dx{}$-partitioning. The derivation and specification of a \emph{sound} and \emph{complete} successor function $S_{\mathsf{init}} \cup S_{\mathsf{next}}$ (computing \emph{only} and \emph{all} canonical q-partitions resulting from any state in the search by a minimal $\dx{}$-transformation) will be treated in Section~\ref{sec:AlgorithmForSuccessorComputation} after the overall algorithm has been discussed in what follows.
	\item \emph{Path Costs / Heuristics.} Since the step costs from one q-partition to a direct successor q-partition of it is of no relevance to our objective, which is finding a sufficiently ``good'' q-partition, we do not specify any step costs. In other words, we are not interested in the way how we constructed the optimal q-partition starting from the initial state, but only in the shape of the optimal q-partition as such. However, what we do introduce is some estimate of ``goodness'' of q-partitions in terms of heuristics. These heuristics are dependent on the used query quality measure $m$ and are specified based on the respective optimality criteria given by requirements $r_m$ (cf.\ Sections~\ref{sec:ActiveLearningInInteractiveOntologyDebugging} and \ref{sec:QPartitionRequirementsSelection}). We characterize these for all measures discussed in Section~\ref{sec:ActiveLearningInInteractiveOntologyDebugging} in the procedure \textsc{heur} in Algorithm~\ref{algo:dx+_best_successor}.
	\item \emph{Goal Test.} A canonical q-partition is considered as a goal state of the search iff it meets the requirements $r_m$ to a ``sufficient degree''. The latter is predefined in terms of some optimality threshold $\tr$. So, a state (canonical q-partition) is a goal if its ``distance'' to optimality, i.e.\ perfect fulfillment of $r_m$, is smaller than $\tr$ (cf.\ \cite{Shchekotykhin2012}). The goal test for all measures discussed in Section~\ref{sec:ActiveLearningInInteractiveOntologyDebugging} is described in Algorithm~\ref{algo:dx+_opt}.
\end{itemize}

The overall search algorithm for finding a (sufficiently) optimal q-partition is presented by Algorithm~\ref{algo:dx+_part}. The inputs to the algorithm are a set $\mD$ of leading minimal diagnoses w.r.t.\ the given DPI $\tuple{\mo,\mb,\Tp,\Tn}_\RQ$, the requirements $r_m$ to an optimal q-partition, a threshold $t_m$ describing the distance from the theoretical optimum of $r_m$ (cf.\ Table~\ref{tab:requirements_for_equiv_classes_of_measures_wrt_equiv_mQ}) within which we still consider a value as optimal, and a probability measure $p$ defined over the sample space $\mD$. The latter assigns a probability $p(\md)\in (0,1]$ to all $\md \in \mD$ such that $\sum_{\md\in\mD} p(\md) = 1$. To be precise, $p(\md)$ describes the probability that $\md$ is the true diagnosis assuming that one diagnosis in $\mD$ must be the correct one.\footnote{For a more in-depth treatment of the diagnosis probability space and a discussion where such probabilities might originate from, see \cite[Sec.~4.6]{Rodler2015phd}} Moreover, $p(\mathbf{S}) = \sum_{\md\in\mathbf{S}} p(\md)$ for any $\mathbf{S} \subseteq \mD$. The output of the algorithm is a canonical q-partition that is optimal (as per $r_m$) within the tolerance $t_m$ if such a q-partition exists, and the best (as per $r_m$) among all canonical q-partitions w.r.t.\ $\mD$, otherwise.

The algorithm consists of a single call of \textsc{$\dx{}$Partition} with forwarded arguments $\tuple{\emptyset,\mD,\emptyset}$, $\langle\emptyset,\mD$, $\emptyset\rangle$, $\emptyset$, $p,t_m,r_m$ which initiates the recursive search procedure. The type of the implemented search can be regarded as a depth-first, ``local best-first'' backtracking algorithm. \emph{Depth-first} because, starting from the initial partition $\tuple{\emptyset,\mD,\emptyset}$ (or:\ root node), the search will proceed \emph{downwards} until (a)~an optimal canonical q-partition has been found, (b)~all successors of the currently analyzed q-partition (or:\ node) have been pruned or (c)~there are no successors of the currently analyzed q-partition (or:\ node). \emph{Local best-first} because the search, at each current q-partition (or:\ node), moves on to the best \emph{successor} q-partition (or: \emph{child} node) as per some heuristics constructed based on $r_m$. \emph{Backtracking} because the search procedure is ready to backtrack in case all successors (or:\ children) of a q-partition (or:\ node) have been explored and no optimal canonical q-partition has been found yet. In this case, the next-best unexplored sibling of the node will be analyzed next according to the implemented local best-first depth-first strategy.

\textsc{$\dx{}$Partition} expects six arguments. The first, $\Pt$, represents the currently analyzed q-partition or, equivalently, node in the search tree. The second, $\Pt_{\mathsf{b}}$, denotes the best q-partition (as per $r_m$) that has been discovered so far during the search. The third, $\mD_{\mathsf{used}}$, constitutes the set of diagnoses that must remain elements of the $\dnx{}$-set, i.e.\ must not be moved to the $\dx{}$-set, for all q-partitions generated as direct or indirect successors of $\Pt$. The elements of $\mD_{\mathsf{used}}$ are exactly those diagnoses in $\mD$ for which all existing canonical q-partitions $\Pt_x$ with $\dx{}(\Pt_x) \supseteq \mD_{\mathsf{used}}$ have already been explored so far, i.e.\ these diagnoses have already been \emph{used} as elements of the $\dx{}$-set. Hence, we do not want to add any of these to any $\dx{}$-set of a generated q-partition anymore. 
The last three arguments $p, t_m$ as well as $r_m$ are explained above. We point out that only the first three arguments $\Pt$, $\Pt_{\mathsf{b}}$ and $\mD_{\mathsf{used}}$ vary throughout the entire execution of the search. All other parameters remain constant.  

The first step within the \textsc{$\dx{}$Partition} procedure (line~\ref{algoline:dx+_part:updateBest}) is the update of the best q-partition found so far. This is accomplished by the function \textsc{updateBest} (see Algorithm~\ref{algo:dx+_update_best} on page~\pageref{algo:dx+_update_best}) which returns the partition among $\setof{\Pt,\Pt_{\mathsf{b}}}$ which is better w.r.t.\ $r_m$. For some of the query quality measures $m$, this will require calculations involving probabilities, which is why the parameter $p$ is handed over to the function as well. After the execution of \textsc{updateBest}, the best currently known q-partition is stored in $\Pt_{\mathsf{best}}$. 

The next step (line~\ref{algoline:dx+_part:opt}) involves a check of $\Pt_{\mathsf{best}}$ for optimality w.r.t.\ the query quality requirements $r_m$ and the threshold $t_m$ and is realized by the function \textsc{opt} (see Algorithm~\ref{algo:dx+_opt} on page~\pageref{algo:dx+_opt}). The latter returns $\true$ iff $\Pt_{\mathsf{best}}$ is an optimal q-partition w.r.t.\ $r_m$ and the threshold $t_m$. If optimality of $\Pt_{\mathsf{best}}$ is given, then the most recent call to \textsc{$\dx{}$Partition} returns (line~\ref{algoline:dx+_part:return--opt_found}) and passes back the tuple $\tuple{\Pt_{\mathsf{best}},\true}$ to the one level higher point in the recursion where the call was made. The parameter $\true$ in this tuple is a flag that tells the caller of \textsc{$\dx{}$Partition} that $\Pt_{\mathsf{best}}$ is already optimal an no further search activities are required.

In case optimality of $\Pt_{\mathsf{best}}$ is not satisfied, the procedure moves on to line~\ref{algoline:dx+_part:prune} where a pruning test is executed, implemented by the function \textsc{prune} (see Algorithm~\ref{algo:dx+_prune} on page~\pageref{algo:dx+_prune}). The latter, given the inputs $\Pt, \Pt_{\mathsf{best}}, p$ and $r_m$, evaluates to $\true$ if exploration of successor q-partitions of the currently analyzed q-partition $\Pt$ cannot lead to the detection of q-partitions that are better w.r.t.\ $r_m$ than $\Pt_{\mathsf{best}}$. If the pruning test returns positively, the tuple $\tuple{\Pt_{\mathsf{best}},\false}$ is returned, where $\false$ signalizes that $\Pt_{\mathsf{best}}$ is not optimal w.r.t.\ $r_m$ and $t_m$.

Facing a negative pruning test, the algorithm continues at line~\ref{algoline:dx+_part:getD+Sucs}, where the function \textsc{get$\dx{}$Sucs} (see Algorithm~\ref{algo:dx+_suc} on page~\pageref{algo:dx+_suc}) is employed to generate and store in $sucs$ all successors of the currently analyzed partition $\Pt$ which result from $\Pt$ by a minimal $\dx{}$-transformation. 

Given the set $sucs$ of all direct successors, the algorithm enters the \textbf{while}-loop in order to examine all successors in turn. To this end, the algorithm is devised to select always the best not-yet-explored successor q-partition in $sucs$ according to some heuristics based on $r_m$. This selection is implemented by the function \textsc{bestSuc} (see Algorithm~\ref{algo:dx+_best_successor} on page~\pageref{algo:dx+_best_successor}) in line~\ref{algoline:dx+_part:bestSuc} which gets $\Pt$, $sucs$, $p$ and $r_m$ as inputs. \textsc{bestSuc} returns the best direct successor $\Pt'$ of $\Pt$ w.r.t.\ $r_m$ and one diagnosis $\md$ which is an element of the $\dx{}$-set of this best successor q-partition and is not an element of the $\dx{}$-set of the current partition $\Pt$, i.e.\ $\md$ has been moved from $\dnx{}$ to $\dx{}$ in the context of the minimal $\dx{}$-transformation that maps $\Pt$ to $\Pt'$. 

\begin{remark}\label{rem:bestSuc_returned_diagnosis}
This diagnosis $\md$ serves as a representative of all diagnoses that are moved from $\dnx{}$ to $\dx{}$ in the context of this minimal transformation. It is later (after the next recursive call to \textsc{$\dx{}$Partition} in line~\ref{algoline:dx+_part:D+Partition_recursive_call}) added to $\mD_{\mathsf{used}}$ in line~\ref{algoline:dx+_part:update_mD_used}. This update of the set $\mD_{\mathsf{used}}$ effectuates that a q-partition $\Pt_x$ where $\md$ together with all other diagnoses in $\mD_{\mathsf{used}}$ is in $\dx{}(\Pt_x)$ will never be encountered again during the complete execution of \textsc{findQPartition}. The reason why this is desirable is that all existing canonical q-partitions with a $\dx{}$-set that is a superset of $\mD_{\mathsf{used}}$ have already been explored before the call of \textsc{$\dx{}$Partition} given $\Pt$ as argument, and all existing canonical q-partitions with a $\dx{}$-set that is a superset of $\mD_{\mathsf{used}} \cup \setof{\md}$ have already been explored during the execution of the most recent call of \textsc{$\dx{}$Partition} given $\Pt'$ as argument in line~\ref{algoline:dx+_part:D+Partition_recursive_call}. 

We point out that although only one representative $\md$ from the set of diagnoses that has been moved from $\dnx{}$ to $\dx{}$ in the context of the minimal transformation from $\Pt$ to $\Pt'$ is stored in $\mD_{\mathsf{used}}$, what holds for $\md$ also holds for all other diagnoses moved from $\dnx{}$ to $\dx{}$ together with $\md$ during the transformation. This is achieved by leveraging the fact (see Corollary~\ref{cor:nec_followers_form_equivalence_class_wrt_md^(y)} later) that all the diagnoses moved during a minimal $\dx{}$-transformation form an equivalence class w.r.t.\ a well-defined equivalence relation introduced later in Definition~\ref{def:equiv_rel}. 

Note that within the entire execution of any call of \textsc{$\dx{}$Partition} with arguments (i.a.) $\Pt$ and $\mD_{\mathsf{used}}$ during the execution 
of \textsc{findQPartition}, the set $\mD_{\mathsf{used}}$ will always comprise only diagnoses that are in $\dnx{}(\Pt)$.\qed
\end{remark}

Having computed the best direct successor $\Pt'$ of $\Pt$ by means of \textsc{bestSuc}, the algorithm proceeds to focus on $\Pt'$ which is now the currently analyzed q-partition. This is reflected in line~\ref{algoline:dx+_part:D+Partition_recursive_call} where the method \textsc{$\dx{}$Partition} calls itself recursively with the arguments $\Pt'$ (currently analyzed q-partition), $\Pt_{\mathsf{best}}$ (best currently known q-partition), $\mD_{\mathsf{used}}$ and the constant parameters $p$, $t_m$ and $r_m$. 

The output $\tuple{\Pt'',isOpt}$ of this recursive call of \textsc{$\dx{}$Partition} is then processed in lines~\ref{algoline:dx+_part:if_not_isOpt}-\ref{algoline:dx+_part:return--unwind_recursion_since_opt_found}. In this vein, if $isOpt = \false$, meaning that $\Pt''$ is not optimal w.r.t.\ $r_m$ and $t_m$, then the currently best known q-partition $\Pt_{\mathsf{best}}$ is updated and set to $\Pt''$. The idea behind this assignment is that, during any call of \textsc{$\dx{}$Partition}, the currently best known q-partition can only become better or remain unmodified, and cannot become worse. After this variable update, the just explored q-partition $\Pt'$ is eliminated from the set $sucs$ of successors of $\Pt$, and the next iteration of the \textbf{while}-loop is initiated.

If, on the other hand, $isOpt = \true$, then this indicates that an optimal q-partition w.r.t.\ $r_m$ and $t_m$ has been located and is stored in $\Pt''$. Hence, the algorithm just forwards the result $\tuple{\Pt'',\true}$ to the next higher level in the recursion (line~\ref{algoline:dx+_part:return--unwind_recursion_since_opt_found}). Note that this effectuates that the search immediately returns by unwinding the recursion completely once an optimal q-partition has been found.

Finally, in case none of the explorations of all successors $sucs$ of $\Pt$ has resulted in the discovery of an optimal q-partition, then $\tuple{\Pt_{\mathsf{best}},\false}$ is returned (i.e.\ the algorithm backtracks) in line~\ref{algoline:dx+_part:return--all_sucs_explored} and the execution continues one level higher in the recursion. Visually, concerning the search tree, this means that the parent node of $\Pt$ is again the currently analyzed node and a sibling of $\Pt$ is explored next.  

\begin{algorithm}[tb]
\small
\caption[$\dx{}$-Partitioning]{$\dx{}$-Partitioning (implements \textsc{findQPartition} of Algorithm~\ref{algo:query_comp})}\label{algo:dx+_part}
\begin{algorithmic}[1]
\Require set of minimal diagnoses $\mD$ w.r.t.\ a DPI $\tuple{\mo,\mb,\Tp,\Tn}_\RQ$ satisfying $|\mD| \geq 2$, probability measure $p$, threshold $t_m$, requirements $r_m$ to optimal q-partition 
\Ensure 
a canonical q-partition $\Pt$ w.r.t.\ $\mD$ that is optimal w.r.t.\ $m$ and $t_m$, if $isOptimal = \true$; the best of all canonical q-partitions w.r.t.\ $\mD$ and $m$, otherwise 
\Procedure{findQPartition}{$\mD, p, t_m ,r_m$}
\State $\tuple{\Pt,isOptimal} \gets$ \Call{$\dx{}$Partition}{$\langle \emptyset, \mD, \emptyset \rangle, \langle \emptyset, \mD, \emptyset \rangle, \emptyset, p, t_m, r_m$} \label{algoline:dx+_part:call_D+Partition_from_FindQPartition}
\State \Return $\Pt$
\EndProcedure
\vspace{2pt}
\hrule
\vspace{2pt}
\Procedure{$\dx{}$Partition}{$\Pt,\Pt_{\mathsf{b}}, \mD_{\mathsf{used}} ,p,t_m,r_m$}				\Comment{$p,t_m,r_m$ are constant throughout entire procedure}
\State $\Pt_{\mathsf{best}} \gets \Call{updateBest}{\Pt,\Pt_{\mathsf{b}},p,r_m}$ \label{algoline:dx+_part:updateBest}		\Comment{see Alg.~\ref{algo:dx+_update_best}}
\If{$\Call{opt}{\Pt_{\mathsf{best}},t_m,p,r_m}$}  \label{algoline:dx+_part:opt}						\Comment{see Alg.~\ref{algo:dx+_opt}}
	\State \Return $\tuple{\Pt_{\mathsf{best}},\true}$			\label{algoline:dx+_part:return--opt_found}						\Comment{$\Pt_{\mathsf{best}}$ is optimal q-partition w.r.t.\ $r_m$ and $t_m$}
\EndIf
\If{$\Call{prune}{\Pt,\Pt_{\mathsf{best}},p,r_m}$}   \label{algoline:dx+_part:prune}   \Comment{see Alg.~\ref{algo:dx+_prune}}
	\State \Return $\tuple{\Pt_{\mathsf{best}},\false}$			\label{algoline:dx+_part:return--pruned}						\Comment{all descendant q-partitions of $\Pt$ are no better than $\Pt_{\mathsf{best}}$ w.r.t.\ $r_m$}
\EndIf
\State $sucs \gets \Call{get$\dx{}$Sucs}{\Pt, \mD_{\mathsf{used}}}$	\label{algoline:dx+_part:getD+Sucs}
\While{$sucs \neq \emptyset$}
\State  $\tuple{\Pt',\md} \gets \Call{bestSuc}{\Pt,sucs,p,r_m}$		\label{algoline:dx+_part:bestSuc}  \Comment{see Alg.~\ref{algo:dx+_best_successor}; $\md$ is some diagnosis in $\dx{}(\Pt') \setminus \dx{}(\Pt)$}
\State  $\tuple{\Pt'',isOpt} \gets \Call{$\dx{}$Partition}{\Pt',\Pt_{\mathsf{best}},\mD_{\mathsf{used}},p,t_m,r_m}$   \label{algoline:dx+_part:D+Partition_recursive_call}
\State $\mD_{\mathsf{used}} \gets \mD_{\mathsf{used}} \cup \setof{\md}$   \label{algoline:dx+_part:update_mD_used}
\If{$\lnot isOpt$}  \label{algoline:dx+_part:if_not_isOpt}
	\State $\Pt_{\mathsf{best}} \gets \Pt''$	\label{algoline:dx+_part:opt_not_found_continue}	\Comment{optimal q-partition not found, continue with next successor in $sucs$}
\Else
	\State \Return $\tuple{\Pt'',\true}$		\label{algoline:dx+_part:return--unwind_recursion_since_opt_found}		\Comment{optimal q-partition found, unwind recursion completely and return $\Pt''$}
\EndIf
\State  $sucs \gets sucs \setminus \setof{\Pt'}$ \label{algoline:dx+_part:update_sucs}
\EndWhile
\State  \Return $\tuple{\Pt_{\mathsf{best}},\false}$	\label{algoline:dx+_part:return--all_sucs_explored}		\Comment{all successors in $sucs$ explored, continue at next higher recursion level}
\EndProcedure
\end{algorithmic}
\normalsize
\end{algorithm}

\begin{algorithm}[]
\small
\caption[Update of Best Q-Partition]{Update best q-partition}\label{algo:dx+_update_best}
\begin{algorithmic}[1]
\Require partitions $\Pt$ and $\Pt_{\mathsf{best}}$ (both with empty $\dz{}$) w.r.t.\ the set of leading diagnoses $\mD$, requirements $r_m$ to an optimal q-partition, a (diagnosis) probability measure $p$ 
\Ensure the better partition among $\setof{\Pt,\Pt_{\mathsf{best}}}$ w.r.t.\ $\mD$ and $r_m$; if both partitions are equally good w.r.t.\ $\mD$ and $r_m$, then $\Pt_{\mathsf{best}}$ is returned 
\Procedure{updateBest}{$\Pt,\Pt_{\mathsf{best}},p,r_m$}
\If{$\dx{}(\Pt_{\mathsf{best}}) = \emptyset \lor \dnx{}(\Pt_{\mathsf{best}}) = \emptyset$} \label{algoline:dx+_update_best:if_Pbest_is_no_q-partition_start}   \Comment{$\Pt_{\mathsf{best}}$ is not a q-partition, hence update anyway}
	\State $\Pt_{\mathsf{best}} \gets \Pt$ 				
	\State \Return $\Pt_{\mathsf{best}}$			\label{algoline:dx+_update_best:if_Pbest_is_no_q-partition_end}
\EndIf
\State $\mD \gets \dx{}(\Pt) \cup \dnx{}(\Pt)$	\label{algoline:dx+_update_best:reconstruct_leading_diags}						\Comment{reconstruct leading diagnoses from q-partition $\Pt$}
\If{$m \in \setof{\mathsf{ENT}}$}   \label{algoline:dx+_update_best:if_case_ENT}   \Comment{if $m$ in eq.\ class {\larger\textcircled{\smaller[2]1}} w.r.t.\ Table~\ref{tab:requirements_for_equiv_classes_of_measures_wrt_equiv_mQ}}
	\If{$|p(\dx{}(\Pt)) - 0.5| < |p(\dx{}(\Pt_{\mathsf{best}})) - 0.5|$}
		\State $\Pt_{\mathsf{best}} \gets \Pt$
	\EndIf
\EndIf
\If{$m \in \setof{\mathsf{SPL}}$}			\label{algoline:dx+_update_best:if_case_SPL}    \Comment{if $m$ in eq.\ class {\larger\textcircled{\smaller[2]2}} w.r.t.\ Table~\ref{tab:requirements_for_equiv_classes_of_measures_wrt_equiv_mQ}}
	\If{$\left|\,|\dx{}(\Pt)| - \frac{|\mD|}{2}\right| < \left|\,|\dx{}(\Pt_{\mathsf{best}})| - \frac{|\mD|}{2}\right|$}
		\State $\Pt_{\mathsf{best}} \gets \Pt$
	\EndIf
\EndIf
\If{$m \in \setof{\mathsf{RIO}}$} 	\label{algoline:dx+_update_best:if_case_RIO}	
\Comment{if $m$ in eq.\ class {\larger\textcircled{\smaller[2]3}} w.r.t.\ Table~\ref{tab:requirements_for_equiv_classes_of_measures_wrt_equiv_mQ}}
			\State $\uc \gets \Call{getCaut}{r_m}$		\label{algoline:dx+_update_best:getCaut}
			\State $n \gets \left\lceil \uc |\mD|\right\rceil$    \label{algoline:dx+_update_best:getMinimalNumOfDiagsToEliminate}
			\If{$|\dx{}(\Pt)| \geq n$}   	\label{algoline:dx+_update_best:if_CardDxPt_geq_n}
				\If{$|\dx{}(\Pt_{\mathsf{best}})| < n$}    \label{algoline:dx+_update_best:if_PtBest_high-risk_q-partition} 
					\State $\Pt_{\mathsf{best}} \gets \Pt$   \label{algoline:dx+_update_best:RIO_update_Pbest_1}  \Comment{$\Pt$ is non-high-risk and $\Pt_{\mathsf{best}}$ high-risk q-partition}
				\ElsIf{$\left|n - |\dx{}(\Pt)| \right| < \left|n - |\dx{}(\Pt_{\mathsf{best}})| \right|$}    \label{algoline:dx+_update_best:if_CardDxPt_closer_to_n_than_CardDxPtBest}        
					\State $\Pt_{\mathsf{best}} \gets \Pt$   \label{algoline:dx+_update_best:RIO_update_Pbest_2}\Comment{both $\Pt_{\mathsf{best}}$ and $\Pt$ non-high-risk, but $\Pt$ less cautious}
				\ElsIf{$\left|n - |\dx{}(\Pt)| \right| = \left|n - |\dx{}(\Pt_{\mathsf{best}})| \right|$}					\label{algoline:dx+_update_best:if_CardDxPt_equal_to_CardDxPtBest}
					\If{$\left|0.5 - p(\dx{}(\Pt))\right| < \left|0.5 - p(\dx{}(\Pt_{\mathsf{best}}))\right|$}			\label{algoline:dx+_update_best:if_ProbDxPt_closer_to_0.5_than_ProbDxPtBest}
						\State $\Pt_{\mathsf{best}} \gets \Pt$ \label{algoline:dx+_update_best:RIO_update_Pbest_3} \Comment{both non-high-risk, equally cautious, but $\Pt$ better $\mathsf{ENT}$ value}
					\EndIf
				\EndIf
			\EndIf
\EndIf
\If{$m \in \setof{\mathsf{KL}}$}				\label{algoline:dx+_update_best:if_case_KL}
	\State $\mathsf{KL}_{\mathsf{best}} \gets \Call{computeKL}{\Pt_{\mathsf{best}}}$
	\State $\mathsf{KL}_{\mathsf{new}} \gets \Call{computeKL}{\Pt}$
	\If{$\mathsf{KL}_{\mathsf{new}} > \mathsf{KL}_{\mathsf{best}}$}
		\State $\Pt_{\mathsf{best}} \gets \Pt$
	\EndIf
\EndIf
\If{$m \in \setof{\mathsf{EMCb}}$}				\label{algoline:dx+_update_best:if_case_EMCb}
	\State $\mathsf{KL}_{\mathsf{best}} \gets \Call{computeEMCb}{\Pt_{\mathsf{best}}}$
	\State $\mathsf{KL}_{\mathsf{new}} \gets \Call{computeEMCb}{\Pt}$
	\If{$\mathsf{EMCb}_{\mathsf{new}} > \mathsf{EMCb}_{\mathsf{best}}$}
		\State $\Pt_{\mathsf{best}} \gets \Pt$
	\EndIf
\EndIf
\If{$m \in \setof{\mathsf{MPS}}$}				\label{algoline:dx+_update_best:if_case_MPS}    \Comment{if $m$ in eq.\ class {\larger\textcircled{\smaller[2]6}} w.r.t.\ Table~\ref{tab:requirements_for_equiv_classes_of_measures_wrt_equiv_mQ}}
	\If{$|\dx{}(\Pt)| = 1 \land p(\dx{}(\Pt)) > p(\dx{}(\Pt_{\mathsf{best}}))$}   \label{algoline:dx+_update_best:MPS_if}
		\State $\Pt_{\mathsf{best}} \gets \Pt$															\label{algoline:dx+_update_best:if_case_set_Pbest}
	\EndIf
\EndIf
\If{$m \in \setof{\mathsf{BME}}$}					\label{algoline:dx+_update_best:if_case_BME}
	\State $\mD' \gets \argmin_{\mD^*\in\setof{\dx{}(\Pt),\dnx{}(\Pt)}}(p(\mD^*))$
	\State $\mD'_{\mathsf{best}} \gets \argmin_{\mD^*\in\setof{\dx{}(\Pt_{\mathsf{best}}),\dnx{}(\Pt_{\mathsf{best}})}}(p(\mD^*))$
	\If{$|\mD'| > |\mD'_{\mathsf{best}}|$}
		\State $\Pt_{\mathsf{best}} \gets \Pt$  \label{algoline:dx+_update_best:BME_Pbest_gets_P}
	\EndIf
\EndIf
\State \Return $\Pt_{\mathsf{best}}$
\EndProcedure
\end{algorithmic}
\normalsize
\end{algorithm}

Before we elucidate the functions (i.e.\ Algorithms~\ref{algo:dx+_update_best}, \ref{algo:dx+_opt}, \ref{algo:dx+_prune} and \ref{algo:dx+_best_successor}) called by \textsc{$\dx{}$Partition} in Algorithm~\ref{algo:dx+_part}, take note of the following remarks:
\begin{remark}\label{rem:dz=0_is_maintained_throughout_search} We want to emphasize that the favorable feature $\dz{}(Q) = \emptyset$ of canonical queries $Q$ is preserved throughout the execution of the $\textsc{calcQuery}$ function (Algorithm~\ref{algo:query_comp}). In other words, the algorithm will always return a query with empty $\dz{}(Q)$ as all upcoming functions (that are executed after \textsc{FindQPartition}) are q-partition preserving, i.e.\ they possibly manipulate the query, but leave the q-partition invariant. \qed
\end{remark}

\begin{remark}\label{rem:RIO_only_onesided_search}
The code specified for the $\mathsf{RIO}$ measure in Algorithms~\ref{algo:dx+_update_best}, \ref{algo:dx+_opt}, \ref{algo:dx+_prune} and \ref{algo:dx+_best_successor} tacitly assumes that we use a one-directional search for the best canonical q-partition with early pruning. That is, since we consider $\dx{}$-partitioning where $\dx{}$ is being gradually filled up with diagnoses starting from the empty set, we premise that the goal is to find only those least cautious non-high-risk queries (cf.\ Section~\ref{sec:rio}) with the feature $|\dx{}|\leq |\dnx{}|$. The reason for this is that, in $\dx{}$-partitioning, the size of the search tree grows proportionally (depth-first) and the search time in the worst case even exponentially with the number of diagnoses in $\dx{}$. Of course, one usually wants to keep the search complexity small. In order to ``complete'' the search for optimal canonical q-partitions (which will be rarely necessary in practical situations as optimal q-partitions seem to be found by means of only one-directional early-pruning search with very high reliability given reasonable values of $t_m$), we would rely on another early-pruning search (possibly executed in parallel) by means of $\dnx{}$-partitioning assuming $|\dnx{}|\leq |\dx{}|$. This can be compared with a bidirectional search (cf.\ \cite[Sec.~3.4.6]{russellnorvig2010}). In this vein, a great deal of time might be saved compared to a one-directional full search.\qed
\end{remark}

\begin{remark}\label{rem:only_1_representative_of_each_measure_eq_class_in_algos}
In Section~\ref{sec:QPartitionRequirementsSelection} we have deduced for each discussed query quality measure $m$ some qualitative requirements $r_m$ to an optimal query w.r.t.\ $m$. These are listed in Table~\ref{tab:requirements_for_all_measures}. However, as our search for optimal q-partitions automatically neglects q-partitions with a non-empty set $\dz{}$, the focus of this search is put exclusively on queries in a set $\mQ \subseteq \mQ_{\mD,\tuple{\mo,\mb,\Tp,\Tn}_\RQ}$ where each $Q \in \mQ$ features $\dz{}(Q) = \emptyset$. For that reason, we can exploit Table~\ref{tab:requirements_for_equiv_classes_of_measures_wrt_equiv_mQ} which states equivalences between the discussed measures in this particular setting. The section ``$\equiv_{\mQ}$'' of the table suggests that there are only seven measures with a different behavior when it comes to the selection of a query (q-partition) from the set $\mQ$. Therefore, we mention only one representative of each equivalence class (namely the same one as given in the first column of Table~\ref{tab:requirements_for_equiv_classes_of_measures_wrt_equiv_mQ}) in Algorithms~\ref{algo:dx+_update_best}, \ref{algo:dx+_opt}, \ref{algo:dx+_prune} and \ref{algo:dx+_best_successor} although the same code is valid for all measures in the same class. For instance, for $\mathsf{EMCa}$ Table~\ref{tab:measure_equiv_classes} tells us that the code given for $\mathsf{ENT}$ must be considered.\qed
\end{remark}

\paragraph{Algorithm for the Update of the Best Q-Partition.} In the following, we make some comments on Algorithm~\ref{algo:dx+_update_best} on page~\pageref{algo:dx+_update_best}. 
Given the inputs consisting of two partitions (not necessarily q-partitions) $\Pt$ and $\Pt_{\mathsf{best}}$ w.r.t.\ the set of leading diagnoses $\mD$, a set of requirements $r_m$ to an optimal q-partition and a (diagnosis) probability measure $p$, the output is the better partition among $\setof{\Pt,\Pt_{\mathsf{best}}}$ w.r.t.\ $\mD$ and $r_m$.

First of all, if either the $\dx{}$ or the $\dnx{}$-set in $\Pt_{\mathsf{best}}$ is empty, then $\Pt_{\mathsf{best}}$ is set to $\Pt$ anyway (lines~\ref{algoline:dx+_update_best:if_Pbest_is_no_q-partition_start}-\ref{algoline:dx+_update_best:if_Pbest_is_no_q-partition_end}). The reason is that, except for the very first call of \textsc{updateBest}, $\Pt$ will always be a q-partition. So, simply put, $\Pt_{\mathsf{best}}$ which is initially no q-partition is made to one at the very first chance to do so.
Second, in line~\ref{algoline:dx+_update_best:reconstruct_leading_diags} the set of leading diagnoses $\mD$, which is needed in the computations performed by the algorithm, is reconstructed by means of the q-partition $\Pt$ (note that $\mD$ is indeed equal to $\dx{}(\Pt) \cup \dnx{}(\Pt)$ since $\dz{}(Q) = \emptyset$ by assumption). What comes next is the check what the used query quality measure $m$ is. Depending on the outcome, the determination of the better q-partition among $\setof{\Pt,\Pt_{\mathsf{best}}}$ is realized in different ways, steered by $r_m$ (cf.\ Table~\ref{tab:requirements_for_equiv_classes_of_measures_wrt_equiv_mQ}). In the case of 
\begin{itemize}
	\item the measure $\mathsf{ENT}$ (line~\ref{algoline:dx+_update_best:if_case_ENT}) and equivalent measures as per the section ``$\equiv_{\mQ}$'' of Table~\ref{tab:measure_equiv_classes}, $r_m$ dictates that the better q-partition is the one for which the sum of the probabilities in $\dx{}$ is closer to $0.5$. Given $p(\dx{}(\Pt))$ is closer to $0.5$ than $p(\dx{}(\Pt_{\mathsf{best}}))$, $\Pt$ becomes the new $\Pt_{\mathsf{best}}$. Otherwise, $\Pt_{\mathsf{best}}$ remains unchanged.
	\item the measure $\mathsf{SPL}$ (line~\ref{algoline:dx+_update_best:if_case_SPL}) and equivalent measures as per the section ``$\equiv_{\mQ}$'' of Table~\ref{tab:measure_equiv_classes}, the better q-partition according to $r_m$ is the one for which the number of diagnoses in $\dx{}$ is closer to $\frac{|\mD|}{2}$. Given $|\dx{}(\Pt)|$ is closer to $\frac{|\mD|}{2}$ than $|\dx{}(\Pt_{\mathsf{best}})|$, $\Pt$ becomes the new $\Pt_{\mathsf{best}}$. Otherwise, $\Pt_{\mathsf{best}}$ remains unchanged.
	\item the measure $\mathsf{RIO}$ (line~\ref{algoline:dx+_update_best:if_case_RIO}) and equivalent measures as per the section ``$\equiv_{\mQ}$'' of Table~\ref{tab:measure_equiv_classes}, at first the current value of the cautiousness parameter is stored in $c$ (line~\ref{algoline:dx+_update_best:getCaut}). Using $c$, the number of leading diagnoses that must be eliminated by the query answer at a minimum are assigned to the variable $n$ (line~\ref{algoline:dx+_update_best:getMinimalNumOfDiagsToEliminate}), cf.\ Section~\ref{sec:rio}. Then, the algorithm makes some tests utilizing $n$ in order to find out whether $\Pt$ is better than $\Pt_{\mathsf{best}}$. In fact, the three requirements in $r_m$ (third row of Table~\ref{tab:requirements_for_equiv_classes_of_measures_wrt_equiv_mQ}) are tested in sequence according to their priority $(\mathrm{I})$--$(\mathrm{III})$. 
	
	That is, if $\dx{}(\Pt)$ has a cardinality larger than or equal to $n$ (line~\ref{algoline:dx+_update_best:if_CardDxPt_geq_n}), i.e.\ $\Pt$ is a non-high-risk q-partition (due to $|\dx{}(\Pt)|\leq|\dnx{}(\Pt)|$, cf.\ Remark~\ref{rem:RIO_only_onesided_search})
	and $\dx{}(\Pt_{\mathsf{best}})$ has a cardinality lower than $n$ (line~\ref{algoline:dx+_update_best:if_PtBest_high-risk_q-partition}), i.e.\ $\Pt_{\mathsf{best}}$ is a high-risk q-partition, $\Pt$ is already better than $\Pt_{\mathsf{best}}$ w.r.t.\ the priority $(\mathrm{I})$ requirement in $r_m$. Hence, $\Pt_{\mathsf{best}}$ is set to $\Pt$ in line~\ref{algoline:dx+_update_best:RIO_update_Pbest_1}.
	
	Otherwise, both $\Pt$ and $\Pt_{\mathsf{best}}$ are non-high-risk q-partitions, i.e.\ equally good w.r.t.\ the priority $(\mathrm{I})$ requirement in $r_m$. In this case, if the cardinality of $\dx{}(\Pt)$ is closer to $n$ than the cardinality of $\dx{}(\Pt_{\mathsf{best}})$ (line~\ref{algoline:dx+_update_best:if_CardDxPt_closer_to_n_than_CardDxPtBest}), then $\Pt$ is less cautious than $\Pt_{\mathsf{best}}$ and thus better regarding the priority $(\mathrm{II})$ requirement in $r_m$ (recall that $\mathsf{RIO}$ searches for the \emph{least cautious} non-high-risk query). Hence, $\Pt_{\mathsf{best}}$ is set to $\Pt$ in line~\ref{algoline:dx+_update_best:RIO_update_Pbest_2}.
	
	Otherwise, both $\Pt$ and $\Pt_{\mathsf{best}}$ are non-high-risk q-partitions and $\Pt$ is at least as cautious as $\Pt_{\mathsf{best}}$. Now, the algorithm checks first whether both are equally cautious (line~\ref{algoline:dx+_update_best:if_CardDxPt_equal_to_CardDxPtBest}). If so, then both $\Pt$ and $\Pt_{\mathsf{best}}$ are equally good w.r.t.\ the priority $(\mathrm{I})$ and $(\mathrm{II})$ requirements in $r_m$. Therefore, the priority $(\mathrm{III})$ requirement in $r_m$ is tested, i.e.\ if $p(\dx{}(\Pt))$ is closer to $0.5$ than $p(\dx{}(\Pt_{\mathsf{best}}))$ (line~\ref{algoline:dx+_update_best:if_ProbDxPt_closer_to_0.5_than_ProbDxPtBest}). If this evaluates to $\true$, then $\Pt_{\mathsf{best}}$ is set to $\Pt$ in line~\ref{algoline:dx+_update_best:RIO_update_Pbest_3}.
	
	In all other cases, $\Pt$ is not better w.r.t.\ $r_m$ than $\Pt_{\mathsf{best}}$, and hence $\Pt_{\mathsf{best}}$ is not updated.
	\item the measure $\mathsf{KL}$ (see line~\ref{algoline:dx+_update_best:if_case_KL}), the algorithm computes the $\mathsf{KL}$ measure for both $\Pt$ and $\Pt_{\mathsf{best}}$ (function \textsc{computeKL}) and compares the obtained values. In case $\Pt$ leads to a better, i.e.\ larger, $\mathsf{KL}$ value, $\Pt_{\mathsf{best}}$ is set to $\Pt$, otherwise not. Note that our analyses conducted in Section~\ref{sec:NewActiveLearningMeasuresForKBDebugging} do not suggest any other plausible, efficient and \emph{general} way to find out which of two given q-partitions is better w.r.t.\ $\mathsf{KL}$. 
	\item the measure $\mathsf{EMCb}$ (see line~\ref{algoline:dx+_update_best:if_case_EMCb}), the algorithm computes the $\mathsf{EMCb}$ measure for both $\Pt$ and $\Pt_{\mathsf{best}}$ (function \textsc{computeEMCb}) and compares the obtained values. In case $\Pt$ leads to a better, i.e.\ larger, $\mathsf{EMCb}$ value, $\Pt_{\mathsf{best}}$ is set to $\Pt$, otherwise not. Note that our analyses conducted in Section~\ref{sec:NewActiveLearningMeasuresForKBDebugging} do not suggest any other plausible, efficient and \emph{general} way to find out which of two given q-partitions is better w.r.t.\ $\mathsf{EMCb}$. 
	\item the measure $\mathsf{MPS}$ (see line~\ref{algoline:dx+_update_best:if_case_MPS}) and equivalent measures as per the section ``$\equiv_{\mQ}$'' of Table~\ref{tab:measure_equiv_classes}, one best q-partition is one for which $|\dx{}| = 1$ and the sum of diagnoses probabilities in $\dx{}$ is maximal (cf.\ Table~\ref{tab:requirements_for_equiv_classes_of_measures_wrt_equiv_mQ} and Proposition~\ref{prop:best_query_wrt_MPS_has_q-partition}). This is reflected by the code in lines~\ref{algoline:dx+_update_best:MPS_if}-\ref{algoline:dx+_update_best:if_case_set_Pbest} which guarantees that only q-partitions with $|\dx{}|=1$ can ever become a currently best q-partition. Note that the heuristics implemented by Algorithm~\ref{algo:dx+_best_successor} will ensure that the best canonical q-partition will always be visited as a very first node (other than the root node) in the search. This holds for $\mathsf{MPS}$, but clearly does not need to hold for the other considered measures.
	\item the measure $\mathsf{BME}$ (see line~\ref{algoline:dx+_update_best:if_case_BME}), the better q-partition is the one for which the set among $\setof{\dx{},\dnx{}}$ with lower probability has the higher cardinality (cf.~Table~\ref{tab:requirements_for_equiv_classes_of_measures_wrt_equiv_mQ}). This is reflected by the code in lines~\ref{algoline:dx+_update_best:if_case_BME}-\ref{algoline:dx+_update_best:BME_Pbest_gets_P}.
\end{itemize}

\begin{algorithm}
\small
\caption[Optimality Check]{Optimality check in $\dx{}$-Partitioning}\label{algo:dx+_opt}
\begin{algorithmic}[1]
\Require partition $\Pt_{\mathsf{best}} = \tuple{\dx{},\dnx{},\emptyset}$ w.r.t.\ the set of leading diagnoses $\mD$, a threshold $t_m$, a (diagnosis) probability measure $p$, requirements $r_m$ to an optimal q-partition,   
\Ensure $\true$ iff $\Pt_{\mathsf{best}}$ is an optimal q-partition w.r.t.\ $r_m$ and $t_m$
\Procedure{opt}{$\Pt_{\mathsf{best}},t_m,p,r_m$}
\If{$\dx{} = \emptyset \lor \dnx{} = \emptyset$}  \label{algoline:dx+_opt:if_dx_or_dnx_emptyset}
	\State \textbf{return} $\false$					\Comment{partition is not a q-partition, cf.\ Proposition~\ref{prop:properties_of_q-partitions},(\ref{prop:properties_of_q-partitions:enum:for_each_q-partition_dx_is_empty_and_dnx_is_empty}.)}
\EndIf
\State $\mD \gets \dx{} \cup \dnx{}$			\label{algoline:dx+_opt:reconstruct_leading_diags}			\Comment{reconstruct leading diagnoses from $\dx{},\dnx{}$}
\If{$m \in \setof{\mathsf{ENT}}$}  \label{algoline:dx+_opt:if_case_ENT}   \Comment{if $m$ in eq.\ class {\larger\textcircled{\smaller[2]1}} w.r.t.\ Table~\ref{tab:requirements_for_equiv_classes_of_measures_wrt_equiv_mQ}}
			\If{$|p(\dx{}) - 0.5| \leq t_m$}
					\State \textbf{return} $\true$
			\EndIf
			\State \textbf{return} $\false$
\EndIf
\If{$m \in \setof{\mathsf{SPL}}$}   \label{algoline:dx+_opt:if_case_SPL}   \Comment{if $m$ in eq.\ class {\larger\textcircled{\smaller[2]2}} w.r.t.\ Table~\ref{tab:requirements_for_equiv_classes_of_measures_wrt_equiv_mQ}}
			\If{$\left|\,|\dx{}| - \frac{|\mD|}{2} \right| \leq t_m$}
					\State \textbf{return} $\true$
			\Else
					\State \textbf{return} $\false$
			\EndIf
\EndIf
\If{$m \in \setof{\mathsf{RIO}}$}		\label{algoline:dx+_opt:if_case_RIO}   \Comment{if $m$ in eq.\ class {\larger\textcircled{\smaller[2]3}} w.r.t.\ Table~\ref{tab:requirements_for_equiv_classes_of_measures_wrt_equiv_mQ}}
			\State $\uc \gets \Call{getCaut}{r_m}$  \label{algoline:dx+_opt:getCaut}
			\State $n \gets \lceil \uc |\mD|\rceil$  \label{algoline:dx+_opt:compute_n}
			\State $t_{\mathsf{card}} \gets \Call{getCardinalityThreshold}{t_m}$  \label{algoline:dx+_opt:t_card}
			\State $t_{\mathsf{ent}} \gets \Call{getEntropyThreshold}{t_m}$    \label{algoline:dx+_opt:t_ent}
			\If{$|\dx{}| \geq n$} 	\label{algoline:dx+_opt:if_|dx|_geq_n}  \Comment{$\Pt_{\mathsf{best}}$ is a non-high-risk q-partition}
				\If{$|\dx{}| - n \leq t_{\mathsf{card}}$}   \label{algoline:dx+_opt:if_|dx|-n_leq_t_card}  \Comment{$\Pt_{\mathsf{best}}$ is sufficiently cautious}
					\If{$|p(\dx{}) - 0.5| \leq t_{\mathsf{ent}}$}  \label{algoline:dx+_opt:if_|p(dx)-0.5|_leq_t_ent} \Comment{$\Pt_{\mathsf{best}}$ has sufficiently high information gain}
							\State \Return $\true$ 
					\EndIf
				\EndIf
			\EndIf
			\State \Return $\false$  \label{algoline:dx+_opt:RIO_return_false}
\EndIf
\If{$m \in \setof{\mathsf{KL}}$}   \label{algoline:dx+_opt:if_case_KL}
			\If{$\left|\frac{|\dx{}|}{|\mD|} \log_2\left(\frac{1}{p(\dx{})}\right) + \frac{|\dnx{}|}{|\mD|} \log_2\left(\frac{1}{p(\dnx{})}\right) - opt_{\mathsf{KL},p,\mD}\right| \leq t_m$}		   \Comment{cf.\ Remark~\ref{rem:kullback_leibler_theoretical_opt}}
					\State \textbf{return} $\true$
			\EndIf
			\State \textbf{return} $\false$
\EndIf
\If{$m \in \setof{\mathsf{EMCb}}$}   \label{algoline:dx+_opt:if_case_EMCb}
			\If{$\left| p(Q=t) |\dnx{}(Q)| + p(Q=f) |\dx{}(Q)| - opt_{\mathsf{EMCb},p,\mD}\right| \leq t_m$}	\Comment{cf.\ Remark~\ref{rem:EMCb_theoretical_opt}}	
					\State \textbf{return} $\true$
			\EndIf
			\State \textbf{return} $\false$
\EndIf
\If{$m \in \setof{\mathsf{MPS}}$}   \label{algoline:dx+_opt:if_case_MPS}     \Comment{if $m$ in eq.\ class {\larger\textcircled{\smaller[2]6}} w.r.t.\ Table~\ref{tab:requirements_for_equiv_classes_of_measures_wrt_equiv_mQ}}
			\If{$|\dx{}| = 1$}
					\State $prob_{\max} \gets \max_{\md\in\mD}(p(\md))$
					\If{$p(\dx{}) = prob_{\max}$}
							\State \textbf{return} $\true$
					\EndIf
			\EndIf
			\State \textbf{return} $\false$
\EndIf
\If{$m \in \setof{\mathsf{BME}}$}   \label{algoline:dx+_opt:if_case_BME}
			\If{$p(\dx{}) = 0.5$}   \label{algoline:dx+_opt:BME_if_p(dx)=0.5}
				\State \Return $\false$
			\EndIf
			\If{$p(\dx{}) < 0.5 \land \left||\dx{}| - (|\mD|- 1) \right| \leq t_m$}  \label{algoline:dx+_opt:BME_if_p(dx)<0.5_and_tm_OK}
					\State \textbf{return} $\true$
			\EndIf
			\If{$p(\dnx{}) < 0.5 \land \left||\dnx{}| - (|\mD|- 1) \right| \leq t_m$} \label{algoline:dx+_opt:BME_if_p(dnx)<0.5_and_tm_OK}
					\State \textbf{return} $\true$
			\EndIf
			\State \Return $\false$
\EndIf
\EndProcedure
\end{algorithmic}
\normalsize
\end{algorithm}

\paragraph{Algorithm for the Optimality Check.} 
In the following, we provide some explanations on Algorithm~\ref{algo:dx+_opt} on page~\pageref{algo:dx+_opt}. Given the inputs consisting of a partition (not necessarily q-partition) $\Pt_{\mathsf{best}} = \tuple{\dx{},\dnx{},\emptyset}$ w.r.t.\ the set of leading diagnoses $\mD$, a threshold $t_m$, requirements $r_m$ to an optimal q-partition and a (diagnosis) probability measure $p$, the output is $\true$ iff $\Pt_{\mathsf{best}}$ is an optimal q-partition w.r.t.\ $r_m$ and $t_m$.

First of all, in line~\ref{algoline:dx+_opt:if_dx_or_dnx_emptyset} the algorithm tests if $\Pt_{\mathsf{best}}$ is not a q-partition, i.e.\ whether $\dx{} = \emptyset$ or $\dnx{} = \emptyset$. If so, $\false$ is immediately returned as a partition which is not a q-partition should never be returned by the algorithm. Then, in line~\ref{algoline:dx+_opt:reconstruct_leading_diags} the set of leading diagnoses $\mD$, which is needed in the computations performed by the algorithm, is reconstructed by means of $\dx{}$ and $\dnx{}$ (note that $\mD$ is indeed equal to $\dx{} \cup \dnx{}$ since $\dz{} = \emptyset$ by assumption). What comes next is the check what the used query quality measure $m$ is. Depending on the outcome, the determination of the output is realized in different ways, steered by $r_m$ (cf.\ Table~\ref{tab:requirements_for_equiv_classes_of_measures_wrt_equiv_mQ}). In the case of 
\begin{itemize}
	\item the measure $\mathsf{ENT}$ (line~\ref{algoline:dx+_opt:if_case_ENT}) and equivalent measures as per the section ``$\equiv_{\mQ}$'' of Table~\ref{tab:measure_equiv_classes}, $\Pt_{\mathsf{best}}$ is optimal according to $r_m$ taking into account the tolerance of $t_m$ iff the sum of probabilities of diagnoses in $\dx{}$ does not deviate from $0.5$ by more than $t_m$. 
	\item the measure $\mathsf{SPL}$ (line~\ref{algoline:dx+_opt:if_case_SPL}) and equivalent measures as per the section ``$\equiv_{\mQ}$'' of Table~\ref{tab:measure_equiv_classes}, $\Pt_{\mathsf{best}}$ is optimal as per $r_m$ taking into account the tolerance of $t_m$ iff the number of diagnoses in $\dx{}$ does not deviate from half the number of diagnoses in $\mD$ by more than $t_m$.
	\item the measure $\mathsf{RIO}$ (line~\ref{algoline:dx+_opt:if_case_RIO}) and equivalent measures as per the section ``$\equiv_{\mQ}$'' of Table~\ref{tab:measure_equiv_classes}, whether $\Pt_{\mathsf{best}} = \tuple{\dx{},\dnx{},\dz{}}$ is optimal depends on the current cautiousness parameter $\uc$ pertinent to the $\mathsf{RIO}$ measure and on two thresholds, $t_{\mathsf{card}}$ and $t_{\mathsf{ent}}$, that we assume are prespecified and extractable from $t_m$. The former is related to the second condition (II) (third row of Table~\ref{tab:requirements_for_equiv_classes_of_measures_wrt_equiv_mQ}) and defines the maximum tolerated deviance of $|\dx{}|$ from the least number $\lceil c |\mD|\rceil$ (line~\ref{algoline:dx+_opt:compute_n}) of leading diagnoses postulated to be eliminated by the next query as per $\uc$ (see also Section~\ref{sec:rio}). In other words, $t_{\mathsf{card}}$ denotes the maximal allowed cardinality deviance from the theoretical least cautious query w.r.t.\ $\uc$. 
		The threshold $t_{\mathsf{ent}}$ addresses to the third condition (III) (third row of Table~\ref{tab:requirements_for_equiv_classes_of_measures_wrt_equiv_mQ}) and characterizes the maximum accepted difference of $p(\dx{})$ from the (theoretically) optimal entropy value $0.5$.
So, if 
\begin{enumerate}
	\item there are at least $n$ diagnoses in $\dx{}$ (line~\ref{algoline:dx+_opt:if_|dx|_geq_n}), i.e.\ $\Pt_{\mathsf{best}}$ is a non-high risk q-partition (requirement $(\mathrm{I})$ as per row three of Table~\ref{tab:requirements_for_equiv_classes_of_measures_wrt_equiv_mQ} is met), and
	\item the cardinality of $\dx{}$ does not deviate from $n$ by more than $t_{\mathsf{card}}$ (line~\ref{algoline:dx+_opt:if_|dx|-n_leq_t_card}), i.e.\ $\Pt_{\mathsf{best}}$ is sufficiently cautious as per $\uc$ (requirement $(\mathrm{II})$ as per row three of Table~\ref{tab:requirements_for_equiv_classes_of_measures_wrt_equiv_mQ} is met), and 
	\item the sum of probabilities of diagnoses in $\dx{}$ does not differ from $0.5$ by more than $t_{\mathsf{ent}}$ (line~\ref{algoline:dx+_opt:if_|p(dx)-0.5|_leq_t_ent}), i.e.\ $\Pt_{\mathsf{best}}$ has a sufficiently high information gain (requirement $(\mathrm{III})$ as per row three of Table~\ref{tab:requirements_for_equiv_classes_of_measures_wrt_equiv_mQ} is satisfied),
\end{enumerate}
  then $\Pt_{\mathsf{best}}$ is (considered) optimal. Otherwise it is not. This is reflected by the code lines \ref{algoline:dx+_opt:if_|dx|_geq_n}-\ref{algoline:dx+_opt:RIO_return_false}.
	\item the measure $\mathsf{KL}$ (line~\ref{algoline:dx+_opt:if_case_KL}), we can use $opt_{\mathsf{KL},p,\mD}$ which was derived in Remark~\ref{rem:kullback_leibler_theoretical_opt}. 
	This means that we consider $\Pt_{\mathsf{best}}$ as optimal if $\mathsf{KL}(Q)$ (see Eq.~\eqref{eq:Q_KL_derived}) does not differ from $opt_{\mathsf{KL},p,\mD}$ by more than the specified threshold $t_m$. 
	\item the measure $\mathsf{EMCb}$ (line~\ref{algoline:dx+_opt:if_case_EMCb}), we can use $opt_{\mathsf{EMCb},p,\mD}$ which was derived in Remark~\ref{rem:EMCb_theoretical_opt}. 
	This means that we consider $\Pt_{\mathsf{best}}$ as optimal if $\mathsf{EMCb}(Q)$ (see Eq.~\eqref{eq:EMCb}) does not differ from $opt_{\mathsf{EMCb},p,\mD}$ by more than the specified threshold $t_m$. 
	\item the measure $\mathsf{MPS}$ (line~\ref{algoline:dx+_opt:if_case_MPS}) and equivalent measures as per the section ``$\equiv_{\mQ}$'' of Table~\ref{tab:measure_equiv_classes}, $\Pt_{\mathsf{best}}$ is optimal iff $\dx{}$ includes only the most probable diagnosis $\md^*$ in $\mD$ and the rest of the diagnoses is assigned to $\dnx{}$ (Proposition~\ref{prop:best_query_wrt_MPS_has_q-partition}). Note that $\tuple{\setof{\md^*},\mD\setminus\setof{\md^*},\emptyset}$ is a canonical q-partition due to Corollary~\ref{cor:--di,MD-di,0--_is_canonical_q-partition_for_all_di_in_mD}. Hence, the proposed search which is complete w.r.t.\ canonical q-partitions will definitely generate this q-partition. This implies that we are, for any set of multiple leading diagnoses, able to find \emph{precisely} the theoretically optimal q-partition w.r.t.\ the $\mathsf{MPS}$ measure. For that reason we do not need any $t_m$ when using $\mathsf{MPS}$.
	%
	%
	%
	\item the measure $\mathsf{BME}$ (line~\ref{algoline:dx+_opt:if_case_BME}), if $p(\dx{}) = 0.5$ (line~\ref{algoline:dx+_opt:BME_if_p(dx)=0.5}), then also $p(\dnx{}) = 0.5$ (since $\dz{} = \emptyset$). That is, requirement $(\mathrm{I})$ (last row in Table~\ref{tab:requirements_for_equiv_classes_of_measures_wrt_equiv_mQ}) is not met for any of $\setof{\dx{},\dnx{}}$. Hence, $\false$ is returned. 
	
	Otherwise, if $p(\dx{}) < 0.5$, then $\Pt_{\mathsf{best}}$ is optimal iff $|\dx{}|$ deviates not more than $t_m$ from its maximal possible cardinality $|\mD|-1$ (line~\ref{algoline:dx+_opt:BME_if_p(dx)<0.5_and_tm_OK}). If so, then $\Pt_{\mathsf{best}}$ is (considered) optimal which is why $\true$ is returned. To understand why the maximal possible cardinality is given by $|\mD|-1$, recall that $\dx{}$ as well as $\dnx{}$ must be non-empty sets due to line~\ref{algoline:dx+_opt:if_dx_or_dnx_emptyset}, i.e.\ $1 \leq |\dx{}|\leq |\mD|-1$ and $1 \leq |\dnx{}|\leq |\mD|-1$. 
	
	Otherwise, if $p(\dnx{}) < 0.5$, then $\Pt_{\mathsf{best}}$ is optimal iff $|\dnx{}|$ deviates not more than $t_m$ from its maximal possible cardinality $|\mD|-1$ (line~\ref{algoline:dx+_opt:BME_if_p(dnx)<0.5_and_tm_OK}). In this case, $\true$ is returned as well.
	
	In all other cases, $\Pt_{\mathsf{best}}$ is (considered) non-optimal. Thence, $\false$ is returned.
\end{itemize}

\begin{algorithm}
\small
\caption[Pruning]{Pruning in $\dx{}$-Partitioning}\label{algo:dx+_prune}
\begin{algorithmic}[1]
\Require partitions $\Pt$ and $\Pt_{\mathsf{best}}$ w.r.t.\ the set of leading diagnoses $\mD$ with $\dz{}(\Pt) = \dz{}(\Pt_{\mathsf{best}}) = \emptyset$, 
requirements $r_m$ to an optimal q-partition, a (diagnosis) probability measure $p$
\Ensure $\true$ if exploring successor q-partitions 
of $\Pt$ cannot lead to the discovery of q-partitions that are better w.r.t.\ $r_m$ than $\Pt_{\mathsf{best}}$ 
\Procedure{prune}{$\Pt,\Pt_{\mathsf{best}},p,r_m$}
\If{$\dx{}(\Pt) = \emptyset \lor \dnx{}(\Pt) = \emptyset$}  \label{algoline:dx+_prune:if_dx_or_dnx_emptyset}
	\State \textbf{return} $\false$		\label{algoline:dx+_prune:return_false_if_no_q-partition}			\Comment{partition is not a q-partition, cf.\ Proposition~\ref{prop:properties_of_q-partitions},(\ref{prop:properties_of_q-partitions:enum:for_each_q-partition_dx_is_empty_and_dnx_is_empty}.)}
\EndIf
\State $\mD \gets \dx{}(\Pt) \cup \dnx{}(\Pt)$  \label{algoline:dx+_prune:reconstruct_leading_diags}  \Comment{reconstruct leading diagnoses from partition $\Pt$}
\If{$m \in \setof{\mathsf{ENT}}$}   \label{algoline:dx+_prune:if_case_ENT}   \Comment{if $m$ in eq.\ class {\larger\textcircled{\smaller[2]1}} w.r.t.\ Table~\ref{tab:requirements_for_equiv_classes_of_measures_wrt_equiv_mQ}}
			\If{$p(\dx{}(\Pt)) \geq 0.5$}   \label{algoline:dx+_prune:ENT_test_condition}
					\State \textbf{return} $\true$
			\EndIf
			\State \textbf{return} $\false$
\EndIf
\If{$m \in \setof{\mathsf{SPL}}$}   \label{algoline:dx+_prune:if_case_SPL}   \Comment{if $m$ in eq.\ class {\larger\textcircled{\smaller[2]2}} w.r.t.\ Table~\ref{tab:requirements_for_equiv_classes_of_measures_wrt_equiv_mQ}}
			\If{$|\dx{}(\Pt)| \geq \left\lfloor\frac{|\mD|}{2}\right\rfloor$}   \label{algoline:dx+_prune:SPL_test_condition}
					\State \textbf{return} $\true$
			\EndIf
			\State \textbf{return} $\false$
\EndIf
\If{$m \in \setof{\mathsf{RIO}}$}			\label{algoline:dx+_prune:if_case_RIO}     \Comment{if $m$ in eq.\ class {\larger\textcircled{\smaller[2]3}} w.r.t.\ Table~\ref{tab:requirements_for_equiv_classes_of_measures_wrt_equiv_mQ}}
			\State $\uc \gets \Call{getCaut}{r_m}$   \label{algoline:dx+_prune:getCaut}
			\State $n \gets \lceil \uc |\mD|\rceil$  \label{algoline:dx+_prune:compute_n}
			\If{$|\dx{}(\Pt)| = n$}							\label{algoline:dx+_prune:RIO_if_|dx|=n}
					\State \textbf{return} $\true$  		\label{algoline:dx+_prune:RIO_return_true_1}
			\EndIf
			\If{$|\dx{}(\Pt_{\mathsf{best}})| = n$}  \label{algoline:dx+_prune:RIO_if_|dx(P_best)|=n}
					\If{$|\dx{}(\Pt)| > n$}								\label{algoline:dx+_prune:RIO_if_|dx|>n}
						\State \textbf{return} $\true$    \label{algoline:dx+_prune:RIO_return_true_2}
					\EndIf
					\If{$p(\dx{}(\Pt)) - 0.5 \geq |p(\dx{}(\Pt_{\mathsf{best}})) - 0.5|$}  \label{algoline:dx+_prune:RIO_if_cond1_and_cond2} \Comment{$|\dx{}(\Pt)| < n$ holds}
						\State \Return $\true$  				\label{algoline:dx+_prune:RIO_return_true_3} \Comment{$p(\dx{}(\Pt)) \geq 0.5$ holds}
					\EndIf
			\EndIf
			\State \textbf{return} $\false$   \label{algoline:dx+_prune:RIO_return_false}
\EndIf
\If{$m \in \setof{\mathsf{KL},\mathsf{EMCb}}$}  \label{algoline:dx+_prune:if_case_KL}
	\State \Return $\false$
\EndIf
\If{$m \in \setof{\mathsf{MPS}}$}   \label{algoline:dx+_prune:if_case_MPS}         \Comment{if $m$ in eq.\ class {\larger\textcircled{\smaller[2]6}} w.r.t.\ Table~\ref{tab:requirements_for_equiv_classes_of_measures_wrt_equiv_mQ}} 
			\If{$|\dx{}(\Pt)| \geq 1$}     \label{algoline:dx+_prune:MPS_test_condition}
					\State \textbf{return} $\true$
			\Else
					\State \textbf{return} $\false$
			\EndIf
\EndIf
\If{$m \in \setof{\mathsf{BME}}$}  \label{algoline:dx+_prune:if_case_BME}
	\If{$p(\dnx{}(\Pt)) < 0.5$}    \label{algoline:dx+_prune:BME_p(dx)>0.5}
		\State \Return $\true$			\label{algoline:dx+_prune:BME_return_true_1}
	\EndIf
	\If{$p(\dx{}(\Pt)) < 0.5$}   \label{algoline:dx+_prune:BME_p(dnx)>0.5}
		\If{$|\dnx{}(\Pt)| - 1 \leq |\dx{}(\Pt)|$}    \label{algoline:dx+_prune:BME_|dx|+1>=|dnx|}
			\State $prob_{\min} \gets \min_{\md\in\dnx{}(\Pt)}(p(\md))$
			\If{$p(\dnx{}(\Pt)) - prob_{\min} < 0.5$ }  \label{algoline:dx+_prune:BME_p(dx)-prob_min<=0.5}
				\State \Return $\true$   \label{algoline:dx+_prune:BME_return_true_2}
			\EndIf
		\EndIf
	\EndIf
	\State \Return $\false$  \label{algoline:dx+_prune:BME_return_false}
\EndIf
\EndProcedure
\end{algorithmic}
\normalsize
\end{algorithm}

\paragraph{Algorithm for the Pruning Check.} \label{par:algorithm_description_prune}
In the following, we provide some explanations on Algorithm~\ref{algo:dx+_prune} on page~\pageref{algo:dx+_prune}. Given the inputs consisting of partitions $\Pt$ and $\Pt_{\mathsf{best}}$ w.r.t.\ the set of leading diagnoses $\mD$ where $\dz{}(\Pt) = \dz{}(\Pt_{\mathsf{best}}) = \emptyset$, requirements $r_m$ to an optimal q-partition and a (diagnosis) probability measure $p$, the output is $\true$ if exploring successor q-partitions of $\Pt$ cannot lead to the discovery of q-partitions that are better w.r.t.\ $r_m$ than $\Pt_{\mathsf{best}}$.

As a first step the algorithm returns $\false$ if $\Pt$ is not a q-partition (line~\ref{algoline:dx+_prune:return_false_if_no_q-partition}). In other words, this ensures that the algorithm can never return on the first iteration where $\Pt$ corresponds to the inital state and thence is no q-partition (in all successive interations the current partition $\Pt$ \emph{will be} a q-partition). 
Then, in line~\ref{algoline:dx+_prune:reconstruct_leading_diags} the set of leading diagnoses $\mD$, which is needed in the computations performed by the algorithm, is reconstructed by means of $\Pt$. What comes next is the check what the used query quality measure $m$ is. Depending on the outcome, the determination of the output is realized in different ways, steered by $r_m$ (cf.\ Table~\ref{tab:requirements_for_equiv_classes_of_measures_wrt_equiv_mQ}). Before we analyze the different behavior of the algorithm for the different measures, we point out that, at the time \textsc{prune} (Algorithm~\ref{algo:dx+_prune}) is called in Algorithm~\ref{algo:dx+_part}, $\Pt_{\mathsf{best}}$ must be at least as good w.r.t.\ $r_m$ as $\Pt$. This is due to the fact that \textsc{updateBest} is always executed before \textsc{prune} in Algorithm~\ref{algo:dx+_part} and effectuates the storage of the best canonical q-partition found so far in $\Pt_{\mathsf{best}}$. In the case of 
\begin{itemize}
	\item the measure $\mathsf{ENT}$ (line~\ref{algoline:dx+_prune:if_case_ENT}) and equivalent measures as per the section ``$\equiv_{\mQ}$'' of Table~\ref{tab:measure_equiv_classes}, if the sum of probabilities of diagnoses in $\dx{}(\Pt)$ is already greater than or equal to $0.5$, then
	the quality of $\Pt$ as per $r_m$ can only become worse if additional diagnoses, each with a positive probability, are added to $\dx{}(\Pt)$. The reason for this is that the absolute difference between $p(\dx{}(\Pt))$ and the theoretical optimum of $0.5$ can only increase in this case. Hence, we test in line~\ref{algoline:dx+_prune:ENT_test_condition} whether $p(\dx{}(\Pt)) \geq 0.5$ and return $\true$ if positively evaluated and $\false$ otherwise.
	\item the measure $\mathsf{SPL}$ (line~\ref{algoline:dx+_prune:if_case_SPL}) and equivalent measures as per the section ``$\equiv_{\mQ}$'' of Table~\ref{tab:measure_equiv_classes}, if the number of diagnoses in $\dx{}(\Pt)$ is already greater than or equal to $\left\lfloor \frac{|\mD|}{2}\right\rfloor$, then
	the quality of $\Pt$ as per $r_m$ can only become worse if additional diagnoses are added to $\dx{}(\Pt)$. The reason for this is that the absolute difference between $|\dx{}(\Pt)|$ and the theoretical optimum of $\frac{|\mD|}{2}$ cannot become better than for $\Pt$. To see this, observe that $|\dx{}(\Pt)| = \left\lfloor \frac{|\mD|}{2}\right\rfloor$ means that the absolute difference between $|\dx{}(\Pt)|$ and $\frac{|\mD|}{2}$ is zero in case $|\mD|$ is even and one half otherwise. Adding at least one diagnosis to $\dx{}(\Pt)$ implies a difference of minimally one in the former case and a difference of minimally one half in the latter. The case $|\dx{}(\Pt)| > \left\lfloor \frac{|\mD|}{2}\right\rfloor$ can be analyzed very easily in the same fashion. 
	Therefore, we test in line~\ref{algoline:dx+_prune:SPL_test_condition} whether $|\dx{}(\Pt)| \geq \left\lfloor \frac{|\mD|}{2}\right\rfloor$ and return $\true$ if positively evaluated and $\false$ otherwise.
	\item the measure $\mathsf{RIO}$ (line~\ref{algoline:dx+_prune:if_case_RIO}) and equivalent measures as per the section ``$\equiv_{\mQ}$'' of Table~\ref{tab:measure_equiv_classes}, the pruning check works as follows. First, the minimal number $n$ of diagnoses that must be eliminated by the answer to the next query is computed (lines~\ref{algoline:dx+_prune:getCaut}-\ref{algoline:dx+_prune:compute_n}). Using this parameter, all necessary conditions for pruning can be verified. That is, if the size of $\dx{}(\Pt)$ is equal to $n$ (line~\ref{algoline:dx+_prune:RIO_if_|dx|=n}), the search tree can be pruned below $\Pt$ since the two highest-priority requirements to an optimal q-partition imposed by $\mathsf{RIO}$ (cf.\ $(\mathrm{I})$ and $(\mathrm{II})$ in the third row of Table~\ref{tab:requirements_for_equiv_classes_of_measures_wrt_equiv_mQ}) postulate that the size of the minimum-cardinality set in $\setof{\dx{},\dnx{}}$ (for which we always plug in $\dx{}$ due to Remark~\ref{rem:RIO_only_onesided_search}) is at least $n$ and its deviance from $n$ is minimal. The latter can only worsen if diagnoses are added to $\dx{}(\Pt)$. Consequently, the algorithm returns $\true$ in line~\ref{algoline:dx+_prune:RIO_return_true_1}. 	
	
On the other hand, if the cardinality of $\dx{}(\Pt_{\mathsf{best}})$ equals $n$ (line~\ref{algoline:dx+_prune:RIO_if_|dx(P_best)|=n}), i.e.\ we know that at least one q-partition with a $\dx{}$-set of cardinality $n$ has already been detected, and the cardinality of $\dx{}(\Pt)$ exceeds $n$ (line~\ref{algoline:dx+_prune:RIO_if_|dx|>n}), then $\Pt$ is worse than $\Pt_{\mathsf{best}}$ regarding requirement $(\mathrm{II})$. Hence $\true$ is returned in line~\ref{algoline:dx+_prune:RIO_return_true_2}. 

Finally, if (a) the cardinality of $\dx{}(\Pt_{\mathsf{best}})$ equals $n$ (line~\ref{algoline:dx+_prune:RIO_if_|dx(P_best)|=n}), (b) $|\dx{}(\Pt)| < n$ (by the negative evaluation of the tests in lines~\ref{algoline:dx+_prune:RIO_if_|dx|=n} and \ref{algoline:dx+_prune:RIO_if_|dx|>n}) and (c) $p(\dx{}(\Pt)) - 0.5 \geq |p(\dx{}(\Pt_{\mathsf{best}})) - 0.5|$ (line~\ref{algoline:dx+_prune:RIO_if_cond1_and_cond2}), 
then continuing the search at $\Pt$ cannot lead to a better q-partition than $\Pt_{\mathsf{best}}$ as per $r_m$. 
This holds since (b) and (c) imply that the addition of any diagnoses, each with a non-zero probability (cf.\ page~\pageref{etc:prob_of_each_diag_must_be_greater_zero}), to $\dx{}(\Pt)$ can never lead to a situation where $|\dx{}(\Pt')|=n$ and $\Pt'$ has a better entropy (i.e.\ $\mathsf{ENT}$ or $\mathsf{ENT}_z$ measure) than $\Pt_{\mathsf{best}}$ for any (direct or indirect) successor $\Pt'$ of $\Pt$. 
Because, 
for any successor $\Pt'$ of $\Pt$ with $|\dx{}(\Pt')|=n$ we have $p(\dx{}(\Pt')) > p(\dx{}(\Pt))$ due to (b), i.e.\ $p(\dx{}(\Pt')) - 0.5 > p(\dx{}(\Pt)) - 0.5 \geq |p(\dx{}(\Pt_{\mathsf{best}})) - 0.5| \geq 0$ due to (c). This means however that $|p(\dx{}(\Pt')) - 0.5| > |p(\dx{}(\Pt_{\mathsf{best}})) - 0.5|$, i.e.\ $\Pt'$ and $\Pt_{\mathsf{best}}$ are equally good w.r.t.\ the priority $(\mathrm{I})$ and $(\mathrm{II})$ requirements in $r_m$ (both are least cautious non-high-risk q-partitions), but $\Pt_{\mathsf{best}}$ is better as to the priority $(\mathrm{III})$ requirement in $r_m$ (information gain). All successors $\Pt''$ of $\Pt$ with $|\dx{}(\Pt'')|\neq n$ can never be better than $\Pt_{\mathsf{best}}$ since they already violate the priority $(\mathrm{I})$ requirement (i.e.\ are high-risk q-partitions) in case $|\dx{}(\Pt'')| < n$ or the priority $(\mathrm{II})$ requirement (i.e.\ are not least cautious) in case $|\dx{}(\Pt'')| > n$.

Hence, as an arbitrary direct or indirect successor of $\Pt$ cannot be better than $\Pt_{\mathsf{best}}$ w.r.t.\ $r_m$, the algorithm returns $\true$ in this case (line~\ref{algoline:dx+_prune:RIO_return_true_3}).

Otherwise, we cannot be sure that $\Pt$ might not generate any more favorable successor w.r.t.\ $r_m$ than $\Pt_{\mathsf{best}}$ is. Hence, $\false$ is returned (line~\ref{algoline:dx+_prune:RIO_return_false}).
	\item one of the measures $\mathsf{KL}$ and $\mathsf{EMCb}$ (line~\ref{algoline:dx+_prune:if_case_KL}), in the light of our analyses (cf.\ Section~\ref{sec:NewActiveLearningMeasuresForKBDebugging}) of the goal functions (Equations~\eqref{eq:Q_KL_derived} and \eqref{eq:EMCb}) that must be maximized to find an optimal q-partition for these measures, we are not able to come up with a straightforward general pruning condition based on probabilities and/or the cardinality of the sets in a q-partition. Hence, simply $\false$ is returned. This inability to prune the search tree early can be also a crucial shortcoming of $\mathsf{KL}$ and $\mathsf{EMCb}$ as opposed to other measures, especially in cases involving a high cardinality set $\mD$ of leading diagnoses.
	\item the measure $\mathsf{MPS}$ (see line~\ref{algoline:dx+_prune:if_case_MPS}) and equivalent measures as per the section ``$\equiv_{\mQ}$'' of Table~\ref{tab:measure_equiv_classes}, since for $\mathsf{MPS}$ the set of requirements $r_m$ postulate 
	a q-partition involving either a singleton $\dx{}$ or a singleton $\dnx{}$ where this singleton has maximal probability, the search is supposed to just return the q-partition with a singleton $\dx{} = \setof{\md}$ where $\md$ is the most probable leading diagnosis. This is also the theoretically optimal q-partition w.r.t.\ $\mD$. In this manner we rely on a similar strategy as with $\mathsf{RIO}$ in which we exploit early pruning (cf.\ Remark~\ref{rem:RIO_only_onesided_search}). Given that one wants to generate also the optimal q-partition with a singleton $\dnx{}$, then $\dnx{}$-partitioning should be used for this purpose.
	
	So, in line~\ref{algoline:dx+_prune:MPS_test_condition}, we just check whether the cardinality of $\dx{}$ is at least one. If so, then $\true$ is returned as $\Pt_{\mathsf{best}}$ must already contain the most probable diagnosis in $\mD$ (guaranteed by the implementation of \textsc{bestSuc}, see lines~\ref{algoline:dx+_best_successor:if_case_MPS}-\ref{algoline:dx+_best_successor:MPS_end} in Algorithm~\ref{algo:dx+_best_successor}) which is why $\true$ is returned. Otherwise, $\false$ is returned.
	\item the measure $\mathsf{BME}$ (see line~\ref{algoline:dx+_prune:if_case_BME}), we first recall that the optimal q-partition w.r.t.\ $r_m$ has the property that the cardinality of its set in $\setof{\dx{},\dnx{}}$ with the lower probability is maximal. 
	Hence, if $p(\dnx{}(\Pt))$ is already less than $0.5$ (note that this is equivalent to $p(\dx{}(\Pt)) > p(\dnx{}(\Pt))$ since $\dz{}(\Pt) = \emptyset$), then the deletion of further diagnoses from $\dnx{}(\Pt)$, each with a probability greater than zero (cf.\ page~\pageref{etc:prob_of_each_diag_must_be_greater_zero}), can only involve a deterioration of $\Pt$ w.r.t.\ $r_m$, i.e.\ a reduction of the cardinality of $\dnx{}(\Pt)$ whereas $p(\dnx{}(\Pt))$ remains smaller than $0.5$. Lines~\ref{algoline:dx+_prune:BME_p(dx)>0.5}-\ref{algoline:dx+_prune:BME_return_true_1} in the algorithm account for that. 
	
	If, on the other hand, (a)~$p(\dx{}(\Pt)) < 0.5$ (line~\ref{algoline:dx+_prune:BME_p(dnx)>0.5}), and (b)~deleting (at least) one diagnosis from $\dnx{}(\Pt)$ makes $\dnx{}(\Pt)$ at most as large in size as $\dx{}(\Pt)$ (line~\ref{algoline:dx+_prune:BME_|dx|+1>=|dnx|}), and (c)~the elimination of any diagnosis from $\dnx{}(\Pt)$ makes the sum of probabilities of diagnoses in $\dnx{}(\Pt)$ less than $0.5$ (line~\ref{algoline:dx+_prune:BME_p(dx)-prob_min<=0.5}), then the search tree can be pruned below the node representing $\Pt$. 
	
	To make this apparent, we reason as follows: Let $\mD^*(\Pt)$ be the set in $\setof{\dx{}(\Pt),\dnx{}(\Pt)}$ satisfying $p(\mD^*(\Pt)) < 0.5$. Then, by (a), the value of $\mathsf{BME}$ for a query with q-partition $\Pt$ is $|\mD^*(\Pt)| = |\dx{}(\Pt)|$. Assume an arbitrary (direct or indirect) successor $\Pt'$ of $\Pt$. Then, $\dnx{}(\Pt') \subset \dnx{}(\Pt)$. 
	By (b), this yields $|\dnx{}(\Pt')| \leq |\dx{}(\Pt)|$. Moreover, by (c), it must be the case that $p(\dnx{}(\Pt')) < 0.5$. 
	Consequently, $\mD^*(\Pt') = \dnx{}(\Pt')$, i.e.\ the value of $\mathsf{BME}$ for a query with q-partition $\Pt'$ is $|\mD^*(\Pt')| = |\dnx{}(\Pt')| \leq |\dx{}(\Pt)| = |\mD^*(\Pt)|$. So, since $r_m$ postulates maximal $|\mD^*(\Pt_{\mathsf{opt}})|$ of some optimal q-partition $\Pt_{\mathsf{opt}}$, no successor $\Pt'$ can be better as per $r_m$ than $\Pt$ and the algorithm returns $\true$ in line~\ref{algoline:dx+_prune:BME_return_true_2}. Otherwise, $\false$ is returned in line~\ref{algoline:dx+_prune:BME_return_false}.
\end{itemize}

\begin{algorithm}
\small
\caption[Finding the Best Successor Q-Partition]{Best successor in $\dx{}$-Partitioning}\label{algo:dx+_best_successor}
\begin{algorithmic}[1]
\Require a partition $\Pt$, the set of successor q-partitions $sucs$ of $\Pt$ (each resulting from $\Pt$ by a minimal $\dx{}$-transformation), a (diagnosis) probability measure $p$, requirements $r_m$ to an optimal q-partition 
\Ensure a tuple $\tuple{\Pt',\md}$ consisting of the best q-partition $\Pt' \in sucs$ according to the implemented heuristics (the smaller the heuristics, the better the corresponding q-partitions) and some diagnosis $\md$ that has been added to $\dx{}(\Pt')$ from $\dnx{}(\Pt)$ in the course of the minimal $\dx{}$-transformation that maps $\Pt$ to $\Pt'$ 
\Procedure{bestSuc}{$\Pt,sucs, p, r_m$}
\State $S_{\mathsf{best}} \gets \Call{getFirst}{sucs}$ \label{algoline:dx+_best_successor:get_first}
\For{$S \in sucs$}
	\If{$\Call{heur}{S, p, r_m} < \Call{heur}{S_{\mathsf{best}}, p, r_m}$}
			\State $S_{\mathsf{best}} \gets S$				\label{algoline:dx+_best_successor:update_Sbest}
	\EndIf
\EndFor  \label{algoline:dx+_best_successor:end_for}
\State $\dx{\mathsf{diff}} \gets \dx{}(S_{\mathsf{best}}) \setminus \dx{}(\Pt)$   \label{algoline:dx+_best_successor:compute_dx+Diff}
\State $\md \gets \Call{getFirst}{\dx{\mathsf{diff}}}$ \label{algoline:dx+_best_successor:getFirst_from_dx+Diff}
\State \textbf{return} $\tuple{S_{\mathsf{best}},\md}$
\EndProcedure
\vspace{2pt}
\hrule
\vspace{2pt}
\Procedure{heur}{$\langle \dx{}, \dnx{}, \emptyset \rangle, p, r_m$}    \label{algoline:dx+_best_successor:procedure_heur_start}
\State $\mD \gets \dx{} \cup \dnx{}$	 \label{algoline:dx+_best_successor:reconstruct_leading_diags}				\Comment{reconstruct leading diagnoses from $\dx{}$ and $\dnx{}$}
\If{$m \in \setof{\mathsf{ENT}}$}   \label{algoline:dx+_best_successor:if_case_ENT}  \Comment{if $m$ in eq.\ class {\larger\textcircled{\smaller[2]1}} w.r.t.\ Table~\ref{tab:requirements_for_equiv_classes_of_measures_wrt_equiv_mQ}}
			\State \textbf{return} $|p(\dx{}) - 0.5|$
\EndIf
\If{$m \in \setof{\mathsf{SPL}}$}   \label{algoline:dx+_best_successor:if_case_SPL}   \Comment{if $m$ in eq.\ class {\larger\textcircled{\smaller[2]2}} w.r.t.\ Table~\ref{tab:requirements_for_equiv_classes_of_measures_wrt_equiv_mQ}}
			\State \textbf{return} $\left||\dx{}| - \frac{|\mD|}{2}\right|$
\EndIf
\If{$m \in \setof{\mathsf{RIO}}$}		\label{algoline:dx+_best_successor:if_case_RIO}    \Comment{if $m$ in eq.\ class {\larger\textcircled{\smaller[2]3}} w.r.t.\ Table~\ref{tab:requirements_for_equiv_classes_of_measures_wrt_equiv_mQ}}
	\State $c \gets \Call{getCaut}{r_m}$   \label{algoline:dx+_best_successor:getCaut}
	\State $n \gets \lceil \uc |\mD|\rceil$		 \label{algoline:dx+_best_successor:compute_n}					\Comment{to be accepted by $\mathsf{RIO}$, $\dx{}$ must include at least $n$ diagnoses}
	\State $numDiagsToAdd \gets n - |\dx{}|$		\label{algoline:dx+_best_successor:RIO_compute_numDiagsToAdd}		\Comment{\# of diagnoses to be added to $\dx{}$ to achieve $|\dx{}| = n$}   
	\State $avgProb \gets \frac{p(\dnx{})}{|\dnx{}|}$	\label{algoline:dx+_best_successor:RIO_compute_avgProb}	\Comment{average probability of diagnoses that might be added to $\dx{}$}
			\State \Return $\left|p(\dx{}) + (numDiagsToAdd) * (avgProb) - 0.5\right|$		\label{algoline:dx+_best_successor:RIO_return}	
\EndIf
\If{$m \in \setof{\mathsf{KL},\mathsf{EMCb}}$}	\label{algoline:dx+_best_successor:if_case_KL}	\Comment{$\frac{|\dx{}|}{|\mD|}$ is expected value (E) of sum of prob.\ of $|\dx{}|$ diagnoses in $\mD$}
			\State \textbf{return} $\frac{|\dx{}|}{|\mD| p(\dx{})} $	\label{algoline:dx+_best_successor:KL_return_1}	\Comment{the higher $p(\dx{})$ compared to E, the better the heuristics for $\langle \dx{}, \dnx{}, \emptyset \rangle$}		
\EndIf 
\If{$m \in \setof{\mathsf{MPS}}$}    \label{algoline:dx+_best_successor:if_case_MPS}   \Comment{if $m$ in eq.\ class {\larger\textcircled{\smaller[2]6}} w.r.t.\ Table~\ref{tab:requirements_for_equiv_classes_of_measures_wrt_equiv_mQ}}
			\State \textbf{return} $-p(\dx{})$       \label{algoline:dx+_best_successor:MPS_return}
\EndIf			\label{algoline:dx+_best_successor:MPS_end}
\If{$m \in \setof{\mathsf{BME}}$}    \label{algoline:dx+_best_successor:if_case_BME}
			\If{$p(\dx{}) < 0.5$}  \label{algoline:dx+_best_successor:BME_if_p(dx)>0.5}
				\State \textbf{return} $-|\dx{}| + p(\dx{})$		\label{algoline:dx+_best_successor:BME_return_1}		
			\EndIf
			\If{$p(\dx{}) > 0.5$}	 \label{algoline:dx+_best_successor:BME_if_p(dx)<0.5}		\Comment{$p(\dnx{}) < 0.5$}
				\State \Return $-|\dnx{}| + p(\dnx{})$			 \label{algoline:dx+_best_successor:BME_return_2}	
			\EndIf
			\State \Return $0$					\Comment{$p(\dx{}) = p(\dnx{}) = 0.5$}	
\EndIf
\EndProcedure   \label{algoline:dx+_best_successor:procedure_heur_end}
\end{algorithmic}
\normalsize
\end{algorithm}

\paragraph{Algorithm to Find Best Successor.} 
In the following, we explicate Algorithm~\ref{algo:dx+_best_successor} on page~\pageref{algo:dx+_best_successor}. Given a partition $\Pt$ (the current node in the search tree), a set of q-partitions $sucs$ each element of which is a q-partition resulting from $\Pt$ by a minimal $\dx{}$-transformation (cf.\ Definition~\ref{def:minimal_transformation}), a (diagnosis) probability measure $p$ and some requirements $r_m$ to an optimal q-partition as inputs, the algorithm returns the best q-partition $\Pt$ in $sucs$ according to the implemented heuristics. The heuristics is constructed from $r_m$ where a smaller value for the heuristics indicates a better q-partition.

The algorithm iterates once over the set $sucs$ to extract the q-partition in $sucs$ with minimal heuristics (lines~\ref{algoline:dx+_best_successor:get_first}-\ref{algoline:dx+_best_successor:end_for}) and stores this q-partition in $S_{\mathsf{best}}$. At this, \textsc{getFirst}($X$) returns the first element in the set $X$ given to it as an argument and afterwards deletes this element from $X$. The calculation of the heuristics for a given q-partition $\tuple{\dx{},\dnx{},\emptyset}$ and $r_m$ is realized by the function \textsc{heur} (line~\ref{algoline:dx+_best_successor:procedure_heur_start}-\ref{algoline:dx+_best_successor:procedure_heur_end}) which we describe next. 

First of all, in line~\ref{algoline:dx+_best_successor:reconstruct_leading_diags}, the set of leading diagnoses $\mD$, which is needed in the computations performed by the algorithm, is reconstructed by means of $\dx{}$ and $\dnx{}$. What comes next is the check what the used query quality measure $m$ is. Depending on the outcome, the determination of the output of \textsc{heur} is realized in different ways, steered by $r_m$ (cf.\ Table~\ref{tab:requirements_for_equiv_classes_of_measures_wrt_equiv_mQ}). In the case of 
\begin{itemize}
	\item the measure $\mathsf{ENT}$ (line~\ref{algoline:dx+_best_successor:if_case_ENT}) and equivalent measures as per the section ``$\equiv_{\mQ}$'' of Table~\ref{tab:measure_equiv_classes}, the heuristic evaluates to the absolute difference between $p(\dx{})$ and the theoretical optimum of $0.5$. That is, the closer the probability of the $\dx{}$ set of a q-partition in $sucs$ comes to a half, the better the q-partition is ranked by the heuristics.
	\item the measure $\mathsf{SPL}$ (line~\ref{algoline:dx+_best_successor:if_case_SPL}) and equivalent measures as per the section ``$\equiv_{\mQ}$'' of Table~\ref{tab:measure_equiv_classes}, the heuristic evaluates to the absolute difference between $|\dx{}|$ and the theoretical optimum of $\frac{|\mD|}{2}$. That is, the closer the cardinality of the $\dx{}$ set of a q-partition in $sucs$ comes to half the number of leading diagnoses, the better the q-partition is ranked by the heuristics.
	\item the measure $\mathsf{RIO}$ (line~\ref{algoline:dx+_best_successor:if_case_RIO}) and equivalent measures as per the section ``$\equiv_{\mQ}$'' of Table~\ref{tab:measure_equiv_classes}, the computation of the heuristic value works as follows. First, the minimal number $n$ of diagnoses that must be eliminated by the answer to the next query is computed (lines~\ref{algoline:dx+_best_successor:getCaut}-\ref{algoline:dx+_best_successor:compute_n}). By means of $n$, which is the target value for the cardinality of $\dx{}$ (cf.\ Section~\ref{sec:rio}), the number $numDiagsToAdd$ of diagnoses that must still be added to $\dx{}$ to reach this target value, is computed in line~\ref{algoline:dx+_best_successor:RIO_compute_numDiagsToAdd}. Furthermore, the average probability of a diagnosis in $\dnx{}$ is calculated and stored in $avgProb$ in line~\ref{algoline:dx+_best_successor:RIO_compute_avgProb}. Finally, the returned heuristic value in line~\ref{algoline:dx+_best_successor:RIO_return} gives the deviance of $p(\dx{}(\Pt'))$ from $0.5$ (the optimal value of $p(\dx{})$ w.r.t.\ the $\mathsf{ENT}$ or $\mathsf{ENT}_z$ measure) that would result if $numDiagsToAdd$ diagnoses, each with exactly the average probability $avgProb$ of diagnoses in $\dnx{}$, were added to $\dx{}$ yielding the new q-partition $\Pt'$. That is, comparing with $r_m$ given in the third row of Table~\ref{tab:requirements_for_equiv_classes_of_measures_wrt_equiv_mQ}, a lower heuristic value means higher proximity to optimality w.r.t.\ condition $(\mathrm{III})$ under the assumption that the $\dx{}$-set of the current q-partition is filled up in a way to satisfy exactly the optimality conditions $(\mathrm{I})$ and $(\mathrm{II})$.
	\item the measures $\mathsf{KL}$ and $\mathsf{EMCb}$ (see line~\ref{algoline:dx+_best_successor:if_case_KL}), as per Propositions~\ref{prop:kl_opt},(\ref{prop:kl_opt:properties_of_best_query}.) and \ref{prop:EMCb_opt},(\ref{prop:EMCb_opt:properties_of_best_query}.), 
	$\Pt$ is the more promising a q-partition, the more the sum of diagnoses probabilities in $\dx{}$ (or $\dnx{}$, respectively) exceeds the expected sum of diagnoses probabilities in $\dx{}$ ($\dnx{}$). Under the assumption that all leading diagnoses probabilities satisfy a distribution with a mean of $\frac{1}{|\mD|}$, the expected value of the sum of diagnoses probabilities in $\dx{}$ can be computed as $\frac{|\dx{}|}{|\mD|}$, and as $\frac{|\dnx{}|}{|\mD|}$ for the $\dnx{}$ case. As a consequence, the value of the heuristics returned for $\Pt$ is either $\frac{|\dx{}|}{|\mD| p(\dx{})}$, as shown in the algorithm (line~\ref{algoline:dx+_best_successor:KL_return_1}). That is, a high sum of diagnoses probabilities $p(\dx{})$ compared to the expected probability sum means a better q-partition and hence involves a low heuristic value, as desired, due to the placement of $p(\dx{})$ in the denominator. An alternative would be tu use $\frac{|\dnx{}|}{|\mD| p(\dnx{})}$ as heuristic value. A usage of the latter would imply the exploration of q-partitions resulting from the addition of smallest probabilities to $\dx{}$ first instead of highest ones as in case of the former. Stated in the terminology of Proposition~\ref{prop:kl_opt},(\ref{prop:kl_opt:properties_of_best_query}.), the former approach seeks to explore the q-partitions with $\dx{}$-sets in $\mathbf{MaxP}_k^+$ first while the latter prefers the q-partitions with $\dnx{}$-sets in $\mathbf{MaxP}_m^-$.
	\item the measure $\mathsf{MPS}$ (line~\ref{algoline:dx+_best_successor:if_case_MPS}) and equivalent measures as per the section ``$\equiv_{\mQ}$'' of Table~\ref{tab:measure_equiv_classes}, clearly, the best successor to select is the one with maximal 
	$p(\dx{})$. 
	This is reflected by returning a heuristic value of $-p(\dx{})$ in line~\ref{algoline:dx+_best_successor:MPS_return}.
	\item the measure $\mathsf{BME}$ (see line~\ref{algoline:dx+_best_successor:if_case_BME}), we recall that an optimal q-partition w.r.t.\ $r_m$ satisfies that its set in $\setof{\dx{},\dnx{}}$ with minimal probability has a maximal possible number of diagnoses in it. Hence if $p(\dx{}) < 0.5$ (line~\ref{algoline:dx+_best_successor:BME_if_p(dx)>0.5}), i.e.\ $\dx{}$ has minimal probability, then, the larger $|\dx{}|$ is, the better the heuristic value of the q-partition should be. Furthermore, the lower the probability $p(\dx{})$ is, the larger the number of diagnoses which can be added to $\dx{}$ without exceeding the probability of $0.5$ tends to be. The return value $-|\dx{}| + p(\dx{})$ in line~\ref{algoline:dx+_best_successor:BME_return_1} accounts for this (recall that a lower heuristic value signifies a better q-partition). Note that the maximality of the number of diagnoses in $\dx{}$ has higher priority than the minimality of the probability mass in $\dx{}$. The latter should serve as a tie-breaker between q-partitions that are equally good w.r.t.\ the $\dx{}$-cardinality. This is in fact achieved by the returned value since an improvement or a deterioration w.r.t.\ the cardinality always amounts to at least 1 whereas probabilities $p(\dx{})$ are always smaller than 1 due to the fact that $\dx{} \subset \mD$ for all q-partitions (cf.\ Proposition~\ref{prop:properties_of_q-partitions}).
	
	On the other hand, if $\dnx{}$ has minimal probability (line~\ref{algoline:dx+_best_successor:BME_if_p(dx)<0.5}), then $-|\dnx{}| + p (\dnx{})$ is returned due to an analogue argumentation (line~\ref{algoline:dx+_best_successor:BME_return_2}). In the situation where $p(\dx{}) = 0.5 = p(\dnx{})$, a heuristic value of zero is returned, signifying rejection of this successor q-partition, as it does not comply with the priority $(\mathrm{I})$ requirement in $r_m$, i.e.\ that either $p(\dx{})$ or $p(\dnx{})$ be smaller than $0.5$.
\end{itemize}
In line~\ref{algoline:dx+_best_successor:compute_dx+Diff}, the algorithm computes $\dx{\mathsf{diff}}$ as the set of diagnoses that have been transferred from $\dnx{}$ in the current partition $\Pt$ to $\dx{}$ in the successor q-partition $S_{\mathsf{best}}$ in the course of the minimal $\dx{}$-transformation that maps $\Pt$ to $S_{\mathsf{best}}$. Then, the first diagnosis in the set $\dx{\mathsf{diff}}$ is assigned to the variable $\md$ (line~\ref{algoline:dx+_best_successor:getFirst_from_dx+Diff}). Note, by the definition of a minimal $\dx{}$-transformation (Definition~\ref{def:minimal_transformation}), $\dx{\mathsf{diff}} \neq \emptyset$ must hold. The tuple $\tuple{S_{\mathsf{best}},\md}$ is finally returned. 

The reason of returning one (arbitrary) diagnosis $\md$ from $\dx{\mathsf{diff}}$ in addition to the best successor is the pruning of the search tree. That is, $\md$ is leveraged in all subtrees of $\Pt$ that have not yet been explored in order to avoid the repeated generation of q-partitions that have already been generated and analyzed in explored subtrees of $\Pt$. Instead of the entire set $\dx{\mathsf{diff}}$, only one diagnosis out of it is needed for this purpose since $\dx{\mathsf{diff}}$ represents exactly one equivalence class of diagnoses. That is, if one diagnosis from this set must not be transferred to $\dx{}$ from $\dnx{}$ in the context of a minimal $\dx{}$-transformation, then none of the diagnoses in this set must do so (see also Definition~\ref{def:necessary_follower} and Corollary~\ref{cor:nec_followers_form_equivalence_class_wrt_md^(y)} later).

\subsubsection{Q-Partition Successor Computation for $\dx{}$-Partitioning}
\label{sec:AlgorithmForSuccessorComputation}
%
We first characterize $S_{\mathsf{init}}$ and then $S_{\mathsf{next}}$. 
\paragraph{Definition of $S_{\mathsf{init}}$.} In the case of $\dx{}$-partitioning, we specified the initial state of the q-partition search as $\tuple{\dx{},\dnx{},\emptyset} := \tuple{\emptyset,\mD,\emptyset}$ (cf.\ page~\pageref{etc:initial_state}). Now, $S_{\mathsf{init}}$ can be easily specified by means of the following corollary which is a direct consequence of Proposition~\ref{prop:properties_of_q-partitions}. It states that for any single diagnosis $\md\in\mD$ there is a canonical q-partition with $\dx{}(Q) = \setof{\md}$.
\begin{corollary}\label{cor:--di,MD-di,0--_is_canonical_q-partition_for_all_di_in_mD}
Let $\mD\subseteq\minD_{\langle\mo,\mb,\Tp,\Tn\rangle_\RQ}$ with $|\mD|\geq 2$. Then, $\langle\{\md_i\},\mD \setminus \{\md_i\},\emptyset\rangle$ is a canonical q-partition for all $\md_i \in \mD$.
\end{corollary}
\begin{proof}
That $\langle\{\md_i\},\mD \setminus \{\md_i\},\emptyset\rangle$ is a q-partition follows immediately from Proposition~\ref{prop:properties_of_q-partitions},(\ref{prop:properties_of_q-partitions:enum:D+=d_i_is_q-partition_and_lower_bound_of_queries}.). We now show that it is canonical. Also from Proposition~\ref{prop:properties_of_q-partitions},(\ref{prop:properties_of_q-partitions:enum:D+=d_i_is_q-partition_and_lower_bound_of_queries}.), we obtain that $Q := U_\mD \setminus \md_i$ is an (explicit-entailments) query associated with the q-partition $\langle \dx{}(Q),\dnx{}(Q),\dz{}(Q)\rangle = \langle\{\md_i\},\mD \setminus \{\md_i\},\emptyset\rangle$. Now, $E_{\mathsf{exp}}(\dx{}(Q)) = \mo\setminus U_{\dx{}(Q)} = \mo \setminus \md_i$ due to Proposition~\ref{prop:E_exp}. As $U_\mD \subseteq \mo$ and $I_\mD \subseteq \md_i$, we can infer that $Q_{\mathsf{can}}(\dx{}(Q)) = E_{\mathsf{exp}}(\dx{}(Q)) \cap \DiscAx_\mD = U_\mD \setminus \md_i$. Hence, $Q_{\mathsf{can}}(\dx{}(Q)) = Q$ wherefore $\langle \dx{}(Q),\dnx{}(Q),\dz{}(Q)\rangle$ is a canonical q-partition.
%
%
\end{proof}
Since only one diagnosis is transferred from the initial $\dnx{}$-set $\mD$ to the $\dx{}$-set of the q-partition to obtain any q-partition as per Corollary~\ref{cor:--di,MD-di,0--_is_canonical_q-partition_for_all_di_in_mD} from the initial state, it is clear that all these q-partitions indeed result from the application of a minimal $\dx{}$-transformation from the initial state. As for all other (q-)partitions the $\dx{}$-set differs by more than one diagnosis from the initial $\dx{}$-set $\emptyset$, 
it is obvious that all other q-partitions do not result from the initial state by some minimal $\dx{}$-transformation. Hence:
\begin{proposition}\label{prop:S_init_sound+complete}
Given the initial state $\Pt_0 := \tuple{\emptyset,\mD,\emptyset}$, the function 
\[S_{\mathsf{init}}: \tuple{\emptyset,\mD,\emptyset} \mapsto \setof{\tuple{\setof{\md},\mD\setminus\setof{\md},\emptyset}\,|\,\md\in\mD}\] 
is sound and complete, i.e.\ it produces from $\Pt_0$ all and only (canonical) q-partitions resulting from $\Pt_0$ by minimal $\dx{}$-transformations.
\end{proposition} 

\begin{example}\label{ex:S_init}
Let us consider our example DPI given by Table~\ref{tab:example_dpi_0} and assume that the set of leading diagnoses $\mD$ is the set of all minimal diagnoses $\minD_{\tuple{\mo,\mb,\Tp,\Tn}_\RQ}$. Then, given the initial state
\begin{align}
\tuple{\emptyset, \setof{\md_1,\dots,\md_6}, \emptyset}
\end{align}
the canonical successor q-partitions of this state are
\begin{align}
&\tuple{\setof{\md_1}, \setof{\md_2,\dots,\md_6}, \emptyset} \\
&\tuple{\setof{\md_2}, \setof{\md_1,\md_3,\dots,\md_6}, \emptyset} \\
&\tuple{\setof{\md_3}, \setof{\md_1,\md_2,\md_4,\md_5,\md_6}, \emptyset} \\
&\tuple{\setof{\md_4}, \setof{\md_1,\md_2,\md_3,\md_5,\md_6}, \emptyset} \\
&\tuple{\setof{\md_5}, \setof{\md_1,\dots,\md_4,\md_6}, \emptyset} \\
&\tuple{\setof{\md_6}, \setof{\md_1,\dots,\md_5}, \emptyset} 
\end{align}
There are no other canonical successor q-partitions of the initial state in the sense of Definition~\ref{def:minimal_transformation} (i.e.\ obtainable by a minimal $\dx{}$-transformation).\qed
\end{example}

\paragraph{Definition of $S_{\mathsf{next}}$.} In order to define $S_{\mathsf{next}}$, we utilize Proposition~\ref{prop:suff+nec_criteria_when_partition_is_q-partition} which 
provides sufficient and necessary criteria when a partition of $\mD$ is a canonical q-partition.
\begin{proposition}\label{prop:suff+nec_criteria_when_partition_is_q-partition}
Let $\mD\subseteq\minD_{\langle\mo,\mb,\Tp,\Tn\rangle_\RQ}$ and $\Pt = \langle \dx{},\dnx{},\emptyset\rangle$ be a partition w.r.t.\ $\mD$ with $\dx{}\neq \emptyset$ and $\dnx{}\neq \emptyset$. Then, $\Pt$ is a canonical q-partition iff 
\begin{enumerate}
	\item $U_{\dx{}} \subset U_{\mD}$ and
	\item there is no $\md_j \in \dnx{}$ such that $\md_j \subseteq U_{\dx{}}$.
\end{enumerate}
\end{proposition}
\begin{proof}
``$\Rightarrow$'': We prove the ``only-if''-direction by contradiction. That is, we derive a contradiction by assuming that $\Pt$ is a canonical \text{q-partition} and that $\neg (1.)$ or $\neg (2.)$ is true.

By the premise that $\Pt$ is a canonical q-partition, the query $Q_{\mathsf{can}}(\dx{}) := E_{\mathsf{exp}}(\dx{}) \cap \DiscAx_\mD$ must have exactly $\Pt$ as its associated q-partition. Since $\mo_i^{*}$ does not violate any $x \in \RQ \cup \Tn$, but, due to $\dnx{}\neq \emptyset$, there is some $\md_i \in \dnx{}(Q_{\mathsf{can}}(\dx{})) = \dnx{}$ such that $\mo_i^{*}\cup Q_{\mathsf{can}}(\dx{})$ does violate some $x \in \RQ \cup \Tn$, we can conclude that $Q_{\mathsf{can}}(\dx{}) \neq \emptyset$ and thence $E_{\mathsf{exp}}(\dx{}) \neq \emptyset$.

By Proposition~\ref{prop:E_exp}, $E_{\mathsf{exp}}(\dx{}) = \mo \setminus U_{\dx{}}$. So, $(*): \emptyset \subset Q_{\mathsf{can}}(\dx{}) \subseteq \mo \setminus U_{\dx{}} \subseteq (\mo \setminus U_{\dx{}}) \cup \mb \cup U_\Tp 
\subseteq (\mo \setminus \md_i) \cup \mb \cup U_\Tp =: \mo_i^*$ for all $\md_i \subseteq U_{\dx{}}$. 

Now, assuming that $(1)$ is false, i.e.\ $U_{\dx{}} \not\subset U_{\mD}$, we observe that this is equivalent to $U_{\dx{}} = U_{\mD}$ since $U_{\dx{}} \subseteq U_{\mD}$ due to $\dx{} \subseteq \mD$. Due to $U_{\dx{}} = U_\mD$, $(*)$ is true for all $\md_i\in\mD$. It follows that $\mo_i^* \supseteq Q_{\mathsf{can}}(\dx{})$ and, due to the fact that the entailment relation in $\mathcal{L}$ is extensive, that $\mo_i^* \models Q_{\mathsf{can}}(\dx{})$. Therefore, we can conclude using Definition~\ref{def:q-partition} that $\dx{} = \dx{}(Q_{\mathsf{can}}(\dx{})) = \mD$ and thus $\dnx{} = \emptyset$. The latter is a contradiction to $\dnx{} \neq \emptyset$.

Assuming $\neg (2)$, on the other hand, we obtain that there is some diagnosis $\md_j\in\dnx{}$ with $\md_j \subseteq U_{\dx{}}$. By $(*)$, however, we can derive that $\mo_j^* \models Q_{\mathsf{can}}(\dx{})$ and therefore $\md_j \in \dx{}(Q_{\mathsf{can}}(\dx{})) = \dx{}$ which contradicts $\md_j\in\dnx{}$ by the fact that $\Pt$ is a partition w.r.t.\ $\mD$ which implies $\dx{} \cap \dnx{} = \emptyset$.
%
%

``$\Leftarrow$'': To show the ``if''-direction, we must prove that $\Pt$ is a canonical q-partition, i.e.\ that $\Pt$ is a q-partition and that 
$\Pt$ is exactly the q-partition associated with $Q_{\mathsf{can}}(\dx{})$ given that $(1)$ and $(2)$ hold. 

By $\dx{} \neq \emptyset$ and $(1)$, it is true that $\emptyset \subset U_{\dx{}} \subset U_\mD$. So, there is some axiom $\tax \in U_\mD \subseteq \mo$ such that $(**): \tax \notin U_{\dx{}}$. Hence, $\tax \in \mo \setminus U_{\dx{}}$, and, by Proposition~\ref{prop:E_exp}, $\tax\in E_{\mathsf{exp}}(\dx{})$. Now, $Q_{\mathsf{can}}(\dx{}) := E_{\mathsf{exp}}(\dx{}) \cap \DiscAx_\mD = E_{\mathsf{exp}}(\dx{}) \cap (U_\mD \setminus I_\mD)$. Since $\tax \in U_\mD$, in order to show that $\tax \in Q_{\mathsf{can}}(\dx{})$, we must demonstrate that $\tax \notin I_\mD$. Therefore, assume that $\tax \in I_\mD$. This implies that $\tax \in \md_i$ for all $\md_i \in \mD$ and particularly for $\md_i \in \dx{}$ which yields $\tax \in U_{\dx{}}$. This, however, is a contradiction to $(**)$. So, we conclude that $\tax \in Q_{\mathsf{can}}(\dx{})$, i.e.\ (Q1): $Q_{\mathsf{can}}(\dx{}) \neq\emptyset$. 

More precisely, since $\tax$ was an arbitrary axiom in $U_\mD$ with property $(**)$, we have that $U_\mD \setminus U_{\dx{}} \subseteq Q_{\mathsf{can}}(\dx{})$. By $(2)$, for all $\md_j \in \dnx{}$ there is an axiom $\tax_j \in \mo$ such that $\tax_j \in \md_j \subset U_\mD$ and $\tax_j \notin U_{\dx{}}$ which implies $\tax_j \in U_\mD \setminus U_{\dx{}} \subseteq Q_{\mathsf{can}}(\dx{})$. Hence, $\mo_j^* \cup Q_{\mathsf{can}}(\dx{})$ must violate some $x \in \RQ\cup\Tn$ since $\tax_j \in \mo$ and by the subset-minimality of $\md_j \in \dnx{}$. Consequently, $\dnx{} \subseteq \dnx{}(Q_{\mathsf{can}}(\dx{}))$. As $\dnx{} \neq \emptyset$ by assumption, we have that (Q2): $\emptyset \subset \dnx{}(Q_{\mathsf{can}}(\dx{}))$.

That $\mo_i^* \models Q_{\mathsf{can}}(\dx{})$ for $\md_i \in \dx{}$ follows by the same argumentation that was used above in $(*)$. Thus, we obtain $\dx{} \subseteq \dx{}(Q_{\mathsf{can}}(\dx{}))$. As $\dx{} \neq \emptyset$ by assumption, we have that (Q3): $\emptyset \subset \dx{}(Q_{\mathsf{can}}(\dx{}))$.
Note that (Q1) -- (Q3) imply that $Q_{\mathsf{can}}(\dx{})$ is a query w.r.t.\ $\mD$ (cf.\ Definition~\ref{def:query}).

Now, let us assume that at least one of the two derived subset-relations is proper, i.e.\ (a)~$\dnx{} \subset \dnx{}(Q_{\mathsf{can}}(\dx{}))$ or (b)~$\dx{} \subset \dx{}(Q_{\mathsf{can}}(\dx{}))$. If (a) holds, then there is some $\md \in \dnx{}(Q_{\mathsf{can}}(\dx{}))$ which is not in $\dnx{}$. Hence, $\md\in\dx{}$ or $\md\in\dz{}$. The former case is impossible since $\md\in\dx{}$ implies $\md\in\dx{}(Q_{\mathsf{can}}(\dx{}))$ by $\dx{} \subseteq \dx{}(Q_{\mathsf{can}}(\dx{}))$, which yields $\dnx{}(Q_{\mathsf{can}}(\dx{})) \cap \dx{}(Q_{\mathsf{can}}(\dx{})) \supseteq \setof{\md} \supset \emptyset$, a contradiction to the fact that $Q_{\mathsf{can}}(\dx{})$ is a query w.r.t.\ $\mD$ and Proposition~\ref{prop:properties_of_q-partitions},(\ref{prop:properties_of_q-partitions:enum:q-partition_is_partition}.).
The latter case cannot be true either as $\dz{} = \emptyset$ by assumption.
In an analogue way we obtain a contradiction if we assume that case (b) holds. So, it must hold that $\Pt = \tuple{\dx{}(Q_{\mathsf{can}}(\dx{})),\dnx{}(Q_{\mathsf{can}}(\dx{})),\emptyset}$. This finishes the proof.
\end{proof}

\begin{example}\label{ex:Pt_is_q-partition_iff}
Let (by referring by $i$ to $\tax_i$)
\begin{align}
\begin{split} \label{eq:ex:Pt_is_q-partition_iff:Pt_1}
\Pt_1 :=& \tuple{\setof{\md_1,\md_2,\md_3},\setof{\md_4,\md_5,\md_6},\emptyset} \\
			=& \tuple{\setof{\setof{2,3},\setof{2,5},\setof{2,6}},\setof{\setof{2,7},\setof{1,4,7},\setof{3,4,7}},\emptyset} 
\end{split} \\
\begin{split} \label{eq:ex:Pt_is_q-partition_iff:Pt_2}
\Pt_2 :=& \tuple{\setof{\md_1,\md_2,\md_5},\setof{\md_3,\md_4,\md_6},\emptyset} \\
			=& \tuple{\setof{\setof{2,3},\setof{2,5},\setof{1,4,7}},\setof{\setof{2,6},\setof{2,7},\setof{3,4,7}},\emptyset}  
\end{split}
\end{align}
Then $U_{\dx{}(\Pt_1)} = \setof{2,3,5,6}$, $U_{\dx{}(\Pt_2)} = \setof{1,2,3,4,5,7}$ and $U_\mD = \setof{1,\dots,7}$. Since $U_{\dx{}(\Pt_1)} \subset U_\mD$ as well as $U_{\dx{}(\Pt_2)} \subset U_\mD$, the first condition of Proposition~\ref{prop:suff+nec_criteria_when_partition_is_q-partition} is satisfied for both partitions of $\mD$. As to the second condition, given that $7 \in \md_4,\md_5,\md_6$, but $7 \not\in U_{\dx{}(\Pt_1)}$, it holds that $\md_j \not\subseteq U_{\dx{}(\Pt_1)}$ for all $j \in \setof{4,5,6}$. Therefore, $\Pt_1$ is a canonical q-partition. $\Pt_2$, on the other hand, is no canonical q-partition because, e.g., $\md_4 = \setof{2,7} \subset \setof{1,2,3,4,5,7} = U_{\dx{}(\Pt_2)}$ (second condition of Proposition~\ref{prop:suff+nec_criteria_when_partition_is_q-partition} violated).

Intuitively, the reason why $\Pt_2$ fails to be a canonical q-partition is that if $\md = \setof{1,4,7}$ moves from the $\dnx{}$-set to the $\dx{}$-set (which at this point contains only $\setof{2,3}$ and $\setof{2,5}$) in the course of the minimal $\dx{}$-transformation, the diagnoses $\setof{2,7}$ and $\setof{3,4,7}$ must necessarily follow $\md$ to $\dx{}$.
For, after this shift of $\md$, it holds that $\setof{2,7} \subseteq U_{\dx{}(\Pt_2)}$ and $\setof{3,4,7} \subseteq U_{\dx{}(\Pt_2)}$. This means that both $\mo \setminus \setof{2,7}$ and $\mo \setminus \setof{3,4,7}$ entail the canonical query $\setof{6}$ w.r.t.\ the seed $\dx{}(\Pt_2)$ while $\setof{2,7}$ and $\setof{3,4,7}$ are elements of $\dnx{}(\Pt_2)$. In fact, it can be easily verified that this canonical query has the q-partition $\tuple{\mD\setminus\setof{\md_3}, \setof{\md_3}, \emptyset}$ which is different from $\Pt_2$. So, by Definition~\ref{def:q-partition}, $\Pt_2$ is not a q-partition.
 

To sum up, starting from a q-partition, we can only generate a canonical q-partition from it by transferring at least one diagnosis $\md$ from $\dnx{}$ to $\dx{}$ and, together with this diagnosis, all the -- as well will call them later -- \emph{necessary followers} of $\md$.

Finally, we point out that one can verify that there is in fact no query with associated q-partition $\Pt_2$. That is, the partition $\Pt_2$ of $\mD$ is no canonical q-partition \emph{and} no q-partition.
\qed 
\end{example}

Let in the following for a DPI $\langle\mo,\mb,\Tp,\Tn\rangle_\RQ$ and a partition $\Pt_k = \langle \dx{k}, \dnx{k}, \dz{k}\rangle$ of $\mD$ and all $\md_i \in \mD \subseteq \minD_{\langle\mo,\mb,\Tp,\Tn\rangle_\RQ}$ 
\begin{align}
\md^{(k)}_i := \md_i \setminus U_{\dx{k}} \label{eq:md_i^(k)}
\end{align}
The next corollary establishes the relationship between Eq.~\eqref{eq:md_i^(k)} and canonical q-partitions based on Proposition~\ref{prop:suff+nec_criteria_when_partition_is_q-partition}:
\begin{corollary}\label{cor:not_q-partition_iff_md_i^(k)=emptyset_for_md_i_in_dnx_k}
Let $\mD\subseteq\minD_{\langle\mo,\mb,\Tp,\Tn\rangle_\RQ}$, $\Pt_k = \langle \dx{k}, \dnx{k}, \emptyset\rangle$ a partition of $\mD$ with $\dx{k},\dnx{k} \neq \emptyset$ and $U_{\dx{k}} \subset U_\mD$. Then $\Pt := \tuple{\dx{k},\dnx{k},\emptyset}$ is a canonical q-partition iff 
\begin{enumerate}
	\item $\md_i^{(k)} = \emptyset$ for all $\md_i \in \dx{k}$, and
	\item $\md_i^{(k)} \neq \emptyset$ for all $\md_i \in \dnx{k}$.
\end{enumerate}

\end{corollary}
\begin{proof}
Ad 1.: This proposition follows directly from the definition of $\md_i^{(k)} := \md_i \setminus U_{\dx{k}}$ (see Eq.~\eqref{eq:md_i^(k)}) and the trivial fact that $\md_i \subseteq U_{\dx{k}}$ for all $\md_i \in \dx{k}$. Thence, this proposition can never be false.

Ad 2.: We show the contrapositive of (2.), i.e.\ that $\Pt := \tuple{\dx{k},\dnx{k},\emptyset}$ is not a canonical q-partition iff $\md_i^{(k)} = \emptyset$ for some $\md_i \in \dnx{k}$:

``$\Leftarrow$'': 
By Proposition~\ref{prop:suff+nec_criteria_when_partition_is_q-partition}, a partition $\Pt_k = \langle \dx{k},\dnx{k},\emptyset\rangle$ with $\dx{k},\dnx{k} \neq \emptyset$ is a canonical q-partition iff (1)~$U_{\dx{k}} \subset U_{\mD}$ and (2)~there is no $\md_j \in \dnx{k}$ such that $\md_j \subseteq U_{\dx{k}}$. If $\md_j \in \dnx{k}$ and $\md_j^{(k)} := \md_j \setminus U_{\dx{k}} = \emptyset$, then $\md_j \subseteq U_{\dx{k}}$, which violates the necessary condition~(2). Therefore, $\Pt_k$ cannot be a canonical q-partition.

``$\Rightarrow$'': By Proposition~\ref{prop:suff+nec_criteria_when_partition_is_q-partition}, a partition $\Pt_k = \langle \dx{k},\dnx{k},\emptyset\rangle$ with $\dx{k},\dnx{k} \neq \emptyset$ is not a canonical q-partition iff ($\neg 1$)~$U_{\dx{k}} \not\subset U_{\mD}$ or ($\neg 2$)~there is some $\md_i \in \dnx{k}$ such that $\md_i \subseteq U_{\dx{k}}$. Since condition ($\neg 1$) is assumed to be false, condition~($\neg 2$) must be true, which implies that $\md_i^{(k)} = \md_i \setminus U_{\dx{k}} = \emptyset$ for some $\md_i \in \dnx{k}$.
\end{proof}
So, Corollary~\ref{cor:not_q-partition_iff_md_i^(k)=emptyset_for_md_i_in_dnx_k} along with Proposition~\ref{prop:canonical_q-partition_has_empty_dz}, which states that $\dz{}$ must be the empty set for all canonical q-partitions, demonstrate that $U_{\dx{}}$ already defines a canonical q-partition uniquely. Consequently, we can compute the number of canonical q-partitions w.r.t.\ a set of leading minimal diagnoses, as the next corollary shows. Note that the number of canonical q-partitions w.r.t.\ $\mD$ is equal to the number of canonical queries w.r.t.\ $\mD$ which in turn constitutes a lower bound of the number of \emph{all queries} w.r.t.\ $\mD$. 
\begin{corollary}\label{cor:upper_lower_bound_for_canonical_q-partitions}
Let $\mD\subseteq\minD_{\langle\mo,\mb,\Tp,\Tn\rangle_\RQ}$ with $|\mD| \geq 2$. Then, for the number $c$ of canonical q-partitions w.r.t.\ $\mD$ the following holds:
\begin{align}
c = \left|\setof{U_{\dx{}}\,|\, \emptyset \subset \dx{} \subset \mD} \setminus \setof{U_\mD}\right| \geq |\mD|
\label{eq:number_of_canonical_q-partitions}
\end{align} 
\end{corollary} 
\begin{proof}
The inequality holds due to Corollary~\ref{cor:--di,MD-di,0--_is_canonical_q-partition_for_all_di_in_mD}. 
We know by Proposition~\ref{prop:canonical_q-partition_has_empty_dz} that a canonical q-partition has empty $\dz{}$. Further, since it is a q-partition, $\dx{} \neq \emptyset$ and $\dnx{} \neq \emptyset$ (Proposition~\ref{prop:properties_of_q-partitions}). By Proposition~\ref{prop:suff+nec_criteria_when_partition_is_q-partition}, $U_{\dx{}} \subset U_{\mD}$ must hold as well. By Corollary~\ref{cor:not_q-partition_iff_md_i^(k)=emptyset_for_md_i_in_dnx_k}, there is a unique canonical q-partition for each set $U_{\dx{}}$. Further, different sets $U_{\dx{i}} \neq U_{\dx{j}}$ clearly imply different sets $\dx{i}$ and $\dx{j}$ and thus different canonical q-partitions. Hence, we must count exactly all different $U_{\dx{}}$ such that $U_{\dx{}} \subset U_{\mD}$. This is exactly what Eq.~\eqref{eq:number_of_canonical_q-partitions} specifies.
%
%
%
%
\end{proof}

\begin{example}\label{ex:md_i^(k)}
Let us examine both partitions discussed in Example~\ref{ex:Pt_is_q-partition_iff} by means of Corollary~\ref{cor:not_q-partition_iff_md_i^(k)=emptyset_for_md_i_in_dnx_k}. The first one, $\Pt_1$ (see Eq.~\eqref{eq:ex:Pt_is_q-partition_iff:Pt_1}), written in the form $\tuple{\setof{\md_i^{(k)}\,|\,\md_i \in \dx{}(\Pt_1)},\setof{\md_i^{(k)}\,|\,\md_i \in \dnx{}(\Pt_1)},\emptyset}$ is given by $\tuple{\setof{\emptyset,\emptyset,\emptyset},\setof{\setof{7},\setof{1,4,7},\setof{4,7}},\emptyset}$ where natural numbers $j$ in the sets again refer to the respective sentences $\tax_j$ (as in Example~\ref{ex:Pt_is_q-partition_iff}). Hence, Corollary~\ref{cor:not_q-partition_iff_md_i^(k)=emptyset_for_md_i_in_dnx_k} confirms the result we obtained in Example~\ref{ex:Pt_is_q-partition_iff}, namely that $\Pt_1$ is a canonical q-partition.

On the contrary, $\Pt_2$ (see Eq.~\eqref{eq:ex:Pt_is_q-partition_iff:Pt_2}), is not a canonical q-partition according to Corollary~\ref{cor:not_q-partition_iff_md_i^(k)=emptyset_for_md_i_in_dnx_k} since, represented in the same form as $\Pt_1$ above, $\Pt_2$ evaluates to $\tuple{\setof{\emptyset,\emptyset,\emptyset},\setof{\setof{6},\emptyset,\emptyset},\emptyset}$. Again, the result we got in Example~\ref{ex:Pt_is_q-partition_iff} is successfully verified.

To concretize Corollary~\ref{cor:upper_lower_bound_for_canonical_q-partitions}, let us apply it to our example DPI (see Table~\ref{tab:example_dpi_0}) using the set of leading diagnoses $\mD := \minD_{\tuple{\mo,\mb,\Tp,\Tn}_\RQ} = \setof{\setof{2,3},\setof{2,5},\setof{2,6},\setof{2,7},\setof{1,4,7},\setof{3,4,7}}$ listed by Table~\ref{tab:min_diagnoses_example_dpi_0}. Building all possible unions of sets in $\mD$, i.e.\ all possible $U_{\dx{}}$, such that each union is not equal to (i.e.\ a proper subset of) $U_\mD = \setof{1,\dots,7}$, yields $29$ different $U_{\dx{}}$ sets. These correspond exactly to the canonical q-partitions w.r.t.\ $\mD$, i.e.\ the canonical q-partition associated with $U_{\dx{}}$ is $\tuple{\dx{},\mD\setminus\dx{},\emptyset}$. There are no other canonical q-partitions.

Note the theoretically possible number of canonical q-partitions w.r.t.\ $\mD$ is $2^{|\mD|}-2 = 2^6 - 2 = 62$ because $\dz{} = \emptyset$ (cf.\ Proposition~\ref{prop:canonical_q-partition_has_empty_dz}) and there are $2^{|\mD|}$ partitions into two parts ($\dx{}$ and $\dnx{}$) of $\mD$. From this $2$ must be subtracted since $\dx{} \neq \emptyset$ and $\dnx{} \neq \emptyset$ must hold for any (canonical) q-partition (cf.\ Proposition~\ref{prop:properties_of_q-partitions},(\ref{prop:properties_of_q-partitions:enum:for_each_q-partition_dx_is_empty_and_dnx_is_empty}.)). Hence, in this example about $50\%$ of the theoretically possible canonical q-partitions or, respectively, of all q-partitions with empty $\dz{}$-set are indeed canonical. 

Since there is one and only one canonical query per canonical q-partition (cf.\ Proposition~\ref{prop:canonical_query_unique_for_seed}), we can directly infer from this result that there are $29$ canonical queries w.r.t.\ $\mD$ and $\tuple{\mo,\mb,\Tp,\Tn}_\RQ$.

Given the leading diagnoses $\mD$ considered in Example~\ref{ex:explicit_ents_query} given by Eq.~\eqref{eq:ex_explicit_ents_query:leading_diags}, we obtain a number of $6$ canonical queries and q-partitions. These correspond exactly to the queries and q-partitions given in Eq.~\eqref{eq:canQ+canQP_for_diags1,2,3} and explicated in Example~\ref{ex:disc_ax}. So, in this case $100\%$ of all possible q-partitions with empty $\dz{}$ are indeed canonical.

Concerning the leading diagnoses $\mD$ dealt with in Example~\ref{ex:canonical_queries_q-partitions} (see Eq.~\eqref{eq:ex_canonical_queries_q-partitions:leading_diags}), the application of Corollary~\ref{cor:upper_lower_bound_for_canonical_q-partitions} yields the $5$ different $U_{\dx{}}$ sets $\setof{\setof{2,3},\setof{2,3,4,7},\setof{3,4,7},\setof{1,3,4,7},\setof{1,4,7}}$ that are different from $U_\mD = \setof{1,2,3,4,7}$. These correspond to the following $\dx{}$-sets (given in the same order):  
\[
\setof{\setof{\md_1},\setof{\md_1,\md_6},\setof{\md_6},\setof{\md_5,\md_6},\setof{\md_5}}
\]
and therefore to the q-partitions (given in the same order):
\begin{align*}
\{\tuple{\setof{\md_1},\setof{\md_5,\md_6},\emptyset}, \\
\tuple{\setof{\md_1,\md_6},\setof{\md_5},\emptyset}, \\
\tuple{\setof{\md_6},\setof{\md_1,\md_5},\emptyset}, \\
\tuple{\setof{\md_5,\md_6},\setof{\md_1},\emptyset}, \\
\tuple{\setof{\md_5},\setof{\md_1,\md_6},\emptyset}
\}
\end{align*}
Thus, we have $5$ canonical queries and q-partitions, as we already showed in Example~\ref{ex:canonical_queries_q-partitions} (cf.\ Table~\ref{tab:Qcan+canQPart_for_example_dpi_0}). Here, the actual number of canonical q-partitions amounts to more than $80\%$ of all possible ($2^3-2 = 6$) q-partitions with empty $\dz{}$.
\qed
\end{example}

Exploiting the results obtained by Proposition~\ref{prop:suff+nec_criteria_when_partition_is_q-partition}, the following proposition provides the tools for defining and computing the successor function $S_{\mathsf{next}}$, more concretely a minimal $\dx{}$-transformation from a canonical q-partition to a canonical q-partition, for $\dx{}$-partitioning search. It establishes criteria when such a minimal $\dx{}$-transformation is given, i.e.\ using these criteria to compute successor q-partitions yields a \emph{sound} successor function. And it shows that all canonical q-partitions obtainable by a minimal $\dx{}$-transformation can be computed using these criteria, i.e.\ the successor function based on these criteria is \emph{complete}. Furthermore, it indicates circumstances under which the successor function can be strongly simplified, meaning that all possible successors of a partition are proven to be canonical q-partitions. The latter result can be used to accelerate the search.
\begin{proposition}\label{prop:minimal_transformation_for_D+_partitioning}
Let $\mD\subseteq\minD_{\langle\mo,\mb,\Tp,\Tn\rangle_\RQ}$ and $\Pt_k = \langle \dx{k}, \dnx{k}, \emptyset\rangle$ be a canonical q-partition of $\mD$. 
Then 
\begin{enumerate}
\item (\emph{soundness}) $\Pt_k \mapsto \Pt_s$ for a partition $\Pt_s := \langle \dx{s}, \dnx{s}, \emptyset\rangle$ of $\mD$ with $\dx{s} \supseteq \dx{k}$ is a minimal $\dx{}$-transformation if
\begin{enumerate}
\item \label{prop:minimal_transformation_for_D+_partitioning:enum:dx_y} $\dx{y} := \dx{k} \cup \setof{\md}$ such that $U_{\dx{y}} \subset U_\mD$ for some $\md \in \dnx{k}$ and $U_{\dx{y}}$ is subset-minimal among all $\md\in\dnx{k}$, and
\item \label{prop:minimal_transformation_for_D+_partitioning:enum:dx_s} $\dx{s} := \{\md_i\,|\,\md_i \in \mD,\md_i^{(y)} = \emptyset\}$ and 
\item \label{prop:minimal_transformation_for_D+_partitioning:enum:dnx_s} $\dnx{s} := \{\md_i\,|\,\md_i \in \mD,\md_i^{(y)} \neq \emptyset\}$. 
\end{enumerate} 
\item (\emph{completeness}) the construction of $\Pt_s$ as per (\ref{prop:minimal_transformation_for_D+_partitioning:enum:dx_y}), (\ref{prop:minimal_transformation_for_D+_partitioning:enum:dx_s}) and (\ref{prop:minimal_transformation_for_D+_partitioning:enum:dnx_s}) yields all possible minimal $\dx{}$-transformations $\Pt_k \mapsto \Pt_s$.
\item (\emph{acceleration}) if all elements in $\setof{\md_i^{(k)}\,|\,\md_i^{(k)} \in \dnx{k}}$ are pairwise disjoint, then all possible partitions $\langle \dx{r}, \dnx{r}, \emptyset\rangle$ with $\mD \supset \dx{r} \supseteq \dx{k}$ are canonical q-partitions.
\end{enumerate}
\end{proposition}
\begin{proof}
%

Ad 1.: 
By the definition of a minimal $\dx{}$-transformation (see Definition~\ref{def:minimal_transformation}), we have to show that (i)~$\Pt_s$ is a canonical q-partition where $\dx{s} \supset \dx{k}$ and that (ii)~there is no canonical q-partition $\langle \dx{l},\dnx{l},\emptyset\rangle$ such that $\dx{k} \subset \dx{l} \subset \dx{s}$. 

Ad (i):
To verify that $\Pt_s$ is indeed a canonical q-partition, we check whether it satisfies the premises, $\dx{s} \neq \emptyset$ and $\dnx{s} \neq \emptyset$, and both conditions of Proposition~\ref{prop:suff+nec_criteria_when_partition_is_q-partition}.

Proposition~\ref{prop:suff+nec_criteria_when_partition_is_q-partition},(1.), i.e. $U_{\dx{s}} \subset U_{\mD}$, is met due to the following argumentation. 
First, the inclusion of only diagnoses $\md_i$ with $\md_i^{(y)} = \emptyset$ (and thus $\md_i \subseteq U_{\dx{y}}$) in $\dx{s}$ implies $U_{\dx{s}} \not\supset U_{\dx{y}}$. Further, $\dx{s} \supseteq \dx{y}$ must hold since, trivially, for each $\md_i \in \dx{y}$ it must be true that $\md_i^{(y)} = \emptyset$ wherefore, by (\ref{prop:minimal_transformation_for_D+_partitioning:enum:dx_s}), $\md_i \in \dx{s}$. Hence, $U_{\dx{s}} \supseteq U_{\dx{y}}$ must be given. Combining these findings yields $U_{\dx{s}} = U_{\dx{y}}$. By the postulation of $U_{\dx{y}} \subset U_\mD$ in (\ref{prop:minimal_transformation_for_D+_partitioning:enum:dx_y}), we obtain $U_{\dx{s}} \subset U_\mD$. This finishes the proof of Proposition~\ref{prop:suff+nec_criteria_when_partition_is_q-partition},(1.).

Due to $U_{\dx{s}} \subset U_\mD$, we must have that $\dx{s} \subset \mD$ which implies that $\dnx{s} = \mD \setminus \dx{s} \neq \emptyset$. Moreover, we have seen above that $\dx{s} \supseteq \dx{y}$. By definition of $\dx{y}$ it holds that $\dx{y} \supset \dx{k}$ which is why $\dx{s} \neq \emptyset$. Thence, both premises of Proposition~\ref{prop:suff+nec_criteria_when_partition_is_q-partition} are given.

Proposition~\ref{prop:suff+nec_criteria_when_partition_is_q-partition},(2.), i.e. that there is no $\md_i \in \dnx{s}$ such that $\md_i \subseteq U_{\dx{s}}$, is shown next. Each $\md_i \in \mD$ with $\md_i \subseteq U_{\dx{s}}$ fulfills $\md_i^{(y)} = \emptyset$ by $U_{\dx{s}} = U_{\dx{y}}$ which we derived above. Thus, by the definition of $\dx{s}$ and $\dnx{s}$ in (\ref{prop:minimal_transformation_for_D+_partitioning:enum:dx_s}) and (\ref{prop:minimal_transformation_for_D+_partitioning:enum:dnx_s}), respectively, each $\md_i \in \mD$ with $\md_i \subseteq U_{\dx{s}}$ must be an element of $\dx{s}$ and cannot by an element of $\dnx{s}$. 
This finishes the proof of Proposition~\ref{prop:suff+nec_criteria_when_partition_is_q-partition},(2.).
Hence, $\Pt_s$ is a canonical q-partition.

Moreover, since $\dx{y} := \dx{k} \cup \setof{\md}$ for some diagnosis, we obtain that $\dx{y} \supset \dx{k}$. But, before we argued that $\dx{s} \supseteq \dx{y}$. All in all, this yields $\dx{s} \supset \dx{k}$.
This finishes the proof of (i).

Ad (ii): 
To show the minimality of the transformation $\Pt_k \mapsto \Pt_s$, let us assume that there is some canonical q-partition $\Pt_l := \langle \dx{l},\dnx{l}, \emptyset\rangle$ with $\dx{k} \subset \dx{l} \subset \dx{s}$. 
From this, we immediately obtain that $U_{\dx{l}} \subseteq U_{\dx{s}}$ must hold. Furthermore, we have shown above that $U_{\dx{s}} = U_{\dx{y}}$. Due to the fact that $\Pt_s$ is already uniquely defined as per (\ref{prop:minimal_transformation_for_D+_partitioning:enum:dx_s}) and (\ref{prop:minimal_transformation_for_D+_partitioning:enum:dnx_s}) given $U_{\dx{y}} = U_{\dx{s}}$ and since $\dx{l} \neq \dx{s}$, we conclude that $U_{\dx{l}} \subset U_{\dx{s}}$. Thence, $U_{\dx{l}} \subset U_{\dx{y}}$. Additionally, by (\ref{prop:minimal_transformation_for_D+_partitioning:enum:dx_y}), for all diagnoses $\md\in\dnx{k}$ it must hold that $U_{\dx{k} \cup \setof{\md}} \not\subset U_{\dx{y}}$. However, as $\dx{k} \subset \dx{l}$, there must be at least one diagnosis among those in $\dnx{k}$ which is an element of $\dx{l}$. If there is exactly one such diagnosis $\md^*$, then we obtain a contradiction immediately as $U_{\dx{l}} = U_{\dx{k} \cup \setof{\md^*}} \not\subset U_{\dx{y}}$. Otherwise, we observe that, if there is a set $\mD'\subseteq\dnx{k}$ of multiple such diagnoses, then there is a single diagnosis $\md' \in \mD' \subseteq \dnx{k}$ such that $U_{\dx{l}} = U_{\dx{k} \cup \mD'} \supseteq U_{\dx{k} \cup \setof{\md'}}$ wherefore we can infer that $U_{\dx{l}} \not\subset U_{\dx{y}}$ must hold. Consequently, the transformation $\Pt_k \mapsto \Pt_s$ is indeed minimal and (ii) is proven.
Ad 2.: 
Assume that $\Pt_k \mapsto \Pt_s$ is a minimal $\dx{}$-transformation and that $\Pt_s$ cannot be constructed as per (\ref{prop:minimal_transformation_for_D+_partitioning:enum:dx_y}), (\ref{prop:minimal_transformation_for_D+_partitioning:enum:dx_s}) and (\ref{prop:minimal_transformation_for_D+_partitioning:enum:dnx_s}). 

By Def.~\ref{def:minimal_transformation}, $\Pt_s$ is a canonical q-partition. Since it is a q-partition, we have that $\dx{s} \neq \emptyset$ and $\dnx{s} \neq \emptyset$. Thence, by Proposition~\ref{prop:suff+nec_criteria_when_partition_is_q-partition}, $U_{\dx{s}} \subset U_\mD$ which is why, by Corollary~\ref{cor:not_q-partition_iff_md_i^(k)=emptyset_for_md_i_in_dnx_k}, there must be some $\dx{y}$ (e.g.\ $\dx{s}$) such that $\dx{s} := \{\md_i\,|\,\md_i \in \mD,\md_i^{(y)} = \emptyset\}$ and $\dnx{s} := \{\md_i\,|\,\md_i \in \mD,\md_i^{(y)} \neq \emptyset\}$. Thence, for each minimal $\dx{}$-transformation there is some $\dx{y}$ such that (\ref{prop:minimal_transformation_for_D+_partitioning:enum:dx_s}) $\land$ (\ref{prop:minimal_transformation_for_D+_partitioning:enum:dnx_s}) is true wherefore we obtain that $\lnot$(\ref{prop:minimal_transformation_for_D+_partitioning:enum:dx_y}) must be given. That is, at least one of the following must be false: (i)~there is some $\md \in \dnx{k}$ such that $\dx{y} = \dx{k} \cup \setof{\md}$, (ii)~$U_{\dx{y}}\subset U_\mD$, (iii)~$U_{\dx{y}}$ is $\subseteq$-minimal among all $\md \in \dnx{k}$.

First, we can argue analogously as done in the proof of (1.)\ above that $U_{\dx{y}} = U_{\dx{s}}$ must hold. This along with Proposition~\ref{prop:suff+nec_criteria_when_partition_is_q-partition} entails that $U_{\dx{y}} \subset U_\mD$ cannot be false. So, (ii) cannot be false.

Second, assume that (i) is false. 
That is, no set $\dx{y}$ usable to construct $\dx{s}$ and $\dnx{s}$ as per (\ref{prop:minimal_transformation_for_D+_partitioning:enum:dx_s}) and (\ref{prop:minimal_transformation_for_D+_partitioning:enum:dnx_s}) is defined as $\dx{y} = \dx{k} \cup \setof{\md}$ for any $\md \in \dnx{k}$. But, clearly, $\dx{s}$ is one such set $\dx{y}$. And, $\dx{s} \supset \dx{k}$ as a consequence of $\Pt_k \mapsto \Pt_s$ being a minimal $\dx{}$-transformation. Hence, there is some set $\dx{y} \supset \dx{k}$ usable to construct $\dx{s}$ and $\dnx{s}$ as per (\ref{prop:minimal_transformation_for_D+_partitioning:enum:dx_s}) and (\ref{prop:minimal_transformation_for_D+_partitioning:enum:dnx_s}). Now, if $\dx{y} = \dx{k} \cup \setof{\md}$ for some diagnosis $\md \in \dnx{k}$, then we have a contradiction to $\lnot$(i). Thus, $\dx{y} = \dx{k} \cup \mathbf{S}$ where $\mathbf{S} \subseteq \dnx{k}$ with $|\mathbf{S}| \geq 2$ must hold. In this case, there is some $\md \in \dnx{y}$ such that $\dx{y} \supset \dx{k} \cup \setof{\md}$ and therefore $U_{\dx{y}} \supseteq U_{\dx{k}\cup\setof{\md}}$.  
Let $\Pt_{s'}$ be the partition induced by $\dx{y'} := \dx{k}\cup\setof{\md}$ as per (\ref{prop:minimal_transformation_for_D+_partitioning:enum:dx_s}) and (\ref{prop:minimal_transformation_for_D+_partitioning:enum:dnx_s}). 

This partition $\Pt_{s'}$ is a canonical q-partition due to Corollary~\ref{cor:not_q-partition_iff_md_i^(k)=emptyset_for_md_i_in_dnx_k}. The latter is applicable in this case, first, by reason of $\dnx{s}\neq \emptyset$ (which means that $\mD \supset \dx{s}$) and $\dx{s} \supseteq \dx{s'} \supseteq \dx{y'} \supset \dx{k} \supset \emptyset$ (where the first two superset-relations hold due to (\ref{prop:minimal_transformation_for_D+_partitioning:enum:dx_s}), (\ref{prop:minimal_transformation_for_D+_partitioning:enum:dnx_s}) and Eq.~\eqref{eq:md_i^(k)}, and the last one since $\Pt_k$ is a q-partition by assumption), which lets us derive $\dx{s'} \neq \emptyset$ and $\dnx{s'} \neq \emptyset$. Second, from the said superset-relations and Proposition~\ref{prop:suff+nec_criteria_when_partition_is_q-partition} along with $\Pt_s$ being a canonical q-partition, we get $U_\mD \supset U_{\dx{s}} \supseteq U_{\dx{s'}}$.

But, due to $\dx{s'} \subseteq \dx{s}$ we can conclude that either $\Pt_k \mapsto \Pt_s$ is not a minimal $\dx{}$-transformation (case $\dx{s'} \subset \dx{s}$) or $\Pt_s$ can be constructed by means of $\dx{k}\cup\setof{\md}$ for some $\md \in \dnx{k}$ (case $\dx{s'} = \dx{s}$). The former case is a contradiction to the assumption that $\Pt_k \mapsto \Pt_s$ is a minimal $\dx{}$-transformation. The latter case is a contradiction to 
$|\mathbf{S}| \geq 2$. Consequently, (i) cannot be false.

Third, as (i) and (ii) have been shown to be true, we conclude that (iii) must be false.
Due to the truth of (i), we can assume that $\dx{y}$ used to construct $\Pt_s$ can be written as $\dx{k}\cup\setof{\md}$ for some $\md \in \dnx{k}$. Now, if $U_{\dx{y}}$ is not $\subseteq$-minimal among all $\md \in \dnx{k}$, then there is some $\md' \in \dnx{k}$ such that $U_{\dx{k}\cup\setof{\md'}} \subset U_{\dx{y}}$. 
Further, due to $U_{\dx{y}} = U_{\dx{s}} \subset U_\mD$ (because of Proposition~\ref{prop:suff+nec_criteria_when_partition_is_q-partition} and the fact that $\Pt_s$ is a canonical q-partition), we get $U_{\dx{k}\cup\setof{\md'}} \subset U_\mD$.

Let $\Pt_{s'}$ be the partition induced by $\dx{y'} := \dx{k}\cup\setof{\md'}$ as per (\ref{prop:minimal_transformation_for_D+_partitioning:enum:dx_s}) and (\ref{prop:minimal_transformation_for_D+_partitioning:enum:dnx_s}). It is guaranteed that $\Pt_{s'}$ is a canonical q-partition due to Corollary~\ref{cor:not_q-partition_iff_md_i^(k)=emptyset_for_md_i_in_dnx_k}. The first reason why the latter is applicable here is $U_{\dx{k}\cup\setof{\md'}} \subset U_\mD$. The second one is $\md \notin \dx{s'}$ which implies $\dx{s'} \neq \mD$ and thus $\dnx{s'} \neq \emptyset$, and $\md' \in \dx{s'}$ (due to (\ref{prop:minimal_transformation_for_D+_partitioning:enum:dx_s})) which means that $\dx{s'} \neq \emptyset$. The fact $\md \notin \dx{s'}$ must hold due to $\md \not\subseteq U_{\dx{k}\cup\setof{\md'}}$. To realize that the latter holds, assume the opposite, i.e.\ $\md \subseteq U_{\dx{k}\cup\setof{\md'}}$. Then, since $U_{\dx{k}} \subseteq U_{\dx{k}\cup\setof{\md'}}$ and $U_{\dx{k}\cup\setof{\md}} = U_{\dx{k}} \cup \md$, we obtain that $U_{\dx{k}\cup\setof{\md'}} \supseteq U_{\dx{k}\cup\setof{\md}} = U_{\dx{y}}$, a contradiction. 
So, both $\Pt_{s'}$ and $\Pt_{s}$ are canonical q-partitions. However, since $\Pt_{s'}$ is constructed as per (\ref{prop:minimal_transformation_for_D+_partitioning:enum:dx_s}),(\ref{prop:minimal_transformation_for_D+_partitioning:enum:dnx_s}) by means of $U_{\dx{k}\cup\setof{\md'}}$ and $\Pt_{s}$ as per (\ref{prop:minimal_transformation_for_D+_partitioning:enum:dx_s}),(\ref{prop:minimal_transformation_for_D+_partitioning:enum:dnx_s}) by means of $U_{\dx{y}}$ and since $U_{\dx{k}\cup\setof{\md'}} \subset U_{\dx{y}}$, it must hold that $\dx{s} \supseteq \dx{s'}$. In addition, we observe that $\md \in \dx{s}$ (due to $\md \subseteq U_{\dx{k}\cup\setof{\md}} = U_{\dx{k}} \cup \md$), but $\md \notin \dx{s'}$ (as shown above). Thence, $\md \in \dx{s}\setminus\dx{s'}$ which is why $\dx{s} \supset \dx{s'}$. This, however, constitutes a contradiction to the assumption that $\Pt_k \mapsto \Pt_s$ is a minimal $\dx{}$-transformation. Consequently, (iii) must be true. 

Altogether, we have shown that neither (i) nor (ii) nor (iii) can be false. The conclusion is that (\ref{prop:minimal_transformation_for_D+_partitioning:enum:dx_y}) and (\ref{prop:minimal_transformation_for_D+_partitioning:enum:dx_s}) and (\ref{prop:minimal_transformation_for_D+_partitioning:enum:dnx_s}) must hold for $\Pt_k \mapsto \Pt_s$, a contradiction.

Ad 3.:
By Proposition~\ref{prop:suff+nec_criteria_when_partition_is_q-partition}, we must demonstrate that (i)~$U_{\dx{r}} \subset U_{\mD}$ and (ii)~there is no $\md \in \dnx{r}$ such that $\md \subseteq U_{\dx{r}}$ for any $\langle \dx{r}, \dnx{r}, \emptyset\rangle$ where $\mD \supset \dx{r} \supseteq \dx{k}$ and $\setof{\md_i^{(k)}\,|\,\md_i^{(k)} \in \dnx{k}}$ are pairwise disjoint. 

Ad (i):
Let us first assume that $U_{\dx{r}} = U_{\mD}$. The unions of diagnoses of both sets $\dx{k}$ and $\dnx{k}$ must comprise all axioms occurring in some diagnosis in $\mD$, i.e. $U_{\mD} = U_{\dx{k}}\cup U_{\dnx{k}}$. However, since by definition each $\md_i^{(k)}$ is exactly the subset of $\md_i \in \dnx{k}$ that is disjoint from $U_{\dx{k}}$, we obtain that $U_{\mD} = U_{\dx{k}}\cup U_{\setof{\md_i^{(k)}\,|\,\md_i\in\dnx{k}}}$. Since $\dx{k} \cup \dnx{k} = \mD \supset \dx{r}$ and $\dx{r} \supseteq \dx{k}$, there must be some $S'\subset \dnx{k}$ such that $\dx{r} = \dx{k} \cup S'$. By $U_{\dx{r}} = U_{\mD}$ we can conclude that $U_\mD = U_{\dx{r}} = U_{\dx{k}} \cup U_{\setof{\md_i^{(k)}\,|\,\md_i\in S'}}$.
Further on, it follows from $\dx{r} \subset \mD$ and the definition of $S'$ that $\mD \setminus \dx{r} = \mD \setminus (\dx{k} \cup S') = (\mD \setminus \dx{k}) \setminus S' = \dnx{k} \setminus S' \supset \emptyset$. So, there is at least one diagnosis $\md_m\in \dnx{k}$ such that 
\begin{enumerate}[label=(\Roman*)]
	\item \label{prop:minimal_transformation_for_D+_partitioning:PROOF:enum:1} $\md_m \notin S'$,
	\item \label{prop:minimal_transformation_for_D+_partitioning:PROOF:enum:2} $\md_m^{(k)} \cap U_{\dx{k}} = \emptyset$, and 
	\item \label{prop:minimal_transformation_for_D+_partitioning:PROOF:enum:3} $\md_m^{(k)} \subseteq \md_m \subseteq U_\mD$
\end{enumerate}
whereas the latter two facts hold due to the definition of $\md_m^{(k)}$, see Eq.~(\eqref{eq:md_i^(k)}).
As $\md_m \in \dnx{k}$, by Corollary~\ref{cor:not_q-partition_iff_md_i^(k)=emptyset_for_md_i_in_dnx_k}, $\md_m^{(k)} \neq \emptyset$ must hold. Moreover, \ref{prop:minimal_transformation_for_D+_partitioning:PROOF:enum:2} and \ref{prop:minimal_transformation_for_D+_partitioning:PROOF:enum:3}\ yield that $\md_m^{(k)} \subseteq U_{\setof{\md_i^{(k)}\,|\,\md_i\in S'}}$ which is, by \ref{prop:minimal_transformation_for_D+_partitioning:PROOF:enum:1}, obviously a contradiction to the pairwise disjointness of elements in $\setof{\md_i^{(k)}\,|\,\md_i \in \dnx{k}}$.

Ad (ii):
According to the argumentation of (i), we already know that $U_{\dx{r}} = U_{\setof{\md_i^{(k)}\,|\,\md_i\in S'}} \cup U_{\dx{k}}$ for some $S' \subset \dnx{k}$
and that there is a diagnosis $\md \in \dnx{r} = \dnx{k} \setminus S'$.
Let $\md_m$ be an arbitrary diagnosis in $\dnx{r}$. We observe that $\md_m \notin S'$.
Let us assume that $\md_m \subseteq U_{\dx{r}}$ holds. Then,
$\md_m^{(k)} \subseteq \md_m \subseteq U_{\dx{r}}$. Moreover, since $\md_m^{(k)} \cap U_{\dx{k}} = \emptyset$ by the definition of $\md_m^{(k)}$, see Eq.~(\eqref{eq:md_i^(k)}), $\md_m^{(k)} \subseteq U_{\setof{\md_i^{(k)}\,|\,\md_i\in S'}}$ must be valid. Now, in an analogue way as in (i) above, we obtain a contradiction to the pairwise disjointness of elements in $\setof{\md_i^{(k)}\,|\,\md_i \in \dnx{k}}$. 
%
\end{proof}
Proposition~\ref{prop:minimal_transformation_for_D+_partitioning} shows that the set of minimal diagnoses $\dx{y} := \dx{k} \cup \setof{\md}$ with $\md\in\dnx{k}$ used to construct the canonical q-partition $\Pt_s$ that results from the q-partition $\Pt_k$ by means of a minimal $\dx{}$-transformation is not necessarily equal to the set $\dx{s}$ of $\Pt_s$. In fact, it might be the case that further minimal diagnoses (in addition to $\md$) must be transferred from $\dnx{k}$ to $\dx{s}$ in order to make $\Pt_s$ a canonical q-partition. We call these further diagnoses the \emph{necessary followers} of $\md$ w.r.t.\ $\Pt_k$, formally:
\begin{definition}\label{def:necessary_follower}
Let $\mD\subseteq\minD_{\langle\mo,\mb,\Tp,\Tn\rangle_\RQ}$ and $\Pt_k = \langle \dx{k}, \dnx{k}, \emptyset\rangle$ be a canonical q-partition of $\mD$. Then, we call $\md'\in\dnx{k}$ a \emph{necessary $\dx{}$-follower of $\md\in\dnx{k}$ w.r.t.\ $\Pt_k$}, $\mathsf{NF}^+_{\Pt_k}(\md',\md)$ for short, iff 
for any canonical q-partition $\langle \mathbf{S}', \mD \setminus \mathbf{S}', \emptyset \rangle$ with $\mathbf{S}' \supseteq \dx{k} \cup \setof{\md}$ it holds that $\md' \in \mathbf{S}'$.
\end{definition}
\begin{example}\label{ex:nec_follower} 
We continue discussing our running example DPI $\tuple{\mo,\mb,\Tp,\Tn}_\RQ$ (see Table~\ref{tab:example_dpi_0}). Assume as in Example~\ref{ex:md_i^(k)} that $\mD = \minD_{\tuple{\mo,\mb,\Tp,\Tn}_\RQ}$. Let us consider the canonical q-partition $\Pt_k := \tuple{\setof{\md_1,\md_2},\setof{\md_3,\md_4,\md_5,\md_6},\emptyset}$. Written in the form 
\begin{align}\label{eq:standard_representation_of_can_q-partitions}
\tuple{U_{\dx{k}},\setof{\md_i^{(k)}\,|\,\md_i \in \dnx{k}}} \qquad \text{(standard representation of canonical q-partitions)}
\end{align}
$\Pt_k$ is given by 
\begin{align} \label{eq:ex_necessary_follower:U_D+_and_traits}
\tuple{\setof{2,3,5},\setof{\setof{6},\setof{7},\setof{1,4,7},\setof{4,7}}}
\end{align}
Note by the definition of $\md_i^{(k)}$ given by Eq.~\eqref{eq:md_i^(k)} the sets that are elements of the right-hand set of this tuple are exactly the diagnoses in $\dnx{k}$ reduced by the elements that occur in the left-hand set of this tuple. For instance, $\setof{6}$ results from $\md_3 \setminus U_{\dx{k}} = \setof{2,6} \setminus \setof{2,3,5}$. We now analyze $\Pt_k$ w.r.t.\ necessary followers of the diagnoses in $\dnx{k}$. The set of necessary followers of $\md_3$ is empty. The same holds for $\md_4$. However, $\md_5$ has two necessary followers, namely $\setof{\md_4,\md_6}$, whereas $\md_6$ has one, given by $\md_4$. The intuition is that transferring $\md_3$ with  $\md_3^{(k)} = \setof{6}$ to $\dx{k}$ yields the new set $\dx{k^*} := \setof{\md_1,\md_2,\md_3}$ with $U_{\dx{k^*}} = \setof{2,3,5,6}$. This new set however causes no set $\md_i^{(k^*)}$ for $\md_i$ in $\dnx{k^*}$ to become the empty set. Hence the transfer of $\md_3$ necessitates no relocations of any other elements in $\dnx{k^*}$. 

For $\md_6$, the situation is different. Here the new set $U_{\dx{k^*}} = \setof{2,3,4,5,7}$ which implicates \[\setof{\md_3^{(k^*)},\md_4^{(k^*)},\md_5^{(k^*)}} = \setof{\setof{6},\emptyset,\setof{1}}\] 
for $\md_i$ in $\dnx{k^*}$.
Application of Corollary~\ref{cor:not_q-partition_iff_md_i^(k)=emptyset_for_md_i_in_dnx_k} now yields that $\Pt_{k^*}$ is not a canonical q-partition due to the empty set $\md_4^{(k^*)}$. To make a canonical q-partition out of $\Pt_{k^*}$ requires the transfer of all diagnoses $\md_i$ associated with empty sets $\md_i^{(k^*)}$ to $\dx{k^*}$.

Importantly, notice that the canonical q-partition $\Pt_{s^*} := \tuple{\setof{\md_1,\md_2,\md_4,\md_6},\setof{\md_3,\md_5},\emptyset}$ resulting from this cannot be reached from $\Pt_k$ by means of a minimal $\dx{}$-transformation. The reason is that $\Pt_{s} := \tuple{\setof{\md_1,\md_2,\md_4},\setof{\md_3,\md_5,\md_6},\emptyset}$ is a canonical q-partition as well and results from $\Pt_k$ by fewer changes to $\dx{k}$ than $\Pt_{s^*}$. In fact, only q-partitions resulting from the transfer of diagnoses $\md_i \in \dnx{k}$ with $\subseteq$-minimal $\md_i^{(k)}$ to $\dx{k}$ are reachable from $\Pt_k$ by a minimal $\dx{}$-transformation. As becomes evident from Eq.~\eqref{eq:ex_necessary_follower:U_D+_and_traits}, only the q-partitions created from $\Pt_k$ by means of a shift of $\md_3$ (with $\md_3^{(k)} = \setof{6}$) or $\md_4$ ($\setof{7}$) to $\dx{k}$ are successors of $\Pt_k$ compliant with the definition of a minimal $\dx{}$-transformation.   \qed
\end{example}
The idea is now to define a relation between two elements of $\dnx{k}$ iff both of them lead to the same canonical q-partition if added to $\dx{k}$:
\begin{definition}\label{def:equiv_rel}
Let $\mD\subseteq\minD_{\langle\mo,\mb,\Tp,\Tn\rangle_\RQ}$ and $\Pt_k = \langle \dx{k}, \dnx{k}, \emptyset\rangle$ be a canonical q-partition of $\mD$. Then, we denote by $\sim_k$ the binary relation over $\dnx{k}$ defined as $\md_i \sim_k \md_j$ iff $\md_i^{(k)} = \md_j^{(k)}$. 
\end{definition}
The following proposition is obvious:
\begin{proposition}\label{def:nec_follower_equiv_relation}
$\sim_k$ is an equivalence relation (over $\dnx{k}$). 
\end{proposition}
The equivalence classes w.r.t.\ $\sim_k$ are denoted by $[\md_i]^{\sim_k}$ where $\md_i \in \dnx{k}$.
\begin{definition}\label{def:trait}
Let $\mD\subseteq\minD_{\langle\mo,\mb,\Tp,\Tn\rangle_\RQ}$ and $\Pt_k = \langle \dx{k}, \dnx{k}, \emptyset\rangle$ be a canonical q-partition of $\mD$. We call $\md_i^{(k)}$ (see Eq.~\eqref{eq:md_i^(k)}) \emph{the trait of the equivalence class} $[\md_i]^{\sim_k}$ w.r.t.\ $\sim_k$ including $\md_i$. Sometimes we will call $\md_i^{(k)}$ for simplicity \emph{the trait of $\md_i$}.
\end{definition}
\begin{example}\label{ex:equiv_rel+traits}
Consider the canonical q-partition $\Pt_k = \setof{\setof{\md_4,\md_5},\setof{\md_1,\md_2,\md_3,\md_6},\emptyset}$ related to our running example DPI (Table~\ref{tab:example_dpi_0}). Using the standard representation of canonical q-partitions (introduced by Eq.~\eqref{eq:standard_representation_of_can_q-partitions}) this q-partition amounts to $\tuple{\setof{1,2,4,7},\setof{\setof{3},\setof{5},\setof{6},\setof{3}}}$. Now, \[
\sim_k \;= \setof{\tuple{\md_1,\md_1},\tuple{\md_2,\md_2},\tuple{\md_3,\md_3},\tuple{\md_6,\md_6},\tuple{\md_1,\md_6},\tuple{\md_6,\md_1}}
\] 
and the equivalence classes w.r.t.\ $\sim_k$ are 
\[
\setof{[\md_1]^{\sim_k},[\md_2]^{\sim_k},[\md_3]^{\sim_k}} = \setof{\setof{\md_1,\md_6},\setof{\md_2},\setof{\md_3}}
\]
Note that $[\md_1]^{\sim_k} = [\md_6]^{\sim_k}$ holds.
The number of the equivalence classes gives an upper bound of the number of successors resulting from a minimal $\dx{}$-transformation from $\Pt_k$. 
The traits of these equivalence classes are given by 
\[
\setof{\setof{3},\setof{5},\setof{6}}
\]
These can be just read from the standard representation above. Since all traits are $\subseteq$-minimal, there are exactly three direct successors of $\Pt_k$ as per Definition~\ref{def:minimal_transformation}. However, in case there is only one equivalence class, e.g.\ when considering $\Pt_k = \setof{\setof{\md_2,\md_3,\md_4,\md_5},\setof{\md_1,\md_6},\emptyset}$ with the standard representation $\tuple{\setof{1,2,4,5,6,7},\setof{\setof{3},\setof{3}}}$ and the single equivalence class $[\md_1]^{\sim_k} = [\md_6]^{\sim_k} = \setof{\md_1,\md_6}$ with trait $\setof{3}$ we find that there are no canonical q-partitions reachable by a minimal $\dx{}$-transformation from $\Pt_k$. The reason behind this is that transferring either diagnosis in this equivalence class to $\dx{k}$ requires the transfer of the other, since both are necessary followers w.r.t.\ $\Pt_k$ of one another. An empty $\dnx{}$-set -- and hence no (canonical) q-partition -- would be the result.
These thoughts are formally captured by Corollary~\ref{cor:nec_followers_form_equivalence_class_wrt_md^(y)} below. \qed
\end{example}
A simple corollary derivable from Proposition~\ref{prop:minimal_transformation_for_D+_partitioning} and Definitions~\ref{def:necessary_follower} and \ref{def:trait} is the following.
It enables to characterize the successors of a canonical q-partition $\Pt_k$ resulting from a minimal $\dx{}$-transformation by means of the traits of the equivalence classes w.r.t.\ $\sim_k$. This is essentially the basis for $S_{\mathsf{next}}$:
\begin{corollary}\label{cor:nec_followers_form_equivalence_class_wrt_md^(y)}
Let $\mD\subseteq\minD_{\langle\mo,\mb,\Tp,\Tn\rangle_\RQ}$ and $\Pt_k$ be a canonical q-partition of $\mD$. Further, let $\Pt_s = \langle \dx{s}, \dnx{s}, \emptyset\rangle$ where $\dx{s}, \dnx{s}$ are as defined in Proposition~\ref{prop:minimal_transformation_for_D+_partitioning},(\ref{prop:minimal_transformation_for_D+_partitioning:enum:dx_s}) and (\ref{prop:minimal_transformation_for_D+_partitioning:enum:dnx_s}). Then, $\md$ is an element of $\dnx{k}$ such that $\dx{y} = \dx{k}\cup\setof{\md}$ is as defined in Proposition~\ref{prop:minimal_transformation_for_D+_partitioning},(\ref{prop:minimal_transformation_for_D+_partitioning:enum:dx_y}) iff either of the following is true:
\begin{enumerate}
	\item There is more than one equivalence class w.r.t.\ $\sim_k$ and $\md$ is in the same equivalence class w.r.t.\ $\sim_k$ as all $\md'$ with $\mathsf{NF}^+_{\Pt_k}(\md',\md)$.
	\item There is more than one equivalence class w.r.t.\ $\sim_k$ and $[\md]^{\sim_k}$ has a subset-minimal trait among all elements $\md\in\dnx{k}$.
\end{enumerate}
\end{corollary}
The next corollary which characterizes the successor function $S_{\mathsf{next}}$ is a direct consequence of Corollary~\ref{cor:nec_followers_form_equivalence_class_wrt_md^(y)} and Proposition~\ref{prop:minimal_transformation_for_D+_partitioning}:
\begin{corollary}\label{cor:S_next_sound+complete}
Let $\Pt_k := \tuple{\dx{k},\dnx{k},\emptyset}$ be a canonical q-partition, $\mathbf{EQ}^{\sim_k}$ be the set of all equivalence classes w.r.t.\ $\sim_k$ and
\begin{align*}
\mathbf{EQ}^{\sim_k}_{\subseteq} &:= \setof{[\md_i]^{\sim_k} \mid\, \not\exists j: \md_j^{(k)} \subset \md_i^{(k)} } 
\end{align*}
i.e.\ $\mathbf{EQ}^{\sim_k}_{\subseteq}$ includes all equivalence classes w.r.t.\ $\sim_k$ which have a $\subseteq$-minimal trait.
Then, the function 
\begin{align*}
S_{\mathsf{next}}: \tuple{\dx{k},\dnx{k},\emptyset} \mapsto \begin{cases}
\setof{\tuple{\dx{k}\cup E,\dnx{k}\setminus E,\emptyset}\,|\,E \in \mathbf{EQ}^{\sim_k}_{\subseteq} } & \text{if $|\mathbf{EQ}^{\sim_k}| \geq 2$}\\
\emptyset & \text{otherwise}
\end{cases} 
\end{align*}
is sound and complete, i.e.\ it produces from $\Pt_k$ all and only (canonical) q-partitions resulting from $\Pt_k$ by minimal $\dx{}$-transformations.
\end{corollary}

\paragraph{The Computation of Successors.} Our successor function for $\dx{}$-partitioning is specified by Algorithm~\ref{algo:dx+_suc}. The input to it is a partition $\Pt_k$. The output is a set of all and only canonical q-partitions that result from $\Pt_k$ by a minimal $\dx{}$-transformation. First, the algorithm checks in line~\ref{algoline:dx+_suc:is_initial_state} whether $\Pt_k = \tuple{\emptyset, \mD, \emptyset}$, i.e.\ whether $\Pt_k$ is the initial state or, equivalently, the root node of the search tree. 
If so (case $S_{\mathsf{init}}$), then Proposition~\ref{prop:S_init_sound+complete} is directly exploited to generate all successors of $\Pt_k$, i.e.\ one canonical q-partition $\tuple{\setof{\md},\dnx{k}\setminus\setof{\md},\emptyset}$ for each leading diagnosis $\md \in \mD$. 
Otherwise (case $S_{\mathsf{next}}$), $\Pt_k$ must be a canonical q-partition due to Corollary~\ref{cor:--di,MD-di,0--_is_canonical_q-partition_for_all_di_in_mD} and Proposition~\ref{prop:minimal_transformation_for_D+_partitioning}. That is, Corollary~\ref{cor:S_next_sound+complete} is applicable and used to extract all equivalence classes w.r.t.\ $\sim_k$ with subset-minimal traits for $\Pt_k$ (lines~\ref{algoline:dx+_suc:S_next_begin}-\ref{algoline:dx+_suc:S_next_end}) as a first step. 

At this, the traits for the diagnoses' equivalence classes w.r.t.\ $\sim_k$ are computed in lines~\ref{algoline:dx+_suc:S_next_begin}-\ref{algoline:dx+_suc:compute_traits_end} according to Definition~\ref{def:trait}. Note that we assume during the \textbf{while}-loop, without explicitly showing it in the pseudocode, that the trait $t_i$ for the equivalence class of a diagnosis $\md_i$ is extracted from the precalculated set of tuples of the form $\tuple{\md_i,t_i}$ stored in $diagsTraits$. 

Moreover, the $\textbf{for}$-loop in line~\ref{algoline:dx+_suc:for_usedDiag_in_D_used} iterates over $\mD_{\mathsf{used}}$ which includes one representative of all diagnoses equivalence classes which are known to yield an already explored, i.e.\ a non-new, q-partition if deleted from $\dnx{k}$ and added to $\dx{k}$. The test in line~\ref{algoline:dx+_suc:if_trait_of_diag_is_trait_of_already_used_diag} checks whether the trait $t_i$ of the currently analyzed diagnosis $\md_i$ is equal to some trait $t_u$ of a diagnosis $\md_u$ in $\mD_{\mathsf{used}}$. If so, then $\md_i^{(k)} = \md_u^{(k)}$ (cf.\ Definition~\ref{def:trait}) which is why $\md_i$ and $\md_u$ belong to the same equivalence class. By Corollary~\ref{cor:nec_followers_form_equivalence_class_wrt_md^(y)}, an addition of $\md_i$ to $\dx{k}$ would imply the addition of (the entire equivalence class of) $\md_u$ to $\dx{k}$ and hence not yield a new q-partition. Therefore, we neglect the equivalence class of $\md_i$ (which is equal to the equivalence class of $\md_u$) as a basis for constructing a successor q-partition of $\Pt_k$. This is accounted for by setting $diagAlreadyUsed$ to $\true$ for $\md_i$ in line~\ref{algoline:dx+_suc:diagAlreadyUsed_gets_true}. Only if this boolean flag for some diagnosis $\md_i$ is false (line~\ref{algoline:dx+_suc:if_not_diagAlreadyUsed}), the respective equivalence class of $\md_i$ is later (line~\ref{algoline:dx+_suc:update_eqClasses}) added to $eqClasses$, from which the finally output successor q-partitions are generated.

As a second step (lines~\ref{algoline:dx+_suc:generate_successors_begin}-\ref{algoline:dx+_suc:generate_successors_end}), all successors (constructed by means of minimal $\dx{}$-transformations) of $\Pt_k$ -- one for each equivalence class -- are constructed by means of the elements in these equivalence classes. That is, given that one of these equivalence classes consists of the minimal diagnoses in the set $E \subseteq \dnx{k}$, then one successor of $\Pt_k$ is given by $\tuple{\dx{k}\cup E, \dnx{k}\setminus E,\emptyset}$. These successors, stored in $sucs$, are finally returned by the algorithm. Further comments and helpful annotations can be found next to the pseudocode in Algorithm~\ref{algo:dx+_suc}.
\begin{proposition}\label{prop:D+sucs_soundness+completeness}
The function \textsc{get$\dx{}$Sucs} (Algorithm~\ref{algo:dx+_suc}), given 
\begin{itemize}
	\item a partition $\Pt = \tuple{\dx{}, \dnx{}, \dz{}}$ of $\mD$ and
	\item a set of diagnoses $\mD_{\mathsf{used}}$ which must not be in the $\dx{}$-set of any successor q-partition of $\Pt$
\end{itemize} 
as arguments, outputs the set of all (\emph{completeness}) and only (\emph{soundness}) canonical successor q-partitions of $\Pt$ resulting from minimal $\dx{}$-transformations. Furthermore, no returned q-partition includes any diagnosis of $\mD_{\mathsf{used}}$ in its $\dx{}$-set. 
\end{proposition}
\begin{proof} (Sketch)
Given that $\Pt$ is the initial state, then all and only canonical q-partitions are returned due to the fact that Algorithm~\ref{algo:dx+_suc} implements exactly Proposition~\ref{prop:S_init_sound+complete} in this case (lines~\ref{algoline:dx+_suc:S_init_begin}-\ref{algoline:dx+_suc:S_init_end}). All the returned q-partitions in this case satisfy the requirement imposed by $\mD_{\mathsf{used}}$ because if \textsc{get$\dx{}$Sucs} is called given the initial state as argument, then only with the second argument $\mD_{\mathsf{used}} = \emptyset$ (cf.\ Algorithm~\ref{algo:dx+_part}).

 If $\Pt$ on the other hand is an intermediate state, then Algorithm~\ref{algo:dx+_suc} implements Corollary~\ref{cor:S_next_sound+complete} and hence is sound and complete w.r.t.\ canonical q-partitions (lines~\ref{algoline:dx+_suc:S_next_begin}-\ref{algoline:dx+_suc:S_next_end}). At this, no successors with a diagnosis from $\mD_{\mathsf{used}}$ in their $\dx{}$-set are generated due to the check performed for all diagnoses that might be moved from $\dnx{k}$ to $\dx{k}$ in the course of the minimal $\dx{}$-transformation whether they are in the same equivalence class w.r.t.\ $\sim_k$ as any diagnosis in $\mD_{\mathsf{used}}$ (lines~\ref{algoline:dx+_suc:for_usedDiag_in_D_used}-\ref{algoline:dx+_suc:end_for_usedDiag_in_D_used}). If this holds for some diagnosis, then $\mathit{diagAlreadyUsed}$ is set to $\true$. Finally, in line~\ref{algoline:dx+_suc:update_eqClasses}, the equivalence class of a diagnosis is only added to the set $\mathit{eqClasses}$ (each of whose elements produces exactly one successor q-partition) if $\mathit{diagAlreadyUsed}$ is $\false$ for it.
\end{proof}
\begin{proposition}\label{prop:findQPartition_is_sound_and_complete}
Algorithm~\ref{algo:dx+_part} ($\dx{}$-Parititioning) implementing \textsc{findQPartition} terminates and constitutes a \emph{sound}, \emph{complete} and \emph{optimal} method for generating canonical q-partitions. Optimality means that it returns a canonical q-partition $\Pt$ w.r.t.\ $\mD$ that is (a)~optimal w.r.t.\ $m$ and $t_m$, if $isOptimal = \true$ or (b)~the best of all canonical q-partitions w.r.t.\ $\mD$ and $m$, otherwise.
\end{proposition}
\begin{proof} (Sketch) \\\\
\emph{(Termination):} Obviously, \textsc{findQPartition} must terminate since
$\textsc{$\dx{}$Partition}$ terminates. That $\textsc{$\dx{}$Partition}$ must terminate can be seen by recognizing that all functions \textsc{updateBest}, \textsc{opt}, \textsc{prune}, \textsc{get$\dx{}$Sucs} and \textsc{bestSuc} called in it terminate, that $sucs$ must be finite and that $\textsc{$\dx{}$Partition}$ considers each element of $sucs$ in the worst case exactly once (due to the \textbf{while}-loop and line~\ref{algoline:dx+_part:update_sucs}) before terminating in line~\ref{algoline:dx+_part:return--all_sucs_explored}. That $|sucs| < \infty$ holds (1)~due to Corollary~\ref{cor:upper_lower_bound_for_canonical_q-partitions} which states that the overall number of canonical q-partitions given a finite set of leading diagnoses $\mD$ must be finite and (2)~Proposition~\ref{prop:D+sucs_soundness+completeness} which states that \textsc{get$\dx{}$Sucs} is sound which is why it must always return a subset of all canonical q-partitions, namely $sucs$.\\
	
\noindent\emph{(Soundness):} To show the soundness, we demonstrate that Algorithm~\ref{algo:dx+_part} (\textsc{findQPartition}) can only output a canonical q-partition. To see this, note that \textsc{$\dx{}$Partition} is first called with arguments $\Pt = \Pt_{\mathsf{b}} = \tuple{\emptyset,\mD,\emptyset}$ (line~\ref{algoline:dx+_part:call_D+Partition_from_FindQPartition}). Hence, until the first call of \textsc{get$\dx{}$Sucs} in line~\ref{algoline:dx+_part:getD+Sucs}, the algorithm cannot return. The latter holds since \textsc{updateBest} can only return either $\Pt$ or $\Pt_{\mathsf{b}}$, which are however both equal to $\tuple{\emptyset,\mD,\emptyset}$, i.e.\ $\Pt_{\mathsf{best}} = \tuple{\emptyset,\mD,\emptyset}$ holds in line~\ref{algoline:dx+_part:opt} where \textsc{opt} is called. This implies that \textsc{opt} returns $\false$ due to line~\ref{algoline:dx+_opt:if_dx_or_dnx_emptyset} in Algorithm~\ref{algo:dx+_opt}. Similarly, \textsc{prune} in line~\ref{algoline:dx+_part:prune} returns $\false$ due to line~\ref{algoline:dx+_prune:if_dx_or_dnx_emptyset} in Algorithm~\ref{algo:dx+_prune}. Therefore, \textsc{get$\dx{}$Sucs} must necessarily be called during the execution of \textsc{findQPartition}. Note since $|\mD| \geq 2$, due to the completeness of \textsc{get$\dx{}$Sucs} w.r.t.\ canonical q-partition computation (Proposition~\ref{prop:D+sucs_soundness+completeness}) and the fact that there are exactly $|\mD|$ canonical q-partition successors of the initial state (Proposition~\ref{cor:--di,MD-di,0--_is_canonical_q-partition_for_all_di_in_mD}), $sucs \neq \emptyset$ must hold when \textsc{get$\dx{}$Sucs} is called for the first time.

Each call of \textsc{$\dx{}$Partition} in line~\ref{algoline:dx+_part:D+Partition_recursive_call} the argument $\Pt$ is a canonical q-partition since 
the passed argument $\Pt'$ must be a canonical q-partition due to the soundness of \textsc{$\dx{}$sucs} (Proposition~\ref{prop:D+sucs_soundness+completeness}). In line~\ref{algoline:dx+_part:updateBest} during the execution of a \textsc{$\dx{}$Partition}-call invoked at line~\ref{algoline:dx+_part:D+Partition_recursive_call}, $\Pt_{\mathsf{best}}$ is necessarily a canonical q-partition
due to lines~\ref{algoline:dx+_update_best:if_Pbest_is_no_q-partition_start}-\ref{algoline:dx+_update_best:if_Pbest_is_no_q-partition_end} of Algorithm~\ref{algo:dx+_update_best} 
which involves setting $\Pt_{\mathsf{best}}$ to $\Pt$ (a canonical q-partition!) if $\Pt_{\mathsf{best}}$ is not a canonical q-partition.

So, if an optimal (w.r.t.\ $m$ and $t_m$) canonical q-partition is found during a \textsc{$\dx{}$Partition}-call invoked at line~\ref{algoline:dx+_part:D+Partition_recursive_call} (\textsc{opt} returns $\true$), a canonical q-partition, namely $\Pt_{\mathsf{best}}$, will be returned in line~\ref{algoline:dx+_part:return--opt_found}. If, on the other hand, pruning takes place (\textsc{prune} returns $\true$), also the canonical q-partition $\Pt_{\mathsf{best}}$ will be returned in line~\ref{algoline:dx+_part:return--pruned}. 
In other words, if $\Pt$ 
is a canonical q-partition, then $\Pt_{\mathsf{best}}$ must be so in lines~\ref{algoline:dx+_part:return--opt_found} and \ref{algoline:dx+_part:return--pruned}. 

So far we have proven the following: If the algorithm returns in line~\ref{algoline:dx+_part:return--opt_found} or \ref{algoline:dx+_part:return--pruned}, then the first element of the returned tuple is a canonical q-partition.  

Now, we examine the cases of returning in line~\ref{algoline:dx+_part:return--unwind_recursion_since_opt_found} or line~\ref{algoline:dx+_part:return--all_sucs_explored}. 
A precondition for returning at line~\ref{algoline:dx+_part:return--unwind_recursion_since_opt_found} is that the second parameter $isOpt$ in the tuple returned in line~\ref{algoline:dx+_part:D+Partition_recursive_call} is $\true$. This can however only happen if line~\ref{algoline:dx+_part:return--opt_found} or again line~\ref{algoline:dx+_part:return--unwind_recursion_since_opt_found} led to the return of \textsc{$\dx{}$Partition} in line~\ref{algoline:dx+_part:D+Partition_recursive_call}. Applying this argument recursively yields that algorithm must have executed line~\ref{algoline:dx+_part:return--opt_found} once. Otherwise one could continue this argumentation infinitely often, implying \textsc{$\dx{}$Partition} being called infinitely often. This is clearly a contradiction to the termination of \textsc{findQPartition} shown before. So, the first element of the tuple returned in line~\ref{algoline:dx+_part:return--unwind_recursion_since_opt_found} must be a canonical q-partition as no assignments to $\Pt_{\mathsf{best}}$ can be made between executing line~\ref{algoline:dx+_part:return--opt_found} and line~\ref{algoline:dx+_part:return--unwind_recursion_since_opt_found}. In fact, only lines~\ref{algoline:dx+_part:update_mD_used} and \ref{algoline:dx+_part:if_not_isOpt} are executed in between.

Finally, assume that the first element in the tuple returned in line~\ref{algoline:dx+_part:return--all_sucs_explored} is not a canonical q-partition (let the call of \textsc{$\dx{}$Partition} where this return happens be denoted by $C_0$). Then this might happen (a)~after the \textbf{while}-loop has been entered and processed or (b)~without entering the \textbf{while}-loop. 
Case (b) can only arise if $sucs = \emptyset$, i.e.\ not in the very first execution of \textsc{$\dx{}$Partition} (by the argumentation above). Hence, it must occur in an execution of \textsc{$\dx{}$Partition} which was invoked in line~\ref{algoline:dx+_part:D+Partition_recursive_call}. However, by the above reasoning, we know that $\Pt_{\mathsf{best}}$ in this case must be a canonical q-partition -- contradiction.

In case (a), either $\Pt_{\mathsf{best}}$ is changed at least once during the processing of the \textbf{while}-loop or not. In the latter case, we know -- just as in case (b) -- that $\Pt_{\mathsf{best}}$ is a canonical q-partition -- contradiction. In the former case, $\Pt_{\mathsf{best}}$ is equal to some $\Pt''$ which is the first element of the tuple returned in line~\ref{algoline:dx+_part:D+Partition_recursive_call}. Since $\Pt_{\mathsf{best}}$ is set in line~\ref{algoline:dx+_part:opt_not_found_continue}, $isOpt = \false$ must hold by line~\ref{algoline:dx+_part:if_not_isOpt}. Thence, the \textsc{$\dx{}$Partition}-call returning $\Pt''$ can only have returned at line~\ref{algoline:dx+_part:return--pruned} or \ref{algoline:dx+_part:return--all_sucs_explored}. Moreover, as $\Pt_{\mathsf{best}}$ is no canonical q-partition by assumption, $\Pt''$ cannot be one either. By the line of argument above, we know that the \textsc{$\dx{}$Partition}-call returning $\Pt''$ can neither have returned in line~\ref{algoline:dx+_part:return--opt_found} nor in line~\ref{algoline:dx+_part:return--pruned}. In total, we have that the \textsc{$\dx{}$Partition}-call (let us denote this call by $C_1$) returning $\Pt''$ must have returned in line~\ref{algoline:dx+_part:return--all_sucs_explored}. If during the execution of this call  $C_1$ the \textbf{while}-loop was not entered, then $\Pt''$ must be a canonical q-partition as argued above -- contradiction. Otherwise, again line~\ref{algoline:dx+_part:opt_not_found_continue} must have been executed at least once in $C_1$. Since the $\Pt''$ in $C_1$ must not be a canonical q-partition either, we again deduce that $C_1$ must have terminated by executing line~\ref{algoline:dx+_part:return--all_sucs_explored}.

Continuing this argumentation lets us conclude that there must be a non-empty set $sucs$ in each of these analyzed \textsc{$\dx{}$Partition}-calls $C_1,C_2, \dots$ (otherwise the \textbf{while}-loop would not be entered). Since each of the returns these calls $C_1,C_2, \dots$ execute corresponds to a backtracking step, we know that if we apply the argument $k$ times, the respective \textsc{$\dx{}$Partition}-call $C_k$ analyzes q-partitions with a $\dx{}$-set of cardinality at least $k$ greater than the call $C_0$. This holds due to the soundness of \textsc{get$\dx{}$Sucs} w.r.t.\ minimal $\dx{}$-transformations (Proposition~\ref{prop:D+sucs_soundness+completeness}) and the definition of a minimal $\dx{}$-transformation (Definition~\ref{def:minimal_transformation}) which maps a partition $\Pt$ to a set of q-partitions, each with a greater number of diagnoses in $\dx{}$ than $\Pt$. Hence, there must be one such call that analyzes q-partitions with $|\dx{}| > |\mD|$, contradiction (due to Proposition~\ref{prop:properties_of_q-partitions}).

First and last we have shown that \textsc{findQPartition} is sound because at each return statement in the algorithm a canonical q-partition must be returned.\\

\noindent\emph{(Completeness):} Assuming Algorithm~\ref{algo:dx+_part} to be forced to execute without pruning or returning before all nodes are generated, it cannot miss any canonical q-partitions. The reasons are (1)~the initial state is $\tuple{\emptyset, \mD, \emptyset}$, i.e.\ $\dx{}$ is initially empty, (2)~the definition of a minimal $\dx{}$-transformation as the shift of a $\subseteq$-minimal subset of $\dnx{}$ to $\dx{}$ such that the result is a canonical q-partition (cf.\ Definition~\ref{def:minimal_transformation}), (3)~the completeness of \textsc{get$\dx{}$Sucs} w.r.t.\ minimal $\dx{}$-transformations (cf.\ Proposition~\ref{prop:D+sucs_soundness+completeness}) and (4)~the backtracking strategy of the algorithm (i.e.\ it backs up and goes one level up to the parent node in the search tree if all successors of a node have been explored, see \textbf{while}-loop and line~\ref{algoline:dx+_part:return--all_sucs_explored}). Therefore, Algorithm~\ref{algo:dx+_part} is complete.\\

\noindent\emph{(Optimality):} If Algorithm~\ref{algo:dx+_part} returns $isOptimal = \true$, then \textsc{$\dx{}$Partition} must at some point have executed line~\ref{algoline:dx+_part:return--opt_found}. Because, to reach the only other place in the algorithm where the second returned argument is $\true$, $isOpt$ must be already $\true$, i.e.\ must be already returned by a previous call of \textsc{$\dx{}$Partition} (see line~\ref{algoline:dx+_part:D+Partition_recursive_call}). Repeating this argument yields that line~\ref{algoline:dx+_part:return--opt_found} must in fact be executed for $isOptimal = \true$ to hold. However, this implies that Algorithm~\ref{algo:dx+_opt} (function \textsc{opt}) must return positively given $\Pt_{\mathsf{best}}$, $t_m$, $p$ and $r_m$. Due to the realization of \textsc{opt}, which directly implements the optimality requirements $r_m$ for a measure $m$ proven in Section~\ref{sec:ActiveLearningInInteractiveOntologyDebugging} (see Table~\ref{tab:requirements_for_equiv_classes_of_measures_wrt_equiv_mQ}), the returned canonical q-partition, i.e.\ $\Pt_{\mathsf{best}}$, must be optimal w.r.t.\ $m$ and $t_m$. Note that $\Pt_{\mathsf{best}}$, after being returned in line~\ref{algoline:dx+_part:return--opt_found}, is never modified until Algorithm~\ref{algo:dx+_part} returns since lines~\ref{algoline:dx+_part:if_not_isOpt} and \ref{algoline:dx+_part:return--unwind_recursion_since_opt_found} in this case guarantee that the entire recursion is immediately unwinded. This proves statement~(a) of the proposition.

The assumption of statement~(b) is that $isOptimal = \false$ is returned by Algorithm~\ref{algo:dx+_part}. So, due to the discussion of statement~(a), line~\ref{algoline:dx+_part:return--opt_found} can never be executed in this case. Now, Algorithm~\ref{algo:dx+_prune} (function \textsc{prune}), whose correctness was argued for in Section~\ref{sec:search_for_q-partitions} on page~\pageref{par:algorithm_description_prune}~ff., prunes the search tree at some node only if no (direct or indirect) successor can be better w.r.t.\ $m$ or $r_m$, respectively, than the currect $\Pt_{\mathsf{best}}$. Hence, the same chain of reasoning as in the (Completeness)-part of this proof can be used to establish the correctness of statement~(b).
\end{proof}

\subsubsection{Q-Partition Search Examples}
\label{sec:search_for_q-partitions_examples}

\paragraph{Overall Description and Used Notation.} In the following we give some examples to illustrate how Algorithm~\ref{algo:dx+_part} (i.e.\ the method \textsc{findQPartition} in Algorithm~\ref{algo:query_comp}) works for our running example DPI (Table~\ref{tab:example_dpi_0}) using different query quality measures $m$ discussed in Section~\ref{sec:ActiveLearningInInteractiveOntologyDebugging}. In order to make the search steps more accessible to the reader, we exploit i.a.\ the standard representation of canonical q-partitions (see Eq.~\eqref{eq:standard_representation_of_can_q-partitions}) to represent (canonical q-)partitions in our search tree. More precisely, a node in the search tree representing the canonical q-partition $\Pt_k = \tuple{\dx{k},\dnx{k},\emptyset}$ will be denoted by a frame of the form given at the top of Figure~\ref{fig:node_representation_in_can_q-partition_search_examples}.
\begin{figure}[t]
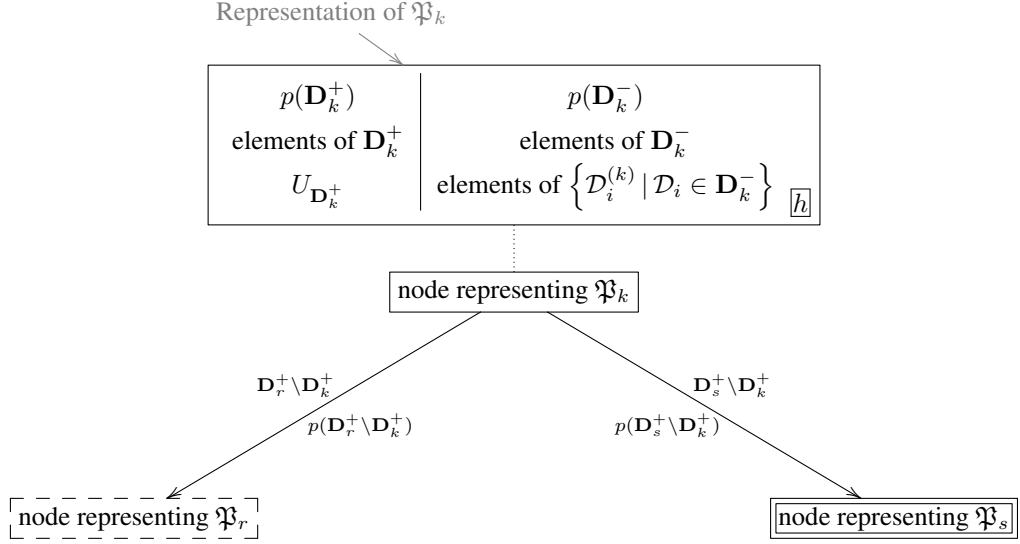

	\setlength{\fboxsep}{1.5pt} 
	\xygraph{
		!{<0cm,0cm>;<1cm,0cm>:<0cm,1.5cm>::}	
		!{(-2.4,0.96)}*+{\color{gray}\text{Representation of }\Pt_k}="Repres_Pt-k"
		!{(0,-0.2)}*+[F]{ {\begin{tabular}{c|c}
					$p(\dx{k})$ & $p(\dnx{k})$ \\
					elements of $\dx{k}$ & elements of $\dnx{k}$ \\
					$U_{\dx{k}}$ & elements of $\setof{\md_i^{(k)}\,|\,\md_i \in \dnx{k}}$
			\end{tabular}}_{\framebox[1.3\width]{$h$}} }="example-node"
		!{(0,-1.5)}*+[F]{\text{node representing } \Pt_k}="Pt-k"
		!{(-5,-3.5)}*+[F--]{\text{node representing } \Pt_r}="Pt-r"
		!{(5,-3.5)}*+[F=]{\text{node representing } \Pt_s}="Pt-s"
		"Repres_Pt-k":@{->}@[gray]"example-node"
		"example-node":@{.}"Pt-k"
		"Pt-k":^{p(\dx{r}\setminus\dx{k})}_{\dx{r}\setminus\dx{k}}"Pt-r"
		"Pt-k":_{p(\dx{s}\setminus\dx{k})}^{\dx{s}\setminus\dx{k}}"Pt-s"
}\caption{Node and edge representation in the q-partition search examples.}
\label{fig:node_representation_in_can_q-partition_search_examples}
\end{figure}
This frame includes a table with three rows where (1)~the topmost row shows the tuple $\tuple{p(\dx{k}),p(\dnx{k})}$ stating the probabilities of the $\dx{k}$ and $\dnx{k}$ sets of the q-partition which are relevant for the computation of some of the query quality measures $m$, (2)~the middle row depicts the tuple $\tuple{\dx{k},\dnx{k}}$ representing the diagnoses sets $\dx{k}$ and $\dnx{k}$ and (3)~the bottommost row constitutes the standard representation of the q-partition (cf.\ Eq.~\eqref{eq:standard_representation_of_can_q-partitions}). The framed $h$ at the bottom right corner of the large frame quotes the heuristic value computed for the canonical q-partition $\Pt_k$ by means of Algorithm~\ref{algo:dx+_best_successor} where smaller values imply (seemingly) better q-partitions w.r.t.\ the measure $m$.  Note that we state no heuristic value for the root node since this node is not a q-partition and hence does not qualify as a solution. Furthermore, no information about the $\dz{k}$-set is provided at any node as this set is always empty in canonical q-partitions.

Beneath the large frame just explained Figure~\ref{fig:node_representation_in_can_q-partition_search_examples} gives a schematic illustration of a (part of a) q-partition search tree where each node is for short just denoted by a small frame with the text ``node representing $\Pt_i$'' for $i \in \setof{k,r,s}$. A complete description would of course involve the explication of a large frame for each node as exemplified only for $\Pt_k$. The possible frame types of the node representations are \emph{dashed}, \emph{continuous line} and \emph{double frame} meaning \emph{generated (but not expanded) node}, \emph{expanded node} and \emph{node of returned q-partition}, respectively. For the given figure, this implies that the node for the q-partition $\Pt_k$ has been (generated and) expanded, the node for $\Pt_r$ has only been generated (and not expanded) and that the q-partition $\Pt_s$ is given back as the search result.

Each arrow has two labels and represents a minimal $\dx{}$-transformation, i.e.\ the arrow's destination q-partition is a result of a minimal $\dx{}$-transformation applied to its source q-partition. The label above the arrow gives the set of diagnoses moved from the $\dnx{}$-set of the (canonical q-)partition at the shaft of the arrow to the $\dx{}$-set of the canonical q-partition at the top of the arrow. The other label indicates the probability mass of the transferred diagnoses, i.e.\ the sum of the single diagnoses probabilities.

For all examples given next, we assume the diagnoses probabilities $p(\md_i)$ for $\md_i \in \mD$ given by Table~\ref{tab:ex:can_q-partition_search_diagnosis_probs}.
\begin{table}
	\centering
\begin{tabular}{lcccccc}
	diagnosis $\md_i$ & $\md_1$ & $\md_2$ & $\md_3$ & $\md_4$ & $\md_5$ & $\md_6$ \\ \hline 
	probability $p(\md_i)$ & 0.01 & 0.33 & 0.14 & 0.07 & 0.41 & 0.04 \\ 
\end{tabular} 
\caption{Diagnoses probabilities used throughout Examples~\ref{ex:canonical_q-partition_search_ENT}--\ref{ex:canonical_q-partition_search_BME}.}
\label{tab:ex:can_q-partition_search_diagnosis_probs}	
\end{table}
 
\begin{example}\label{ex:canonical_q-partition_search_ENT} (Figure~\ref{fig:ex:can_q-part_search_ENT}, $m = \mathsf{ENT}$, $t_{\mathsf{ENT}} := 0.05$) \\
Starting from the root node representing the partition $\tuple{\emptyset,\mD,\emptyset}$ the algorithm generates all possible canonical q-partitions resulting from the transfer of a single diagnosis from the $\dnx{}$-set of the initial state to its $\dx{}$-set (which is the same for all examples to come in this section). Since there are six diagnoses in $\mD$, the initial state has exactly six direct canonical q-partition successors. The theoretical justification for the fact that all so constructed partitions are canonical q-partitions is given by Proposition~\ref{prop:S_init_sound+complete}. From all these generated neighbor nodes of the initial state, the best one according to the heuristic function \textsc{heur} (cf.\ Algorithm~\ref{algo:dx+_best_successor}) is selected for expansion. In the case of $\mathsf{ENT}$, this heuristic value (see framed numbers at the bottom right corner of the nodes) indicates the deviation from the theoretically optimal value of the $\mathsf{ENT}$ measure (cf.\ Section~\ref{sec:ExistingActiveLearningMeasuresForKBDebugging}), i.e.\ the deviation of $p(\dx{})$ from $0.5$. Hence, the canonical q-partition $\Pt_1 := \tuple{\setof{\md_5},\setof{\md_1,\md_2,\md_3,\md_4,\md_6},\emptyset}$ exhibiting the best heuristic value $0.09$ among all the six open nodes is expanded. This is indicated by a continuous line frame of the node representing this q-partition.

This expanded q-partition has exactly two direct successor q-partitions that are canonical and result from it by a minimal $\dx{}$-transformation. This can be seen by considering the traits of the diagnoses in $\dnx{}(\Pt_1)$ shown in the right-hand column of the third row in the table representing $\Pt_1$. Among the five traits there are only two $\subseteq$-minimal ones, i.e.\ $\md_4^{(1)} := \setof{2}$ and $\md_6^{(1)} := \setof{3}$. All the other traits are proper supersets of either of these. This means that all canonical successors can be generated by shifting either $\md_4$ or $\md_6$ from $\dnx{}(\Pt_1)$ to $\dx{}(\Pt_1)$ yielding $\Pt_{21} := \tuple{\setof{\md_4,\md_5},\setof{\md_1,\md_2,\md_3,\md_6},\emptyset}$ and $\Pt_{22} := \tuple{\setof{\md_5,\md_6},\setof{\md_1,\md_2,\md_3,\md_4},\emptyset}$, respectively.  

At this stage, the best among these two direct successors of $\Pt_1$ is determined for expansion by means of the heuristic function. As $p(\dx{}(\Pt_{21}))$ differs by less ($0.02$) from $0.5$ than $p(\dx{}(\Pt_{22}))$ ($0.05$), $\Pt_{21}$ is chosen. However, as $t_{\mathsf{ENT}}$ has been set to $0.05$, meaning that all q-partitions featuring a $\dx{}$ probability deviation from $0.5$ by less than $0.05$ are considered ``optimal'', the optimality test (cf.\ Algorithm~\ref{algo:dx+_opt}) for $\Pt_{21}$ returns positively. Therefore, $\Pt_{21}$ is returned by \textsc{findQPartition} as the solution of the q-partition search. Overall, the count of generated and expanded nodes in this example is $9$ and $2$ (viewing the goal node not as an expanded one), respectively. There are no backtrackings or tree prunings necessary as the used heuristic function guides the search directly towards a goal state. 

Please notice that at the expansion of a node $\Pt$ at any depth of the tree, the best q-partition to be expanded next is selected only from the direct successors of $\Pt$, not among all open nodes. This shows the \emph{local search characteristic} of the algorithm. Further on, observe that the algorithm proceeds in \emph{depth-first manner}, i.e.\ continues with the expansion of a new node at depth $k$ before a second node at depth $k-1$ is expanded. Nevertheless, it is ready to \emph{backtrack} after reaching a ``dead end'', i.e.\ the situation of no successors at node expansion, along a depth-first branch. In this case, other not-yet-expanded open nodes are considered.\qed
\end{example}
\begin{example}\label{ex:canonical_q-partition_search_SPL} (Figure~\ref{fig:ex:can_q-part_search_SPL}, $m = \mathsf{SPL}$, $t_{\mathsf{SPL}} := 0$) \\
Here we have a look at the search tree produced by the $\mathsf{SPL}$ measure. Whereas the initial node and the generated direct successor q-partitions of it are identical to the $\mathsf{ENT}$ example, the heuristic value (framed quantity at the the right bottom side of the node-representing rectangles) of the q-partitions is now the difference between the intended size of the $\dx{}$-set, i.e.\ $\frac{|\mD|}{2}$, and its actual size. Since all successors of the initial state have the same value $2$, we assume the algorithm being implemented so as to queue heuristically equally good nodes in a first-in-first-out (FIFO) order based on their generation time (which is chronological according to diagnosis index). This implies the selection of the q-partition $\Pt_1 := \tuple{\setof{\md_1},\setof{\md_2,\dots,\md_6},\emptyset}$ for expansion. From this, three neighbor nodes result. Among these again all have the same heuristic value ($1$). So, the first node in the FIFO queue representing the q-partition $\Pt_{21} := \tuple{\setof{\md_1,\md_2},\setof{\md_2,\dots,\md_6},\emptyset}$	is chosen for expansion. Due to only two $\subseteq$-minimal traits $\setof{6},\setof{7}$ for equivalence classes $[\md_3]^{\sim_{21}},[\md_4]^{\sim_{21}}$, respectively, (as opposed to the non-$\subseteq$-minimal traits $\setof{1,4,7},\setof{4,7}$), there are only two successors resulting from a minimal $\dx{}$-transformation from $\Pt_{21}$. These are the ones that shift $\md_3$ or $\md_4$, respectively, to the from the $\dnx{}$ to the $\dx{}$-set of $\Pt_{21}$. Since both now have a zero heuristic value, meaning that the optimality test involving the threshold $t_m := 0$ would be positive for both, the first expanded canonical q-partition $\Pt_{31} := \tuple{\setof{\md_1,\md_2,\md_3},\setof{\md_4,\md_5,\md_6},\emptyset}$ is returned by the algorithm due to the used FIFO strategy. Here, we have $12$ generated and $3$ expanded nodes, no prunings and no backtrackings.

Observe that the $\mathsf{SPL}$ measure imposes a relatively coarse grained classification on the q-partitions, i.e.\ there are many q-partitions with the same heuristic value w.r.t.\ $\mathsf{SPL}$. Hence, the $\mathsf{SPL}$ measure is amenable to a usage in a multiple-phase optimization of queries. For instance, in a first stage, a set of q-partitions optimal w.r.t.\ $\mathsf{SPL}$ could be sought in order to find an optimal q-partition among these w.r.t.\ another measure in a second phase. The $\mathsf{RIO}$ measure roughly uses such a multi-phased optimization incorporating $\mathsf{ENT}$ as a second measure besides a measure related to $\mathsf{SPL}$ in terms of considering the number of diagnoses in $\dx{}$. Of course, other combinations of measures are thinkable.

In this example, using $\mathsf{ENT}$ to discriminate between the two equally good q-partitions w.r.t.\ $\mathsf{SPL}$ would lead to a different output. Instead of $\Pt_{31}$, $\Pt_{32} := \tuple{\setof{\md_1,\md_2,\md_4},\setof{\md_3,\md_5,\md_6},\emptyset}$ would be returned since its $\dx{}$ probability ($0.34$) is closer to $0.5$ (which is the theoretical optimum of the $\mathsf{ENT}$ measure, cf.\ Section~\ref{sec:ExistingActiveLearningMeasuresForKBDebugging}) than for $\Pt_{31}$ ($0.15$).
\qed
\end{example}
\begin{example}\label{ex:canonical_q-partition_search_RIO_c=0.4} (Figure~\ref{fig:ex:can_q-part_search_RIO_c=0.4}, $m = \mathsf{RIO}$, $t_{\mathsf{card}} := 0$, $t_{\mathsf{ent}} := 0.05$, $\uc := 0.4$) \\
The $\mathsf{RIO}$ measure chooses $\Pt_1 := \tuple{\setof{\md_3},\setof{\md_1,\md_2,\md_4,\md_5,\md_6},\emptyset}$ at the first tree level since its heuristic value $0.016$ is minimal among all six successors of the root node. This value is computed (see Algorithm~\ref{algo:dx+_best_successor}) by first using the cautiousness parameter $\uc := 0.4$ to compute the minimal number $n$ of diagnoses that must be in the $\dx{}$-set of a goal q-partition as $n := \lceil 0.4 \cdot 6\rceil = 3$. Since $\dx{}(\Pt_1)$ includes only one diagnosis, the number $numDiagsToAdd$ of diagnoses that must be still transferred from $\dnx{}(\Pt_1)$ to $\dx{}(\Pt_1)$ to achieve a $\dx{}$-set of size $n$ is then $3-1 = 2$. Further on, the average probability $avgProb$ of diagnoses in $\dnx{}(\Pt_1)$ amounts to $\frac{p(\dnx{}(\Pt_1))}{|\dnx{}(\Pt_1)|} = \frac{0.86}{5} = 0.172$. Now, the heuristic value computes the probability deviation from $0.5$ of the new $\dx{}$-set resulting from the addition of two diagnoses with probability $avgProb$ to $\dx{}(\Pt_1)$, i.e.\ $|0.14 + 2 * 0.172 - 0.5| = |0.484 - 0.5|$. Note whereas the heuristic value in the case of $m \in \setof{\mathsf{ENT},\mathsf{SPL}}$ was directly compared with the given threshold $t_m$ in the goal check (function \textsc{opt}, cf.\ Algorithms~\ref{algo:dx+_best_successor} and \ref{algo:dx+_opt}), the goal check for $\mathsf{RIO}$ works differently. Namely, a q-partition is regarded optimal iff i.a.\ $|\dx{}| \geq n$ (cf.\ Algorithm~\ref{algo:dx+_opt}) which is not valid for $\Pt_1$ as $|\dx{}(\Pt_1)| = 1 \not\geq 3$, i.e.\ the $\dx{}$-set of $\Pt_1$ has still too small a cardinality.

Hence, the algorithm continues by expanding $\Pt_1$ and at the second tree level its best successor $\Pt_{21} := \langle\{\md_3,\md_4\},\{\md_1, \md_2, \md_5, \md_6\},\emptyset\rangle$ (with a heuristic value of $0.0925$) until it finds that the best q-partition $\Pt_{31} := \tuple{\setof{\md_2,\md_3,\md_4},\setof{\md_1,\md_5,\md_6},\emptyset}$ (heuristic value $0.04$) at tree depth $3$ satisfies the optimality conditions. That is, $\dx{}(\Pt_{31})$ comprises $n = 3$ diagnoses, $|\dx{}(\Pt_{31})|$ differs from $n$ by less than $t_{\mathsf{card}} = 0$ and $p(\dx{}(\Pt_{31})) = 0.54$ differs from $0.5$ by less than $t_{\mathsf{ent}} = 0.05$. So, finally $\Pt_{31}$ is returned. The number of generated and expanded nodes, prunings and backtrackings in this example was, respectively, $13$, $3$, $0$ and $0$.
	\qed
\end{example}
\begin{example}\label{ex:canonical_q-partition_search_RIO_c=0.3} (Figure~\ref{fig:ex:can_q-part_search_RIO_c=0.3}, $m = \mathsf{RIO}$, $t_{\mathsf{card}} := 0$, $t_{\mathsf{ent}} := 0.05$, $\uc := 0.3$) \\
Let us now investigate the q-partition search using the $\mathsf{RIO}$ measure with a cautiousness parameter $\uc := 0.3$ and compare it with the search discussed in Example~\ref{ex:canonical_q-partition_search_RIO_c=0.4} using $\uc := 0.4$. At the first tree level, the q-partition with $\dx{} = \setof{\md_5}$ (instead of $\setof{\md_3}$) is selected for expansion. Note that the heuristic values of the q-partitions are now different to those in Example~\ref{ex:canonical_q-partition_search_RIO_c=0.4} due to the dependence of the heuristic function on $\uc$, i.e.\ now the required number $n$ of diagnoses in the $\dx{}$-set of a goal q-partition is $2$ (instead of $3$). At the second level of the search tree (as opposed to the third in Example~\ref{ex:canonical_q-partition_search_RIO_c=0.4}) an optimal q-partition is already found. We have $9$ generated and $2$ expanded nodes and the search did not involve any prunings or backtrackings.
\qed
\end{example}
\begin{example}\label{ex:canonical_q-partition_search_MPS} (Figure~\ref{fig:ex:can_q-part_search_MPS}, $m = \mathsf{MPS}$) \\
The $\mathsf{MPS}$ measure is pretty simple in that it selects the one successor q-partition
of the initial state whose $\dx{}$-set contains the most probable diagnosis in $\mD$, i.e.\ the q-partition with a singleton $\dx{}$ of maximal probability. The heuristic value -- where a lower value indicates a better q-partition -- is then just the negative probability of the diagnosis in $\dx{}$. Here we count $7$ generated and one expanded nodes.\qed 
\end{example}
\begin{example}\label{ex:canonical_q-partition_search_BME} (Figure~\ref{fig:ex:can_q-part_search_BME}, $m = \mathsf{BME}$) \\
The $\mathsf{BME}$ measure selects the one successor q-partition $\Pt$ whose set $\mD_{Q,p,\min} \in \setof{\dx{}(\Pt),\dnx{}(\Pt)}$ with least probability -- or, more precisely, a probability less than $50\%$ -- has the maximal cardinality. Intuitively, this means that, with a probability of more than $50\%$, the answer will be such that the set $\mD_{Q,p,\min}$ (whose cardinality has been maximized) is invalidated. The (minimization of the) heuristic function $-|\mD_{Q,p,\min}| + p(\mD_{Q,p,\min})$ (see lines~\ref{algoline:dx+_best_successor:if_case_BME} ff.\ in Algorithm~\ref{algo:dx+_best_successor}) achieves exactly this. For it puts higher weight on the cardinality of $\mD_{Q,p,\min}$ (integer) than on the probability of $\mD_{Q,p,\min}$ (which falls into $(0,1)$). In other words, primarily the cardinality it maximized (due to the minus in front of it) and secondarily the probability is minimized. The reason for the probability minimization is the 
larger expected 
number of diagnoses that can be added to $\mD_{Q,p,\min}$ without having $p(\mD_{Q,p,\min})$ exceed $0.5$, i.e.\ the higher chance of finding a very good q-partition w.r.t.\ $\mathsf{BME}$ in the subtree below the successor q-partition. 
In this vein, the algorithm in this example finally returns a q-partition with a bias of $24\%$ towards eliminating four of six leading diagnoses. The stop criterion given by the optimiality threshold $t_{\mathsf{BME}} = 1$ means that it is already sufficient to invalidate the maximum theoretically possible diagnoses number ($5$ in this case) minus $t_{\mathsf{BME}}$ diagnoses. Here we count $16$ generated and four expanded nodes.\qed 
\end{example}
\begin{figure}[tb]
\setlength{\fboxsep}{1.5pt}
	\scriptsize 
	\xygraph{
		!{<0cm,0cm>;<2cm,0cm>:<0cm,1.6cm>::}
		!{(4,5)}*+[F]{\begin{tabular}{c|c}
					$0$ & $1$ \\
					$\emptyset$&$\md_1,\md_2,\md_3,\md_4,\md_5,\md_6$ \\
					$\emptyset$&$\setof{2,3},\setof{2,5},\setof{2,6},\setof{2,7},\setof{1,4,7},\setof{3,4,7}$
					\end{tabular} }="init"
		!{(0,4)}*+[F--]{ {\begin{tabular}{c|c}
				$0.01$ & $0.99$ \\
				$\md_1$&$\md_2,\md_3,\md_4,\md_5,\md_6$ \\
				$\setof{2,3}$&$\setof{5},\setof{6},\setof{7},\setof{1,4,7},\setof{4,7}$
		\end{tabular}}_{\framebox[1.3\width]{0.49}} }="lev1_D1"
		!{(0,3)}*+[F--]{ {\begin{tabular}{c|c}
				$0.33$ & $0.67$ \\
				$\md_2$&$\md_1,\md_3,\md_4,\md_5,\md_6$ \\
				$\setof{2,5}$&$\setof{3},\setof{6},\setof{7},\setof{1,4,7},\setof{3,4,7}$
		\end{tabular}}_{\framebox[1.3\width]{0.17}} }="lev1_D2"
		!{(0,2)}*+[F--]{ {\begin{tabular}{c|c}
				$0.14$ & $0.86$ \\
				$\md_3$&$\md_1,\md_2,\md_4,\md_5,\md_6$ \\
				$\setof{2,6}$&$\setof{3},\setof{5},\setof{7},\setof{1,4,7},\setof{3,4,7}$
		\end{tabular}}_{\framebox[1.3\width]{0.36}} }="lev1_D3"
		!{(0,1)}*+[F--]{ {\begin{tabular}{c|c}
				$0.07$ & $0.93$ \\
				$\md_4$&$\md_1,\md_2,\md_3,\md_5,\md_6$ \\
				$\setof{2,7}$&$\setof{3},\setof{5},\setof{6},\setof{1,4},\setof{3,4}$
		\end{tabular}}_{\framebox[1.3\width]{0.43}} }="lev1_D4"
		!{(0,-1)}*+[F]{ {\begin{tabular}{c|c}
				$0.41$ & $0.59$ \\
				$\md_5$&$\md_1,\md_2,\md_3,\md_4,\md_6$ \\
				$\setof{1,4,7}$&$\setof{2,3},\setof{2,5},\setof{2,6},\setof{2},\setof{3}$
		\end{tabular}}_{\framebox[1.3\width]{0.09}} }="lev1_D5_best"
		!{(0,0)}*+[F--]{ {\begin{tabular}{c|c}
				$0.04$ & $0.96$ \\
				$\md_6$&$\md_1,\md_2,\md_3,\md_4,\md_5$ \\
				$\setof{3,4,7}$&$\setof{2},\setof{2,5},\setof{2,6},\setof{2},\setof{1}$
		\end{tabular}}_{\framebox[1.3\width]{0.46}} }="lev1_D6"
		!{(4,-2)}*+[F=]{ {\begin{tabular}{c|c}
				$0.48$ & $0.52$ \\
				$\md_4,\md_5$&$\md_1,\md_2,\md_3,\md_6$ \\
				$\setof{1,2,4,7}$&$\setof{3},\setof{5},\setof{6},\setof{3}$
		\end{tabular} }_{\framebox[1.3\width]{0.02}}}="lev2_D5,D4"
		!{(4,-3)}*+[F--]{ {\begin{tabular}{c|c}
				$0.45$ & $0.55$ \\
				$\md_5,\md_6$&$\md_1,\md_2,\md_3,\md_4$ \\
				$\setof{1,3,4,7}$&$\setof{2},\setof{2,5},\setof{2,6},\setof{2}$
		\end{tabular} }_{\framebox[1.3\width]{0.05}}}="lev2_D5,D6"
		"init":^{0.01}_{\setof{\md_1}}"lev1_D1"
		"init":@/^2em/^{0.33}_{\setof{\md_2}}"lev1_D2"
		"init":@/^4.6em/^{0.14}_{\setof{\md_3}}"lev1_D3"
		"init":@/^7em/^{0.07}_{\setof{\md_4}}"lev1_D4"
		"init":@/^12em/^{0.41}_{\setof{\md_5}}"lev1_D5_best"
		"init":@/^9em/^{0.04}_{\setof{\md_6}}"lev1_D6"
		"lev1_D5_best":_{0.07}^{\setof{\md_4}}"lev2_D5,D4"
		"lev1_D5_best":@/_2em/_{0.04}^{\setof{\md_6}}"lev2_D5,D6"
	}
	\caption{Search for an optimal canonical q-partition using Algorithm~\ref{algo:dx+_part} (\textsc{findQPartition}) for the example DPI (Table~\ref{tab:example_dpi_0}) w.r.t.\ $m := \mathsf{ENT}$ with threshold $t_{\mathsf{ENT}} := 0.05$.}
	\label{fig:ex:can_q-part_search_ENT}
\end{figure}
\begin{figure}[tb]
\setlength{\fboxsep}{1.5pt}
	\scriptsize 
	\xygraph{
		!{<0cm,0cm>;<2cm,0cm>:<0cm,1.6cm>::}
		!{(4,5)}*+[F]{\begin{tabular}{c|c}
					$0$ & $1$ \\
					$\emptyset$&$\md_1,\md_2,\md_3,\md_4,\md_5,\md_6$ \\
					$\emptyset$&$\setof{2,3},\setof{2,5},\setof{2,6},\setof{2,7},\setof{1,4,7},\setof{3,4,7}$
					\end{tabular} }="init"
		!{(0,-1)}*+[F]{ {\begin{tabular}{c|c}
				$0.01$ & $0.99$ \\
				$\md_1$&$\md_2,\md_3,\md_4,\md_5,\md_6$ \\
				$\setof{2,3}$&$\setof{5},\setof{6},\setof{7},\setof{1,4,7},\setof{4,7}$
		\end{tabular}}_{\framebox[1.7\width]{2}} }="lev1_D1_best"
		!{(0,4)}*+[F--]{ {\begin{tabular}{c|c}
				$0.33$ & $0.67$ \\
				$\md_2$&$\md_1,\md_3,\md_4,\md_5,\md_6$ \\
				$\setof{2,5}$&$\setof{3},\setof{6},\setof{7},\setof{1,4,7},\setof{3,4,7}$
		\end{tabular}}_{\framebox[1.7\width]{2}} }="lev1_D2"
		!{(0,3)}*+[F--]{ {\begin{tabular}{c|c}
				$0.14$ & $0.86$ \\
				$\md_3$&$\md_1,\md_2,\md_4,\md_5,\md_6$ \\
				$\setof{2,6}$&$\setof{3},\setof{5},\setof{7},\setof{1,4,7},\setof{3,4,7}$
		\end{tabular}}_{\framebox[1.7\width]{2}} }="lev1_D3"
		!{(0,2)}*+[F--]{ {\begin{tabular}{c|c}
				$0.07$ & $0.93$ \\
				$\md_4$&$\md_1,\md_2,\md_3,\md_5,\md_6$ \\
				$\setof{2,7}$&$\setof{3},\setof{5},\setof{6},\setof{1,4},\setof{3,4}$
		\end{tabular}}_{\framebox[1.7\width]{2}} }="lev1_D4"
		!{(0,1)}*+[F--]{ {\begin{tabular}{c|c}
				$0.41$ & $0.59$ \\
				$\md_5$&$\md_1,\md_2,\md_3,\md_4,\md_6$ \\
				$\setof{1,4,7}$&$\setof{2,3},\setof{2,5},\setof{2,6},\setof{2},\setof{3}$
		\end{tabular}}_{\framebox[1.7\width]{2}} }="lev1_D5"
		!{(0,0)}*+[F--]{ {\begin{tabular}{c|c}
				$0.04$ & $0.96$ \\
				$\md_6$&$\md_1,\md_2,\md_3,\md_4,\md_5$ \\
				$\setof{3,4,7}$&$\setof{2},\setof{2,5},\setof{2,6},\setof{2},\setof{1}$
		\end{tabular}}_{\framebox[1.7\width]{2}} }="lev1_D6"
		!{(4,-4)}*+[F]{ {\begin{tabular}{c|c}
				$0.34$ & $0.66$ \\
				$\md_1,\md_2$&$\md_3,\md_4,\md_5,\md_6$ \\
				$\setof{2,3,5}$&$\setof{6},\setof{7},\setof{1,4,7},\setof{4,7}$
		\end{tabular} }_{\framebox[1.7\width]{1}}}="lev2_D1,D2_best"
		!{(4,-2)}*+[F--]{ {\begin{tabular}{c|c}
				$0.15$ & $0.85$ \\
				$\md_1,\md_3$&$\md_2,\md_4,\md_5,\md_6$ \\
				$\setof{2,3,6}$&$\setof{5},\setof{7},\setof{1,4,7},\setof{4,7}$
		\end{tabular} }_{\framebox[1.7\width]{1}}}="lev2_D1,D3"
		!{(4,-3)}*+[F--]{ {\begin{tabular}{c|c}
				$0.08$ & $0.92$ \\
				$\md_1,\md_4$&$\md_2,\md_3,\md_5,\md_6$ \\
				$\setof{2,3,7}$&$\setof{5},\setof{6},\setof{1,4},\setof{4}$
		\end{tabular} }_{\framebox[1.7\width]{1}}}="lev2_D1,D4"
		!{(0,-5)}*+[F--]{ {\begin{tabular}{c|c}
				$0.34$ & $0.66$ \\
				$\md_1,\md_2,\md_4$&$\md_3,\md_5,\md_6$ \\
				$\setof{2,3,5,7}$&$\setof{6},\setof{1,4},\setof{4}$
		\end{tabular} }_{\framebox[1.7\width]{0}}}="lev3_D1,D2,D4"
		!{(0,-6)}*+[F=]{ {\begin{tabular}{c|c}
				$0.15$ & $0.85$ \\
				$\md_1,\md_2,\md_3$&$\md_4,\md_5,\md_6$ \\
				$\setof{2,3,5,6}$&$\setof{7},\setof{1,4,7},\setof{4,7}$
		\end{tabular} }_{\framebox[1.7\width]{0}}}="lev3_D1,D2,D3"
		"init":@/^12em/^{0.01}_{\setof{\md_1}}"lev1_D1_best"
		"init":^{0.33}_{\setof{\md_2}}"lev1_D2"
		"init":@/^2em/^{0.14}_{\setof{\md_3}}"lev1_D3"
		"init":@/^4.6em/^{0.07}_{\setof{\md_4}}"lev1_D4"
		"init":@/^7em/^{0.41}_{\setof{\md_5}}"lev1_D5"
		"init":@/^9em/^{0.04}_{\setof{\md_6}}"lev1_D6"
		"lev1_D1_best":@/_4em/_{0.33}^{\setof{\md_2}}"lev2_D1,D2_best"
		"lev1_D1_best":_{0.14}^{\setof{\md_3}}"lev2_D1,D3"
		"lev1_D1_best":@/_2em/_{0.07}^{\setof{\md_4}}"lev2_D1,D4"
		"lev2_D1,D2_best":^{0.07}_{\setof{\md_4}}"lev3_D1,D2,D4"
		"lev2_D1,D2_best":@/^2em/^{0.14}_{\setof{\md_3}}"lev3_D1,D2,D3"
	}
	\caption{Search for an optimal canonical q-partition using Algorithm~\ref{algo:dx+_part} (\textsc{findQPartition}) for the example DPI (Table~\ref{tab:example_dpi_0}) w.r.t.\ $m := \mathsf{SPL}$ with threshold $t_{\mathsf{SPL}} := 0$.}
	\label{fig:ex:can_q-part_search_SPL}
\end{figure}
\begin{figure}[tb]
\setlength{\fboxsep}{1.5pt}
	\scriptsize 
	\xygraph{
		!{<0cm,0cm>;<2cm,0cm>:<0cm,1.6cm>::}
		!{(4,5)}*+[F]{\begin{tabular}{c|c}
					$0$ & $1$ \\
					$\emptyset$&$\md_1,\md_2,\md_3,\md_4,\md_5,\md_6$ \\
					$\emptyset$&$\setof{2,3},\setof{2,5},\setof{2,6},\setof{2,7},\setof{1,4,7},\setof{3,4,7}$
					\end{tabular} }="init"
		!{(0,4)}*+[F--]{ {\begin{tabular}{c|c}
				$0.01$ & $0.99$ \\
				$\md_1$&$\md_2,\md_3,\md_4,\md_5,\md_6$ \\
				$\setof{2,3}$&$\setof{5},\setof{6},\setof{7},\setof{1,4,7},\setof{4,7}$
		\end{tabular}}_{\framebox[1.7\width]{0.094}} }="lev1_D1"
		!{(0,3)}*+[F--]{ {\begin{tabular}{c|c}
				$0.33$ & $0.67$ \\
				$\md_2$&$\md_1,\md_3,\md_4,\md_5,\md_6$ \\
				$\setof{2,5}$&$\setof{3},\setof{6},\setof{7},\setof{1,4,7},\setof{3,4,7}$
		\end{tabular}}_{\framebox[1.7\width]{0.098}} }="lev1_D2"
		!{(0,-1)}*+[F]{ {\begin{tabular}{c|c}
				$0.14$ & $0.86$ \\
				$\md_3$&$\md_1,\md_2,\md_4,\md_5,\md_6$ \\
				$\setof{2,6}$&$\setof{3},\setof{5},\setof{7},\setof{1,4,7},\setof{3,4,7}$
		\end{tabular}}_{\framebox[1.7\width]{0.016}} }="lev1_D3_best"
		!{(0,2)}*+[F--]{ {\begin{tabular}{c|c}
				$0.07$ & $0.93$ \\
				$\md_4$&$\md_1,\md_2,\md_3,\md_5,\md_6$ \\
				$\setof{2,7}$&$\setof{3},\setof{5},\setof{6},\setof{1,4},\setof{3,4}$
		\end{tabular}}_{\framebox[1.7\width]{0.058}} }="lev1_D4"
		!{(0,1)}*+[F--]{ {\begin{tabular}{c|c}
				$0.41$ & $0.59$ \\
				$\md_5$&$\md_1,\md_2,\md_3,\md_4,\md_6$ \\
				$\setof{1,4,7}$&$\setof{2,3},\setof{2,5},\setof{2,6},\setof{2},\setof{3}$
		\end{tabular}}_{\framebox[1.7\width]{0.146}} }="lev1_D5"
		!{(0,0)}*+[F--]{ {\begin{tabular}{c|c}
				$0.04$ & $0.96$ \\
				$\md_6$&$\md_1,\md_2,\md_3,\md_4,\md_5$ \\
				$\setof{3,4,7}$&$\setof{2},\setof{2,5},\setof{2,6},\setof{2},\setof{1}$
		\end{tabular}}_{\framebox[1.7\width]{0.076}} }="lev1_D6"
		!{(4,-1.5)}*+[F--]{ {\begin{tabular}{c|c}
				$0.15$ & $0.85$ \\
				$\md_1,\md_3$&$\md_2,\md_4,\md_5,\md_6$ \\
				$\setof{2,3,6}$&$\setof{5},\setof{7},\setof{1,4,7},\setof{4,7}$
		\end{tabular} }_{\framebox[1.7\width]{0.1375}}}="lev2_D3,D1"
		!{(4,-2.5)}*+[F--]{ {\begin{tabular}{c|c}
				$0.47$ & $0.53$ \\
				$\md_2,\md_3$&$\md_1,\md_4,\md_5,\md_6$ \\
				$\setof{2,5,6}$&$\setof{3},\setof{7},\setof{1,4,7},\setof{3,4,7}$
		\end{tabular} }_{\framebox[1.7\width]{0.1025}}}="lev2_D3,D2"
		!{(4,-3.5)}*+[F]{ {\begin{tabular}{c|c}
				$0.21$ & $0.79$ \\
				$\md_3,\md_4$&$\md_1,\md_2,\md_5,\md_6$ \\
				$\setof{2,6,7}$&$\setof{3},\setof{5},\setof{1,4},\setof{3,4}$
		\end{tabular} }_{\framebox[1.7\width]{0.0925}}}="lev2_D3,D4_best"
		!{(0,-4.5)}*+[F--]{ {\begin{tabular}{c|c}
				$0.22$ & $0.78$ \\
				$\md_1,\md_3,\md_4$&$\md_2,\md_5,\md_6$ \\
				$\setof{2,3,6,7}$&$\setof{5},\setof{1,4},\setof{4}$
		\end{tabular} }_{\framebox[1.7\width]{0.28}}}="lev3_D3,D4,D1"
		!{(0,-6.5)}*+[F=]{ {\begin{tabular}{c|c}
				$0.54$ & $0.46$ \\
				$\md_2,\md_3,\md_4$&$\md_1,\md_5,\md_6$ \\
				$\setof{2,5,6,7}$&$\setof{3},\setof{1,4},\setof{3,4}$
		\end{tabular} }_{\framebox[1.7\width]{0.04}}}="lev3_D3,D4,D2_best"
		!{(0,-5.5)}*+[F--]{ {\begin{tabular}{c|c}
				$0.62$ & $0.38$ \\
				$\md_3,\md_4,\md_5$&$\md_1,\md_2,\md_6$ \\
				$\setof{1,2,4,6,7}$&$\setof{3},\setof{5},\setof{3}$
		\end{tabular} }_{\framebox[1.7\width]{0.12}}}="lev3_D3,D4,D5"
		"init":@/^0em/^{0.01}_{\setof{\md_1}}"lev1_D1"
		"init":@/^2.5em/^{0.33}_{\setof{\md_2}}"lev1_D2"
		"init":@/^13em/^{0.14}_{\setof{\md_3}}"lev1_D3_best"
		"init":@/^5.1em/^{0.07}_{\setof{\md_4}}"lev1_D4"
		"init":@/^7.5em/^{0.41}_{\setof{\md_5}}"lev1_D5"
		"init":@/^10em/^{0.04}_{\setof{\md_6}}"lev1_D6"
		"lev1_D3_best":@/_5em/_{0.07}^{\setof{\md_4}}"lev2_D3,D4_best"
		"lev1_D3_best":_{0.01}^{\setof{\md_1}}"lev2_D3,D1"
		"lev1_D3_best":@/_3em/_{0.33}^{\setof{\md_2}}"lev2_D3,D2"
		"lev2_D3,D4_best":^{0.01}_{\setof{\md_1}}"lev3_D3,D4,D1"
		"lev2_D3,D4_best":@/^4em/^{0.33}_{\setof{\md_2}}"lev3_D3,D4,D2_best"
		"lev2_D3,D4_best":@/^2em/^{0.41}_{\setof{\md_5}}"lev3_D3,D4,D5"
	}
	\caption{Search for an optimal canonical q-partition using Algorithm~\ref{algo:dx+_part} (\textsc{findQPartition}) for the example DPI (Table~\ref{tab:example_dpi_0}) w.r.t.\ $m := \mathsf{RIO}$ with cautiousness parameter $\uc := 0.4$ and thresholds $t_{\mathsf{card}} := 0$ (as $t_{\mathsf{SPL}}$ in Example~\ref{ex:canonical_q-partition_search_SPL}) and $t_{\mathsf{ent}} := 0.05$ (as $t_{\mathsf{ENT}}$ in Example~\ref{ex:canonical_q-partition_search_ENT}).}
	\label{fig:ex:can_q-part_search_RIO_c=0.4}
\end{figure}
\begin{figure}[tb]
	\setlength{\fboxsep}{1.5pt}
	\scriptsize 
	\xygraph{
		!{<0cm,0cm>;<2cm,0cm>:<0cm,1.6cm>::}
		!{(4,5)}*+[F]{\begin{tabular}{c|c}
				$0$ & $1$ \\
				$\emptyset$&$\md_1,\md_2,\md_3,\md_4,\md_5,\md_6$ \\
				$\emptyset$&$\setof{2,3},\setof{2,5},\setof{2,6},\setof{2,7},\setof{1,4,7},\setof{3,4,7}$
		\end{tabular} }="init"
		!{(0,4)}*+[F--]{ {\begin{tabular}{c|c}
					$0.01$ & $0.99$ \\
					$\md_1$&$\md_2,\md_3,\md_4,\md_5,\md_6$ \\
					$\setof{2,3}$&$\setof{5},\setof{6},\setof{7},\setof{1,4,7},\setof{4,7}$
			\end{tabular}}_{\framebox[1.7\width]{0.292}} }="lev1_D1"
		!{(0,3)}*+[F--]{ {\begin{tabular}{c|c}
					$0.33$ & $0.67$ \\
					$\md_2$&$\md_1,\md_3,\md_4,\md_5,\md_6$ \\
					$\setof{2,5}$&$\setof{3},\setof{6},\setof{7},\setof{1,4,7},\setof{3,4,7}$
			\end{tabular}}_{\framebox[1.7\width]{0.036}} }="lev1_D2"
		!{(0,2)}*+[F--]{ {\begin{tabular}{c|c}
					$0.14$ & $0.86$ \\
					$\md_3$&$\md_1,\md_2,\md_4,\md_5,\md_6$ \\
					$\setof{2,6}$&$\setof{3},\setof{5},\setof{7},\setof{1,4,7},\setof{3,4,7}$
			\end{tabular}}_{\framebox[1.7\width]{0.188}} }="lev1_D3"
		!{(0,1)}*+[F--]{ {\begin{tabular}{c|c}
					$0.07$ & $0.93$ \\
					$\md_4$&$\md_1,\md_2,\md_3,\md_5,\md_6$ \\
					$\setof{2,7}$&$\setof{3},\setof{5},\setof{6},\setof{1,4},\setof{3,4}$
			\end{tabular}}_{\framebox[1.7\width]{0.244}} }="lev1_D4"
		!{(0,-1)}*+[F]{ {\begin{tabular}{c|c}
					$0.41$ & $0.59$ \\
					$\md_5$&$\md_1,\md_2,\md_3,\md_4,\md_6$ \\
					$\setof{1,4,7}$&$\setof{2,3},\setof{2,5},\setof{2,6},\setof{2},\setof{3}$
			\end{tabular}}_{\framebox[1.7\width]{0.028}} }="lev1_D5_best"
		!{(0,0)}*+[F--]{ {\begin{tabular}{c|c}
					$0.04$ & $0.96$ \\
					$\md_6$&$\md_1,\md_2,\md_3,\md_4,\md_5$ \\
					$\setof{3,4,7}$&$\setof{2},\setof{2,5},\setof{2,6},\setof{2},\setof{1}$
			\end{tabular}}_{\framebox[1.7\width]{0.268}} }="lev1_D6"
		!{(4,-1.5)}*+[F=]{ {\begin{tabular}{c|c}
					$0.48$ & $0.52$ \\
					$\md_4,\md_5$&$\md_1,\md_2,\md_3,\md_6$ \\
					$\setof{1,2,4,7}$&$\setof{3},\setof{5},\setof{6},\setof{3}$
			\end{tabular} }_{\framebox[1.7\width]{0.02}}}="lev2_D5,D4_best"
		!{(4,-2.5)}*+[F--]{ {\begin{tabular}{c|c}
					$0.45$ & $0.55$ \\
					$\md_5,\md_6$&$\md_1,\md_2,\md_3,\md_4$ \\
					$\setof{1,3,4,7}$&$\setof{2},\setof{2,5},\setof{2,6},\setof{2}$
			\end{tabular} }_{\framebox[1.7\width]{0.05}}}="lev2_D5,D6"
		"init":@/^0em/^{0.01}_{\setof{\md_1}}"lev1_D1"
		"init":@/^2.5em/^{0.33}_{\setof{\md_2}}"lev1_D2"
		"init":@/^5.1em/^{0.14}_{\setof{\md_3}}"lev1_D3"
		"init":@/^7.5em/^{0.07}_{\setof{\md_4}}"lev1_D4"
		"init":@/^13em/^{0.41}_{\setof{\md_5}}"lev1_D5_best"
		"init":@/^10em/^{0.04}_{\setof{\md_6}}"lev1_D6"
		"lev1_D5_best":_{0.07}^{\setof{\md_4}}"lev2_D5,D4_best"
		"lev1_D5_best":@/_3em/_{0.04}^{\setof{\md_6}}"lev2_D5,D6"
	}
	\caption{Search for an optimal canonical q-partition using Algorithm~\ref{algo:dx+_part} (\textsc{findQPartition}) for the example DPI (Table~\ref{tab:example_dpi_0}) w.r.t.\ $m := \mathsf{RIO}$ with cautiousness parameter $\uc := 0.3$ and thresholds $t_{\mathsf{card}} := 0$ (as $t_{\mathsf{SPL}}$ in Example~\ref{ex:canonical_q-partition_search_SPL}) and $t_{\mathsf{ent}} := 0.05$ (as $t_{\mathsf{ENT}}$ in Example~\ref{ex:canonical_q-partition_search_ENT}).}
	\label{fig:ex:can_q-part_search_RIO_c=0.3}
\end{figure}
\begin{figure}[tb]
	\setlength{\fboxsep}{1.5pt}
	\scriptsize 
	\xygraph{
		!{<0cm,0cm>;<2cm,0cm>:<0cm,1.6cm>::}
		!{(4,5)}*+[F]{\begin{tabular}{c|c}
				$0$ & $1$ \\
				$\emptyset$&$\md_1,\md_2,\md_3,\md_4,\md_5,\md_6$ \\
				$\emptyset$&$\setof{2,3},\setof{2,5},\setof{2,6},\setof{2,7},\setof{1,4,7},\setof{3,4,7}$
		\end{tabular} }="init"
		!{(0,4)}*+[F--]{ {\begin{tabular}{c|c}
					$0.01$ & $0.99$ \\
					$\md_1$&$\md_2,\md_3,\md_4,\md_5,\md_6$ \\
					$\setof{2,3}$&$\setof{5},\setof{6},\setof{7},\setof{1,4,7},\setof{4,7}$
			\end{tabular}}_{\framebox[1.3\width]{-0.01}} }="lev1_D1"
		!{(0,3)}*+[F--]{ {\begin{tabular}{c|c}
					$0.33$ & $0.67$ \\
					$\md_2$&$\md_1,\md_3,\md_4,\md_5,\md_6$ \\
					$\setof{2,5}$&$\setof{3},\setof{6},\setof{7},\setof{1,4,7},\setof{3,4,7}$
			\end{tabular}}_{\framebox[1.3\width]{-0.33}} }="lev1_D2"
		!{(0,2)}*+[F--]{ {\begin{tabular}{c|c}
					$0.14$ & $0.86$ \\
					$\md_3$&$\md_1,\md_2,\md_4,\md_5,\md_6$ \\
					$\setof{2,6}$&$\setof{3},\setof{5},\setof{7},\setof{1,4,7},\setof{3,4,7}$
			\end{tabular}}_{\framebox[1.3\width]{-0.14}} }="lev1_D3"
		!{(0,1)}*+[F--]{ {\begin{tabular}{c|c}
					$0.07$ & $0.93$ \\
					$\md_4$&$\md_1,\md_2,\md_3,\md_5,\md_6$ \\
					$\setof{2,7}$&$\setof{3},\setof{5},\setof{6},\setof{1,4},\setof{3,4}$
			\end{tabular}}_{\framebox[1.3\width]{-0.07}} }="lev1_D4"
		!{(0,-1)}*+[F=]{ {\begin{tabular}{c|c}
					$0.41$ & $0.59$ \\
					$\md_5$&$\md_1,\md_2,\md_3,\md_4,\md_6$ \\
					$\setof{1,4,7}$&$\setof{2,3},\setof{2,5},\setof{2,6},\setof{2},\setof{3}$
			\end{tabular}}_{\framebox[1.3\width]{-0.41}} }="lev1_D5_best"
		!{(0,0)}*+[F--]{ {\begin{tabular}{c|c}
					$0.04$ & $0.96$ \\
					$\md_6$&$\md_1,\md_2,\md_3,\md_4,\md_5$ \\
					$\setof{3,4,7}$&$\setof{2},\setof{2,5},\setof{2,6},\setof{2},\setof{1}$
			\end{tabular}}_{\framebox[1.3\width]{-0.04}} }="lev1_D6"
		"init":^{0.01}_{\setof{\md_1}}"lev1_D1"
		"init":@/^2em/^{0.33}_{\setof{\md_2}}"lev1_D2"
		"init":@/^4.6em/^{0.14}_{\setof{\md_3}}"lev1_D3"
		"init":@/^7em/^{0.07}_{\setof{\md_4}}"lev1_D4"
		"init":@/^12em/^{0.41}_{\setof{\md_5}}"lev1_D5_best"
		"init":@/^9em/^{0.04}_{\setof{\md_6}}"lev1_D6"
	}
	\caption{Search for an optimal canonical q-partition using Algorithm~\ref{algo:dx+_part} (\textsc{findQPartition}) for the example DPI (Table~\ref{tab:example_dpi_0}) w.r.t.\ $m := \mathsf{MPS}$. No threshold is needed for $\mathsf{MPS}$.}
	\label{fig:ex:can_q-part_search_MPS}
\end{figure}
\begin{figure}[tb]
	\setlength{\fboxsep}{1.5pt}
	\scriptsize 
	\xygraph{
		!{<0cm,0cm>;<2cm,0cm>:<0cm,1.6cm>::}
		!{(4,5)}*+[F]{\begin{tabular}{c|c}
				$0$ & $1$ \\
				$\emptyset$&$\md_1,\md_2,\md_3,\md_4,\md_5,\md_6$ \\
				$\emptyset$&$\setof{2,3},\setof{2,5},\setof{2,6},\setof{2,7},\setof{1,4,7},\setof{3,4,7}$
		\end{tabular} }="init"
		!{(0,-0.25)}*+[F]{ {\begin{tabular}{c|c}
					$0.01$ & $0.99$ \\
					$\md_1$&$\md_2,\md_3,\md_4,\md_5,\md_6$ \\
					$\setof{2,3}$&$\setof{5},\setof{6},\setof{7},\setof{1,4,7},\setof{4,7}$
			\end{tabular}}_{\framebox[1.3\width]{-0.99}} }="lev1_D1_best"
		!{(0,4.75)}*+[F--]{ {\begin{tabular}{c|c}
					$0.33$ & $0.67$ \\
					$\md_2$&$\md_1,\md_3,\md_4,\md_5,\md_6$ \\
					$\setof{2,5}$&$\setof{3},\setof{6},\setof{7},\setof{1,4,7},\setof{3,4,7}$
			\end{tabular}}_{\framebox[1.3\width]{-0.77}} }="lev1_D2"
		!{(0,3.75)}*+[F--]{ {\begin{tabular}{c|c}
					$0.14$ & $0.86$ \\
					$\md_3$&$\md_1,\md_2,\md_4,\md_5,\md_6$ \\
					$\setof{2,6}$&$\setof{3},\setof{5},\setof{7},\setof{1,4,7},\setof{3,4,7}$
			\end{tabular}}_{\framebox[1.3\width]{-0.86}} }="lev1_D3"
		!{(0,2.75)}*+[F--]{ {\begin{tabular}{c|c}
					$0.07$ & $0.93$ \\
					$\md_4$&$\md_1,\md_2,\md_3,\md_5,\md_6$ \\
					$\setof{2,7}$&$\setof{3},\setof{5},\setof{6},\setof{1,4},\setof{3,4}$
			\end{tabular}}_{\framebox[1.3\width]{-0.93}} }="lev1_D4"
		!{(0,1.75)}*+[F--]{ {\begin{tabular}{c|c}
					$0.41$ & $0.59$ \\
					$\md_5$&$\md_1,\md_2,\md_3,\md_4,\md_6$ \\
					$\setof{1,4,7}$&$\setof{2,3},\setof{2,5},\setof{2,6},\setof{2},\setof{3}$
			\end{tabular}}_{\framebox[1.3\width]{-0.59}} }="lev1_D5"
		!{(0,0.75)}*+[F--]{ {\begin{tabular}{c|c}
					$0.04$ & $0.96$ \\
					$\md_6$&$\md_1,\md_2,\md_3,\md_4,\md_5$ \\
					$\setof{3,4,7}$&$\setof{2},\setof{2,5},\setof{2,6},\setof{2},\setof{1}$
			\end{tabular}}_{\framebox[1.3\width]{-0.96}} }="lev1_D6"
		!{(4,-0.5)}*+[F--]{ {\begin{tabular}{c|c}
					$0.34$ & $0.66$ \\
					$\md_1,\md_2$&$\md_3,\md_4,\md_5,\md_6$ \\
					$\setof{2,3,5}$&$\setof{6},\setof{7},\setof{1,4,7},\setof{4,7}$
			\end{tabular} }_{\framebox[1.7\width]{-1.66}}}="lev2_D1,D2"
		!{(4,-1.5)}*+[F--]{ {\begin{tabular}{c|c}
					$0.15$ & $0.85$ \\
					$\md_1,\md_3$&$\md_2,\md_4,\md_5,\md_6$ \\
					$\setof{2,3,6}$&$\setof{5},\setof{7},\setof{1,4,7},\setof{4,7}$
			\end{tabular} }_{\framebox[1.7\width]{-1.85}}}="lev2_D1,D3"
		!{(4,-2.5)}*+[F]{ {\begin{tabular}{c|c}
					$0.08$ & $0.92$ \\
					$\md_1,\md_4$&$\md_2,\md_3,\md_5,\md_6$ \\
					$\setof{2,3,7}$&$\setof{5},\setof{6},\setof{1,4},\setof{4}$
			\end{tabular} }_{\framebox[1.7\width]{-1.92}}}="lev2_D1,D4_best"
		!{(0,-2.75)}*+[F--]{ {\begin{tabular}{c|c}
					$0.41$ & $0.59$ \\
					$\md_1,\md_2,\md_4$&$\md_3,\md_5,\md_6$ \\
					$\setof{2,3,5,7}$&$\setof{6},\setof{1,4},\setof{4}$
			\end{tabular} }_{\framebox[1.7\width]{-2.59}}}="lev3_D1,D4,D2"
		!{(0,-3.75)}*+[F--]{ {\begin{tabular}{c|c}
					$0.22$ & $0.78$ \\
					$\md_1,\md_3,\md_4$&$\md_2,\md_5,\md_6$ \\
					$\setof{2,3,6,7}$&$\setof{5},\setof{1,4},\setof{4}$
			\end{tabular} }_{\framebox[1.7\width]{-2.78}}}="lev3_D1,D4,D3"
		!{(0,-4.75)}*+[F]{ {\begin{tabular}{c|c}
					$0.12$ & $0.88$ \\
					$\md_1,\md_4,\md_6$&$\md_2,\md_3,\md_5$ \\
					$\setof{2,3,4,7}$&$\setof{5},\setof{6},\setof{1}$
			\end{tabular} }_{\framebox[1.7\width]{-2.88}}}="lev3_D1,D4,D6_best"
		!{(4,-6)}*+[F--]{ {\begin{tabular}{c|c}
					$0.53$ & $0.47$ \\
					$\md_1,\md_4,\md_5,\md_6$&$\md_2,\md_3$ \\
					$\setof{1,2,3,4,7}$&$\setof{5},\setof{6}$
			\end{tabular} }_{\framebox[1.7\width]{-3.47}}}="lev4_D1,D4,D6,D5"
		!{(4,-7)}*+[F=]{ {\begin{tabular}{c|c}
					$0.26$ & $0.74$ \\
					$\md_1,\md_3,\md_4,\md_6$&$\md_2,\md_5$ \\
						$\setof{2,3,4,6,7}$&$\setof{5},\setof{1}$
			\end{tabular} }_{\framebox[1.7\width]{-3.74}}}="lev4_D1,D4,D6,D3_best"
		!{(4,-5)}*+[F--]{ {\begin{tabular}{c|c}
					$0.45$ & $0.55$ \\
					$\md_1,\md_2,\md_4,\md_6$&$\md_3,\md_5$ \\
					$\setof{2,3,4,5,7}$&$\setof{6},\setof{1}$
			\end{tabular} }_{\framebox[1.7\width]{-3.55}}}="lev4_D1,D4,D6,D2"
		"init":@/^13em/^{0.01}_{\setof{\md_1}}"lev1_D1_best"
		"init":@/^0em/^{0.33}_{\setof{\md_2}}"lev1_D2"
		"init":@/^3em/^{0.14}_{\setof{\md_3}}"lev1_D3"
		"init":@/^5.6em/^{0.07}_{\setof{\md_4}}"lev1_D4"
		"init":@/^8em/^{0.41}_{\setof{\md_5}}"lev1_D5"
		"init":@/^10em/^{0.04}_{\setof{\md_6}}"lev1_D6"
		"lev1_D1_best":@/_0em/_{0.33}^{\setof{\md_2}}"lev2_D1,D2"
		"lev1_D1_best":@/_3em/_{0.14}^{\setof{\md_3}}"lev2_D1,D3"
		"lev1_D1_best":@/_5em/_{0.07}^{\setof{\md_4}}"lev2_D1,D4_best"
		"lev2_D1,D4_best":^{0.33}_{\setof{\md_2}}"lev3_D1,D4,D2"
		"lev2_D1,D4_best":@/^2em/^{0.14}_{\setof{\md_3}}"lev3_D1,D4,D3"
		"lev2_D1,D4_best":@/^4em/^{0.04}_{\setof{\md_6}}"lev3_D1,D4,D6_best"
		"lev3_D1,D4,D6_best":@/_0em/_{0.33}^{\setof{\md_2}}"lev4_D1,D4,D6,D2"
		"lev3_D1,D4,D6_best":@/_2em/_{0.41}^{\setof{\md_5}}"lev4_D1,D4,D6,D5"
		"lev3_D1,D4,D6_best":@/_4em/_{0.14}^{\setof{\md_3}}"lev4_D1,D4,D6,D3_best"
	}
	\caption{Search for an optimal canonical q-partition using Algorithm~\ref{algo:dx+_part} (\textsc{findQPartition}) for the example DPI (Table~\ref{tab:example_dpi_0}) w.r.t.\ $m := \mathsf{BME}$ with threshold $t_{\mathsf{BME}} := 1$.}
	\label{fig:ex:can_q-part_search_BME}
\end{figure}

\begin{algorithm}
\small
\caption[Computation of Successor Q-Partitions]{Computing successors in $\dx{}$-Partitioning}\label{algo:dx+_suc}
\begin{algorithmic}[1]
\Require a partition $\Pt_k = \langle \dx{k},\dnx{k}, \emptyset \rangle$ w.r.t.\ $\mD$
\Ensure the set of all canonical q-partitions $sucs$ that result from $\Pt_k$ by a minimal $\dx{}$-transformation
\Procedure{get$\dx{}$Sucs}{$\Pt_k, \mD_{\mathsf{used}}$}
\State $sucs \gets \emptyset$					\Comment{stores successors of $\Pt_k$ that result from minimal $\dx{}$-transformation}
\State $diagsTraits \gets \emptyset$	\Comment{stores tuples including a diagnosis and the trait of the eq.\ class w.r.t.\ $\sim_k$ it belongs to}
\State $eqClasses \gets \emptyset$		\Comment{set of sets of diagnoses, each set is eq.\ class with set-minimal trait, cf.\ Cor.~\ref{cor:nec_followers_form_equivalence_class_wrt_md^(y)}}
\If{$\dx{k} = \emptyset$}	\label{algoline:dx+_suc:is_initial_state}	\Comment{initial State, apply $S_{\mathsf{init}}$, cf.\ Cor.~\ref{cor:--di,MD-di,0--_is_canonical_q-partition_for_all_di_in_mD}}
	\For{$\md \in \dnx{k}$} \label{algoline:dx+_suc:S_init_begin}
		\State $sucs \gets sucs \cup \setof{\tuple{\setof{\md},\dnx{k}\setminus\setof{\md},\emptyset}}$ \label{algoline:dx+_suc:S_init_end}
	\EndFor
\Else				\Comment{$\Pt_k$ is canonical q-partition (due to Cor.~\ref{cor:--di,MD-di,0--_is_canonical_q-partition_for_all_di_in_mD} and Proposition~\ref{prop:minimal_transformation_for_D+_partitioning}), apply $S_{\mathsf{next}}$}
	\For{$\md_i \in \dnx{k}$}		\label{algoline:dx+_suc:S_next_begin}
		\State $t_i \gets \md_i \setminus U_{\dx{k}}$				\Comment{compute trait of the eq.\ class of $\md_i$ (cf.\ Def.~\ref{def:trait})}
		\State $diagsTraits \gets diagsTraits \cup \setof{\tuple{\md_i,t_i}}$   \Comment{enables to retrieve $t_i$ for $\md_i$ in operations below}
	\EndFor \label{algoline:dx+_suc:compute_traits_end}
	\State $diags \gets \dnx{k}$
	\State $minTraitDiags \gets \emptyset$					\Comment{to store one representative of each eq.\ class with set-minimal trait}
	\State $sucsExist \gets \false$									\Comment{will be set to $\true$ if $\Pt_k$ is found to have some canonical successor q-partition}
	\While{$diags \neq \emptyset$}
		\State $\md_i \gets \Call{getFirst}{diags}$		\Comment{$\md_i$ is first element in $diags$ and then $diags \gets diags \setminus \setof{\md_i}$}
		\State $diagAlreadyUsed \gets \false$			\label{algoline:dx+_suc:diagAlreadyUsed_gets_false}   \Comment{will be set to $\true$ if adding $\md_i$ to $\dx{k}$ yields no \emph{new} q-partition}
		\For{$\md_u \in \mD_{\mathsf{used}}$}			\label{algoline:dx+_suc:for_usedDiag_in_D_used}
				\If{$t_u = t_i$}							\label{algoline:dx+_suc:if_trait_of_diag_is_trait_of_already_used_diag}
						\State $diagAlreadyUsed \gets \true$	\label{algoline:dx+_suc:diagAlreadyUsed_gets_true}
						\State $\mathbf{break}$
				\EndIf
		\EndFor		\label{algoline:dx+_suc:end_for_usedDiag_in_D_used}
		\State $necFollowers \gets \emptyset$					\Comment{to store all necessary followers of $\md_i$, cf.\ Def.~\ref{def:necessary_follower}}
		\State $diagOK \gets \true$										\Comment{will be set to $\false$ if $\md_i$ is found to have a non-set-minimal trait}
		\For{$\md_j \in diags \cup minTraitDiags$}
			\If{$t_i \supseteq t_j$}
				\If{$t_i = t_j$}													\Comment{equal trait, $\md_i$ and $\md_j$ are in same eq.\ class}
					\State $necFollowers \gets necFollowers \cup \setof{\md_j}$				\Comment{cf.\ Cor.~\ref{cor:nec_followers_form_equivalence_class_wrt_md^(y)}}
				\Else 					\Comment{$t_i \supset t_j$}
					\State $diagOK \gets \false$					\Comment{eq.\ class of $\md_i$ has a non-set-minimal trait}
				\EndIf
			\EndIf
		\EndFor
		\State $eqCls \gets \setof{\md_i}\cup necFollowers$
		\If{$\lnot sucsExist \land eqCls = \dnx{k}$}					\Comment{test only executed in first iteration of \textbf{while}-loop}
			\State \Return $\emptyset$													\Comment{$\Pt_k$ has single successor partition with $\dnx{} = \emptyset$ which is thus no 
			q-partition}
		\EndIf
		\State $sucsExist \gets \true$												\Comment{existence of $\geq 1$ canonical successor q-partition of $\Pt_k$ guaranteed} 
		\If{$diagOK$} \Comment{$\dnx{y} := \dnx{k}\cup\setof{\md_i}$ satisfies Proposition~\ref{prop:minimal_transformation_for_D+_partitioning},(\ref{prop:minimal_transformation_for_D+_partitioning:enum:dx_y})}
			\If{$\lnot diagAlreadyUsed$}	\label{algoline:dx+_suc:if_not_diagAlreadyUsed}			\Comment{$\tuple{\dx{k}\cup eqCls, \dnx{k}\setminus eqCls, \emptyset}$ is \emph{new} q-partition}
					\State $eqClasses \gets eqClasses \cup \setof{eqCls}$      \label{algoline:dx+_suc:update_eqClasses}
			\EndIf
			\State $minTraitDiags \gets minTraitDiags \cup \setof{\md_i}$ 			\Comment{add one representative for eq.\ class}
		\EndIf
		\State $diags \gets diags \setminus eqCls$				\Comment{delete all representatives for eq.\ class}
	\EndWhile  \label{algoline:dx+_suc:S_next_end}
	\For{$E \in eqClasses$}		\label{algoline:dx+_suc:generate_successors_begin}	\Comment{construct all canonical successor q-partitions by means of eq.\ classes}
		\State $sucs \gets sucs \cup \setof{\tuple{\dx{k}\cup E, \dnx{k}\setminus E, \emptyset}}$   
	\EndFor   \label{algoline:dx+_suc:generate_successors_end}
\EndIf
\State \textbf{return} $sucs$
\EndProcedure
\end{algorithmic}
\normalsize
\end{algorithm}

\subsection{Finding Optimal Queries Given an Optimal Q-Partition}
\label{sec:FindingOptimalQueriesGivenAnOptimalQPartition}
By now, we have demonstrated in Section~\ref{sec:FindingOptimalQPartitions} how a (sufficiently) optimal q-partition w.r.t.\ any measure discussed in Section~\ref{sec:ActiveLearningInInteractiveOntologyDebugging} can be computed. What we also have at hand so far is one particular well-defined query for the identified ``best'' q-partition, namely the canonical query. If we impose no further constraints on the query than an optimal associated q-partition (which already guarantees optimal features of the query w.r.t.\ the used query quality measure), then we are already done and can simply output the canonical query and ask the user to answer it. However, in many practical scenarios, we can expect that a least quality criterion apart from an optimal q-partition is the set-minimality of queries. That means that we will usually want to minimize the number of logical formulas appearing in a query and hence the effort given for the user to walk through them and answer the query. In other words, we consider the consultation of the user a very expensive ``operation''. Another (usually additional) quality criterion might be the understandability of the query for the respective interacting user. In case there are some fault probabilities for this particular user available to the debugging system, one can try to minimize the weight of the returned query in terms of these probabilities. Intuitively, if e.g.\ (a) the sum of fault probabilities of formulas in a query is minimal, then we might rate the probability of the user having understanding difficulties regarding the set of formulas that constitute the query as minimal. Another approach could be to postulate that (b) the maximal fault probability of single formulas appearing in a query must be minimized. This would reflect a situation where each formula should be as easy to understand as possible. In this case (b), queries with a higher cardinality than in (a) might be favored in case all single fault probabilities of formulas in the query for (b) are lower than some fault probability of a formula in the best query for (a). 

The problem we tackle now is how to obtain a set-minimal query associated with a given canonical q-partition (that has been returned by the \textsc{findQPartition} function described in Section~\ref{sec:FindingOptimalQPartitions}) given this q-partition and the related canonical query i.e.\ how to implement the \textsc{selectQueryForQPartition} function in Algorithm~\ref{algo:query_comp}. In other words, the intention is to minimize the canonical query in size while preserving the associated q-partition. At first sight, we might simply choose to employ the same approach that has been exploited in \cite{Rodler2015phd, Rodler2013, Shchekotykhin2012}. This approach involves the usage of a modification of an algorithm called QuickXPlain (QX for short; originally due to \cite{junker04}\footnote{A formal proof of correctness, a detailed description and examples can be found in \cite[Sec.~4.4.1]{Rodler2015phd}.}) which implements a divide-and-conquer strategy with regular calls to a reasoning service to find one set-minimal subset $Q'$ (of generally exponentially many) of a given query $Q$ 
such that $\Pt(Q') = \Pt(Q)$. Although the number of calls to a reasoner required by QX in order to extract a set-minimal query $Q'$ from a query $Q$ is polynomial in $O(|Q'| \log_2 \frac{|Q|}{|Q'|})$, we will learn in this section that we can in fact do without any calls to a reasoner. This is due to the task constituting of a search for a set-minimal explicit-entailments query given an explicit-entailments query.

Subsequently, we give two lemmata that provide essential information about the properties satisfied by all explicit-entailments queries associated with a given q-partition $\Pt$. In particular, if we consider the lattice $(2^{\DiscAx_{\mD}},\subseteq)$ (where $2^{\DiscAx_{\mD}}$ denotes the powerset of $\DiscAx_\mD$) consisting of all subsets of $\DiscAx$ which are partially ordered by $\subseteq$, Lemma~\ref{lem:explicit-ents_query_lower_bound}, given a q-partition $\Pt$, characterizes the set of all lower bounds of the set of explicit-entailments queries $Q \subseteq \DiscAx_\mD$ associated with $\Pt$ in the lattice. All these lower bounds are themselves elements of the set of explicit-entailments queries w.r.t.\ $\Pt$.
\begin{lemma}\label{lem:explicit-ents_query_lower_bound} 
Let $\mD \subseteq \minD_{\tuple{\mo,\mb,\Tp,\Tn}_\RQ}$, $\Pt = \langle \dx{}, \dnx{}, \emptyset\rangle$ be a q-partition w.r.t.\ $\mD$ and $Q \subseteq \DiscAx_\mD$ an (explicit-entailments) query associated with $\Pt$. 
Then $Q' \subseteq Q$ is a query with associated q-partition $\Pt$ iff $Q' \cap \md_i \neq \emptyset$ for each $\md_i \in \dnx{}$.
\end{lemma}
\begin{proof}
``$\Leftarrow$'': Proof by contraposition. Assume there is a $\md_i \in \dnx{}$ such that $Q' \cap \md_i = \emptyset$. Then $\mo_i^* = (\mo \setminus \md_i)\cup\mb\cup U_\Tp \supseteq Q'$ since $\mo \setminus \md_i \supseteq \DiscAx_\mD \setminus \md_i \supseteq Q'$. From this $\mo_i^* \models Q'$ follows by the fact that the entailment relation in $\mathcal{L}$ is extensive. As a result, we have that $\md_i \in \dx{}(Q')$. Consequently, as $\md_i \in \dnx{}$, the q-partition of $Q'$ must differ from the q-partition $\Pt$ of $Q$.

``$\Rightarrow$'': Proof by contradiction. Assume that $Q' \subseteq Q$ is a query with q-partition $\Pt$ and $Q' \cap \md_i = \emptyset$ for some $\md_i \in \dnx{}$. Then $(\mo \setminus \md_i) \cup Q' = (\mo \setminus \md_i)$ since $Q' \subseteq Q \subseteq \DiscAx_\mD \subseteq \mo$ and $Q'\cap\md_i = \emptyset$. Therefore $\mo_i^* \cup Q' = \mo_i^*$ which implies that $Q' \subseteq \mo_i^*$ and thus $\mo_i^* \models Q'$ due the extensive entailment relation in $\mathcal{L}$. Consequently, $\md_i \in \dx{}(Q')$ must hold. Since $\md_i \in \dnx{}$, we can derive that the q-partition of $Q'$ is not equal to the q-partition $\Pt$ of $Q$, a contradiction.
\end{proof}
Next, Lemma~\ref{lem:explicit-ents_query_upper_bound} defines the set of all upper bounds of the set of explicit-entailments queries $Q \subseteq \DiscAx_\mD$ associated with a given q-partition $\Pt$. In fact, it will turn out that this set is a singleton containing exactly the canonical query associated with $\Pt$. In other words, the canonical query is the unique least superset w.r.t.\ $\subseteq$ (and hence a supremum) of all explicit-entailments queries associated with $\Pt$. 
\begin{lemma}\label{lem:explicit-ents_query_upper_bound}
Let $\mD \subseteq \minD_{\tuple{\mo,\mb,\Tp,\Tn}_\RQ}$, $\Pt = \langle \dx{}, \dnx{}, \emptyset\rangle$ be a q-partition w.r.t.\ $\mD$ and $Q \subseteq \DiscAx_\mD$ an (explicit-entailments) query associated with $\Pt$. 
Then $Q'$ with $\DiscAx_\mD \supseteq Q' \supseteq Q$ is a query with associated q-partition $\Pt$ iff $Q' \subseteq U_\mD \setminus U_{\dx{}}$.
\end{lemma}
\begin{proof}
``$\Rightarrow$'': Proof by contraposition. If $Q' \not\subseteq U_\mD \setminus U_{\dx{}}$ then there is an axiom $\tax \in Q'$ such that $\tax \notin U_\mD \setminus U_{\dx{}}$. This implies that $\tax \in U_{\dx{}}$ because $\tax \in Q' \subseteq \DiscAx_\mD = U_\mD \setminus I_\mD$ which means in particular that $\tax \in U_\mD$. Consequently, $\tax \in \md_j$ for some diagnosis $\md_j \in \dx{}$ must apply which is why $\mo_j^* \cup Q'$ must violate some $x \in \RQ\cup \Tn$ due to the subset-minimality of $\md_j$. As a result, $\md_j$ must belong to $\dnx{}(Q')$ and since $\md_j \not\in \dnx{}$, we obtain that the q-partition $\Pt(Q')$ of $Q'$ is different from $\Pt$.

``$\Leftarrow$'': Direct proof. If $Q' \supseteq Q$ and $Q' \subseteq U_\mD \setminus U_{\dx{}}$, then for each $\md_i \in \dx{}$ it holds that $\mo_i^* \models Q'$ by the fact that the entailment relation in $\mathcal{L}$ is extensive and as $Q' \subseteq U_\mD \setminus U_{\dx{}} \subseteq \mo \setminus \md_i \subseteq \mo_i^*$. Hence, each $\md_i \in \dx{}$ is an element of $\dx{}(Q')$. 

For each $\md_j \in \dnx{}$, $\mo_j^* \cup Q'$ must violate some $x\in\RQ\cup\Tn$ by the monotonicity of the entailment relation in $\mathcal{L}$ and since $\mo_j^* \cup Q$ violates some $x\in\RQ\cup\Tn$ as well as $Q' \supseteq Q$. Thus, each $\md_j \in \dnx{}$ is an element of $\dnx{}(Q')$. 

So far, we have shown that $\dx{} \subseteq \dx{}(Q')$ as well as $\dnx{} \subseteq \dnx{}(Q')$. To complete the proof, assume that that some of these set-inclusions is proper, e.g.\ $\dx{} \subset \dx{}(Q')$. In this case, by $\dz{} = \emptyset$, we can deduce that there is some $\md \in \dnx{}$ such that $\md \in \dx{}(Q')$. This is clearly a contradiction to the fact that $\dnx{} \subseteq \dnx{}(Q')$ and the disjointness of the sets $\dx{}(Q')$ and $\dnx{}(Q')$ which must hold by Proposition~\ref{prop:properties_of_q-partitions},(\ref{prop:properties_of_q-partitions:enum:q-partition_is_partition}.). The other case $\dnx{} \subset \dnx{}(Q')$ can be led to a contradiction in an analogue way. Hence, we conclude that $\Pt(Q') = \Pt$.
\end{proof}
In order to construct a minimize an explicit-entailments query for a fixed q-partition $\tuple{\dx{},\dnx{},\emptyset}$, one needs to find a minimal hitting set of all diagnoses in $\dnx{}$, as the following proposition states. Let in the following $\mathsf{MHS}(X)$ denote the set of all minimal hitting sets of the collection of sets $X$. 
\begin{proposition}\label{prop:explicit-ents_query_lower+upper_bound}
Let $\mD \subseteq \minD_{\tuple{\mo,\mb,\Tp,\Tn}_\RQ}$ and $\Pt = \langle \dx{}, \dnx{}, \emptyset\rangle$ be a q-partition w.r.t.\ $\mD$. Then $Q \subseteq \DiscAx_\mD$ is a query with q-partition $\Pt$ iff there is some $H \in \mathsf{MHS}(\dnx{})$ such that $H \subseteq Q \subseteq Q_{\mathsf{can}}(\dx{})$. 
\end{proposition} 
\begin{proof}
The left set inclusion follows directly from Lemma~\ref{lem:explicit-ents_query_lower_bound}. The right set inclusion can be derived as follows. First, by Lemma~\ref{lem:explicit-ents_query_upper_bound}, $Q \subseteq U_\mD \setminus U_{\dx{}}$ holds. Second, we have that $Q_{\mathsf{can}}(\dx{}) := \DiscAx_\mD \cap E_{\mathsf{exp}}(\dx{}) = (U_\mD \setminus I_\mD) \cap (\mo \setminus U_{\dx{}}) = (U_\mD \cap \mo) \setminus (I_{\mD} \cup U_{\dx{}}) = U_\mD \setminus U_{\dx{}}$ since $U_{\mD} \subseteq \mo$ ($U_{\mD}$ is a union of diagnoses and diagnoses are subsets of $\mo$, cf.\ Definition~\ref{def:diagnosis}) and $I_{\mD} \subseteq U_{\dx{}}$ ($I_{\mD}$ is the intersection of all diagnoses in $\mD$, hence a subset of all diagnoses in $\mD$ and in particular of the ones in $\dx{} \subset \mD$, hence a subset of the union $U_{\dx{}}$ of diagnoses in $\dx{}$).
\end{proof}
Notice that Proposition~\ref{prop:explicit-ents_query_lower+upper_bound} in particular implies that the canonical query $Q_{\mathsf{can}}(\dx{})$ is the explicit-entailments query of maximal size for a given q-partition $\Pt = \langle \dx{}, \dnx{}, \emptyset\rangle$.

By means of Proposition~\ref{prop:explicit-ents_query_lower+upper_bound}, the search for set-minimal explicit-entailment queries given a fixed (canonical) q-partition $\Pt = \tuple{\dx{},\dnx{},\emptyset}$ is easily accomplished by building a hitting set tree of the collection of sets $\dnx{}$ in breadth-first manner (cf.\ \cite[Sec.~4.5.1]{Rodler2015phd} and originally \cite{Reiter87}). Let the complete hitting set tree be denoted by $T$. Then the set of all set-minimal queries with associated q-partition $\Pt$ is given by 
\begin{align*}
\setof{H(\mathsf{n})\,|\,\mathsf{n} \mbox{ is a node of $T$ labeled by $valid$ ($\checkmark$)}}
\end{align*}
where $H(\mathsf{n})$ denotes the set of edge labels on the path from the root node to the node $\mathsf{n}$ in $T$. We want to make the reader explicitly aware of the fact that the main source of complexity when constructing a hitting set tree is usually the computation of the node labels which is normally very expensive, e.g.\ needing calls to a reasoning service. In our situation, however, all the sets used to label the nodes of the tree are already \emph{explicitly given}. Hence the construction of the hitting set tree will usually be very efficient in the light of the fact that the number of diagnoses in $\dnx{}$ are bounded above by $|\mD|$ which is a predefined fixed parameter (which is normally relatively small, e.g.\ $\approx 10$, cf.\ \cite{Shchekotykhin2012,Rodler2013}). Apart from that, we are usually satisfied with a single set-minimal query which implies that we could stop the tree construction immediately after having found the first node labeled by $valid$.

Despite of this search being already very efficient, it can in fact be even further accelerated. The key observation to this end is that each explicit-entailments query w.r.t.\ $\Pt_k = \tuple{\dx{k},\dnx{k},\emptyset}$, by Lemma~\ref{lem:explicit-ents_query_upper_bound} and Proposition~\ref{prop:explicit-ents_query_lower+upper_bound}, must not include any axioms in $U_{\dx{}}$. This brings us back to Eq.~\eqref{eq:md_i^(k)} which characterizes the trait $\md_i^{(k)} := \md_i \setminus U_{\dx{k}}$ of a diagnosis $\md_i \in \dnx{k}$ (cf.\ Definition~\ref{def:trait}) given $\Pt_k$. Let in the following $\mathsf{Tr}(\Pt_k)$ denote the set of all traits of diagnoses in $\dnx{k}$ w.r.t.\ the q-partition $\Pt_k$.
Actually, we can state the following which is a straightforward consequence of Proposition~\ref{prop:explicit-ents_query_lower+upper_bound}:
\begin{corollary}\label{cor:min_exp-ents_queries_are_minHS_of_all_traits_of_diags_in_Dnx}
Let $\mD \subseteq \minD_{\tuple{\mo,\mb,\Tp,\Tn}_\RQ}$ and $\Pt_k = \langle \dx{k}, \dnx{k}, \emptyset\rangle$ be a q-partition w.r.t.\ $\mD$. Then $Q \subseteq \DiscAx_\mD$ is a set-minimal query with q-partition $\Pt_k$ iff $Q = H$ for some $H \in \mathsf{MHS}(\mathsf{Tr}(\Pt_k))$.
\end{corollary}
Contrary to minimal diagnoses, traits of minimal diagnoses might be equal to or proper subsets of one another (see Example~\ref{}). By \cite[Prop.~12.6]{Rodler2015phd} which states 
\begin{quote} If $F$ is a collection of sets, and if $S \in F$ and $S' \in F$ such that $S \subset S'$, then $F_{sub} := F \setminus \setof{S'}$ has the same minimal hitting sets as $F$. 
\end{quote}
we can replace $\mathsf{Tr}(\Pt_k)$ by $\mathsf{Tr}_{setmin}(\Pt_k)$ in Corollary~\ref{cor:min_exp-ents_queries_are_minHS_of_all_traits_of_diags_in_Dnx} where $\mathsf{Tr}_{setmin}(\Pt_k)$ terms the set of all \emph{set-minimal} traits of diagnoses in $\dnx{k}$ w.r.t.\ $\Pt_k$, i.e.\ all traits $t$ in $\mathsf{Tr}(\Pt_k)$ for which there is no trait $t'$ in $\mathsf{Tr}(\Pt_k)$ such that $t' \subset t$. 

In case some axiom preference criteria or axiom fault probabilities 
 $p(\tax)$ for axioms $\tax \in \mo$ are available to the debugging system and should be taken into account as mentioned at the beginning of this section, then the hitting set tree could alternatively be constructed e.g.\ by using a uniform-cost search strategy preferring nodes $\mathsf{n}$ for labeling with 
\begin{itemize}
	\item (\emph{MinSum}) a minimal sum $\sum_{\tax\in H(\mathsf{n})} p(\tax)$ (to realize (a) in the first paragraph of this section) or 
	\item (\emph{MinMax}) a minimal $\max_{\tax\in H(\mathsf{n})} p(\tax)$ (to realize (b) in the first paragraph of this section)
\end{itemize}
instead of breadth-first. In this case, the search would detect subset-minimal queries compliant with these criteria first.

Before we explicate the algorithm resulting from the ideas given in this section, consider the following remark.
\begin{remark}\label{rem:traits_enable_systematic_search_for_queries_with_specified_properties}
We want to underline that the insights gained in this section enable us to \emph{construct} a set-minimal query w.r.t.\ a given q-partition \emph{systematically}. That is, we are capable of detecting minimized queries with particular properties (first). If required, for instance, we can in this way guarantee that we present only queries w.r.t.\ a given q-partition $\Pt$ to the user which have least cardinality among all queries w.r.t.\ $\Pt$. When using QX to minimize a query in a manner its q-partition is preserved (cf.\ \cite{Rodler2015phd, Rodler2013, Shchekotykhin2012}), one has little influence on the properties of a query that is returned. One more or less gets \emph{any} set-minimal query w.r.t.\ the given q-partition or, put another way, one cannot prove in general that properties of interest such as minimum cardinality or minimal sum of probabilities as mentioned above hold for the output of QX. Taking into account the fact that set-minimal subsets with a (monotonic\footnote{A property is \emph{monotonic} iff the binary function that returns $1$ if the property holds for the input set and $0$ otherwise is a monotonic function (cf.\ \cite[p.~46]{Rodler2015phd}).}) property (e.g.: the same q-partition) of a set can have a non-negligible range of cardinalities, the ability to locate minimum-cardinality queries might bring significant savings in terms of the effort for a user to answer them.\qed
\end{remark}

\subsubsection*{The Algorithm}
\label{sec:TheAlgorithm}
Algorithm~\ref{algo:select_query_for_q-partition} presents the pseudocode for the selection of one, some or all set-minimal queries w.r.t.\ a given q-partition based on the discussion in this section. A very similar hitting set algorithm -- intended for the computation of minimal diagnoses from minimal conflict sets -- appeared in \cite[Algorithm~2 on p.~66]{Rodler2015phd}. We solely made minor modifications to adapt it to our setting for query computation. The principle, however, remains exactly the same, and so do the optimality, soundness and correctness proofs given in \cite[Sec.~4.5.2 and Sec.~4.6.3]{Rodler2015phd}. In what follows, we provide a walkthrough of Algorithm~\ref{algo:select_query_for_q-partition} strongly based on the one given in \cite[Sec.~4.5.1]{Rodler2015phd}.

\paragraph{Notation.} 
A \emph{node} $\mathsf{n}$ in Algorithm~\ref{algo:select_query_for_q-partition} is defined as the set of formulas that label the edges on the path from the root node to $\mathsf{n}$. In other words, we associate a node $\mathsf{n}$ with $H(\mathsf{n})$.
In this vein, Algorithm~\ref{algo:select_query_for_q-partition} \emph{internally} does not store a labeled tree, but only ``relevant'' sets of nodes. That is, it does not store any
\begin{itemize}
	\item non-leaf nodes,
	\item labels of non-leaf nodes, i.e.\ it does not store which set-minimal trait labels which node,
	\item edges between nodes, 
	\item labels of edges and
	\item leaf nodes labeled by $closed$.
\end{itemize}
Let $T$ denote the (partial) HS-tree produced by Algorithm~\ref{algo:select_query_for_q-partition} at some point during its execution. Then, Algorithm~\ref{algo:select_query_for_q-partition} only stores 
\begin{itemize}
	\item a set of nodes $\mQ_{calc}$ where each node corresponds to the set of edge labels along a path in $T$ leading to a leaf node that has been labeled by $valid$ (set-minimal queries associated with the q-partition $\Pt_k$ given as an input to the algorithm) and
	\item a list of open (non-closed) nodes $\Queue$ where each node in $\Queue$ corresponds to the edge labels along a path in $T$ leading from the root node to a leaf node that has been generated, but has not yet been labeled.
\end{itemize}

This internal representation of the constructed (partial) HS-tree does not constrain the functionality of the algorithm. 
This holds as queries are paths from the root, i.e.\ nodes in the internal representation, and the goal of a the HS-tree is to determine set-minimal queries. 
The node labels or edge labels along a certain path and their order along this path is completely irrelevant when it comes to finding a label for the leaf node of this path. Instead, only the set of edge labels is required for the computation of the label for a leaf node. Also, to rule out nodes corresponding to non-set-minimal queries, it is sufficient to know the set of already found set-minimal queries $\mQ_{calc}$. No already closed nodes are needed for the correct functionality of Algorithm~\ref{algo:select_query_for_q-partition}.

\paragraph{Inputs.} The algorithm takes as input a q-partition $\Pt_k = \tuple{\dx{k},\dnx{k},\emptyset}$ w.r.t.\ $\mD \subseteq \minD_{\langle\mo,\mb,\Tp,\Tn\rangle_\RQ}$, a probability measure $p$ (assigning probabilities to formulas in $\mo$), some search strategy $strat$, a desired computation timeout $time$, and a desired minimal ($n_{\min}$) as well as maximal ($n_{\max}$) number of set-minimal queries with associated q-partition $\Pt$ to be returned. 
Within the algorithm, $strat$ determines whether breadth-first or either strategy in \{(MinSum),(MinMax)\} as discussed in this section is to be used. It is exploited to keep the queue of open nodes sorted in the respective order to achieve desired properties of the (first) found queries.  
The meaning of $n_{\min}$, $n_{\max}$ and $time$ is that the algorithm computes at least the $n_{\min}$ best (w.r.t.\ $strat$) set-minimal queries with associated q-partition $\Pt_k$ and goes on computing further next best set-minimal queries until either the overall computation time reaches the time limit $time$ or $n_{\max}$ diagnoses have been computed. Note that this feature enabling the computation of an arbitrary number of (existing) set-minimal queries is not needed in the default case where only a single set-minimal query, namely the best w.r.t.\ $strat$, is demanded. However, acting with foresight, one can use this feature to precompute a range of set-minimal queries for the case that a user rejects the query and requests an alternative one. If only a single query should be computed, this can be accomplished by the setting $n_{\min} = n_{\max} := 1$ (value of $time$ is irrelevant in this case). 
 
\paragraph{Initialization.}
First, Algorithm~\ref{algo:select_query_for_q-partition} computes the set of traits of all diagnoses in $\dnx{k}$ and stores it in $traits$ (lines~\ref{algoline:hs:compute_traits_start}-\ref{algoline:hs:compute_traits_end}). Next, it applies the function \textsc{deleteNonSetMinimalSets} to $traits$ in order to clean this set from non-set-minimal traits resulting in the set $setminTraits$ (line~\ref{algoline:hs:compute_setmin_traits}). Then, it calls the HS-tree function \textsc{HS} passing the same parameters as discussed in the \emph{Inputs} paragraph above, except that $\Pt_k$ is replaced by $setminTraits$ (simply denoted by $traits$ within the \textsc{HS} function). 

As first steps in \textsc{HS}, the variable $time_{start}$ is initialized with the current system time (function \textsc{getTime}), the set of calculated set-minimal queries $\mQ_{calc}$ is initialized with the empty set and the ordered queue of open nodes $\Queue$ is set to a list including the empty set only (i.e.\ only the unlabeled root node). 

\paragraph{The Main Loop.}
Within the loop (line~\ref{algoline:hs:repeat}) the algorithm gets the node to be processed next, namely the first node $\mathsf{node}$ (\textsc{getFirst}, line~\ref{algoline:hs:getfirst}) in the list of open nodes $\Queue$ ordered by the search strategy $strat$ and (optionally, if $strat$ requires probabilities) $p$, and removes $\mathsf{node}$ from $\Queue$. 

\paragraph{Computation of Node Labels.}
Then, a label is computed for $\mathsf{node}$ in line~\ref{algoline:hs:label}. Nodes are labeled by $valid$, $closed$ or a set-minimal trait (i.e.\ an element of $traits$) by the procedure \textsc{label} (line~\ref{algoline:hs:procedure_label} ff.). This procedure gets as inputs the current node $\mathsf{node}$, the set $traits$, the already computed set-minimal queries ($\mQ_{calc}$) and the queue $\Queue$ of open nodes, and it returns a label for $\mathsf{node}$. It works as follows:

A node $\mathsf{node}$ is labeled by $closed$ iff 
\begin{itemize}
	\item there is an already computed set-minimal query that is a subset of this node, which means that $\mathsf{node}$ cannot be a set-minimal query (non-minimality criterion, lines~\ref{algoline:hs:non_min_crit_start}-\ref{algoline:hs:non_min_crit_end}) or
	\item there is some node $\mathsf{nd}$ in the queue of open nodes $\Queue$ such that $\mathsf{node} = \mathsf{nd}$ which means that one of the two tree branches with an equal set of edge labels can be closed, i.e.\ removed from $\Queue$ (duplicate criterion, lines~\ref{algoline:hs:duplicate_crit_start}-\ref{algoline:hs:duplicate_crit_end}).
\end{itemize}

If none of these $closed$-criteria is met, the algorithm searches for some $t$ in $traits$ such that $t \cap \mathsf{node} = \emptyset$ and returns the label $t$ for $\mathsf{node}$ (lines~\ref{algoline:hs:reuse_crit_start}-\ref{algoline:hs:reuse_crit_end}). This means that the path represented by $\mathsf{node}$ cannot be a query as there is (at least) one set-minimal trait, namely $t$, that is not hit by $\mathsf{node}$ (cf.\ Corollary~\ref{cor:min_exp-ents_queries_are_minHS_of_all_traits_of_diags_in_Dnx}). 

If none of the mentioned criteria is satisfied, then $\mathsf{node}$ must be a hitting set of all set-minimal traits. Furthermore, is must represent a minimal hitting set. Hence, $valid$ is returned by \textsc{label} in this case.

Using $strat$ prescribing a breath-first search strategy, the minimality of $\mathsf{node}$ is guaranteed by the sorting of $\Queue$ by ascending node cardinality and the non-minimality criterion (see above) which is tested first within the \textsc{label} function. In case $strat$ dictates the usage of (MinSum), then, since all formula probabilities are required to be larger than zero, any subset of a node $\mathsf{n}$ in $\Queue$ must be processed before $\mathsf{n}$ is processed (as this subset must have a strictly smaller sum of formulas' probabilities than $\mathsf{n}$). Minimality of the hitting set $\mathsf{node}$ is then also enforced by the non-minimality criterion. If (MinMax) is used, then the $\Queue$ must be sorted in ascending order by maximal probability of an element (i.e.\ a formula) in a node (as the maximum of a set can only get larger or remain constant if the set grows) and, if two nodes are equal w.r.t.\ this criterion, then the minimum cardinality node has to be given precedence. Then, also in this case, the non-minimality criterion guarantees minimality of the hitting set $\mathsf{node}$. 

\paragraph{Processing of a Node Label.}
Back in the main procedure, the label $L$ returned by \textsc{label} is processed as follows: 

If $L = valid$, then $\mathsf{node}$ is added to the set of calculated set-minimal queries $\mQ_{calc}$. 

If, $L=closed$, then there is either a query in $\mQ_{calc}$ that is a subset of the current node $\mathsf{node}$ or a duplicate of $\mathsf{node}$ is already included in $\Queue$. Consequently, $\mathsf{node}$ must simply be removed from $\Queue$ which has already been executed in line~\ref{algoline:hs:getfirst}.

In the third case, if a set-minimal trait $L$ is returned in line~\ref{algoline:hs:label}, then $L$ is a label for $\mathsf{node}$ meaning that $|L|$ successor nodes of $\mathsf{node}$ need to be inserted into $\Queue$ in a way the order as per $strat$ in $\Queue$ is maintained (\textsc{insertSorted}, line~\ref{algoline:hs:generate_nodes}).



\begin{algorithm}
\small
\caption[Computation of Set-Minimal Queries From Q-Partition]{HS-Tree computation of set-minimal queries from a given q-partition} \label{algo:select_query_for_q-partition}
\begin{algorithmic}[1]
\Require a q-partition $\Pt_k = \tuple{\dx{k},\dnx{k},\emptyset}$ w.r.t.\ $\mD \subseteq \minD_{\langle\mo,\mb,\Tp,\Tn\rangle_\RQ}$, probability measure $p$, query selection parameters $\tuple{strat,time,n_{\min},n_{\max}}$ consisting of search strategy $strat$, desired computation timeout $time$, a desired minimal ($n_{\min}$) and maximal ($n_{\max}$) number of set-minimal queries with associated q-partition $\Pt$ to be returned,  
\Ensure a set $\mQ$ which is \newline
(a)~a set of best (according to $strat$ and possibly $p$) set-minimal queries with associated q-partition $\Pt$ such that $n_{\min} \leq |\mQ| \leq n_{\max}$, if at least $n_{\min}$ such queries exist, or \newline
(b)~the set of all set-minimal queries with associated q-partition $\Pt$ otherwise
\Procedure{selectQueryForQPartition}{$\Pt_k, p, \tuple{strat, time, n_{\min}, n_{\max}}$}
	\State $traits \gets \emptyset$    \label{algoline:hs:compute_traits_start}
	\For{$\md_i \in \dnx{k}$}		\label{algoline:}
		\State $t_i \gets \md_i \setminus U_{\dx{k}}$				\Comment{compute trait $t_i$ of $\md_i$ (cf.\ Def.~\ref{def:trait})}
		\State $traits \gets traits \cup \setof{t_i}$
	\EndFor    \label{algoline:hs:compute_traits_end}
	\State $setminTraits \gets \Call{deleteNonSetMinimalSets}{traits}$   \label{algoline:hs:compute_setmin_traits}
	\State \Return $\Call{HS}{setminTraits, p, strat, time, n_{\min}, n_{\max}}$
\EndProcedure
\vspace{2pt}
\hrule
\vspace{2pt}
\Procedure{$\scHS$}{$traits, p, strat, time, n_{\min}, n_{\max}$}
\State $time_{start} \gets \Call{getTime}{ }$
\State $\mQ_{calc} \gets \emptyset$     \Comment{stores already computed set-minimal queries}
\State $\Queue \gets [\emptyset]$				\Comment{queue of open nodes}						
\Repeat 				\Comment{a node $\mathsf{n}$ represents the set of edge labels $H(\mathsf{n})$ along its branch from the root node}									\label{algoline:hs:repeat}
\State $\mathsf{node} \gets \Call{getFirst}{\Queue}$	\label{algoline:hs:getfirst}  \Comment{$\mathsf{node}$ is first element in $\Queue$ and then $\Queue \gets \Queue \setminus \setof{\mathsf{node}}$}
\State $L \gets \Call{label}{\mathsf{node}, traits, \mQ_{calc}, \Queue}$\label{algoline:hs:label}
\State $\mathbf{C}_{calc} \gets \mathbf{C}$    				\label{algoline:hs:update_Ccalc}
\If{$L = valid$}\label{algoline:hs:L=valid} 			\Comment{$\mathsf{node}$ is a set-minimal query}
	\State $\mQ_{calc} \gets \mQ_{calc} \cup \setof{\mathsf{node}}$		\label{algoline:hs:update_Dcalc}
\ElsIf{$L = closed$}\label{algoline:hs:do_nothing}  \Comment{$\mathsf{node}$ is closed, hence do nothing}
\Else 	\Comment{$L$ must be a set-minimal trait}
	\For{$e \in L$}						
		\State $\Queue \gets \Call{insertSorted}{ \mathsf{node} \cup \setof{e}, \Queue, strat, p}$\label{algoline:hs:generate_nodes}   \Comment{generate successors}
\EndFor
\EndIf
\Until{$\Queue = [] \lor [|\mQ_{calc}| \geq n_{\min} \land ( |\mQ_{calc}| = n_{\max} \lor \Call{getTime}{ } - time_{start} > time)]$}  \label{algoline:hs:until}
\State \Return $\mQ_{calc}$
\EndProcedure
\vspace{2pt}
\hrule
\vspace{2pt}
\Procedure{\textsc{label}}{$\mathsf{node},traits,\mQ_{calc}, \Queue$}    \label{algoline:hs:procedure_label}
\For{$\mathsf{nd} \in \mQ_{calc}$}																								\label{algoline:hs:non_min_crit_start}
	\If{$\mathsf{node} \supseteq \mathsf{nd}$}  \Comment{non-minimality}
			\State \Return $closed$														\label{algoline:hs:non_min_crit_end}
	\EndIf
\EndFor
\For{$\mathsf{nd} \in \Queue$}																										\label{algoline:hs:duplicate_crit_start}
	\If{$\mathsf{node} = \mathsf{nd}$}  \Comment{remove duplicates}
			\State \Return $closed$														\label{algoline:hs:duplicate_crit_end}
	\EndIf
\EndFor
\For{$t \in traits$}																									\label{algoline:hs:reuse_crit_start}
	\If{$t \cap \mathsf{node} = \emptyset$}\label{algoline:hs:test_cs_not_hit}   \Comment{not-yet-hit trait found}
		\State \Return $t$																\label{algoline:hs:reuse_crit_end}
	\EndIf
\EndFor
\State \Return $valid$														\label{algoline:hs:return_valid}
\EndProcedure
\end{algorithmic}
\normalsize
\end{algorithm}

\subsection{Query Enrichment}
\label{sec:EnrichmentOfAQuery}
So far, we have explicated how we can obtain an explicit-entailments query which
\begin{itemize}
	\item is \emph{set-minimal},
	\item has \emph{optimal} properties regarding \emph{user effort} (minimal size) \emph{or comprehensibility} (minimal sum of fault probabilities) and
	\item has \emph{optimal} properties w.r.t.\ \emph{diagnosis discrimination} (optimal value w.r.t.\ some query quality measure $m$).
\end{itemize}
However, sometimes a user might find it hard to directly assess the correctness of axioms formulated by herself or by some other author, i.e.\ by being asked to classify axioms occurring explicitly in the KB as correct or faulty in the intended domain. One reason for this can be the complexity of axioms in the KB. Moreover, especially in case the axiom was specified by the interacting user, the user is usually convinced that the axiom is OK. Probably most of the people tend to work to the best of their knowledge and, being aware of this, have particular problems when it comes to recognizing and realizing their own faults. Instead, it might be easier, less error-prone and more convenient for users to be presented with a query including simple formulas \emph{not} in the KB that need to be assessed regarding their truth in the domain the user intends to model.

Therefore, we deal in this section with the \textsc{enrichQuery} function in Algorithm~\ref{algo:query_comp}. This function realizes the expansion of a given explicit-entailments query $Q \subseteq \mo$ w.r.t.\ a set of leading diagnoses $\mD \subseteq \minD_{\tuple{\mo,\mb,\Tp,\Tn}_\RQ}$ by (finitely many) additional formulas $\alpha_1,\dots,\alpha_r$. At this, we postulate that 

\begin{enumerate}
	\item \label{enum:EQ1} $\alpha_1,\dots,\alpha_r \notin \mo \cup \mb \cup U_{\Tp}$, that
	\item \label{enum:EQ2} $S \models \setof{\alpha_1,\dots,\alpha_r}$ where $S$ is some solution KB $S$ w.r.t.\ $\tuple{\mo,\mb,\Tp,\Tn}_\RQ$ satisfying $Q \subseteq S \subseteq \mo \cup \mb \cup U_{\Tp}$ (note that any solution KB is in particular consistent, cf.\ Definitions~\ref{def:solution_KB} and \ref{def:dpi}), and that
	\item \label{enum:EQ3} no $\alpha_i$ for $i\in\setof{1,\dots,r}$ is an entailment of $S \setminus Q$.
\end{enumerate}
In other words, we want to extract a set of implicit entailments from $S$ that depend on $Q$, i.e.\ which hold only in the presence of (some axioms in) $Q$. This is equivalent to the postulation that $J \cap Q \neq \emptyset$ for each justification $J \in \bigcup_{i = 1}^{r} \Just(\alpha_i,S)$.

To accomplish that, we require a reasoning service that, given a set of formulas $X$ in the logic $\mathcal{L}$, deterministically returns a set $E(X)$ of (not only explicit) entailments of $X$ such that (R1)~$E(X') \subseteq E(X'')$ whenever $X' \subset X''$ and (R2)~if 
$Y \subseteq X' \subset X''$, then for all entailments $e_Y$ of $Y$ it holds that $e_Y \in E(X')$ iff $e_Y \in E(X'')$.
One possibility to realize such a service is to employ a reasoner for the logic $\mathcal{L}$ and use it to extract all entailments of a predefined type it can compute (cf.\ \cite[Remark 2.3 and p.~101]{Rodler2015phd}). For propositional Horn logic, e.g.\ one might extract only all literals that are entailments of $X$. For general Propositional Logic, e.g.\ one might calculate all formulas of the form $A \odot B$ for propositional variables $A,B$ and logical operators $\odot \in \setof{\rightarrow, \leftrightarrow}$, and for Description Logics \cite{Baader2007}, e.g.\ only all subsumption and/or class assertion formulas that are entailments could be computed. An example of entailment types that might be extracted for (decidable fragments of) first-order logic can be found in \cite[Example~8.1]{Rodler2015phd}.

Now, assuming available such a function $E$, we propose to compute an enriched query $Q'$ from an explicit-entailments query $Q$ as 
\begin{align}\label{eq:Q'_enriched_query}
Q' := Q \cup Q_{impl}
\end{align}
 where
\begin{align}\label{eq:Q_impl}
Q_{impl} := \Big[ E\big( (\mo\setminus U_{\mD}) \cup Q \cup \mb \cup U_{\Tp}\big) \;\setminus \; E\big( (\mo\setminus U_{\mD}) \cup \mb \cup U_{\Tp}\big) \Big] \;\setminus\; Q
\end{align}
\begin{remark}\label{rem:only_2_reasoner_calls_to_compute_Q'}
It should be noted that only two calls to a reasoning engine (i.e.\ function $E$) are required to compute $Q'$. The reader is also reminded of the fact that these are the first invocations of a reasoner in the entire query computation process described so far.\qed 
\end{remark}
Next, we prove that the requirements to the formulas $Q_{impl}$ used to expand $Q$ in terms of the enumeration (\ref{enum:EQ1}.), (\ref{enum:EQ2}.)\ and (\ref{enum:EQ3}.)\ above are actually met by Eq.~\eqref{eq:Q_impl}:
\begin{proposition}\label{prop:EQ1-EQ3_hold_for_Q_impl}
Let $\mD \subseteq \minD_{\tuple{\mo,\mb,\Tp,\Tn}_\RQ}$, $Q\subseteq \mo$ be an (explicit-entailments) query w.r.t.\ $\mD$ and $\tuple{\mo,\mb,\Tp,\Tn}_\RQ$ with associated q-partition $\Pt = \tuple{\dx{},\dnx{},\emptyset}$ and let $Q_{impl} = \setof{\alpha_1,\dots,\alpha_r}$ be defined as in Eq.~\eqref{eq:Q_impl}. Then (\ref{enum:EQ1}.)-(\ref{enum:EQ3}.)\ (stated above) hold.
\end{proposition}
\begin{proof}
Ad \ref{enum:EQ1}.: The function $E$ either does or does not compute explicit entailments (amongst other entailments). In case the function $E$ does not compute explicit entailments, $Q_{impl}$ clearly cannot contain any explicit entailments. Otherwise, we distinguish between explicit entailments in $(\mo\setminus U_{\mD}) \cup \mb \cup U_{\Tp}$ and those in $Q$ (clearly, there cannot be any other explicit entailments in $Q_{impl}$). Note that $Q \subseteq U_{\mD} \setminus U_{\dx{}} \subseteq U_\mD$ due to Lemma~\ref{lem:explicit-ents_query_upper_bound}. Additionally, $Q \cap \mb = \emptyset$ due to $Q \subseteq \mo$ and Definition~\ref{def:dpi}. And, $Q \cap U_{\Tp} = \emptyset$ due to $Q \subseteq U_\mD$ and since no element of any minimal diagnosis $\md$ (in $\mD$), and hence no element in $U_\mD$, can occur in $U_{\Tp}$. The latter holds as in case $\md' \cap U_{\Tp} \neq \emptyset$ for $\md' \in \mD$ we would have that $\md'' := \md' \setminus U_{\Tp} \subset \md'$ is a diagnosis w.r.t.\ $\tuple{\mo,\mb,\Tp,\Tn}_\RQ$, a contradiction to the subset-minimality of $\md'$. All in all, we have derived that $(\mo\setminus U_{\mD}) \cup \mb \cup U_{\Tp}$ and $Q$ are disjoint sets. 

Now, $Q_{impl}$ cannot include any elements of $(\mo\setminus U_{\mD}) \cup \mb \cup U_{\Tp}$. This must be satisfied since, first, $(\mo\setminus U_{\mD}) \cup \mb \cup U_{\Tp}$ is a subset of the left- as well as right-hand $E()$ expression in the definition of $Q_{impl}$ (Eq.~\eqref{eq:Q_impl}) and, second, both $E()$ expressions must return the same set of entailments of 
$(\mo\setminus U_{\mD}) \cup \mb \cup U_{\Tp}$ by assumption (R2) made about the function $E$ above. Therefore, the set defined by the squared brackets in Eq.~\eqref{eq:Q_impl} cannot include any (explicit) entailments of $(\mo\setminus U_{\mD}) \cup \mb \cup U_{\Tp}$. 

Further on, $Q_{impl}$ cannot contain any elements of $Q$. This is guaranteed by the elimination of all elements of $Q$ from the set defined by the squared brackets in Eq.~\eqref{eq:Q_impl}. Finally, we summarize that $Q_{impl} \cap (\mo \cup \mb \cup U_{\Tp}) = \emptyset$.

Ad \ref{enum:EQ2}.: Clearly, by the definition of a diagnosis (Definition~\ref{def:diagnosis}), $(\mo\setminus \md) \cup \mb \cup U_{\Tp}$ is a solution KB w.r.t.\ $\tuple{\mo,\mb,\Tp,\Tn}_\RQ$ for all $\md \in \mD$. In addition, since $\dx{} \neq \emptyset$ (cf.\ Proposition~\ref{prop:properties_of_q-partitions},(\ref{prop:properties_of_q-partitions:enum:for_each_q-partition_dx_is_empty_and_dnx_is_empty}.)), there must be some diagnosis $\md' \in \dx{} \subset \mD$ such that $(\mo\setminus \md') \cup \mb \cup U_{\Tp} \models Q$. This implies that $(\mo\setminus \md') \cup \mb \cup U_{\Tp} \cup Q$ is a solution KB w.r.t.\ $\tuple{\mo,\mb,\Tp,\Tn}_\RQ$. By $U_\mD \supseteq \md'$ and by the monotonicity of the logic $\mathcal{L}$ we conclude that $S := (\mo\setminus U_\mD) \cup \mb \cup U_{\Tp} \cup Q$ is a solution KB w.r.t.\ $\tuple{\mo,\mb,\Tp,\Tn}_\RQ$. Further, by the definition of a DPI (Definition~\ref{def:dpi}) which states that the requirements $\RQ$ specified in the DPI must include \emph{consistency}, $S$ is consistent as well.
%

Obviously, $S \supseteq Q$ and, by the left-hand $E()$ expression in Eq.~\eqref{eq:Q_impl}, $Q_{impl}$ includes entailments of $S$.
Finally, by the proof of (\ref{enum:EQ1}.)\ above, where we have shown that $Q_{impl} \cap [(\mo\setminus U_{\mD}) \cup \mb \cup U_{\Tp} \cup Q] = \emptyset$, we immediately obtain that $Q_{impl} \cap S = \emptyset$. That is, $Q_{impl}$ does not include any explicit entailments of $S$.

Ad \ref{enum:EQ3}.: Assume that $S$ is as defined in the proof of (\ref{enum:EQ2}.)\ above and that there is some $\alpha_i \in Q_{impl}$ such that $S \setminus Q \models \alpha_i$. Then, $(\mo\setminus U_\mD) \cup \mb \cup U_{\Tp} \models \alpha_i$. However, in the proof of (\ref{enum:EQ1}.)\ above we have derived that $Q_{impl}$ cannot comprise any entailments of $(\mo\setminus U_{\mD}) \cup \mb \cup U_{\Tp}$. Hence, $\alpha_i \notin Q_{impl}$, contradiction.

We sum up that (\ref{enum:EQ1}.)-(\ref{enum:EQ3}.)\ holds for $Q_{impl} = \setof{\alpha_1,\dots,\alpha_r}$.
\end{proof}
\begin{remark} \label{rem:Q_impl=emptyset_if_no_implicit_entailments}
Please take a note of the fact that $Q_{impl} = \emptyset$ in case there are no implicit entailments of $(\mo\setminus U_{\mD}) \cup Q \cup \mb \cup U_{\Tp}$ which are not entailed by $(\mo\setminus U_{\mD}) \cup \mb \cup U_{\Tp}$ (cf.\ Eq.~\eqref{eq:Q_impl}). In this case the enriched query $Q'$ constructed from the explicit-entailments query $Q$ is equal to $Q$.\qed
\end{remark}

It is crucial that the expanded query resulting from query enrichment given an input explicit-entail-ments query $Q$ has the same q-partition as $Q$. Recall that we first selected the optimal q-partition for which we then chose $Q$ as a query. Thence, of course, we want to preserve this q-partition throughout all further refinement operations applied to $Q$. The next proposition witnesses that query enrichment does in fact comply with this requirement. Note that the inclusion of $Q$ itself in the enriched query $Q'$ ensures the q-partition preservation.
\begin{proposition}\label{prop:entailment_extraction_is_q-partition_preserving}
Let $\mD \subseteq \minD_{\tuple{\mo,\mb,\Tp,\Tn}_\RQ}$ and $Q\subseteq \mo$ be an (explicit-entailments) query w.r.t.\ $\mD$ and $\tuple{\mo,\mb,\Tp,\Tn}_\RQ$. 
Further, let 
$Q'$ be defined as in Eq.~\eqref{eq:Q'_enriched_query}. Then $Q'$ is a query w.r.t.\ $\mD$ and $\tuple{\mo,\mb,\Tp,\Tn}_\RQ$ and $\Pt(Q') = \Pt(Q)$.
\end{proposition}
\begin{proof}
Let $\md \in \dx{}(Q)$. Then, $(\mo\setminus \md) \cup \mb \cup U_{\Tp} \models Q$. Since the entailment relation in $\mathcal{L}$ is idempotent, we have that (*): $(\mo\setminus \md) \cup \mb \cup U_{\Tp} \cup Q \equiv (\mo\setminus \md) \cup \mb \cup U_{\Tp}$. Further, since $Q_{impl}$ is a set of entailments of $(\mo\setminus U_\mD) \cup \mb \cup U_{\Tp} \cup Q$ (see left-hand $E()$ expression in Eq.~\eqref{eq:Q_impl}), by the monotonicity of the entailment relation in $\mathcal{L}$ and because of $(\mo\setminus \md) \cup \mb \cup U_{\Tp} \cup Q \supseteq (\mo\setminus U_\mD) \cup \mb \cup U_{\Tp} \cup Q$ we deduce that $(\mo\setminus \md) \cup \mb \cup U_{\Tp} \cup Q \models Q_{impl}$. 
By (*), $(\mo\setminus \md) \cup \mb \cup U_{\Tp} \models Q_{impl}$ and $(\mo\setminus \md) \cup \mb \cup U_{\Tp} \models Q$ which is why $(\mo\setminus \md) \cup \mb \cup U_{\Tp}$ must entail $Q' = Q_{impl}\cup Q$ as well. Thus, $\md \in \dx{}(Q')$ holds.

Let $\md \in \dnx{}(Q)$. Then, $(\mo\setminus \md) \cup \mb \cup U_{\Tp} \cup Q$ violates some $x \in \RQ \cup \Tn$. Due to the monotonicity of $\mathcal{L}$ and the fact that $Q' = Q_{impl} \cup Q \supseteq Q$, we immediately obtain that $(\mo\setminus \md) \cup \mb \cup U_{\Tp} \cup Q'$ violates some $x \in \RQ \cup \Tn$. Thus, $\md \in \dnx{}(Q')$.

Since $Q$ is an explicit-entailments query, Proposition~\ref{prop:d0} ensures that $\dz{}(Q) = \emptyset$. At this point, an analogue argumentation as we gave in the last paragraph of the proof of Lemma~\ref{lem:explicit-ents_query_upper_bound} can be used to realize that $\dx{}(Q) = \dx{}(Q')$, $\dnx{}(Q) = \dnx{}(Q')$ as well as $\dz{}(Q) = \dz{}(Q')$. Hence, $\Pt(Q) = \Pt(Q')$.
\end{proof}
To sum up, Eq.~\eqref{eq:Q'_enriched_query} specifies
\begin{itemize}
	\item an enrichment of a given explicit-entailments query $Q$
	\item by solely implicit entailments (set $Q_{impl}$) of a consistent part of the extended KB $\mo \cup \mb \cup U_{\Tp}$,
	\item such that each implicit entailment in $Q_{impl}$ is dependent on $Q$, and
	\item such that the enriched query has the same q-partition as $Q$.
\end{itemize}
For the sake of completeness, albeit implementing exactly Eq.~\eqref{eq:Q'_enriched_query}, we present with Algorithm~\ref{algo:enrich_query} the pseudocode illustrating how an enriched query complying with the postulations from the beginning of this section can be computed.

\begin{algorithm*}
\small
\caption{Query Enrichment \normalsize} \label{algo:enrich_query}
\begin{algorithmic}[1]
\Require a DPI $\tuple{\mo,\mb,\Tp,\Tn}_\RQ$, an explicit-entailments query $Q$ w.r.t.\ $\tuple{\mo,\mb,\Tp,\Tn}_\RQ$, the q-partition $\Pt = \tuple{\dx{},\dnx{},\emptyset}$ associated with $Q$, a reasoner $Rsnr$, a set $ET$ of entailment types (w.r.t.\ the used logic $\mathcal{L}$) to be extracted ($ET$ must be supported by the employed reasoner $Rsnr$) 
\Ensure a superset $Q'$ of $Q$ such that the set $Q_{impl} := Q'\setminus Q = \setof{\alpha_1,\dots,\alpha_r}$ satisfies requirements (\ref{enum:EQ1}.)\ - (\ref{enum:EQ3}.)\ stated at the beginning of Section~\ref{sec:EnrichmentOfAQuery}
\Procedure{$\textsc{enrichQuery}$}{$\langle\mo,\mb,\Tp,\Tn\rangle_\RQ,Q,\Pt, Rsnr, ET$}   \label{algoline:query_enrichQuery_start}
	\State $\mD \gets \dx{} \cup \dnx{}$
	\State $E_{+Q} \gets \Call{getEntailments}{ET,Rsnr,(\mo \setminus U_{\mD})\cup Q \cup \mb \cup U_{\Tp}}$
	\State $E_{-Q} \gets \Call{getEntailments}{ET,Rsnr,(\mo \setminus U_{\mD}) \cup \mb \cup U_{\Tp}}$
	\State $Q_{impl} \gets [E_{+Q} \setminus E_{-Q}] \setminus Q$    
	\State $Q' \gets Q \cup Q_{impl}$
	\State \Return $Q'$   
\EndProcedure
\end{algorithmic}
\normalsize
\end{algorithm*}

\subsection{Query Optimization}
\label{sec:QueryOptimization}
The enriched query $Q'$ (Eq.~\eqref{eq:Q'_enriched_query}) returned by the function \textsc{enrichQuery} constitutes the input to the final function in Algorithm~\ref{algo:query_comp}, i.e.\ \textsc{optimizeQuery}. The objective of the latter is 
\begin{enumerate}
	\item \label{enum:OQ1} the q-partition-preserving minimization of $Q'$ to obtain a set-minimal subset $Q^{\min}$ of $Q'$. In addition, we postulate that
	\item \label{enum:OQ2} \textsc{optimizeQuery} must yield some $Q^{\min}$ such that $Q^{\min} \cap Q = \emptyset$, if such a $Q^{\min}$ exists. 
	\item \label{enum:OQ3} Otherwise, if axiom preference criteria or axiom fault probabilities $p(\tax)$ for $\tax\in\mo$ are given, then $Q^{\min}$ is required to be the one query that minimizes the maximum probability over all axioms of $Q$ occurring in it. More precisely,
\begin{align}
Q^{\min} := \argmin_{S\in\mathsf{minQPPS}(Q')} (\mathsf{maxP}(S,Q))  \label{eq:Q^min_requirement_iii_section:QueryOptimization}
\end{align} 
where
\begin{align*}
\mathsf{minQPPS}(Q') &:= \setof{X\,|\, X  \text{ is set-minimal q-partition-preserving subset of } Q'} \\
\mathsf{maxP}(X,Q)  &:= \max_{\tax\in X\cap Q} (p(\tax))
\end{align*}
\end{enumerate}
In that, the set-minimality postulated by (\ref{enum:OQ1}.)\ is required to avoid asking the user queries that comprise formulas which are not necessary in order to achieve the desired discrimination properties (which are determined by the q-partition) of the query. An explanation for demanding (\ref{enum:OQ2}.), i.e.\ disjointness between $Q^{\min}$ and $Q$, is that formulas in the KB (those contained in the explicit-entailments query $Q$) are usually more complex in structure and hence for the interacting user more difficult to understand or interpret, respectively, than formulas that correspond to specific simple predefined entailment types such as, e.g.\, $A \rightarrow B$ for atoms $A$ and $B$ in case of Propositional Logic (cf.\ Section~\ref{sec:EnrichmentOfAQuery}). In a situation where requirement (\ref{enum:OQ2}.)\ is not satisfiable, we require that the maximum complexity (represented by the fault probability) of a formula from $Q$ in $Q^{\min}$ is minimal. That is, assuming that all implicit entailments of predefined types are at least as easy to comprehend as the easiest formula in the KB (which is plausible, see above), (\ref{enum:OQ3}.) corresponds to the requirement that the hardest formula in the query $Q^{\min}$ is easiest to understand for the user. 

To realize (\ref{enum:OQ1}.), as already mentioned in Section~\ref{sec:EnrichmentOfAQuery}, we can apply a modified version of the QX algorithm \cite{junker04}. This modified version of QX called \textsc{minQ} is described, proven correct and illustrated by examples in \cite[Chap.~8]{Rodler2015phd} (the \textsc{minQ} algorithm is stated as part of Algorithm~4 on page 103 in \cite{Rodler2015phd}). To make this paper self-contained, we quote the relevant explanatory text regarding \textsc{minQ} from \cite{Rodler2015phd} next.

\paragraph{\textsc{minQ}.}
Like QX (see, e.g.\, \cite[p.~48]{Rodler2015phd}), \textsc{minQ}, described in Algorithm~\ref{algo:optimize_query} starting from line~\ref{algoline:query:minQ_start}, carries out a divide-and-conquer strategy to find a set-minimal set with a monotonic property. In this case, the monotonic property is not the invalidity of a subset of the KB w.r.t.\ a DPI (as per Definition~\ref{def:valid_onto}) as it is for the computation of minimal conflict sets using $\scQX$, but the property of some $Q^{\min}\subset Q$ having the same q-partition as $Q$. So, the crucial difference between $\scQX$ and \textsc{minQ} is the function that checks this monotonic property. For \textsc{minQ}, this function -- that checks a subset of a query for constant q-partition -- is \textsc{isQPartConst} (see Algorithm~\ref{algo:optimize_query}, line~\ref{algoline:query:isQPartConst_start} ff.).	 

\paragraph{\textsc{minQ} -- Input Parameters.}
\textsc{minQ} gets five parameters as input. The first three, namely $X, Q$ and $QB$, are relevant for the divide-and-conquer execution, whereas the last two, namely the original q-partition $\tuple{\dx{}, \dnx{}, \dz{}}$ of the query (i.e.\ the parameter $Q$) that should be minimized, and the DPI $\langle\mo,\mb,\Tp,\Tn\rangle_\RQ$ are both needed as an input to the function \textsc{isQPartConst}. Besides the latter two, another argument $QB$ is passed to this function where $QB$ is a subset of the original query $Q$. \textsc{isQPartConst} then checks whether the q-partition for the (potential) query $QB$ is equal to the q-partition $\tuple{\dx{}, \dnx{}, \dz{}}$ of the original query given as an argument. The DPI is required as the parameters $\mo, \mb,\Tp,\Tn$ and $\RQ$ are necessary for these checks. 

\paragraph{\textsc{minQ} -- Testing Sub-Queries for Constant Q-Partition.}
The function \textsc{isQPartConst} tests for each $\md_r\in\dnx{}$ whether $\mo_r^*\cup QB$ is valid (w.r.t.\ $\tuple{\cdot,\emptyset,\emptyset,N}_\RQ$). If so, this means that $\md_r \notin \dnx{}(QB)$ and thus that the q-partition of $QB$ is different to the one of $Q$ why $\false$ is immediately returned. If $\true$ for all $\md_r\in\dnx{}$, it is tested for $\md_r\in\dz{}$ whether $\mo_r^* \models QB$. If so, this means that $\md_r \notin \dz{}(QB)$ and thus that the q-partition of $QB$ is different from the one of $Q$ why $\false$ is immediately returned. If $\false$ is not returned for any $\md_r\in\dnx{}$ or $\md_r\in\dz{}$, then the conclusion is that $QB$ is a query w.r.t.\ to $\mD = \dx{}\cup\dnx{}\cup\dz{}$ and $\langle\mo,\mb,\Tp,\Tn\rangle_\RQ$ and has the same q-partition as $Q$ why the function returns $\true$.

To check the validity as per Definition~\ref{def:valid_onto}, \textsc{isQPartConst} makes use of the function \textsc{isKBValid} (see Algorithm~\ref{algo:optimize_query}, line~\ref{algoline:query:isKBValid} ff.)\ which directly ``implements'' Definition~\ref{def:valid_onto}. To this end, \textsc{isKBValid} relies on the function \textsc{verifyReq} (line~\ref{algoline:query:call_verifyReq}) which, given a KB $\mo'$ and a set of requirements $\RQ$ in the sense of Definition~\ref{def:dpi}, returns $\true$ iff all requirements $r\in\RQ$ are met for $\mo'$.

\paragraph{\textsc{minQ} -- The Divide-and-Conquer Strategy.}
Intuitively, \textsc{minQ} partitions the given query $Q$ in two parts $Q_1$ and $Q_2$ and first analyzes $Q_2$ while $Q_1$ is part of $QB$ (line~\ref{algoline:query:recursive_call1}). Note that in each iteration $QB$ is the subset of $Q$ that is currently assumed to be part of the sought minimized query (i.e.\ the one query that will finally be output by \textsc{minQ}). In other words, analysis of $Q_2$ while $Q_1$ is part of $QB$ means that all irrelevant formulas in $Q_2$ should be located and removed from $Q_2$ resulting in $Q^{\min}_2 \subseteq Q_2$. That is, $Q^{\min}_2$ must include only relevant formulas which means that $Q^{\min}_2$ along with $QB$ is a query with an equal q-partition as $Q$, but the deletion of any further formula from $Q^{\min}_2$ changes the q-partition.

After the relevant subset $Q^{\min}_2$ of $Q_2$, i.e.\ the subset that is part of the minimized query, has been returned, $Q_1$ is removed from $QB$, $Q^{\min}_2$ is added to $QB$ and $Q_1$ is analyzed for a relevant subset that is part of the minimized query (line~\ref{algoline:query:recursive_call2}). This relevant subset, $Q^{\min}_1$, together with $Q^{\min}_2$, then builds a set-minimal subset of the input $Q$ that is a query and has a q-partition equal to that of $Q$. Note that the argument $X$ of \textsc{minQ} is the subset of $Q$ that has most recently been added to $QB$.

For each call in line~\ref{algoline:query:recursive_call1} or line~\ref{algoline:query:recursive_call2}, the input $Q$ to \textsc{minQ} is recursively analyzed until a trivial case arises, i.e.\ (a)~until $Q$ is identified to be irrelevant for the computed minimized query wherefore $\emptyset$ is returned (lines~\ref{algoline:query:validitytest2} and \ref{algoline:query:return_emptyset}) or (b)~until $|Q|=1$ and $Q$ is not irrelevant for the computed minimized query wherefore $Q$ is returned (lines~\ref{algoline:query:test_singleton} and \ref{algoline:query:return_Q}).

\paragraph{\textsc{optimizeQuery}.} While serving as a procedure to solve requirement (\ref{enum:OQ1}.), \textsc{minQ} as described in \cite{Rodler2015phd} does not yet satisfy requirements (\ref{enum:OQ2}.)\ and (\ref{enum:OQ3}.)\ from the beginning of this section. As we will learn next, we can accomplish (\ref{enum:OQ2}.)\ and (\ref{enum:OQ3}.)\ by using an appropriate order of formulas in $Q'$. Bringing $Q'$ into this order is the task of the \textsc{sort} function which constitutes the first step in the \textsc{optimizeQuery} function (Algorithm~\ref{algo:optimize_query}). The \textsc{sort} function (line~\ref{algoline:query:sort}) effectuates that $Q'$ (which is the union of the implicit entailments $Q_{impl}$ and the explicit ones $Q$, cf.\ Section~\ref{sec:EnrichmentOfAQuery}) is sorted as 
\begin{align}
[Q_{impl},asc_{\FP}(Q)] \label{eq:sorting_for_minQ}
\end{align}
where $[X,Y]$ denotes a list containing first (i.e.\ leftmost) all elements of the collection $X$ and then all elements of the collection $Y$, and $asc_{crit}(X)$ refers to the list of elements in the collection $X$ sorted ascending based on $crit$. In our case $crit := \FP$ which means that given fault probabilities (which constitute or, respectively, can be used to compute formula fault probabilities) are used to sort $Q$ ascending by formula probabilities.

The second and at the same time final step of \textsc{optimizeQuery} is a call to \textsc{minQ} to which $Q'$, sorted as per Eq.~\eqref{eq:sorting_for_minQ}, is passed as an argument. The output of this call to \textsc{minQ} is then directly returned by Algorithm~\ref{algo:optimize_query}.

Finally, we prove that the sorting according to Eq.~\eqref{eq:sorting_for_minQ} indeed yields the fulfillment of postulations (\ref{enum:OQ2}.)\ and (\ref{enum:OQ3}.)\ for the query $Q^{\min}$ returned by \textsc{optimizeQuery}. The next proposition is easy, but rather elaborate to prove. Hence we omit the proof here.
\begin{proposition}\label{prop:QX_returns_the_one_subset_with_leftmost_rightmost_element}
Let $prop$ be some monotonic property, $X = [x_1,\dots,x_k]$ be a sorted list and $\mathbf{X}_{sub}$ denote all set-minimal subsets of $X$ which are equivalent to $X$ w.r.t.\ $prop$. Then, used to find a set-minimal subset of $X$ which is equivalent to $X$ w.r.t.\ $prop$, \textsc{QX} (and thus \textsc{minQ}) always returns the set-minimal subset $X_{sub}$ of $X$ with the leftmost rightmost element among all elements in $\mathbf{X}_{sub}$.   
%
Formally:
\begin{align}
X_{sub} = \argmin_{X \in \mathbf{X}_{sub}} \left(\max_{x_i \in X}(i)\right)
\end{align}
\end{proposition}
The next corollary proves the compliance of $\textsc{minQ}$ used with a sorted input query $Q'$ as per Eq.~\eqref{eq:sorting_for_minQ} with the requirements to the optimized output query $Q^{\min}$ stated at the beginning of this section.
\begin{corollary}\label{cor:using_minQ_with_sorting_as_given_by_eq-sorting-for-minQ_finds_query_meeting_properties_i_ii_iii}
Application of $\textsc{minQ}$ to the enriched query $Q' = [x_1,\dots,x_k]$ with q-partition $\Pt = \tuple{\dx{},\dnx{},\emptyset}$ sorted as per Eq.~\eqref{eq:sorting_for_minQ} returns a query $Q^{\min}$ with q-partition $\Pt$ compliant with requirements (\ref{enum:OQ1}.)\ - (\ref{enum:OQ3}.)\ given at the beginning of this section.
\end{corollary}
\begin{proof}
Requirement (\ref{enum:OQ1}.)\ follows directly from \cite[Prop.~8.7]{Rodler2015phd}. To show that requirement (\ref{enum:OQ2}.)\ must hold for the returned query $Q^{\min}$, we must demonstrate the following: If there is some set-minimal query $Q_j \subseteq Q'$ satisfying $Q_j \cap Q = \emptyset$, then $Q^{\min} \cap Q = \emptyset$. So, let us assume that such a $Q_j$ exists. Then, $Q_{impl} \neq \emptyset$ must hold by $Q' = Q_{impl} \cup Q$ and since $Q_j$ is a query why $Q_j \neq \emptyset$. Moreover, $Q \neq \emptyset$ due to Proposition~\ref{prop:properties_of_q-partitions},(\ref{prop:properties_of_q-partitions:enum:for_each_q-partition_dx_is_empty_and_dnx_is_empty}.)\ which states that $\dnx{}$ must not be the empty set and due to Proposition~\ref{prop:explicit-ents_query_lower+upper_bound} which states that $Q$ must be a hitting set of all diagnoses in $\dnx{}$. Taking into account the sorting of $Q' = [Q_{impl},asc_{\FP}(Q)]$ according to Eq.~\eqref{eq:sorting_for_minQ}, we can infer that there is some $r$ satisfying $1 < r \leq k$ such that $x_r$ is the formula with least index to belong to $Q$. Now, we can deduce that $\max_{x_i \in Q_j}(i) < r$. Let $\mathbf{X}_{sub}$ denote all set-minimal queries that are a subset of $Q'$. Then, by Proposition~\ref{prop:QX_returns_the_one_subset_with_leftmost_rightmost_element}, the returned query $Q^{\min}$ is equal to $\argmin_{X \in \mathbf{X}_{sub}} (\max_{x_i \in X}(i))$. Hence, we have that $\max_{x_i \in Q^{\min}}(i) \leq \max_{x_i \in Q_j}(i) < r$ which is why $Q^{\min} \cap Q = \emptyset$ must be true.

To show that requirement (\ref{enum:OQ3}.)\ is met for the returned query $Q^{\min}$, we must demonstrate the following: If there is no set-minimal query $Q_j \subseteq Q'$ satisfying $Q_j \cap Q = \emptyset$, then Eq.~\eqref{eq:Q^min_requirement_iii_section:QueryOptimization} must hold for $Q^{\min}$, 
i.e.\ $Q^{\min}$ is required to be the one query that minimizes the maximum probability over all axioms of $Q$ occurring in it. 
Now, since no query $Q_j \subseteq Q'$ satisfying $Q_j \cap Q = \emptyset$ exists, we obtain that $Q^{\min} \cap Q \neq \emptyset$. The truth of Eq.~\eqref{eq:Q^min_requirement_iii_section:QueryOptimization} is now a direct consequence of Proposition~\ref{prop:QX_returns_the_one_subset_with_leftmost_rightmost_element} and the sorting of $Q$ in ascending order of formula probabilities.
\end{proof}

\begin{algorithm*}
\small
\caption{Query Optimization \normalsize} \label{algo:optimize_query}
\begin{algorithmic}[1]
\Require a DPI $\tuple{\mo,\mb,\Tp,\Tn}_\RQ$, a query $Q'$ w.r.t.\ $\tuple{\mo,\mb,\Tp,\Tn}_\RQ$ for which a set-minimal subset with equal q-partition should be found, explicit-entailments query $Q$ w.r.t.\ $\tuple{\mo,\mb,\Tp,\Tn}_\RQ$ such that $Q\subseteq Q'$, q-partition $\Pt$ with associated queries $Q$ and $Q'$, fault information $\FP$ 
\Ensure a set-minimal subset $Q^{\min}$ of $Q'$ such that $Q^{\min}$ satisfies requirements (\ref{enum:OQ2}.)\ and (\ref{enum:OQ3}.)\ stated at the beginning of Section~\ref{sec:QueryOptimization}
\Procedure{$\textsc{optimizeQuery}$}{$Q',Q,\Pt,\langle\mo,\mb,\Tp,\Tn\rangle_\RQ,\FP$}   \label{algoline:query_optimizeQuery_start}
	\State $Q' \gets \Call{sort}{Q',Q,\FP}$    \label{algoline:query:sort}
	\State \Return $\Call{minQ}{\emptyset,Q',\emptyset,\Pt,\langle\mo,\mb,\Tp,\Tn\rangle_\RQ}$   \label{algoline:query:call_minQ}
\EndProcedure
\vspace{2pt}
\hrule
\vspace{2pt}
\Procedure{$\textsc{minQ}$}{$X,Q,QB,\tuple{\dx{}, \dnx{}, \dz{}},\langle\mo,\mb,\Tp,\Tn\rangle_\RQ$}   \label{algoline:query:minQ_start}
\If{$\ X \neq \emptyset \land \Call{isQPartConst}{QB,\tuple{\dx{}, \dnx{}, \dz{}}, \tuple{\mo,\mb,\Tp,\Tn}_\RQ}$}  \label{algoline:query:validitytest2}  
	\State \Return $\emptyset$    \label{algoline:query:return_emptyset}
\EndIf
\If{$|Q| = 1$}  \label{algoline:query:test_singleton}              
  \State \Return $Q$ \label{algoline:query:return_Q}
\EndIf
\State $k \gets \Call{split}{|Q|}$     \label{algoline:query:split}
\State $Q_1 \gets \Call{get}{Q, 1, k}$
\State $Q_2 \gets \Call{get}{Q, k + 1, |Q|}$
\State $Q^{\min}_2 \gets \Call{\textsc{minQ}}{Q_1,Q_2,QB\cup Q_1,\tuple{\dx{}, \dnx{}, \dz{}},\langle\mo,\mb,\Tp,\Tn\rangle_\RQ}$ \label{algoline:query:recursive_call1}
\State $Q^{\min}_1 \gets \Call{\textsc{minQ}}{Q^{\min}_2,Q_1,QB\cup Q^{\min}_2	,\tuple{\dx{}, \dnx{}, \dz{}},\langle\mo,\mb,\Tp,\Tn\rangle_\RQ}$ \label{algoline:query:recursive_call2}
\State \Return $Q^{\min}_1 \cup Q^{\min}_2$  
\EndProcedure
\vspace{2pt}
\hrule
\vspace{2pt}
\Procedure{\textsc{isQPartConst}}{$Q,\tuple{\dx{}, \dnx{}, \dz{}}, \tuple{\mo,\mb,\Tp,\Tn}_\RQ$}  \label{algoline:query:isQPartConst_start}
\For{$\md_r \in \dnx{}$} 
			\If{$\Call{isKBValid}{\mo^{*}_r \cup Q, \tuple{\cdot,\emptyset,\emptyset,\Tn}_{\RQ}}$} 
					\State \Return \false
			\EndIf
\EndFor
\For{$\md_r \in \dz{}$} 
			\If{$\mo^{*}_r \models Q$}
					\State \Return \false
			\EndIf
\EndFor
\State \Return \true
\EndProcedure
\vspace{2pt}
\hrule
\vspace{2pt}
\Procedure{\textsc{isKBValid}}{$\mo, \langle\cdot,\mb,\Tp,\Tn\rangle_\RQ$}  \label{algoline:query:isKBValid}
\State $\mo' \gets \mo \cup \mb \cup \bigcup_{\tp\in\Tp} \tp$
\If{$\neg \Call{verifyReq}{\mo',\RQ}$}   \label{algoline:query:call_verifyReq}
	\State \Return \false
\EndIf
\For{$\tn \in \Tn$}
\If{$\mo'\models \tn$}
\State \Return \false
\EndIf
\EndFor
\State \Return \true
\EndProcedure
\end{algorithmic}
\normalsize
\end{algorithm*}

\clearpage
\section{Summary and Conclusion}
\label{sec:Conclusion}
In this work, we have tackled the problem of efficient query computation in interactive knowledge base debugging. The performance and scalability of query generation algorithms has a direct and material impact on the reaction time of the debugger (i.e.\ the time between the submission of an answered query until the appearance of the next one), alongside the time consumed by the (re-)calculation of candidate KB repairs which are exploited for the determination of query candidates. 

Roughly, the structure of this work can be categorized into five parts:
\begin{enumerate}
	\item \emph{Analysis of active learning query quality measures:} We studied a variety of active learning measures, e.g.\ information gain, various manifestations of entropy measures or the Gini Index, that have been proposed in literature and a reinforcement learning system with a focus on their use in the scope of an interactive debugger. We determined the global optima, if existent, for all these quantitative measures, i.e.\ real-valued functions, and used these to extract a set of qualitative requirements a query must meet in order to optimize the respective measure. Further, we characterized a plausible preference order on queries, the discrimination-preference order, which indicates which queries are generally preferable to others. And we defined different equivalence relations over measures whose equivalence classes give information about for which classes of queries which measures lead to the same query selection behavior. Beyond that, we defined a superiority relation over query quality measures providing a preference criterion for measures w.r.t.\ the degree of their fulfillment of the discrimination-preference relation. 
	
What we found out is that only four of almost $20$ discussed measures give us a guarantee to always give preference to the discrimination-preferred queries. These four are superior to all others. Another two measures are at least consistent with the discrimination-preference relation, whereas the rest is proven to generally violate this relation. Interestingly, the two measures that have already been used for KB debugging, belong to the latter two groups. However, for both measures, we developed slightly modified alternatives that grant fulfillment of the discrimination-preference relation. We gained insight that, up to equivalence, $15$ of the investigated measures are different. For a restricted class of queries where all query pairs are indistinguishable in terms of discrimination-preference, $7$ differently behaving measures remain up to equivalence over this restricted query space.
	
	It could be demonstrated that a global optimum exists for all but two of the examined measures. From these global optima, sufficient and necessary requirements could be deduced. In case of the two measures having no optimal values, we were able to at least give a set of necessary criteria of the best query among a concrete set of queries. Such necessary criteria are helpful to restrict the search space for optimal queries. 
	
	We also introduced several parametrized measures constituting generalizations of existing measures and discussed their properties w.r.t.\ different used parameters. The results we obtained in this first part of the work are general and apply to any active learning scenario where binary queries are used, e.g.\ when the learner tries to find a binary classifier, and where each query partitions the hypothesis space into two (three) parts, one consistent only with the positive answer, the second consistent only with the negative answer (and the third consistent with both answers).
	\item \emph{Heuristic search for optimal q-partitions:} In this part we exploited the qualitative requirements for active learning measures derived in the first part to devise a heuristic search method for the best q-partition. A q-partition 
	is a partition of the set of candidate KB repairs into three subsets that
	incorporates \emph{all} relevant information about an associated query regarding active learning, e.g.\ the possible elimination rate of repair candidates or the probabilities of query answers.
	
	Since (1) not each partition of the KB repair candidates is a q-partition, (2) the verification whether a partition is a q-partition requires a query, and (3) there can be exponentially many different queries for one and the same q-partition, we introduced the notion of a canonical query, a well-defined, easily computable and unique query for a q-partition. The most notable advantage of a canonical query is that it enables the computation of queries completely \emph{without any expensive calls of a logical reasoner}. Moreover, it ensures that no partitions which are no q-partitions (i.e.\ for which no query exists) can ever be considered in the search. Existing methods have to struggle with both of these issues, i.e.\ they rely on reasoner calls to figure out whether a candidate is actually a q-partition (query) and might generate many unnecessary partitions for which queries do not exist.
	
	The search for q-partitions itself can be seen as a depth-first, ``local best-first'' backtracking search. At each step, the best successor q-partition resulting from the current one by minimal changes is determined by means of heuristics based on the given qualitative requirements. The latter are also used to specify search tree pruning conditions that avoid the useless exploration of worse q-partitions than the currently best one.
	 
	Once a (sufficiently near to) optimal q-partition is detected, all queries (including the canonical one) associated with this q-partition have exactly the same quality w.r.t.\ the used query quality measure.
	%
	\item \emph{Search for the best query w.r.t.\ an optimal q-partition:} The goal of this part was to suggest a method of locating the best query w.r.t.\ the single optimal q-partition filtered out in the previous part. Theoretical analyses yielded that each set-minimal query associated with a given q-partition is a set-minimal hitting set of well-defined subsets, called traits, of all diagnoses in one part of the q-partition, namely the part comprising exactly the diagnoses that are inconsistent with the positive query answer. Hence, we explicated how such a hitting set tree search can be implemented and pointed out the crucial difference to most other scenarios where this kind of search is adopted, namely that all sets that must be ``hit'' are known in advance. Hence, unlike the usual hitting set tree applications such as for the computation of KB repair candidates where the computation of the next set to be ``hit'' is very expensive (involving e.g.\ reasoner calls), in our scenario the construction of the tree does not involve any expensive operations at all. 
	
	In addition to that, the fact that a hitting set tree can be employed for query extraction means that set-minimal queries can be generated in best-first order according to a desired criterion such as minimal query cardinality or best comprehensibility of the query for the interacting user. Existing approaches lack such a feature and provide only means to compute \emph{any} set-minimal query. We point out that the entire query computation given a q-partition is accomplished \emph{without a single call to a logical reasoning service}. 
	
	\emph{Consequently, the proposed algorithms are the first methods that enable a completely reasoner-free query generation for interactive KB debugging while at the same time guaranteeing optimality conditions of the generated query that existing methods cannot realize.}
	\item \emph{Enrichment of the optimal query:} Optionally, our system provides the opportunity to enrich the single best query $Q$ computed in the previous part with logical sentences of very simple predefined form, e.g.\ for Propositional Logic this form might be simple implications $A \to B$ with a single atom on both sides. The motivation to use this enrichment might be to simplify the query allowing for a better understanding of the interacting user.	
	We formulated several requirements such a query enrichment operation must fulfill, e.g.\ the invariance of the q-partition. The reason is simply that the q-partition is already optimal and fixed and should of course not be altered as this would affect the quality of the query w.r.t.\ the used query quality measure (negatively). Another postulated criterion is the addition of exclusively such logical sentences which are ``logically dependent'' on $Q$. This is necessary to avoid the addition of just any logical statements and should restrict the number of such possible statements. This query enrichment step requires exactly two reasoner calls. 
	\item \emph{Optimization of the enriched query:} Given an enriched query, this part aimed at describing an approach to minimize this query again. However, just any minimization clearly will not do, as this would call into question the enrichment step before. On the contrary, we again had to impose conditions on this minimization step. In particular, we postulated that, under plausible assumptions, the difficulty of answering the resulting minimized query should be minimal among all possible query minimizations. Roughly, this involves that a minimum of the logical sentences that do not have the simple form of the enrichment sentences, should be eliminated from the enriched query. Furthermore, if there is no minimized query with all such sentences having been eliminated, then, simply put, these sentences should be as easily understandable as possible. By means of a slight modification of a well-known existing strategy for query minimization, we could devise a method that achieves exactly this. The crucial point in the performed modification is an appropriate ordering of the logical sentences in the query. This query minimization and optimization step requires only a polynomial number of reasoner calls.
\end{enumerate}
\clearpage

\bibliographystyle{splncs03}
\bibliography{library}
\end{document}